\RequirePackage{fix-cm}
\documentclass[smallextended]{svjour3}       
\smartqed  
\usepackage{mathtools}
\usepackage{amssymb}

\usepackage{graphicx}
\usepackage{subcaption}
\usepackage{multirow}
\usepackage{tabularx}
\usepackage{tabulary}
\usepackage{booktabs}

\usepackage[numbers,sort&compress]{natbib}
\usepackage{url}

\usepackage{xcolor}
\usepackage[ruled,vlined]{algorithm2e}

\DeclareMathOperator*{\argmin}{arg\,min}

\usepackage{hyperref}
\begin{document}

\title{An ICTM–RMSAV Framework for Bias-Field Aware Image Segmentation under Poisson and Multiplicative Noise 
}


\author{Xinyu Wang       \and
        Wenjun Yao   \and
        Fanghui Song* \and
        Zhichang Guo
}


\institute{Xinyu Wang \at
              School of Mathematics, Harbin Institute of Technology, Harbin 150006, China \\
              \email{25b312002@stu.hit.edu.cn}           
           \and
           Wenjun Yao \at
              School of Mathematics, Harbin Institute of Technology, Harbin 150006, China \\
              \email{mathywj@hit.edu.cn}
           \and
           Fanghui Song*(Corresponding Author) \at
              School of Mathematics, Harbin Institute of Technology, Harbin 150006, China \\
              \email{fanghuisong@stu.hit.edu.cn}
           \and
           Zhichang Guo \at
              School of Mathematics, Harbin Institute of Technology, Harbin 150006, China \\
              \email{mathgzc@hit.edu.cn}
}

\date{Received: date / Accepted: date}

\maketitle

\begin{abstract}
 Image segmentation is a core task in image processing, yet many methods degrade when images are heavily corrupted by noise and exhibit intensity inhomogeneity. Within the iterative-convolution thresholding method (ICTM) framework, we propose a variational segmentation model that integrates denoising terms. Specifically, the denoising component consists of an I-divergence term and an adaptive total-variation (TV) regularizer, making the model well suited to images contaminated by Gamma--distributed multiplicative noise and Poisson noise. A spatially adaptive weight derived from a gray-level indicator guides diffusion differently across regions of varying intensity. To further address intensity inhomogeneity, we estimate a smoothly varying bias field, which improves segmentation accuracy. Regions are represented by characteristic functions, with contour length encoded accordingly. For efficient optimization, we couple ICTM with a relaxed modified scalar auxiliary variable (RMSAV) scheme. Extensive experiments on synthetic and real-world images with intensity inhomogeneity and diverse noise types show that the proposed model achieves superior accuracy and robustness compared with competing approaches.
\end{abstract}

\section{Introduction}
As an important part of image processing, image segmentation partitions an image into disjoint regions corresponding to objects of interest. It is a useful tool for computer vision \cite{ liu2024rotated}, medical image analysis \cite{wali2023level}, autonomous driving \cite{papadeas2021real} and satellite remote sensing \cite{li2024review}. However, severe intensity inhomogeneity and complex noise often occur in images. Due to illumination variations and inherent limitations of imaging devices, many images exhibit intensity inhomogeneity, which results in blurred object boundaries.
In addition, various noise sources are inevitably introduced during image acquisition and transmission; for example, Poisson noise and multiplicative Gamma noise frequently arise in positron emission tomography (PET) \cite{vardi1985statistical}, X-ray computed tomography (CT) \cite{ding2018statistical}, ultrasound imaging \cite{wagner1983statistics}, and synthetic aperture radar (SAR) images \cite{lee1986speckle}. Consequently, segmenting images affected by strong intensity inhomogeneity and high noise remains a significant challenge.

 Various image segmentation methods have been proposed, including active contour models (ACMs) \cite{chen2023overview}, graph-based methods \cite{chen2018survey}, wavelet-based methods \cite{gao2020wavelet}, thresholding methods \cite{jardim2023image}, and clustering-based methods \cite{mittal2022comprehensive}. Among these traditional segmentation approaches, ACMs have been widely used due to their ability to accurately capture object boundaries. These models construct energy functionals using image features such as gradients and intensity, converting contour evolution into an energy minimization problem. One of the earliest ACMs is the Snakes model \cite{kass1988snakes} proposed by Kass et al. in 1988. The energy functional consists of internal and external energy, and the minimizer of the energy approximates the boundary of the target object.
Following the Snake model, many parametric active contour models were subsequently introduced \cite{cohen1991active,xu1998snakes,xu1998generalized}.
However, such models have low segmentation accuracy and cannot effectively handle topological changes such as object splitting or merging.

To overcome the above topological limitations, Osher and Sethian \cite{osher1988fronts} introduced the level set method, which was subsequently incorporated into ACMs by Caselles et al. \cite{caselles1997geodesic}. In level set–based ACMs, evolving contours are represented implicitly as the zero level set of a level set function (LSF), and their motion toward object boundaries is driven by minimizing an appropriate energy. In general, ACMs can be categorized into edge-based and region-based models. The geodesic active contour (GAC) model \cite{caselles1997geodesic} is a classical edge-based approach that uses image gradients to build an edge detector; Li et al. \cite{li2010distance} later proposed a distance-regularized level set evolution (DRLSE) model to avoid reinitializing the LSF by adding a penalty term. Because edge-based models rely primarily on gradient information, they are sensitive to noise and often degrade on images with weak or missing edges.

In contrast, region-based models are better suited to images with weak edges. The Mumford–Shah (MS) model \cite{mumford1989optimal} is highly influential, but its nonconvexity leads to low computational efficiency. The Chan–Vese (CV) model \cite{chan2001active} approximates images as piecewise constant and is robust to initialization, yet it fails under intensity inhomogeneity. To address this, Li et al. first proposed the local binary fitting (LBF) model \cite{li2007implicit}, assuming locally constant intensities, and later proposed the local intensity clustering (LIC) model incorporating multiplicative bias-field estimation \cite{li2011level}. Subsequent work introduced additive \cite{weng2021level} and hybrid bias corrections \cite{pang2023image}. These level-set formulations handle topological changes and improve segmentation with inhomogeneous intensities; however, most bias-field models assume Gaussian noise, limiting accuracy on images affected by complex noise.

Unlike additive noise, Poisson and multiplicative Gamma noise are intensity-dependent, causing stronger degradation and aggravating intensity inhomogeneity. To mitigate noise, Cai et al. \cite{cai2013two} proposed a two-stage scheme--convex denoising or deblurring followed by thresholding--which is effective for Gaussian noise with blur. Chan et al. \cite{chan2014two} extended this to blurred images with multiplicative and Poisson noise, which enables easy multi-phase segmentation tends to degrade under strong inhomogeneity.
Although two-stage methods have demonstrated effectiveness in segmentation, their practical application necessitates careful assessment of whether the first-stage outcome is beneficial to the performance of the second-stage model. As a result, recent research trends are moving from denoise-then-segment toward unified variational formulations that jointly incorporate denoising and segmentation terms.
In joint variational/level-set methods, the VLSGIS model proposed by Ali et al. \cite{ali2018image} combines two denoising terms to handle inhomogeneity and noise, but some subtle object details are not well preserved in its results; the AVLSM model by Cai et al. \cite{cai2021avlsm} employs an adaptive-scale bias field with a denoising term. Zhang et al. \cite{zhang2022rvlsm} map data into a diffusion-induced space to suppress noise while enhancing details; Kumar et al. \cite{kumar2024ecdm} couple a despeckling equation with bias-field correction (ECDM) to preserve boundaries under speckle and inhomogeneity; more recently, Zhou et al. \cite{zhou2025variational} introduced a variational model with a nonlinear transformation, achieving simultaneous segmentation and denoising for multiplicative Gamma noise.

Recent years have seen growing interest in segmenting images degraded by noise and intensity inhomogeneity. Wang et al. \cite{wang2022iterative} proposed a novel framework applicable to various active contour models for image segmentation. In this framework, different segments are represented by their characteristic functions, and the regularization term is formulated as a concave functional of these characteristic functions. Based on this framework, an iterative convolution-thresholding method (ICTM) was proposed to minimize the energy functional, providing a faster approach compared with level set methods. However, this method performs poorly when applied to noise-sensitive models. To handle this problem, Liu et al. \cite{liu2023active} proposed an improved energy functional incorporating a local variance force term. The energy functional was solved using a modification ICTM. However, the model yields unsatisfactory on images with severe intensity inhomogeneity. Similarly, Hsieh et al. \cite{hsieh2024efficient} integrated a denoising mechanism with the LIC model and employed an efficient ICTM to solve the minimization problem.

In this paper, we propose a variational model incorporating denoising terms to solve the problem of segmenting images degraded by noise and intensity inhomogeneity. The model employs the I--divergence for noise removal and a bias field to correct intensity inhomogeneity, thereby improving segmentation robustness under challenging imaging conditions. Specifically, our main contributions are as follows.
\begin{itemize}
  \item We employ a bias field to correct images with intensity inhomogeneity and use characteristic functions to represent regions and approximate contour length. This approach enhances the robustness of severe intensity inhomogeneity and noise. Representing regions with characteristic functions eliminates the need for additional penalty terms or re-initialization.
  \item The I-divergence and adaptive total variation (TV) regularization are incorporated into the segmentation model as denoising terms. Since images exhibit intensity inhomogeneity and noise is signal-dependent, a gray level indicator is introduced as an adaptive weight in the TV regularization term. This enables adaptive denoising and better preservation of image details, thereby improving segmentation accuracy.
  \item We employ the RMSAV algorithm and ICTM to efficiently solve the model. In addition, we conduct segmentation experiments on various images with noise and intensity inhomogeneity, as well as real medical images, demonstrating the effectiveness of our model.
\end{itemize}

The paper is organized as follows. In Sect.~\ref{sec:related}, a brief introduction to classical image segmentation models is provided. In Sect.~\ref{sec:proposed}, we elucidate the motivation behind our model and propose a new active contour model. In Sect.~\ref{sec:algorithms}, the ICTM and RMSAV algorithms are employed to obtain the minimizer of the proposed model. In Sect.~\ref{sec:experiments}, numerical experiments are given to validate the effectiveness and segmentation accuracy of our model. Finally, conclusions are given in Sect.~\ref{sec:conclusions}.

\section{Related Work}
\label{sec:related}
In this section, we introduce some representative region-based ACMs. Let $\Omega \subset \mathbb{R}^2 $ be the image domain, and $ f:\Omega \rightarrow \mathbb{R}^d $ be an input image, where $d=1$ represents the grayscale image and $d=3$ represents the color image. The goal of image segmentation is to partition $\Omega$ into $n$ disjoint regions, denoted as $\Omega_1,\Omega_2,\dots,\Omega_n$, that is, $\Omega=\displaystyle \bigcup\limits_{i=1}^n \Omega_i$ and $\Omega_i \cap \Omega_j=\emptyset$ for $i\neq j$.
\subsection{The CV Model and LIC Model}
The CV model \cite{chan2001active} assumes that the image intensity is constant in different regions and formulates the $n$-phase segmentation problem as follows:
\begin{equation}\label{eq:CV}
  E_{\text{CV}}(\Omega_1,\dots, \Omega_n,\mathbf{c})=\sum_{i=1}^n\int_{\Omega_i} \lambda_i|f(x)-c_i|^2 \,dx + \mu \sum_{i=1}^n |\partial \Omega_i| ,
\end{equation}
where $\lambda_i$ and $\mu$ are parameters. $\partial \Omega_i$ is the boundary of $\Omega_i$, and $|\partial \Omega_i|$ represents the length of the boundary. $\mathbf{c}=(c_1,\dots,c_n)$ is a vector, where each $c_i$ approximates the mean intensity of the image in $\Omega_i$. The CV model can effectively segment images with blurred boundaries. However, when the image intensity in a region is not constant, the CV model fails to produce accurate segmentation results.

To address the issue of intensity inhomogeneity in images, Horn et al. \cite{horn1974determining} proposed a classical multiplicative model of intensity inhomogeneity, expressed as $f(x)=b(x)J(x)+n(x)$. In this model, $f(x)$ is an observed image, $J(x)$ is the real image (bias-free image), $b(x)$ is the bias field that accounts for intensity inhomogeneity and $n(x)$ is additive noise. Based on this model, Li et al. \cite{li2011level} proposed the LIC model, which assumes that the bias field $b$ varies slowly. The energy functional of the LIC model is defined as
\begin{equation}\label{eq:LIC}
\begin{split}
E_{\text{LIC}}(\Omega_1,\dots, \Omega_n,b,\mathbf{c})
= &\sum_{i=1}^n\lambda_i\int_{\Omega} \int_{\Omega_i} G_\rho(y-x)|f(y)-b(x)c_i|^2 \,dydx \\
&+ \mu \sum_{i=1}^n |\partial \Omega_i| ,
\end{split}
\end{equation}
where $G_\rho$ denotes a Gaussian kernel function. $\mathbf{c}=(c_1,\dots,c_n)$ is a vector, where each $c_i$ approximates the mean intensity of the real image $J(x)$ in $\Omega_i$. $\lambda_i$ and $\mu$ are parameters. Using a multiplicative model and clustering criteria, the LIC model simultaneously estimates the bias field $b(x)$ and segments the image, effectively handling intensity inhomogeneity.

\subsection{The General Model of ICTM}
To enhance the computational efficiency of variational models, Wang et al. \cite{wang2022iterative} introduced a novel approximation for contour length representation based on characteristic functions and Gaussian kernels. Specifically, the boundary length $|\partial \Omega_i|$ can be approximated as
\begin{equation}
  |\partial \Omega_i \cap \partial \Omega_j |\approx \sqrt{\frac{\pi}{\tau}}\int_{\Omega}u_iG_{\tau}*u_j\,dx ,
\end{equation}
where $u_i$ is the characteristic function of the region $\Omega_i$, $*$ represents convolution, i.e., $G_{\tau}*u_j(x)=\int G_{\tau}(x-y)u_j(y)dy $, with $G_{\tau}$ being the Gaussian kernel. This variant approach, based on the heat kernel, has been shown \cite{ledoux1994semigroup,pallara2007short} to provide an accurate estimate of boundary length.

Based on this length representation, they proposed a new framework applicable to most segmentation models, expressed as
 \begin{equation}\label{eq:ICTM}
 \begin{split}
   E_{ICTM}(\mathbf{u}, \mathbf{\Theta_1}, \dots, \mathbf{\Theta_n})
   = & \sum_{i=1}^{n} \lambda_i \int_{\Omega} u_i F_i (f, \mathbf{\Theta_1}, \dots, \mathbf{\Theta_n}) \,dx \\
   & + \mu \sum_{i=1}^{n} \sum_{j=1,j \neq i}^{n} \sqrt{\frac{\pi}{\tau}} \int_{\Omega} u_i G_{\tau} * u_j \,dx ,
 \end{split}
\end{equation}
where $\mathbf{u}=(u_1,\dots,u_n)$  is a vector of characteristic functions, and each $u_i$ represents the characteristic function of $\Omega_i$. $\mathbf{\Theta_i}= ( \Theta_{i,1}, \dots, \Theta_{i,m} )$ denotes all variables or functions of the data fitting terms. To solve the problem, the authors proposed the ICTM, which features high computational efficiency and unconditional energy descent.

\section{The Proposed Model}
\label{sec:proposed}
In this section, we first explain the motivation behind our model. Subsequently, to overcome the insufficient robustness of existing active contour models to noise and intensity inhomogeneity, we propose a novel variational model, aiming to achieve more accurate segmentation results.

\subsection{Motivation}
\label{subsec:motivation}

\begin{figure}
  \centering
  \begin{subfigure}[b]{0.16\linewidth}
      \centering
      \includegraphics[width=\textwidth]{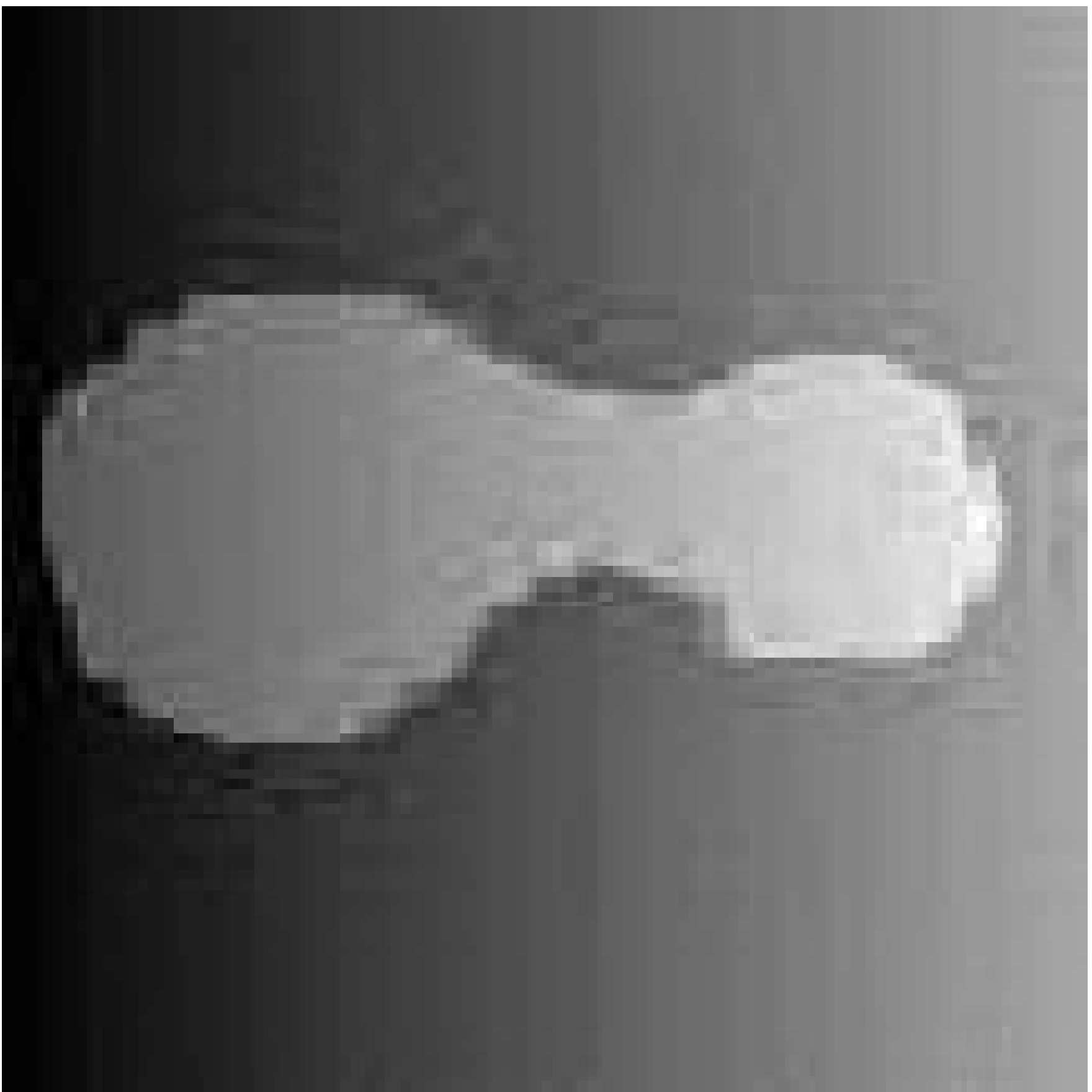}
      \caption{}
      \label{fig:1-origianl}
  \end{subfigure}
  \hfill
  \begin{subfigure}[b]{0.16\linewidth}
      \centering
      \includegraphics[width=\textwidth]{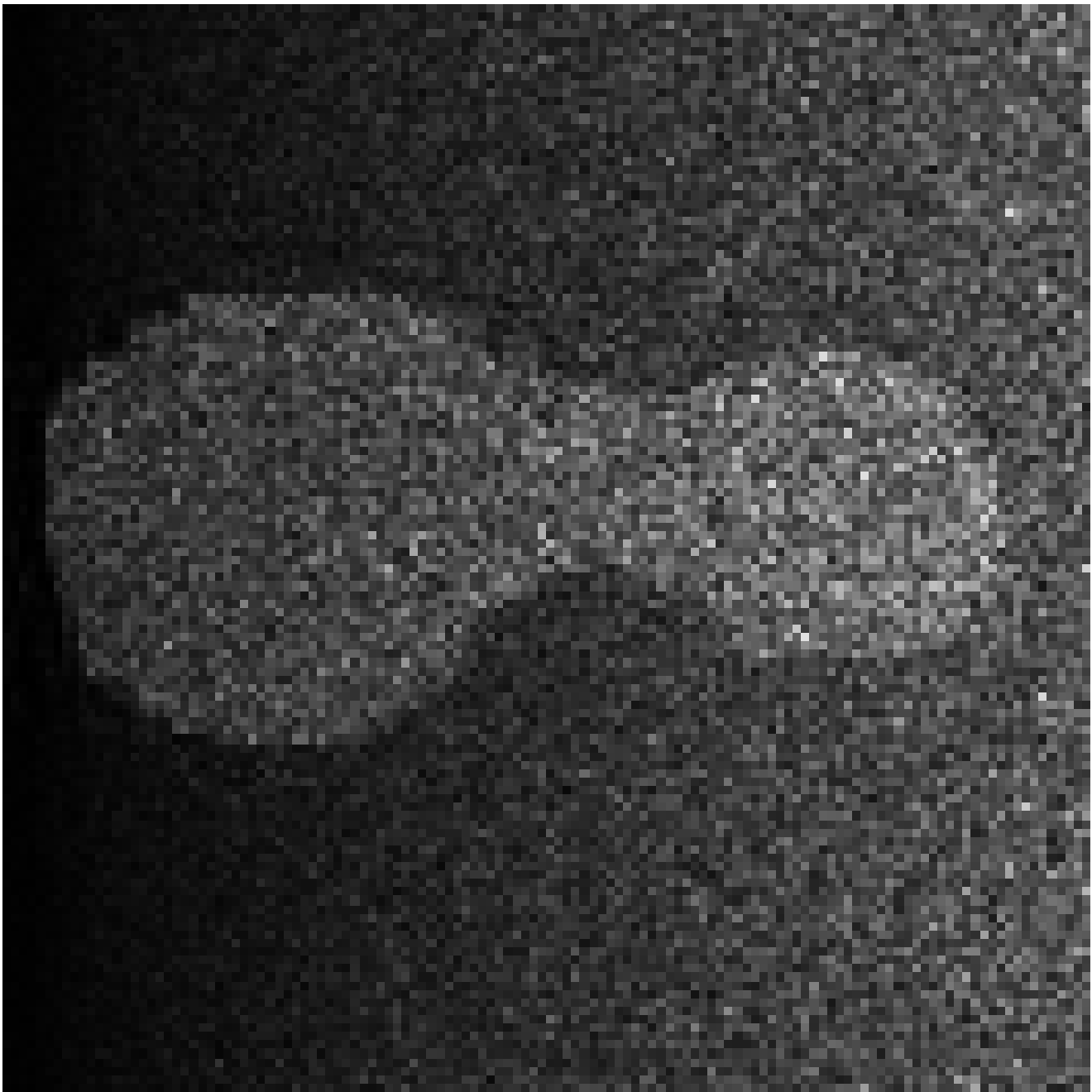}
      \caption{}
      \label{fig:1-L10}
  \end{subfigure}
  \hfill
    \begin{subfigure}[b]{0.16\linewidth}
      \centering
      \includegraphics[width=\textwidth]{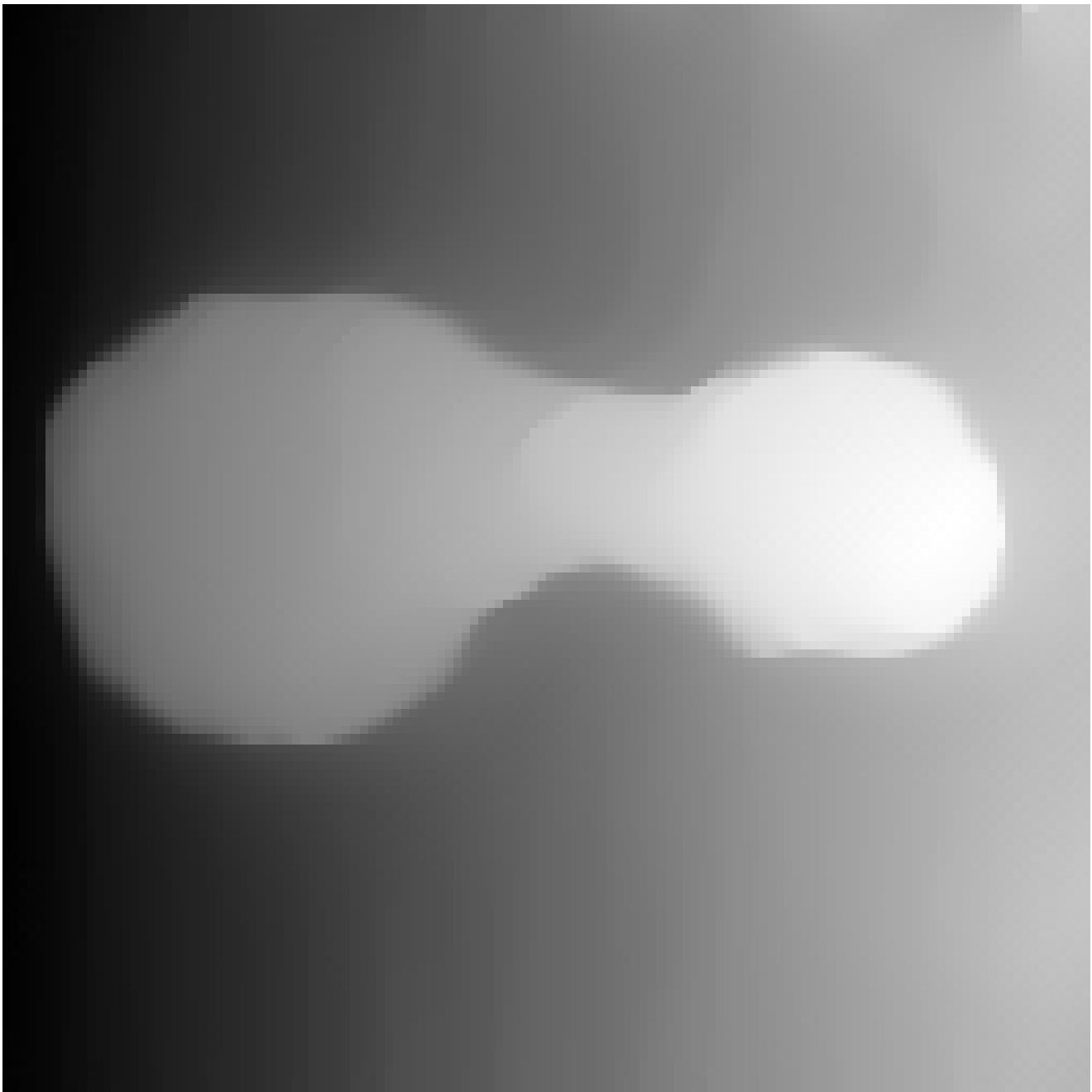}
      \caption{}
      \label{fig:1-L10-denoise}
  \end{subfigure}
 \hfill
   \begin{subfigure}[b]{0.16\linewidth}
      \centering
      \includegraphics[width=\textwidth]{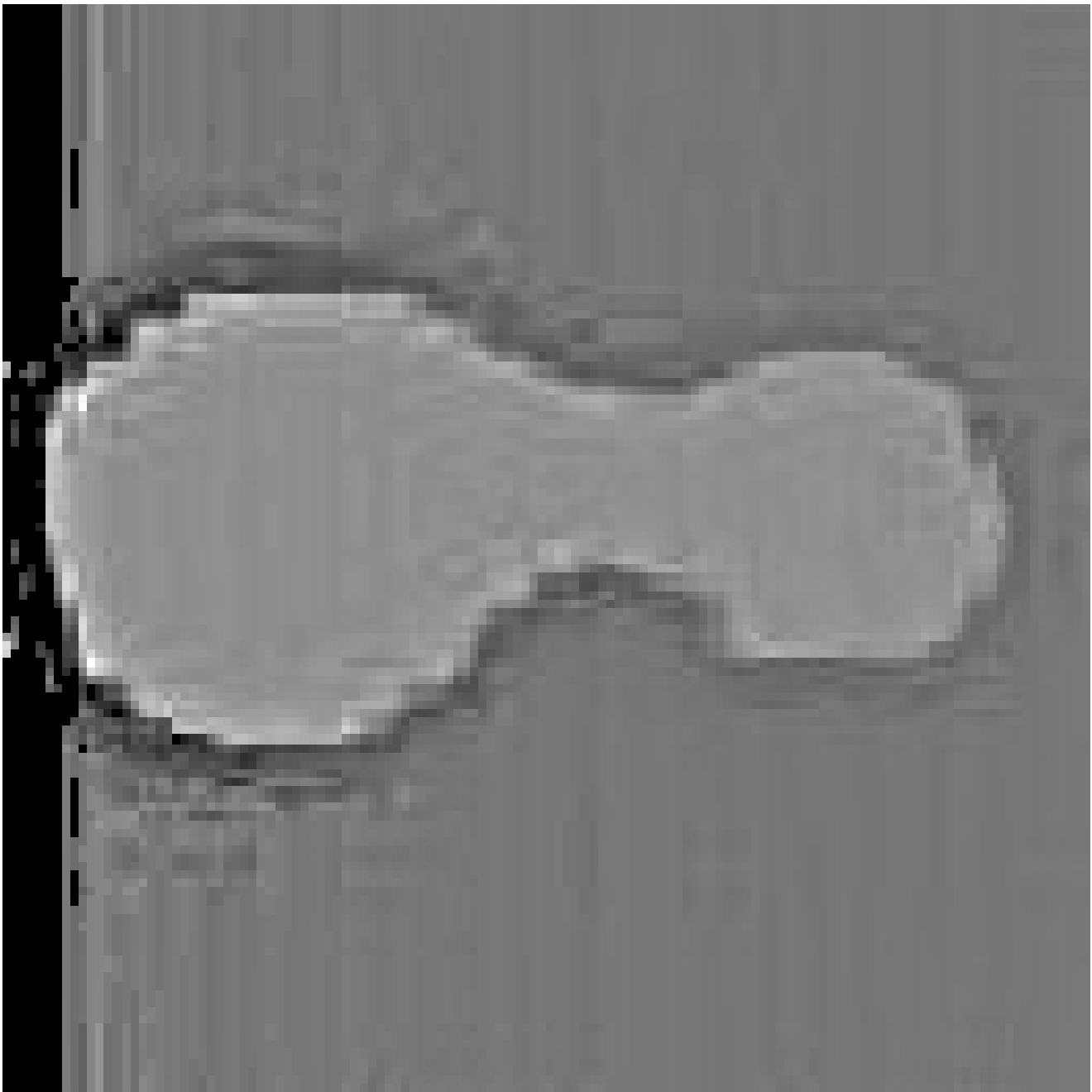}
      \caption{}
      \label{fig:1-bias}
  \end{subfigure}
  \hfill
  \begin{subfigure}[b]{0.16\linewidth}
      \centering
      \includegraphics[width=\textwidth]{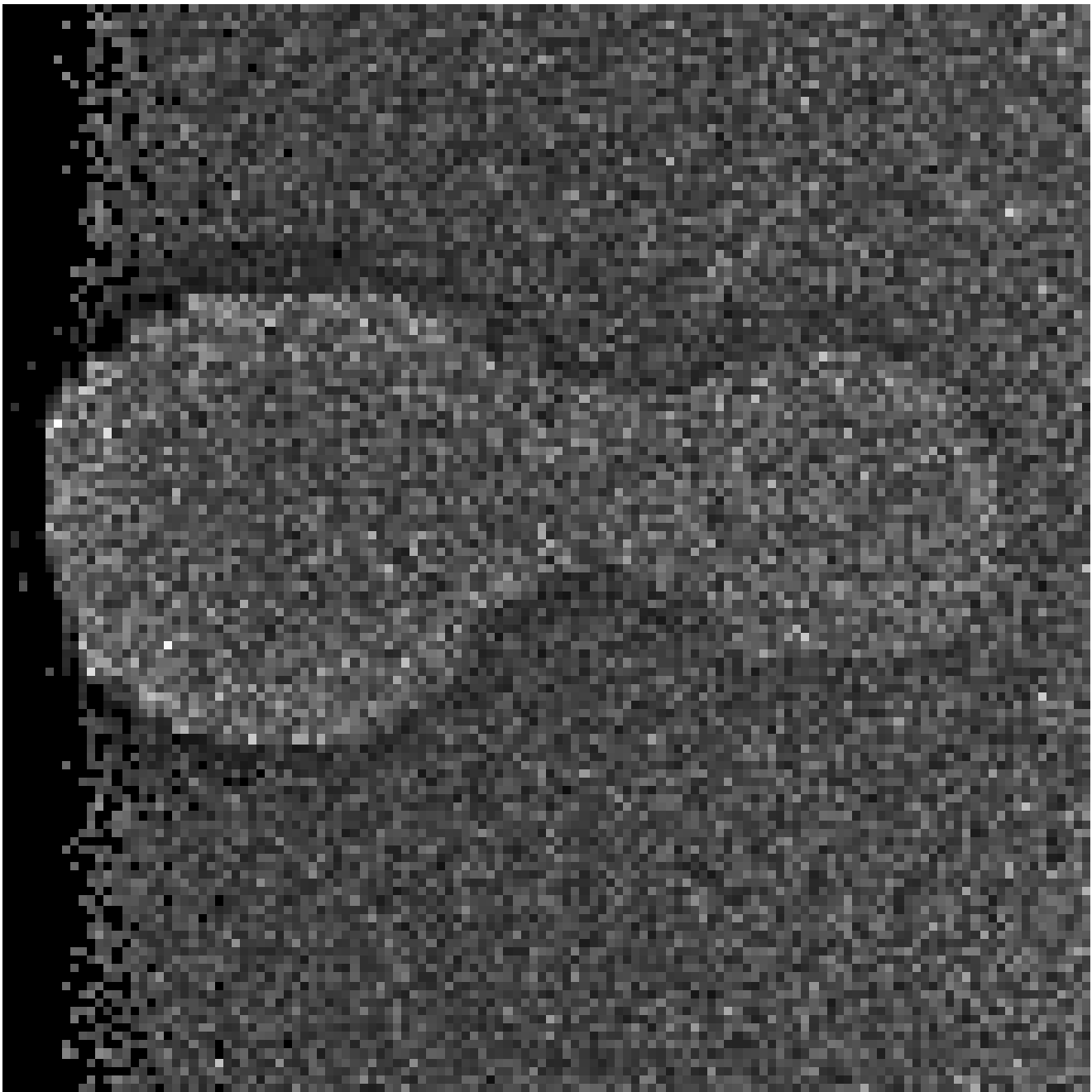}
      \caption{}
      \label{fig:1-L10-bias}
  \end{subfigure}
    \hfill
  \begin{subfigure}[b]{0.16\linewidth}
      \centering
      \includegraphics[width=\textwidth]{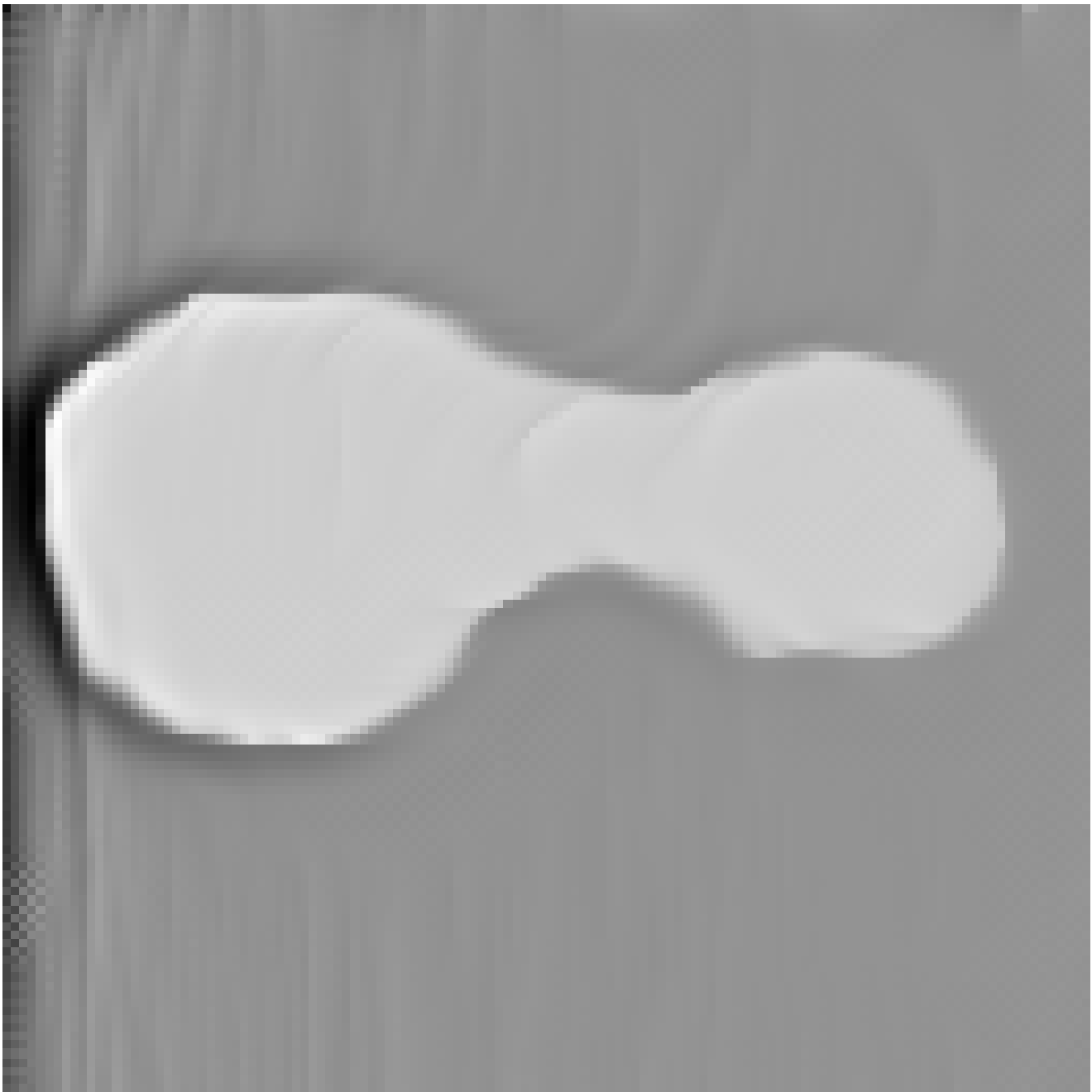}
      \caption{}
      \label{fig:1-L10-denoise-bias}
  \end{subfigure}

  \begin{subfigure}[b]{0.16\linewidth}
      \centering
      \includegraphics[width=\textwidth]{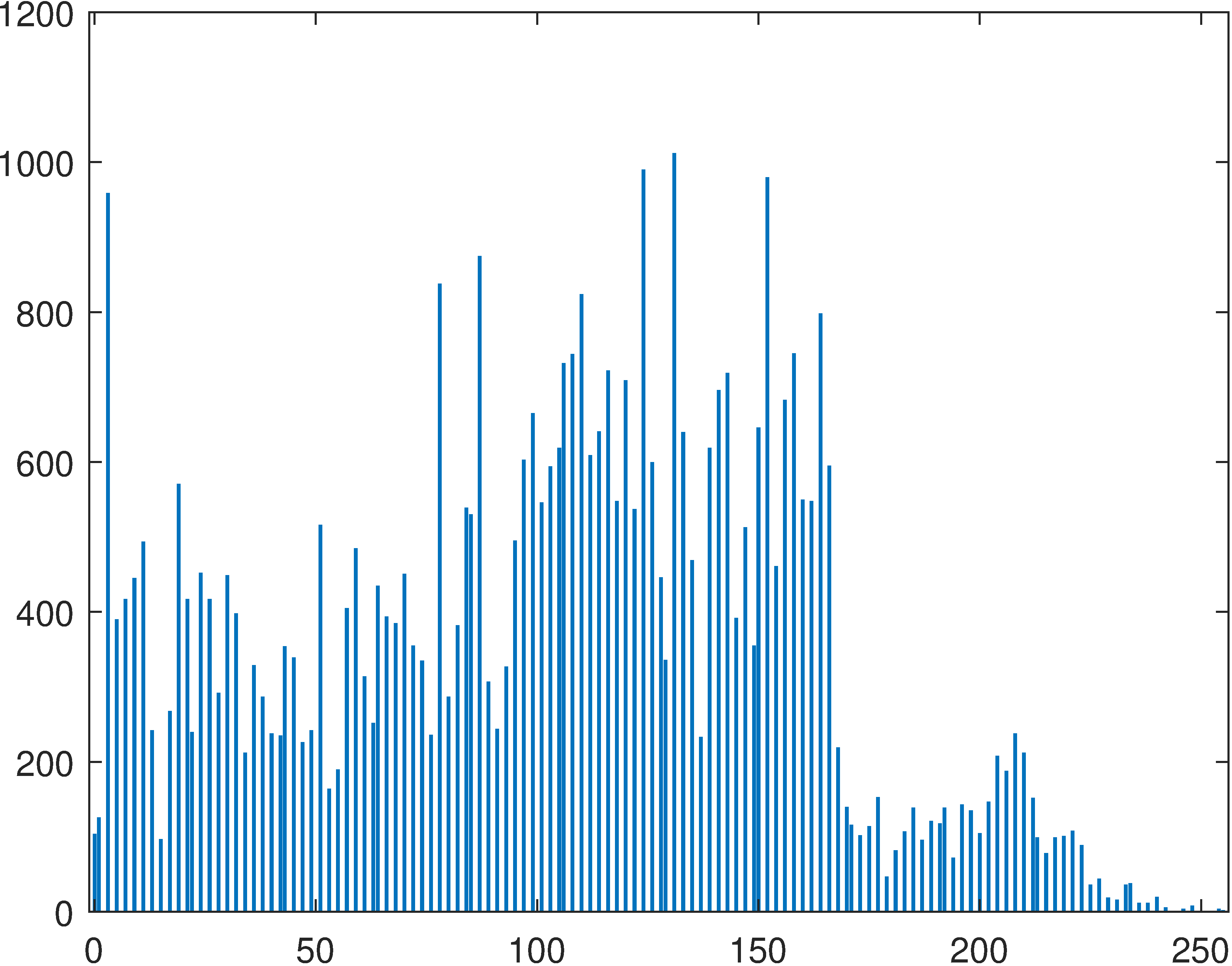}
      \caption{}
      \label{fig:1-intensity}
  \end{subfigure}
  \hfill
  \begin{subfigure}[b]{0.16\linewidth}
      \centering
      \includegraphics[width=\textwidth]{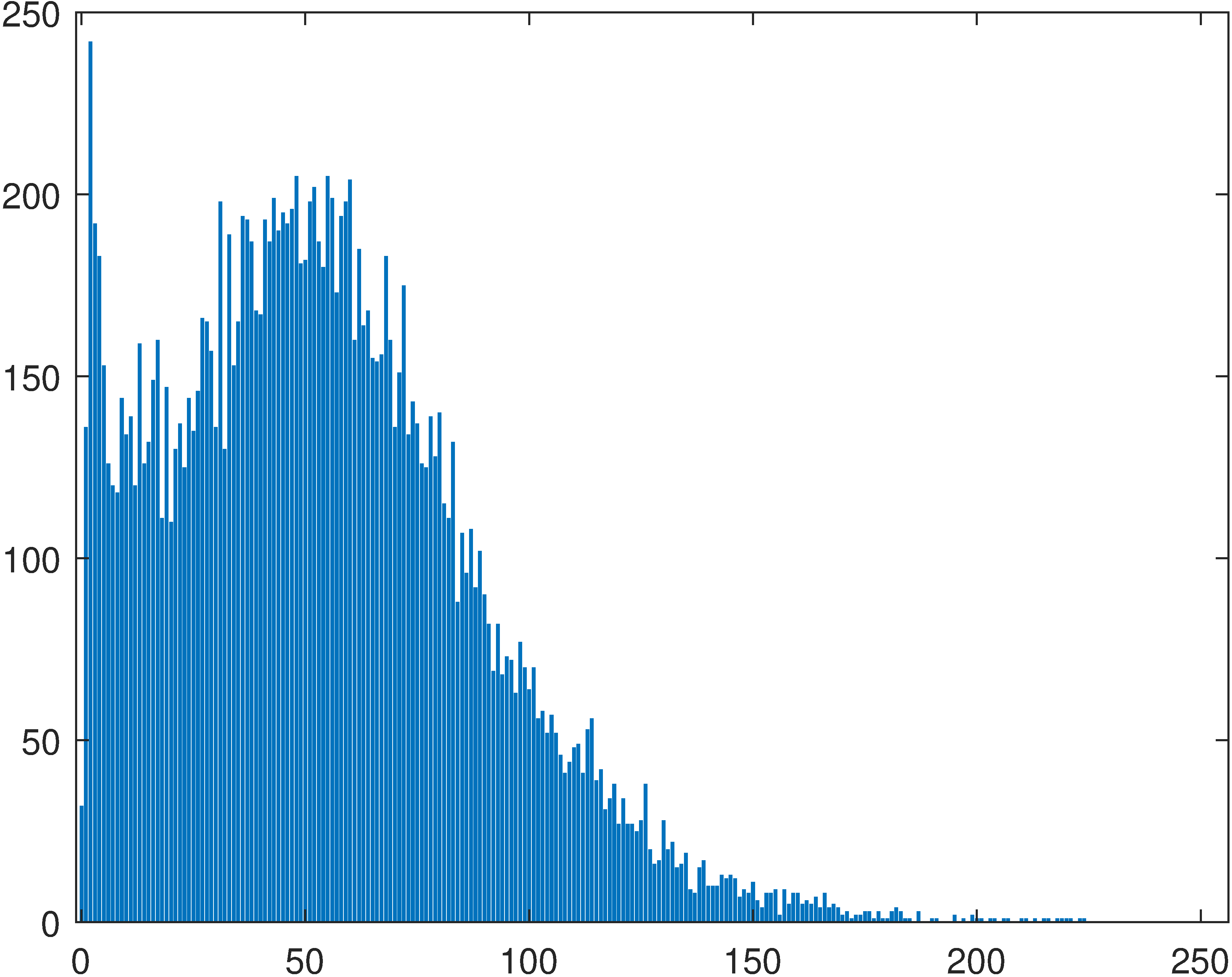}
      \caption{}
      \label{fig:1-L10-intensity}
  \end{subfigure}
     \hfill
    \begin{subfigure}[b]{0.16\linewidth}
      \centering
      \includegraphics[width=\textwidth]{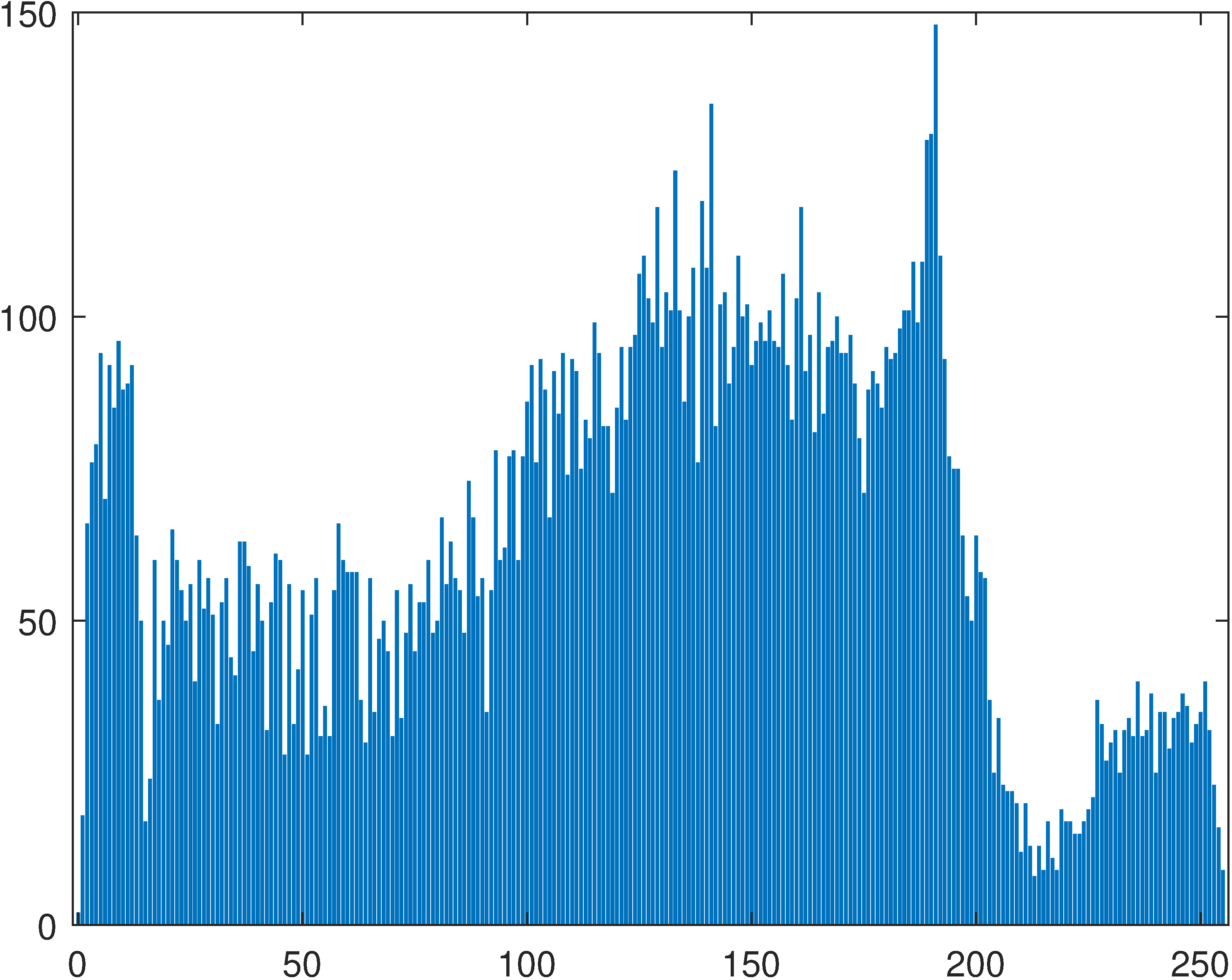}
      \caption{}
      \label{fig:1-L10-denoise-intensity}
  \end{subfigure}
   \hfill
     \begin{subfigure}[b]{0.16\linewidth}
      \centering
      \includegraphics[width=\textwidth]{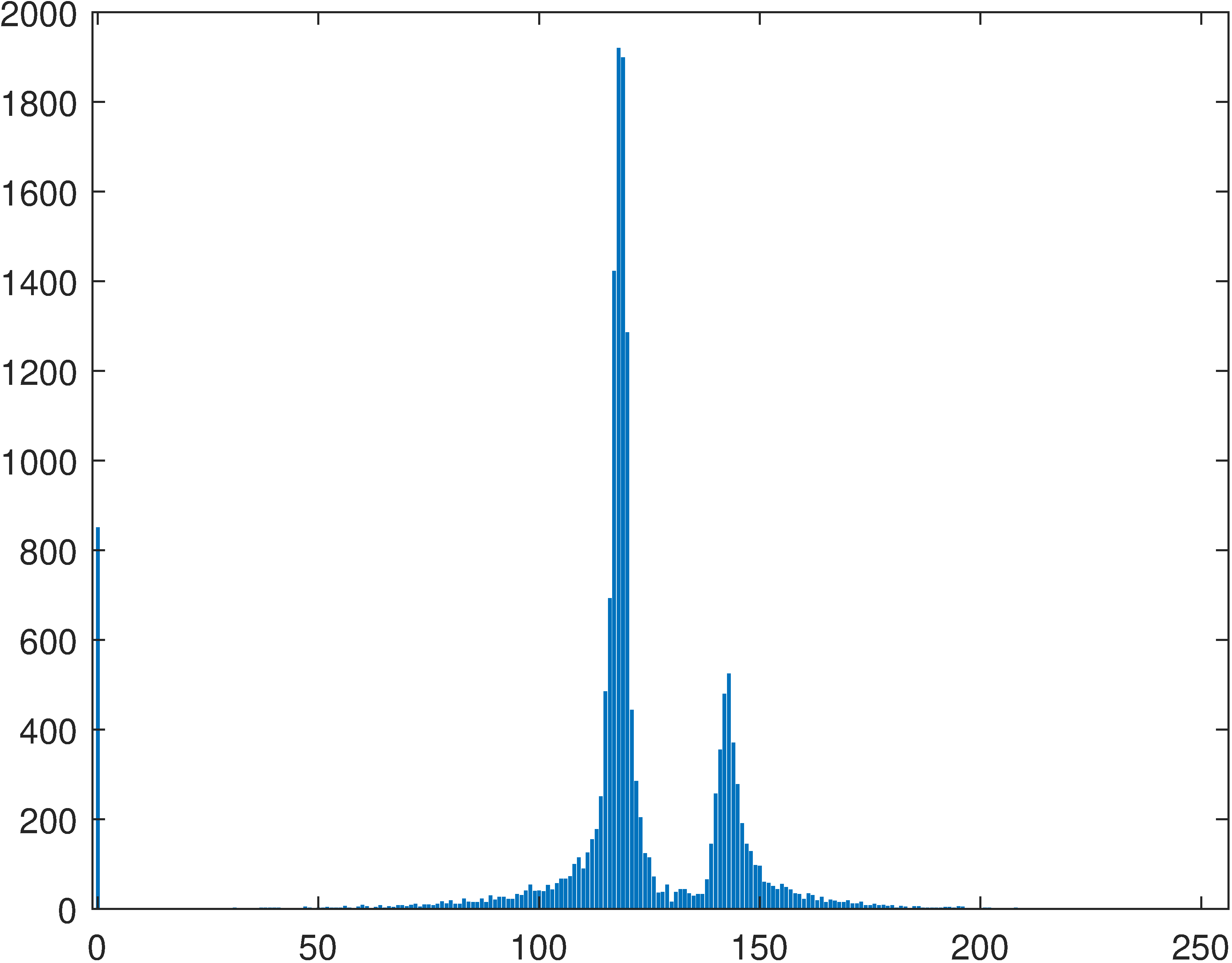}
      \caption{}
      \label{fig:1-bias-intensity}
  \end{subfigure}
   \hfill
  \begin{subfigure}[b]{0.16\linewidth}
      \centering
      \includegraphics[width=\textwidth]{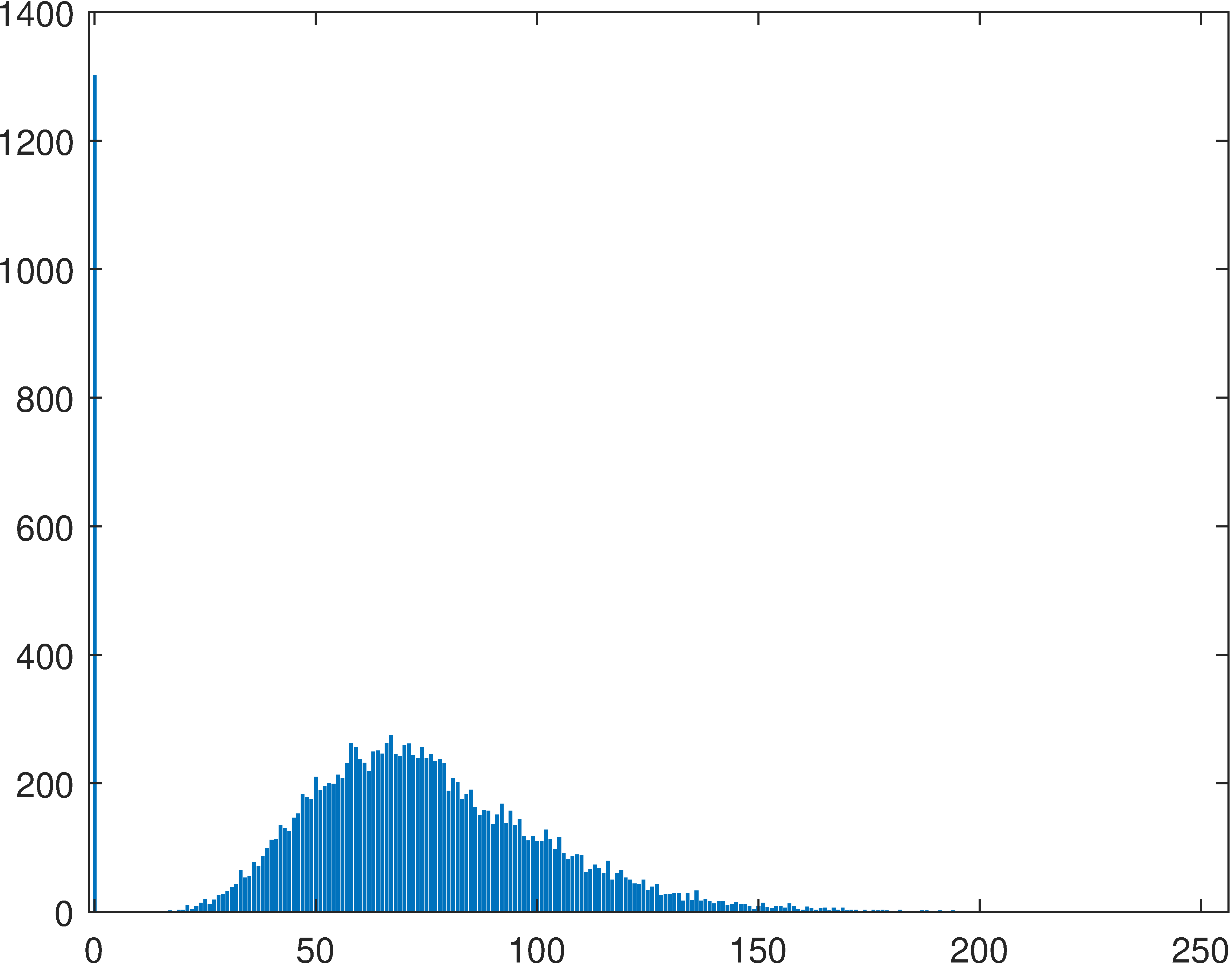}
      \caption{}
      \label{fig:1-L10-bias-intensity}
  \end{subfigure}
   \hfill
  \begin{subfigure}[b]{0.16\linewidth}
      \centering
      \includegraphics[width=\textwidth]{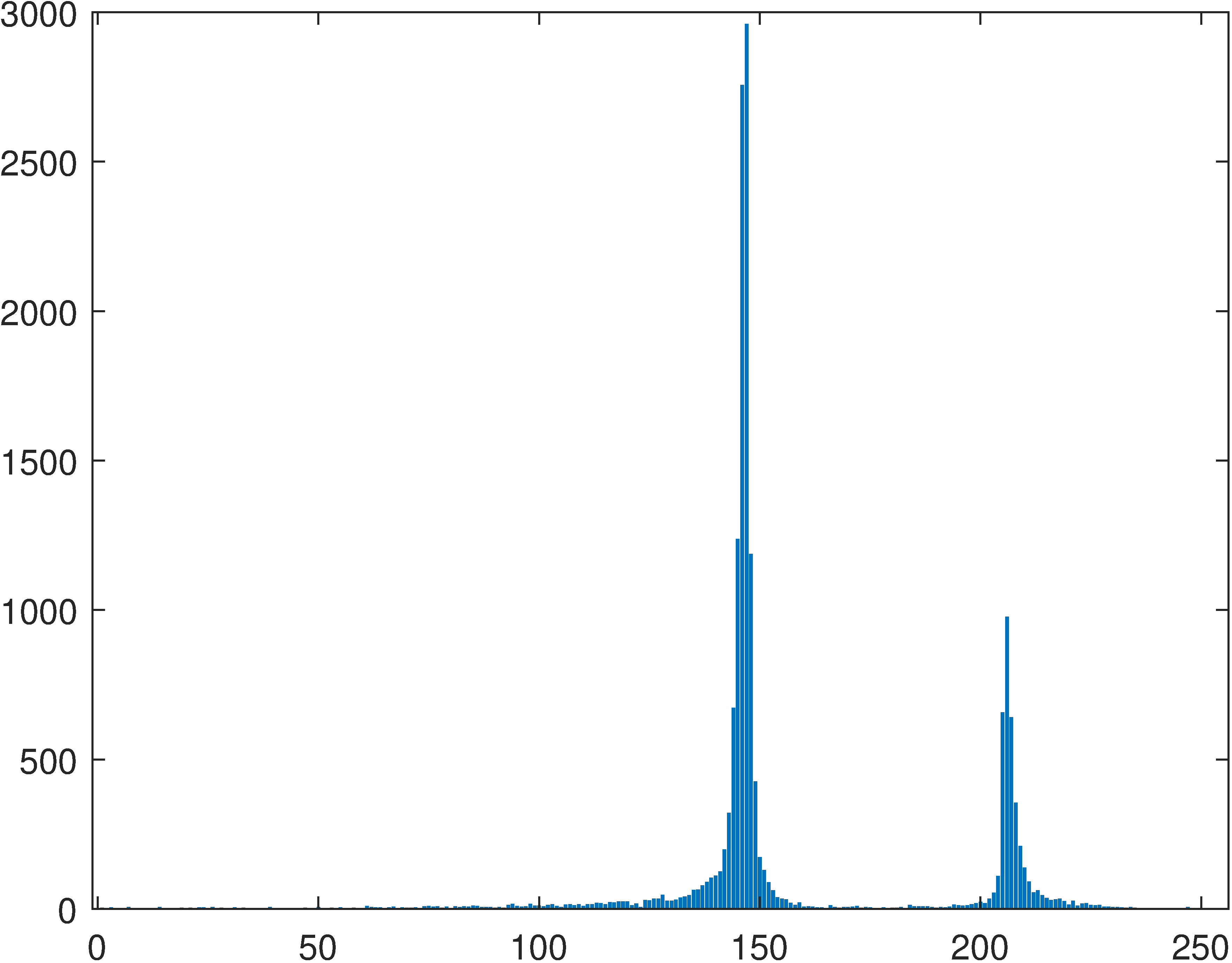}
      \caption{}
      \label{fig:1-L10-denoise-bias-intensity}
  \end{subfigure}

  \caption{Intensity Distribution Histogram. (a) is the original image; (b) is the image corrupted by multiplicative noise; (c) is the denoised image; (d)--(f) are the bias-corrected versions of (a)--(c), respectively; (g)--(i) are the intensity histograms corresponding to (a)--(f), respectively}
  \label{fig:inetsity}
\end{figure}
\begin{figure}
  \centering
  \begin{subfigure}[b]{0.19\linewidth}
      \centering
      \includegraphics[width=\textwidth]{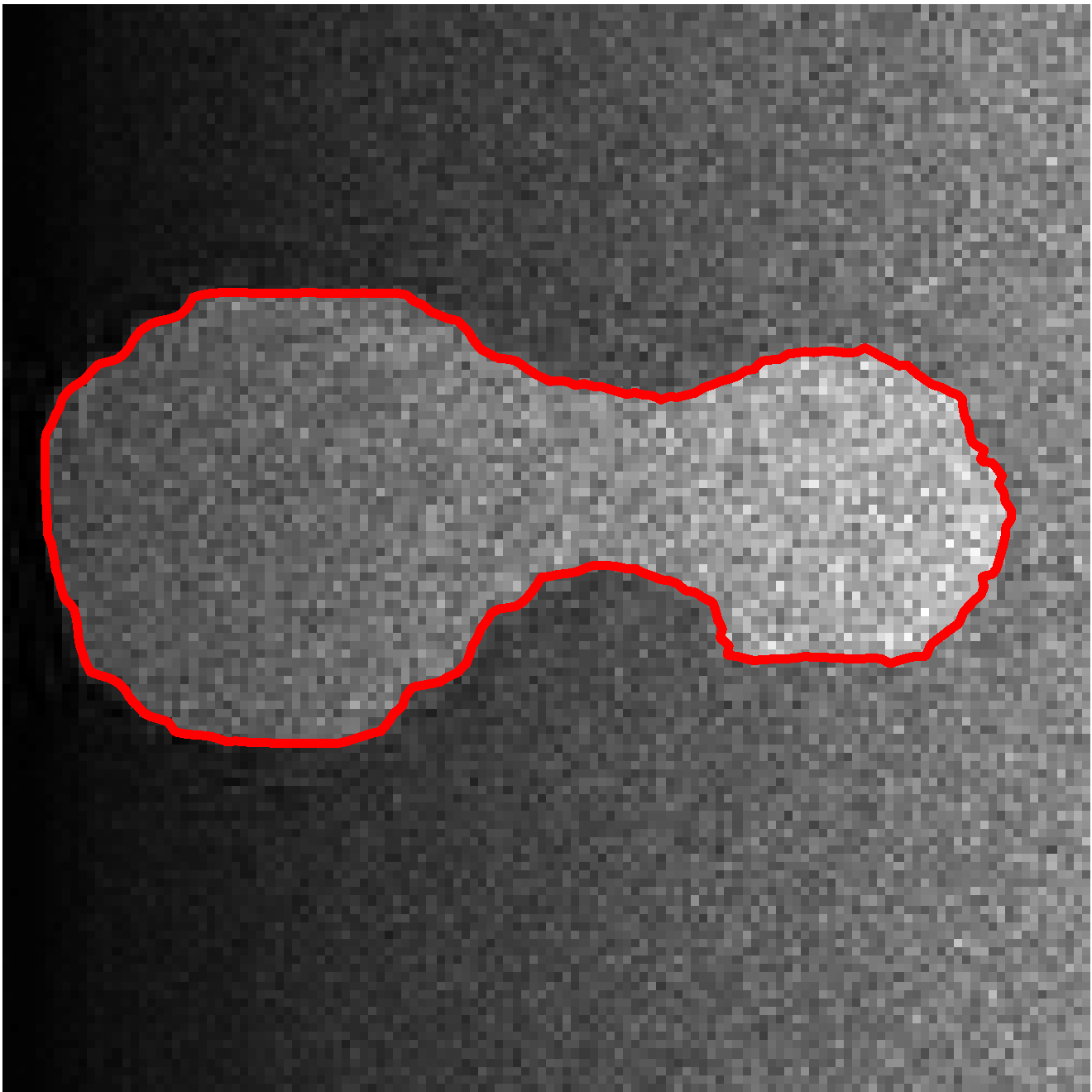}
      \caption{}
      \label{fig:1-origianl}
  \end{subfigure}
  \hfill
  \begin{subfigure}[b]{0.19\linewidth}
      \centering
      \includegraphics[width=\textwidth]{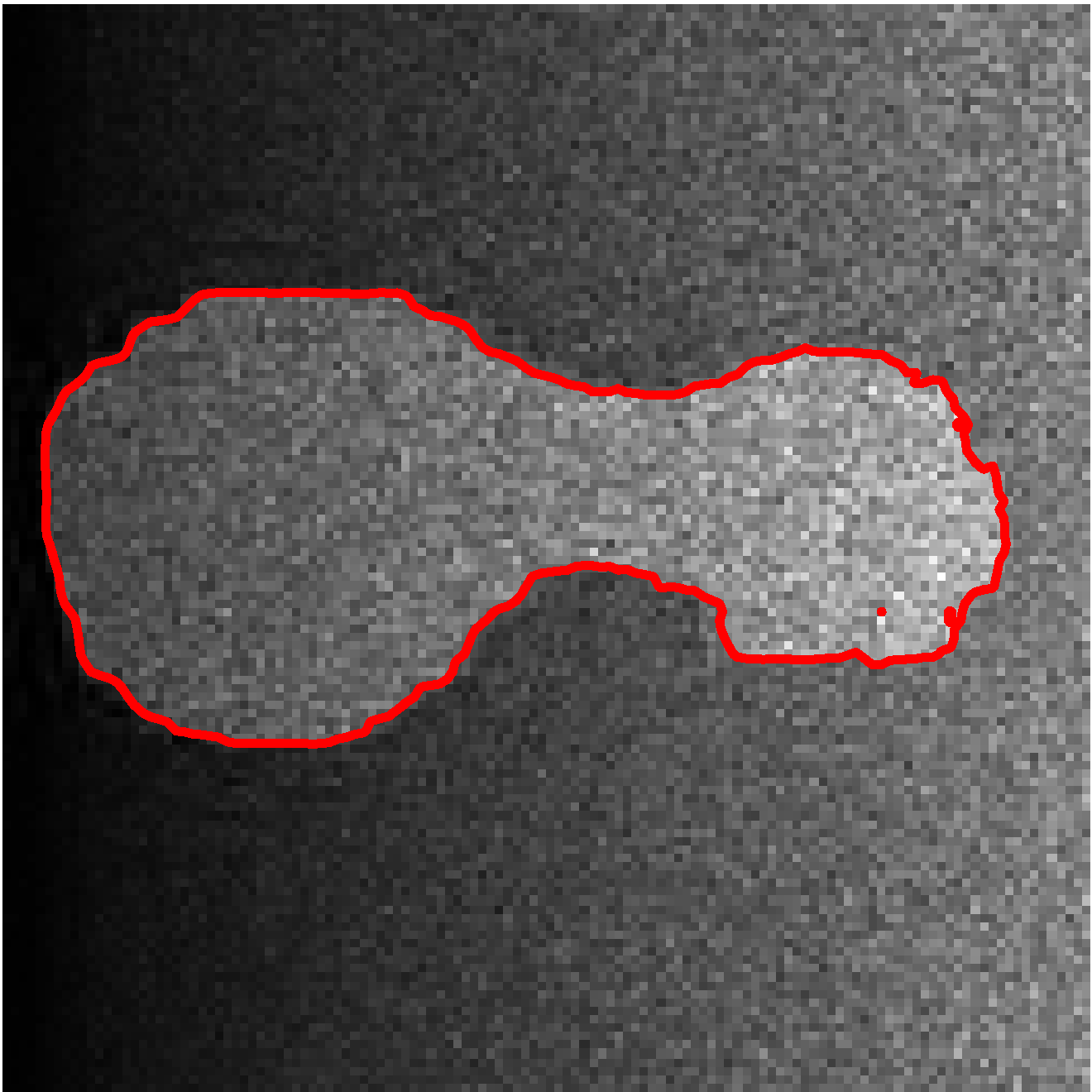}
      \caption{}
      \label{fig:1-L40-seg}
  \end{subfigure}
  \hfill
  \begin{subfigure}[b]{0.19\linewidth}
      \centering
      \includegraphics[width=\textwidth]{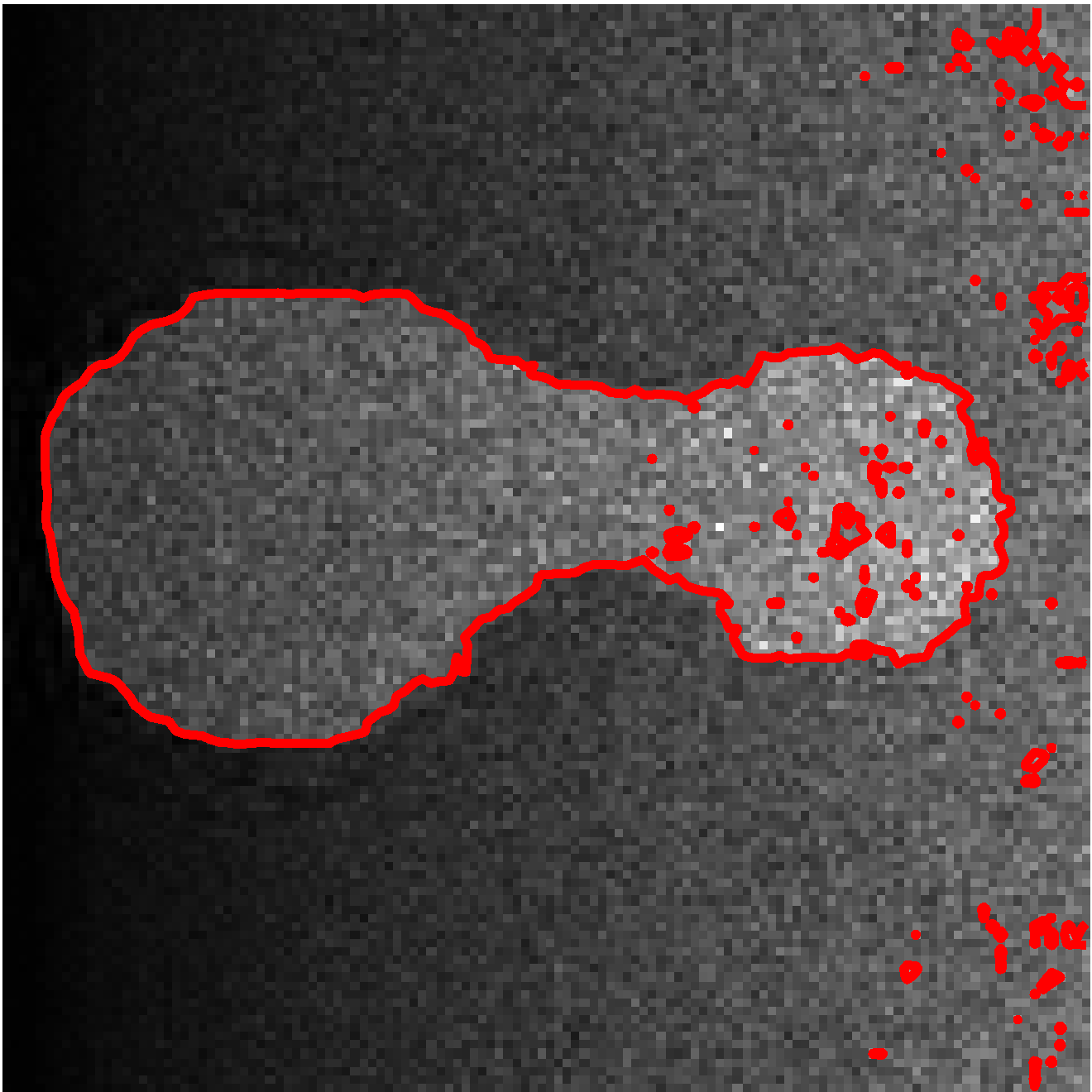}
      \caption{}
      \label{fig:1-L30-seg}
  \end{subfigure}
  \hfill
    \begin{subfigure}[b]{0.19\linewidth}
      \centering
      \includegraphics[width=\textwidth]{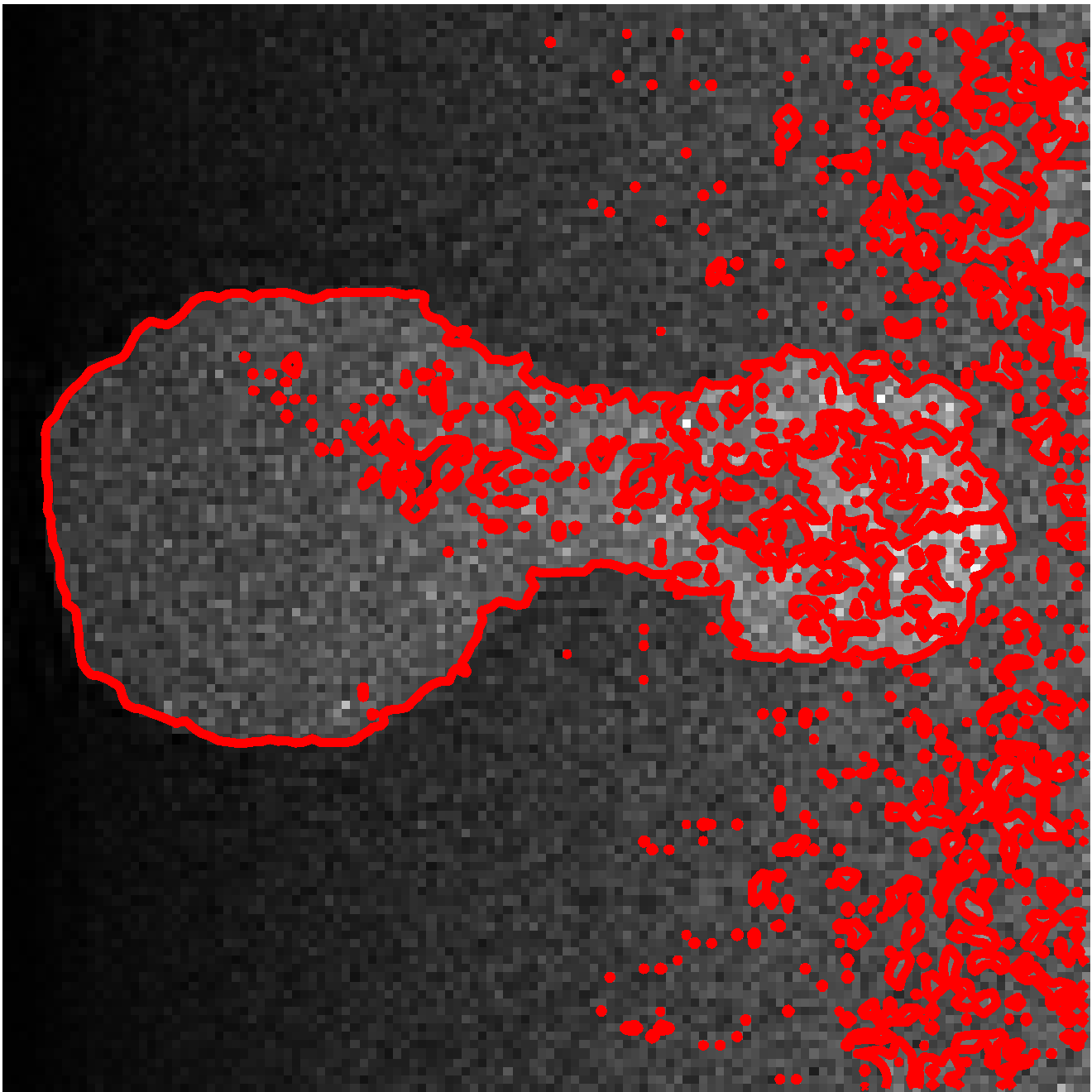}
      \caption{}
      \label{fig:1-L20-seg}
  \end{subfigure}
 \hfill
   \begin{subfigure}[b]{0.19\linewidth}
      \centering
      \includegraphics[width=\textwidth]{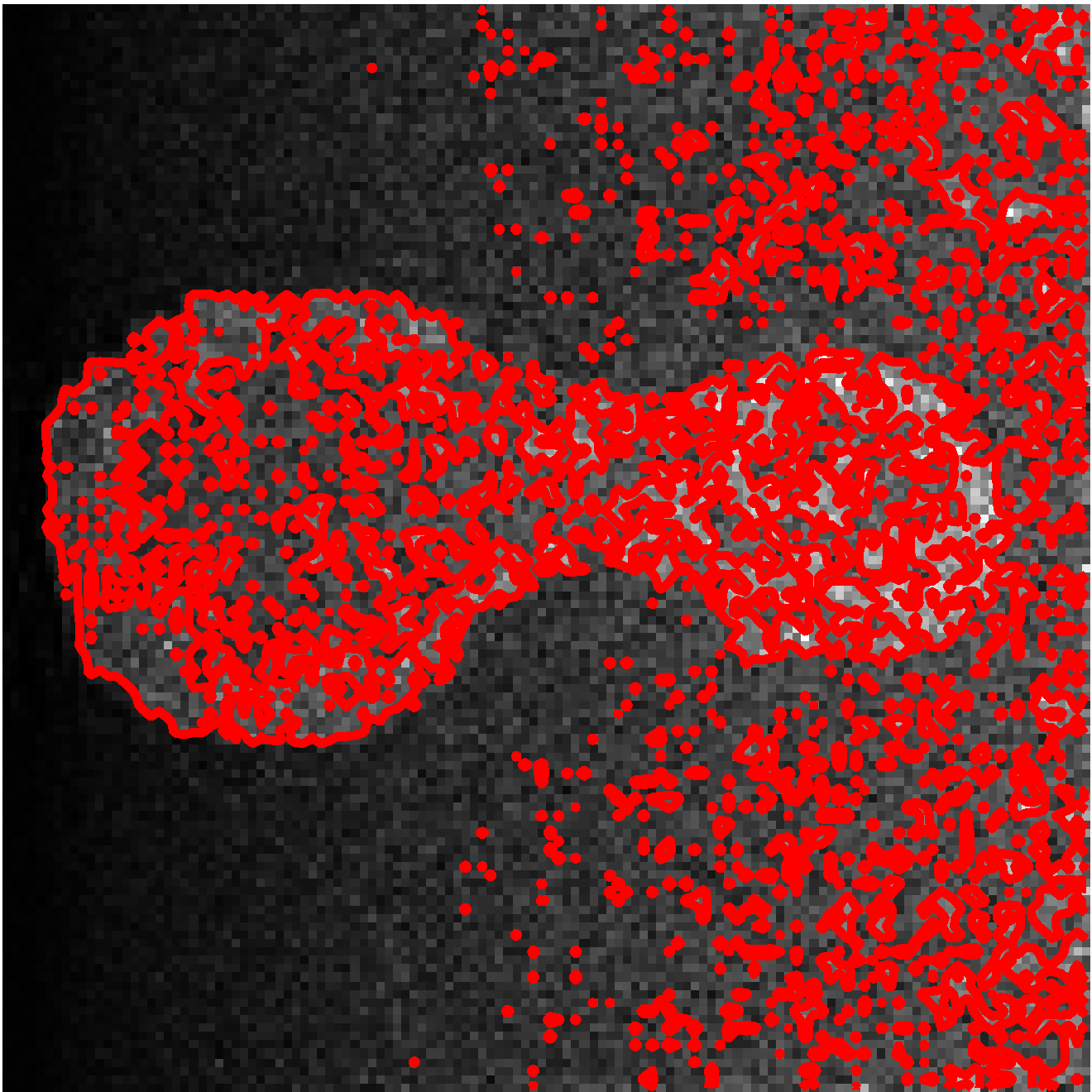}
      \caption{}
      \label{fig:1-L10-seg}
  \end{subfigure}
  \caption{Segmentation results with a model containing regularization (length) term. (a)–(e) correspond to multiplicative noise levels of $ L=50,40,30,20, \text{and }  10$, respectively}
  \label{fig:length term}
\end{figure}
To illustrate the impact of noise on the accuracy of image segmentation, we present the statistical distributions of the image intensity values under different conditions in Fig.~\ref{fig:inetsity}. As shown in Fig.~\ref{fig:1-intensity}, the intensity histogram of the image does not exhibit clearly separable peaks, indicating that the foreground and background have similar intensity values. This explains why the CV model cannot accurately segment such images. In contrast, the LIC model assumes that images with intensity inhomogeneity follow a multiplicative model, i.e., an image with intensity inhomogeneity can be represented as the product of an intensity-homogeneous image and a bias field.
After bias correction, the intensity values in different regions are distinguishable, allowing the model to utilize the image intensity information for segmentation. The intensity distribution after bias correction using the LIC model is illustrated in Fig.~\ref{fig:1-bias-intensity}. It can be seen that the corrected image has clearly separable peaks, which enables the LIC model to achieve accurate segmentation. However, as shown in Fig.~\ref{fig:1-L10} and Fig.~\ref{fig:1-L10-bias}, when the image is degraded by noise, the intensity of the target and background remains similar even after bias correction with the LIC model.
After denoising and bias correction (Fig.~\ref{fig:1-L10-denoise-bias}), a clear difference in intensity values emerges between the foreground and background, resulting in a distinct boundary.

Liu et al. \cite{liu2013image} claimed the regularization term can enhance the robustness of the model to noise. However, as shown in Fig.~\ref{fig:length term}, the regularized model produces correct boundaries at low noise levels, but still yields inaccurate segmentation under severe noise, as strong noise significantly perturbs the image intensity distribution.
Therefore, to improve segmentation accuracy for noisy images, it is necessary to incorporate denoising terms into the segmentation model.

\subsection{Robust Variational Image Segmentation Model}

\subsubsection{Denoising Terms of Segmentation Model}
\label{subsubsec:noise}
\begin{figure}
    \centering
    \begin{subfigure}[b]{0.24\linewidth}
        \centering
        \includegraphics[width=\linewidth]{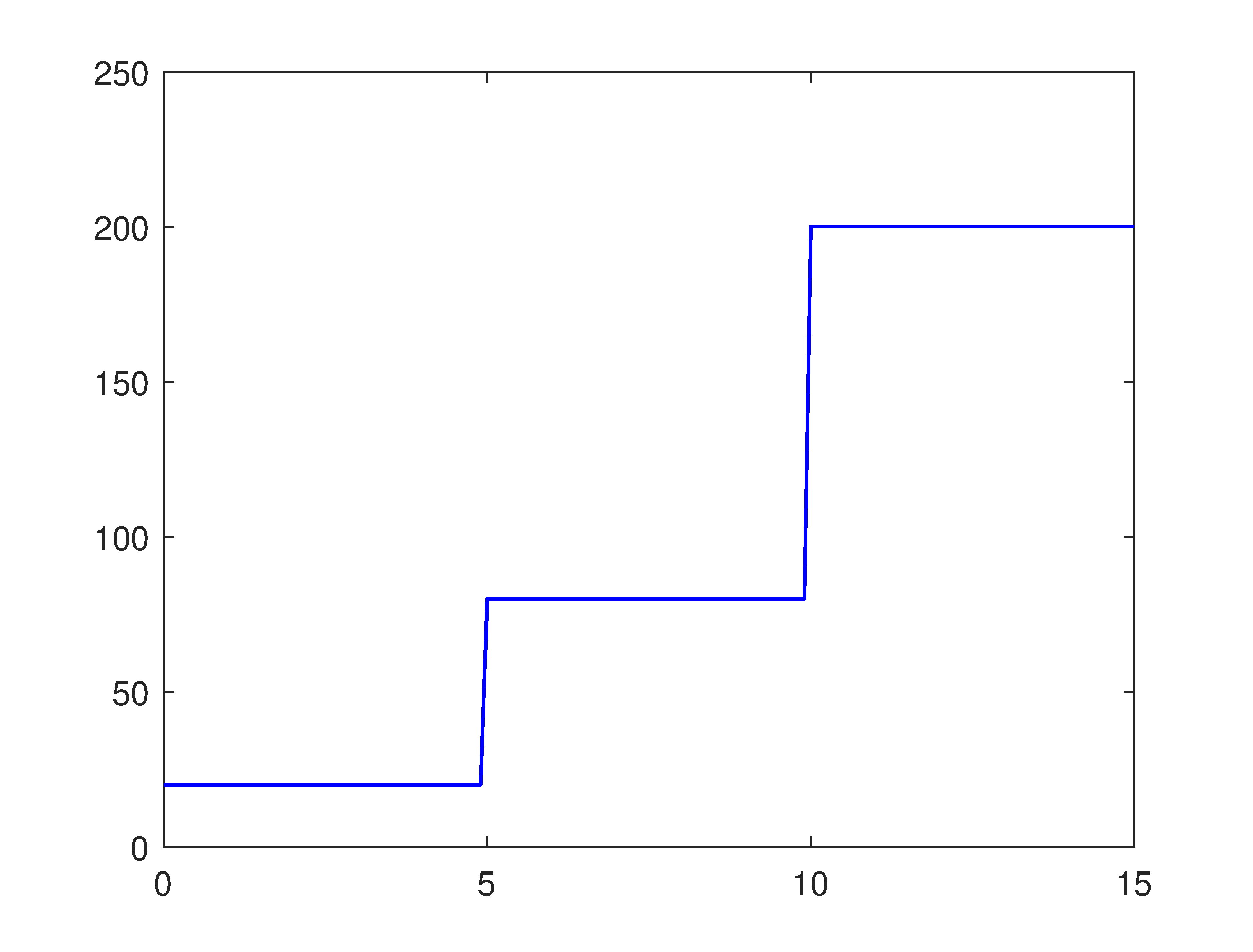}
        \caption{}
        \label{fig:OriSingle}
    \end{subfigure}
   \hfill
   \begin{subfigure}[b]{0.24\linewidth}
       \centering
       \includegraphics[width=\linewidth]{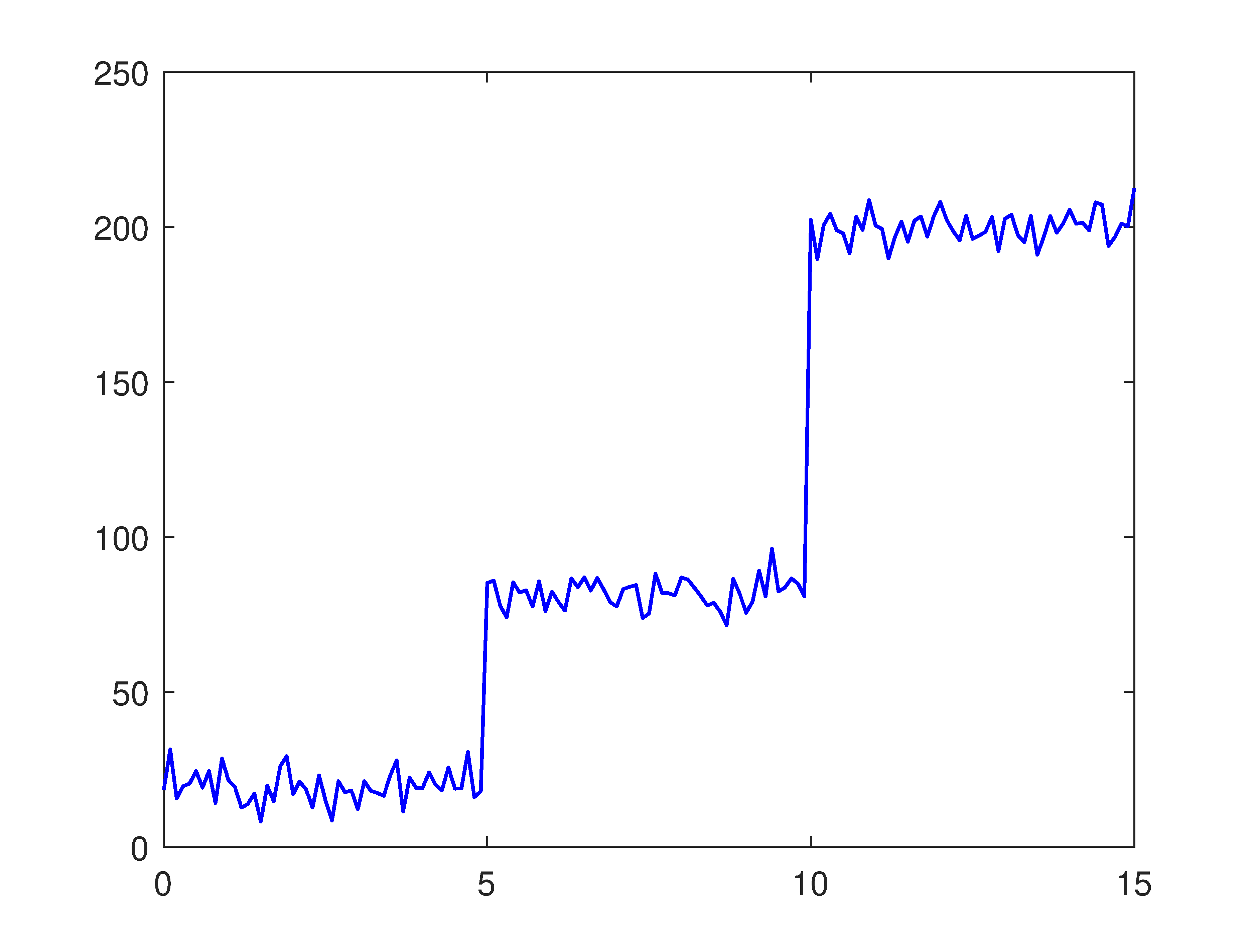}
       \caption{}
       \label{fig:Gaussian}
   \end{subfigure}
    \hfill
    \begin{subfigure}[b]{0.24\linewidth}
        \centering
        \includegraphics[width=\linewidth]{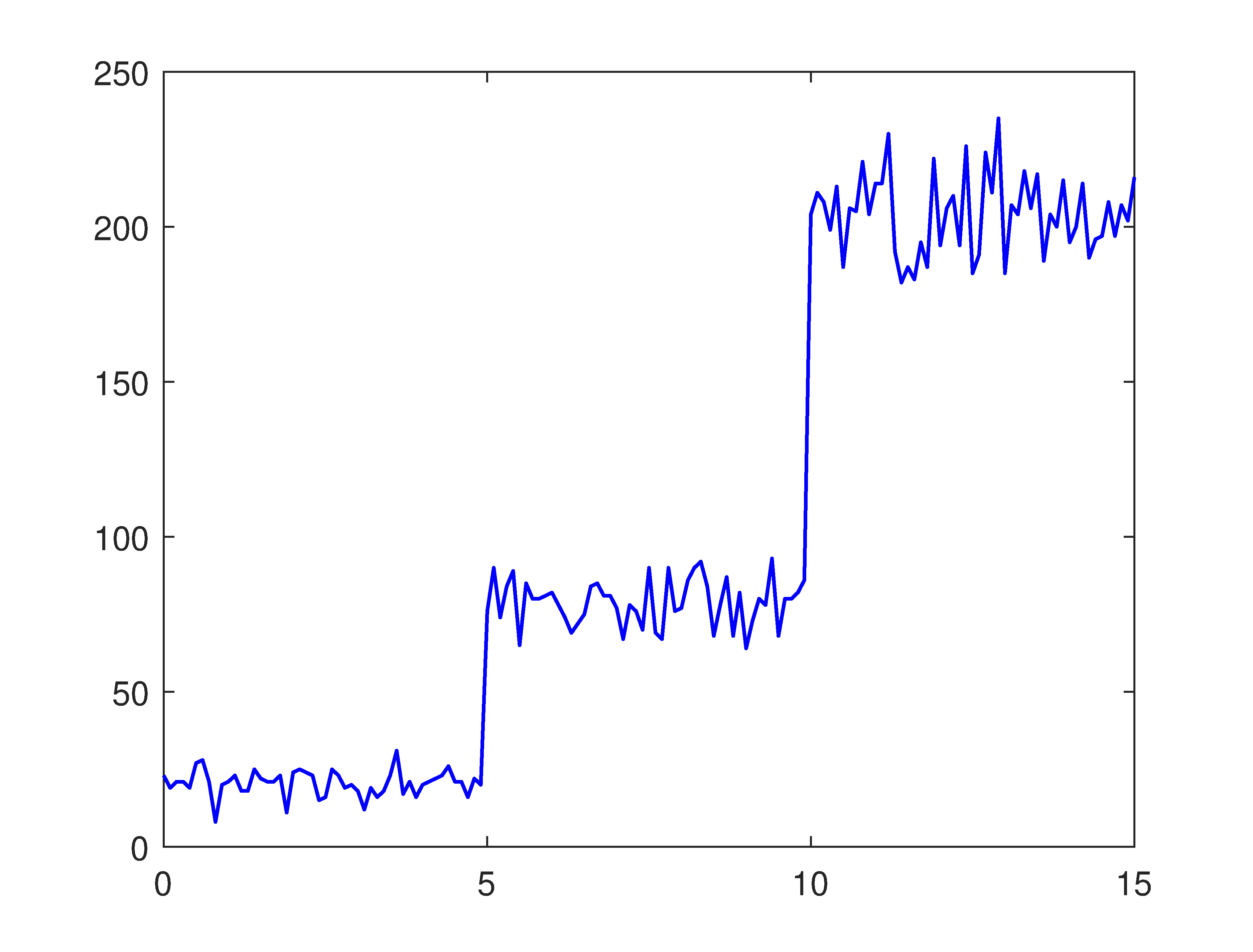}
        \caption{}
        \label{fig:Poisson}
    \end{subfigure}
    \hfill
    \begin{subfigure}[b]{0.24\linewidth}
        \centering
        \includegraphics[width=\linewidth]{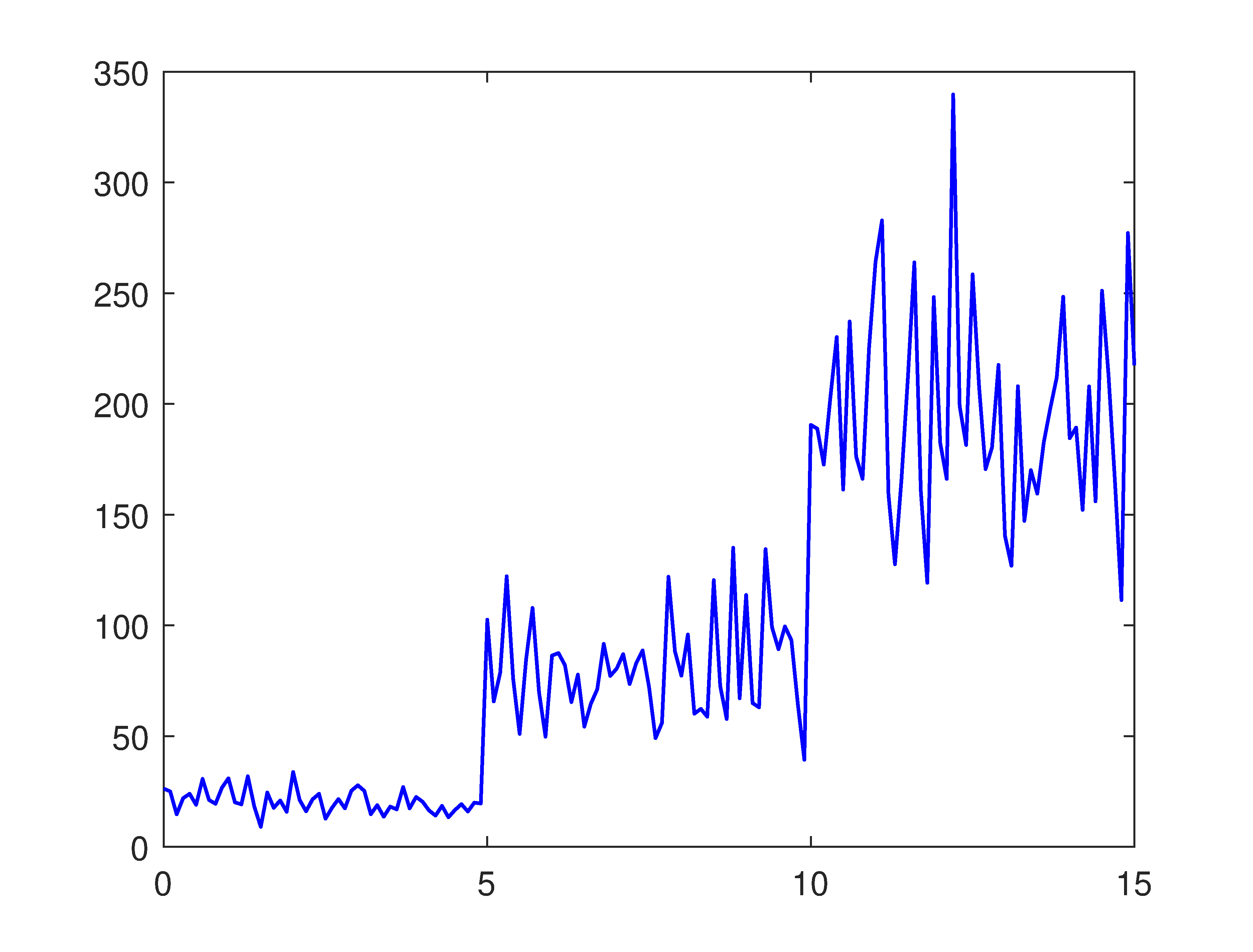}
        \caption{}
        \label{fig:MultiGamma}
    \end{subfigure}
    \caption{Effects of different noise on 1D signals. (a) Original 1D signal; (b) Signal with added Gaussian noise; (c) Signal with added Poisson noise; (d) Signal with multiplicative Gamma noise ($L=10$)}
    \label{fig:noise}
\end{figure}

With the gradual advancement of image segmentation techniques, incorporating appropriate noise models has become increasingly important, particularly for applications such as PET and SAR imaging. For data with a low signal-to-noise ratio (SNR), considering the statistical characteristics of corresponding noise in the segmentation process is crucial \cite{sawatzky2013variational}. Unlike additive Gaussian noise, both Poisson noise and multiplicative noise are signal-dependent. As illustrated in Fig.~\ref{fig:noise}, higher signal intensity leads to a greater influence of Poisson and multiplicative noise.

For any $x \in \Omega$, let $f(x)$ be an observed image degraded by Poisson noise, and let $g(x)$ denote the corresponding clean image. Then $f(x)$ follows a Poisson distribution with both mean and variance equal to $g(x)$, i.e.,
\begin{equation*}
  p_{f(x)}(k)=\frac{g(x)^k e^{-g(x)}}{k!} .
\end{equation*}
Using maximum a posteriori (MAP) estimation and total variation (TV) regularization, Le et al. \cite{le2007variational} proposed a variational model for Poisson denoising, formulated as
\begin{equation}
\label{eq:I-div}
  E(g)=\int_{\Omega} (g - f\log g)\,dx+\gamma \int_{\Omega}|\nabla g|\,dx ,
\end{equation}
where $g\in BV(\Omega)$ and $\log g \in L^1(\Omega)$. The first term in \eqref{eq:I-div} is also known as the I-divergence (or Kullback-Leibler (KL) divergence) between $g(x)$ and $f(x)$. Unlike the $L^2$-norm used as a fitting term, the I-divergence is highly nonlinear, which increases the computational complexity of solving the associated minimization problem \cite{sawatzky2013variational}. Therefore, it is important to develop efficient algorithms.

Consider the multiplicative Gamma noise model, let $g(x)$ denote the clean image defined on $\Omega$. The observed image $f(x)$ is corrupted by multiplicative Gamma noise if:
\begin{equation*}
  f(x)=g(x)\eta(x),\ \forall x\in \Omega ,
\end{equation*}
where the random variable $\eta(x)$ follows a Gamma distribution. The probability density function is given by
\begin{equation*}
  p_{\eta(x)}(k;\theta,L)=\frac{1}{\theta^L\Gamma(L)}k^{L-1}e^{-\frac{k}{\theta}},\ k>0 ,
\end{equation*}
where $\Gamma(\cdot)$ denotes the standard Gamma function, $\theta$ and $L$ are the scale and shape parameters, respectively. The mean and variance of $\eta$ are given by $\theta L$ and $\theta^2 L$, respectively. Typically, it is assumed that the mean of $\eta$ is equal to 1, that is, $\theta L=1$. Consequently, the variance of $\eta$ is $1/L$.

Although the functional in \eqref{eq:I-div} was originally established as a denoising model for Poisson noise, Steidl et al. \cite{steidl2010removing} theoretically and experimentally demonstrated that it is also effective for multiplicative Gamma noise.

To address the signal-dependent nature of noise, we incorporate a gray level indicator operator as a weight into the TV regularization term. This enables the model to perform adaptive denoising across regions of different intensity levels, preserving boundary details, and thereby improving segmentation accuracy. The denoising terms of our segmentation model are formulated as
\begin{equation}
\label{eq:denoising terms}
    E_{d}=\gamma \int_{\Omega}(g(x)-f(x)\log g(x)) \,dx + \nu \int_{\Omega} \alpha(x) |\nabla g(x)|dx ,
\end{equation}
where $\gamma$ and $\nu >0$ are equilibrium parameters. The function $\alpha(x)>0$ represents the gray level indicator \cite{zhang2021adaptive}. Specifically, $\alpha$ is represented as follows:
\begin{equation}
\label{eq:alpha}
    \alpha(x)=\left( \frac{G_{\sigma}*f}{M} \right)^p ,
\end{equation}
where $p>0$ is a parameter, $M=\max(G_{\sigma}*f)$, $G_{\sigma}$ is the Gaussian kernel.

By incorporating the gray level indicator as an adaptive weight into the denoising terms, this model can adaptively control the diffusion rate based on the intensity values of images. In regions with high intensity values, where noise induces more severe degradation, the diffusion coefficient increases, leading to faster diffusion. In contrast, in regions with low intensity values, where noise has less impact, the diffusion coefficient decreases. This results in slower diffusion, thus preventing over-smoothing.

\subsubsection{Robust Segmentation Model}
In order to improve the model's effectiveness in segmenting images with intensity inhomogeneity and noisy images, we propose a segmentation model that incorporates denoising terms.

Let $f:\Omega \rightarrow \mathbb{R}^d$ denote the observed image. The image domain $\Omega$ is divided into $n$ disjoint regions.
Each subregion $\Omega_i$ is represented by its characteristic function $u_i(x)$, defined as
\begin{equation}\label{eq:CharFunction}
  u_i(x)=
  \begin{cases}
    1, & x\in\Omega_i , \\
    0, & \text{otherwise} .
  \end{cases}
\end{equation}
Let $\mathbf{u(x)}:=\{u_1(x),u_2(x),\dots,u_n(x) \}$, where $\mathbf{u(x)} \in \Lambda_1$, with $\Lambda_1$ defined as
\begin{equation*}
    \Lambda_1= \left\{\mathbf{u(x)} \in BV(\Omega,\  \mathbb{R}^n) \mid u_i(x)\in \{ 0,1 \} , \ \sum_{i=1}^{n}u_i(x)=1 \right\} .
\end{equation*}
As illustrated in Sect.~\ref{subsubsec:noise}, the denoising terms \eqref{eq:I-div} are capable of removing both Poisson noise and multiplicative Gamma noise. To enhance segmentation accuracy for images with high noise, we propose an $n$-phase segmentation model that incorporates these denoising terms and the bias field. The energy functional of the model is formulated as
\begin{equation}\label{eq:PoissonICTM}
\begin{split}
   \min_{\mathbf{c},b,g,\mathbf{u}} E(\mathbf{c},b,g,\mathbf{u})
   = &
   \underbrace{\sum_{i=1}^{n} \lambda_i\int_{\Omega}\int_{\Omega} u_i(x)G_\rho(y-x)(g(x)-b(y)c_i)^2\,dydx}_{\text{fitting term}}  \\
   & + \underbrace{\mu \sum_{i=1}^{n} \sum_{j=1,j \neq i}^{n} \sqrt{\tfrac{\pi}{\tau}} \int_{\Omega} u_i(x) (G_{\tau} * u_j)(x) \,dx}_{\text{regularization term}} \\
   & + \underbrace{\gamma \int_{\Omega}(g(x)-f(x)\log g(x)) \,dx + \nu \int_{\Omega} \alpha(x)|\nabla g(x)|dx}_{\text{denoising terms}} , \\
\end{split}
\end{equation}
where $\lambda_i, \mu, \gamma, \nu$ are positive constants and $G_\rho, G_\tau$ are two-dimensional Gaussian kernels with standard deviations $\rho$ and $\tau$, respectively. $b$ represents the bias field and $g \in BV(\Omega)$ is the denoised image. $\mathbf{c}=\{ c_1, c_2, \dots, c_n \} \in \mathbb{R}^n $ and $\mathbf{u} \in \Lambda_1$.

The fitting term of the denoised image $g$ uses the bias field to correct the image with intensity inhomogeneity. The regularization term approximates the contour length based on the characteristic functions of the regions. Minimizing this term leads to smoother contours and enhances the model's robustness to noise. The denoising terms specifically address Poisson noise and multiplicative Gamma noise. By integrating these denoising terms into the segmentation model, a smoother image $g$ can be obtained. Segmentation using the smoother image can reduce the impact of high noise and lead to more accurate segmentation results.

 By incorporating bias correction and adaptive denoising terms into our model, we enhance its robustness to noise and intensity inhomogeneity in images. In addition, this model simultaneously performs image segmentation and denoising, thereby functioning as a multi-task model.
 Furthermore, integrating the denoising terms into the segmentation process allows adaptive adjustment of the smoothed image based on intermediate results. By tuning the smoothing parameter, our model can effectively segment images with severe noise. The model reduces to the LIC model when the denoising parameters are set to 0, and to the CV model with added denoising terms when $b(y)=1$.

\section{Numerical Algorithms}
\label{sec:algorithms}
In this section, we present the algorithmic framework for solving our model, which is based on the ICTM and RMSAV methods. We also demonstrate the energy decay property for each subproblem. The model \eqref{eq:PoissonICTM} is decomposed into four subproblems, which are solved using the alternating iterative method. The subproblems are formulated as follows:
\begin{align*}
     & \mathbf{c}^{k+1} = \argmin_{\mathbf{c}\in \mathbb{R}^n} E(\mathbf{c},b^k,g^k,\mathbf{u}^k) , \\
     & b^{k+1} = \argmin_{b\in L^2(\Omega)} E(\mathbf{c}^{k+1},b,g^k,\mathbf{u}^k) , 
\\
     & g^{k+1} = \argmin_{g\in BV(\Omega)} E(\mathbf{c}^{k+1},b^{k+1},g,\mathbf{u}^k) , 
\\
     & \mathbf{u}^{k+1} = \argmin_{\mathbf{u}\in\Lambda_1} E(\mathbf{c}^{k+1},b^{k+1},g^{k+1},\mathbf{u}) . 
\end{align*}

\subsection{Minimizer of Subproblem with Respect to $\mathbf{c}$}
For fixed $b^k$, $g^k$, and $\mathbf{u}^k$, the subproblem with respect to $\mathbf{c}$ is formulated as
\begin{equation*} 
  E_\mathbf{c}(\mathbf{c})=\sum_{i=1}^{n} \lambda_i \int_{\Omega}\int_{\Omega} u_i^k(x)G_\rho(y-x)(g^k(x)-b^k(y)c_i)^2\,dydx .
\end{equation*}
The minimizer of the subproblem with respect to $\mathbf{c}=(c_1,\dots,c_n)$ satisfies the following optimality condition:
\begin{equation*}
  2\lambda_i\int_{\Omega}\int_{\Omega}u_i^k(x)G_\rho(y-x)(g^k(x)-b^k(y)c_i)b^k(y)\,dydx=0 .
\end{equation*}
Consequently, we can obtain
\begin{equation}\label{eq:c}
  c_i^{k+1}=\frac{\int_{\Omega} u_i^k(x)g^k(x) (G_\rho * b^k)(x)\,dx}{\int_{\Omega} u_i^k(x) (G_\rho * (b^k)^2)(x)\,dx},\, i=1,\dots,n .
\end{equation}

\subsection{Minimizer of subproblem with respect to $b$}
For fixed $\mathbf{c}^{k+1}$, $g^k$, and $\mathbf{u}^k$, the subproblem with respect to $b$ is formulated as
\begin{equation*} 
  E_b(b)=\sum_{i=1}^{n} \lambda_i \int_{\Omega}\int_{\Omega} u_i^k(x)G_\rho(y-x)(g^k(x)-b(y)c_i^{k+1})^2\,dydx .
\end{equation*}
The corresponding Euler--Lagrange equation is as follows:
\begin{equation*}
  \sum_{i=1}^{n}2\lambda\int_{\Omega}u_i^k(x)G_\rho(y-x)(g^k(x)-b(y)(c^{k+1}_i))c^{k+1}_i\,dx=0 .
\end{equation*}
The minimizer of $E_b$ is given by
\begin{equation}\label{eq:b}
  b^{k+1}(y)=\frac{\sum_{i=1}^{n}\lambda_i c_i^{k+1} (G_\rho*(u^k_ig^k))(y)}{\sum_{i=1}^{n}\lambda_i (c^{k+1}_i)^2 (G_\rho*u^k_i)(y)} .
\end{equation}

\subsection{Minimizer of Subproblem with Respect to $g$}
In order to compute the gradient flow efficiently and stably, the scalar auxiliary variable (SAV) algorithm was proposed in \cite{shen2018scalar}. The SAV algorithm has received a lot of attention because of the unconditional energy stability of the numerical schemes and the simplicity of the implementation. In \cite{liu2023efficient}, a relaxed version of the modified SAV (RMSAV) method was proposed for energy minimization problems without linear terms. To solve the minimization subproblem efficiently, we employ the RMSAV algorithm to obtain the minimizer. For fixed $\mathbf{c}^{k+1}$, $b^{k+1}$, and $\mathbf{u}^k$, the subproblem with respect to $g$ is formulated as
\begin{equation}\label{eq:E-g}
\begin{split}
    E_g(g)= &  \sum_{i=1}^{n}\lambda_i\int_{\Omega}\int_{\Omega}u^k(x)G_\rho(y-x)(g(x)-b^{k+1}(y)c_i^{k+1})^2\,dydx \\
     & +\gamma\int_{\Omega}(g(x)-f(x)\log g(x))\,dx+\nu\int_{\Omega} \alpha(x)|\nabla g(x)|\,dx .
\end{split}
\end{equation}
The $L^2$-gradient flow of the functional \eqref{eq:E-g} is
\begin{equation*}
\left\{
\begin{aligned}
   &  \frac{\partial g}{\partial t}=\sum_{i=1}^{n}2\lambda_i \left( c^{k+1}_i(G_\rho *b^{k+1})u^k_i-(\mathbf{1}_G)u^k_ig\right)+\gamma\frac{f-g}{g}+\nu\text{div}\left(\alpha(x) \frac{\nabla g}{|\nabla g|} \right) , \\
   & \frac{\partial g}{\partial \overrightarrow{\mathbf{n}}} =0,\ \text{on} \  \partial \Omega ,
\end{aligned}
\right.
\end{equation*}
where $\mathbf{1}_G(x)=\int_{\Omega} G(y-x) \,dy$. \\
We denote
\begin{equation} \label{eq:F'}
     F'(g) = \sum_{i=1}^{n}2\lambda_i u^k_i\left((\mathbf{1}_G)g-c^{k+1}_iG_\rho *b^{k+1}\right) -\gamma\frac{f-g}{g}-\nu\text{div}\left( \alpha(x)\frac{\nabla g}{|\nabla g|}\right)  .
\end{equation}
We reformulate the above system of equations in the following form:
\begin{equation*}
\left\{
\begin{aligned}
& \frac{\partial g}{\partial t} + \mathcal{L}g +F'(g)-\mathcal{L}g=0 , \\
& \frac{\partial g}{\partial \overrightarrow{\mathbf{n}}} =0,\ \text{on} \  \partial \Omega ,
\end{aligned}
\right.
\end{equation*}
where $\mathcal{L}$ is a non-negative linear operator. In this paper, we take $\mathcal{L}g=\Delta^2g=g_{xxxx}+2g_{xxyy}+g_{yyyy}$.

The energy functional $E_g(g)$ has a lower bound. Thus, there exists a constant $C_0 > 0$ such that $E_g(g) + C_0 > 0$. Let the scalar auxiliary variable be $z=\sqrt{E_g(g)+C_0}$. The gradient flow can be rewritten as
\begin{equation} \label{eq:SAV}
\left\{
\begin{aligned}
   &  \frac{\partial g}{\partial t} +\mathcal{L}g + \frac{F'(g)}{\sqrt{E_g(g)+C_0}}z-\mathcal{L}g =0 , \\
   & \frac{\partial z}{\partial t} = \frac{1}{2\sqrt{E_g(g)+C_0}}\int_{\Omega}F'(g)\frac{\partial g}{\partial t} \,dx .
\end{aligned}
\right.
\end{equation}
Jiang et al. \cite{jiang2022improving} proposed a relaxed SAV (RSAV) method, which ensures that the numerical results satisfy the original energy law and guarantees energy dissipation.
Based on the RSAV method and with time step $\Delta t$, the first-order scheme of \eqref{eq:SAV} is given as follows:
\begin{subequations}
\begin{align}
   &  \frac{g^{k,j+1}-g^{k,j}}{\Delta t} + \mathcal{L}g^{k,j+1} + \frac{\tilde{z}^{k,j+1}}{\sqrt{E_g(g^{k,j})+C_0}}F'(g^{k,j}) - \mathcal{L}g^{k,j} = 0 \label{eq:SAV-g} , \\
   & \frac{\tilde{z}^{k,j+1}-z^{k,j}}{\Delta t} = \frac{1}{2\sqrt{E_g(g^{k,j})+C_0}}\int_{\Omega}F'(g^{k,j})\frac{g^{k,j+1}-g^{k,j}}{\Delta t} \,dx \label{eq:SAV-z} , \\
   & z^{k,j+1}=\xi \tilde{z}^{k,j+1} + (1-\xi)\sqrt{E_g(g^{k,j+1})+C_0} , \label{eq:SAV-rez}
\end{align}
\end{subequations}
where $k$ and $j$ denote the outer and inner iteration indices, respectively. Let $g^k$ denote the estimate of the outer iteration at step $k$, then set $g^{k,1}=g^k$, i.e., $g^k$ is the initial value of the inner iteration. When $g^{k,j+1}$ satisfies the stopping condition, update the outer iteration by setting $g^{k+1}=g^{k,j+1}$. For notational simplicity, throughout this subsection, we omit the outer iteration index $k$ and simply write $g^j$, $\tilde{z}^j$, and $z^j$ for $g^{k,j}$, $\tilde{z}^{k,j}$, and $z^{k,j}$, respectively.

Here, $\xi$ denotes a scalar relaxation parameter selected from the set
\begin{equation*}
    \Lambda_{\xi}=\left\{ \xi \in [0,1]:(z^{j+1})^2-(\tilde{z}^{j+1})^2-(\tilde{z}^{j+1}-z^j)^2\leq\eta\mathcal{G}(g^{j+1},g^j)\right\} ,
\end{equation*}
where $\mathcal{G}(g^{j+1},g^j)=\frac{1}{\Delta t}\left( \left( g^{j+1}-g^j \right), A\left( g^{j+1}-g^j \right) \right)$, with $A=I+\Delta t \mathcal{L}$ and $\eta\in [0,1]$ is a parameter. In our experiments, we set $\eta=0.99$.

By substituting \eqref{eq:SAV-g} into \eqref{eq:SAV-z}, we obtain
\begin{align*}
  & \tilde{z}^{j+1}=\frac{z^j}{1+\frac{\Delta t}{2}(m^j,\hat{m}^j)} ,  \\
  & g^{j+1}=g^{j}-\Delta t \hat{m}^j \tilde{z}^{j+1} ,
\end{align*}
where $m^j=\frac{F'(g^j)}{\sqrt{E_g(g^{j})+C_0}}$ and $\hat{m}^j=A^{-1}m^j$.
As $\xi \in\Lambda_{\xi}$, it follows that $\xi$ satisfies
\begin{equation*}
    q\xi^2+2d\xi+h\leq 0 ,
\end{equation*}
where
\begin{align*}
   & q=\left( \tilde{z}^{j+1}-\sqrt{E_g(g^{j+1})+C_0} \right)^2 , \\
   & d=2\left( \tilde{z}^{j+1}-\sqrt{E_g(g^{j+1})+C_0} \right)\sqrt{E_g(g^{j+1})+C_0} , \\
   & h=E_g(g^{j+1})+C_0-(\tilde{z}^{j+1})^2-(\tilde{z}^{j+1}-z^j)^2-\eta\mathcal{G}(g^{j+1},g^j) .
\end{align*}
In particular, following \cite{jiang2022improving}, we choose
\begin{equation}
    \xi = \max \left\{ 0,\ \frac{-d-\sqrt{d^2-4qh}}{2q} \right\} .
\end{equation}
Taking the inner product of \eqref{eq:SAV-g} with $g^{j+1}-g^j$, multiplying \eqref{eq:SAV-z} by $2\tilde{z}^{j+1}$, and then combining the two, we obtain
\begin{equation} \label{eq:SAV-G}
    \mathcal{G}(g^{j+1},g^j)=\frac{1}{\Delta t}\left( g^{j+1}-g^j, A(g^{j+1}-g^j)\right)=-2(\tilde{z}^{j+1})^2+2\tilde{z}^{j+1}z^j .
\end{equation}
Therefore, when $\xi$ is nonzero, it can be expressed as
\begin{equation} \label{eq:SAV-xi}
\begin{aligned}
      \xi  = & \frac{-d-\sqrt{d^2-4qh}}{2q} \\
       =&-\frac{\sqrt{(\sqrt{E_g(g^{j+1})+C_0}-\tilde{z}^{j+1})^2((1-\eta)((\tilde{z}^{j+1})^2+(\tilde{z}^{j+1}-z^j)^2)+\eta(z^j)^2}}{(\sqrt{E_g(g^{j+1})+C_0}-\tilde{z}^{j+1})^2}   \\
       & + \frac{\sqrt{E_g(g^{j+1})+C_0}}{\sqrt{E_g(g^{j+1})+C_0}-\tilde{z}^{j+1}}
\end{aligned}
\end{equation}


\begin{theorem}
  For any $\Delta t$, the scheme \eqref{eq:SAV-g}--\eqref{eq:SAV-rez} is unconditionally stable in the sense that
  \begin{equation*}
    ({z}^{j+1})^2-(z^{j})^2=-(1-\eta)\mathcal{G}(g^{j+1},g^j)\leq 0 .
  \end{equation*}
\end{theorem}
\begin{proof}
From \eqref{eq:SAV-G}, we obtain
\begin{equation*}
    (\tilde{z}^{j+1}-z^j)^2-(z^j)^2=-\mathcal{G}(g^{j+1},g^j)-(\tilde{z}^{j+1})^2 .
\end{equation*}
It follows that
\begin{equation*}
    (z^{j+1})^2-(z^j)^2=-\mathcal{G}(g^{j+1},g^j)-(\tilde{z}^{j+1}-z^j)^2-(\tilde{z}^{j+1})^2+(z^{j+1})^2 .
\end{equation*}
Because $\xi \in \Lambda_{\xi}$, we can obtain
\begin{equation*}
\begin{aligned}
     (z^{j+1})^2-(z^j)^2 & =-\mathcal{G}(g^{j+1},g^j)-(\tilde{z}^{j+1}-z^j)^2-(\tilde{z}^{j+1})^2+(z^{j+1})^2   \\
     & \leq -(1-\eta)\mathcal{G}(g^{j+1},g^j) \leq 0 .
\end{aligned}
\end{equation*}

\end{proof}


\subsection{Minimizer of Subproblem with Respect to $\mathbf{u}$}
 To efficiently solve $u$, we employ the ICTM \cite{wang2022iterative}. For fixed $\mathbf{c}^{k+1}$, $b^{k+1}$, and $g^{k+1}$, the subproblem with respect to $\mathbf{u}$ is formulated as
\begin{equation}\label{eq:E-u}
    E_\mathbf{u}(\mathbf{u}) =\sum_{i=1}^{n} \lambda_i\int_{\Omega}u_ie_i\,dx + \mu \sum_{i=1}^{n} \sum_{j=1,j \neq i}^{n} \sqrt{\frac{\pi}{\tau}} \int_{\Omega} u_i (G_{\tau} * u_j) \,dx ,
\end{equation}
where $\mathbf{u}\in\Lambda_1$, and $e_i=\int_{\Omega}G_\rho(y-x)(g^{k+1}(x)-b^{k+1}(y)c^{k+1}_i)^2\,dy$.
Due to the non-convexity of the set $\Lambda_1$, the subproblem is difficult to solve numerically. To address this, we adopt the relaxation strategy proposed in \cite{wang2022iterative}, and the relaxed formulation is given by
\begin{equation}\label{eq:Eu-relaxed}
  \min_{\mathbf{u}\in\Lambda_2} E_\mathbf{u}(\mathbf{u}) =\sum_{i=1}^{n} \lambda_i\int_{\Omega}u_ie_i\,dx + \mu \sum_{i=1}^{n} \sum_{j=1,j \neq i}^{n} \sqrt{\frac{\pi}{\tau}} \int_{\Omega} u_i (G_{\tau} * u_j) \,dx ,
\end{equation}
where $\Lambda_2= \left\{\mathbf{u}(x) \in BV(\Omega,\mathbb{R}^n) \mid u_i(x)\in [0,1] ,\ \sum_{i=1}^{n}u_i(x)=1 \right\}$ is the convex hull of $\Lambda_1$.

\begin{theorem}
  Let $\hat{\mathbf{u}} \in \Lambda_2$ be a minimizer of the relaxed problem \eqref{eq:Eu-relaxed}. Then $\hat{\mathbf{u}} \in \Lambda_1$, and $\hat{\mathbf{u}}$ is also a minimizer of the original problem \eqref{eq:E-u}.
\end{theorem}

\begin{proof}
  Let $\hat{\mathbf{u}}\in\Lambda_2$ be the minimizer of the relaxed model \eqref{eq:Eu-relaxed}. Since $\Lambda_1 \subset \Lambda_2$, we obtain
  \begin{equation*}
    E_\mathbf{u}(\hat{\mathbf{u}})=\min_{\mathbf{u}\in\Lambda_2} E_\mathbf{u}(\mathbf{u})\leq \min_{\mathbf{u}\in\Lambda_1} E_\mathbf{u}(\mathbf{u}) .
  \end{equation*}
  It suffices to prove that $\hat{\mathbf{u}}\in\Lambda_1$.

  We now prove this by contradiction. Assume that $\hat{\mathbf{u}}\notin\Lambda_1$; that is, there exists a measurable set $D\subset \Omega$ of positive measure and a constant $d\in(0,0.5)$ such that for some indices $i\ne j$, we have $\hat{u}_i(x),\ \hat{u}_j(x)\in(d,1-d)$ a.e. $x\in D$. Assume without loss of generality that they are $\hat{u}_1$ and $\hat{u}_2$.
  Let $\chi_D(x)$  denote the characteristic function of $D$.
  \begin{equation*}
    \chi_D(x)=
    \begin{cases}
      1, & x\in D , \\
      0, & \text{otherwise}.
    \end{cases}
  \end{equation*}
  For any $s\in (-d,d)$, let $u_1^s=\hat{u}_1+s\chi_D,\ u_2^s=\hat{u}_2-s\chi_D$ and $u_i^s=\hat{u}_i,\ i=3,\dots,n$. Then, it is clear that $\mathbf{u}^s=(u_1^s,\dots,u_n^s) \in \Lambda_2$. The corresponding energy functional can be written as
  \begin{equation*}
  \begin{split}
     E_\mathbf{u}(\mathbf{u}^s) = & \sum_{i=1}^{n} \lambda_i\int_{\Omega}u_i^se_i\,dx + \mu \sum_{i=1}^{n} \sum_{j=1,j \neq i}^{n} \sqrt{\frac{\pi}{\tau}} \int_{\Omega} u_i^s (G_{\tau} * u_j^s) \,dx \\
      = & \lambda_1\int_{\Omega}(\hat{u}_1+s\chi_D)e_1\,dx + \lambda_2\int_{\Omega}(\hat{u}_2-s\chi_D)e_2\,dx + \sum_{i=3}^{n} \lambda_i\int_{\Omega}\hat{u}_ie_i\,dx \\
       & +\mu\sqrt{\frac{\pi}{\tau}} \int_{\Omega} (\hat{u}_1+s\chi_D) (G_{\tau} * (1-\hat{u}_1-s\chi_D))\,dx  \\
       & +\mu\sqrt{\frac{\pi}{\tau}} \int_{\Omega} (\hat{u}_2-s\chi_D) (G_{\tau} * (1-\hat{u}_2+s\chi_D)) \,dx \\
       & +\mu\sum_{i=3}^{n} \sqrt{\frac{\pi}{\tau}} \int_{\Omega} \hat{u}_i (G_{\tau} * (1-\hat{u}_i)) \,dx . \\
  \end{split}
  \end{equation*}
  Taking the first and second derivatives of $E_\mathbf{u}(\mathbf{u}^s)$ with respect to $s$, we obtain
  \begin{equation*}
  \begin{split}
     \frac{dE_\mathbf{u}(\mathbf{u}^s)}{ds}  = & \lambda_1\int_{\Omega}\chi_De_1\,dx - \lambda_2\int_{\Omega}\chi_De_2\,dx \\
       & + \mu\sqrt{\frac{\pi}{\tau}} \int_{\Omega} \chi_D (G_{\tau} * (1-\hat{u}_1-s\chi_D))\,dx \\
       &  - \mu\sqrt{\frac{\pi}{\tau}} \int_{\Omega} \chi_D (G_{\tau} * (\hat{u}_1+s\chi_D))\,dx \\
       & - \mu\sqrt{\frac{\pi}{\tau}} \int_{\Omega} \chi_D (G_{\tau} * (1-\hat{u}_2+s\chi_D))\,dx \\
       & + \mu\sqrt{\frac{\pi}{\tau}} \int_{\Omega} \chi_D (G_{\tau} * (\hat{u}_2-s\chi_D))\,dx ,
  \end{split}
  \end{equation*}
  and
  \begin{equation*}
    \frac{d^2E_\mathbf{u}(\mathbf{u}^s)}{ds^2}=-4\mu\sqrt{\frac{\pi}{\tau}} \int_{\Omega} \chi_D (G_{\tau} * \chi_D)\,dx .
  \end{equation*}
  Since $\chi_D>0$ on $D$ and $G_\tau$ is strictly positive, it follows that $\frac{d^2E_\mathbf{u}(\mathbf{u}^s)}{ds^2}<0$ for all $s\in(-d,d)$. Therefore, $\mathbf{u}^s|_{s=0}=\hat{\mathbf{u}}$ is not a minimizer of the relaxed problem, which contradicts the assumption.
\end{proof}

Following the ICTM, we linearize the energy functional \eqref{eq:Eu-relaxed} at $\mathbf{u}^k$, leading to the corresponding linearized problem:
\begin{equation}\label{eq:Eu-liner}
  \min_{\mathbf{u}\in\Lambda_2} \mathcal{L}(\mathbf{u}^{k},\mathbf{u}) = \sum_{i=1}^{n} \lambda_i\int_{\Omega}u_i e_i\,dx + 2\mu \sum_{i=1}^{n} \sum_{j=1,j \neq i}^{n} \sqrt{\frac{\pi}{\tau}} \int_{\Omega} u_i (G_{\tau} * u_j^k) \,dx .
\end{equation}
To simplify the expression, we define
\begin{equation}\label{eq:phi}
  \phi_i^k= \lambda_i e_i + 2\mu \sqrt{\frac{\pi}{\tau}} \sum_{j=1,j \neq i}^{n} G_{\tau} * u_j^k .
\end{equation}
The linearized problem \eqref{eq:Eu-liner} can be rewritten as
\begin{equation}
\label{eq:Eu-phi}
  \mathcal{L}(\mathbf{u}^{k},\mathbf{u}) = \sum_{i=1}^{n} \int_{\Omega} u_i \phi_i^k \,dx .
\end{equation}
The solution is given by $\mathbf{u}^{k+1}=\argmin_{\mathbf{u}\in\Lambda_1}\mathcal{L}(\mathbf{u}^{k},\mathbf{u})$. Since $\mathbf{u} \in \Lambda_2$ and $\phi_i^k\geq 0$, we can obtain
\begin{equation}\label{eq:u}
  u_i^{k+1}(x)=
  \begin{cases}
    1, & \text{if } i=\argmin_{l\in[n]} \phi^{k}_l(x), \\
    0, & \text{otherwise}.
  \end{cases}
\end{equation}

\begin{theorem}
  Let $\mathbf{u}^{k+1}$ be the $k+1$ iteration obtained by \eqref{eq:u}. Then, for any $\tau > 0$, it holds that
  \begin{equation*}
    E_\mathbf{u}(\mathbf{u}^{k+1})\leq E_\mathbf{u}(\mathbf{u}^{k}) .
  \end{equation*}
\end{theorem}

\begin{proof}
  From the functional \eqref{eq:Eu-phi}, we have
  \begin{equation*}
    \mathcal{L}(\mathbf{u}^{k},\mathbf{u}^k) = E_\mathbf{u}(\mathbf{u}^k)+\mu \sqrt{\frac{\pi}{\tau}}\sum_{i=1}^{n} \sum_{j=1,j \neq i}^{n} \int_{\Omega}u_i^kG_{\tau} * u_j^k\,dx ,
  \end{equation*}
  and
    \begin{equation*}
    \begin{split}
       \mathcal{L}(\mathbf{u}^{k},\mathbf{u}^{k+1}) = & E_\mathbf{u}(\mathbf{u}^{k+1})+2\mu \sqrt{\frac{\pi}{\tau}}\sum_{i=1}^{n} \sum_{j=1,j \neq i}^{n} \int_{\Omega}u_i^{k+1}G_{\tau} * u_j^k\,dx \\
         & -\mu \sqrt{\frac{\pi}{\tau}}\sum_{i=1}^{n} \sum_{j=1,j \neq i}^{n} \int_{\Omega}u_i^{k+1}G_{\tau} * u_j^{k+1}\,dx .
    \end{split}
  \end{equation*}
  According to \eqref{eq:u}, we have $\mathcal{L}(\mathbf{u}^{k},\mathbf{u}^{k+1})\leq\mathcal{L}(\mathbf{u}^{k},\mathbf{u}^{k})$, which implies that
  \begin{equation*}
     E_\mathbf{u}(\mathbf{u}^{k+1})-E_\mathbf{u}(\mathbf{u}^k) \leq M ,
  \end{equation*}
  where
  \begin{equation*}
  \begin{aligned}
     M &  =\mu \sqrt{\frac{\pi}{\tau}}\sum_{i=1}^{n} \sum_{j=1,j \neq i}^{n} \int_{\Omega}\left[ u_i^{k+1} (G_\tau * u_j^{k+1}) + u_i^k (G_\tau * u_j^k) - 2 u_i^{k+1} ( G_\tau * u_j^k) \right]\,dx  \\
      & = -\mu \sqrt{\frac{\pi}{\tau}}\sum_{i=1}^{n}\int_{\Omega}(u_i^{k+1}-u_i^k)G_{\tau} * (u_i^{k+1}-u_i^k)\,dx \\
      & = -\mu \sqrt{\frac{\pi}{\tau}}\sum_{i=1}^{n}\int_{\Omega}[G_{\frac{\tau}{2}} * (u_i^{k+1}-u_i^k)]^2\,dx \\
      & \leq 0 .
  \end{aligned}
  \end{equation*}
This implies that $E_\mathbf{u}(\mathbf{u}^{k+1})-E_\mathbf{u}(\mathbf{u}^k) \leq 0$.
\end{proof}
\begin{remark}
    As the RMSAV algorithm only guarantees the decrease in the auxiliary energy, the reduction in energy in the two subproblems does not necessarily guarantee that $E(\mathbf{c}^{k+1},b^{k+1},g^{k+1},\mathbf{u}^{k+1})\leq E(\mathbf{c}^{k},b^{k},g^{k},\mathbf{u}^{k})$. However, the monotonic decrease within the subproblems can still provide a certain level of stability for the algorithm and help prevent divergence.
\end{remark}

The original ICTM algorithm consists of the first two convolution steps for updating $\mathbf{c}^{k+1}$ and $b^{k+1}$, followed by a thresholding step for updating $\mathbf{u}^{k+1}$.
Based on the original ICTM framework, we introduce the RMSAV algorithm with additional denoising terms to obtain a smoothed image for segmentation.
This modification significantly enhances the model’s robustness to high noise.

The algorithm used to solve the proposed model is summarized in Algorithm~\ref{alg:ICTM-SAV}.
\begin{algorithm}
\caption{The algorithm for the proposed model \eqref{eq:PoissonICTM}}
\label{alg:ICTM-SAV}
\KwIn{original image $f$, initialization $b^0,\ g^0,\ \mathbf{u}^0$, parameters $\lambda_i,\ \mu,\ \gamma,\ \nu,\ \rho,\ \tau,\ \Delta t,\ C_0$, tolerance $tol_1,\ tol_2$, error $err_1 \gg tol_1,\ err_2 \gg tol_2$}
\While {$err_1>tol_1$ }{
update $\mathbf{c}$ by \eqref{eq:c}

update $b$ by \eqref{eq:b}

compute $z_0=\sqrt{E_g(g^{k})+C_0}$

\While {$err_2>tol_2$}{
compute $F'(g^{k,j})$ by \eqref{eq:F'}

compute $m^{k,j}=\frac{F'(g^{k,j})}{\sqrt{E_g(g^{k,j})+C_0}}$

compute $\hat{m}^{k,j}=A^{-1}m^{k,j}$

compute $\tilde{z}^{k,j+1}=\frac{z^{k,j}}{1+\frac{\Delta t}{2}(m^{k,j},\hat{m}^{k,j})}$

update $g^{k,j+1}=g^{k,j}-\Delta t \tilde{z}^{k,j+1} \hat{m}^{k,j}$

compute $\xi$ by \eqref{eq:SAV-xi}

update $\xi$ by $\xi=\max\{0,\xi\}$

update $z^{k,j+1}$ by $z^{k,j+1}=\xi \tilde{z}^{k,j+1}+(1-\xi)\sqrt{E_g(g^{k,j+1})+C_0}$

update $err_2$ by $err_2=\frac{\|E_g(g^{k,j+1})-E_g(g^{k,j}) \|}{\|E_g(g^{k,j+1}) \|}$
}
update $g$ by $g^{k+1}=g^{k,j+1}$

compute $\phi_i^k = \lambda_i e_i + 2\mu \sqrt{\frac{\pi}{\tau}} \sum\limits_{j=1 , j \neq i}^{n} G_{\tau} * u_j^k$, for $i=1,\dots,n$

update $\mathbf{u}$ by
\begin{equation*}
  u_i^{k+1}(x)=
  \begin{cases}
    1, & \text{if } i=\argmin_{l\in[n]} \phi^{k}_l(x), \\
    0, & \text{otherwise};
  \end{cases}
\end{equation*}

update $err_1$ by $err_1=\|u^{k+1}-u^{k}\|_{L^2}$
}
\KwOut{$\mathbf{c}^{k+1}, b^{k+1}, g^{k+1}, \mathbf{u}^{k+1}$}
\end{algorithm}

\section{Experimental Results}
\label{sec:experiments}
In this section, segmentation experiments are conducted to evaluate the robustness of the proposed model against noise and intensity inhomogeneity. The experiments are performed on synthetic and real images exhibiting intensity inhomogeneity and corrupted by Poisson and multiplicative Gamma noise. In our algorithm, we set $\sigma = 1$, $p = 1.3$, $\tau = 0.02$, $\rho = 3$, $tol_1 = 10^{-8}$, and $tol_2 = 10^{-3}$. We perform comparative experiments with the LIC model \cite{li2011level} to demonstrate that incorporating denoising terms improves segmentation performance on noisy images. Furthermore, our model is compared with several state-of-the-art models, including the TSS model \cite{chan2014two}, the VLSGIS model \cite{ali2018image}, the NCASTV model \cite{wang2020variational}, the ICTM-CV model \cite{wang2022iterative}, and the ICTM-LVF-CV model \cite{liu2023active}. For a fair comparison, we use the same initial contours for all methods and carefully tune the parameters of each model to achieve optimal performance.

To quantitatively assess segmentation performance, we compute the Dice Similarity Coefficient (DSC), Intersection over Union (IoU), Accuracy, and $\kappa$ score ($\kappa$). Higher metric values, particularly those close to 1, indicate more accurate segmentation results.
\subsection{Robustness of Our Model}
The robustness of a segmentation model is a critical factor that influences the reliability and accuracy of its performance across diverse scenarios. In this section, we evaluate the model’s robustness with respect to noise, initialization, and denoising parameters by introducing various types of noise, randomly selecting different initial contours, and varying the denoising parameters to assess segmentation stability under these conditions.
\subsubsection{Robustness to Noise}
\begin{figure}
  \centering
  \begin{subfigure}[b]{0.16\linewidth}
      \centering
      \includegraphics[width=\textwidth]{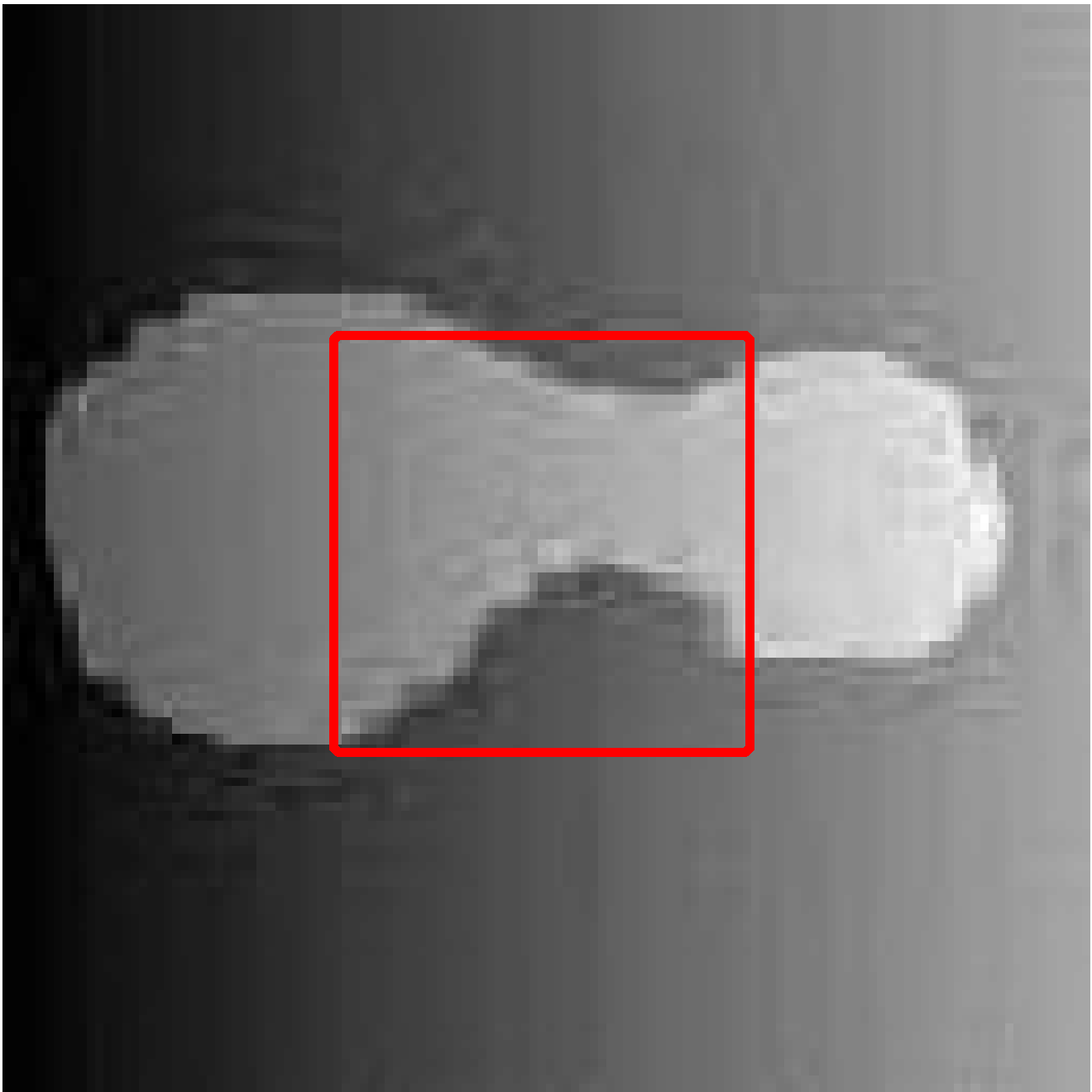}
  \end{subfigure}
  \hfill
  \begin{subfigure}[b]{0.16\linewidth}
      \centering
      \includegraphics[width=\textwidth]{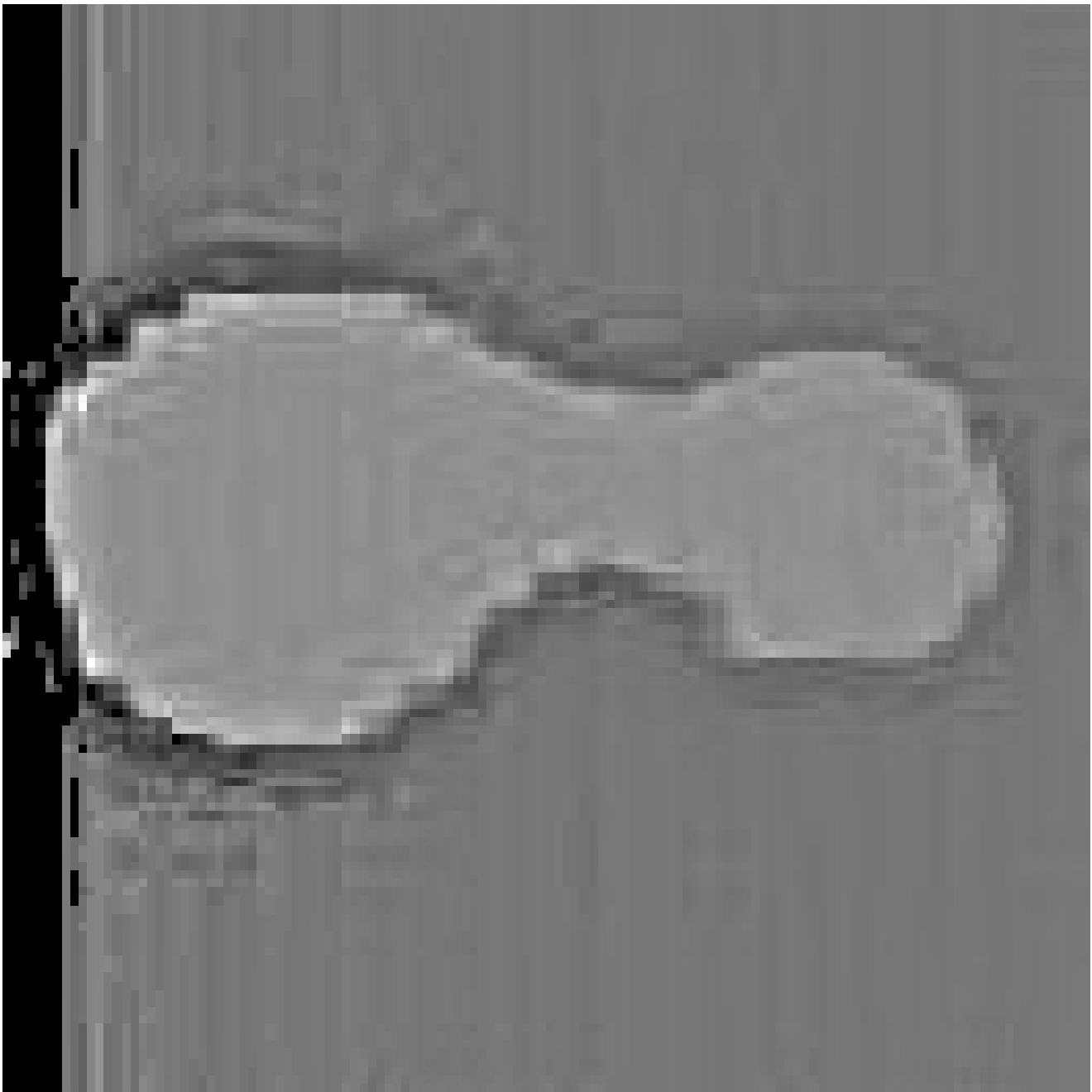}
  \end{subfigure}
 \hfill
  \begin{subfigure}[b]{0.16\linewidth}
      \centering
      \includegraphics[width=\textwidth]{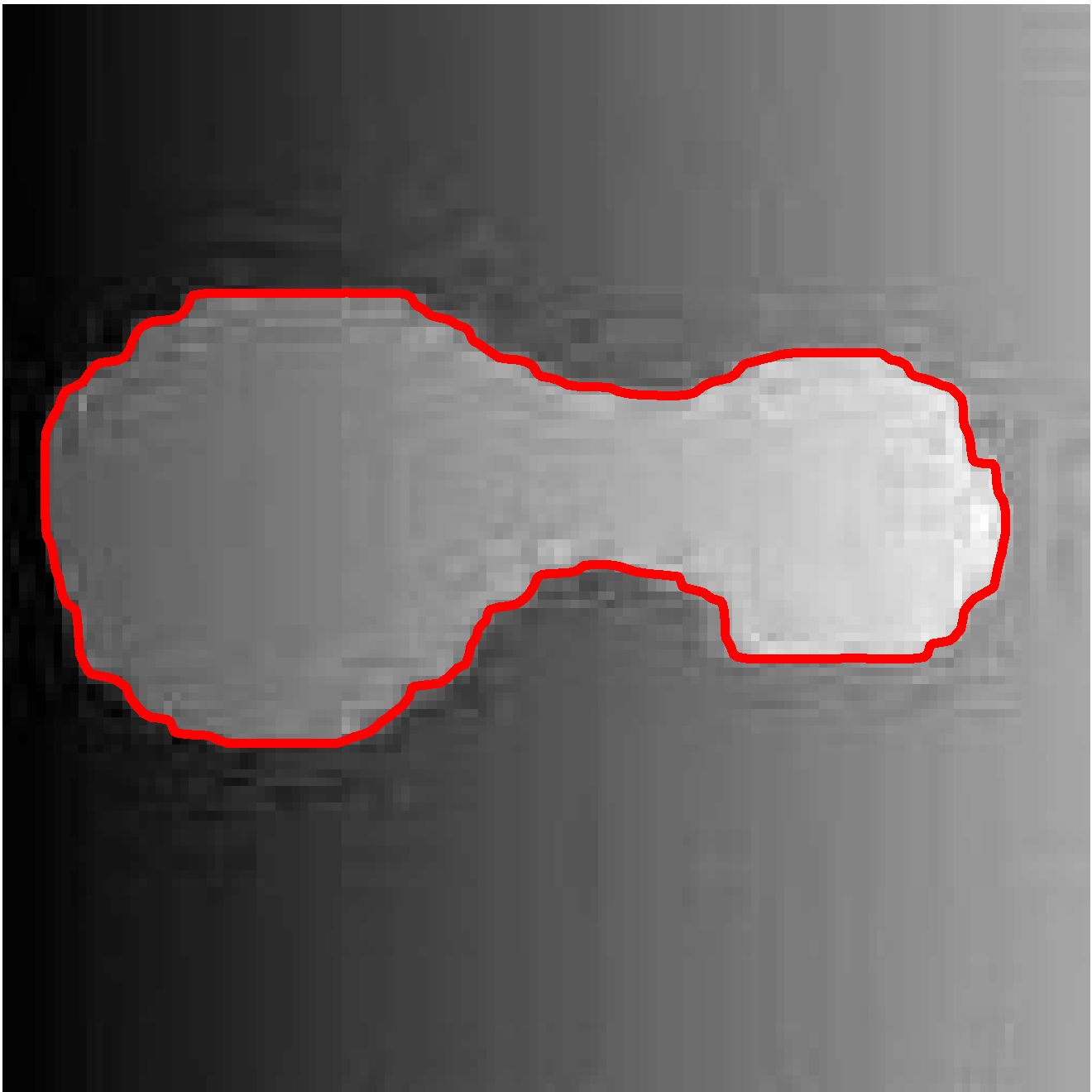}
  \end{subfigure}
  \hfill
  \begin{subfigure}[b]{0.16\linewidth}
      \centering
      \includegraphics[width=\textwidth]{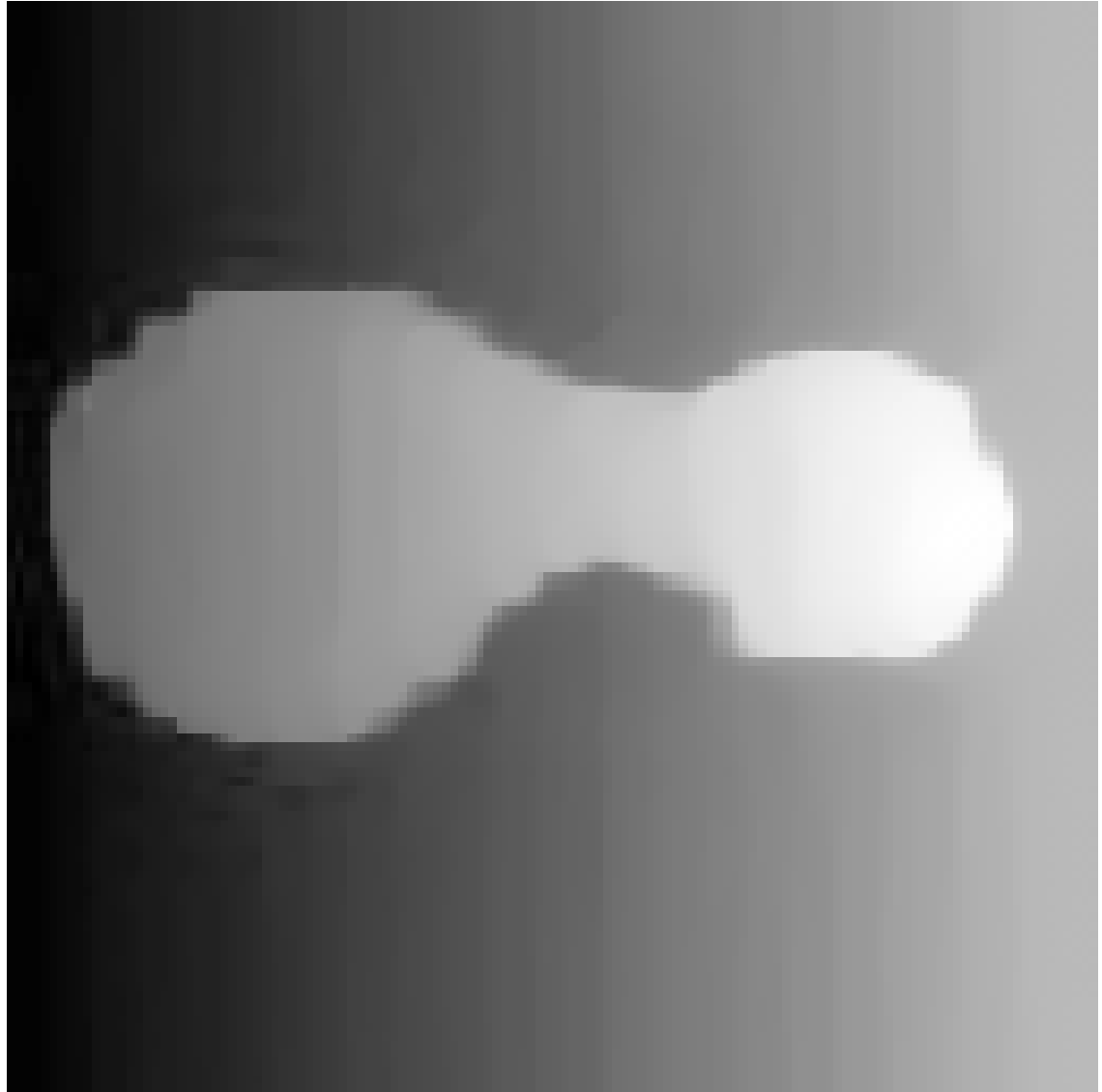}
  \end{subfigure}
  \hfill
  \begin{subfigure}[b]{0.16\linewidth}
      \centering
      \includegraphics[width=\textwidth]{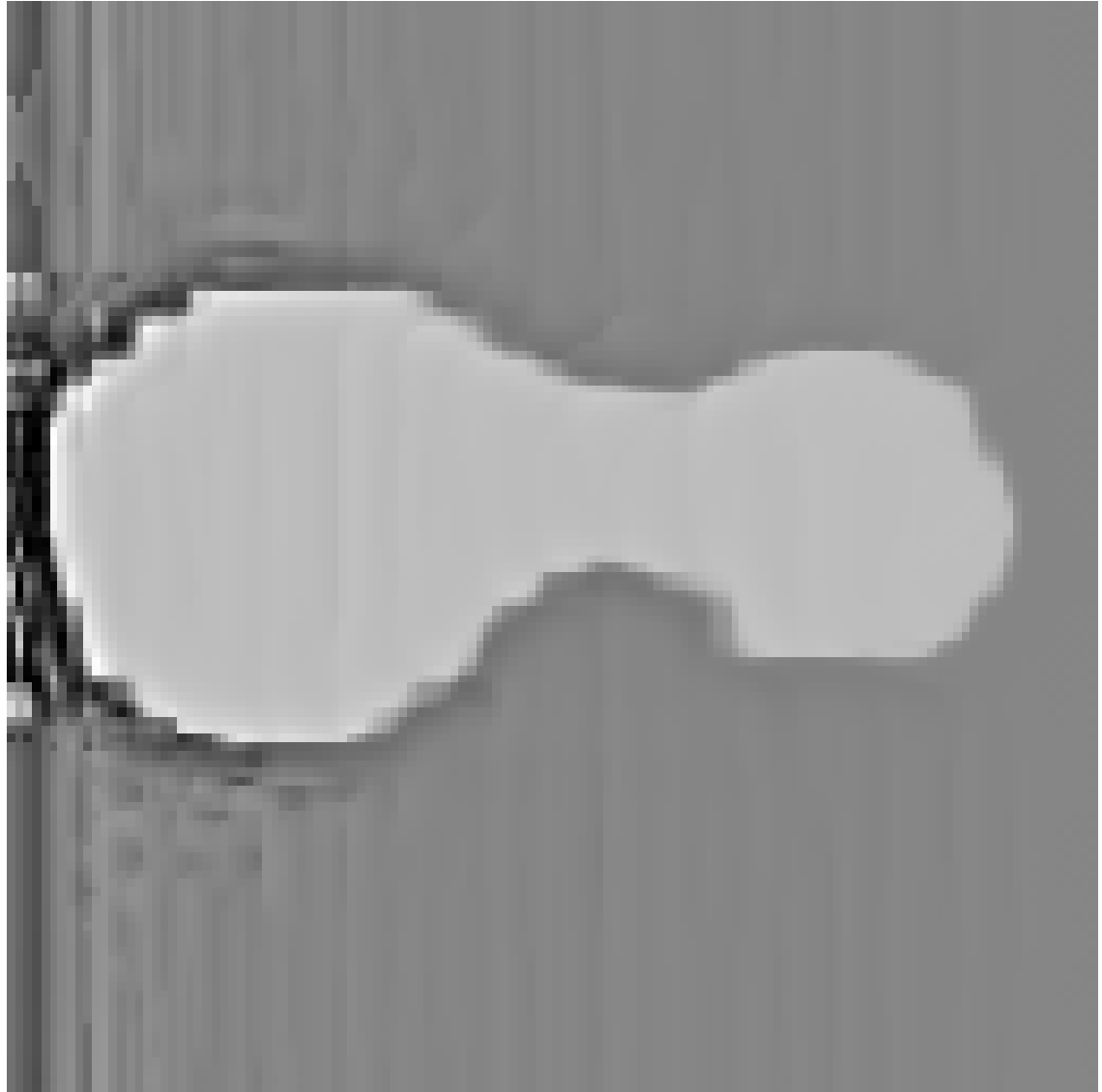}
  \end{subfigure}
  \hfill
  \begin{subfigure}[b]{0.16\linewidth}
      \centering
      \includegraphics[width=\textwidth]{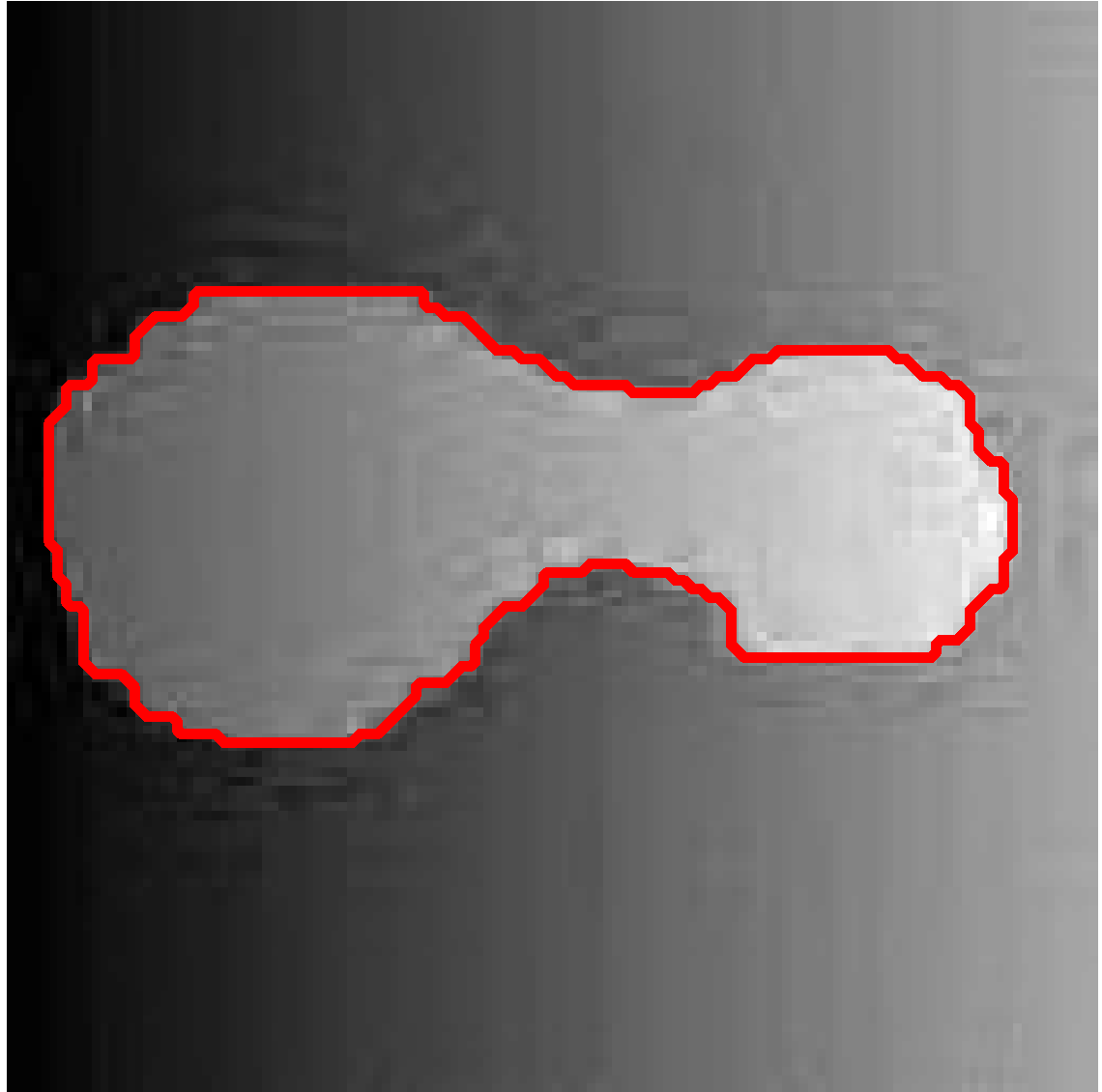}
  \end{subfigure}

  \begin{subfigure}[b]{0.16\linewidth}
      \centering
      \includegraphics[width=\textwidth]{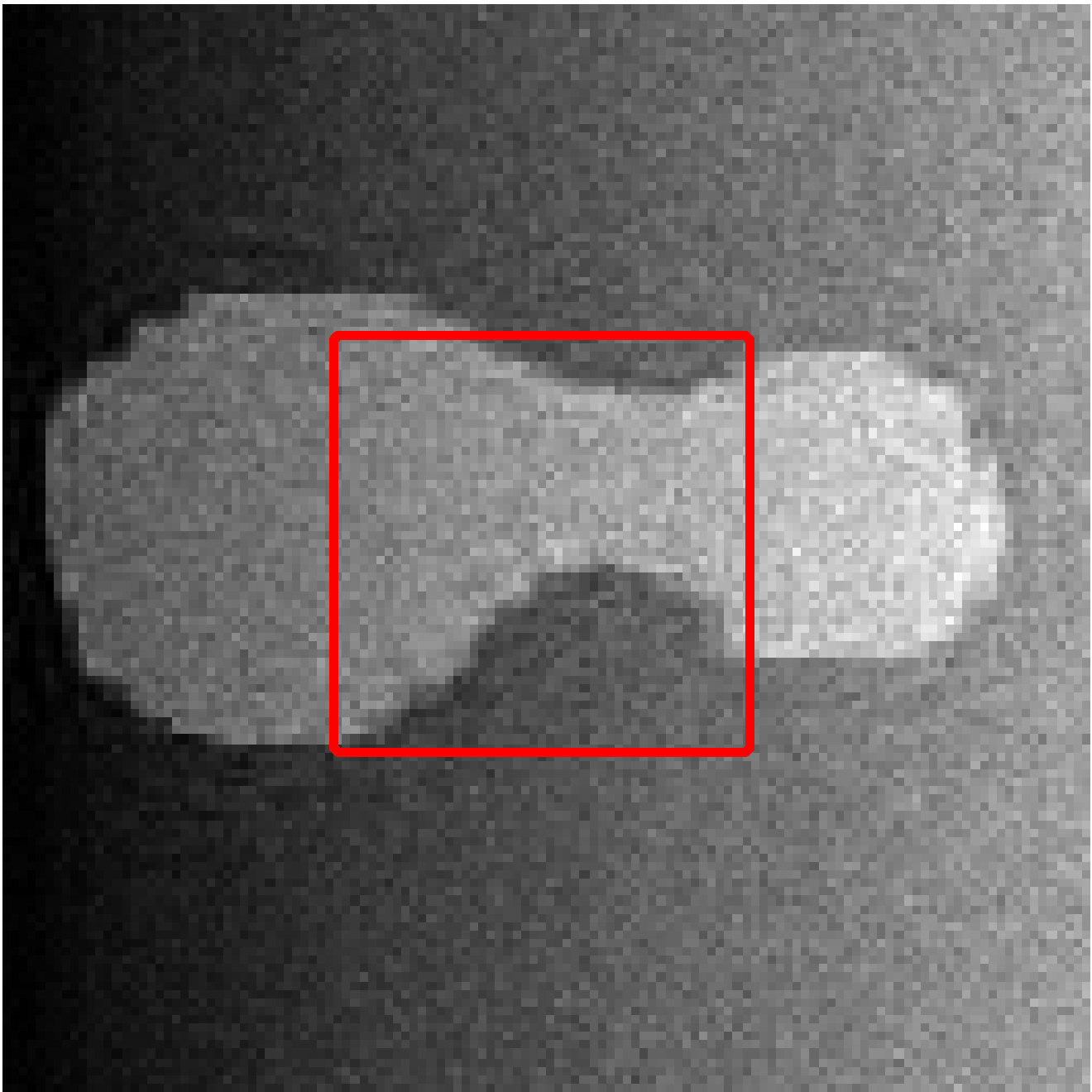}
  \end{subfigure}
   \hfill
  \begin{subfigure}[b]{0.16\linewidth}
      \centering
      \includegraphics[width=\textwidth]{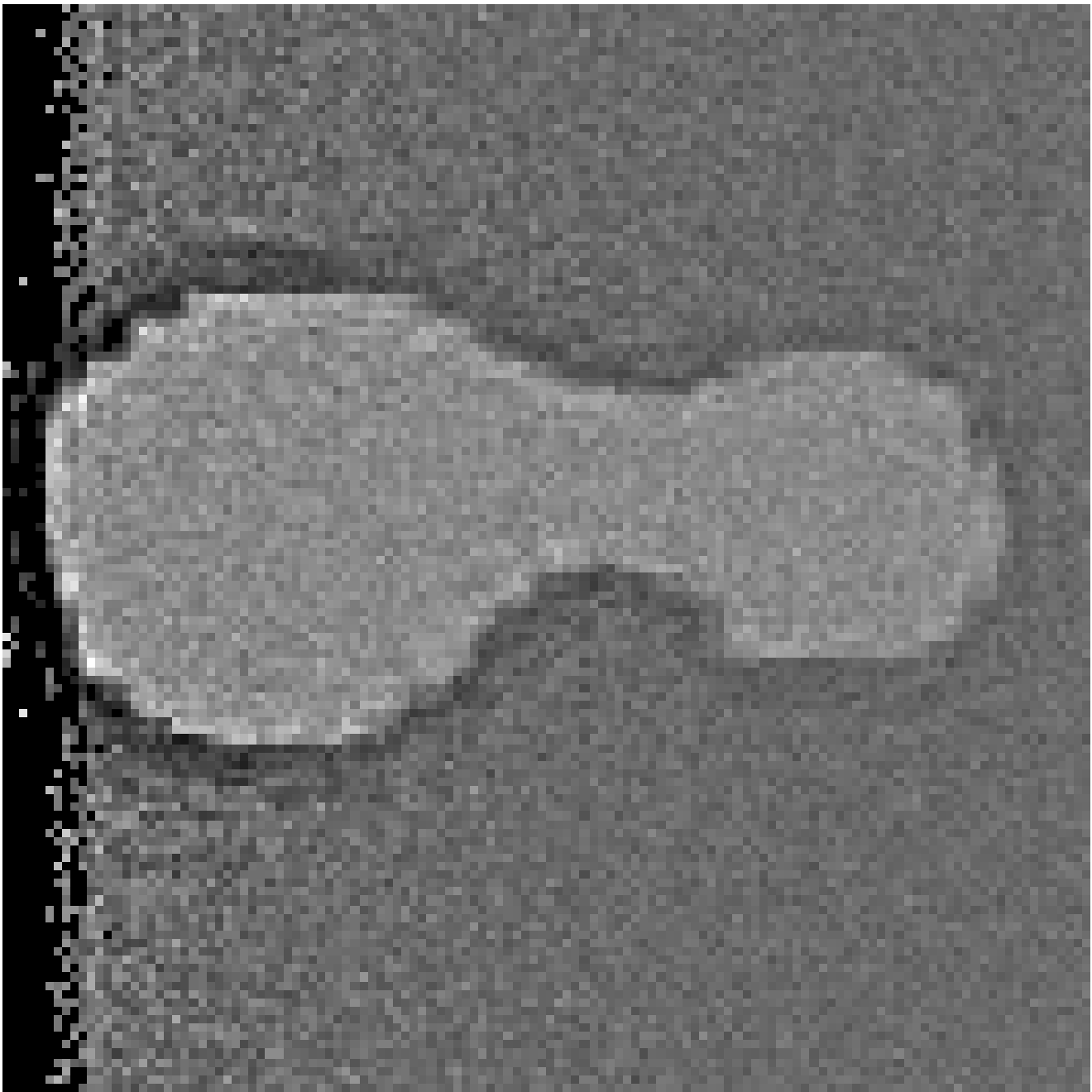}
  \end{subfigure}
   \hfill
  \begin{subfigure}[b]{0.16\linewidth}
      \centering
      \includegraphics[width=\textwidth]{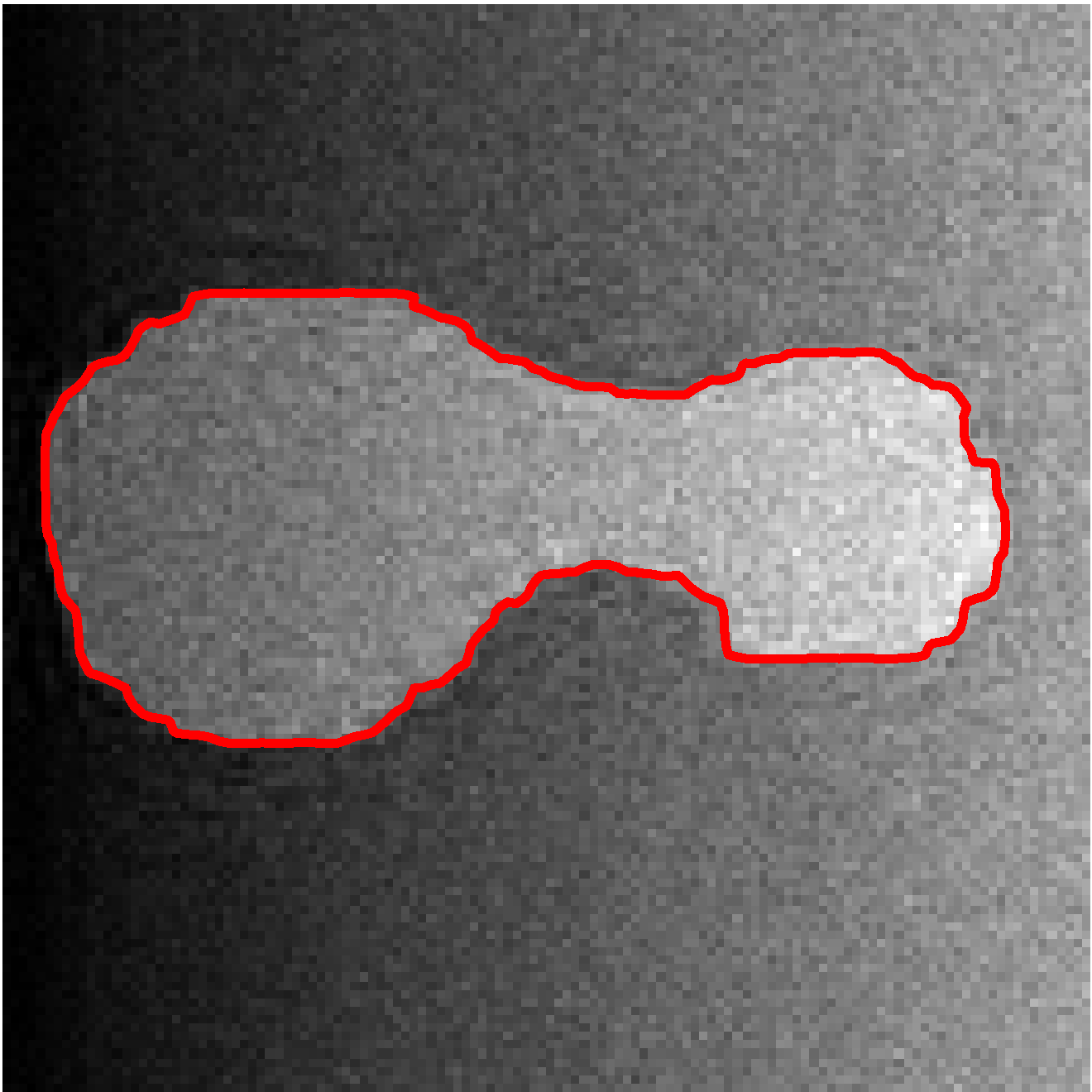}
  \end{subfigure}
   \hfill
  \begin{subfigure}[b]{0.16\linewidth}
      \centering
      \includegraphics[width=\textwidth]{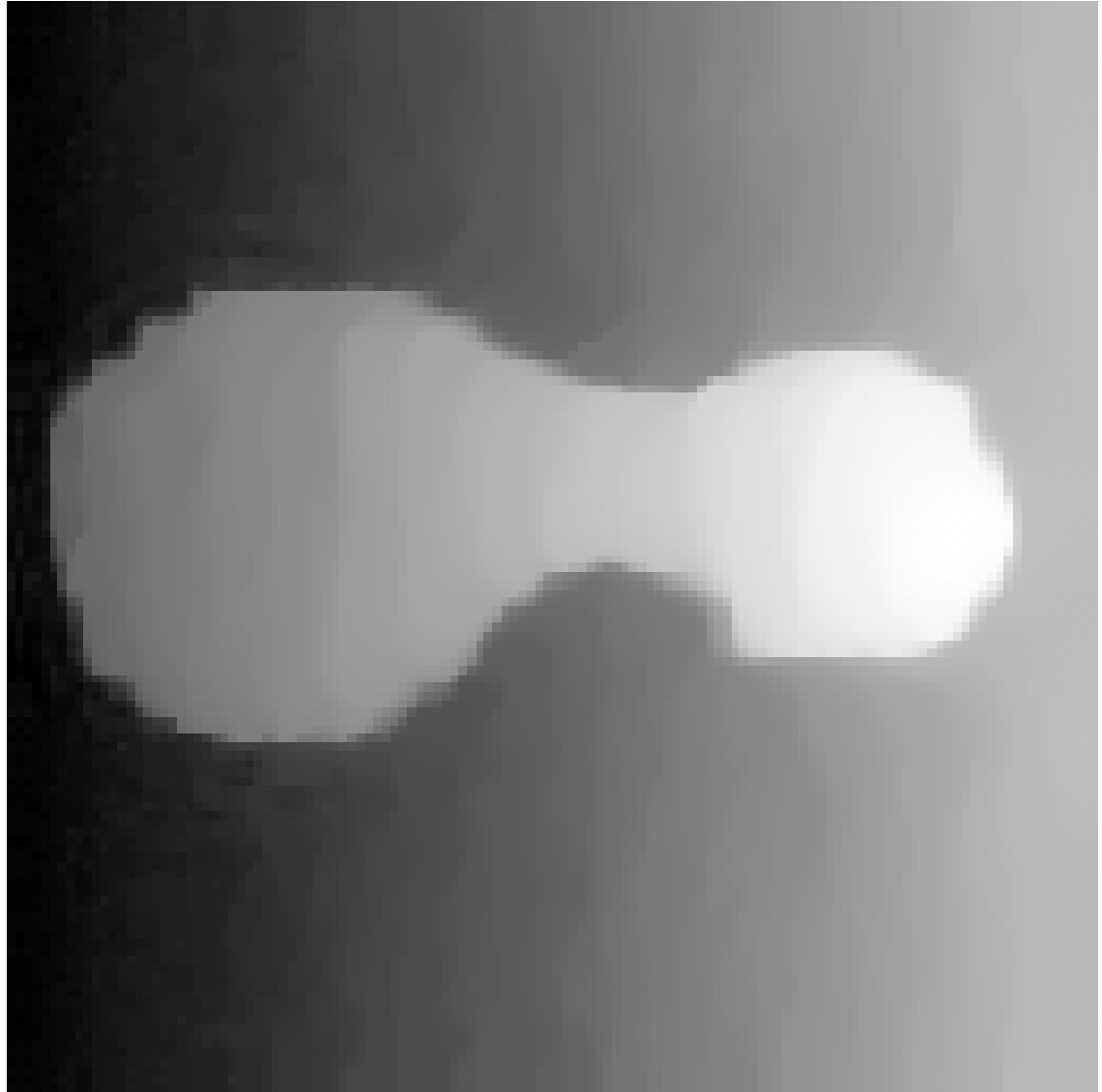}
  \end{subfigure}
   \hfill
  \begin{subfigure}[b]{0.16\linewidth}
      \centering
      \includegraphics[width=\textwidth]{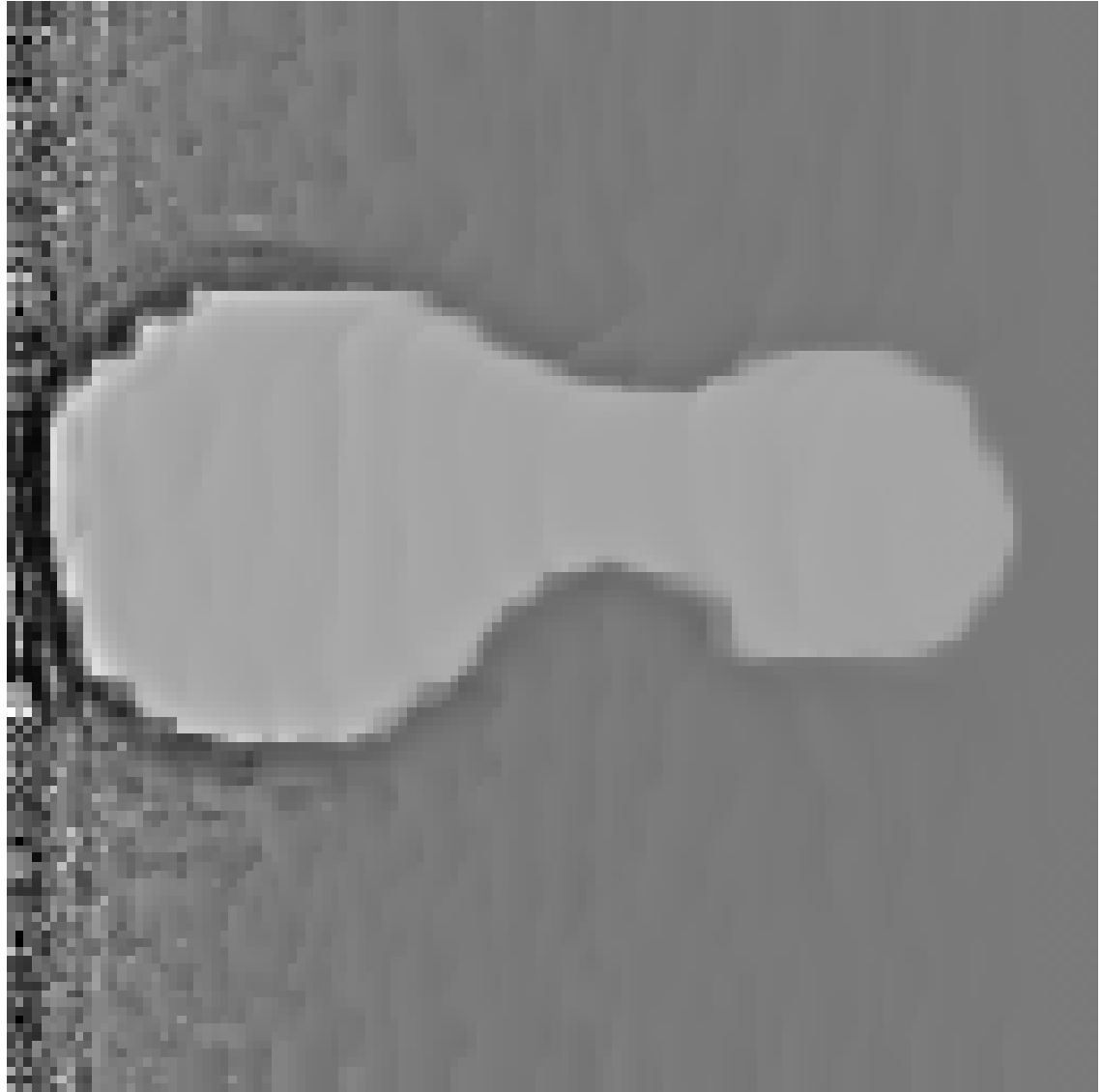}
  \end{subfigure}
   \hfill
  \begin{subfigure}[b]{0.16\linewidth}
      \centering
      \includegraphics[width=\textwidth]{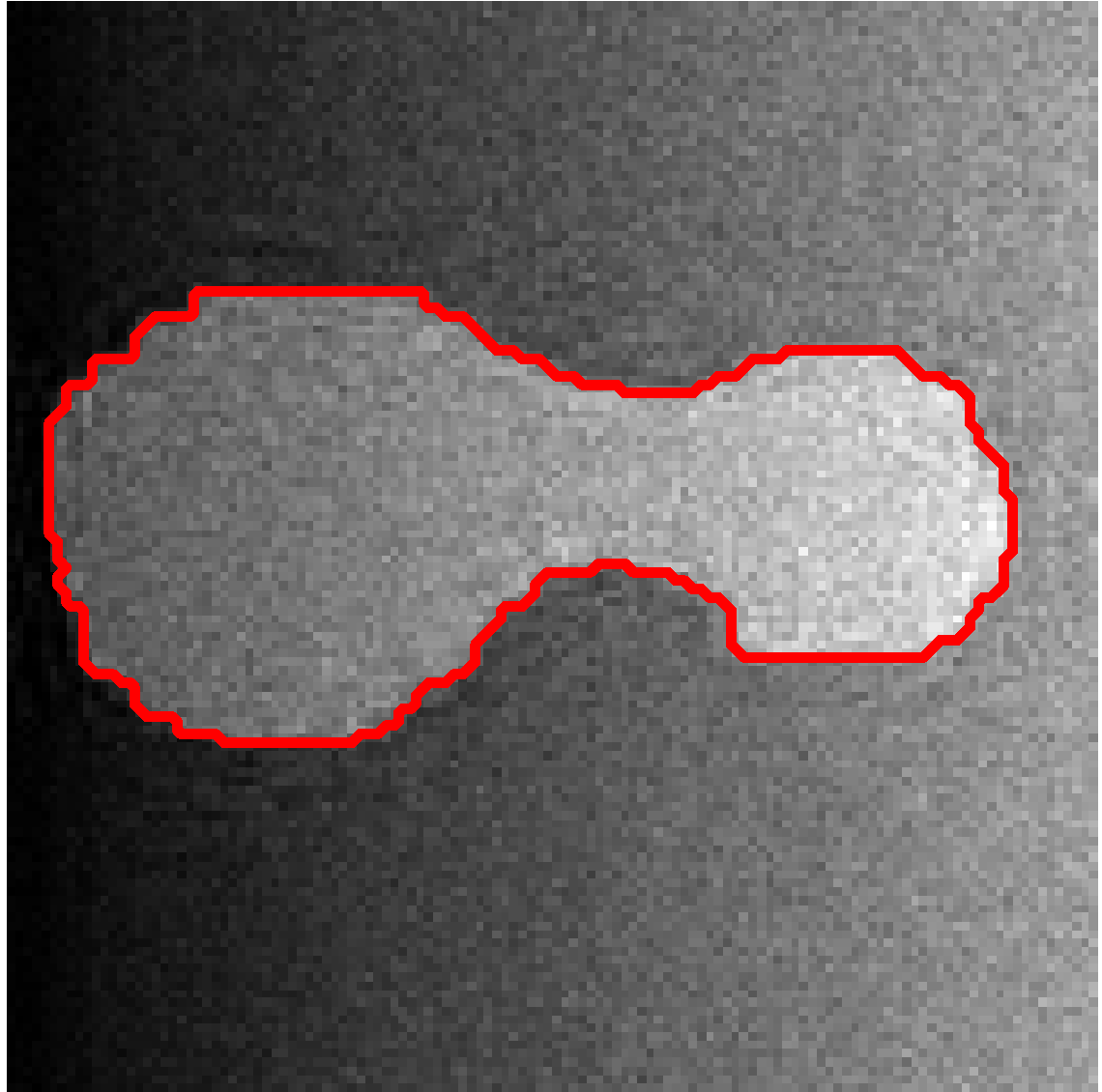}
  \end{subfigure}

  \begin{subfigure}[b]{0.16\linewidth}
      \centering
      \includegraphics[width=\textwidth]{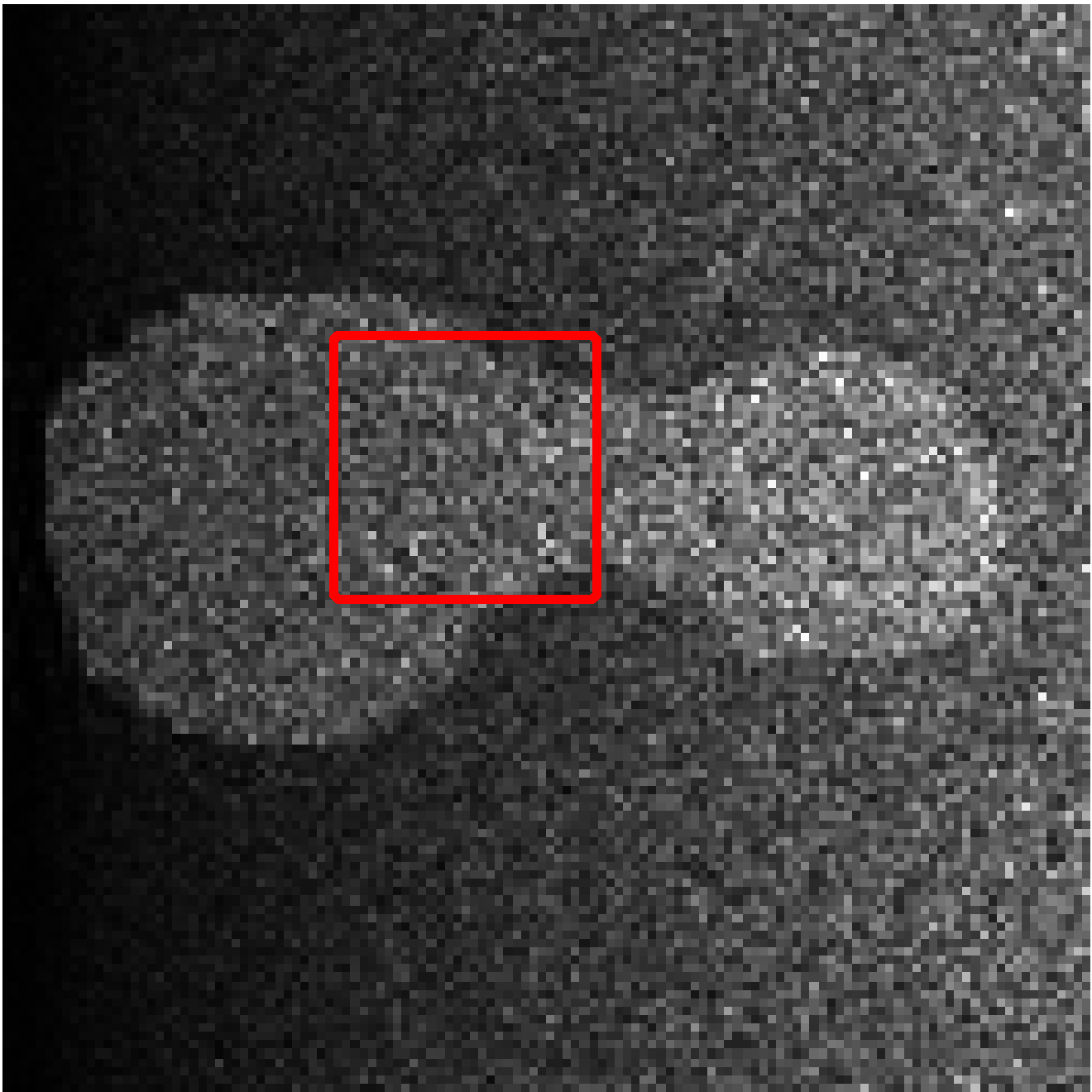}
  \end{subfigure}
   \hfill
  \begin{subfigure}[b]{0.16\linewidth}
      \centering
      \includegraphics[width=\textwidth]{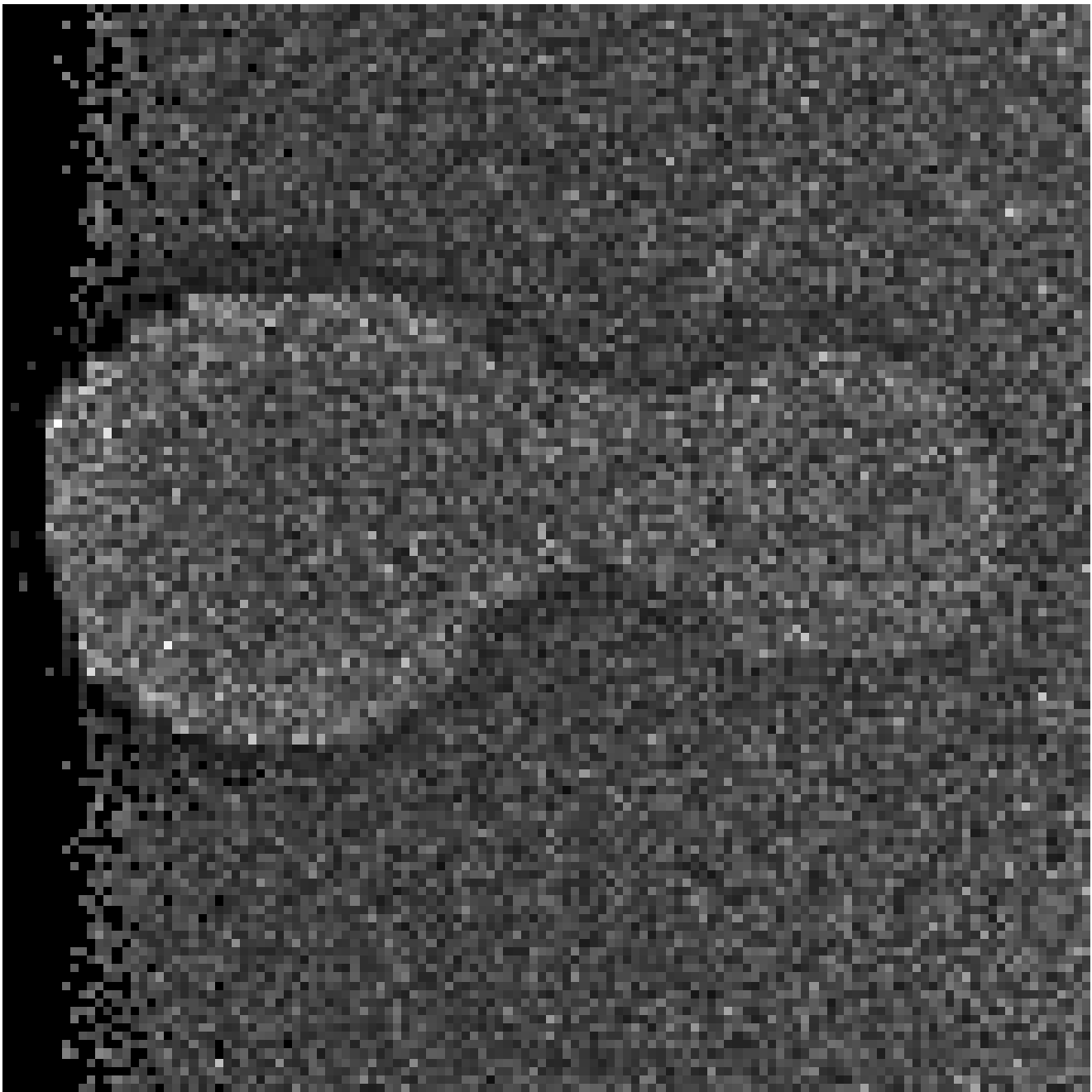}
  \end{subfigure}
   \hfill
  \begin{subfigure}[b]{0.16\linewidth}
      \centering
      \includegraphics[width=\textwidth]{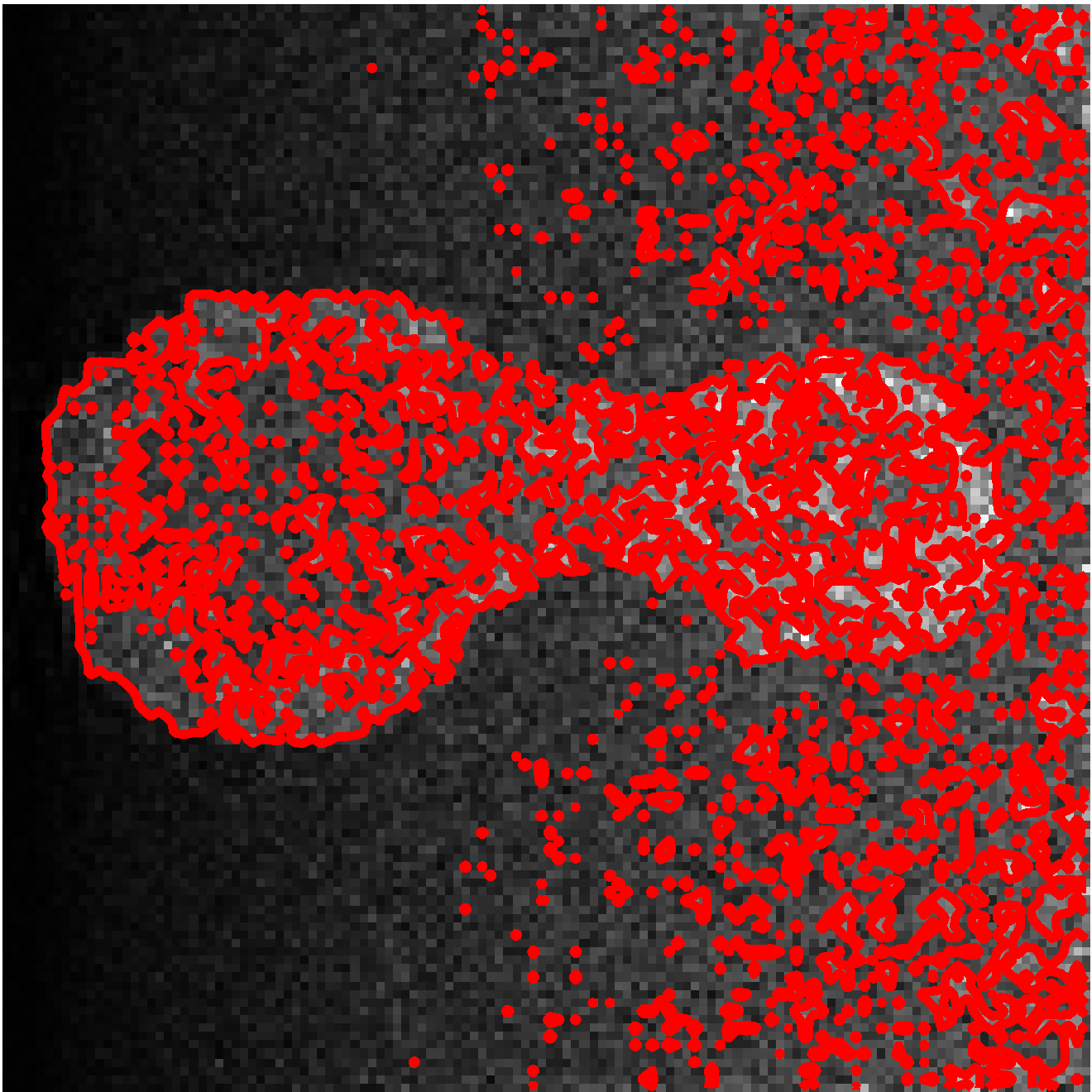}
  \end{subfigure}
   \hfill
  \begin{subfigure}[b]{0.16\linewidth}
      \centering
      \includegraphics[width=\textwidth]{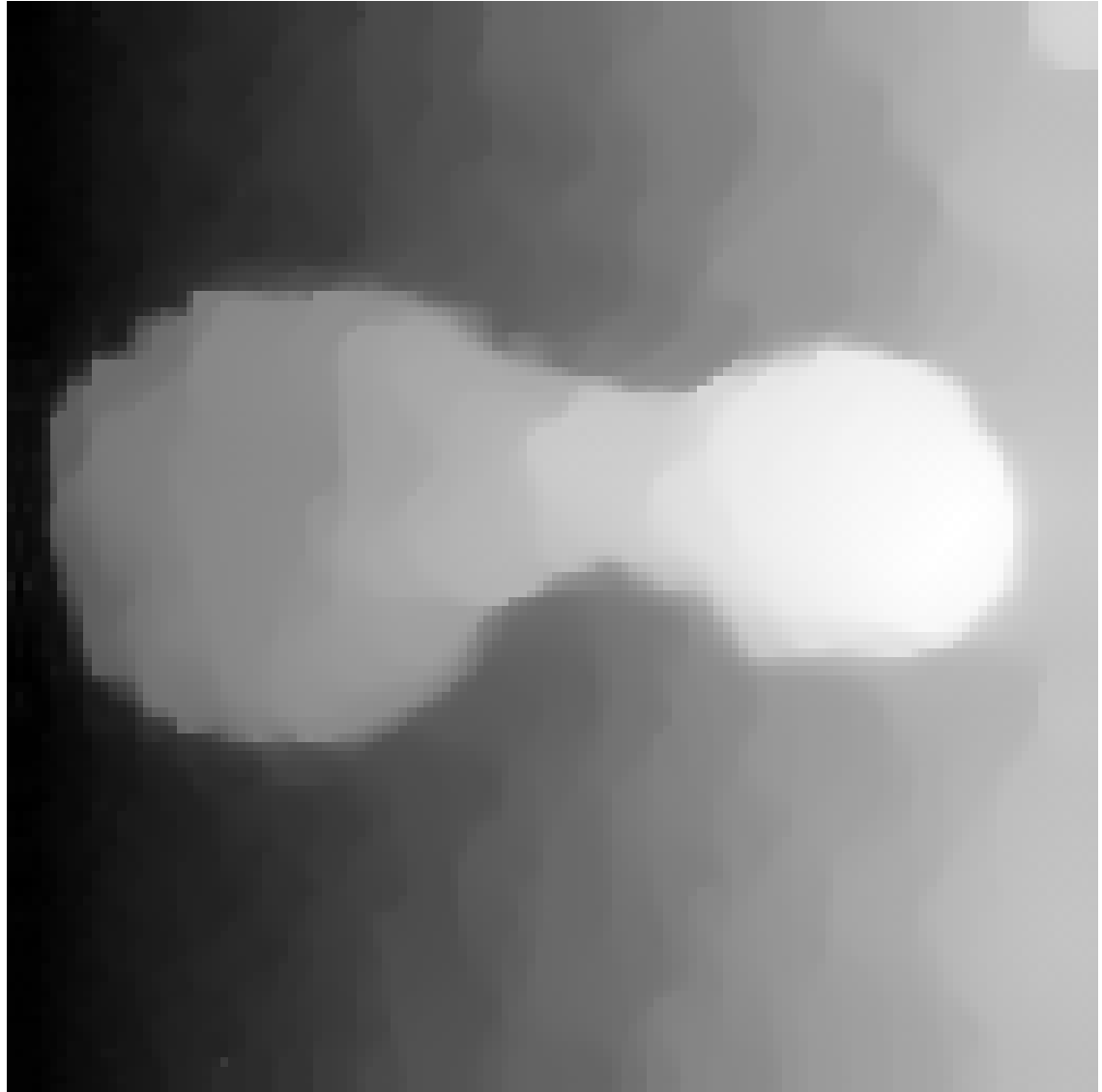}
  \end{subfigure}
   \hfill
  \begin{subfigure}[b]{0.16\linewidth}
      \centering
      \includegraphics[width=\textwidth]{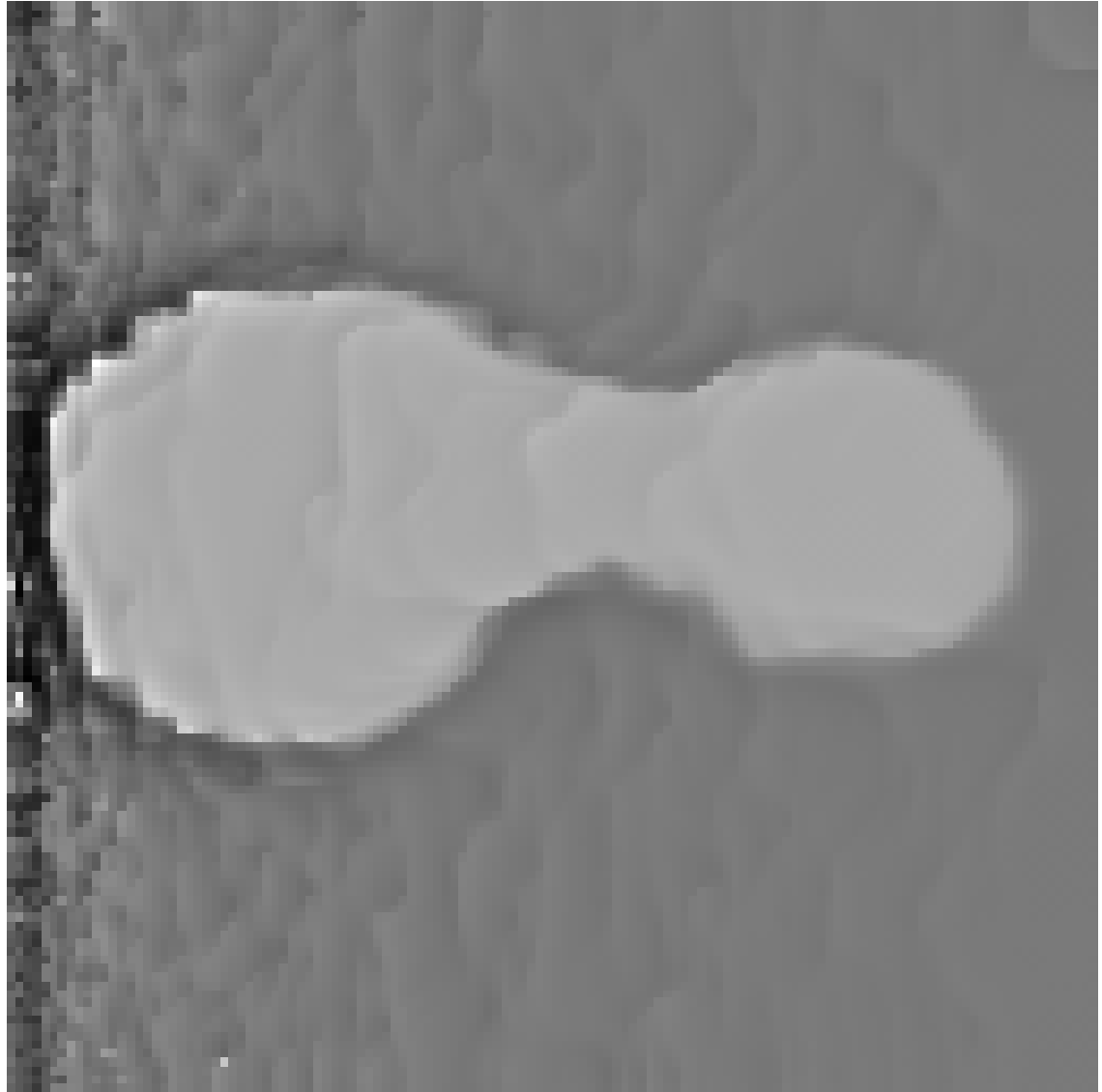}
  \end{subfigure}
   \hfill
  \begin{subfigure}[b]{0.16\linewidth}
      \centering
      \includegraphics[width=\textwidth]{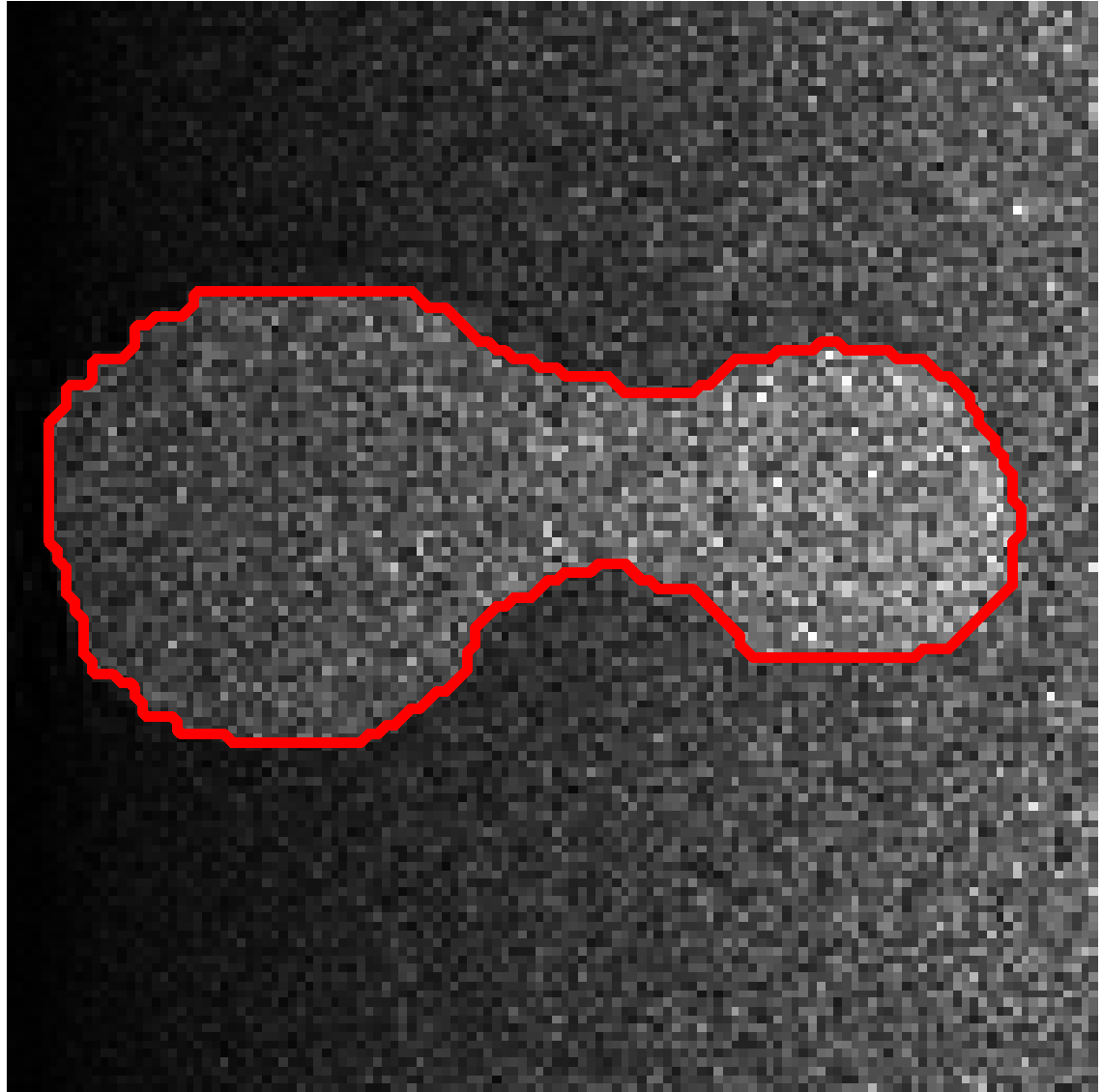}
  \end{subfigure}

  \begin{subfigure}[b]{0.16\linewidth}
      \centering
      \includegraphics[width=\textwidth]{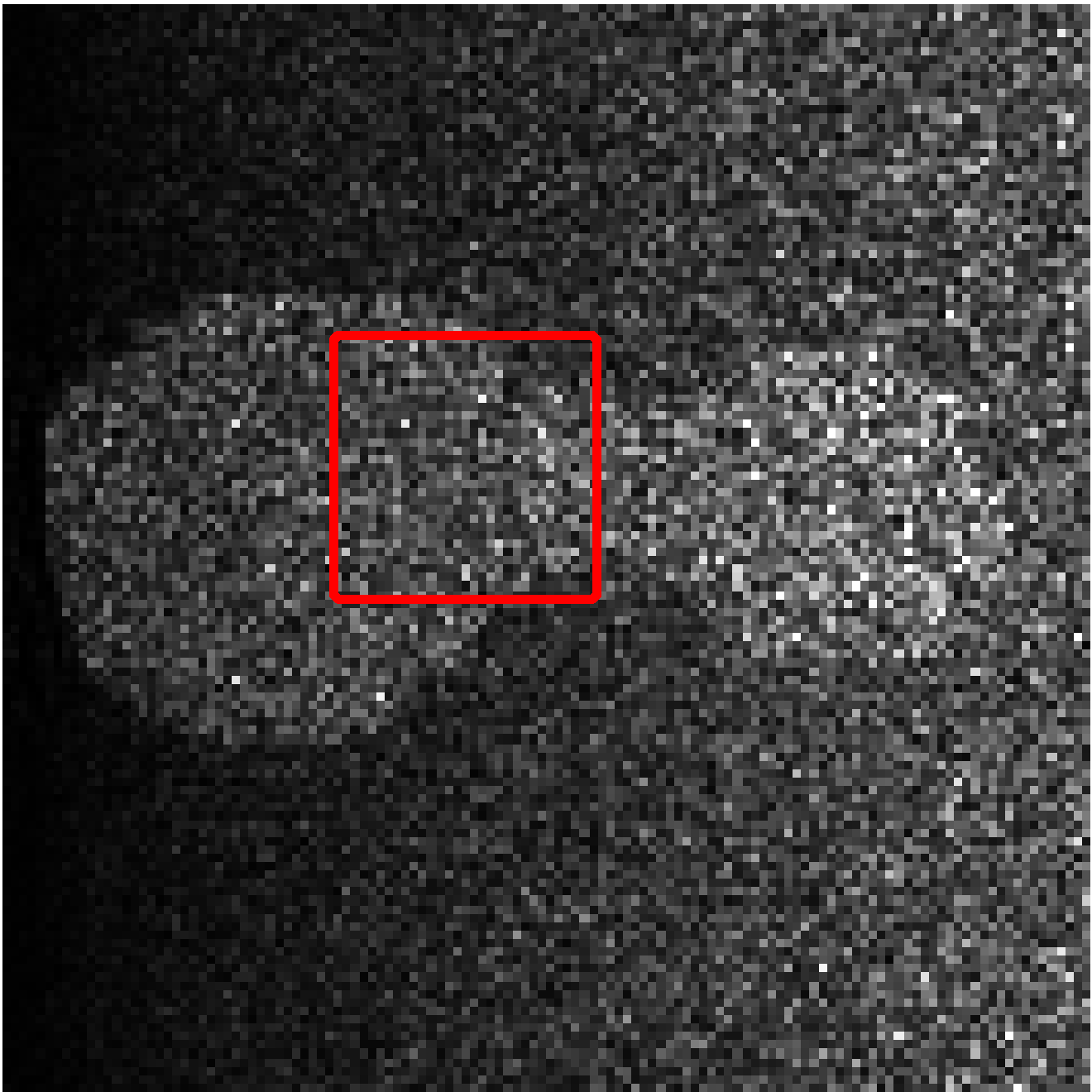}
  \end{subfigure}
   \hfill
  \begin{subfigure}[b]{0.16\linewidth}
      \centering
      \includegraphics[width=\textwidth]{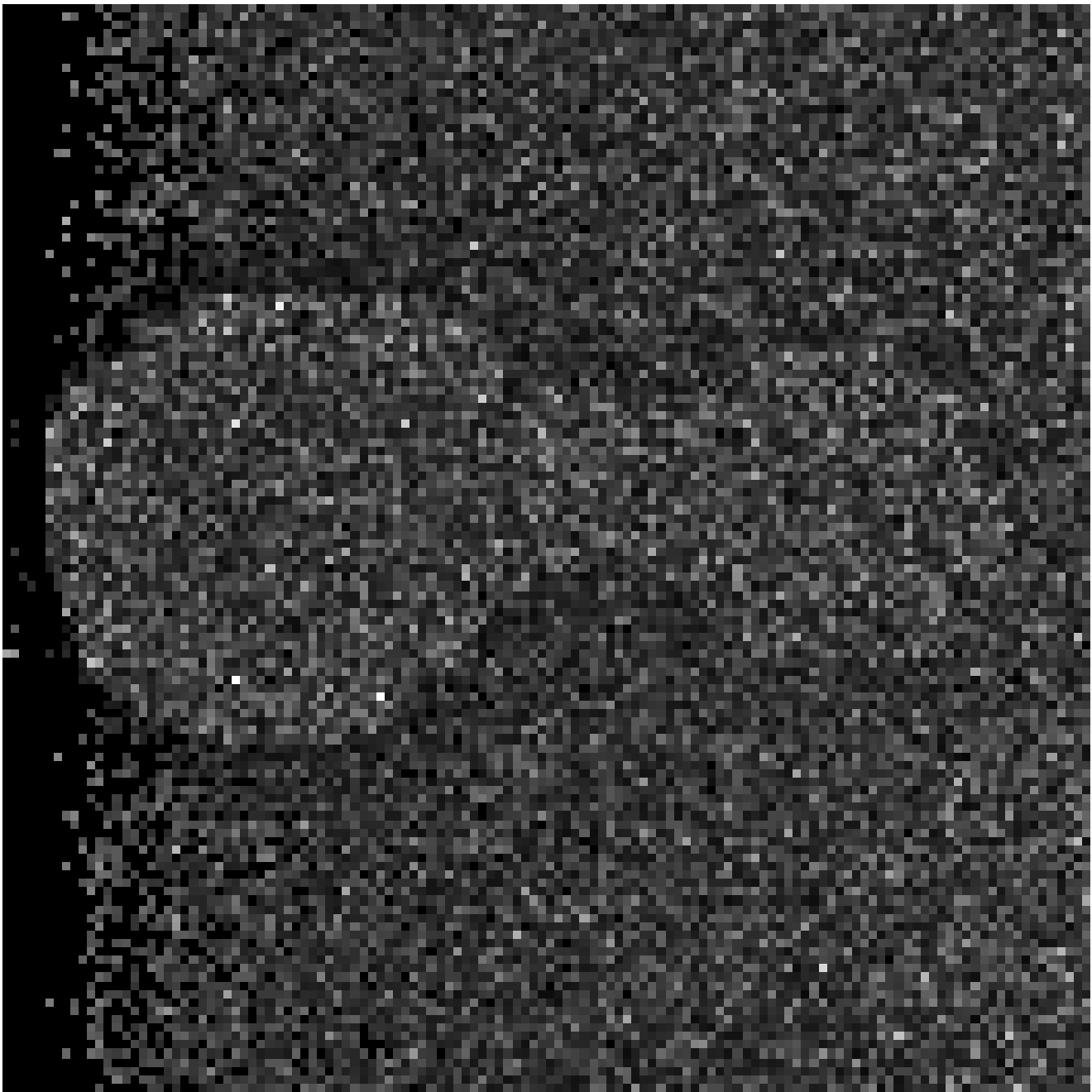}
  \end{subfigure}
   \hfill
  \begin{subfigure}[b]{0.16\linewidth}
      \centering
      \includegraphics[width=\textwidth]{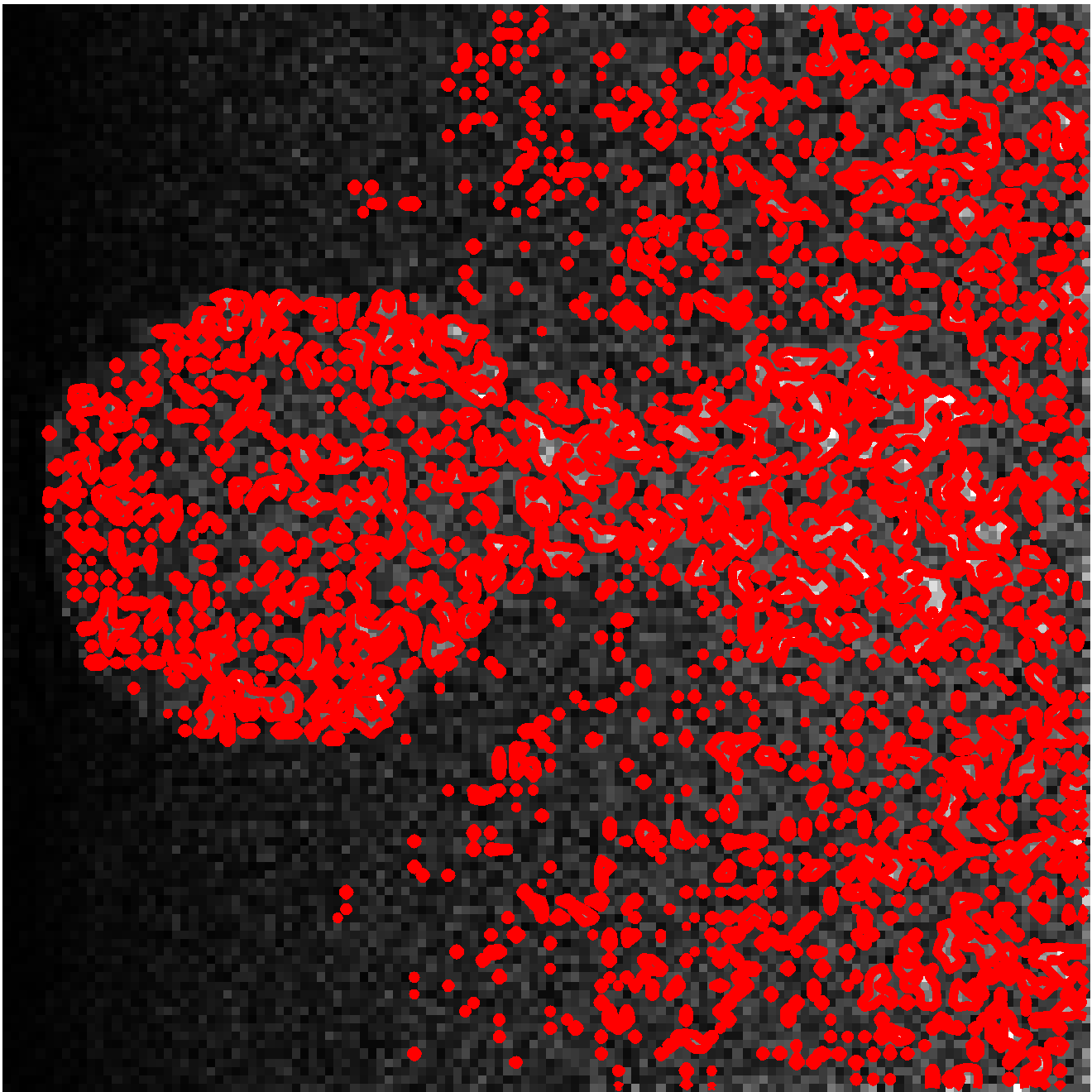}
  \end{subfigure}
   \hfill
  \begin{subfigure}[b]{0.16\linewidth}
      \centering
      \includegraphics[width=\textwidth]{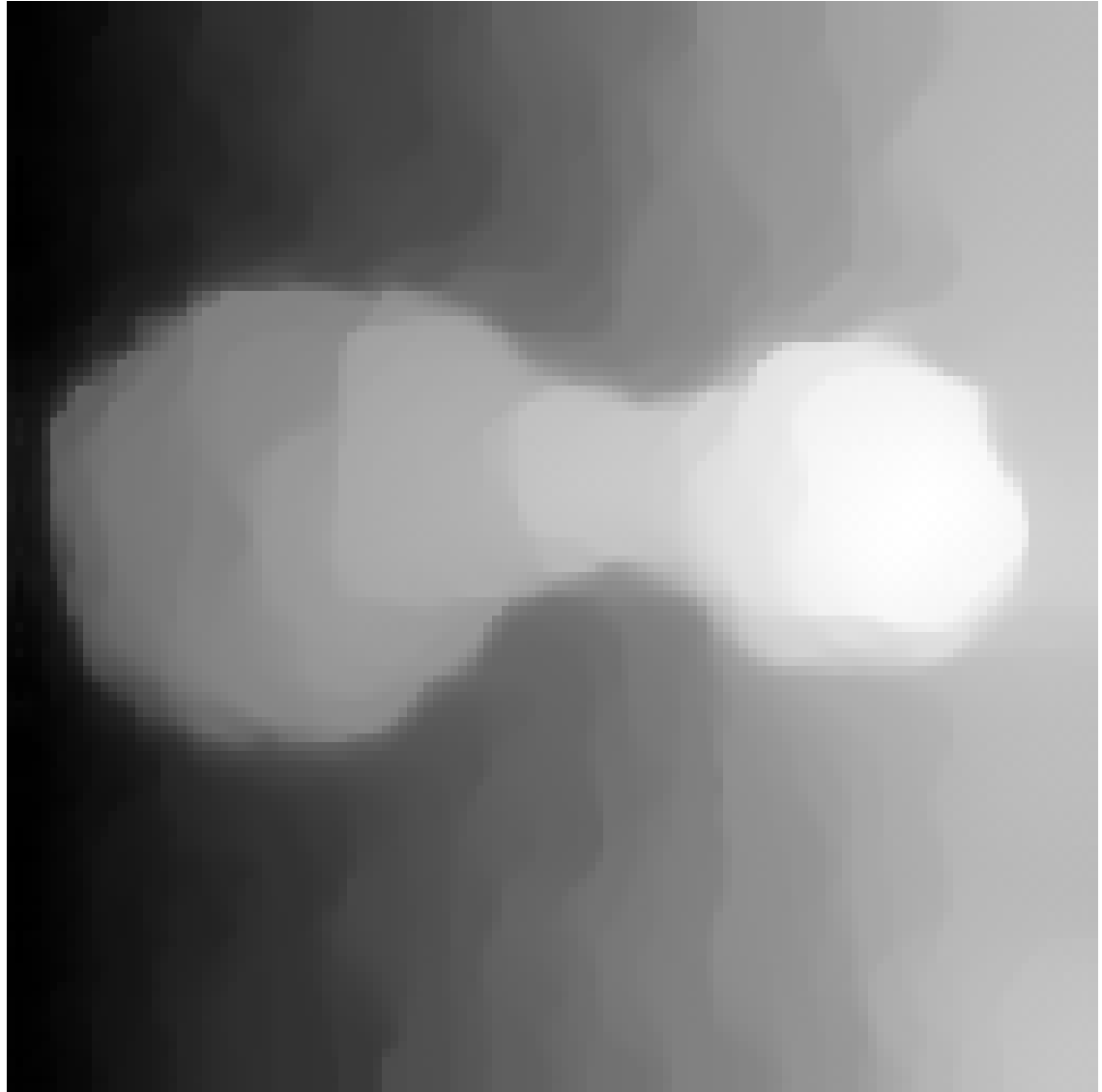}
  \end{subfigure}
   \hfill
  \begin{subfigure}[b]{0.16\linewidth}
      \centering
      \includegraphics[width=\textwidth]{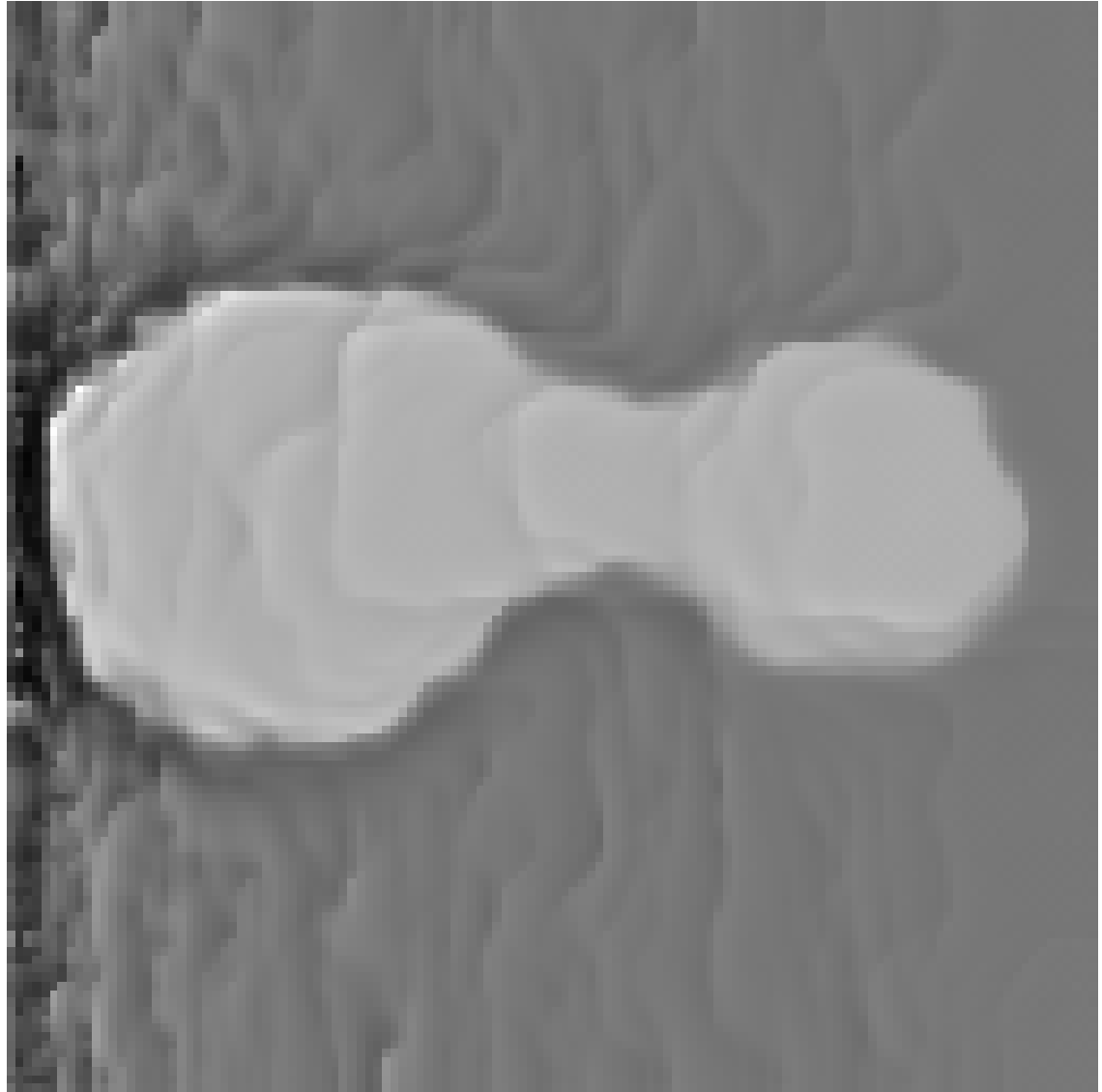}
  \end{subfigure}
   \hfill
  \begin{subfigure}[b]{0.16\linewidth}
      \centering
      \includegraphics[width=\textwidth]{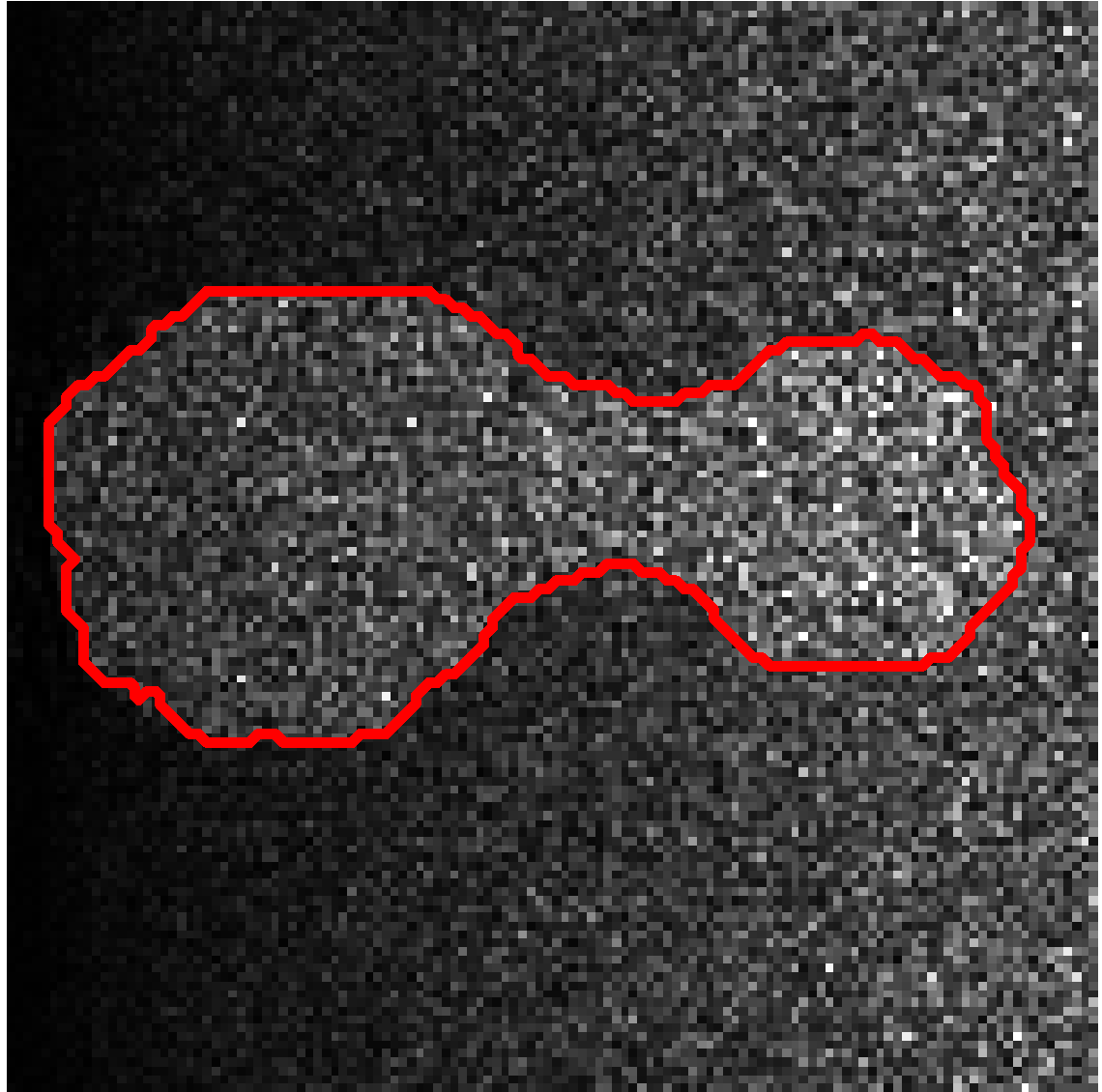}
  \end{subfigure}

  \begin{subfigure}[b]{0.16\linewidth}
      \centering
      \includegraphics[width=\textwidth]{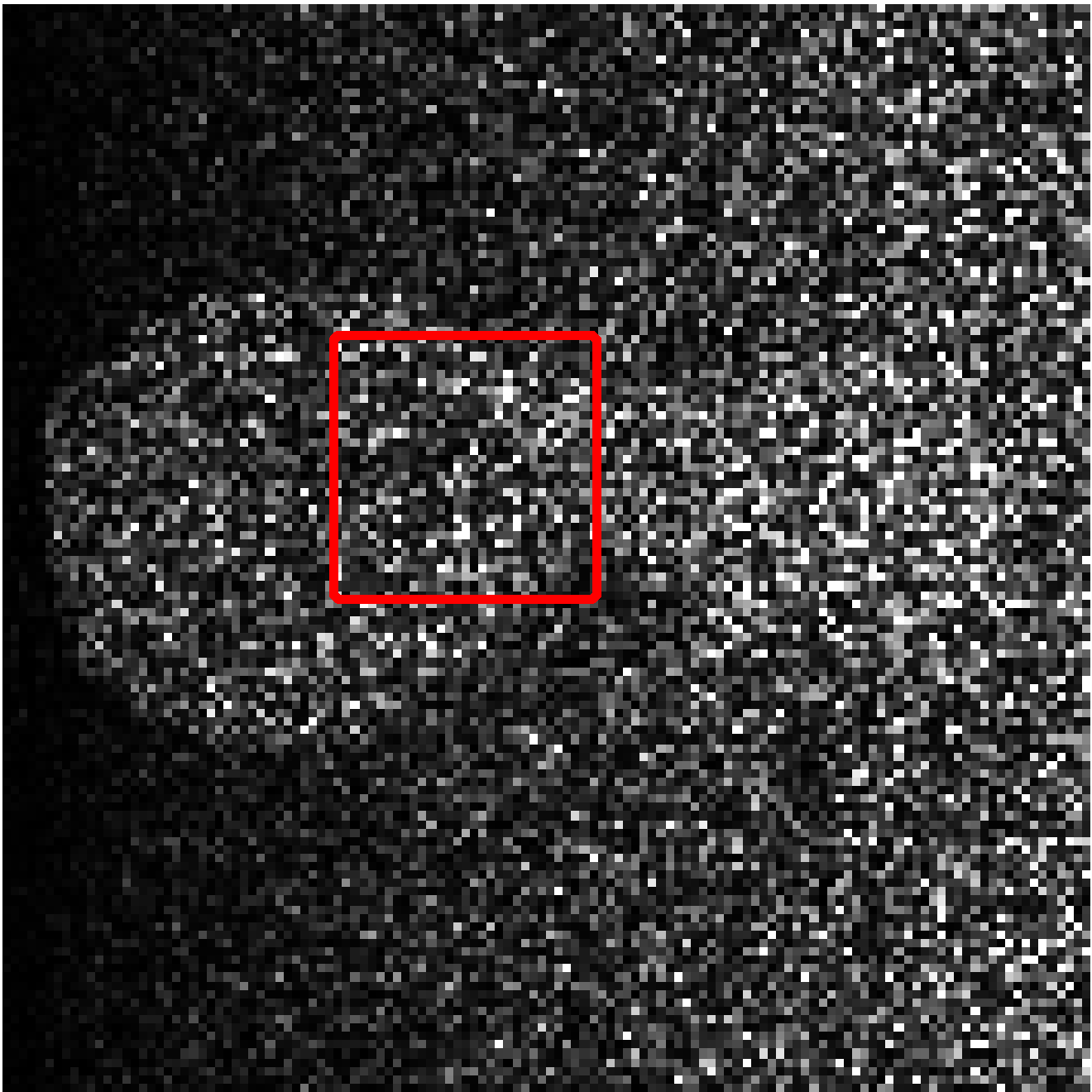}
      \caption{}
  \end{subfigure}
   \hfill
  \begin{subfigure}[b]{0.16\linewidth}
      \centering
      \includegraphics[width=\textwidth]{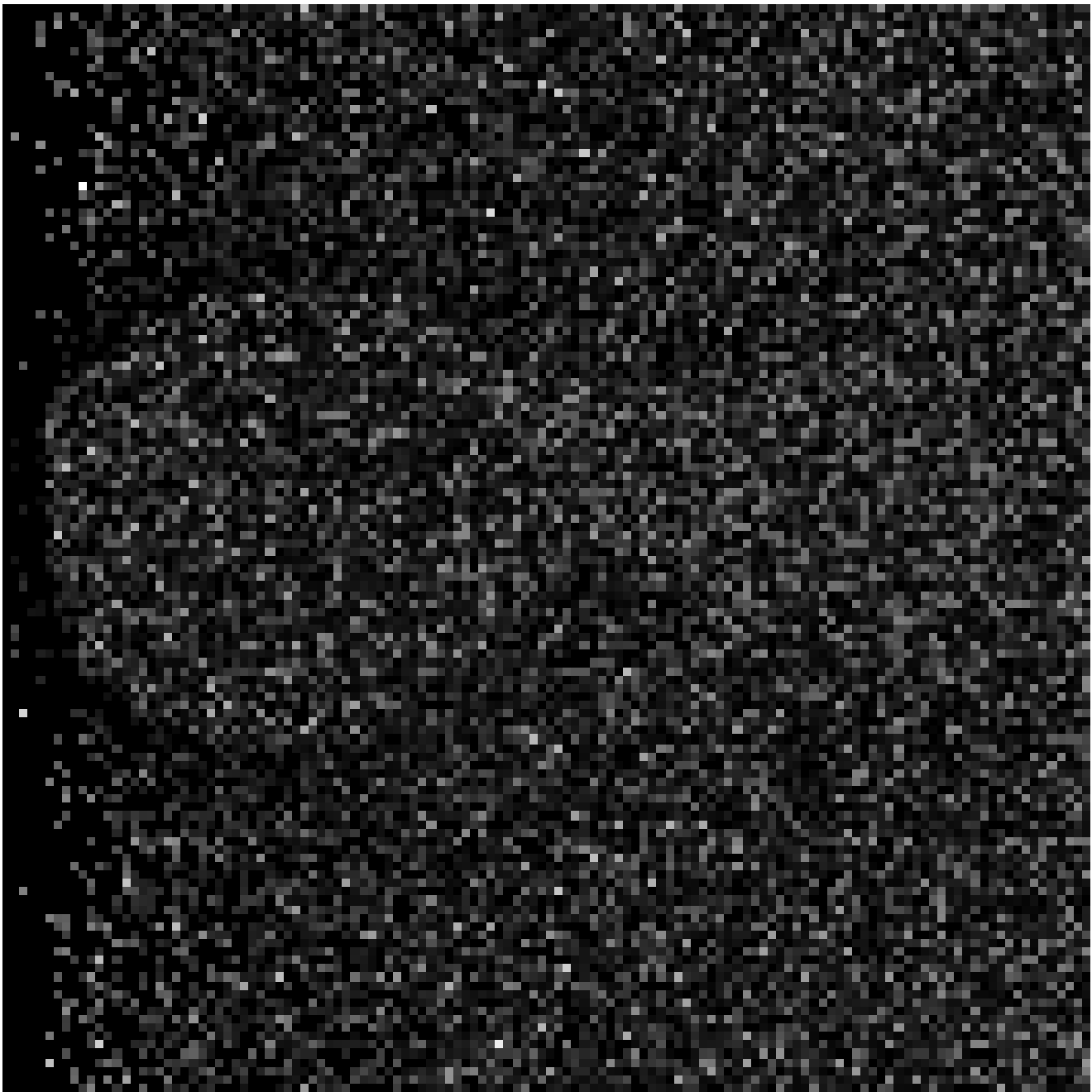}
      \caption{}
  \end{subfigure}
   \hfill
  \begin{subfigure}[b]{0.16\linewidth}
      \centering
      \includegraphics[width=\textwidth]{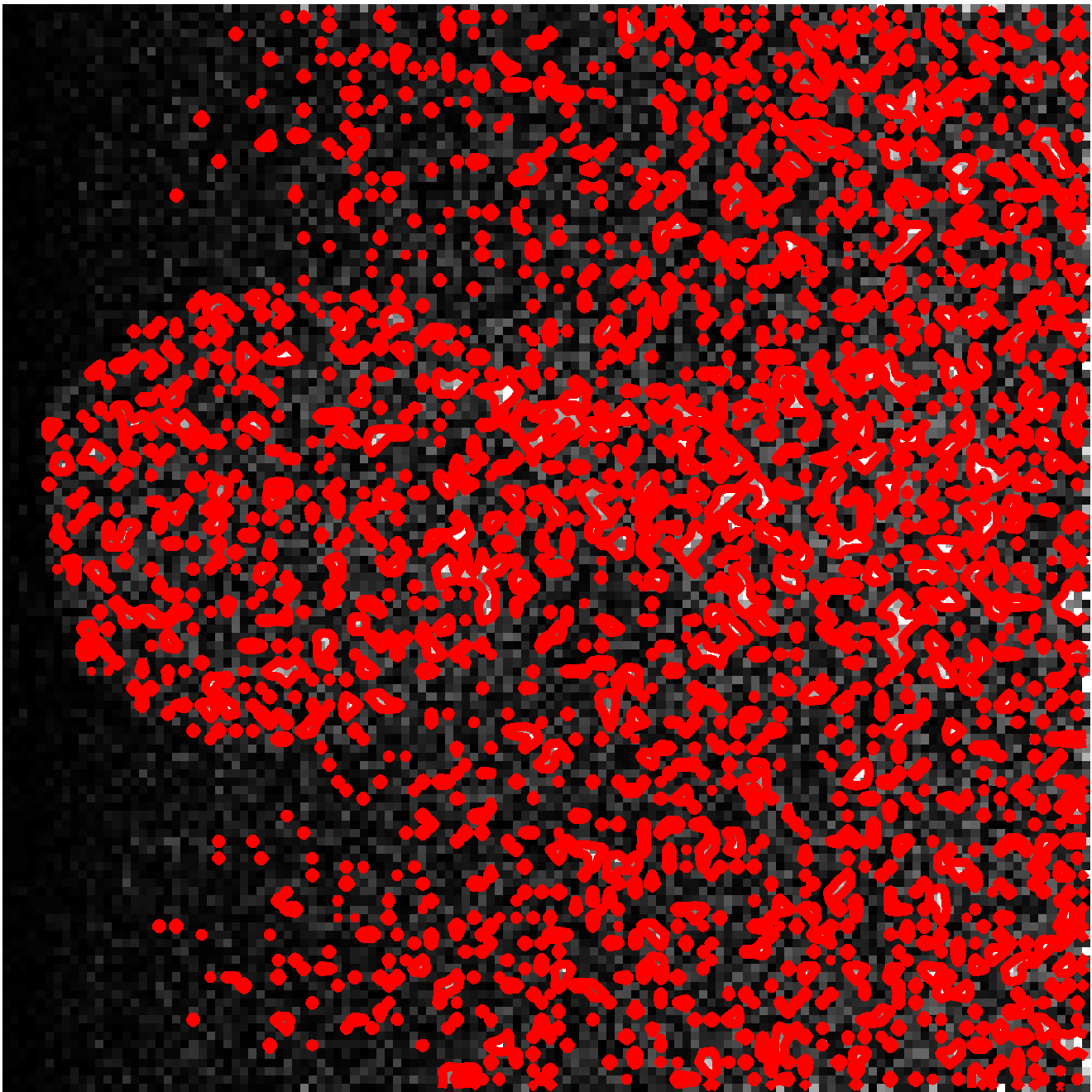}
      \caption{}
  \end{subfigure}
   \hfill
  \begin{subfigure}[b]{0.16\linewidth}
      \centering
      \includegraphics[width=\textwidth]{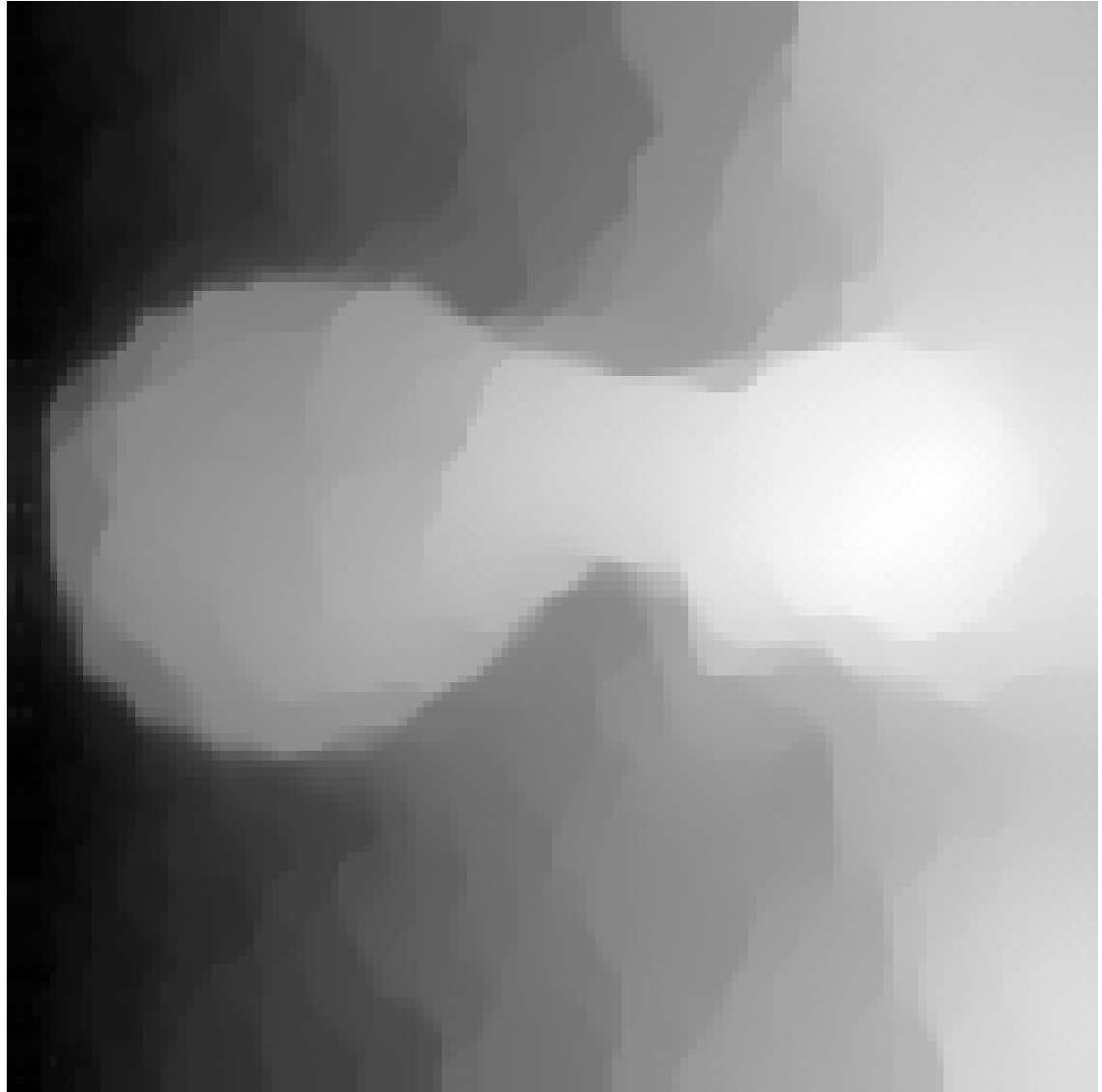}
      \caption{}
  \end{subfigure}
   \hfill
  \begin{subfigure}[b]{0.16\linewidth}
      \centering
      \includegraphics[width=\textwidth]{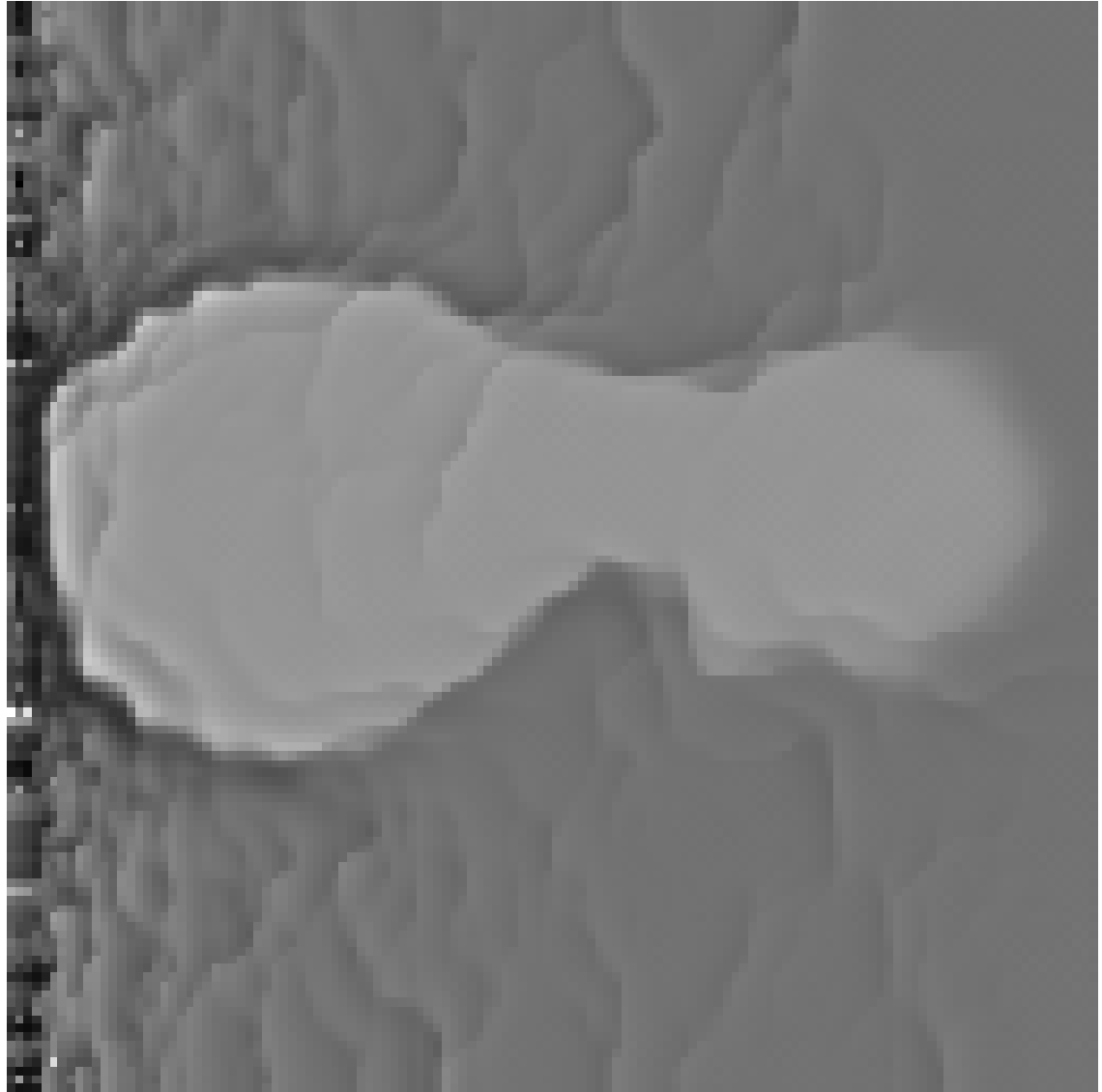}
      \caption{}
  \end{subfigure}
   \hfill
  \begin{subfigure}[b]{0.16\linewidth}
      \centering
      \includegraphics[width=\textwidth]{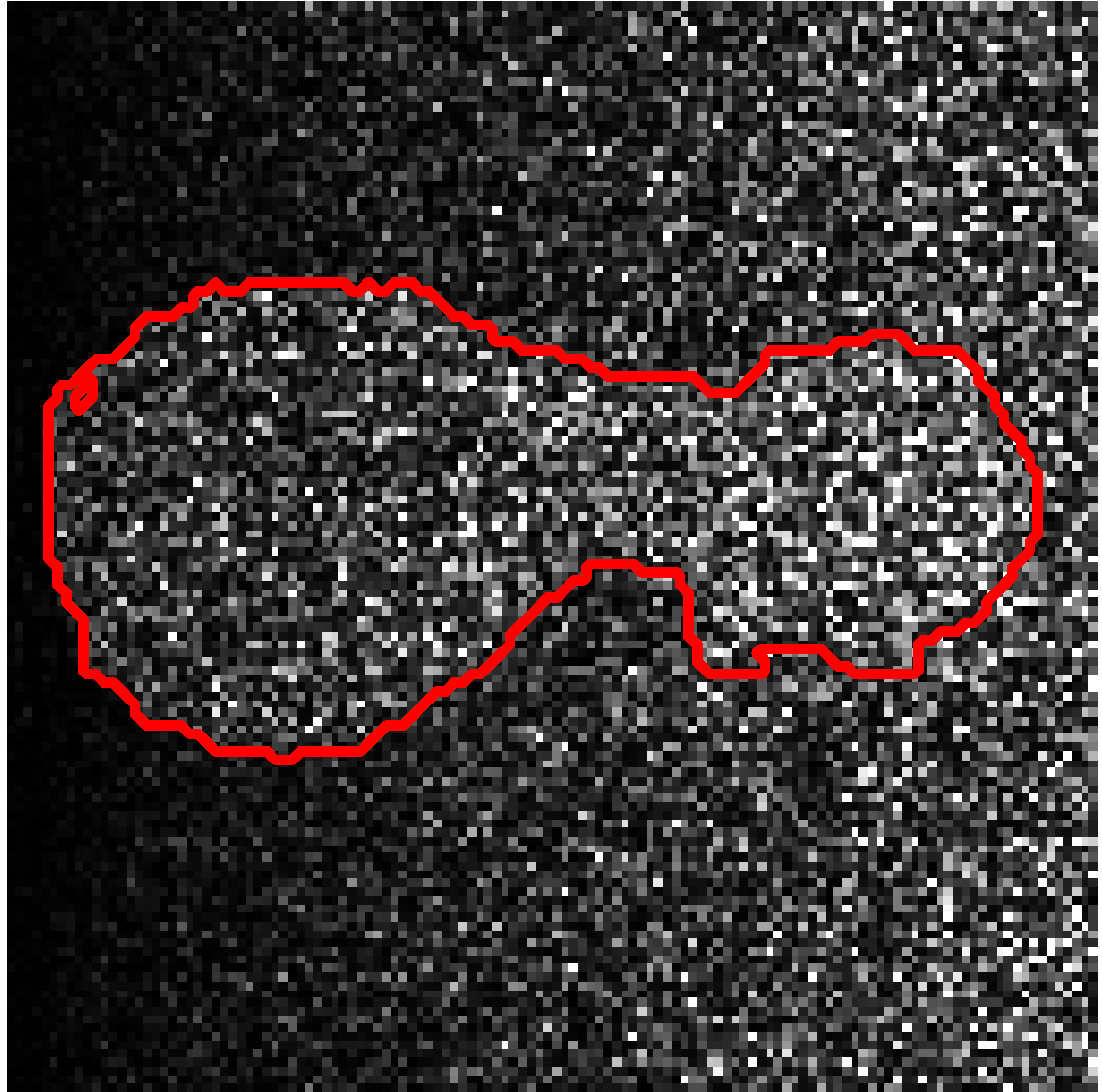}
      \caption{}
  \end{subfigure}
  \caption{Robustness to noise. (a): input images (original, Poisson noise, and Gamma noise with $L=10,4,1$); (b)--(c): bias correction and segmentation by LIC; (d)--(f): denoising, bias correction and segmentation by our model}
  \label{fig:noise-seg}
\end{figure}
To evaluate the robustness of our model to noise, we conduct segmentation experiments comparing it with the LIC model on synthetic images with intensity inhomogeneity degraded by various types of noise. The segmentation results are presented in Fig.~\ref{fig:noise-seg}. When the noise level is relatively low, the LIC model produces accurate segmentation results. This can be attributed to the incorporation of a regularization term (length term), which enhances the model's robustness to noise to some extent. However, the LIC model remains susceptible to noise, leading to misclassification of isolated background noise points as part of the target at higher noise levels. This effect becomes progressively more pronounced as the noise level increases. In contrast, our model achieves accurate segmentation results on images corrupted by different types of noise. In addition, due to the denoising mechanism, our corrected image can be regarded as denoised and then bias corrected; consequently, it is smoother than that obtained by the LIC model. Since our corrected images exhibit clear boundaries, accurate segmentation results can be obtained even when the original image is severely corrupted by noise.

\subsubsection{Robustness to Initialization}
\begin{figure}
  \centering
  \begin{subfigure}[b]{0.19\linewidth}
      \centering
      \includegraphics[width=\textwidth]{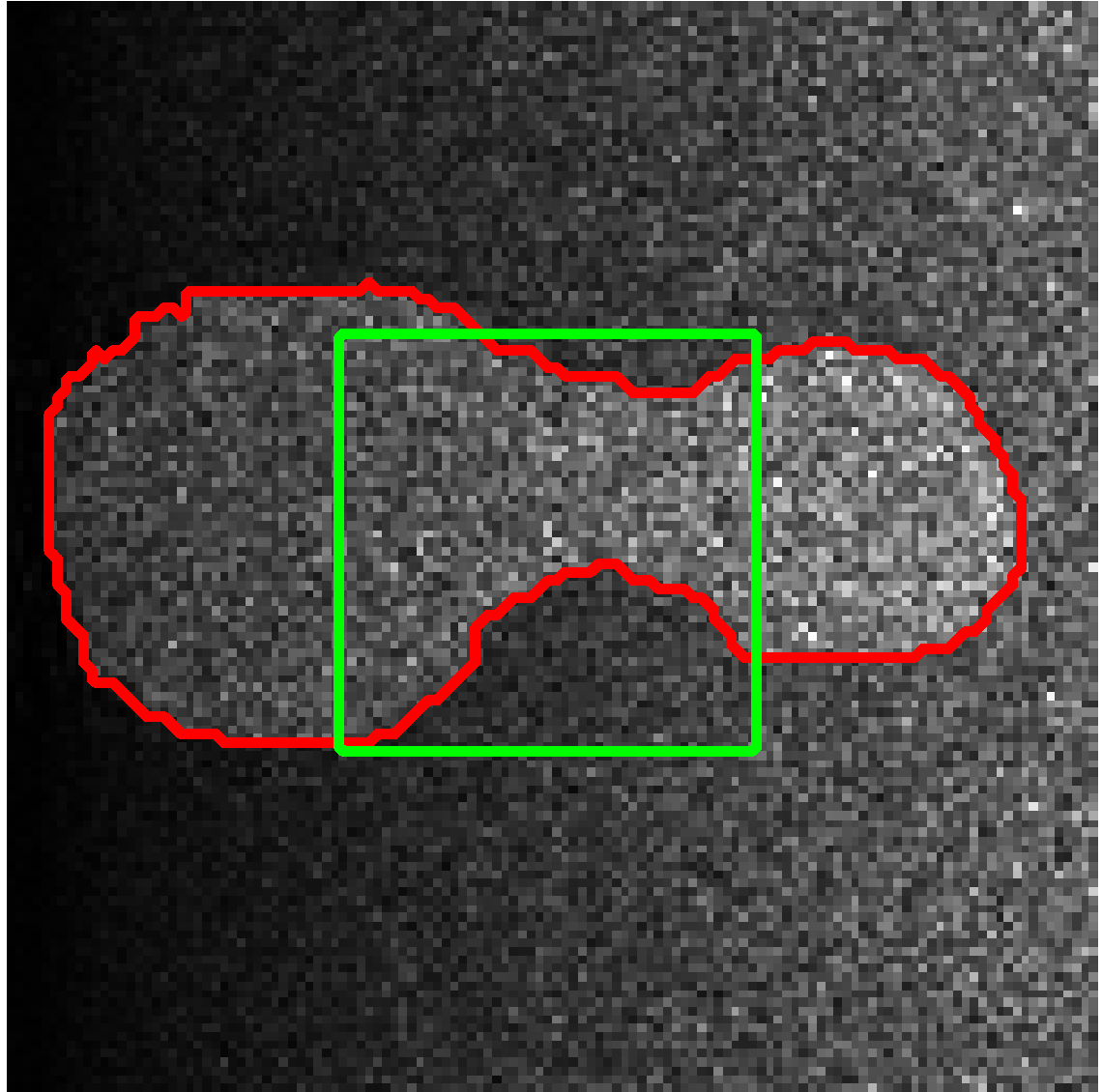}
  \end{subfigure}
  \hfill
  \begin{subfigure}[b]{0.19\linewidth}
      \centering
      \includegraphics[width=\textwidth]{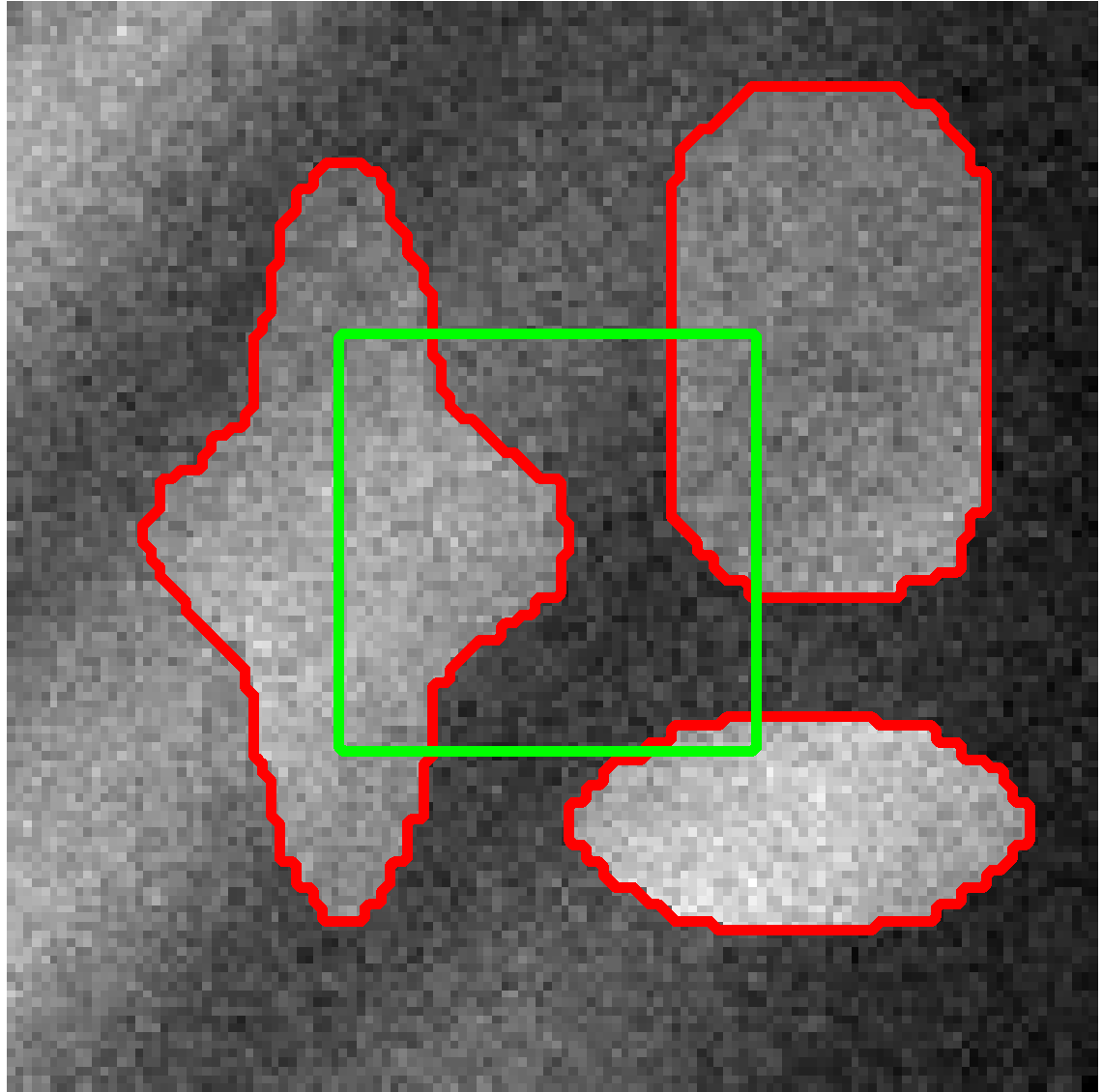}
  \end{subfigure}
 \hfill
  \begin{subfigure}[b]{0.19\linewidth}
      \centering
      \includegraphics[width=\textwidth]{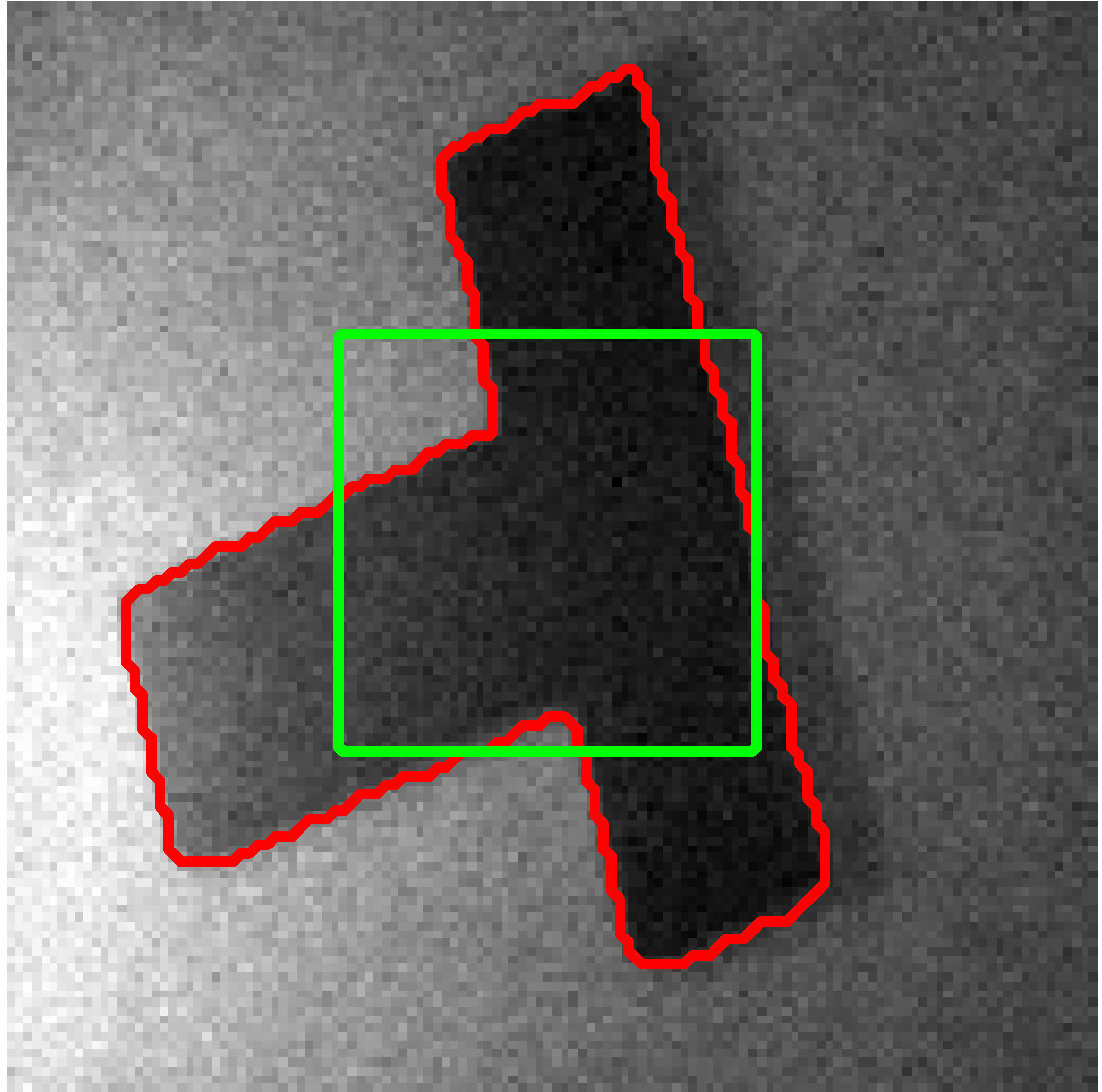}
  \end{subfigure}
  \hfill
  \begin{subfigure}[b]{0.19\linewidth}
      \centering
      \includegraphics[width=\textwidth]{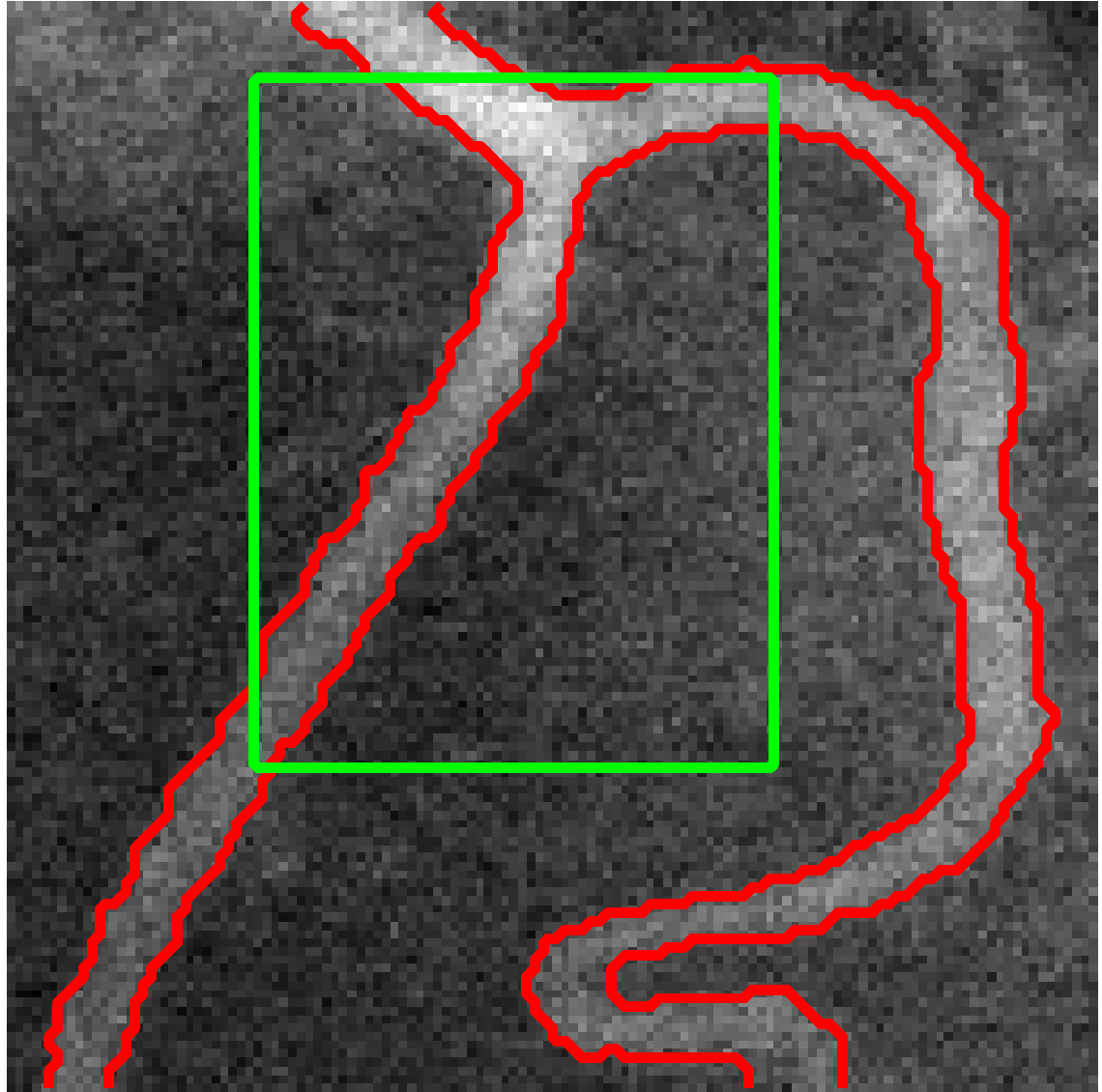}
  \end{subfigure}
  \hfill
  \begin{subfigure}[b]{0.19\linewidth}
      \centering
      \includegraphics[width=\textwidth]{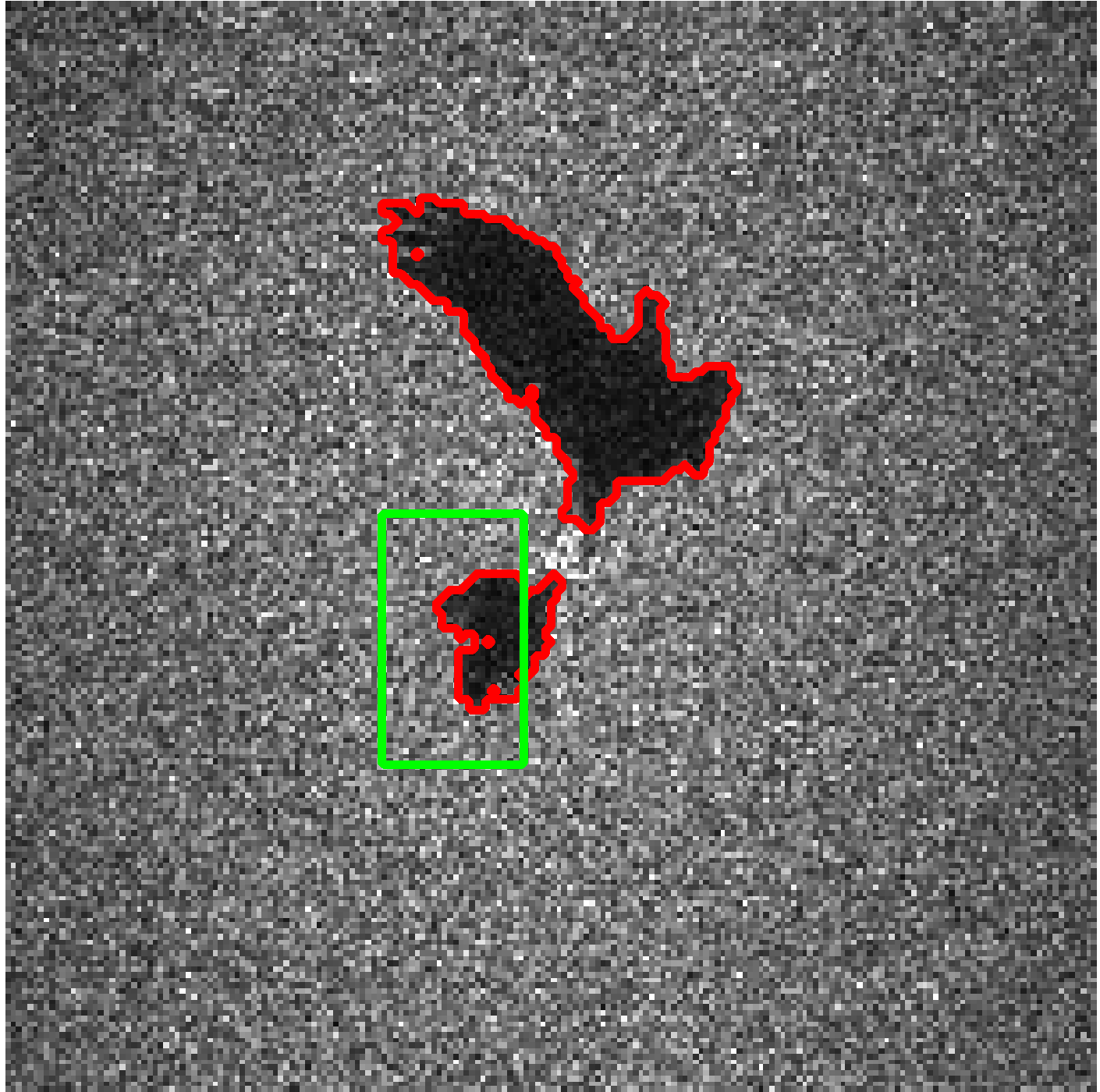}
  \end{subfigure}

   \begin{subfigure}[b]{0.19\linewidth}
      \centering
      \includegraphics[width=\textwidth]{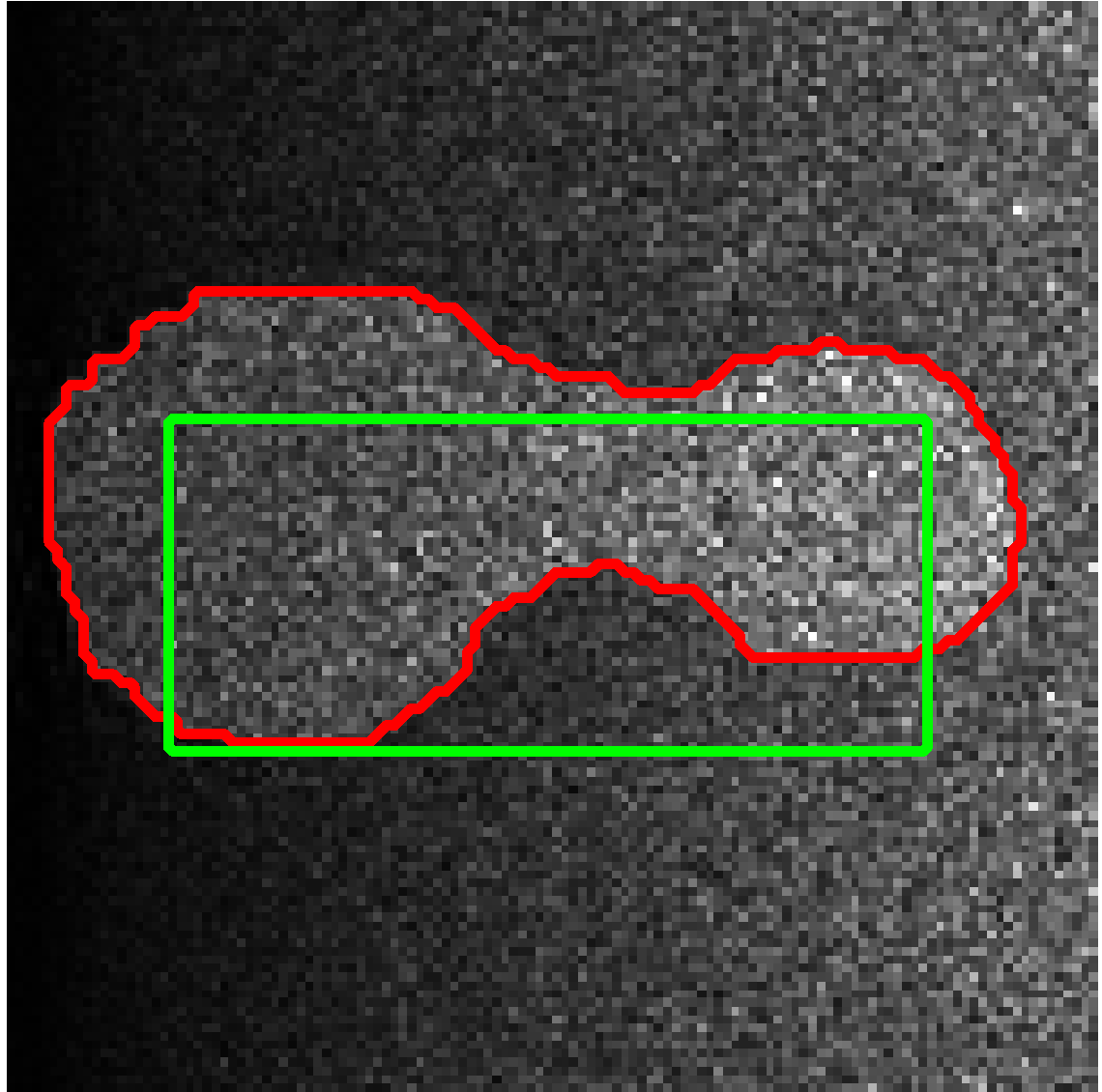}
  \end{subfigure}
  \hfill
  \begin{subfigure}[b]{0.19\linewidth}
      \centering
      \includegraphics[width=\textwidth]{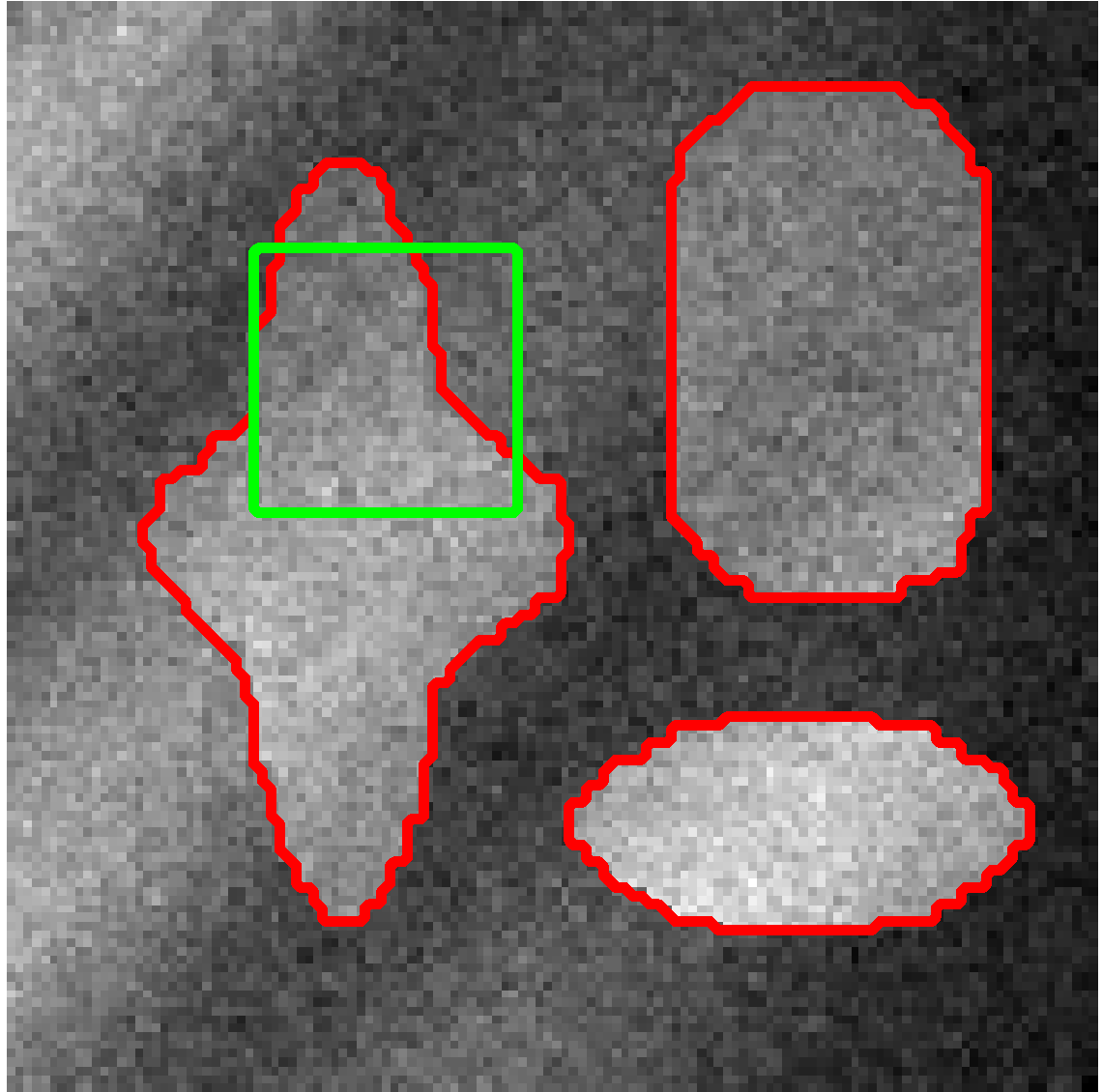}
  \end{subfigure}
 \hfill
  \begin{subfigure}[b]{0.19\linewidth}
      \centering
      \includegraphics[width=\textwidth]{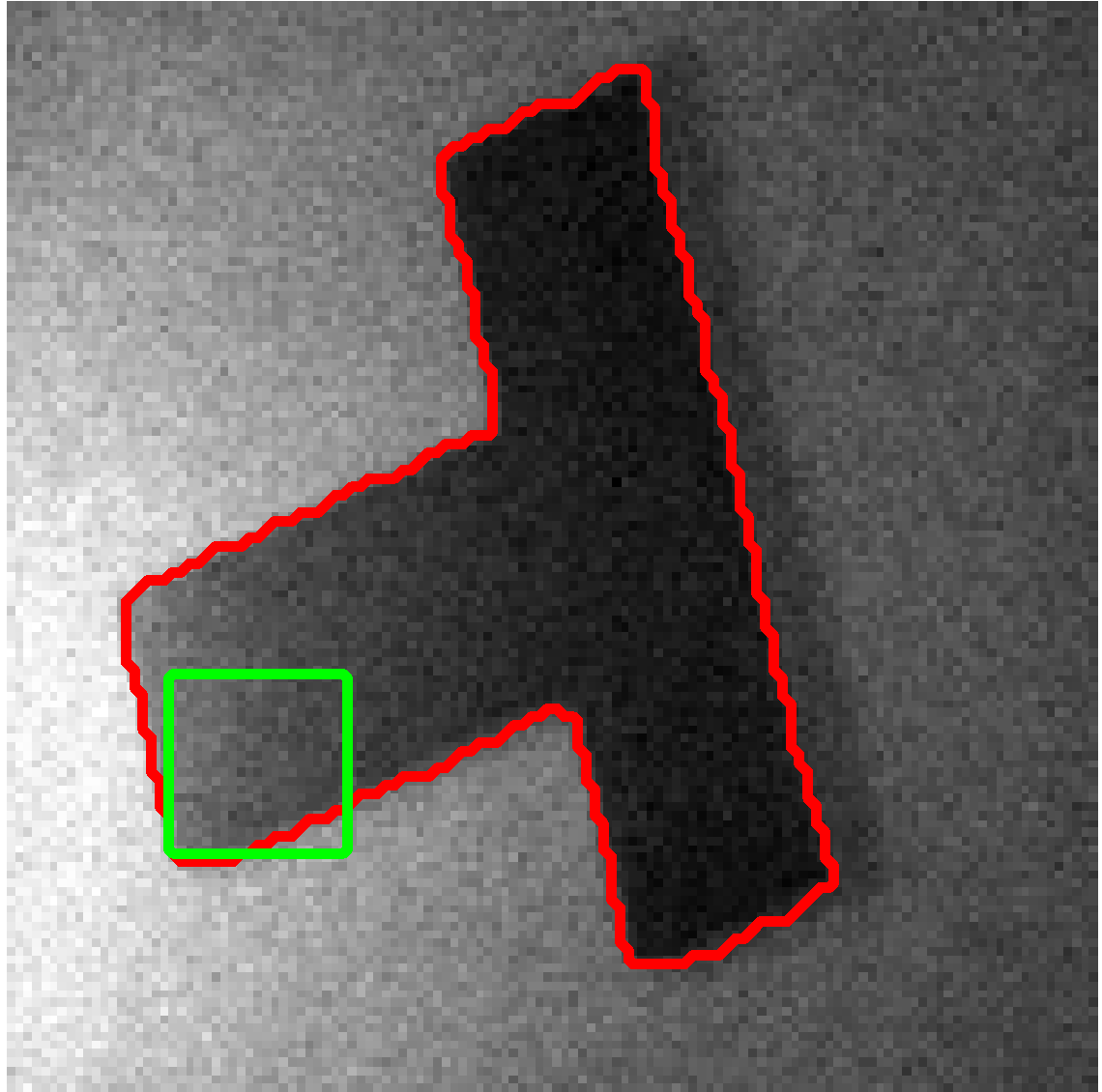}
  \end{subfigure}
  \hfill
  \begin{subfigure}[b]{0.19\linewidth}
      \centering
      \includegraphics[width=\textwidth]{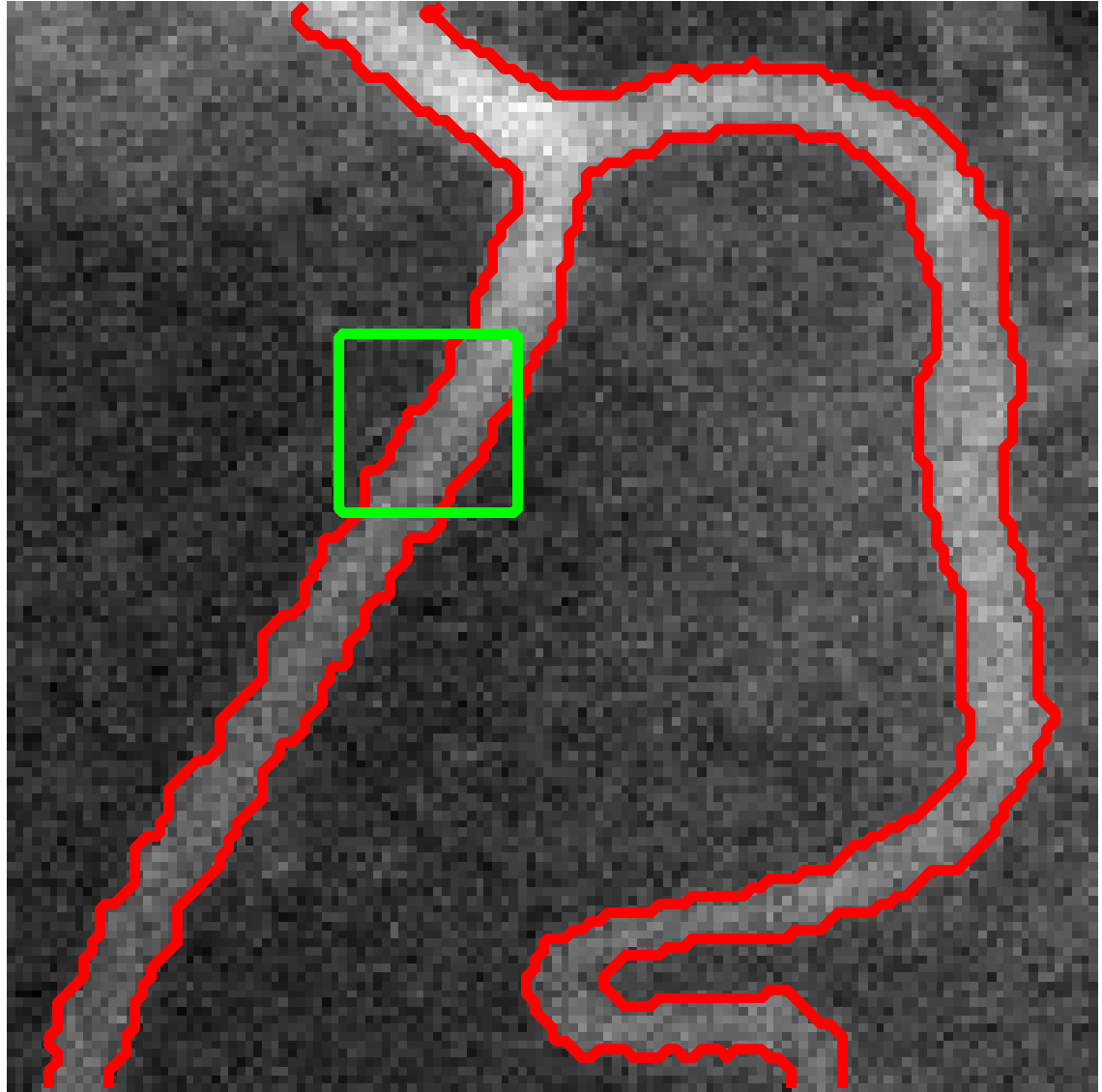}
  \end{subfigure}
  \hfill
  \begin{subfigure}[b]{0.19\linewidth}
      \centering
      \includegraphics[width=\textwidth]{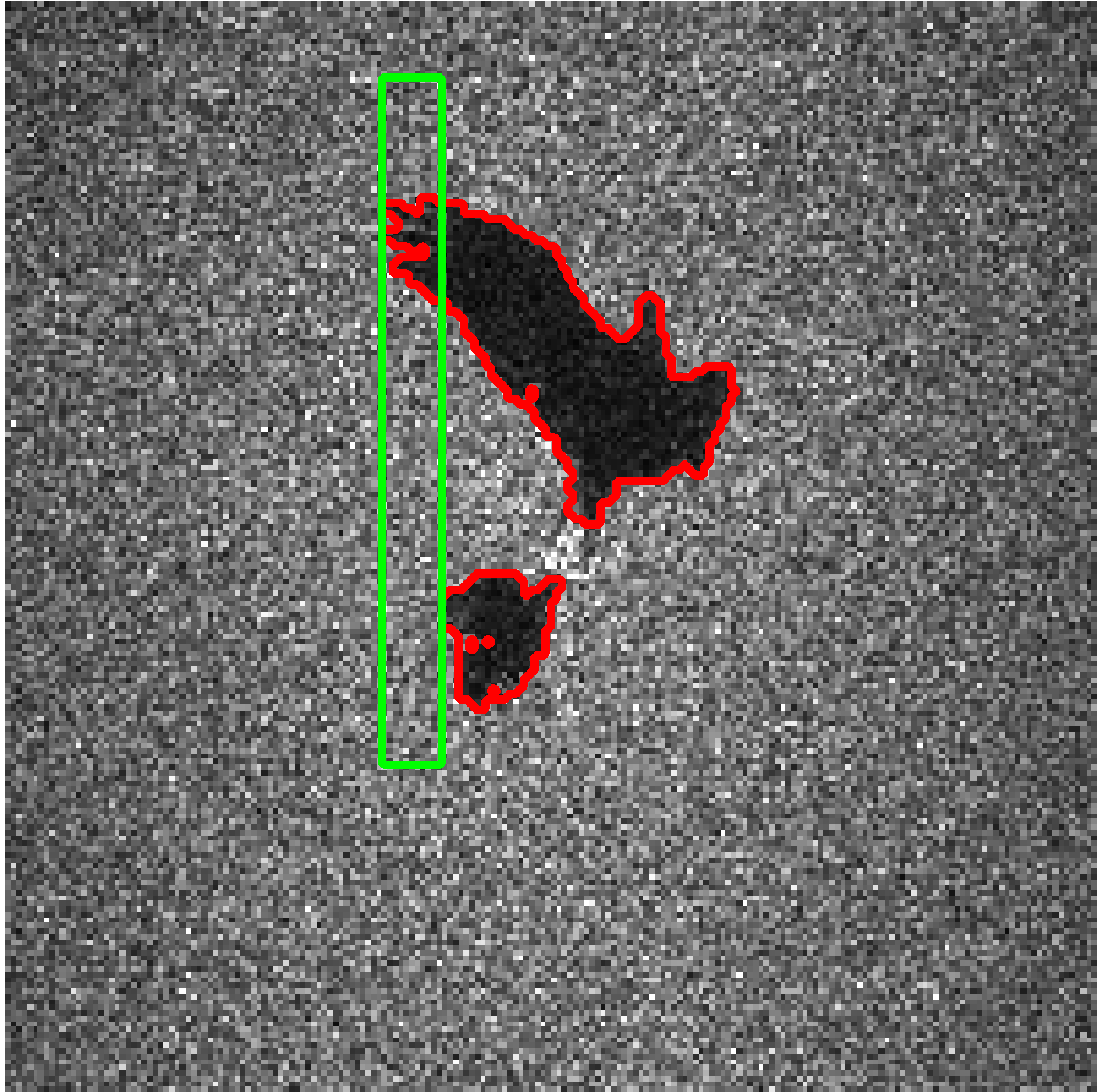}
  \end{subfigure}

    \begin{subfigure}[b]{0.19\linewidth}
      \centering
      \includegraphics[width=\textwidth]{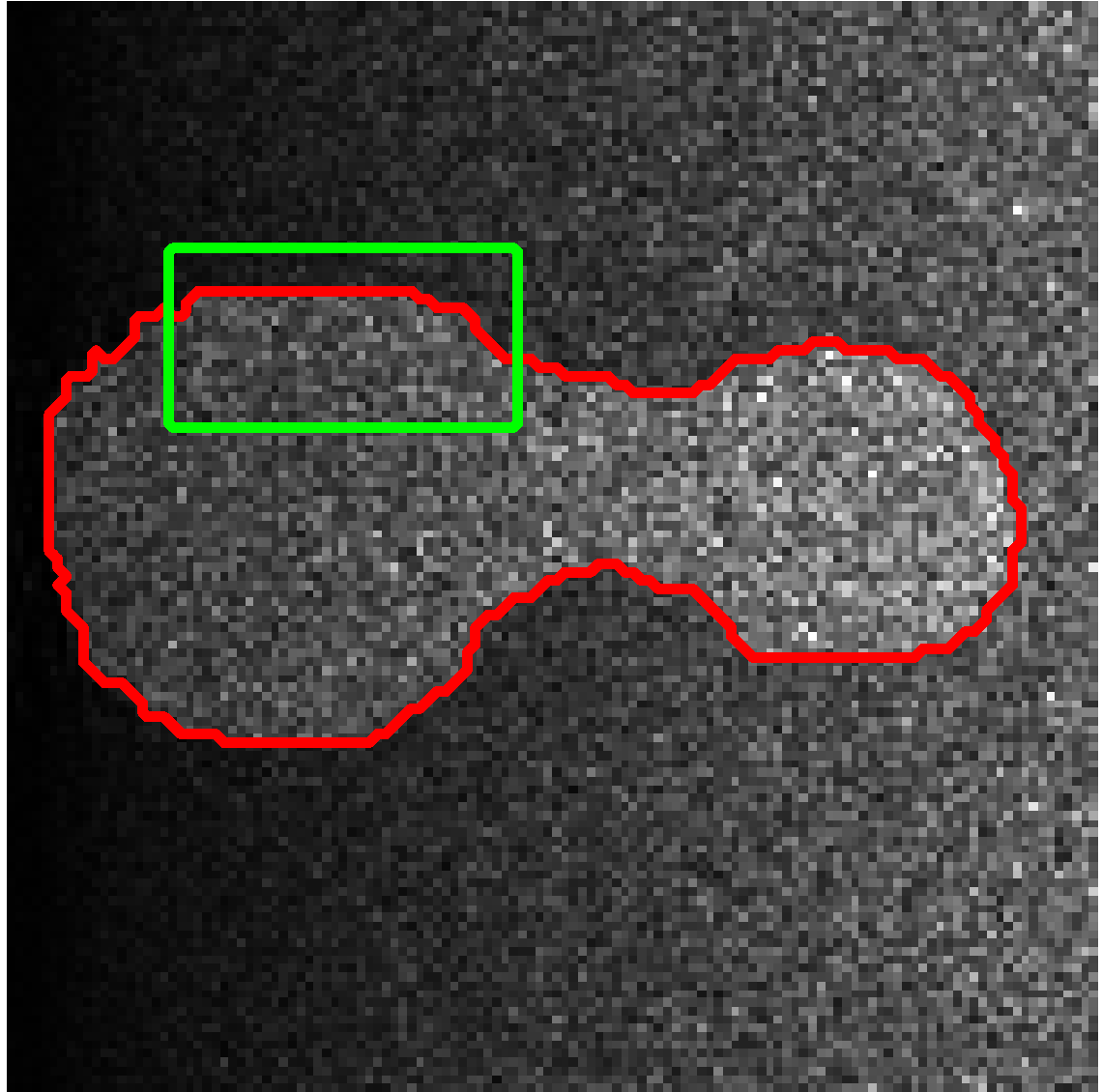}
  \end{subfigure}
  \hfill
  \begin{subfigure}[b]{0.19\linewidth}
      \centering
      \includegraphics[width=\textwidth]{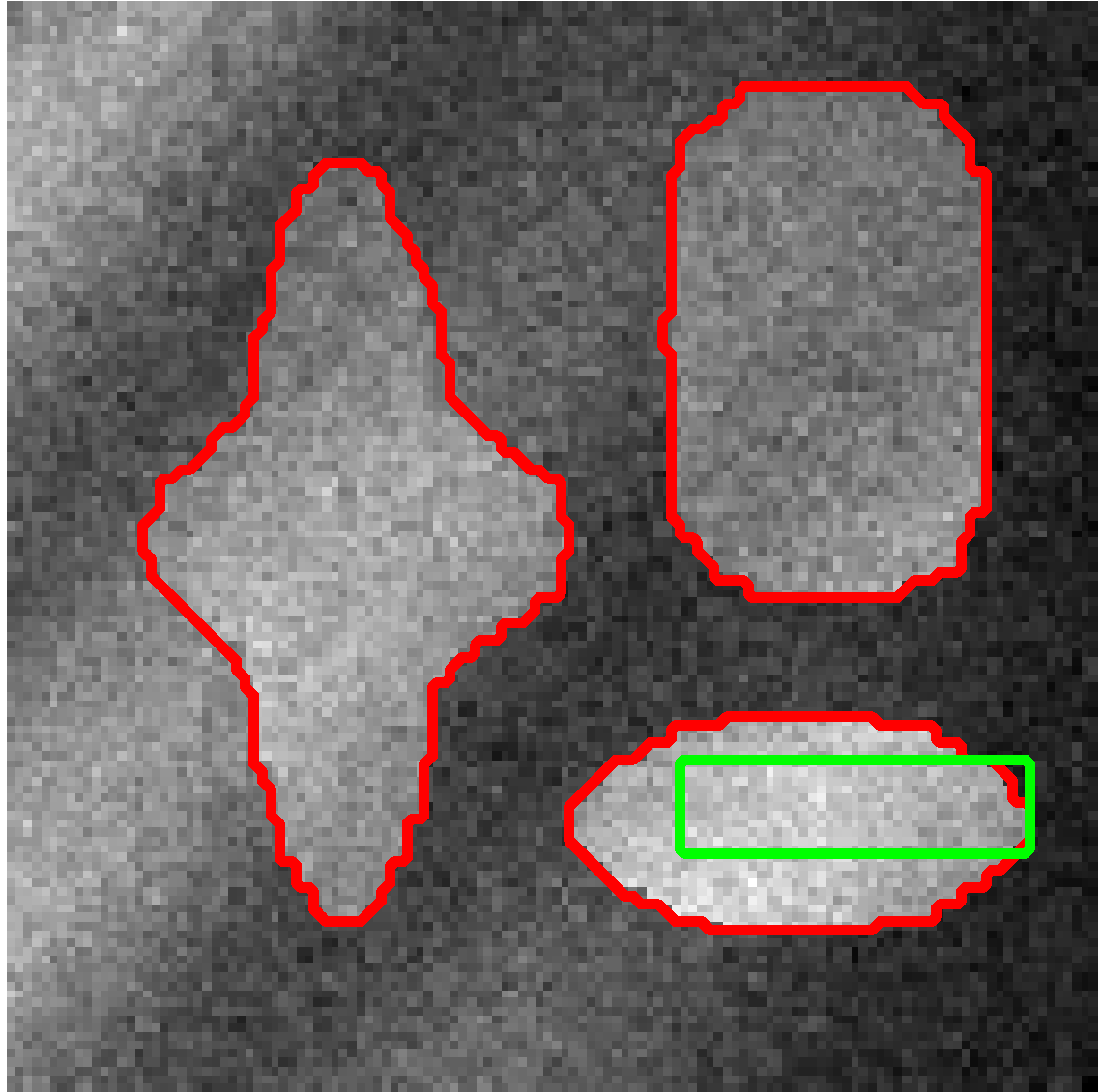}
  \end{subfigure}
 \hfill
  \begin{subfigure}[b]{0.19\linewidth}
      \centering
      \includegraphics[width=\textwidth]{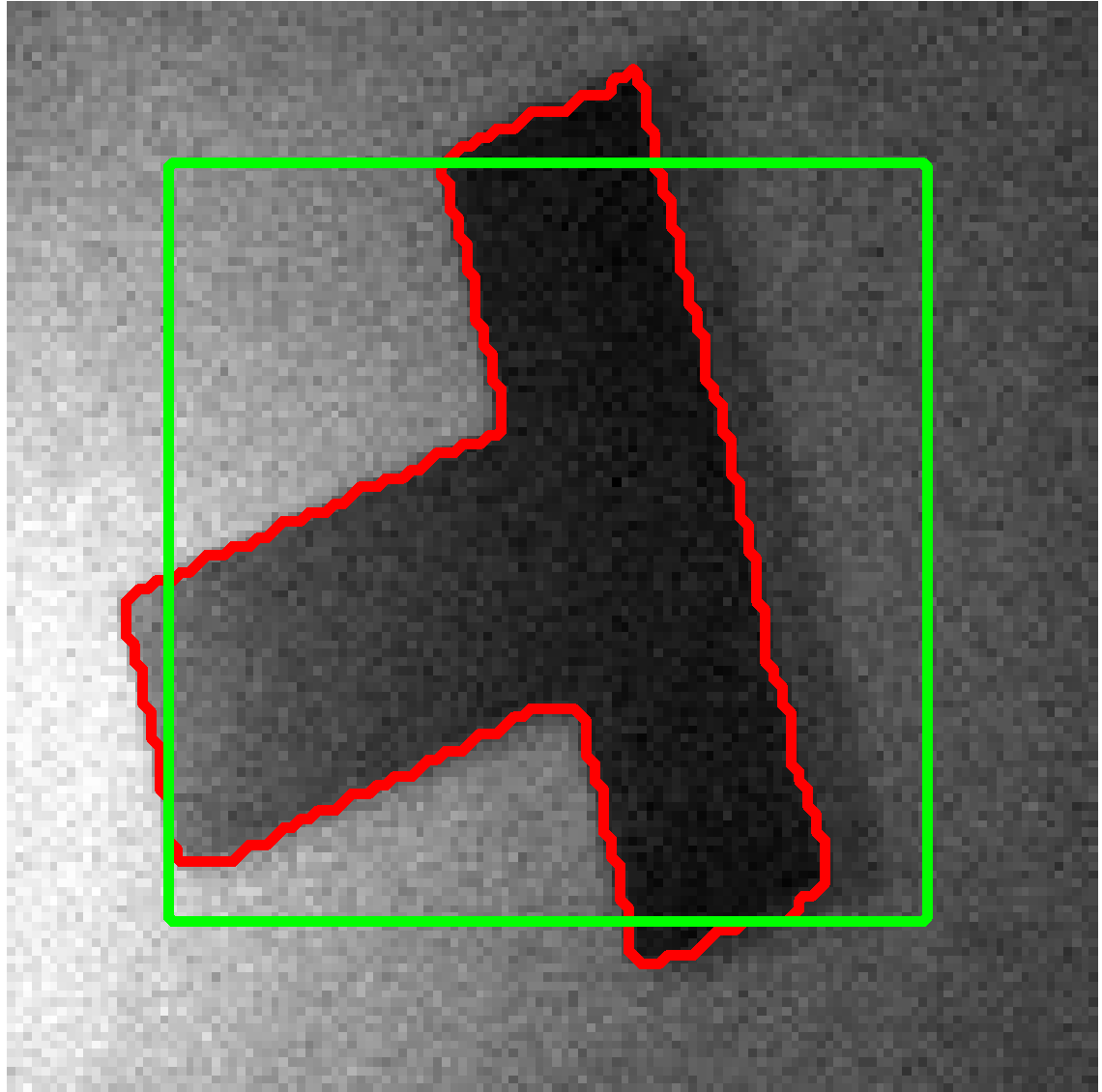}
  \end{subfigure}
  \hfill
  \begin{subfigure}[b]{0.19\linewidth}
      \centering
      \includegraphics[width=\textwidth]{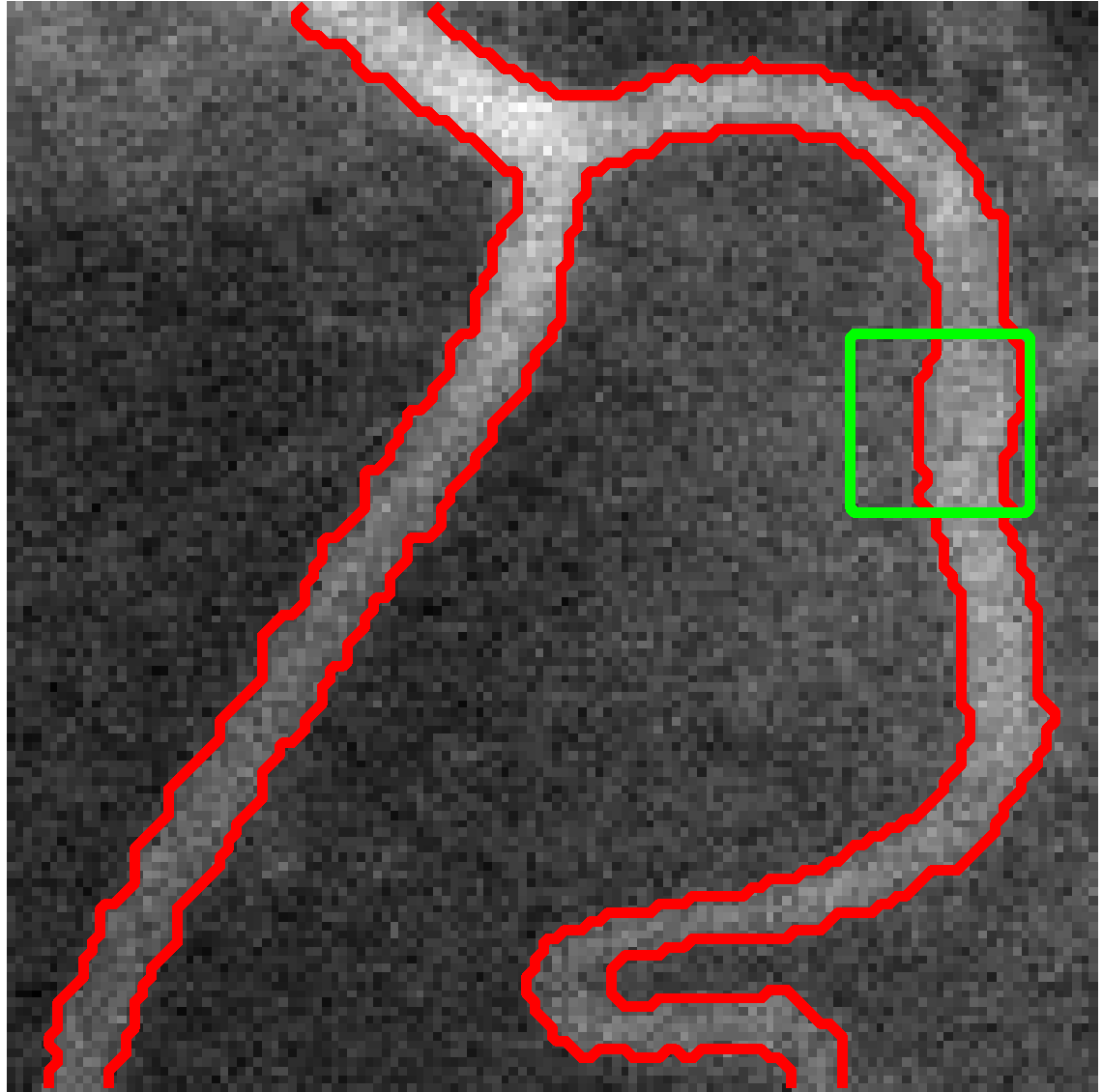}
  \end{subfigure}
  \hfill
  \begin{subfigure}[b]{0.19\linewidth}
      \centering
      \includegraphics[width=\textwidth]{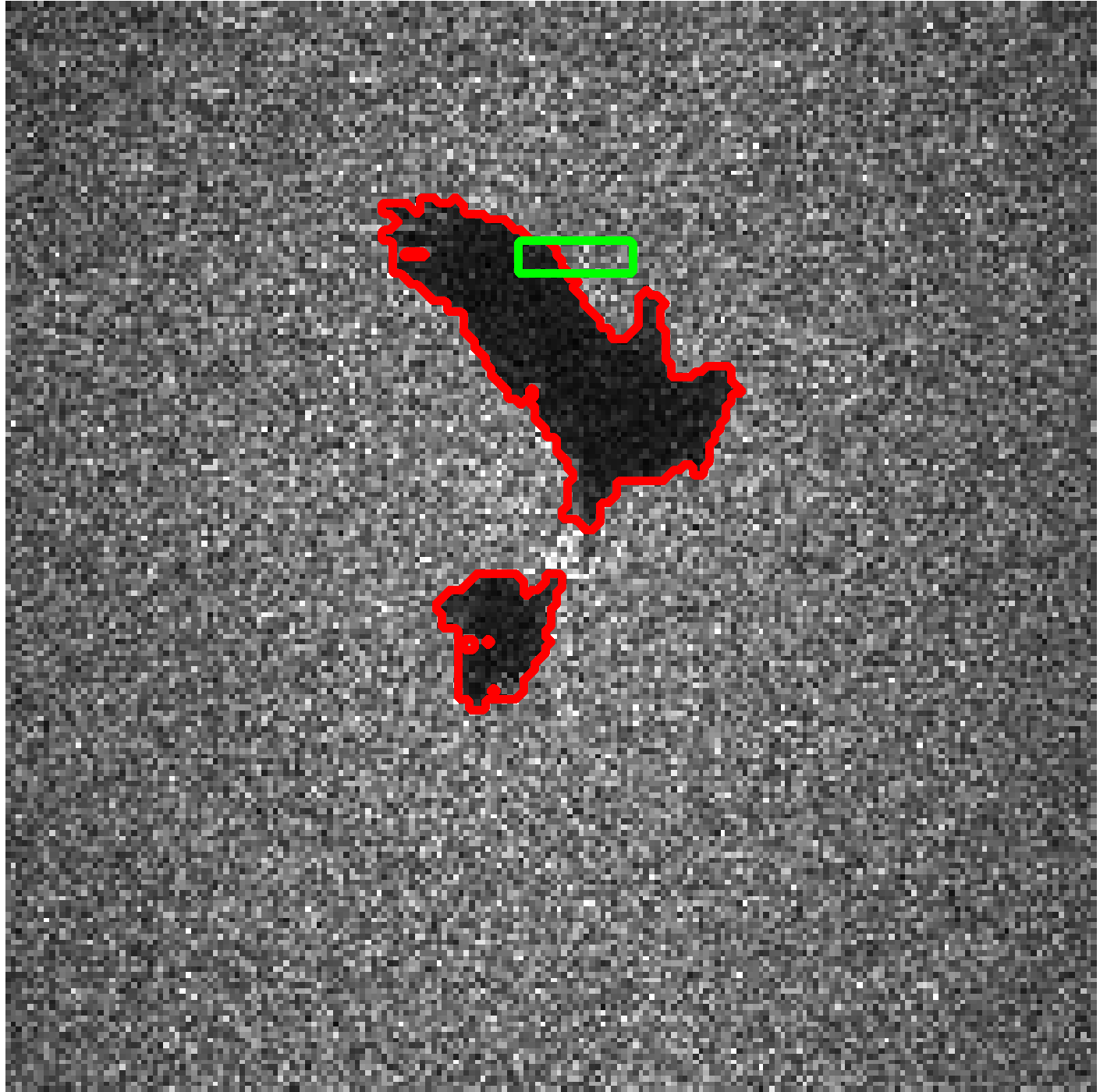}
  \end{subfigure}

    \begin{subfigure}[b]{0.19\linewidth}
      \centering
      \includegraphics[width=\textwidth]{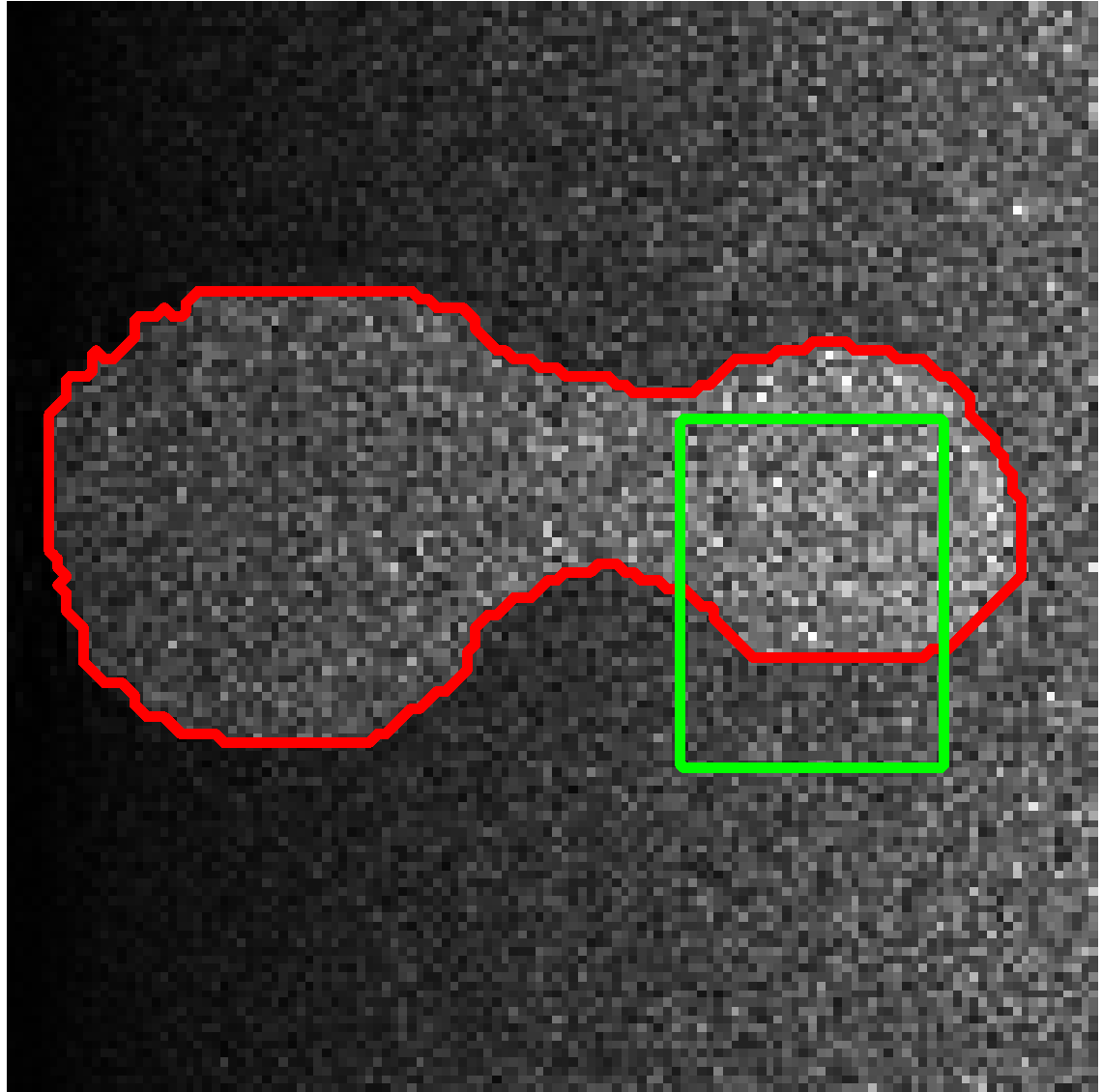}
  \end{subfigure}
  \hfill
  \begin{subfigure}[b]{0.19\linewidth}
      \centering
      \includegraphics[width=\textwidth]{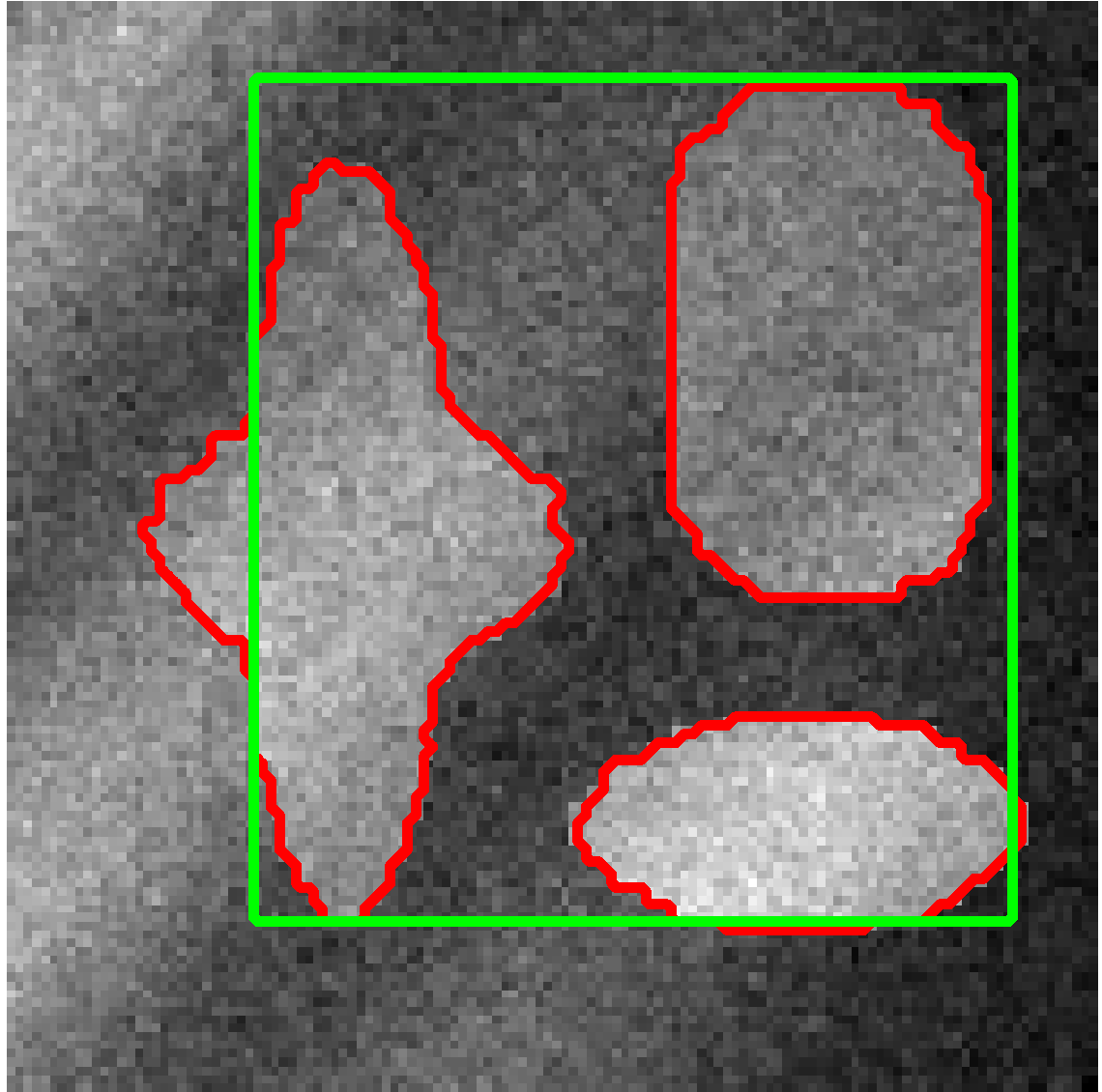}
  \end{subfigure}
 \hfill
  \begin{subfigure}[b]{0.19\linewidth}
      \centering
      \includegraphics[width=\textwidth]{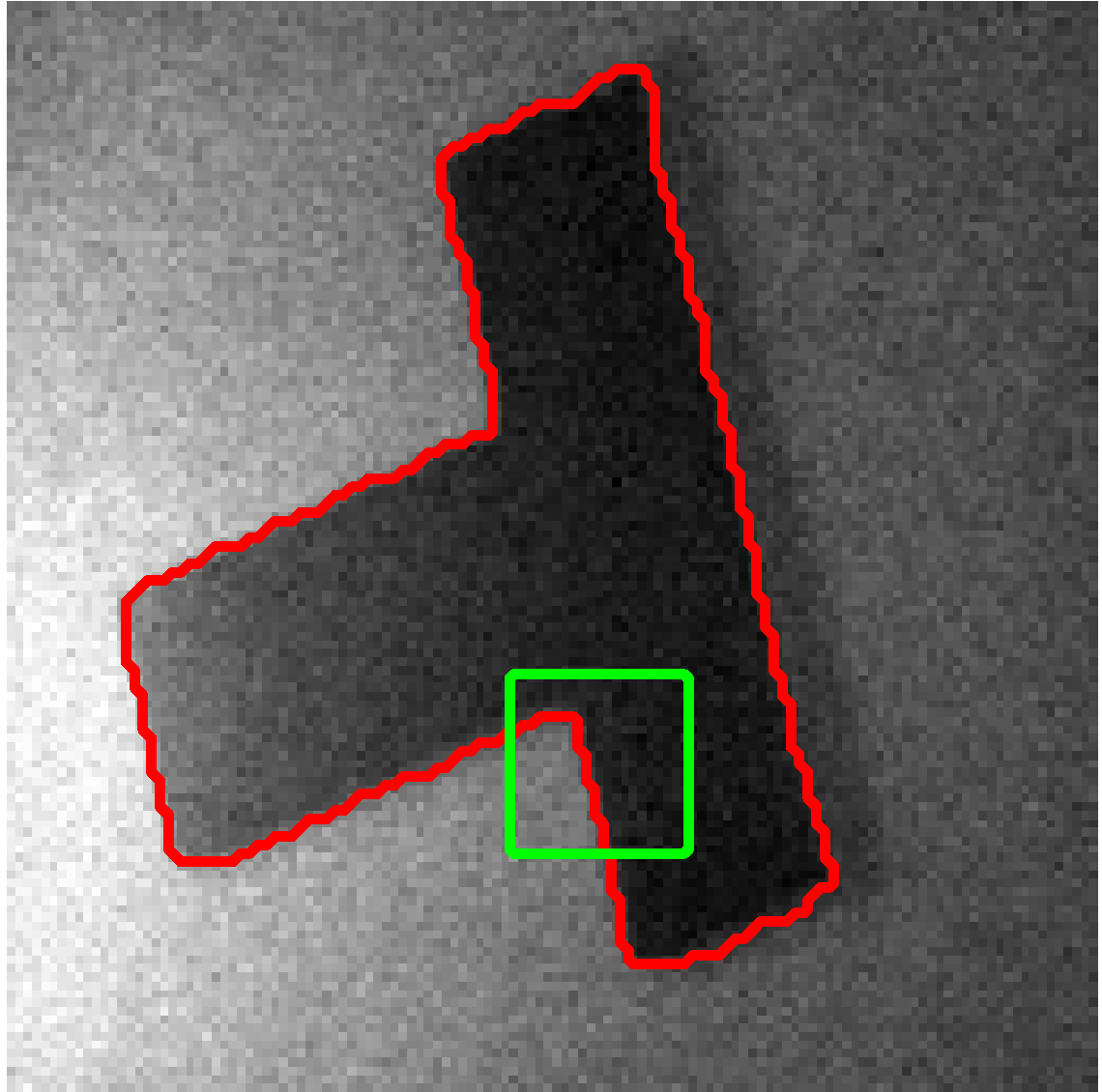}
  \end{subfigure}
  \hfill
  \begin{subfigure}[b]{0.19\linewidth}
      \centering
      \includegraphics[width=\textwidth]{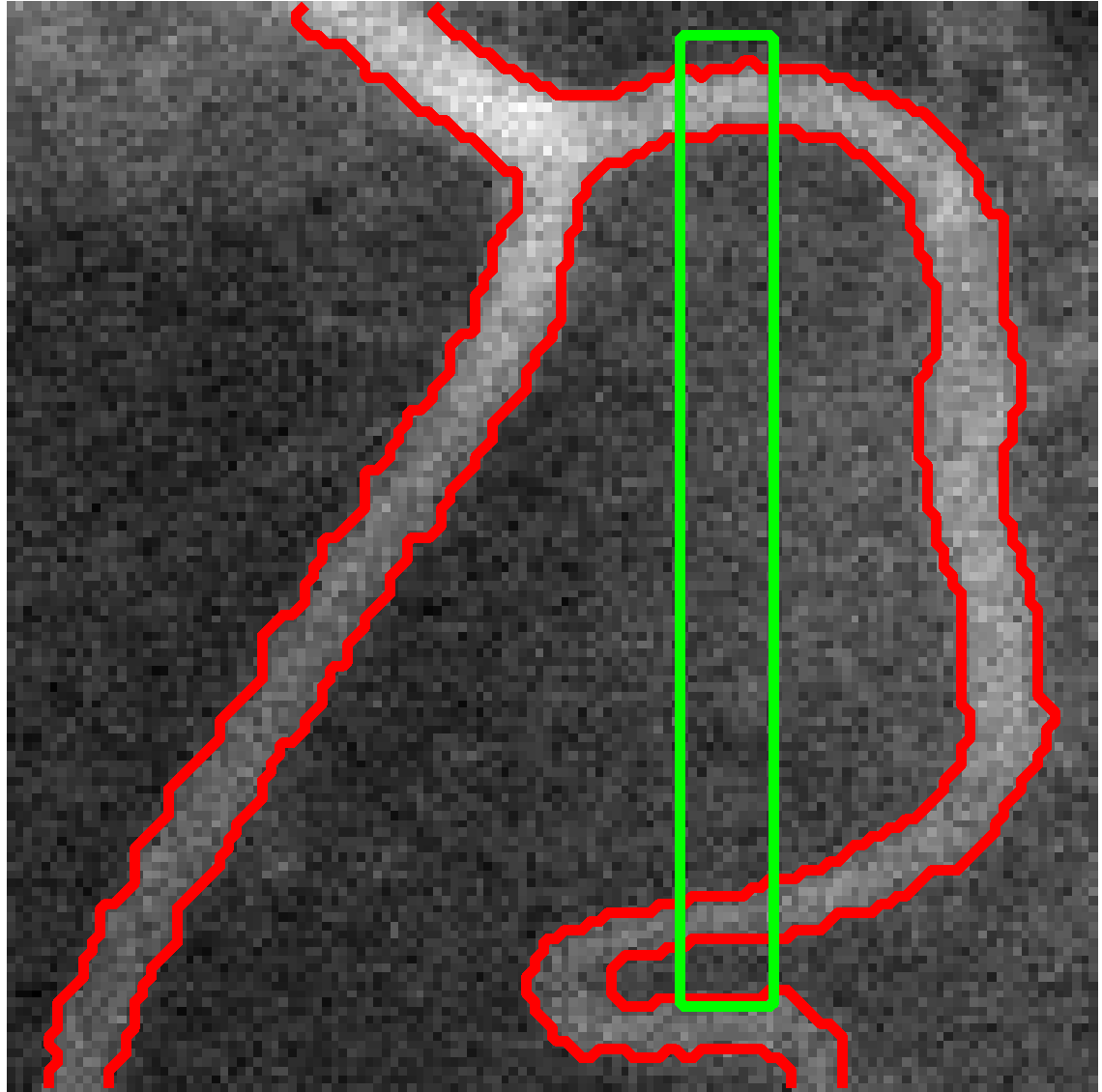}
  \end{subfigure}
  \hfill
  \begin{subfigure}[b]{0.19\linewidth}
      \centering
      \includegraphics[width=\textwidth]{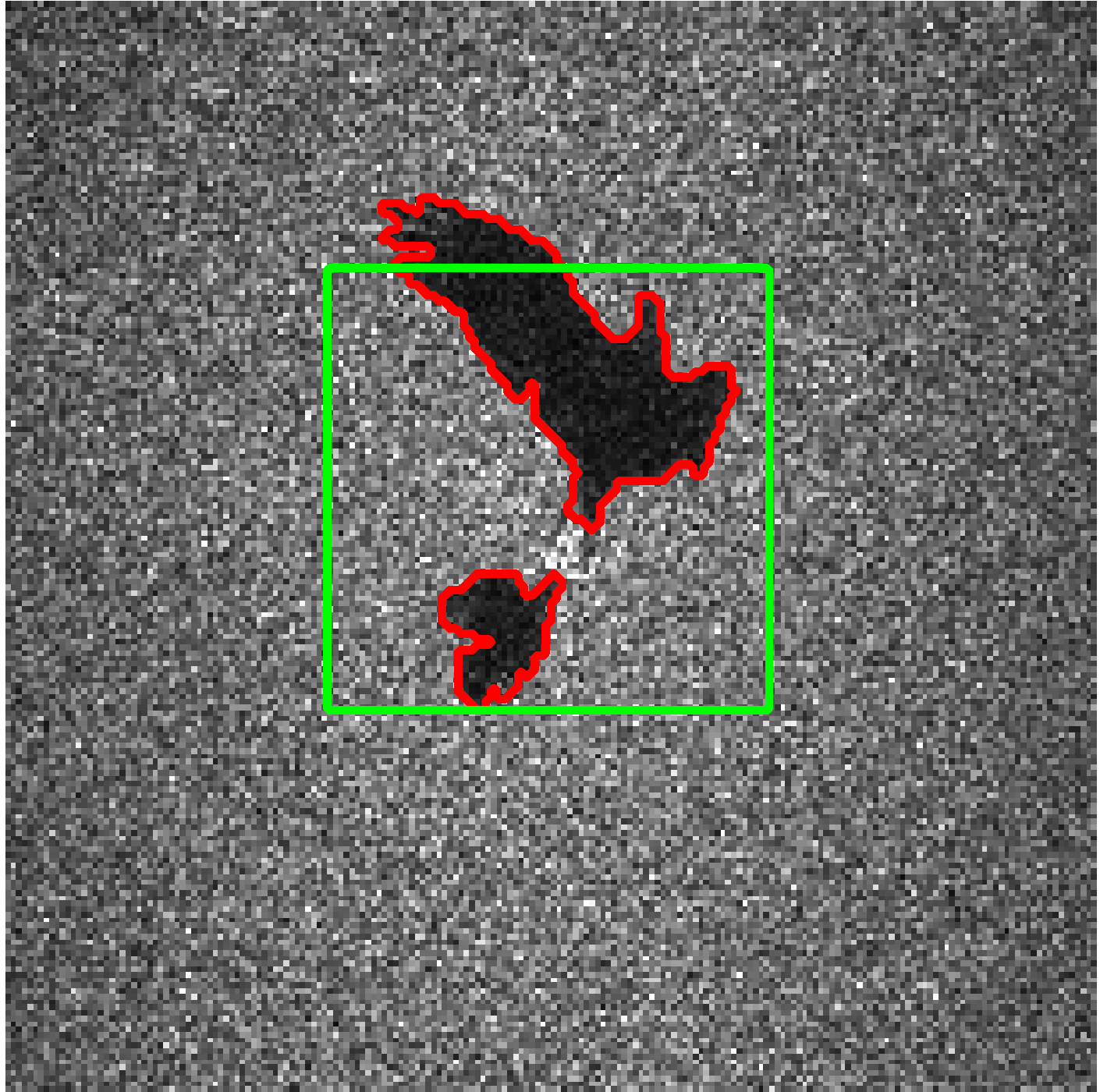}
  \end{subfigure}

    \begin{subfigure}[b]{0.19\linewidth}
      \centering
      \includegraphics[width=\textwidth]{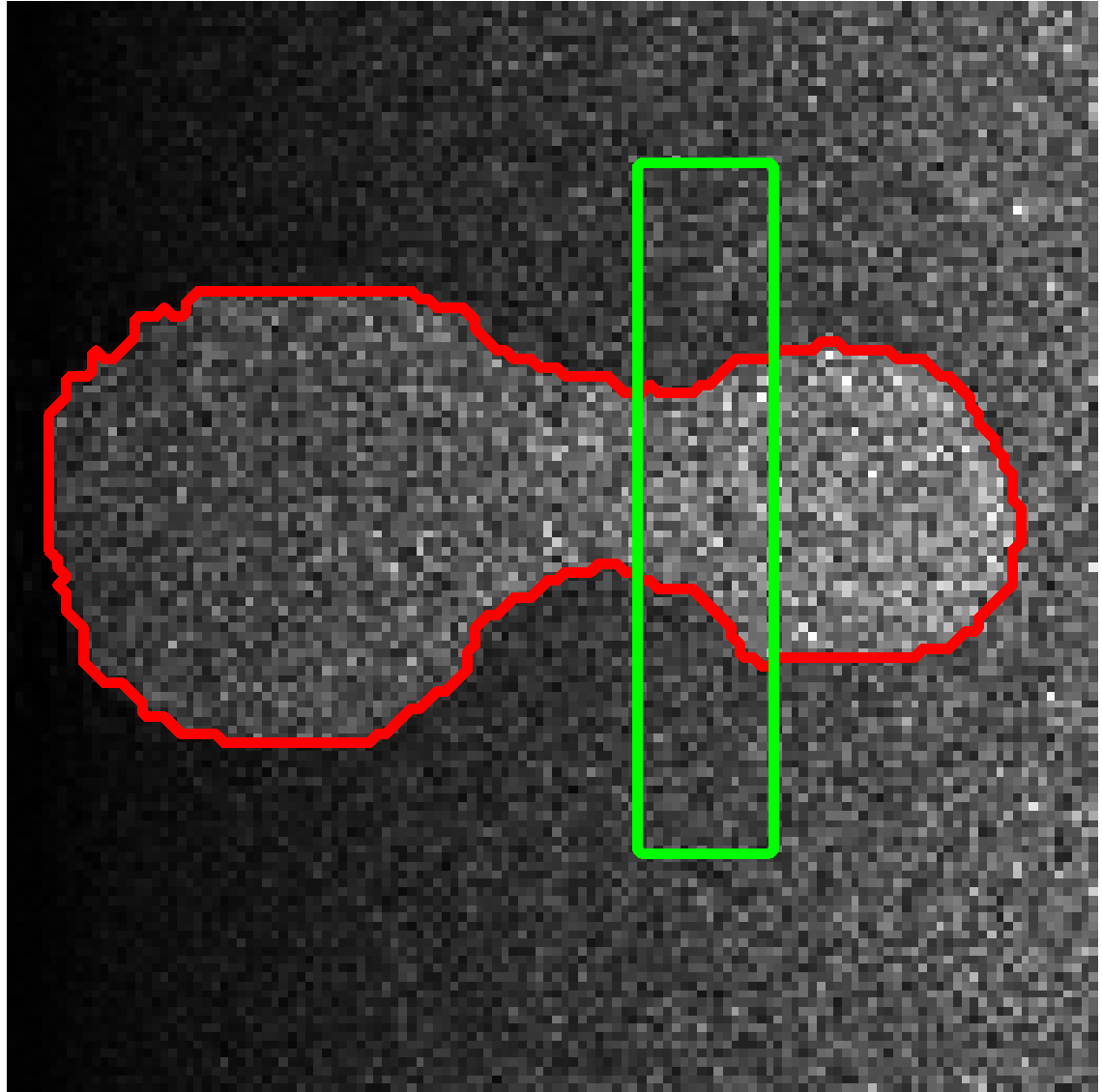}
      \caption{}
      \label{fig:initial-1}
  \end{subfigure}
  \hfill
  \begin{subfigure}[b]{0.19\linewidth}
      \centering
      \includegraphics[width=\textwidth]{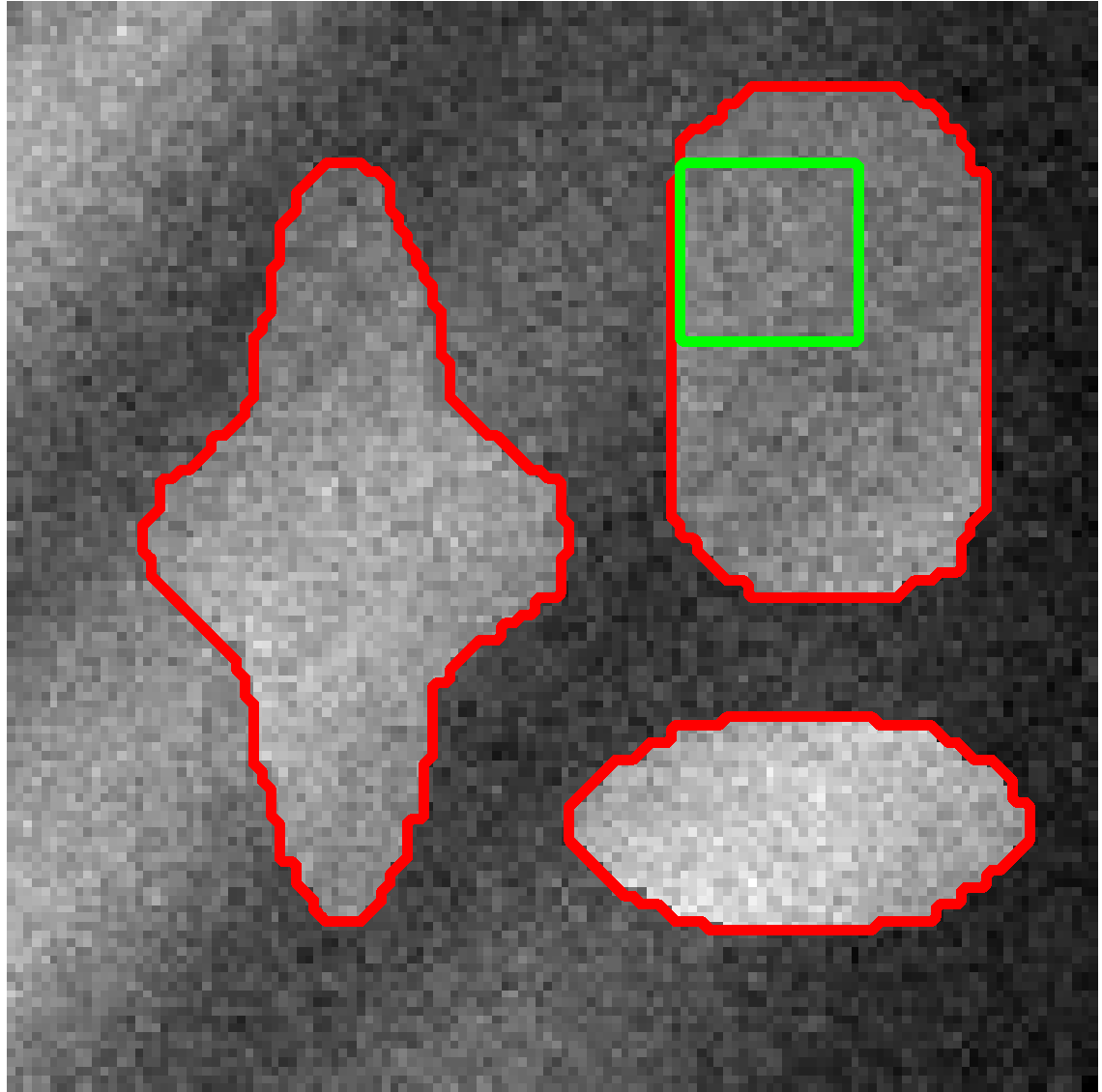}
      \caption{}
      \label{fig:1initial-2}
  \end{subfigure}
 \hfill
  \begin{subfigure}[b]{0.19\linewidth}
      \centering
      \includegraphics[width=\textwidth]{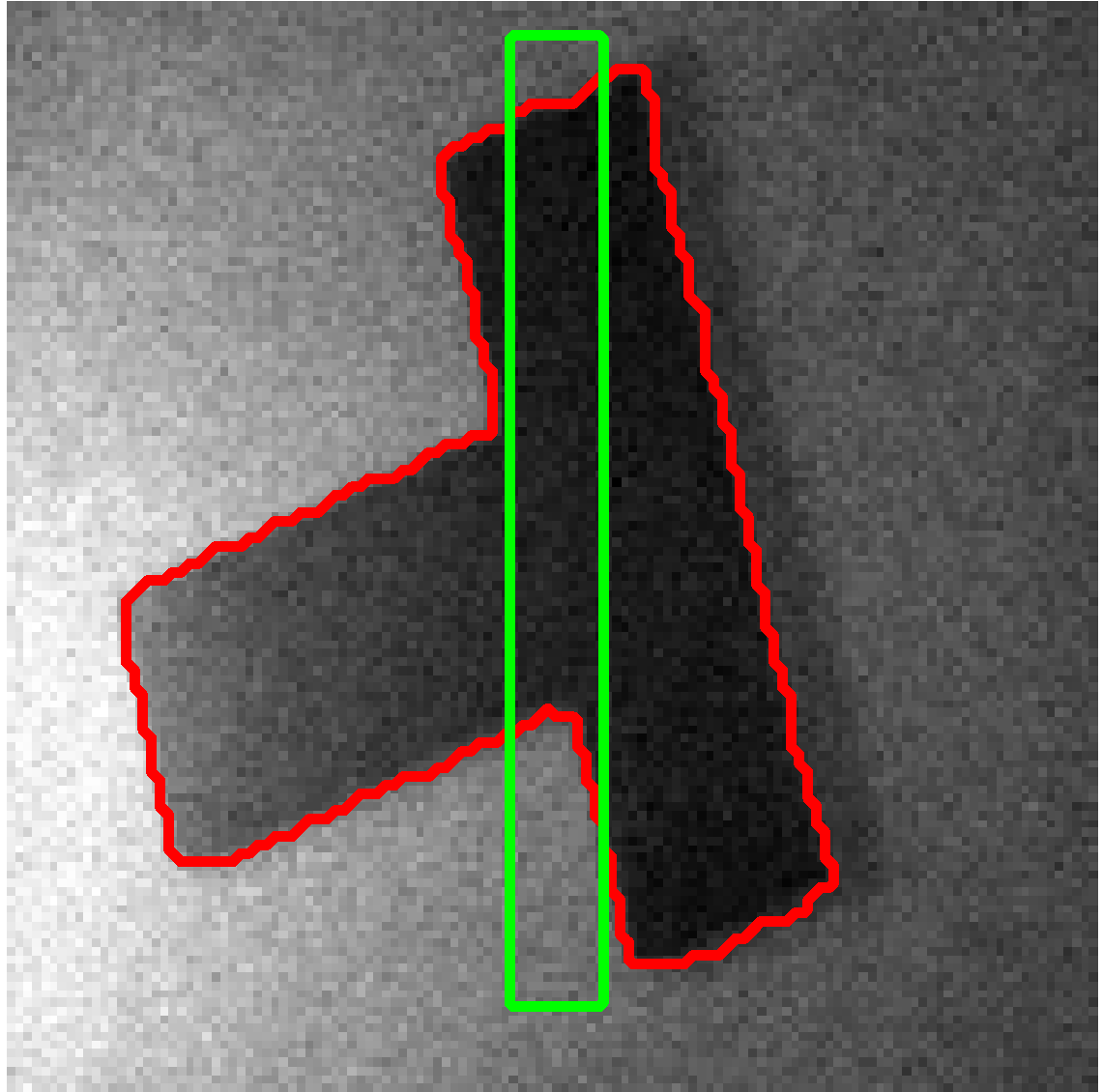}
      \caption{}
      \label{fig:initial-3}
  \end{subfigure}
  \hfill
  \begin{subfigure}[b]{0.19\linewidth}
      \centering
      \includegraphics[width=\textwidth]{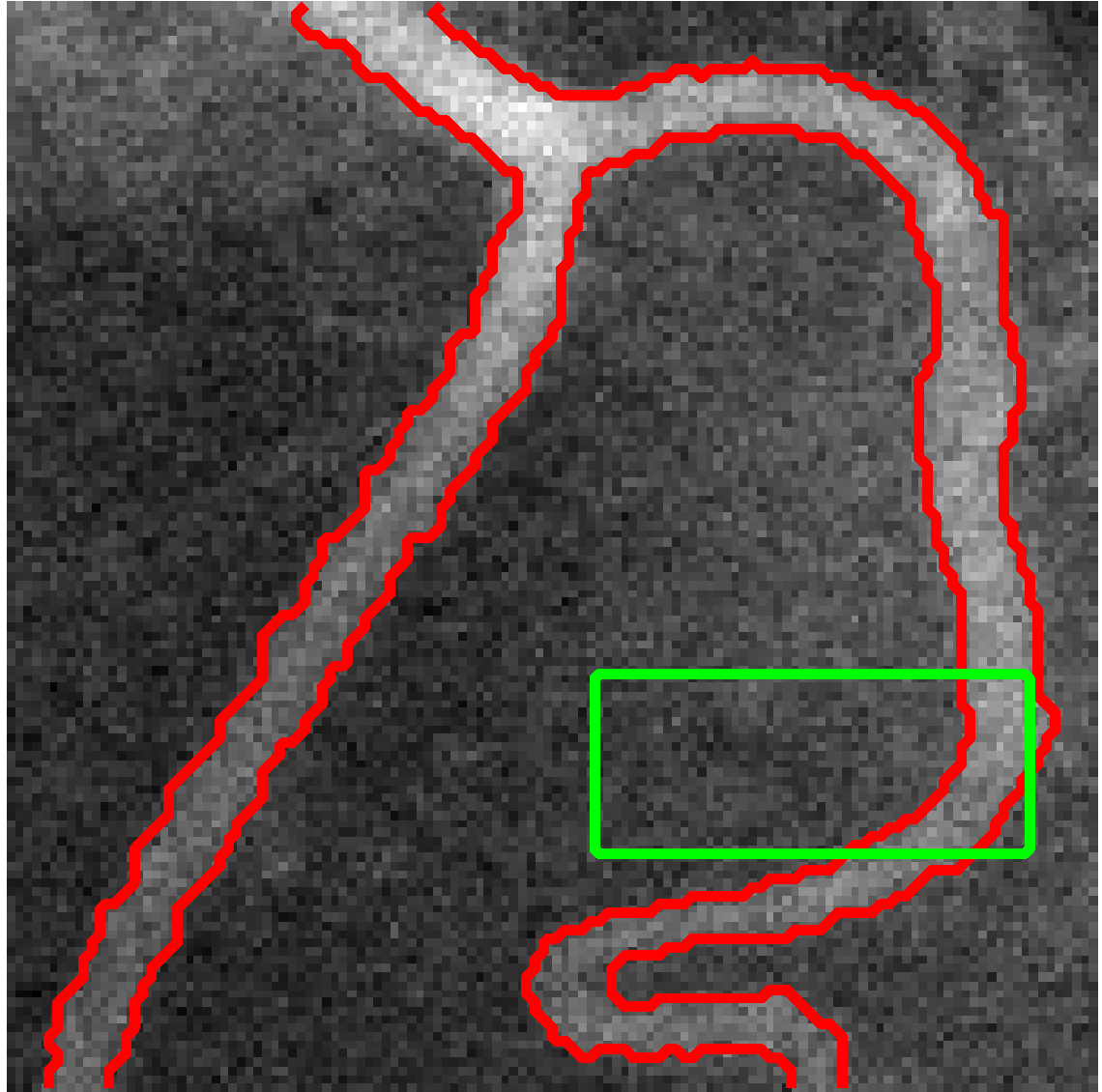}
      \caption{}
      \label{fig:initial-4}
  \end{subfigure}
  \hfill
  \begin{subfigure}[b]{0.19\linewidth}
      \centering
      \includegraphics[width=\textwidth]{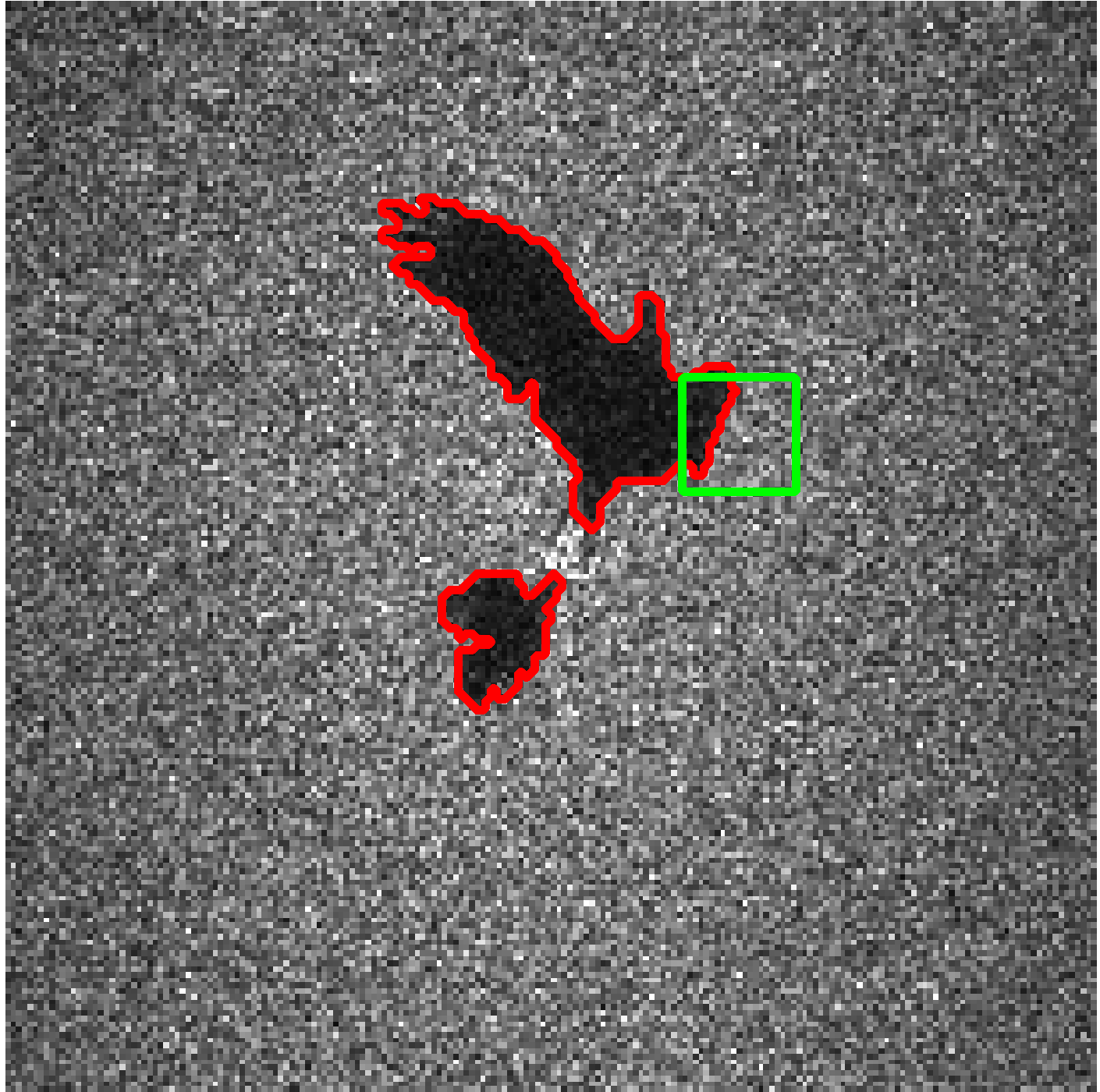}
      \caption{}
      \label{fig:initial-5}
  \end{subfigure}
  \caption{Robustness to initial contours. Each column shows an input image with different initial contours (green) and the corresponding segmentation results (red)}
  \label{fig:initial}
\end{figure}
\begin{figure}
  \centering
  \begin{subfigure}[b]{0.32\linewidth}
      \centering
      \includegraphics[width=\textwidth]{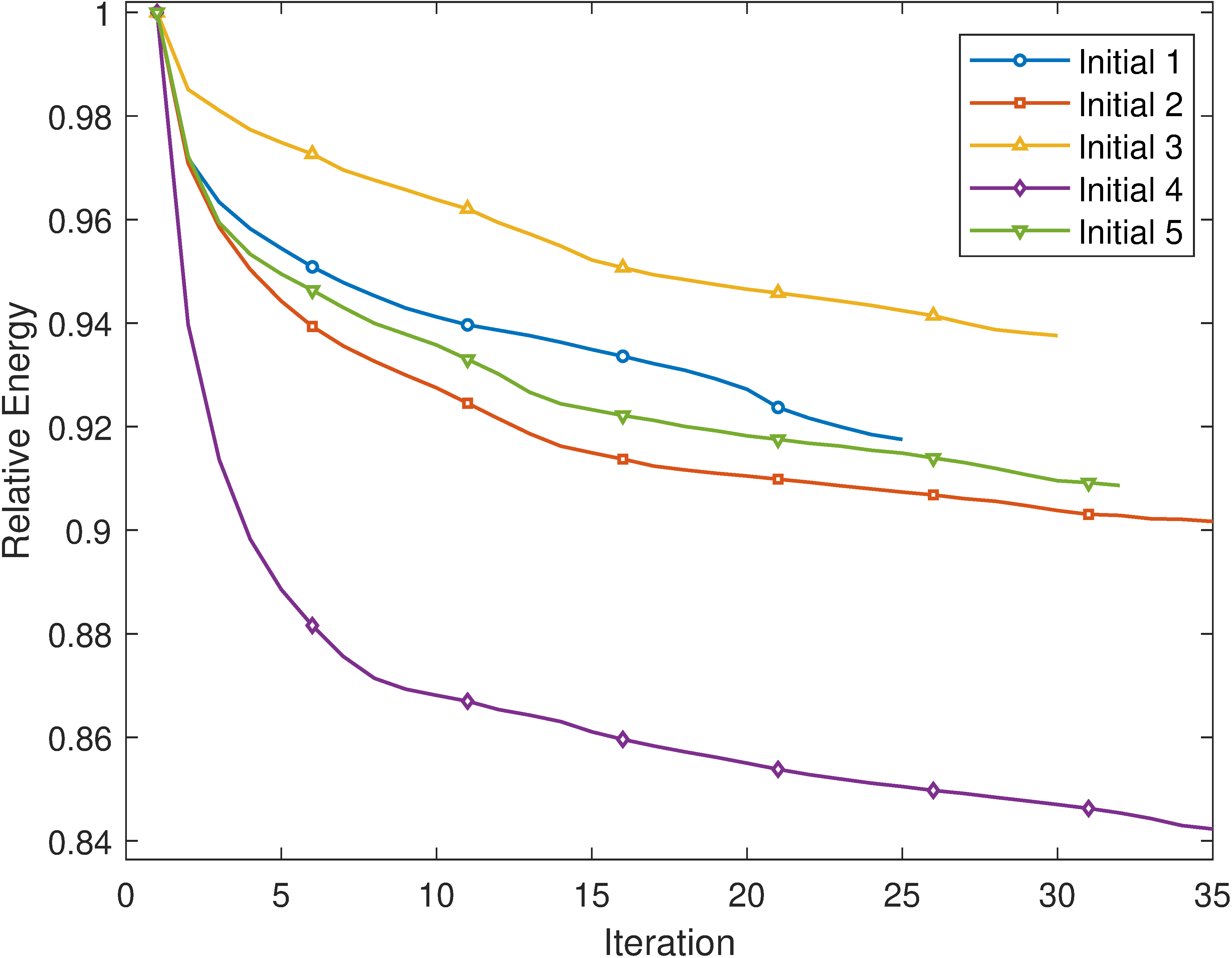}
      \caption{}
  \end{subfigure}
  \hfill
  \begin{subfigure}[b]{0.32\linewidth}
      \centering
      \includegraphics[width=\textwidth]{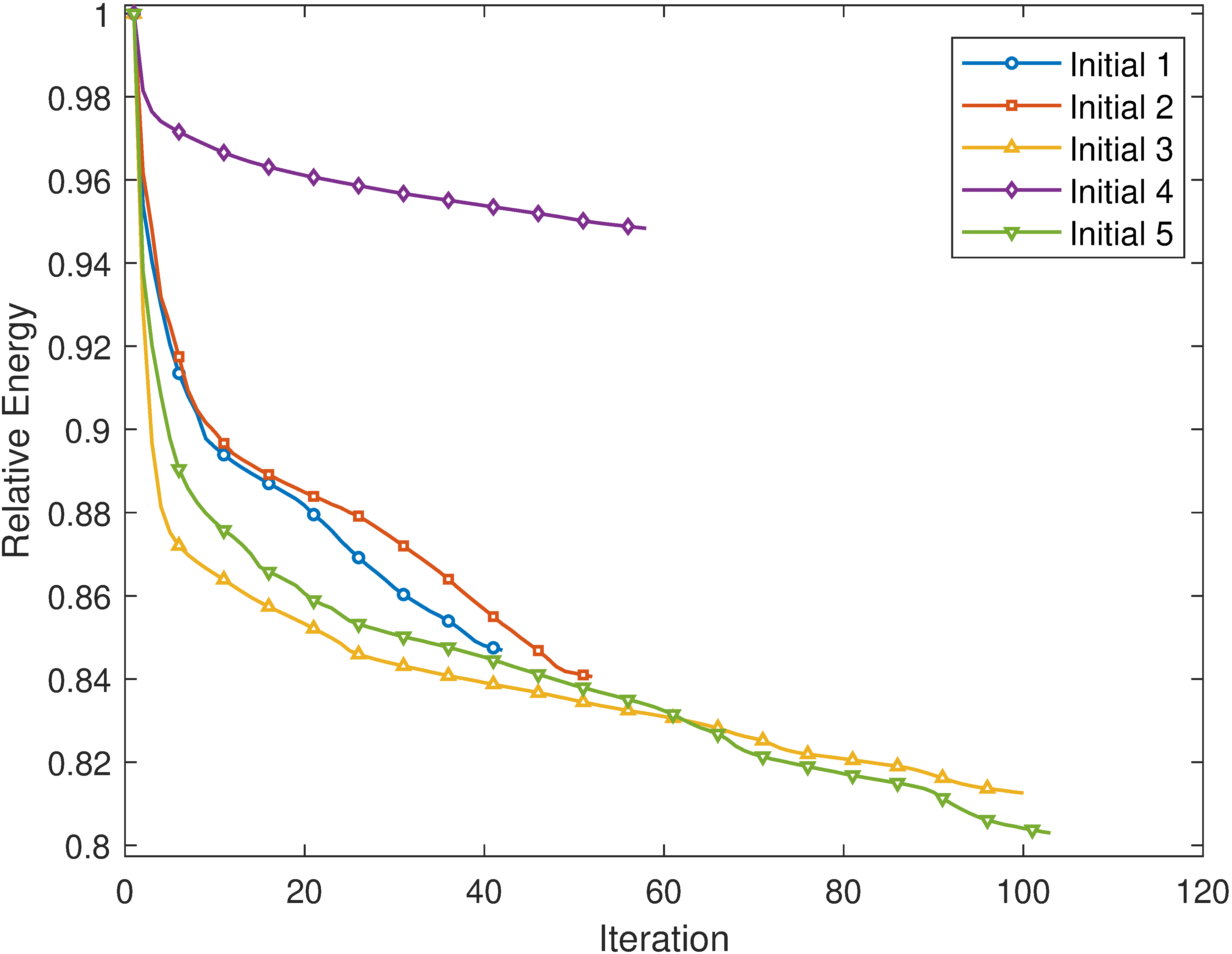}
      \caption{}
  \end{subfigure}
 \hfill
  \begin{subfigure}[b]{0.32\linewidth}
      \centering
      \includegraphics[width=\textwidth]{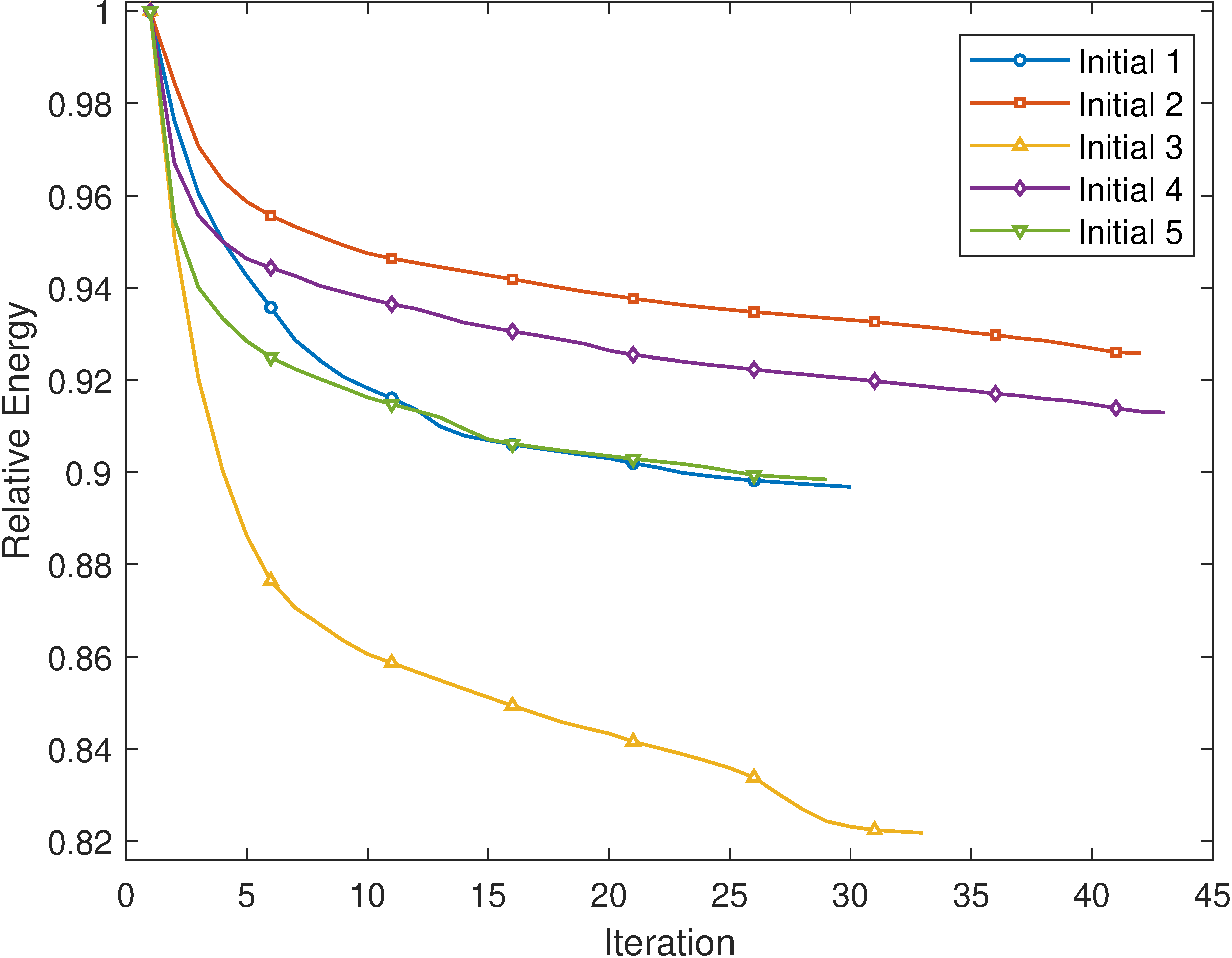}
      \caption{}
  \end{subfigure}

  \begin{subfigure}[b]{0.32\linewidth}
      \centering
      \includegraphics[width=\textwidth]{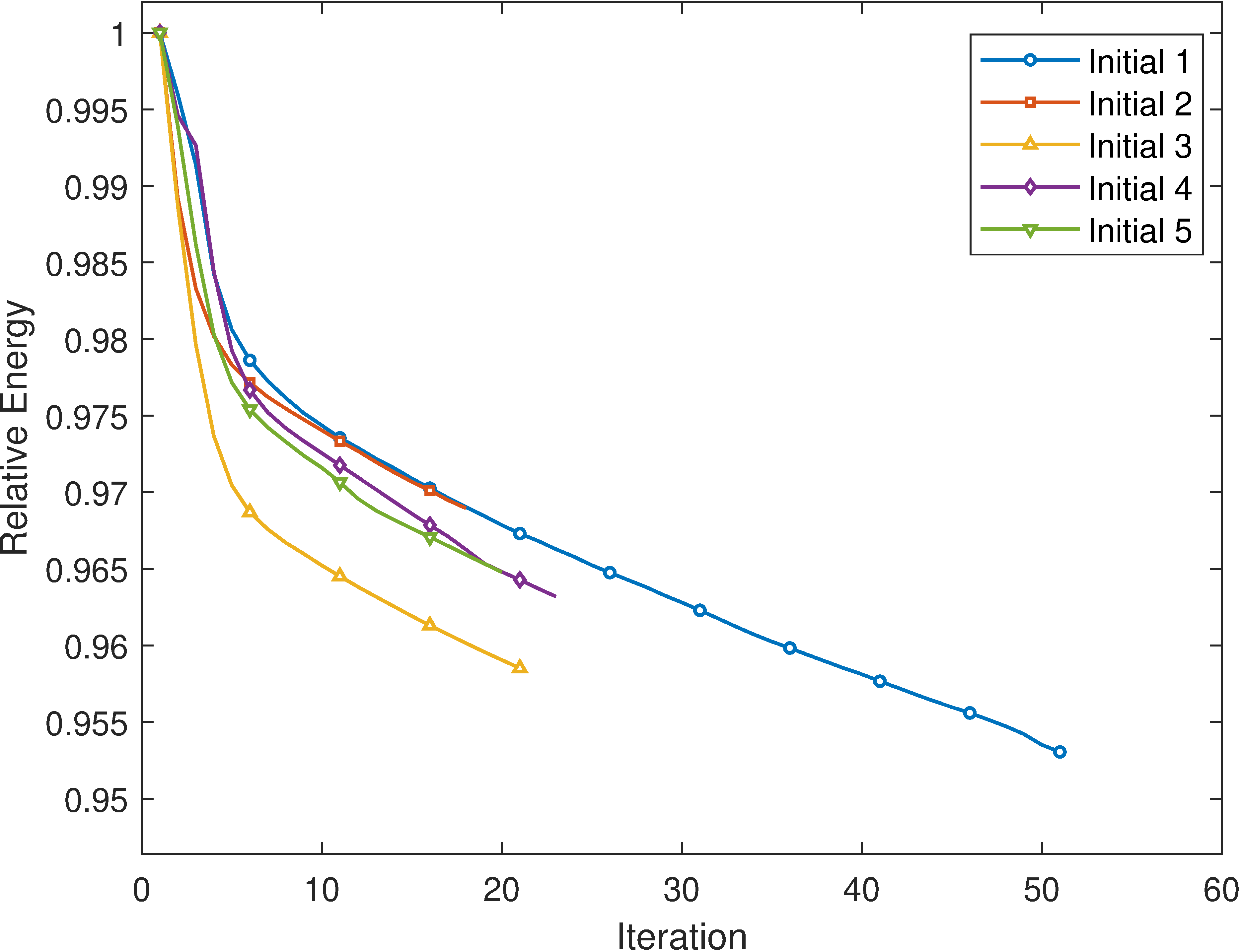}
      \caption{}
  \end{subfigure}
  \hfill
  \begin{subfigure}[b]{0.32\linewidth}
      \centering
      \includegraphics[width=\textwidth]{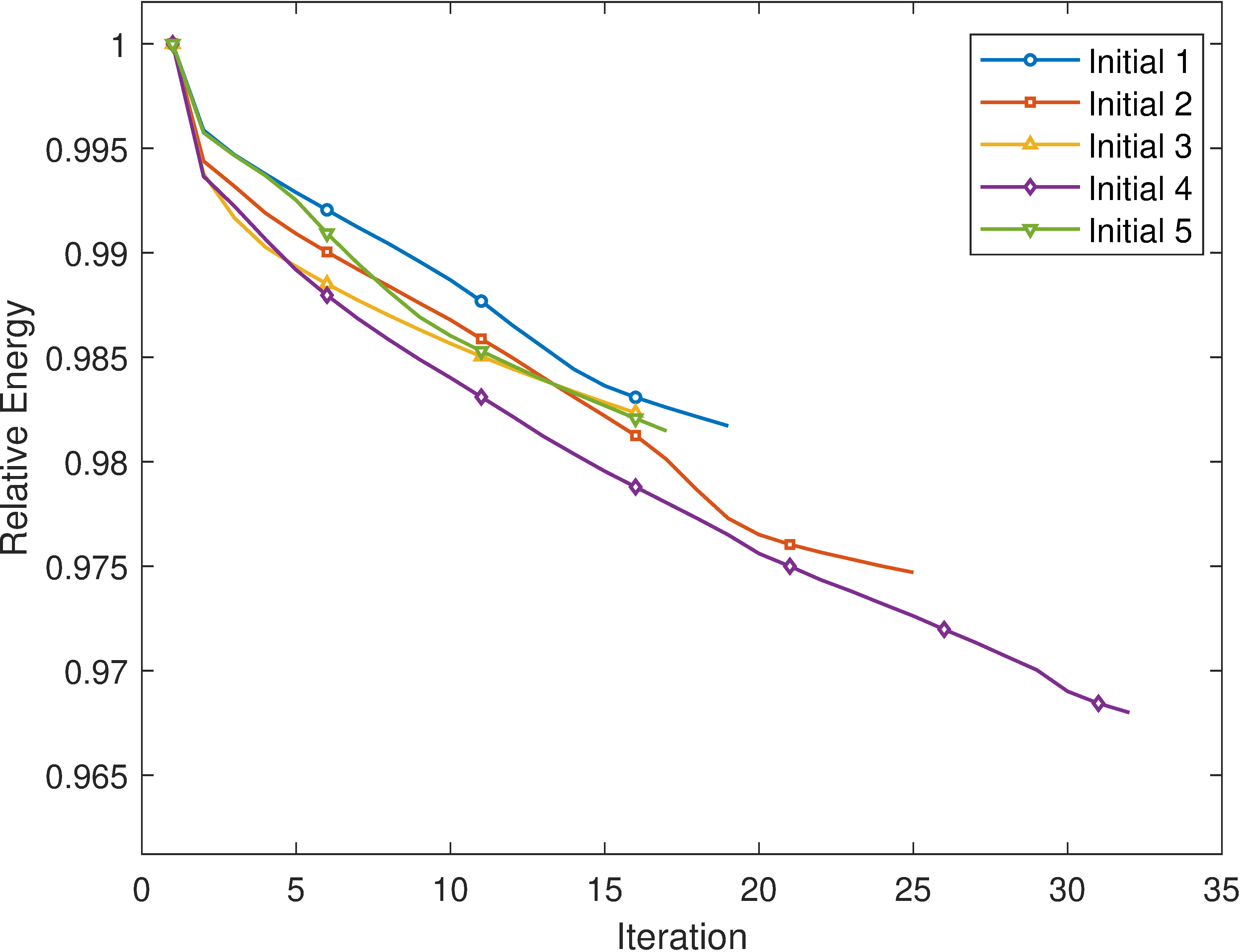}
      \caption{}
  \end{subfigure}
  \hfill
  \begin{subfigure}[b]{0.32\linewidth}
      \centering
      \includegraphics[width=\textwidth]{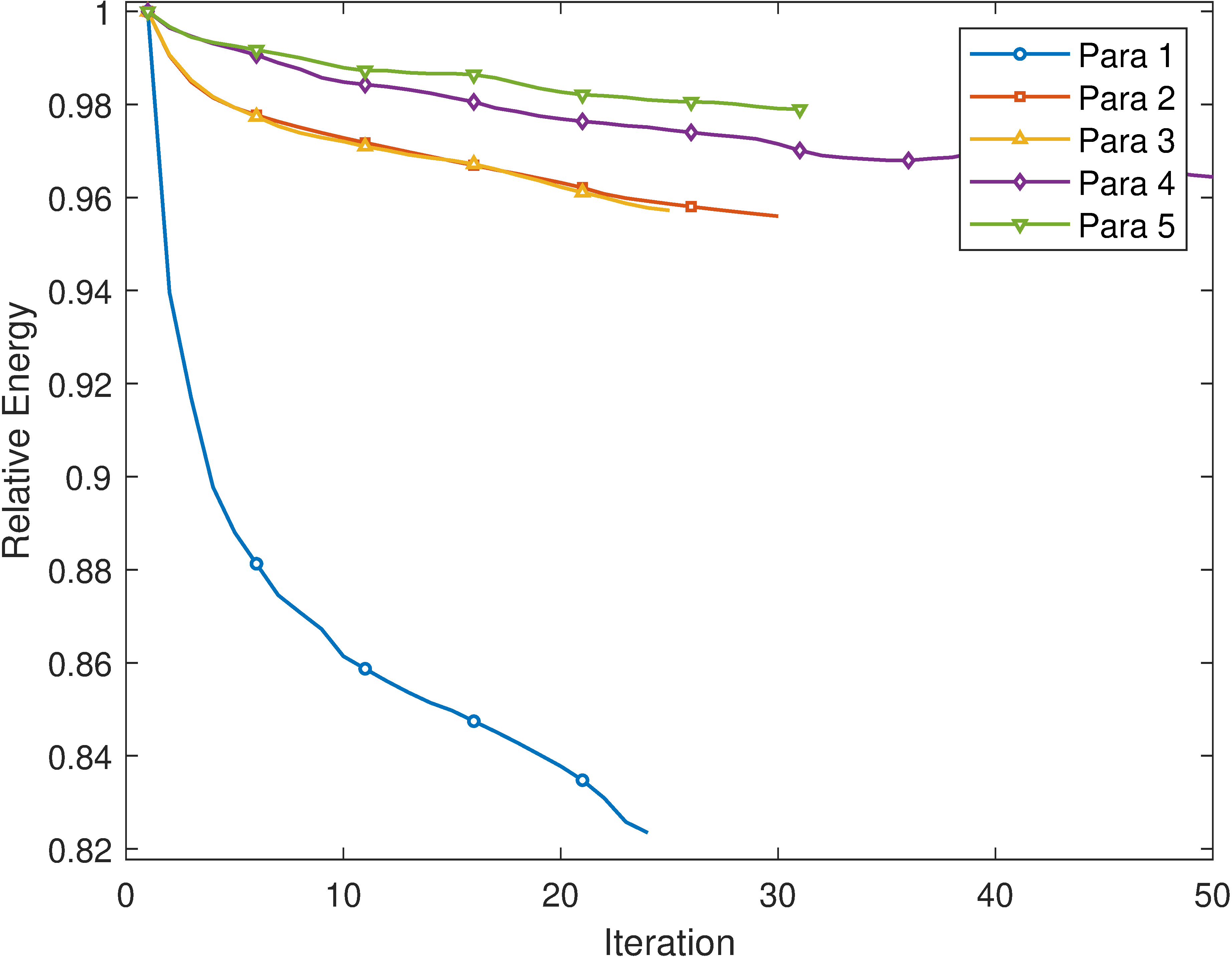}
      \caption{}
      \label{fig:energy-004-para}
  \end{subfigure}
  \caption{Energy of different image. from (a)--(e) correspond to Fig.~\ref{fig:initial-1}--\ref{fig:initial-5}, and (f) correspond to Fig.~\ref{fig:004-para}}
  \label{fig:energy}
\end{figure}
To evaluate the robustness of our model with respect to the initial contour, we conduct segmentation experiments on five images, initialized with five randomly selected contours. In Fig.~\ref{fig:initial}, the green and red curves depict the initial and final contours, respectively. It can be observed that our model consistently yields accurate segmentation results despite different initializations. Furthermore, Fig.~\ref{fig:energy} illustrates the energy evolution throughout the segmentation process for different initializations. For clarity of presentation, all energy values are normalized. These energy curves further demonstrate the effectiveness and stability of the proposed numerical scheme.

\subsubsection{Robustness to Denoising Parameter}
\begin{figure}
  \centering
  \begin{subfigure}[b]{0.16\linewidth}
      \centering
      \includegraphics[width=\textwidth]{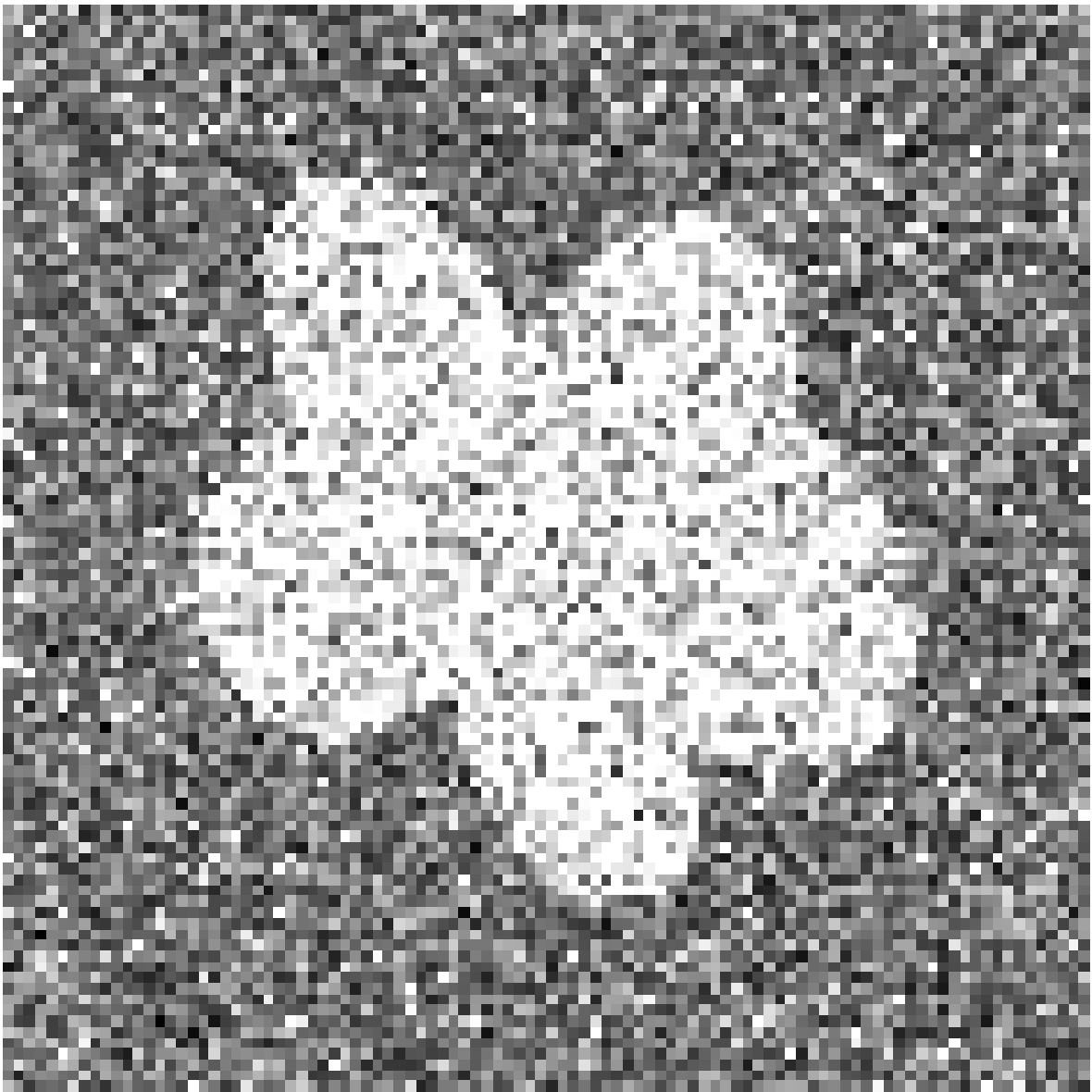}
      \caption{}
  \end{subfigure}
  \hfill
  \begin{subfigure}[b]{0.16\linewidth}
      \centering
      \includegraphics[width=\textwidth]{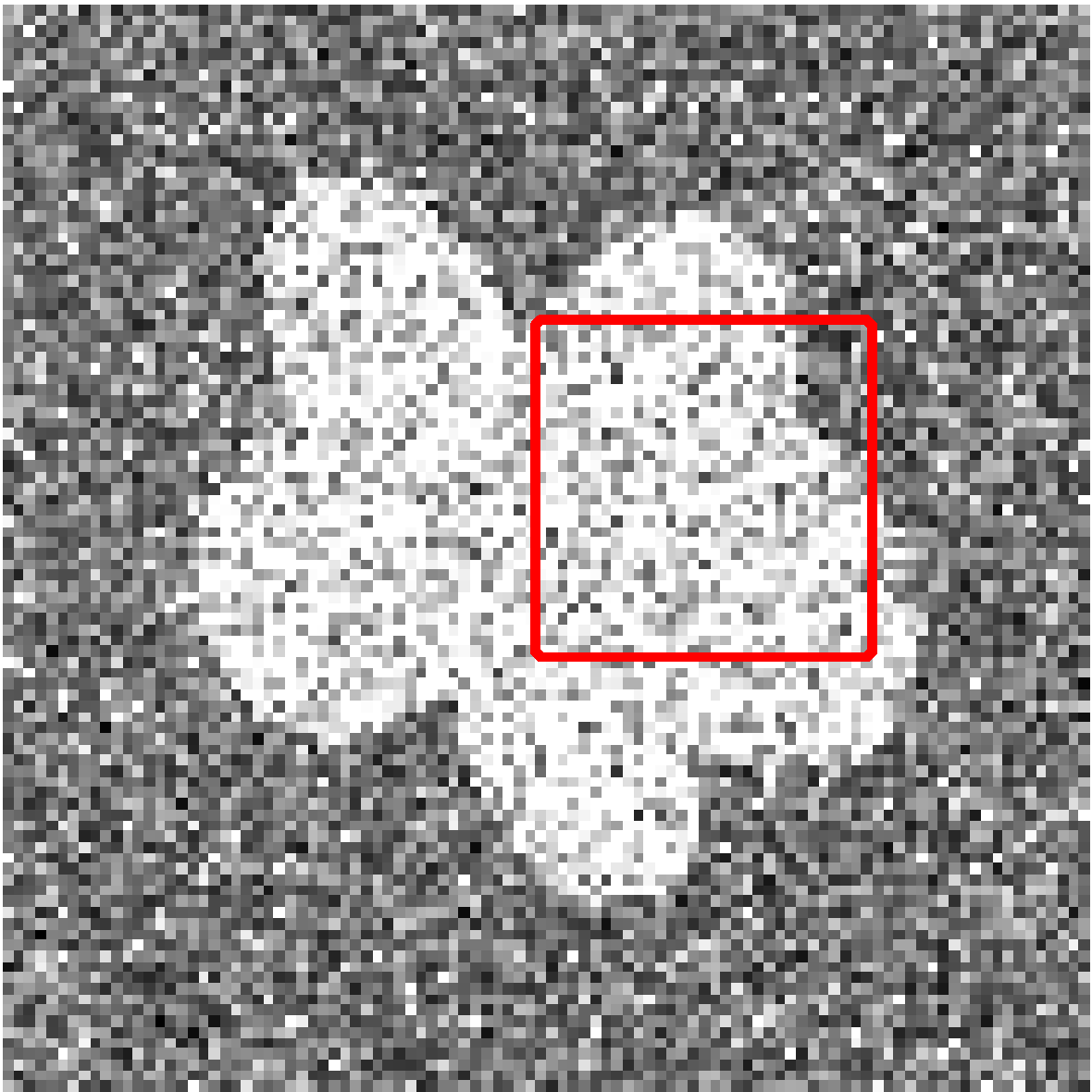}
      \caption{}
  \end{subfigure}
   \hfill
  \begin{subfigure}[b]{0.16\linewidth}
      \centering
      \includegraphics[width=\textwidth]{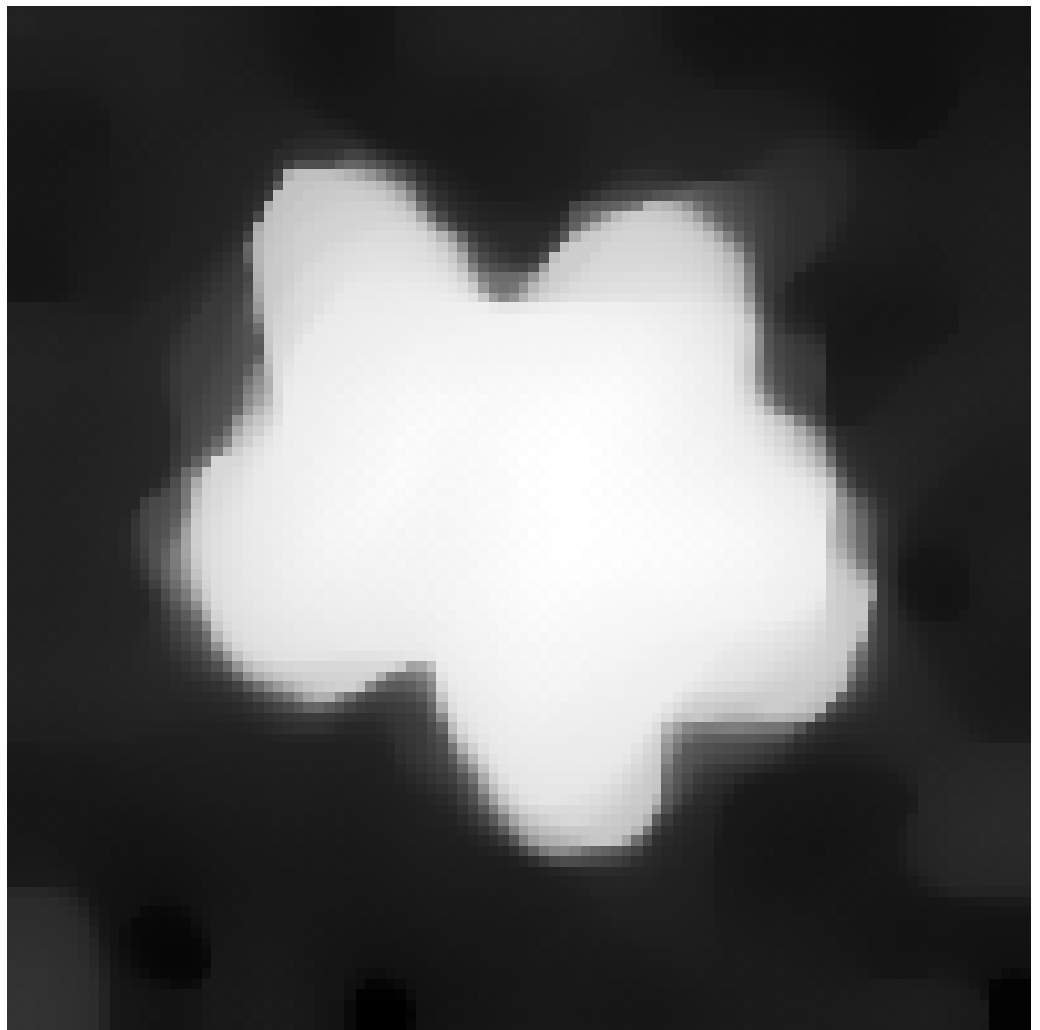}
      \caption{}
  \end{subfigure}
   \hfill
  \begin{subfigure}[b]{0.16\linewidth}
      \centering
      \includegraphics[width=\textwidth]{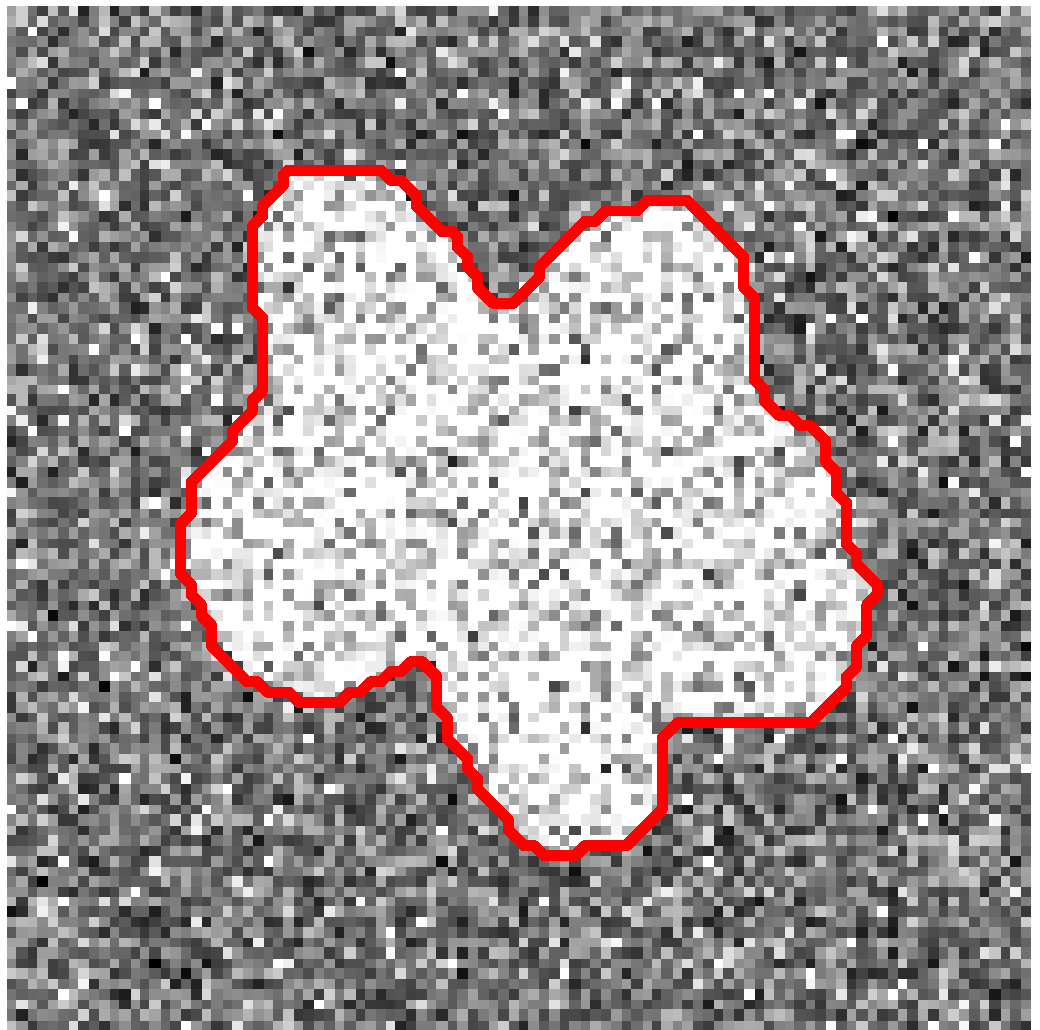}
      \caption{}
  \end{subfigure}
   \hfill
  \begin{subfigure}[b]{0.16\linewidth}
      \centering
      \includegraphics[width=\textwidth]{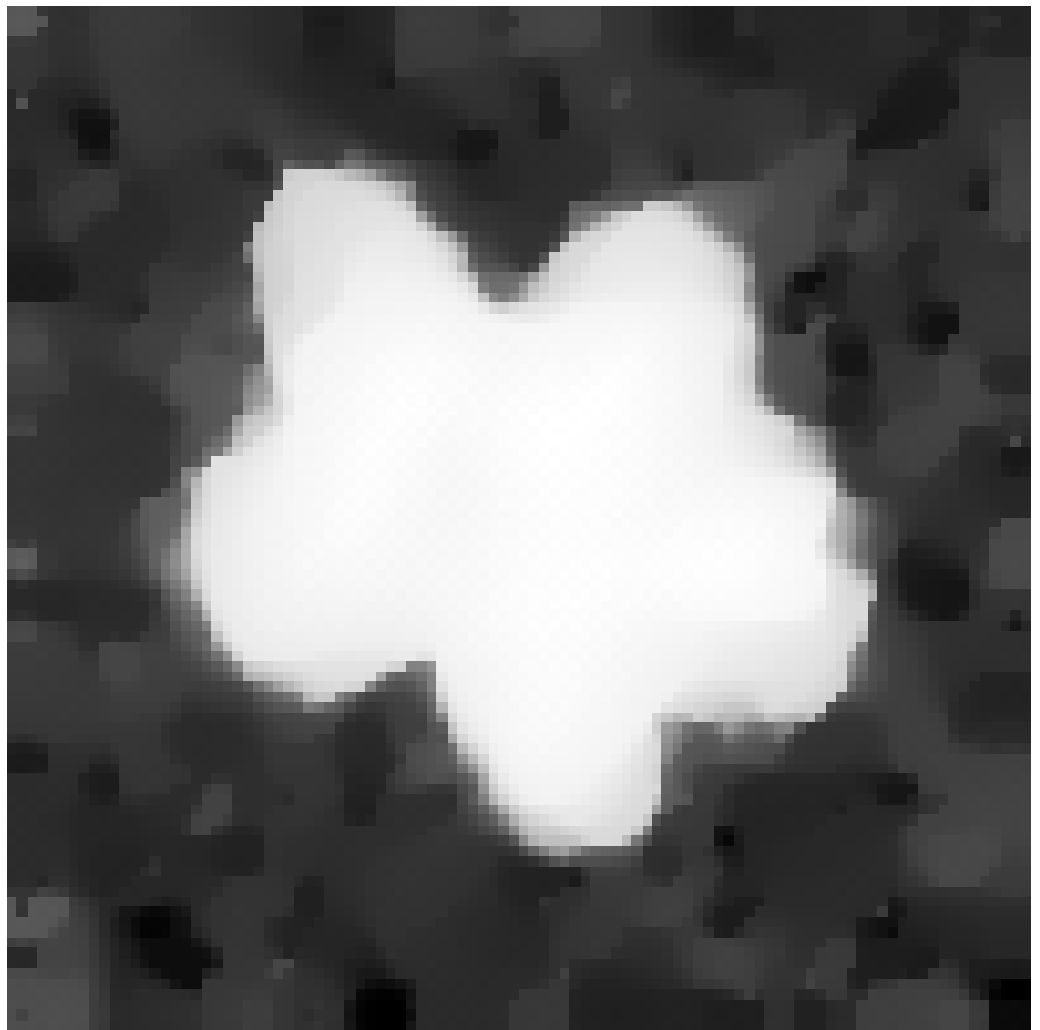}
      \caption{}
  \end{subfigure}
   \hfill
  \begin{subfigure}[b]{0.16\linewidth}
      \centering
      \includegraphics[width=\textwidth]{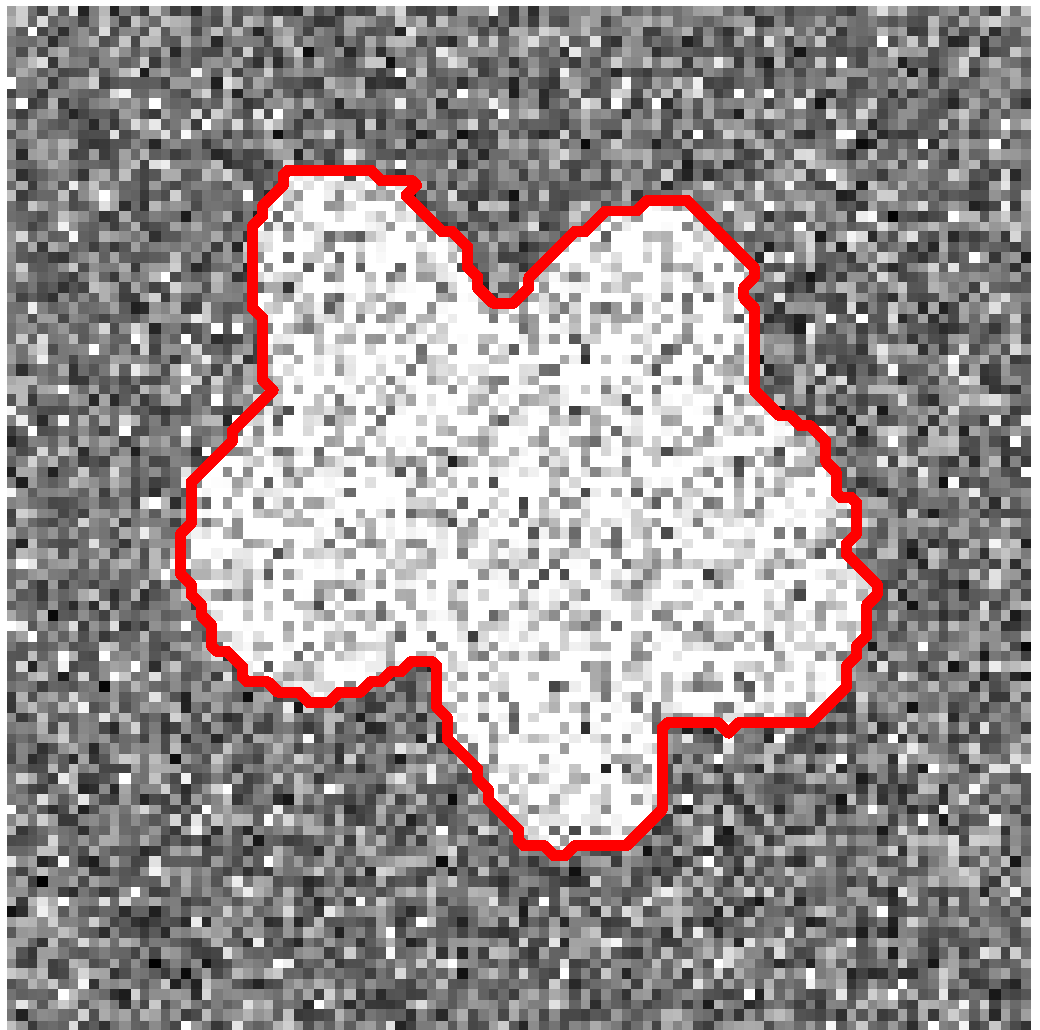}
      \caption{}
  \end{subfigure}

  \begin{subfigure}[b]{0.16\linewidth}
      \centering
      \includegraphics[width=\textwidth]{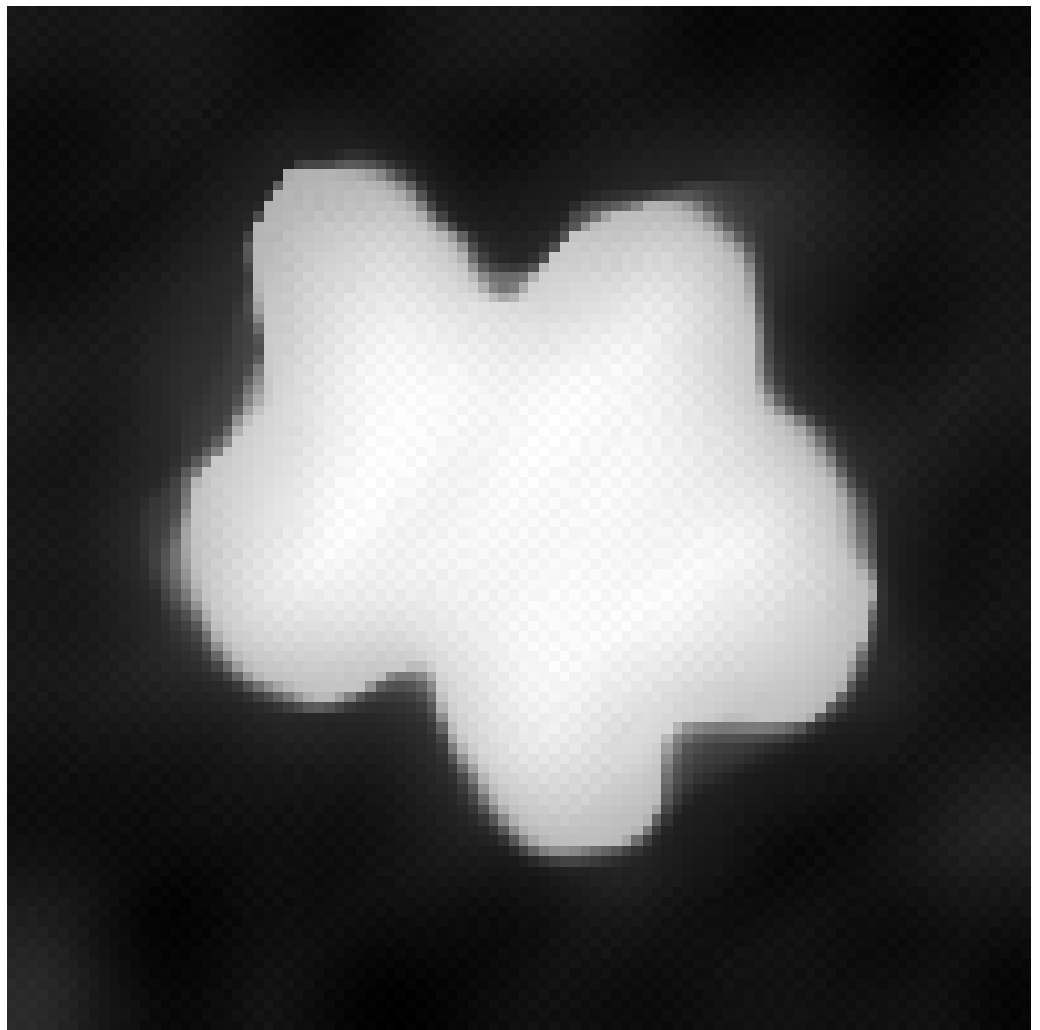}
      \caption{}
  \end{subfigure}
  \hfill
  \begin{subfigure}[b]{0.16\linewidth}
      \centering
      \includegraphics[width=\textwidth]{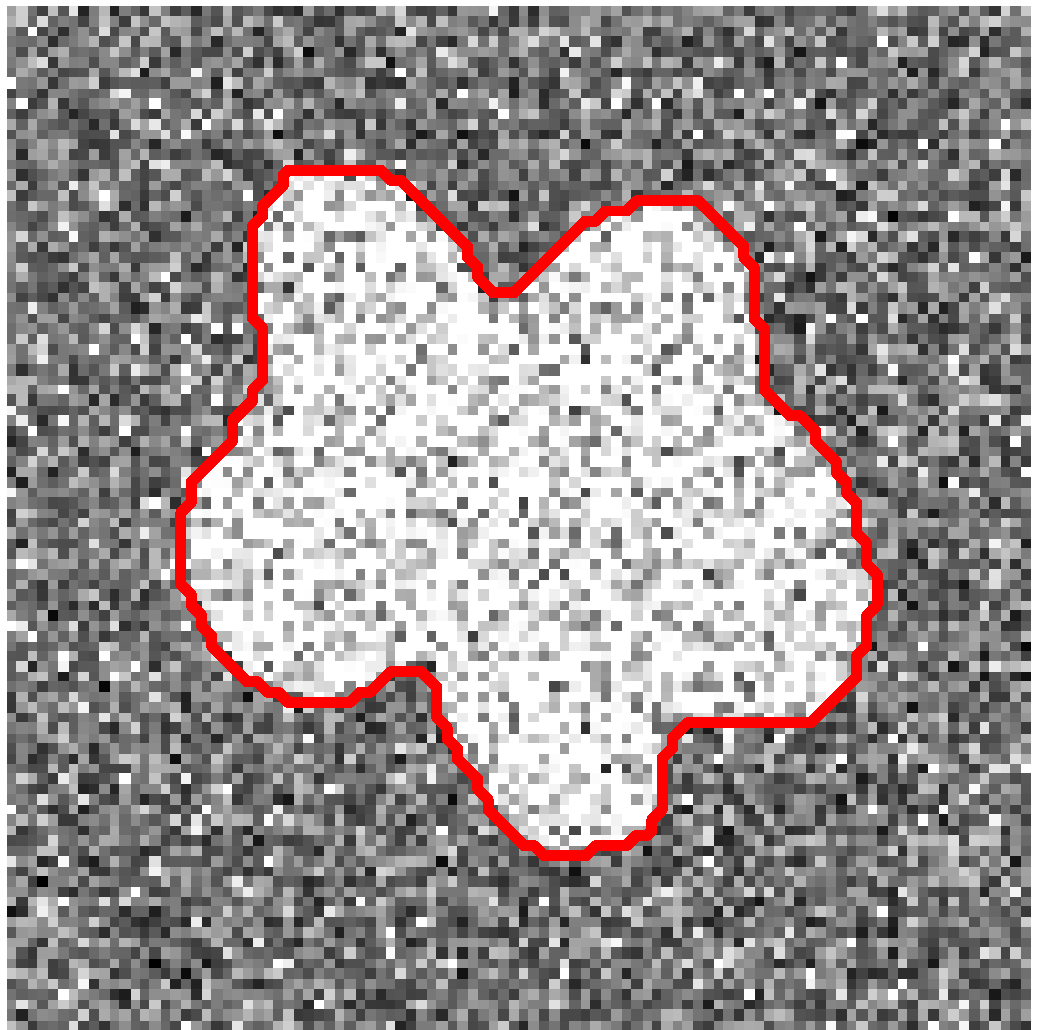}
      \caption{}
  \end{subfigure}
  \hfill
  \begin{subfigure}[b]{0.16\linewidth}
      \centering
      \includegraphics[width=\textwidth]{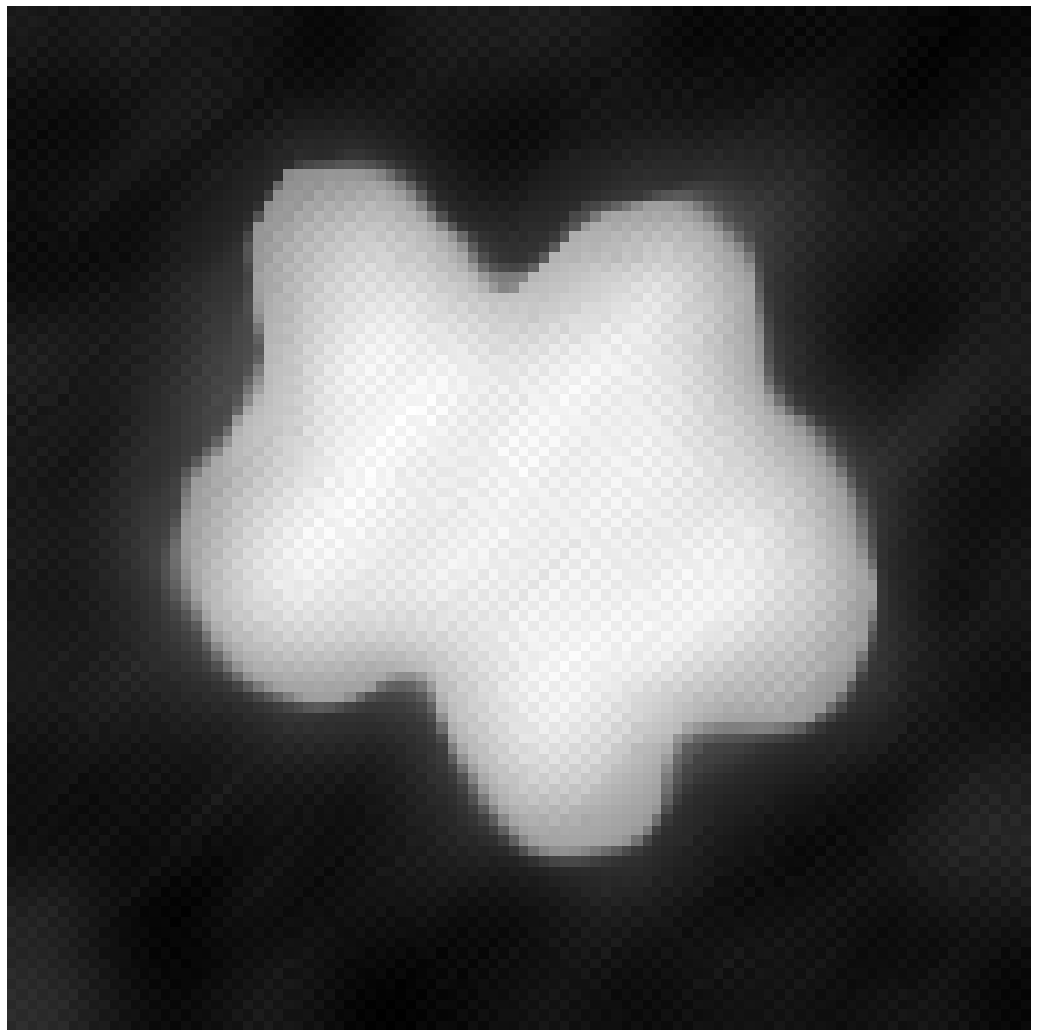}
      \caption{}
  \end{subfigure}
 \hfill
  \begin{subfigure}[b]{0.16\linewidth}
      \centering
      \includegraphics[width=\textwidth]{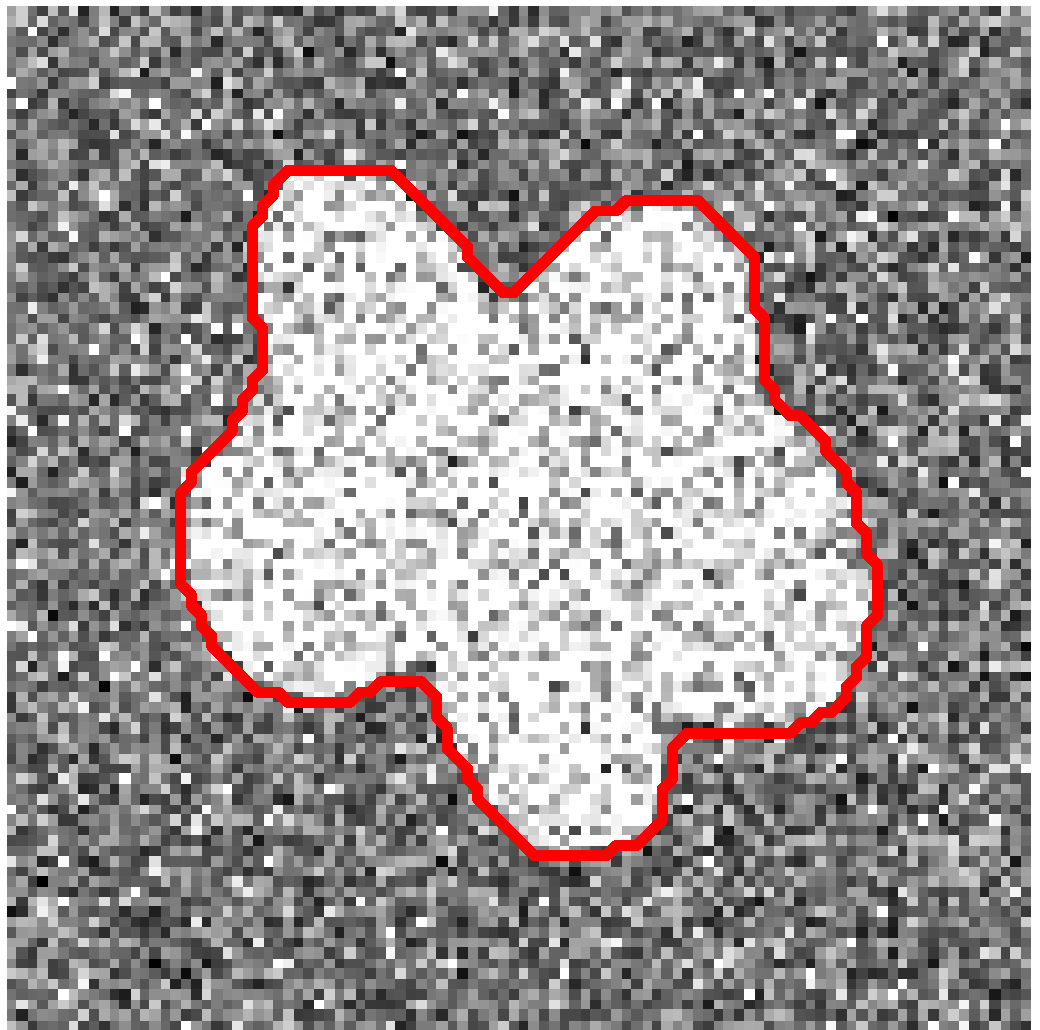}
      \caption{}
  \end{subfigure}
  \hfill
  \begin{subfigure}[b]{0.16\linewidth}
      \centering
      \includegraphics[width=\textwidth]{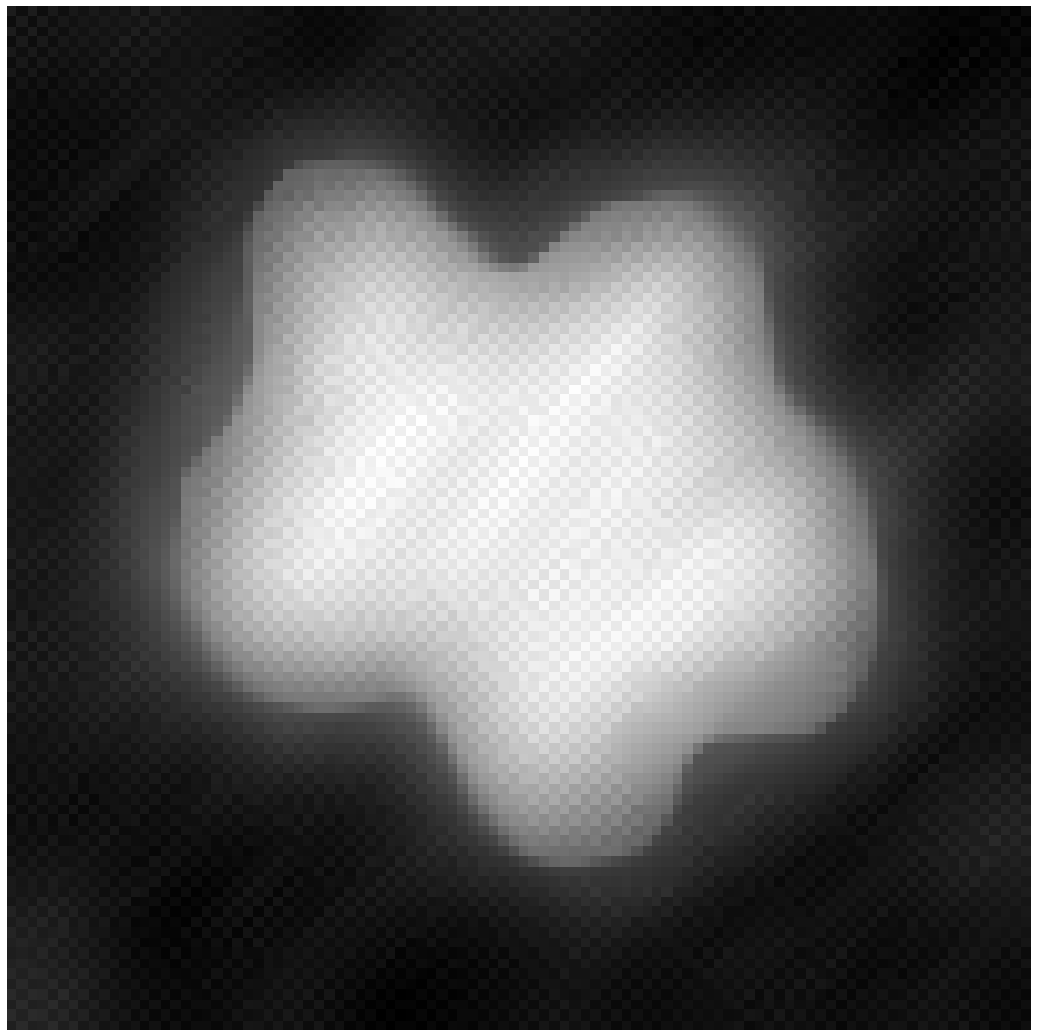}
      \caption{}
  \end{subfigure}
  \hfill
  \begin{subfigure}[b]{0.16\linewidth}
      \centering
      \includegraphics[width=\textwidth]{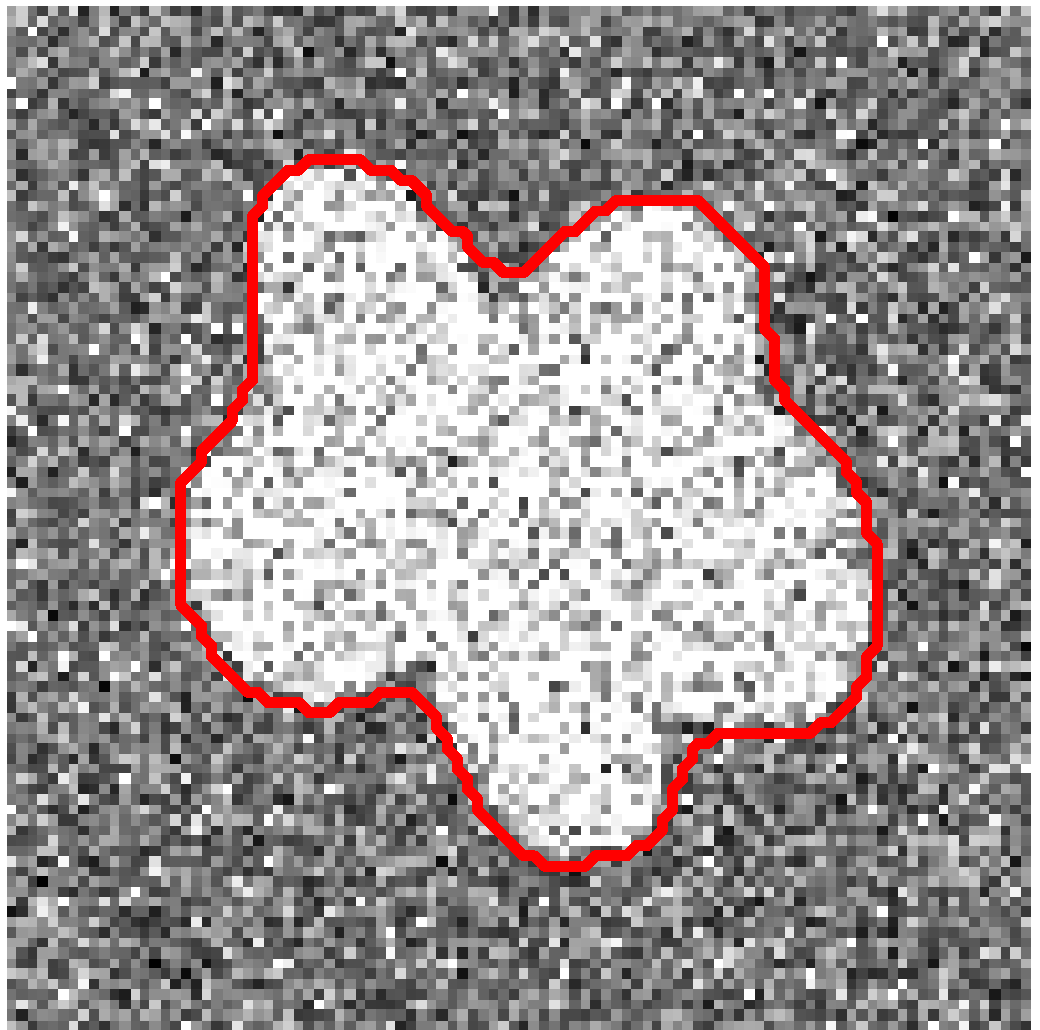}
      \caption{}
  \end{subfigure}

  \caption{Denoising and segmentation with different parameter settings $(\gamma , \nu)$. (a) noisy image; (b) initial contour; (c)–(d) (0.01,1); (e)–(f) (0.1,1); (g)–(h) (0.1,4); (i)–(j) (0.5,7); (k)–(l) (0.5,10). In each pair, the first image shows denoising and the second shows segmentation}
  \label{fig:004-para}
\end{figure}
\begin{table}
\centering
\caption{DSC, IoU, Accuracy and $\kappa$ values of the models under different parameter settings}
\label{tab:004-results}
\begin{tabularx}{\textwidth}{@{}>{\hsize=1.8\hsize}X>{\hsize=0.8\hsize}X>{\hsize=0.8\hsize}X>{\hsize=0.8\hsize}X>{\hsize=0.8\hsize}X@{}}
\toprule
Parameters & DSC & IoU & Accuracy & $\kappa$ \\
\midrule
$\gamma=0.01,\nu=1$ & 0.9765 & 0.9541 & 0.9865 & 0.9671 \\
$\gamma=0.1,\nu=1$ & 0.9779 & 0.9567 & 0.9873 & 0.9690 \\
$\gamma=0.1,\nu=4$ & 0.9665 & 0.9351 & 0.9805 & 0.9528 \\
$\gamma=0.1,\nu=5$ & 0.9574 & 0.9183 & 0.9750  & 0.9398 \\
$\gamma=0.5,\nu=10$ & 0.9381 & 0.8834 & 0.9629 & 0.9118 \\
\bottomrule
\end{tabularx}
\end{table}
We evaluate the robustness of our model with respect to variations in  the denoising parameters $\nu$ and $\gamma$. We perform segmentation experiments on a synthetic image corrupted by multiplicative Gamma noise with $L=10$. Several denoising parameter settings are tested, while the other parameters fixed at $\lambda_i=1$ and $\mu=10^{-9}\times255^2$. The experimental results and the corresponding evaluation metrics are presented in Fig.~\ref{fig:004-para} and Table~\ref{tab:004-results}, respectively. The evolution of the energy over iterations under different denoising parameter settings is depicted in Fig.~\ref{fig:energy-004-para}. It can be seen that all parameter settings yield accurate segmentation results, with only slight variations in the evaluation metrics, thereby demonstrating the robustness of our model to the denoising parameters. As a region-based method, our approach can effectively segment images without clear edges. Consequently, even when the denoised images are over-smoothed and the boundaries become blurred, the model still achieves accurate segmentation.

\subsection{Application on SAR Images}
\begin{figure}
  \centering
  \begin{subfigure}[b]{0.19\linewidth}
      \centering
      \includegraphics[width=\textwidth]{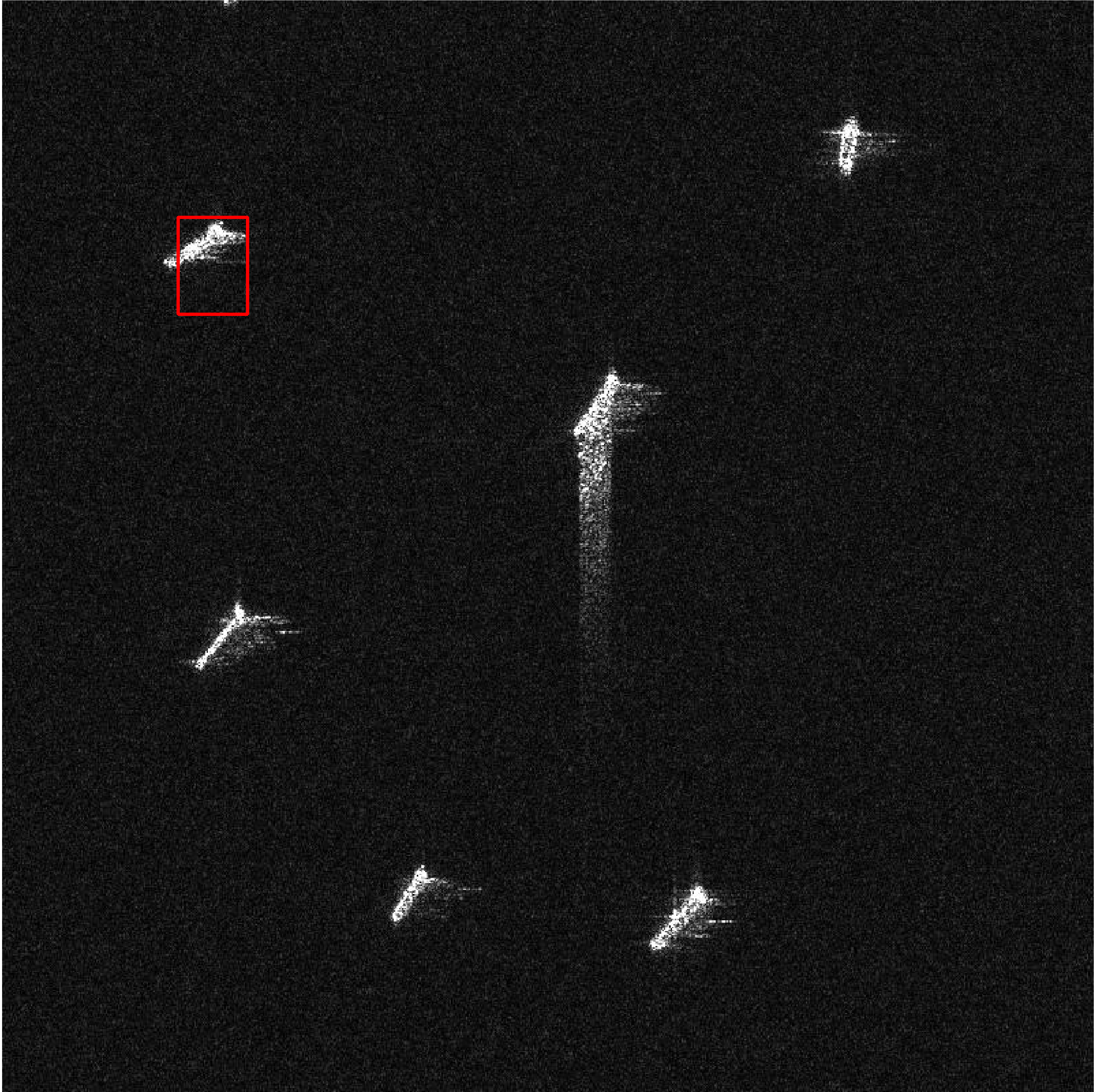}
  \end{subfigure}
  \hfill
  \begin{subfigure}[b]{0.19\linewidth}
      \centering
      \includegraphics[width=\textwidth]{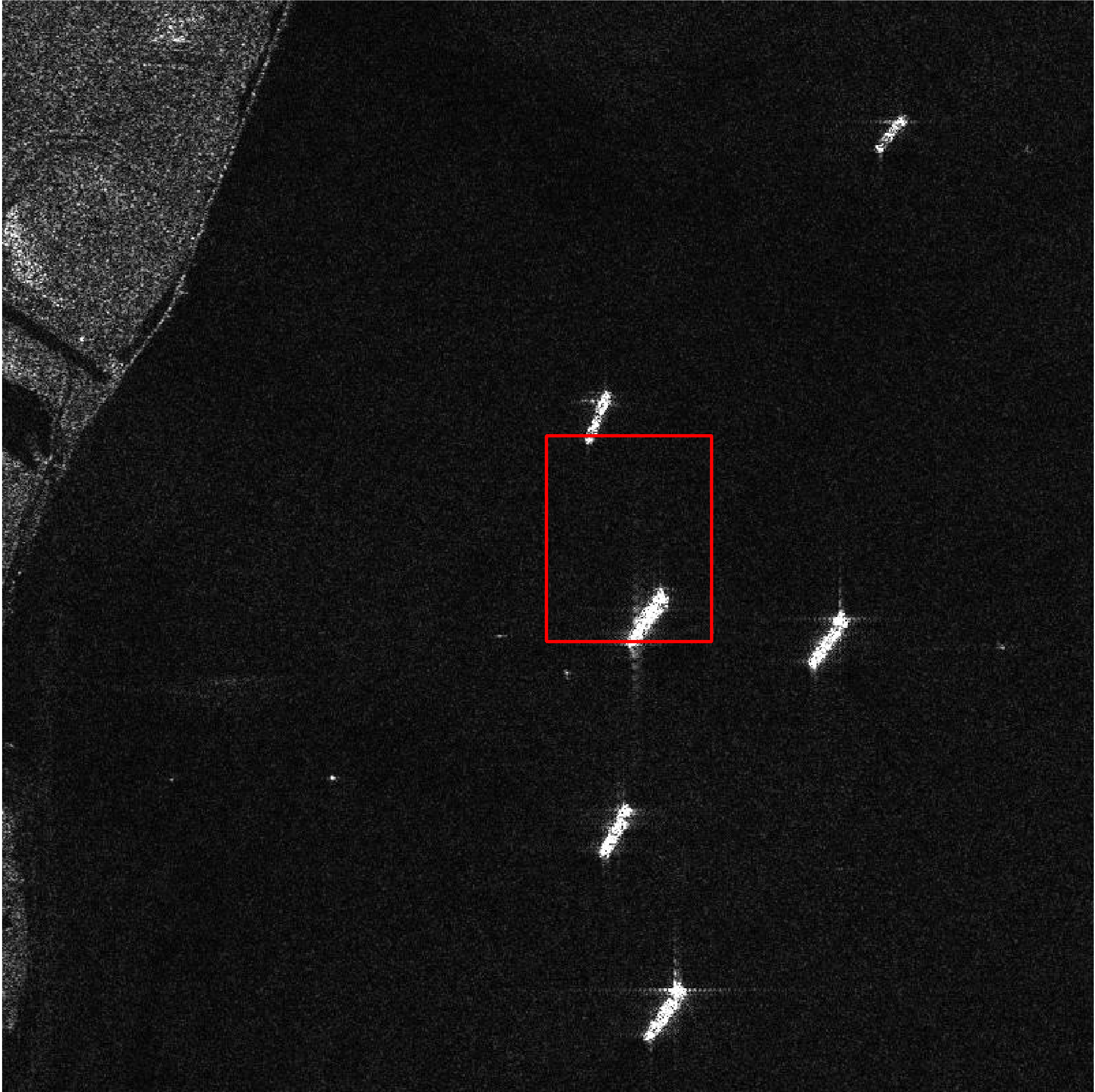}
  \end{subfigure}
 \hfill
 \begin{subfigure}[b]{0.19\linewidth}
      \centering
      \includegraphics[width=\textwidth]{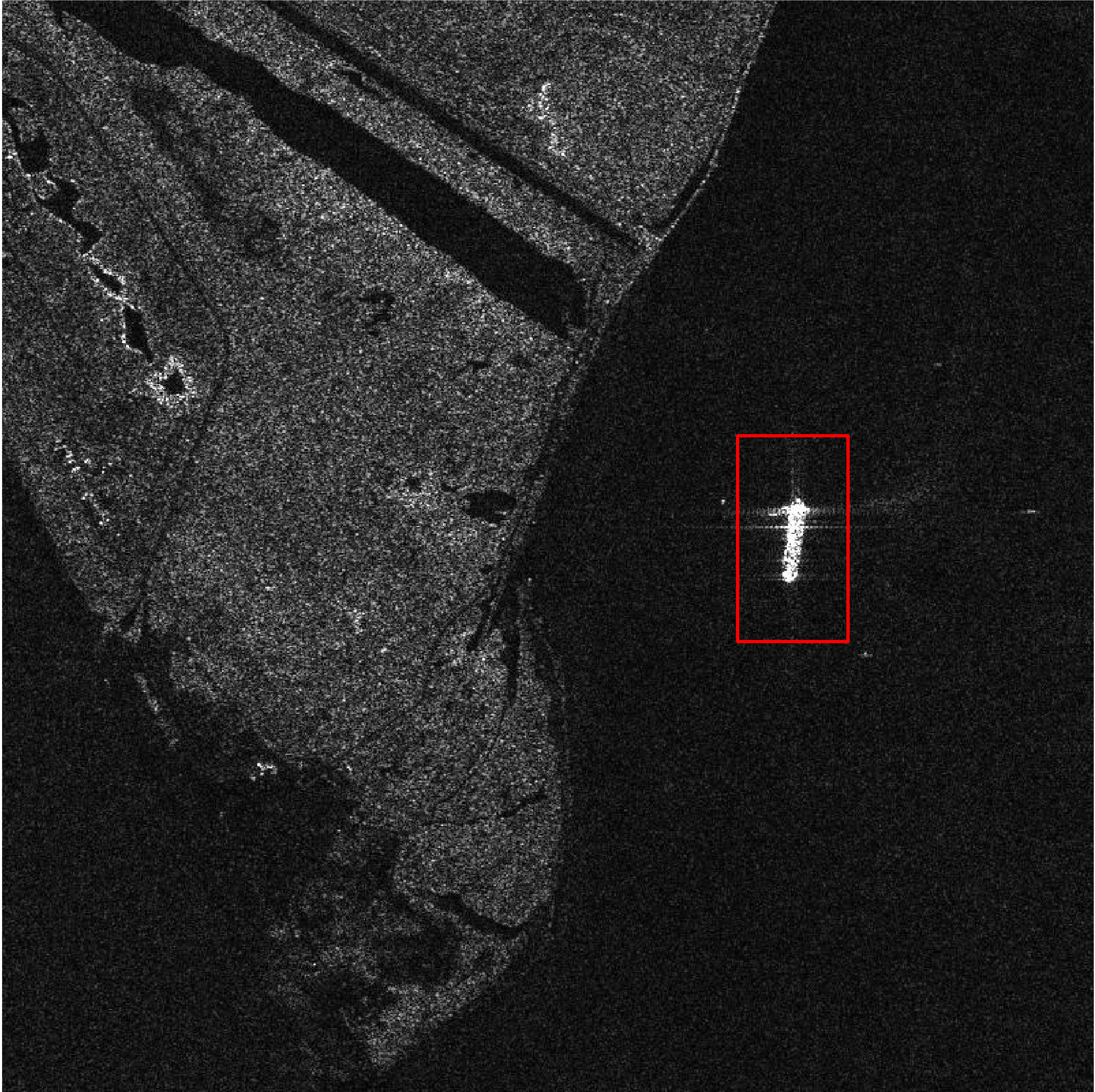}
  \end{subfigure}
  \hfill
   \begin{subfigure}[b]{0.19\linewidth}
      \centering
      \includegraphics[width=\textwidth]{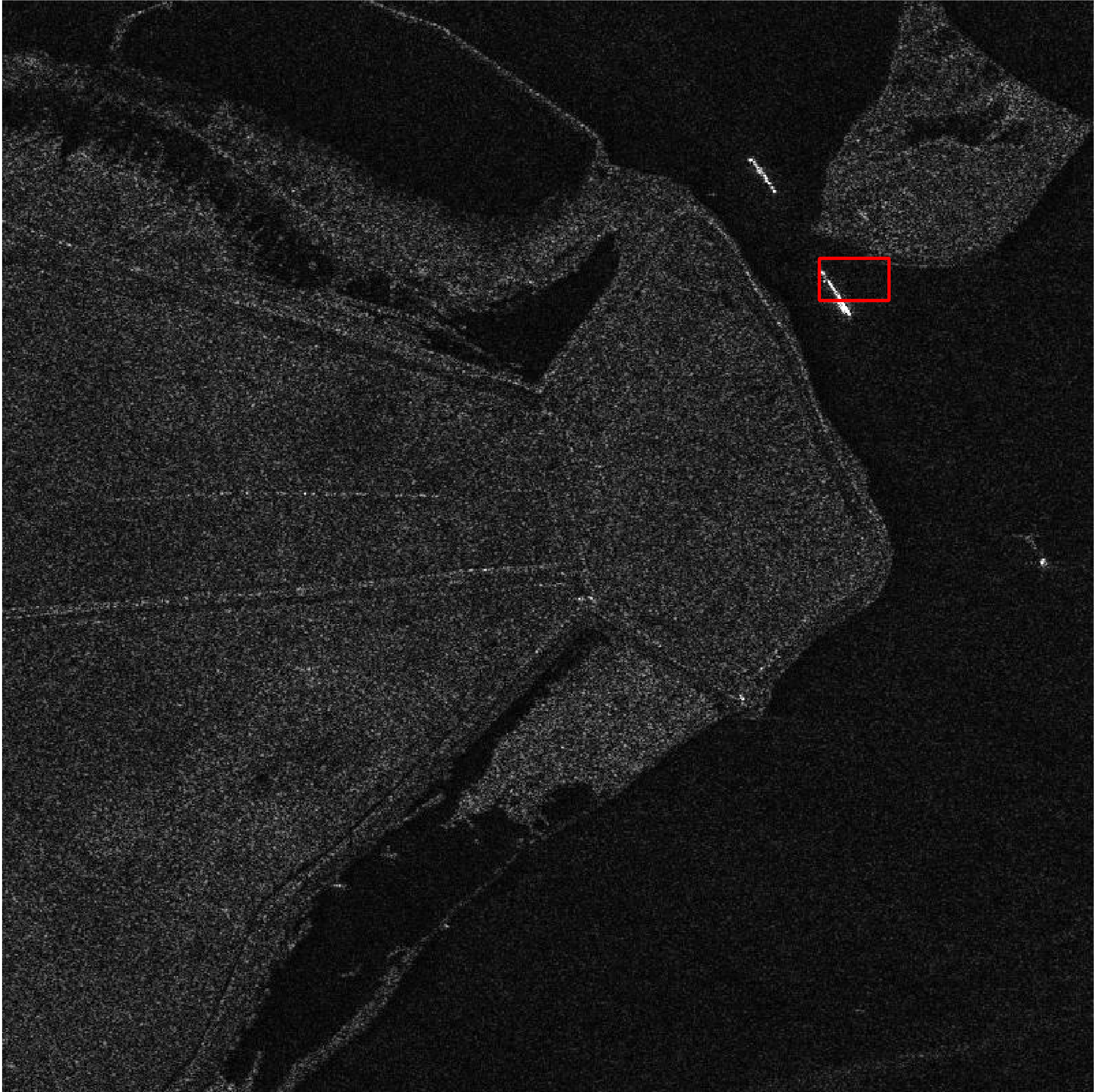}
  \end{subfigure}
  \hfill
  \begin{subfigure}[b]{0.19\linewidth}
      \centering
      \includegraphics[width=\textwidth]{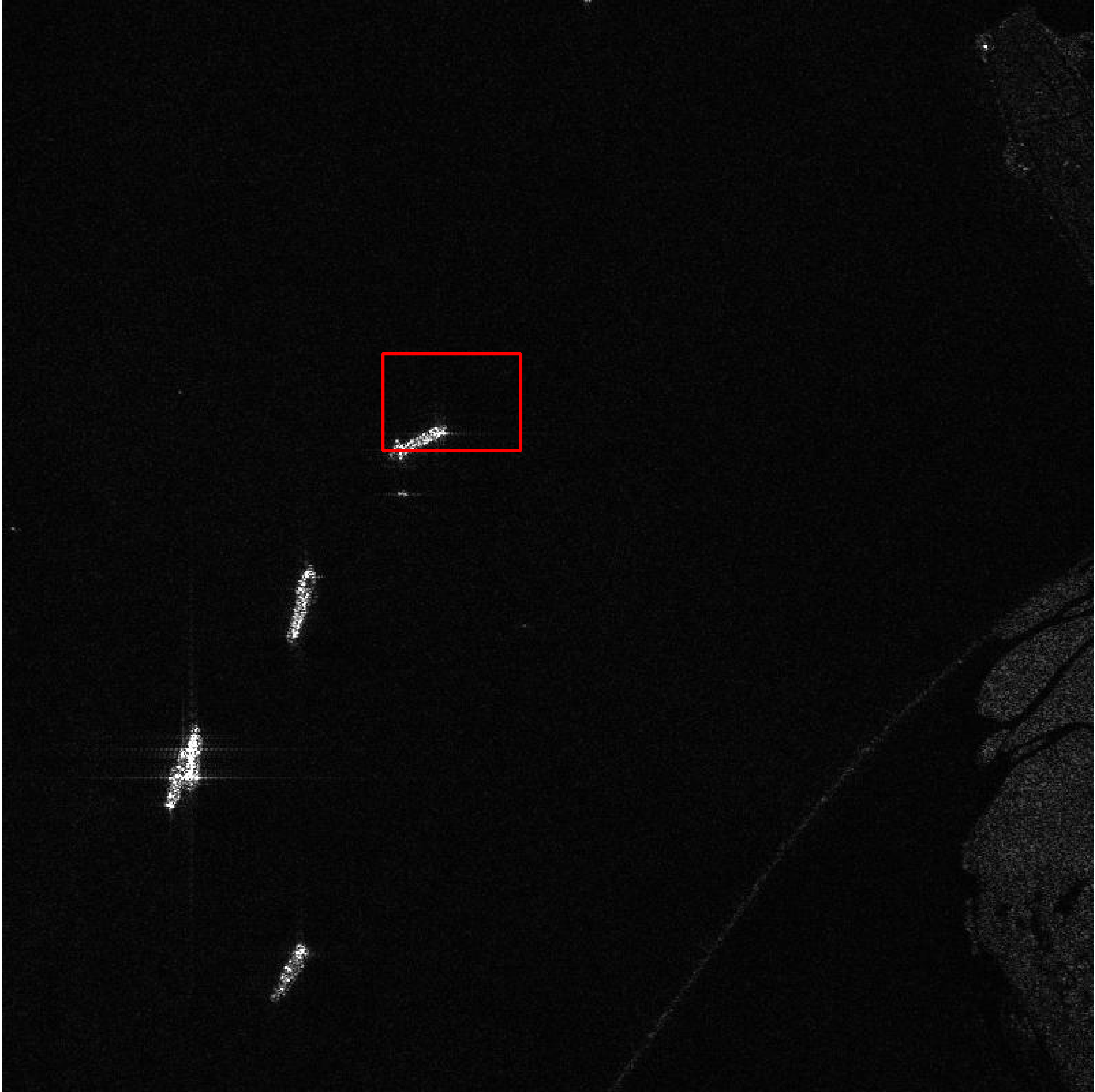}
  \end{subfigure}

  \begin{subfigure}[b]{0.19\linewidth}
      \centering
      \includegraphics[width=\textwidth]{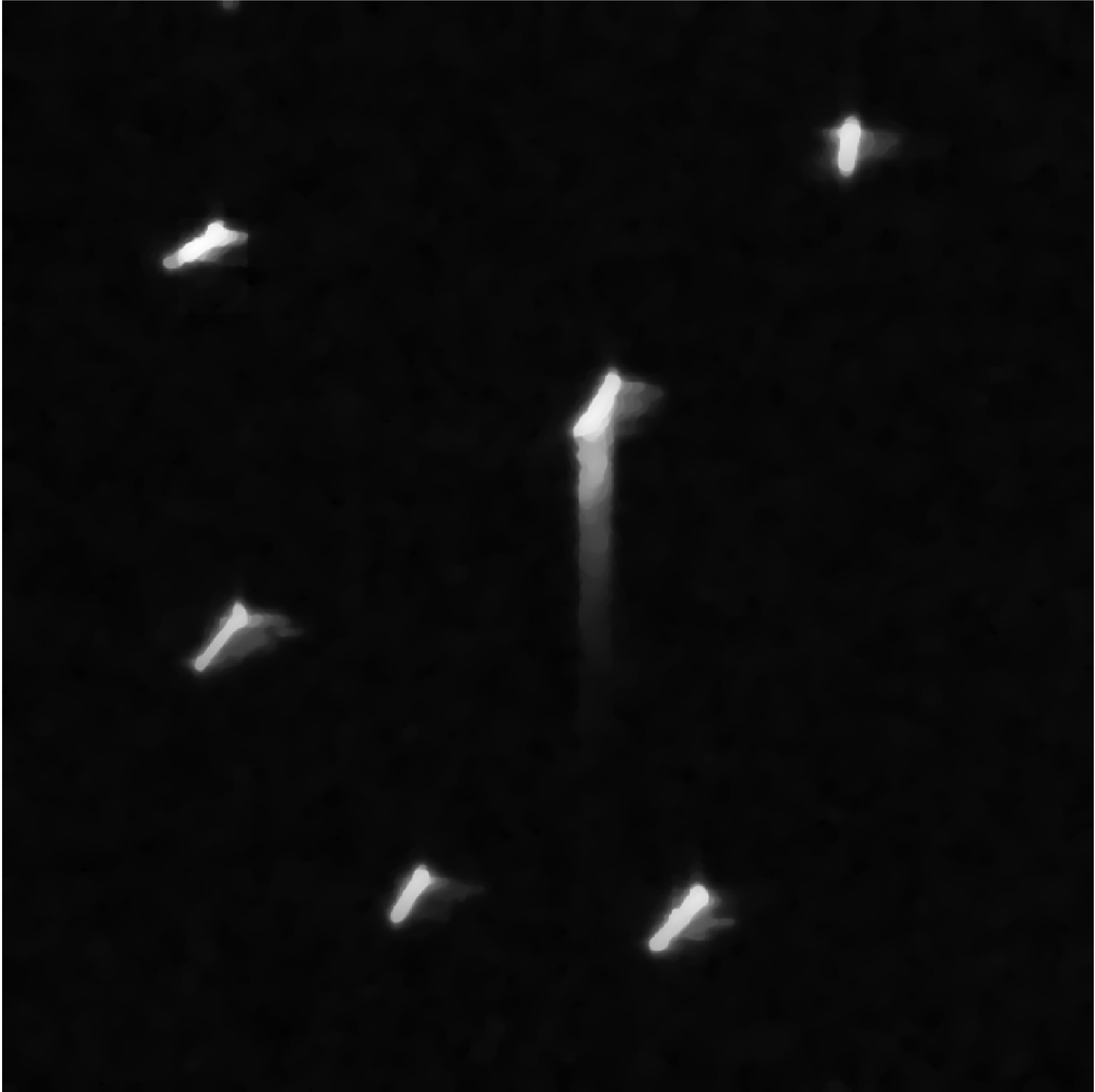}
  \end{subfigure}
  \hfill
  \begin{subfigure}[b]{0.19\linewidth}
      \centering
      \includegraphics[width=\textwidth]{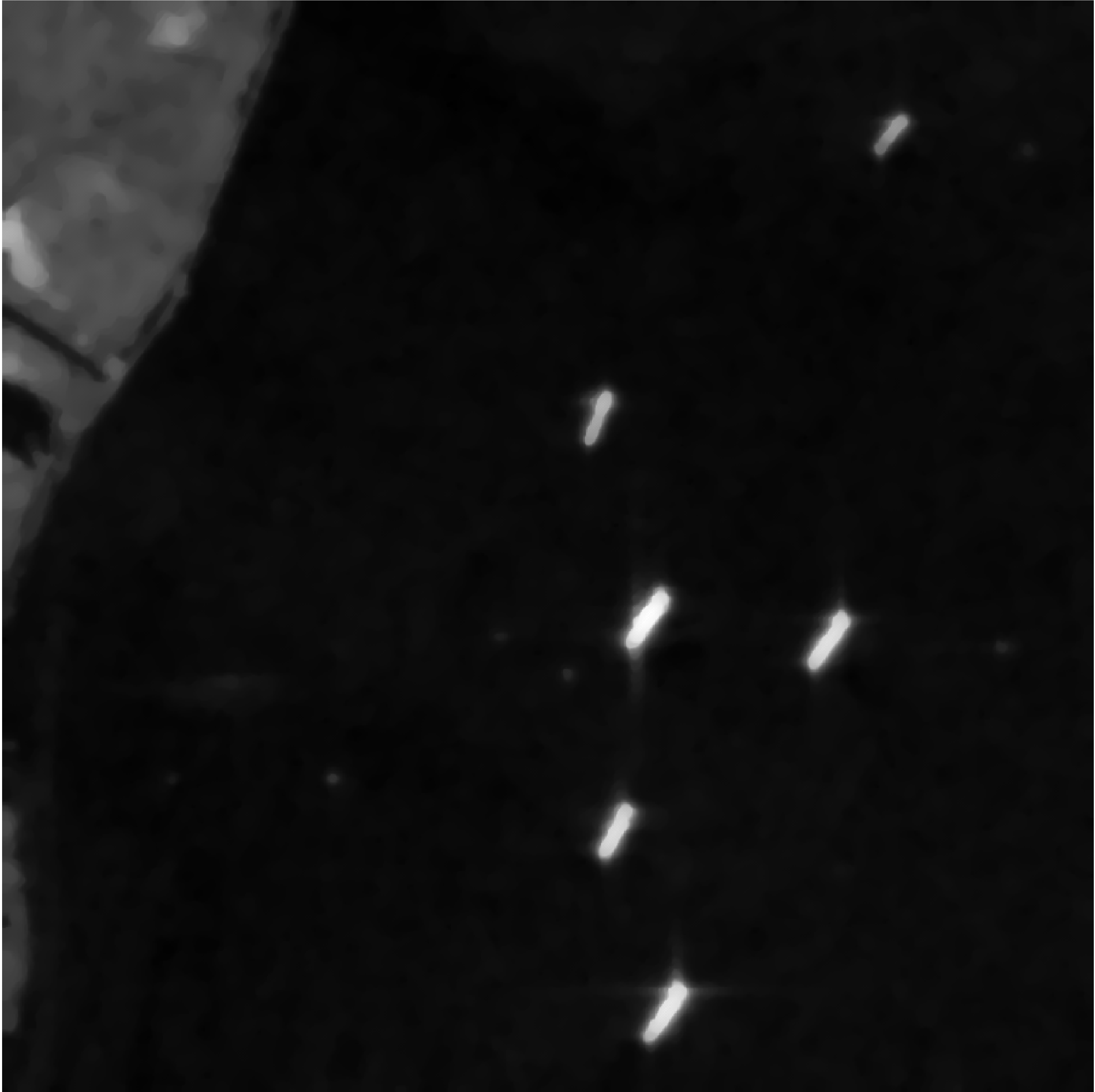}
  \end{subfigure}
 \hfill
 \begin{subfigure}[b]{0.19\linewidth}
      \centering
      \includegraphics[width=\textwidth]{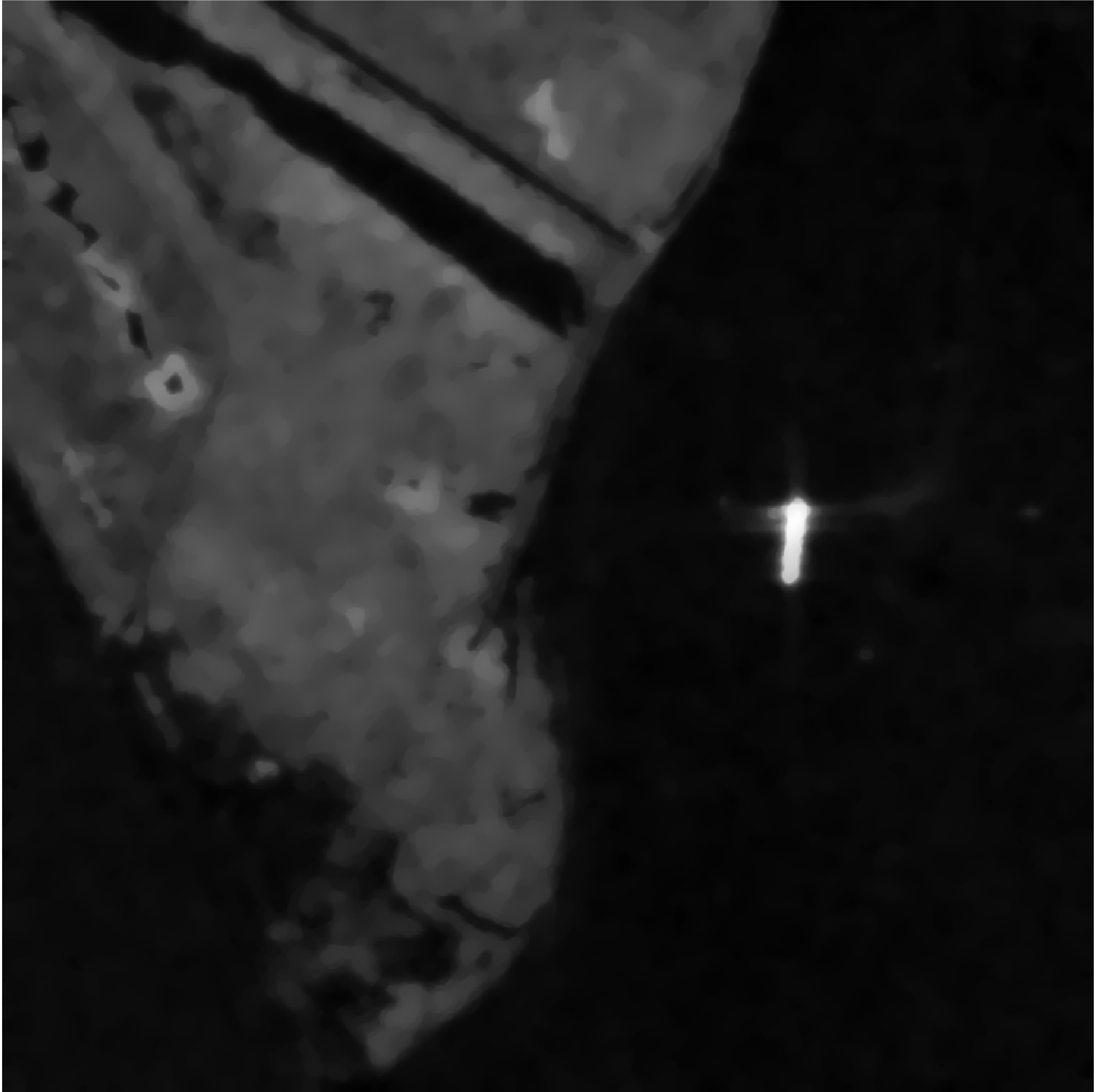}
  \end{subfigure}
 \hfill
   \begin{subfigure}[b]{0.19\linewidth}
      \centering
      \includegraphics[width=\textwidth]{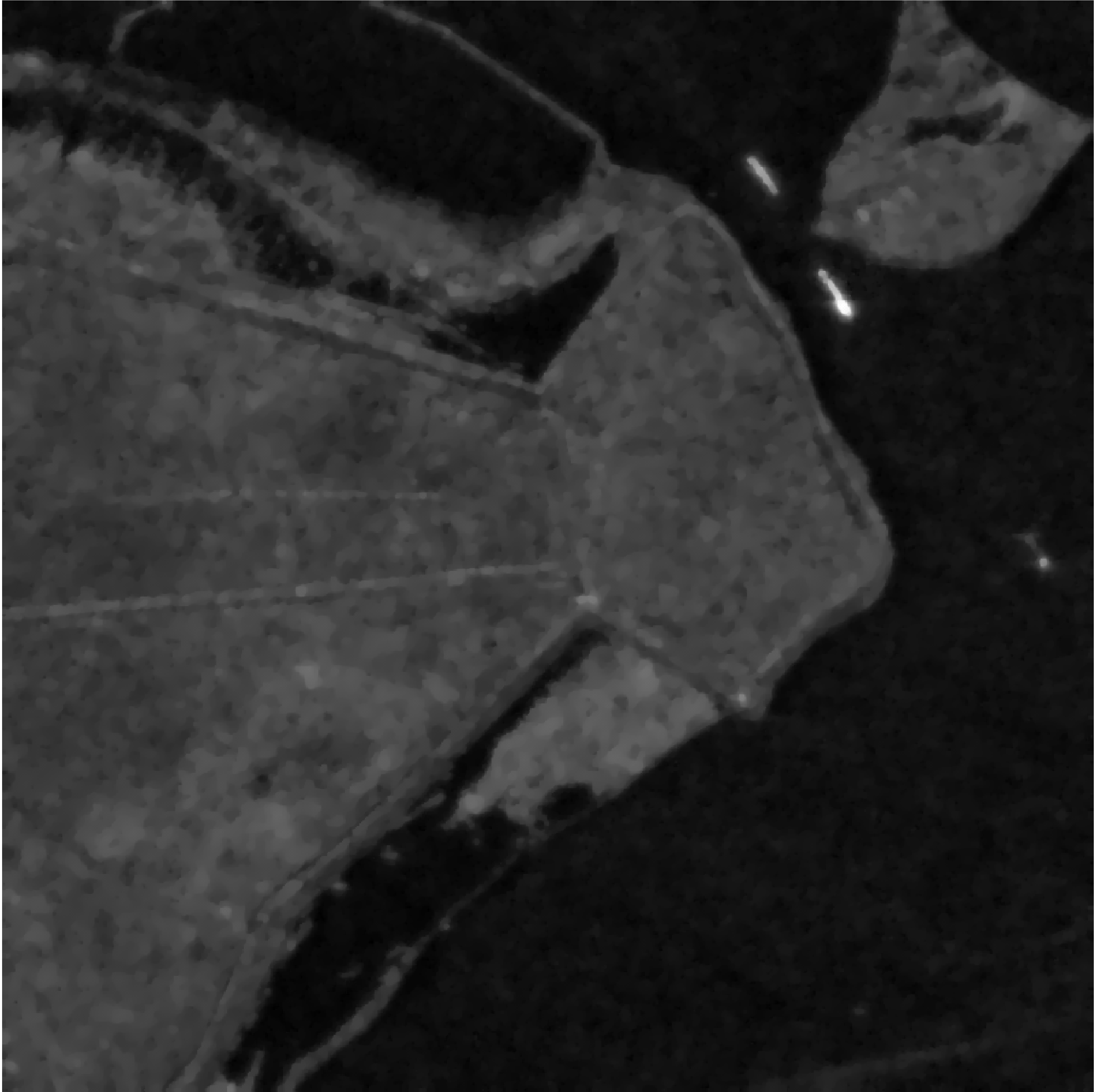}
  \end{subfigure}
   \hfill
  \begin{subfigure}[b]{0.19\linewidth}
      \centering
      \includegraphics[width=\textwidth]{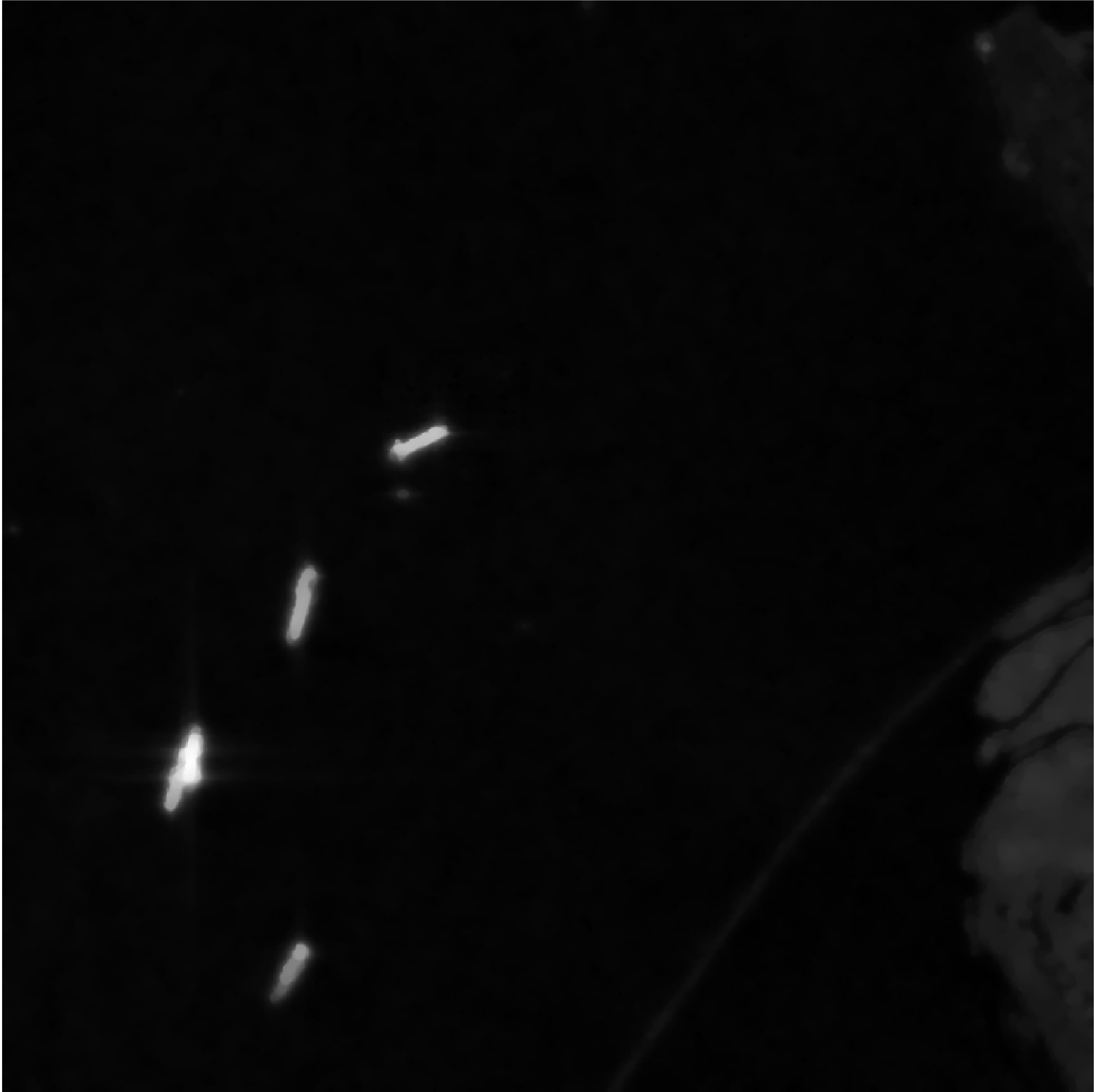}
  \end{subfigure}

    \begin{subfigure}[b]{0.19\linewidth}
      \centering
      \includegraphics[width=\textwidth]{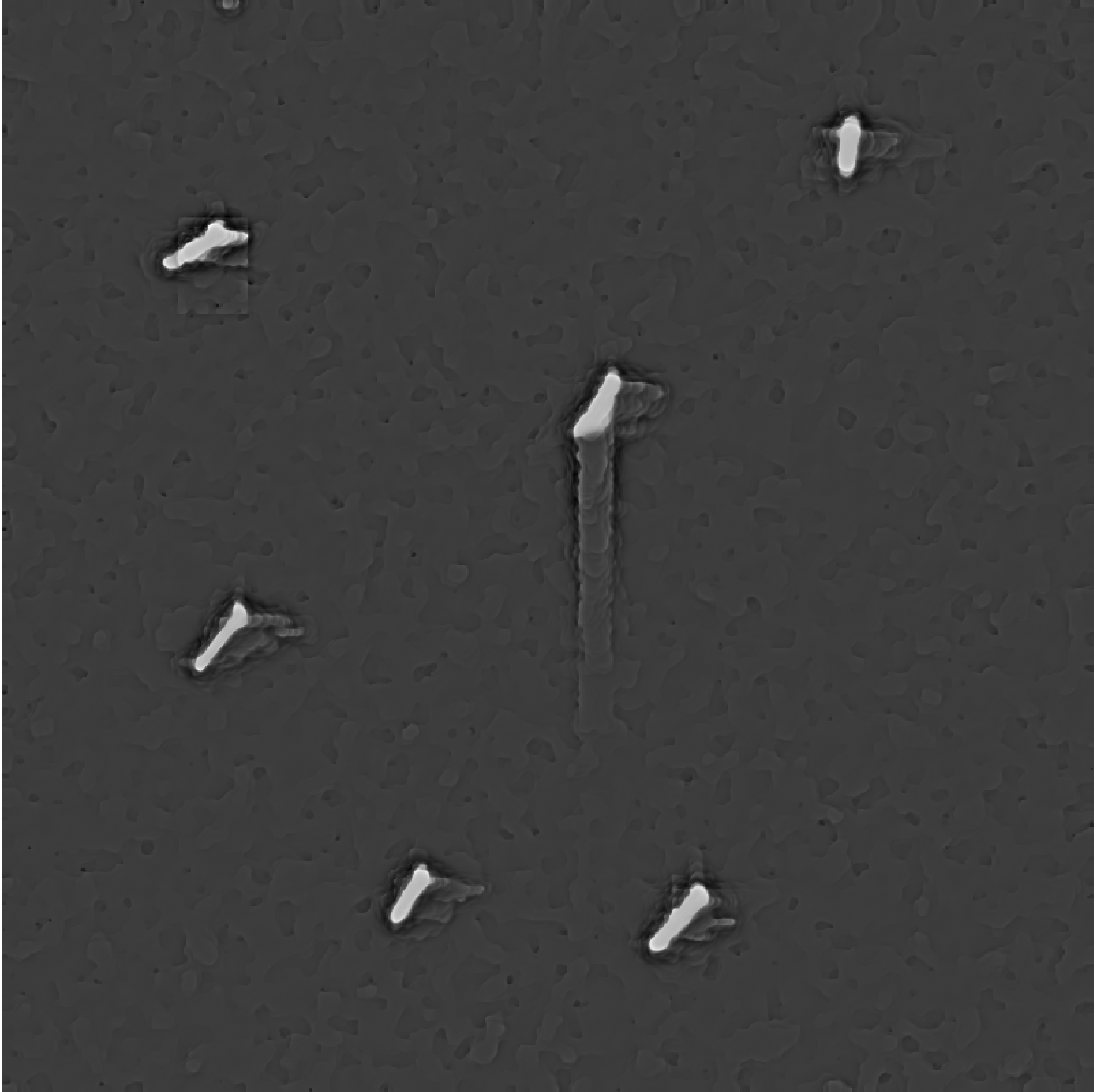}
  \end{subfigure}
  \hfill
    \begin{subfigure}[b]{0.19\linewidth}
      \centering
      \includegraphics[width=\textwidth]{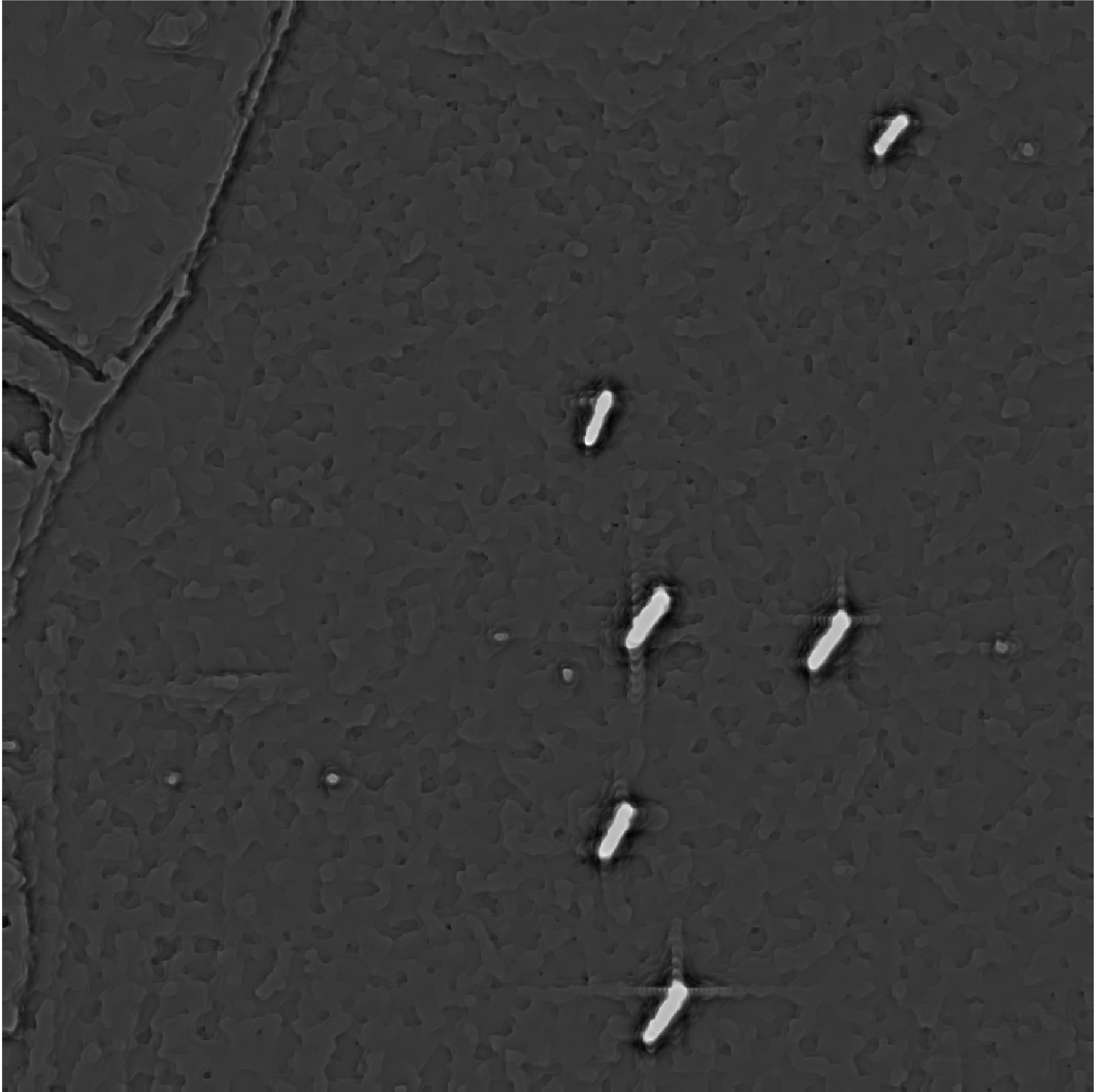}
  \end{subfigure}
  \hfill
  \begin{subfigure}[b]{0.19\linewidth}
      \centering
      \includegraphics[width=\textwidth]{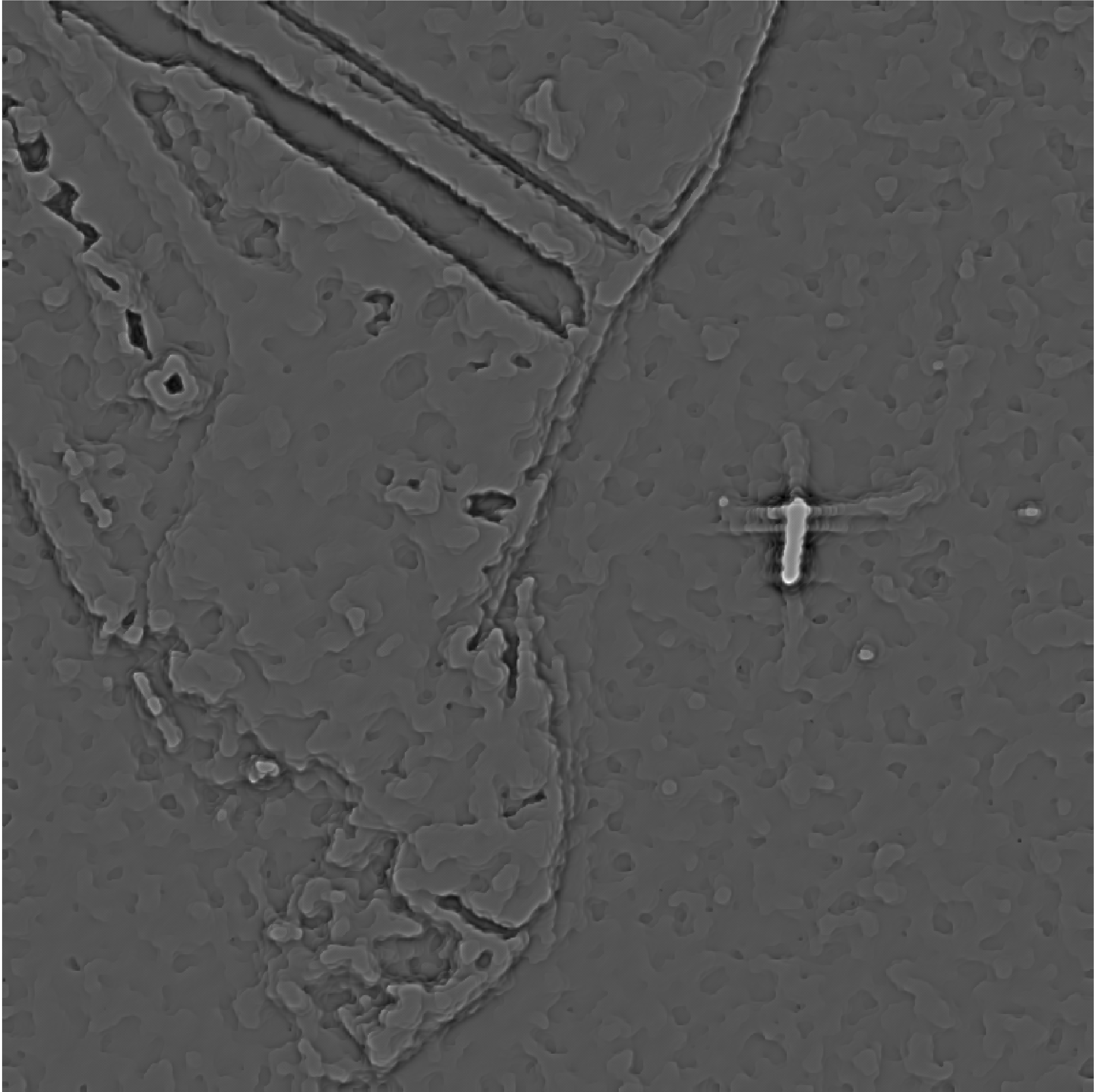}
  \end{subfigure}
  \hfill
   \begin{subfigure}[b]{0.19\linewidth}
      \centering
      \includegraphics[width=\textwidth]{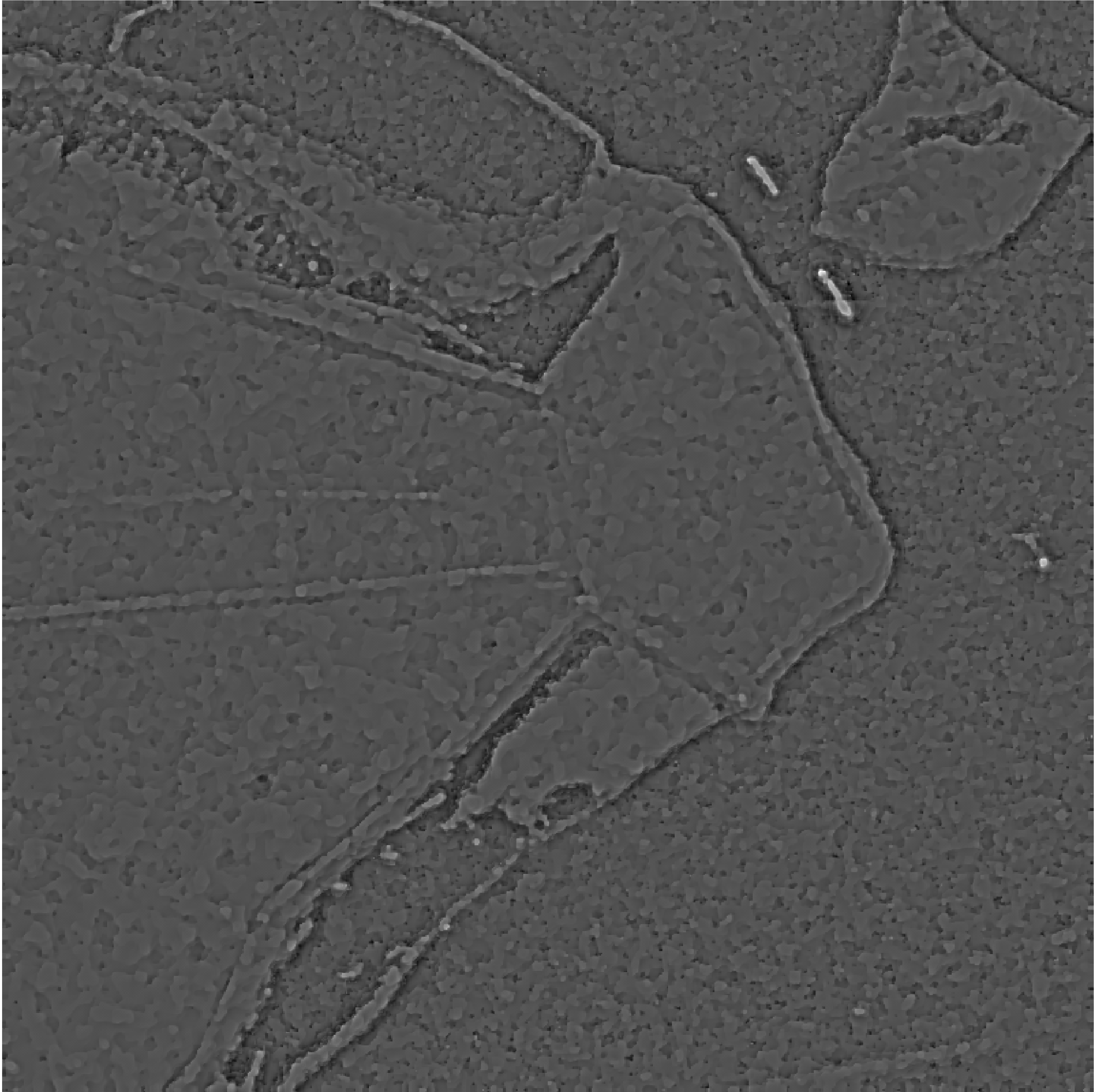}
  \end{subfigure}
  \hfill
  \begin{subfigure}[b]{0.19\linewidth}
      \centering
      \includegraphics[width=\textwidth]{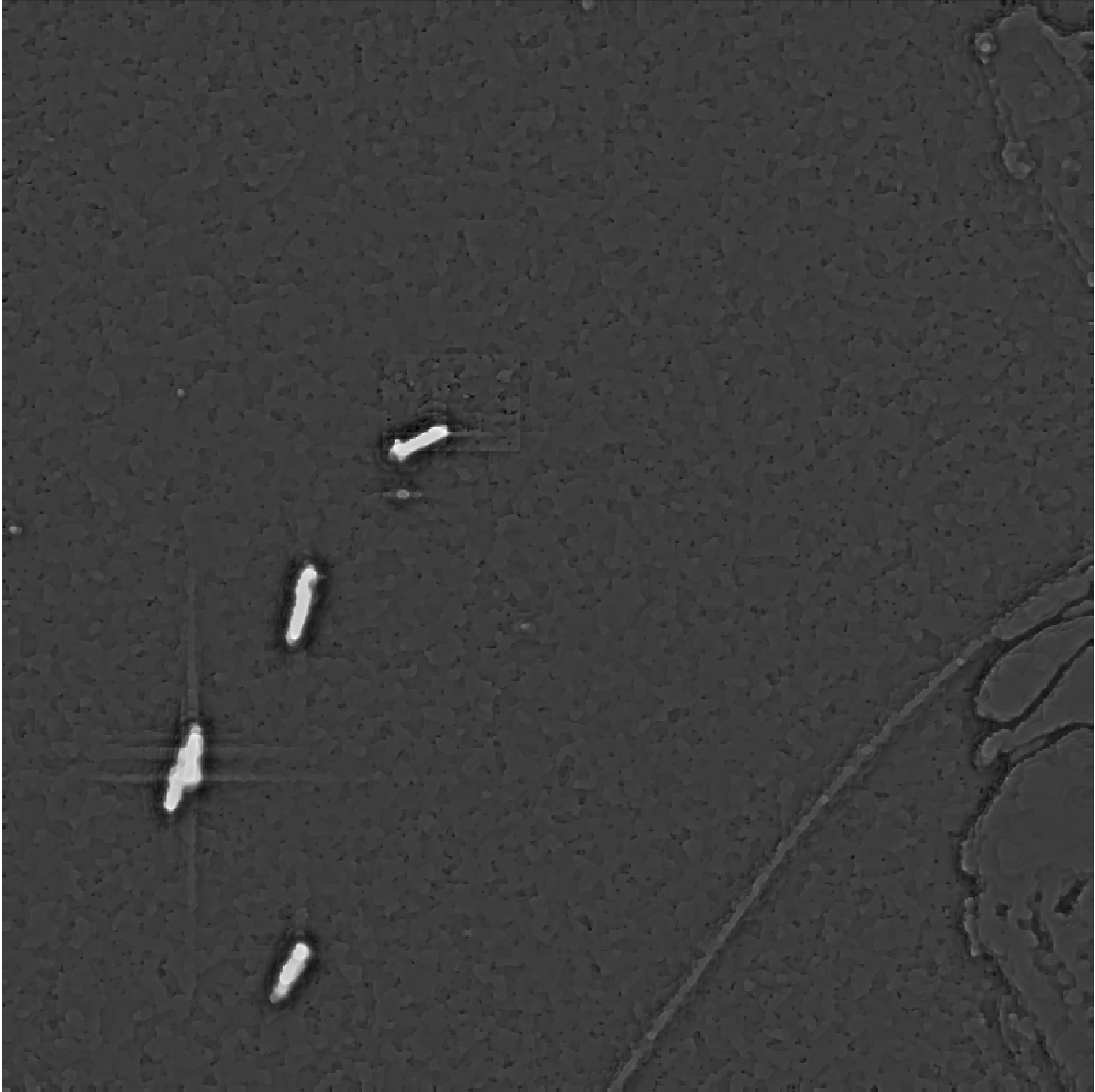}
  \end{subfigure}

  \begin{subfigure}[b]{0.19\linewidth}
      \centering
      \includegraphics[width=\textwidth]{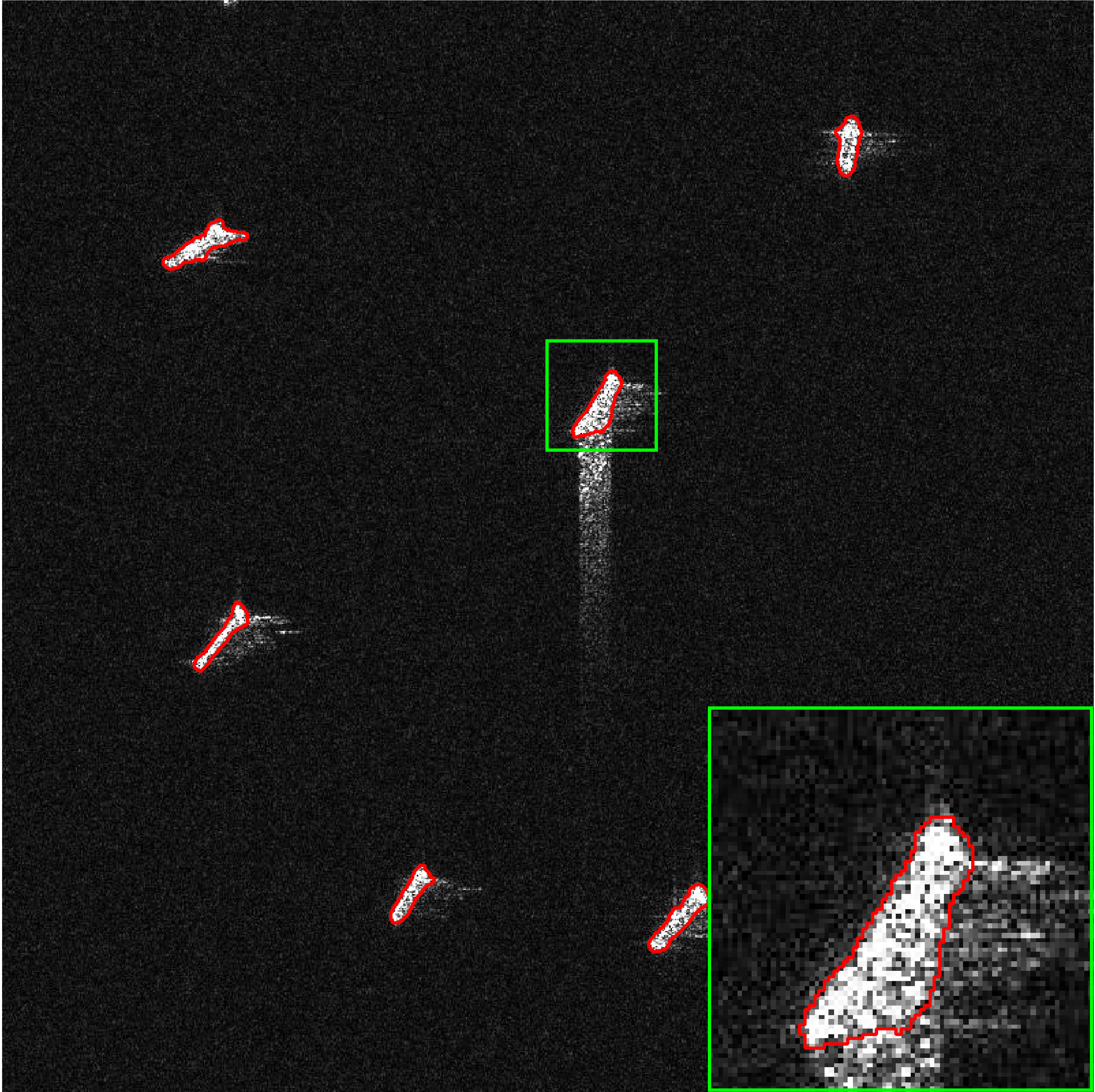}
  \end{subfigure}
  \hfill
    \begin{subfigure}[b]{0.19\linewidth}
      \centering
      \includegraphics[width=\textwidth]{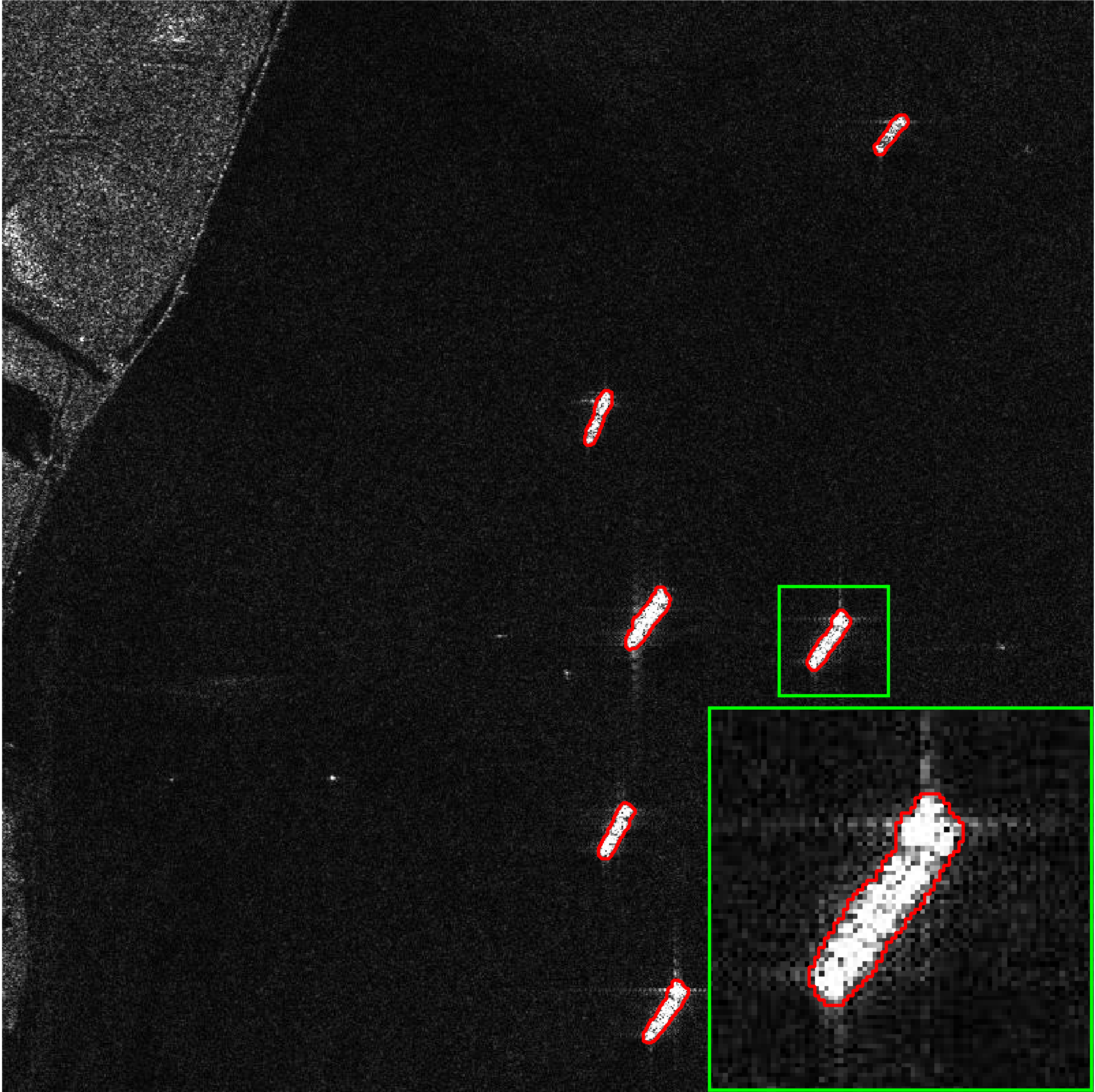}
  \end{subfigure}
 \hfill
  \begin{subfigure}[b]{0.19\linewidth}
      \centering
      \includegraphics[width=\textwidth]{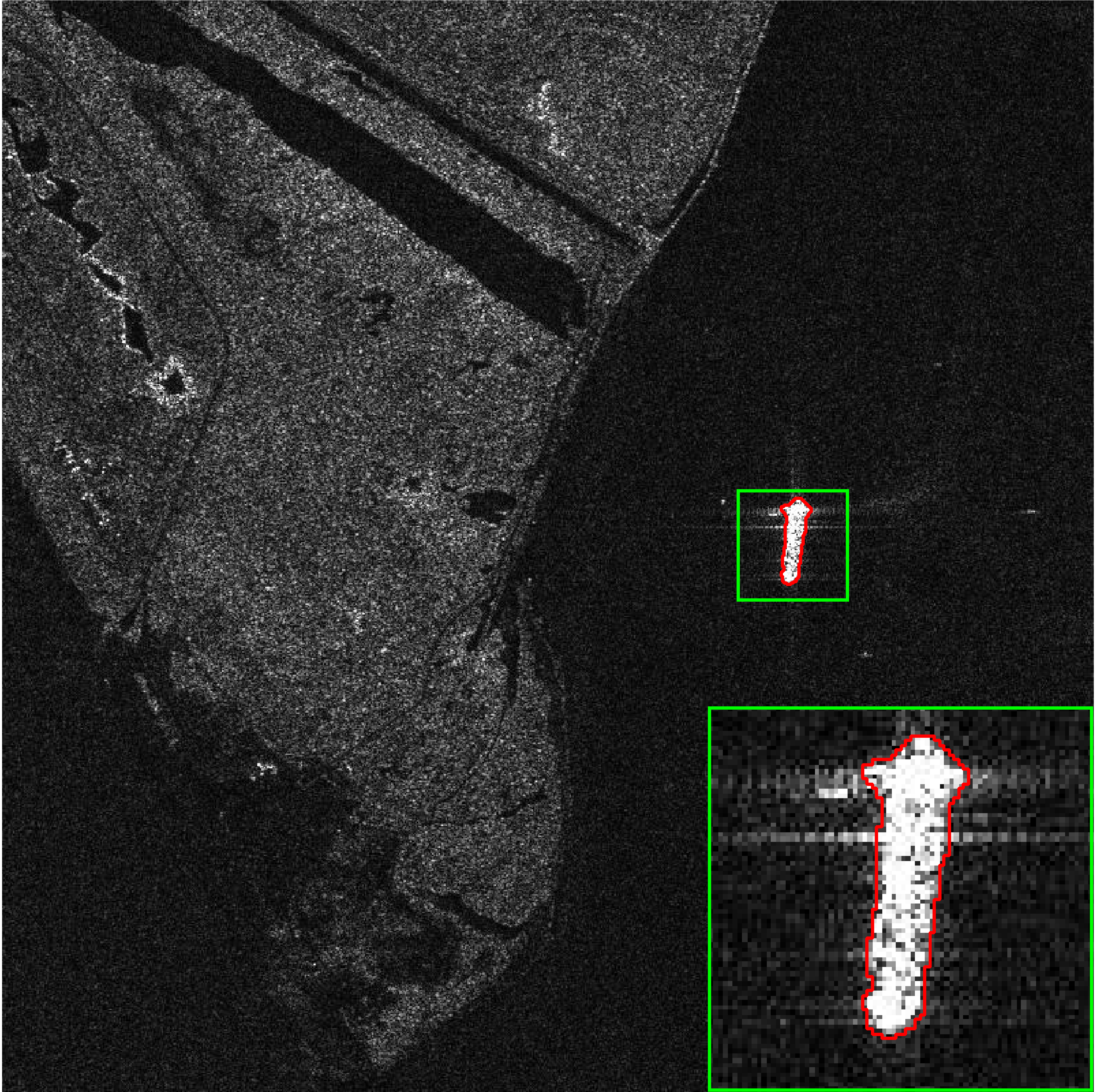}
  \end{subfigure}
  \hfill
  \begin{subfigure}[b]{0.19\linewidth}
      \centering
      \includegraphics[width=\textwidth]{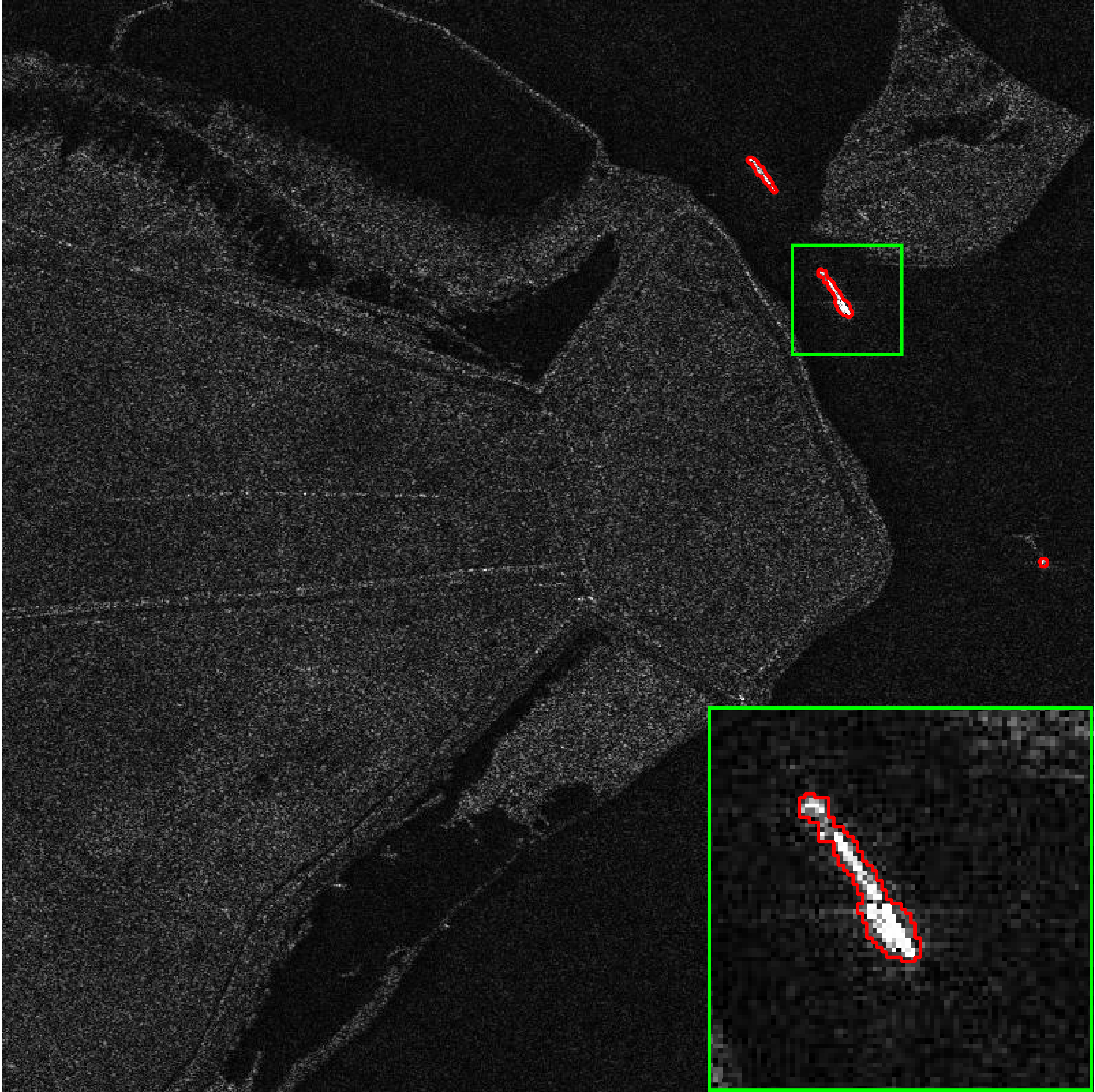}
  \end{subfigure}
  \hfill
  \begin{subfigure}[b]{0.19\linewidth}
      \centering
      \includegraphics[width=\textwidth]{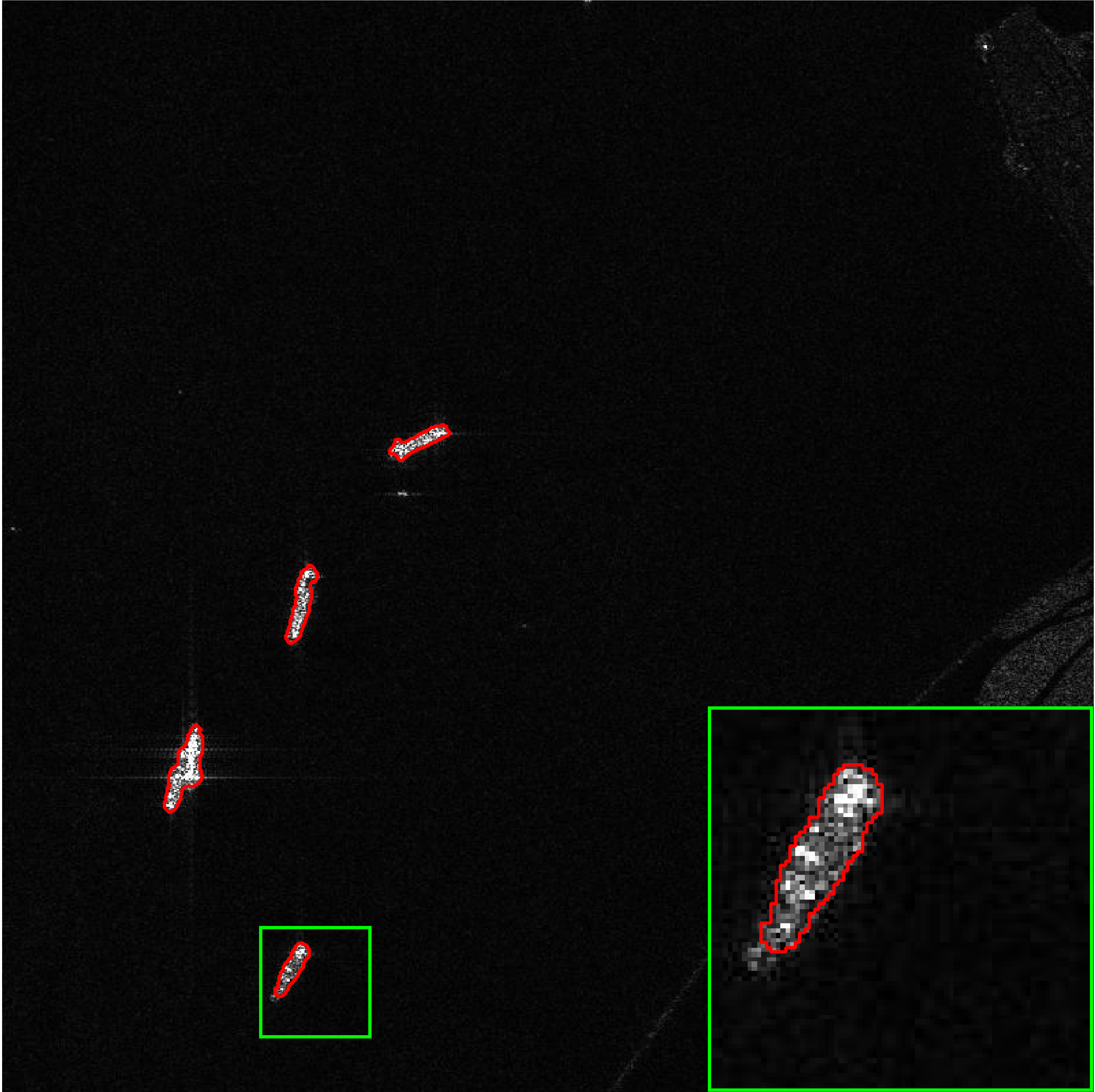}
  \end{subfigure}
  \caption{SAR image segmentation. Row 1: the input images with initial contours; Row 2: the denoised images; Row 3: the bias-field corrected images; Row 4: the segmentation results}
  \label{fig:ship}
\end{figure}
Segmentation performance on real-world images is crucial for practical applications. In SAR images, multiplicative Gamma noise is inevitably introduced by the imaging mechanism. To evaluate the segmentation performance of the proposed model, we conduct ship segmentation experiments on the High-Resolution SAR Images Dataset (HRSID) \cite{wei2020hrsid}. As shown in the zoomed-in regions of Fig.~\ref{fig:ship}, the target boundaries become indistinct due to noise. In addition, imaging artifacts often appear around objects as a result of the imaging process. The segmentation results obtained on the HRSID are presented in Fig.~\ref{fig:ship}. The proposed model effectively removes noise from images and corrects for intensity inhomogeneity. After denoising and bias correction, the object boundaries become clearer, thereby enabling more accurate segmentation results. These results confirm the effectiveness of the proposed model for real SAR images.

\subsection{Comparison with Some State-of-the-Art Algorithms}
\begin{figure}
  \centering
  \begin{subfigure}[b]{0.16\linewidth}
      \centering
      \includegraphics[width=\textwidth]{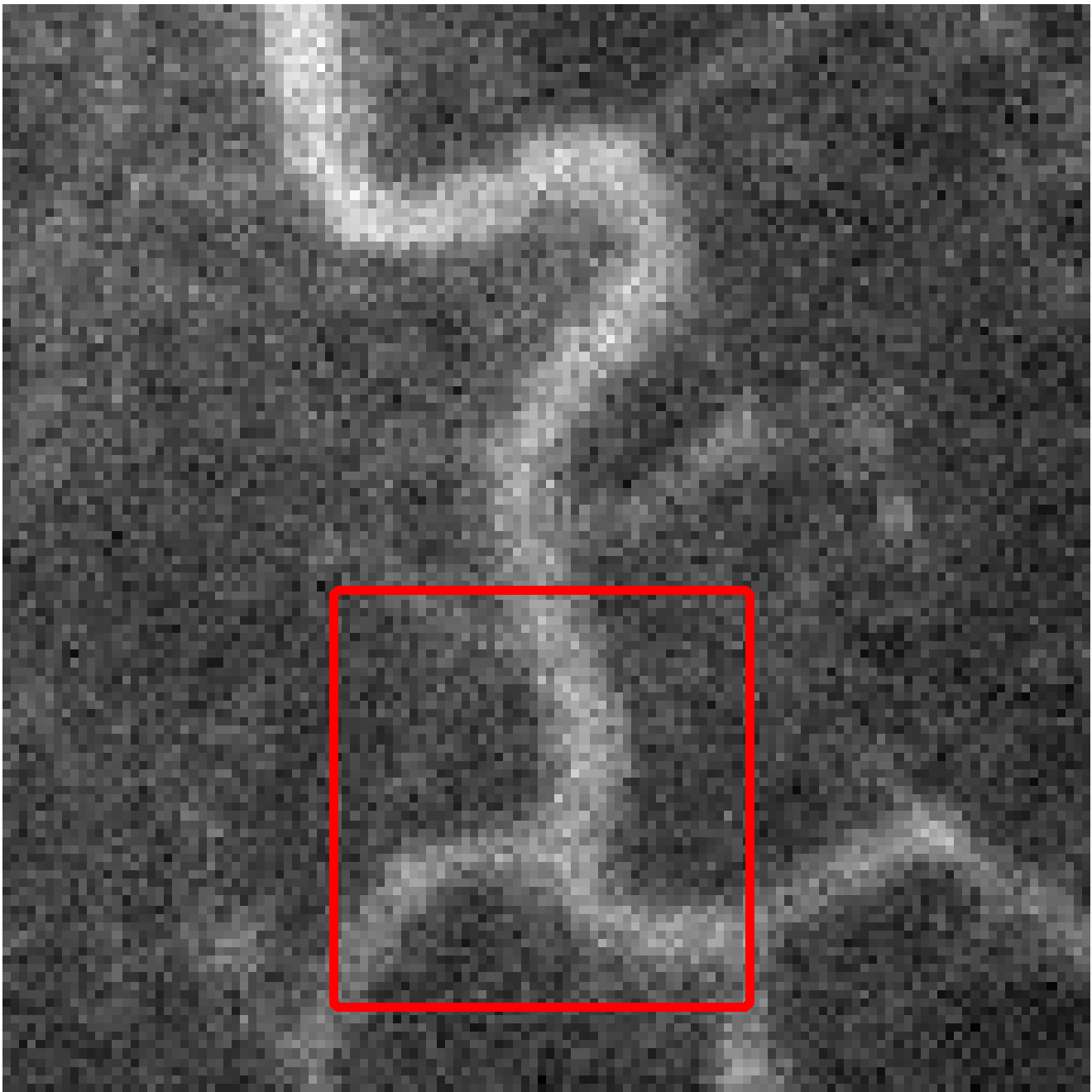}
  \end{subfigure}
  \hfill
  \begin{subfigure}[b]{0.16\linewidth}
      \centering
      \includegraphics[width=\textwidth]{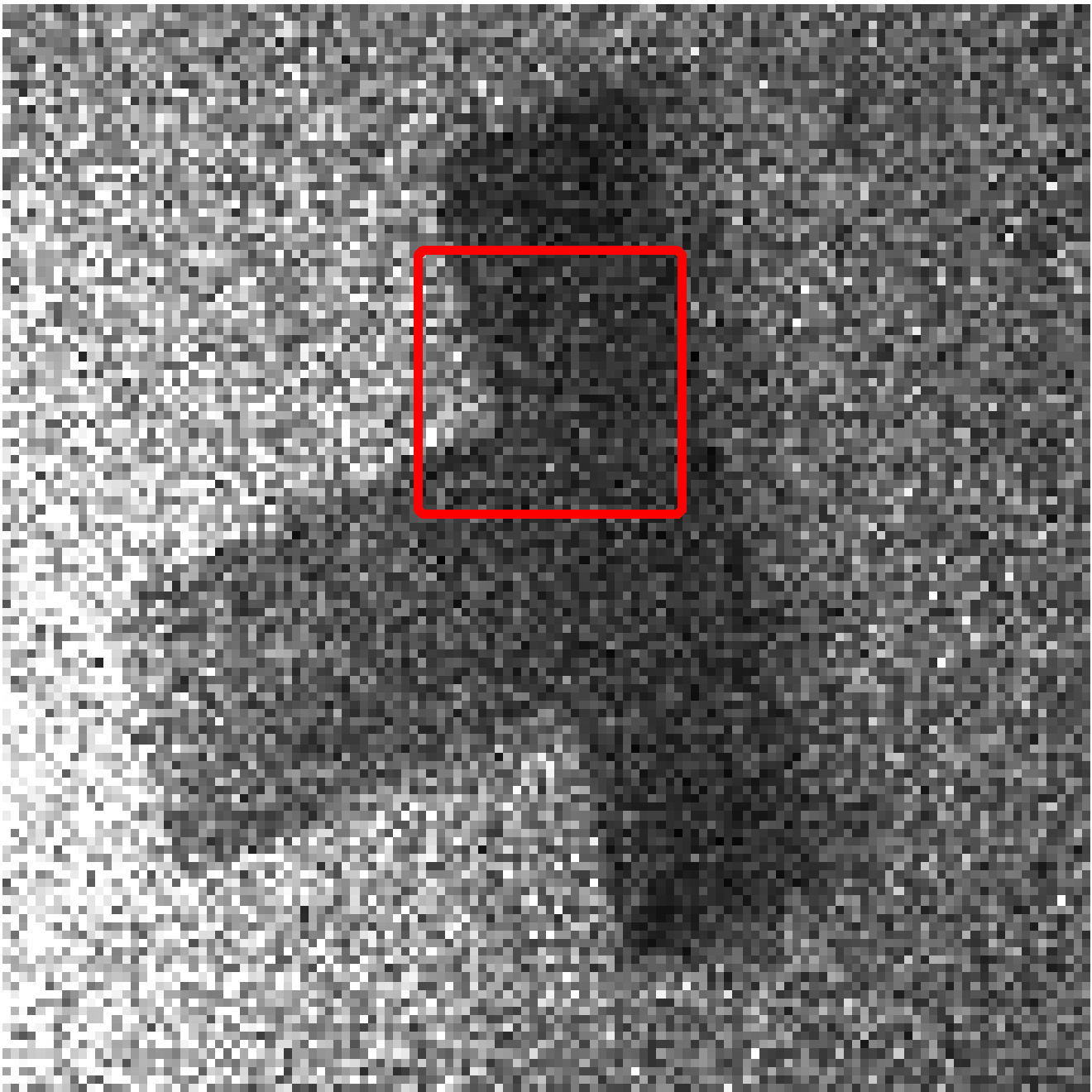}
  \end{subfigure}
  \hfill
  \begin{subfigure}[b]{0.16\linewidth}
      \centering
      \includegraphics[width=\textwidth]{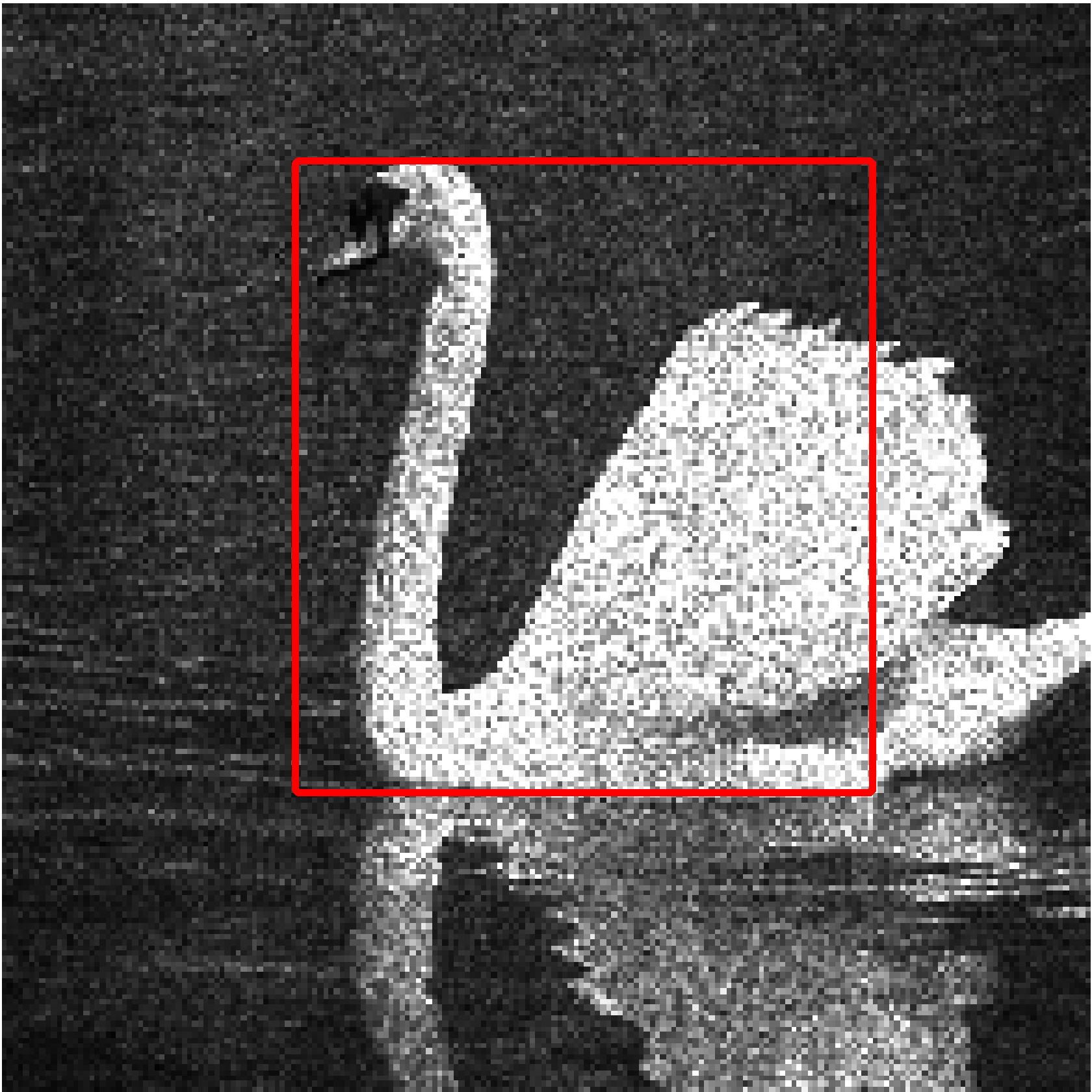}
  \end{subfigure}
  \hfill
  \begin{subfigure}[b]{0.16\linewidth}
      \centering
      \includegraphics[width=\textwidth]{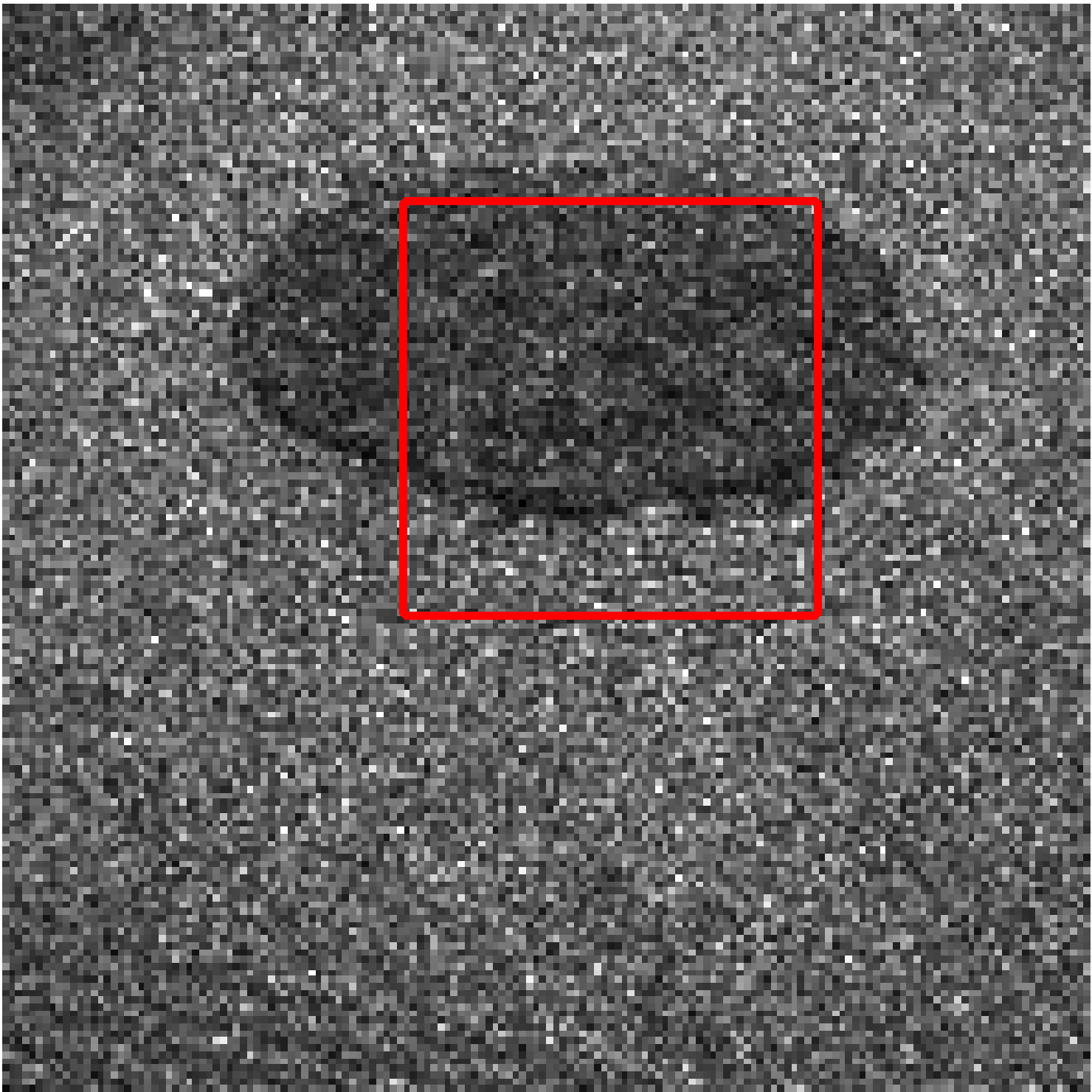}
  \end{subfigure}
  \hfill
  \begin{subfigure}[b]{0.16\linewidth}
      \centering
      \includegraphics[width=\textwidth]{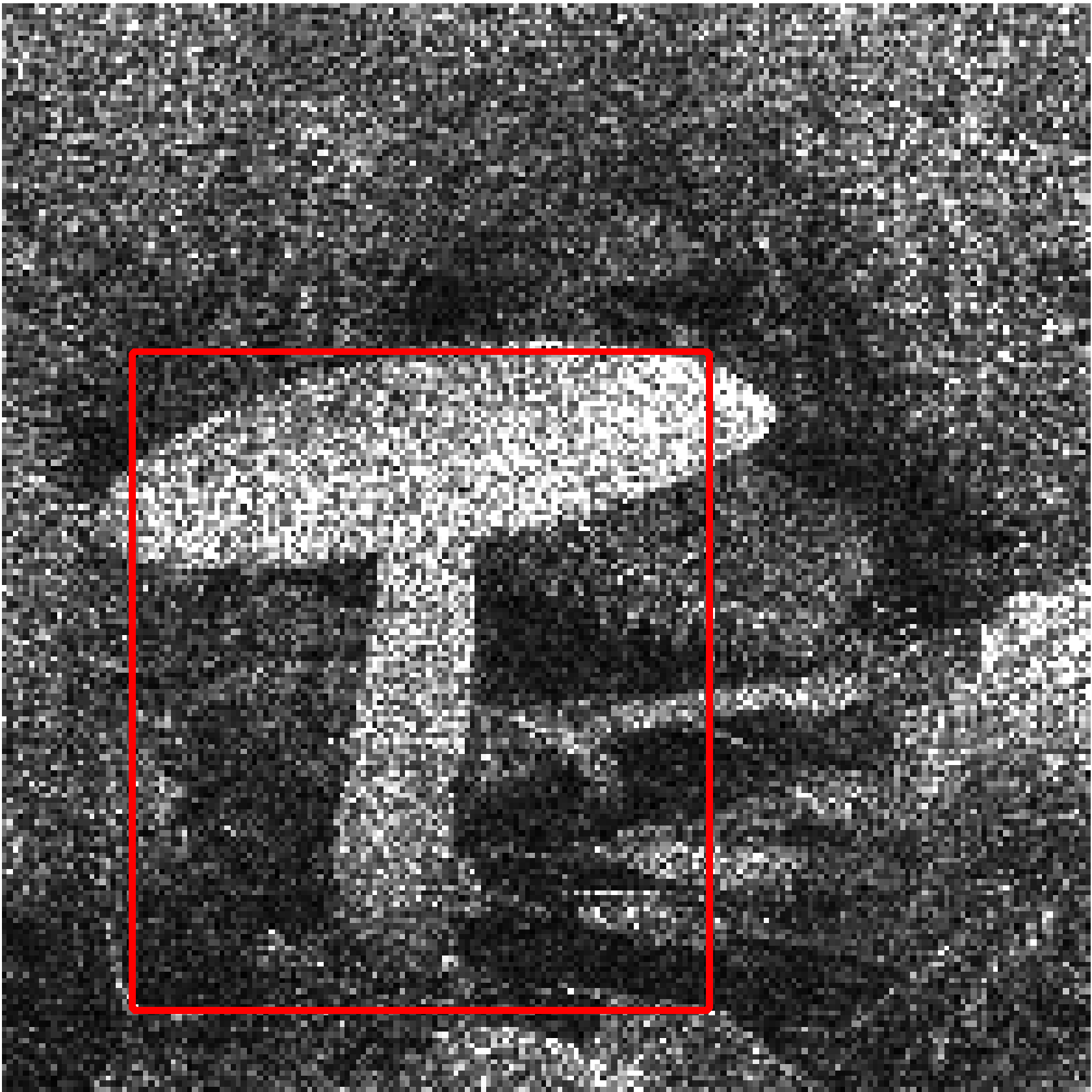}
  \end{subfigure}
  \hfill
  \begin{subfigure}[b]{0.16\linewidth}
      \centering
      \includegraphics[width=\textwidth]{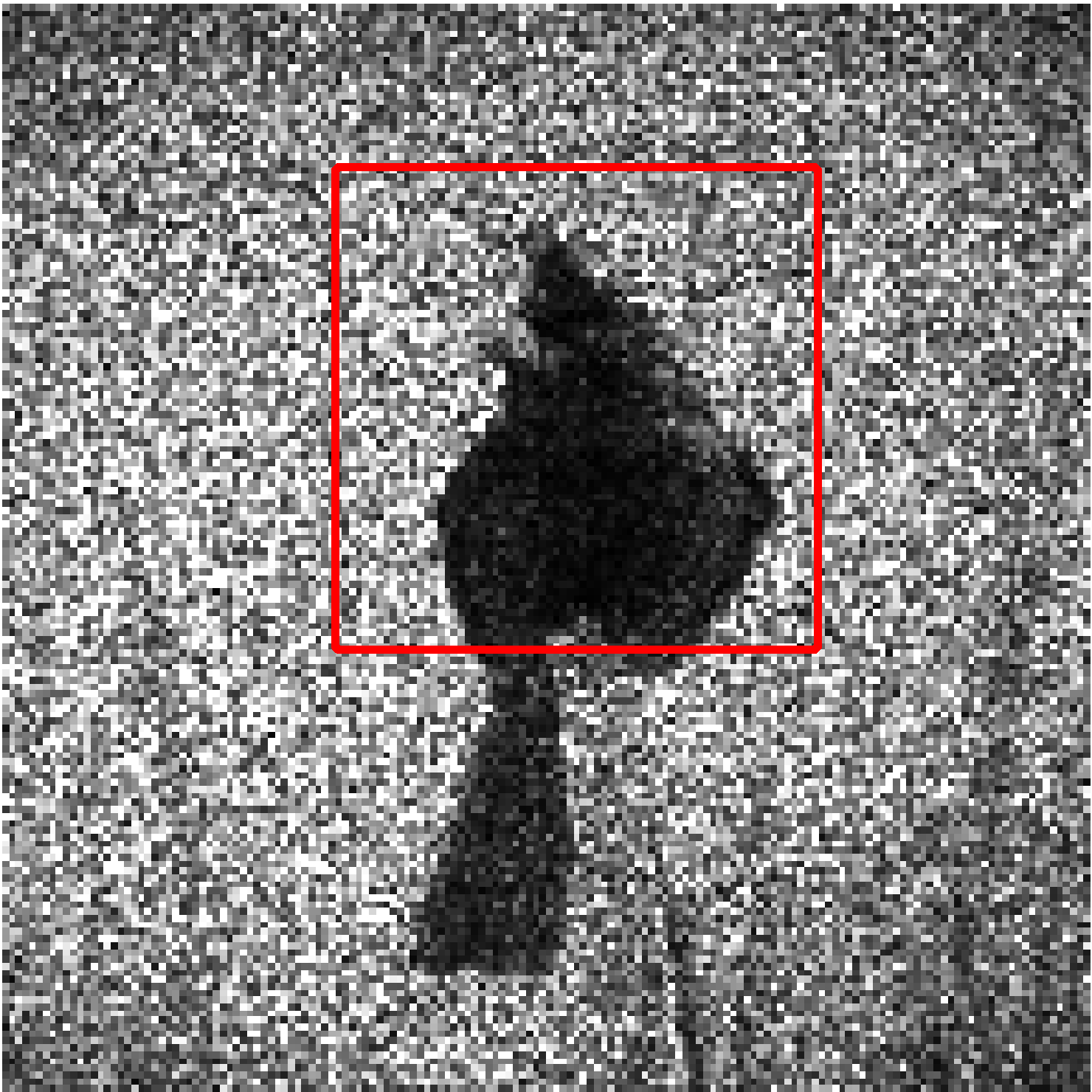}
  \end{subfigure}

  \begin{subfigure}[b]{0.16\linewidth}
      \centering
      \includegraphics[width=\textwidth]{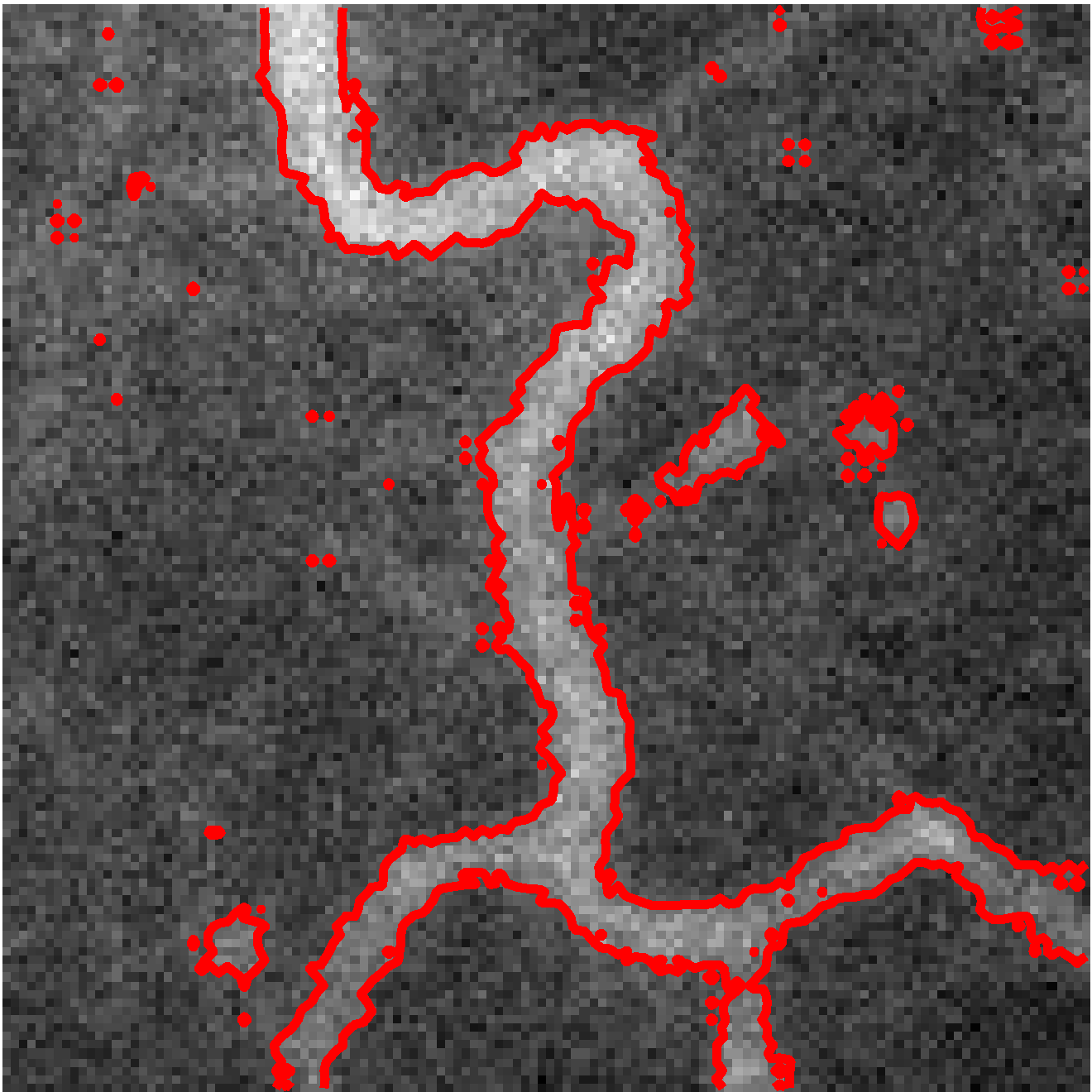}
  \end{subfigure}
  \hfill
  \begin{subfigure}[b]{0.16\linewidth}
      \centering
      \includegraphics[width=\textwidth]{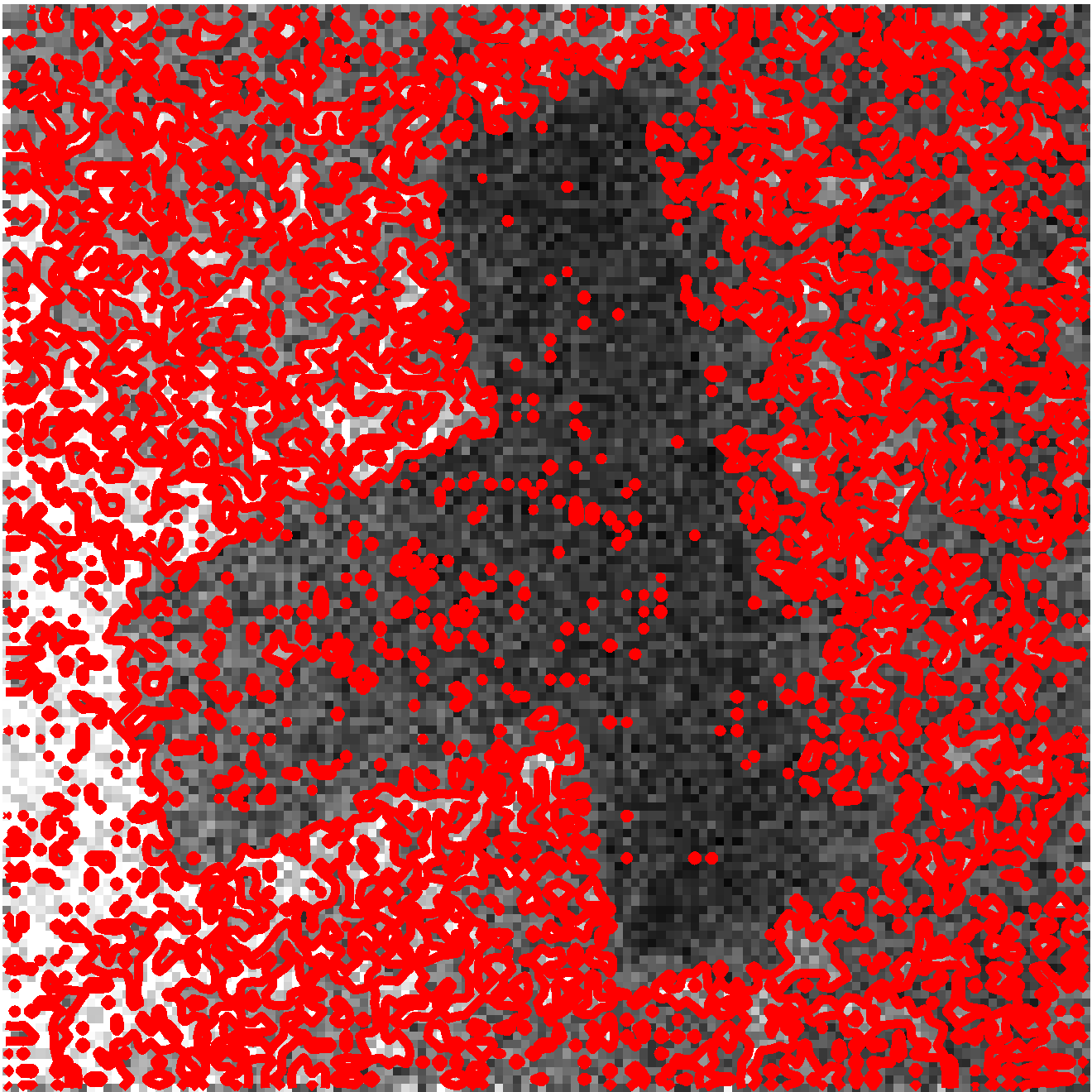}
  \end{subfigure}
  \hfill
  \begin{subfigure}[b]{0.16\linewidth}
      \centering
      \includegraphics[width=\textwidth]{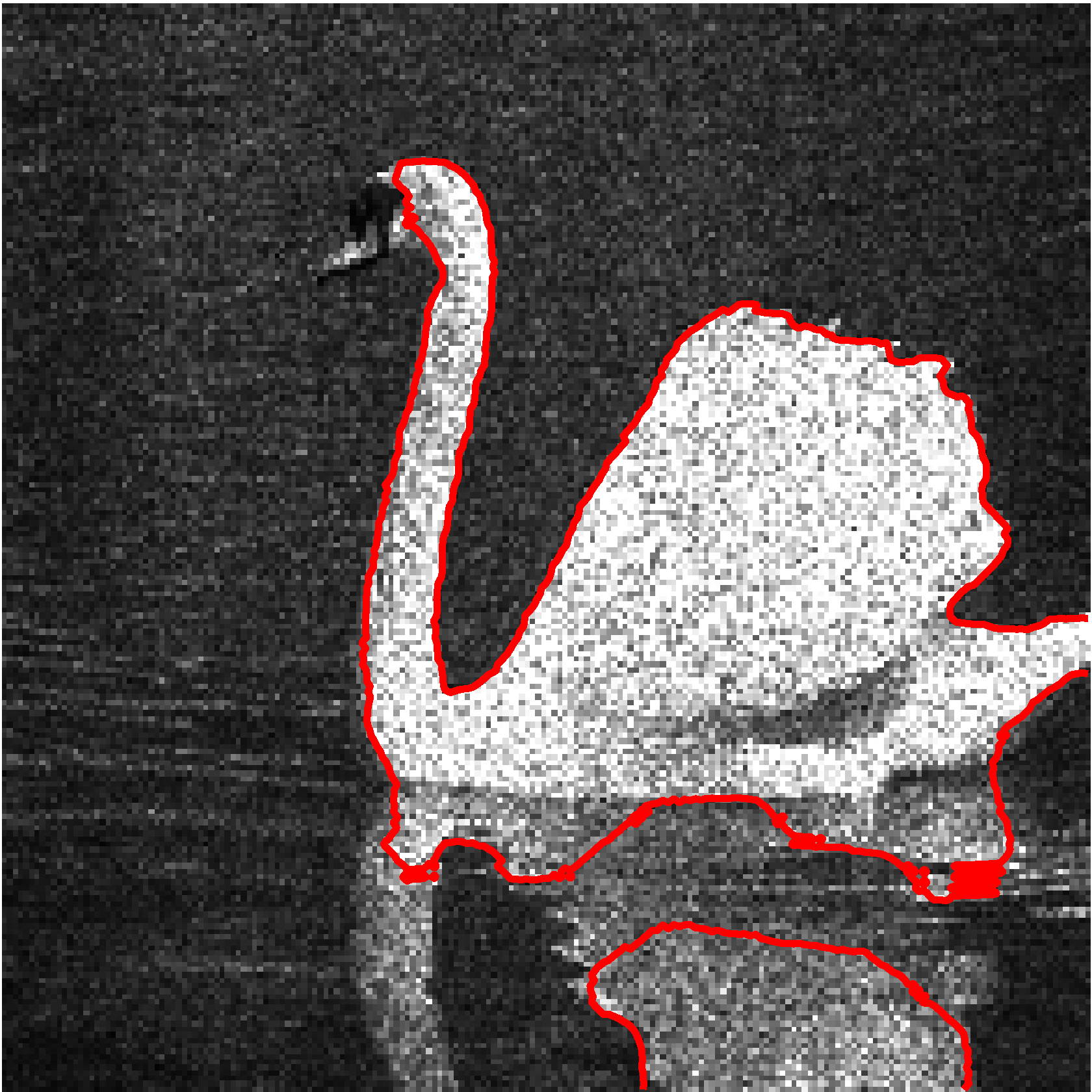}
  \end{subfigure}
  \hfill
  \begin{subfigure}[b]{0.16\linewidth}
      \centering
      \includegraphics[width=\textwidth]{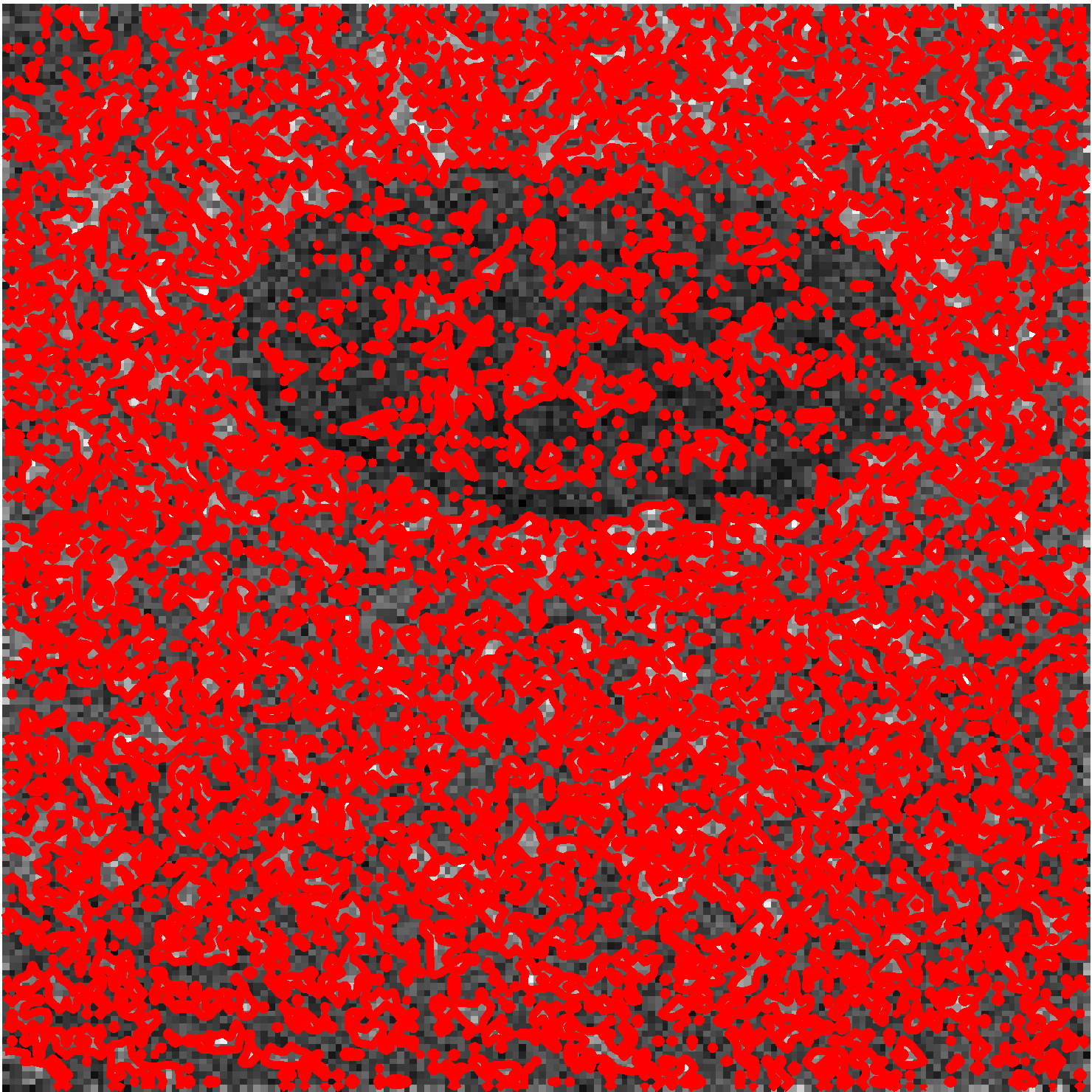}
  \end{subfigure}
  \hfill
  \begin{subfigure}[b]{0.16\linewidth}
      \centering
      \includegraphics[width=\textwidth]{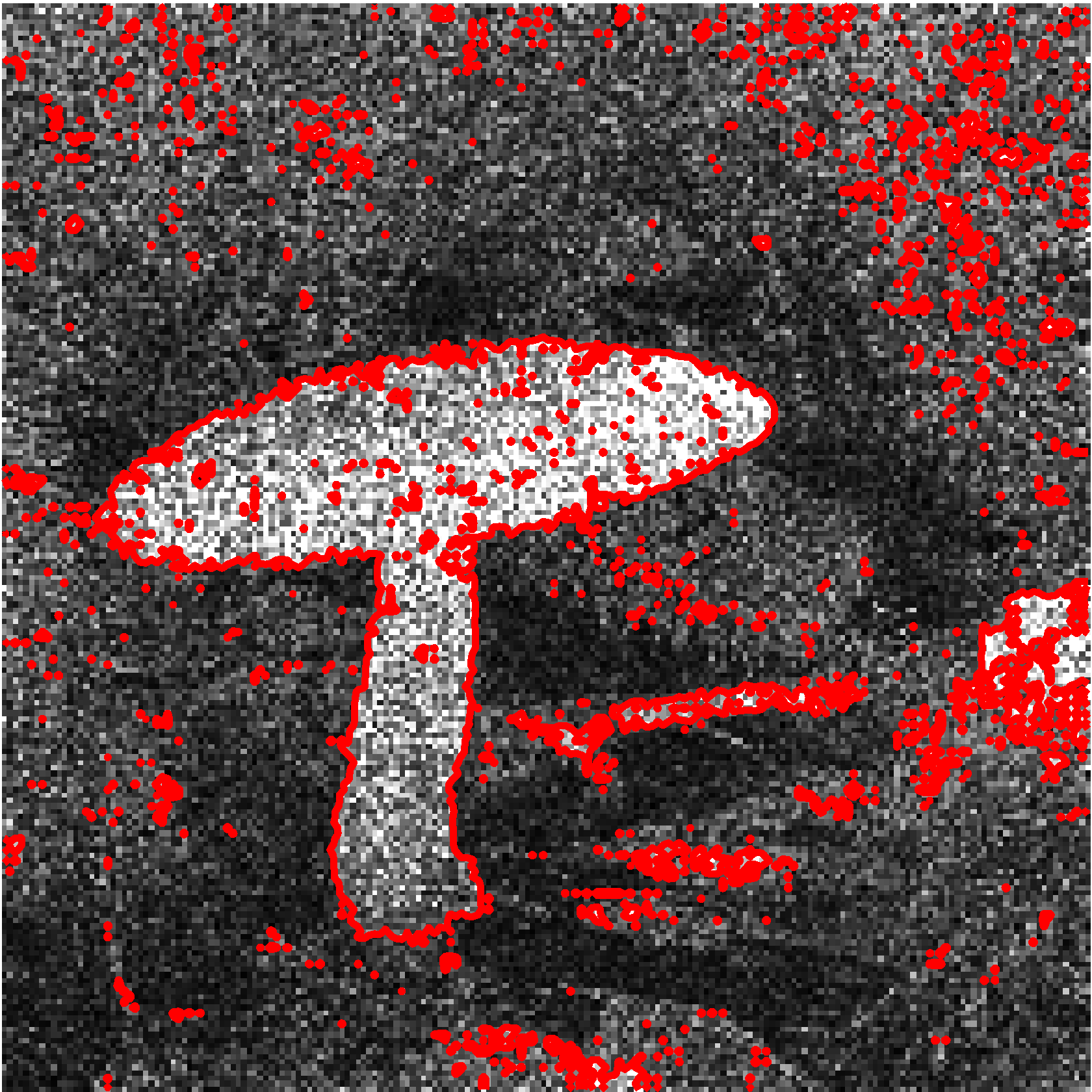}
  \end{subfigure}
  \hfill
  \begin{subfigure}[b]{0.16\linewidth}
      \centering
      \includegraphics[width=\textwidth]{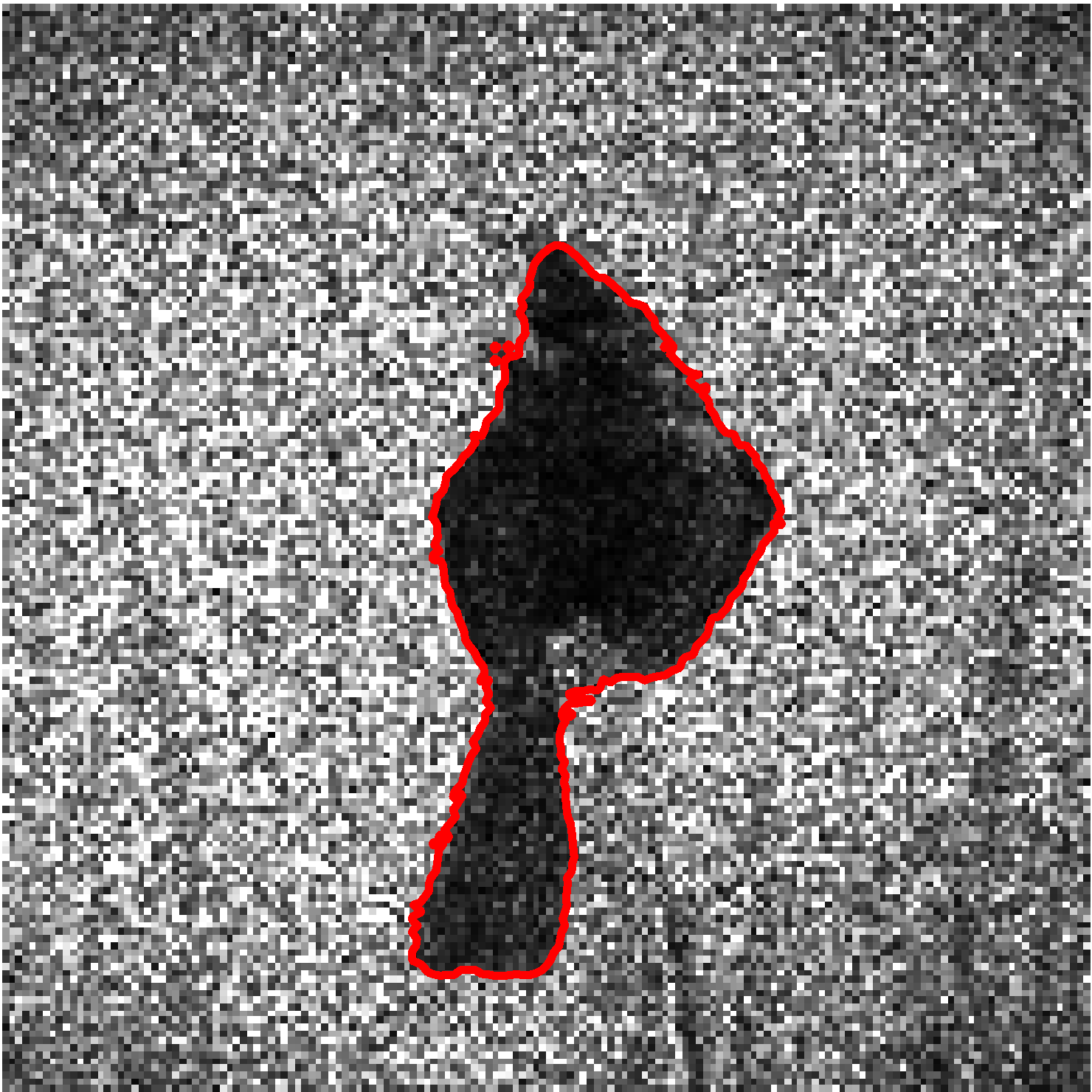}
  \end{subfigure}

  \begin{subfigure}[b]{0.16\linewidth}
      \centering
      \includegraphics[width=\textwidth]{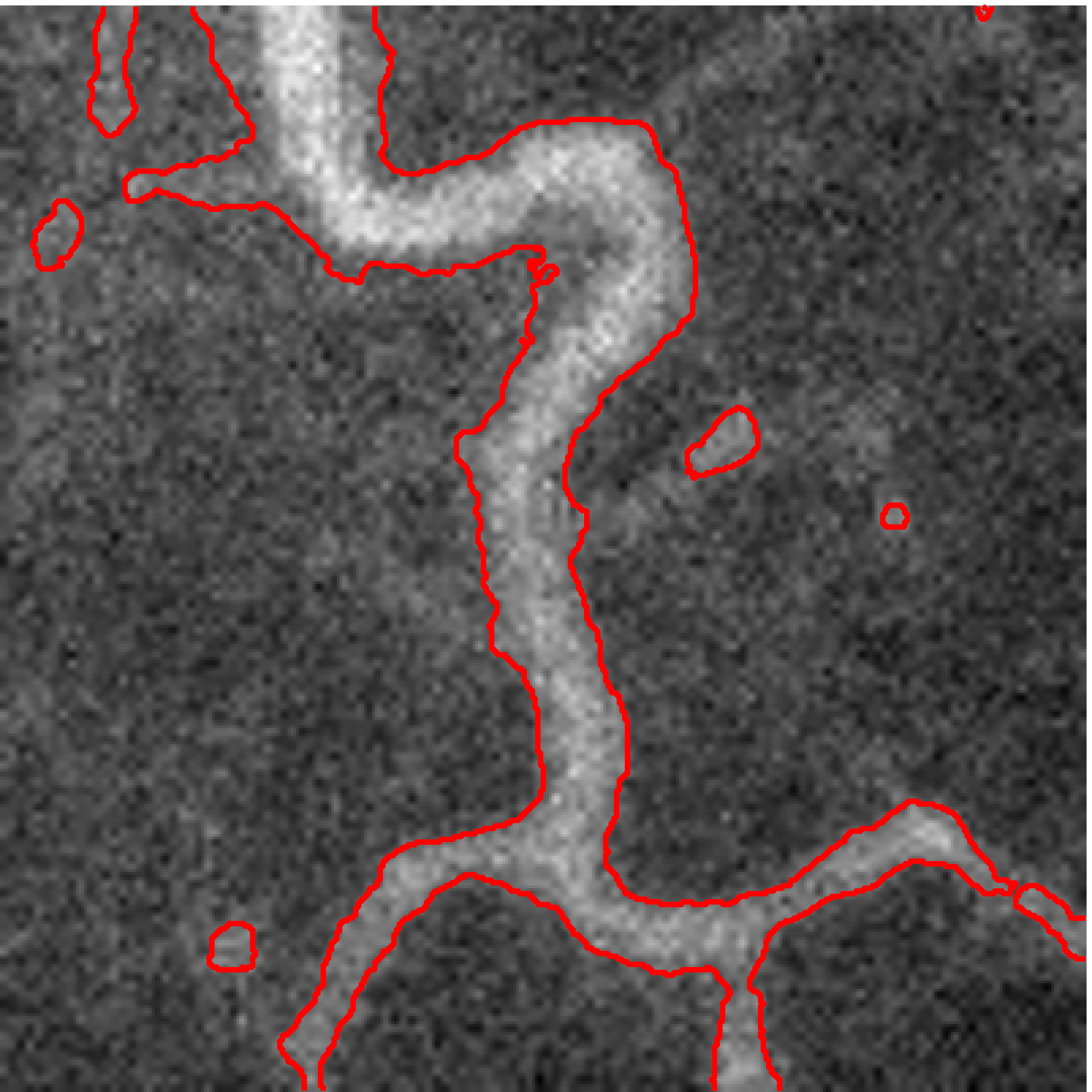}
  \end{subfigure}
  \hfill
  \begin{subfigure}[b]{0.16\linewidth}
      \centering
      \includegraphics[width=\textwidth]{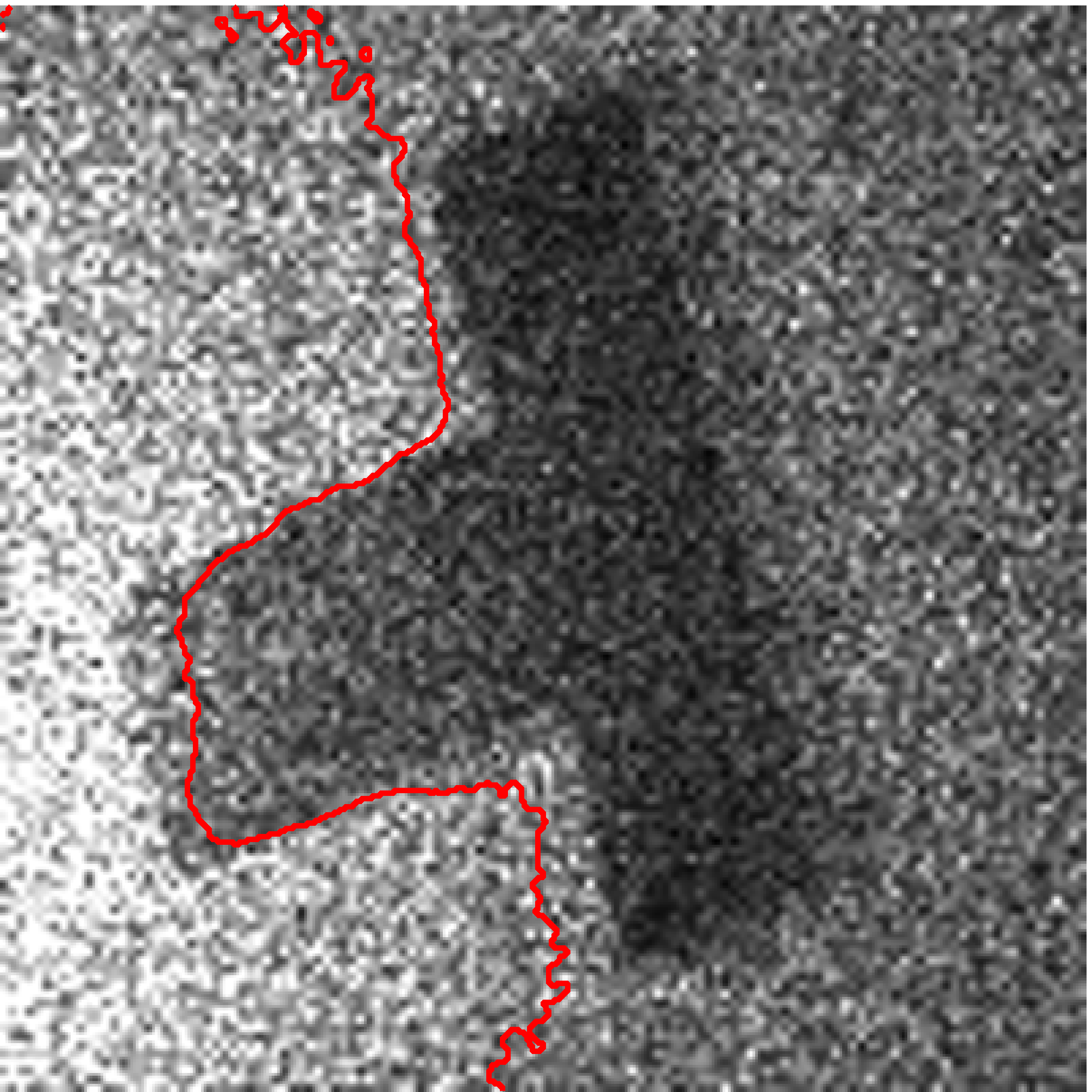}
  \end{subfigure}
  \hfill
  \begin{subfigure}[b]{0.16\linewidth}
      \centering
      \includegraphics[width=\textwidth]{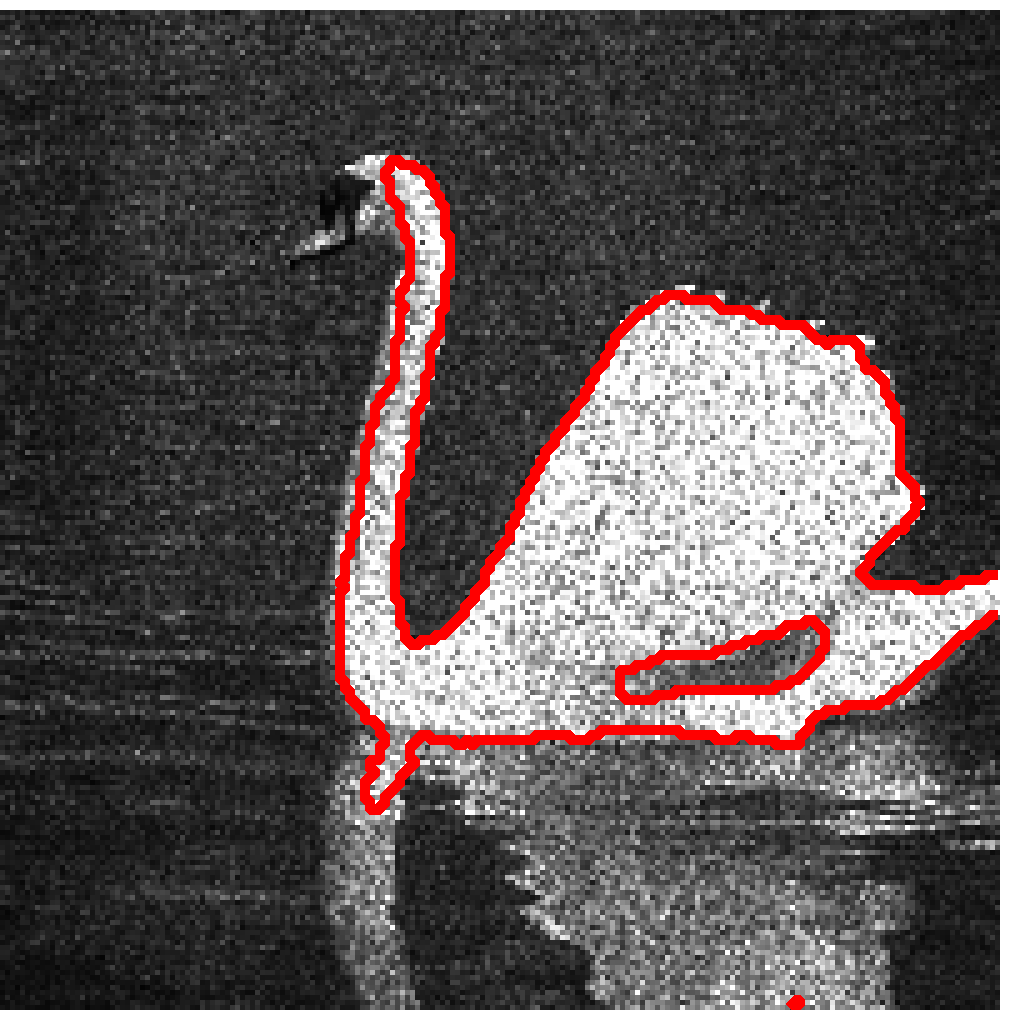}
  \end{subfigure}
  \hfill
  \begin{subfigure}[b]{0.16\linewidth}
      \centering
      \includegraphics[width=\textwidth]{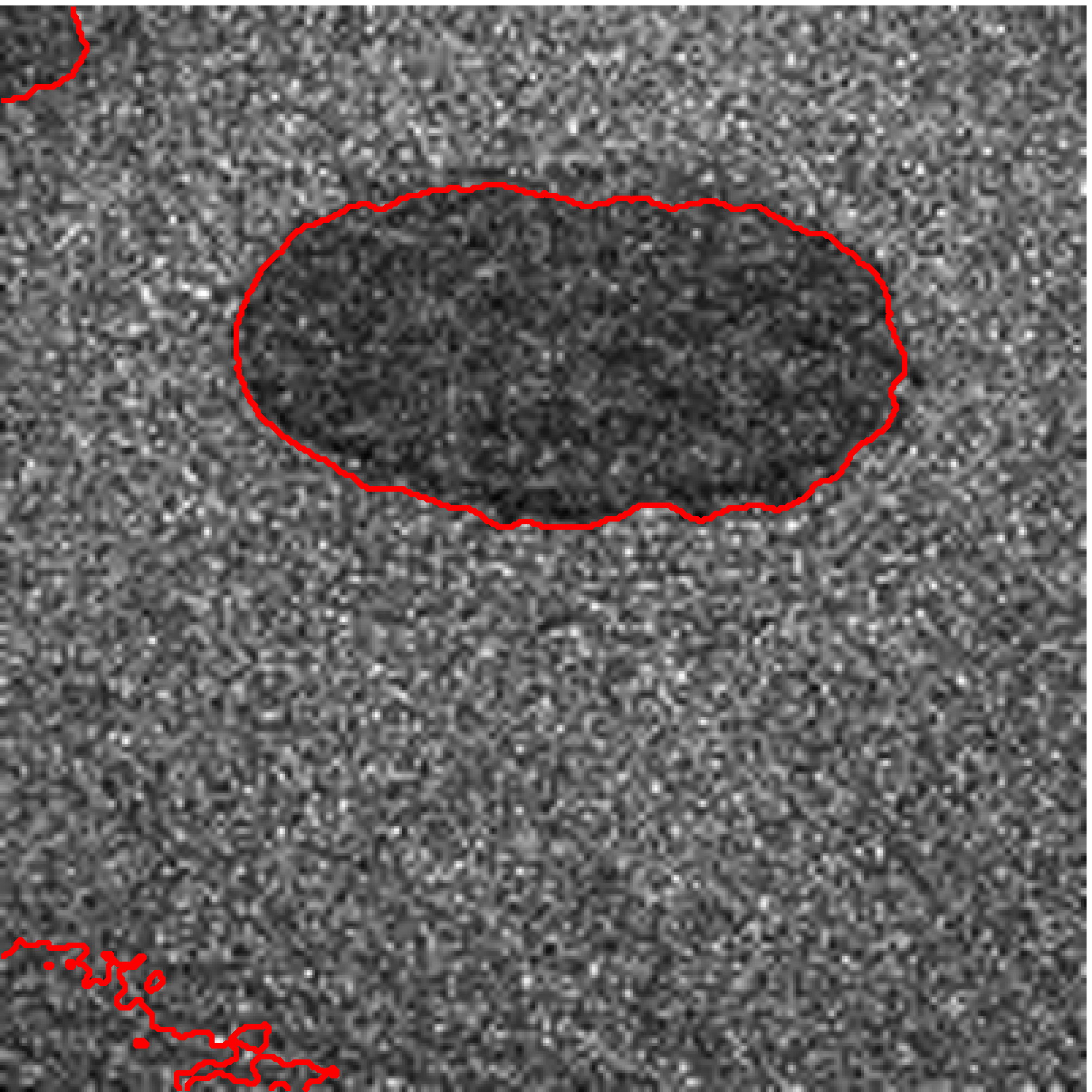}
  \end{subfigure}
  \hfill
  \begin{subfigure}[b]{0.16\linewidth}
      \centering
      \includegraphics[width=\textwidth]{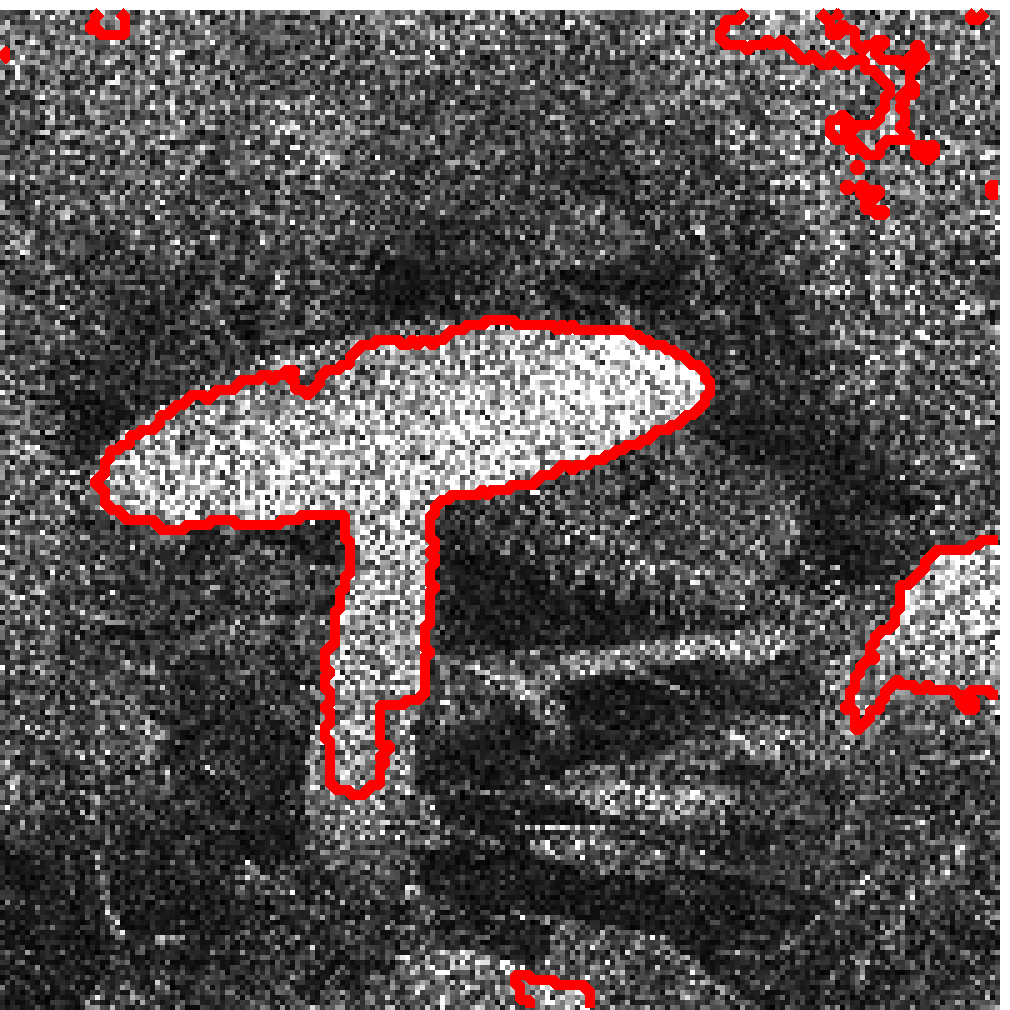}
  \end{subfigure}
  \hfill
  \begin{subfigure}[b]{0.16\linewidth}
      \centering
      \includegraphics[width=\textwidth]{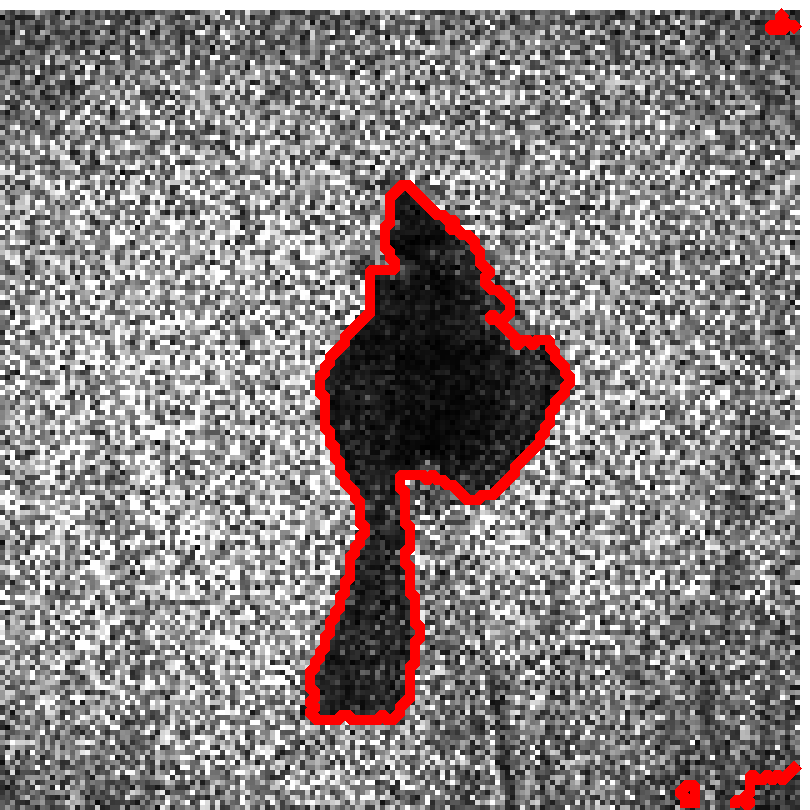}
  \end{subfigure}

   \begin{subfigure}[b]{0.16\linewidth}
      \centering
      \includegraphics[width=\textwidth]{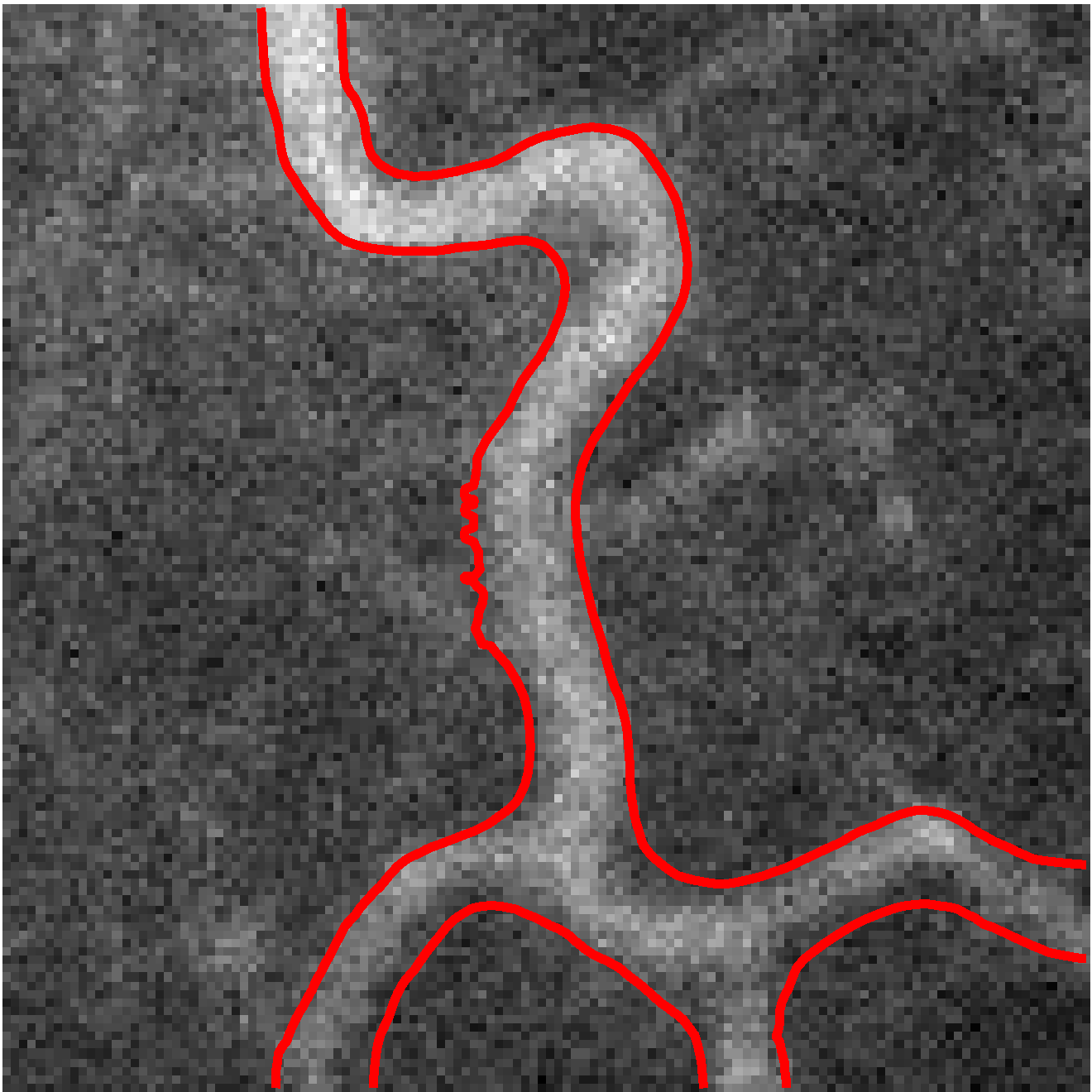}
  \end{subfigure}
  \hfill
  \begin{subfigure}[b]{0.16\linewidth}
      \centering
      \includegraphics[width=\textwidth]{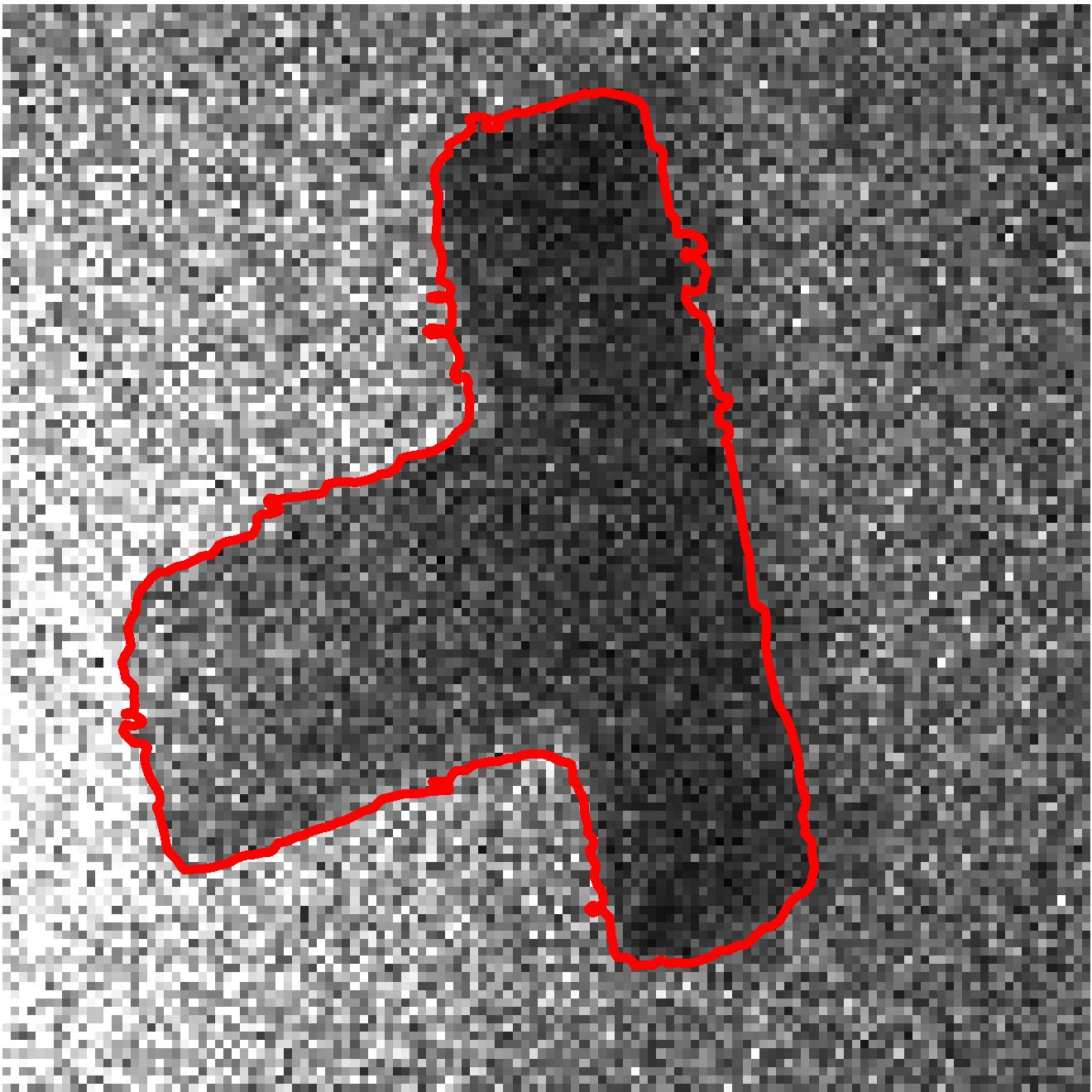}
  \end{subfigure}
  \hfill
  \begin{subfigure}[b]{0.16\linewidth}
      \centering
      \includegraphics[width=\textwidth]{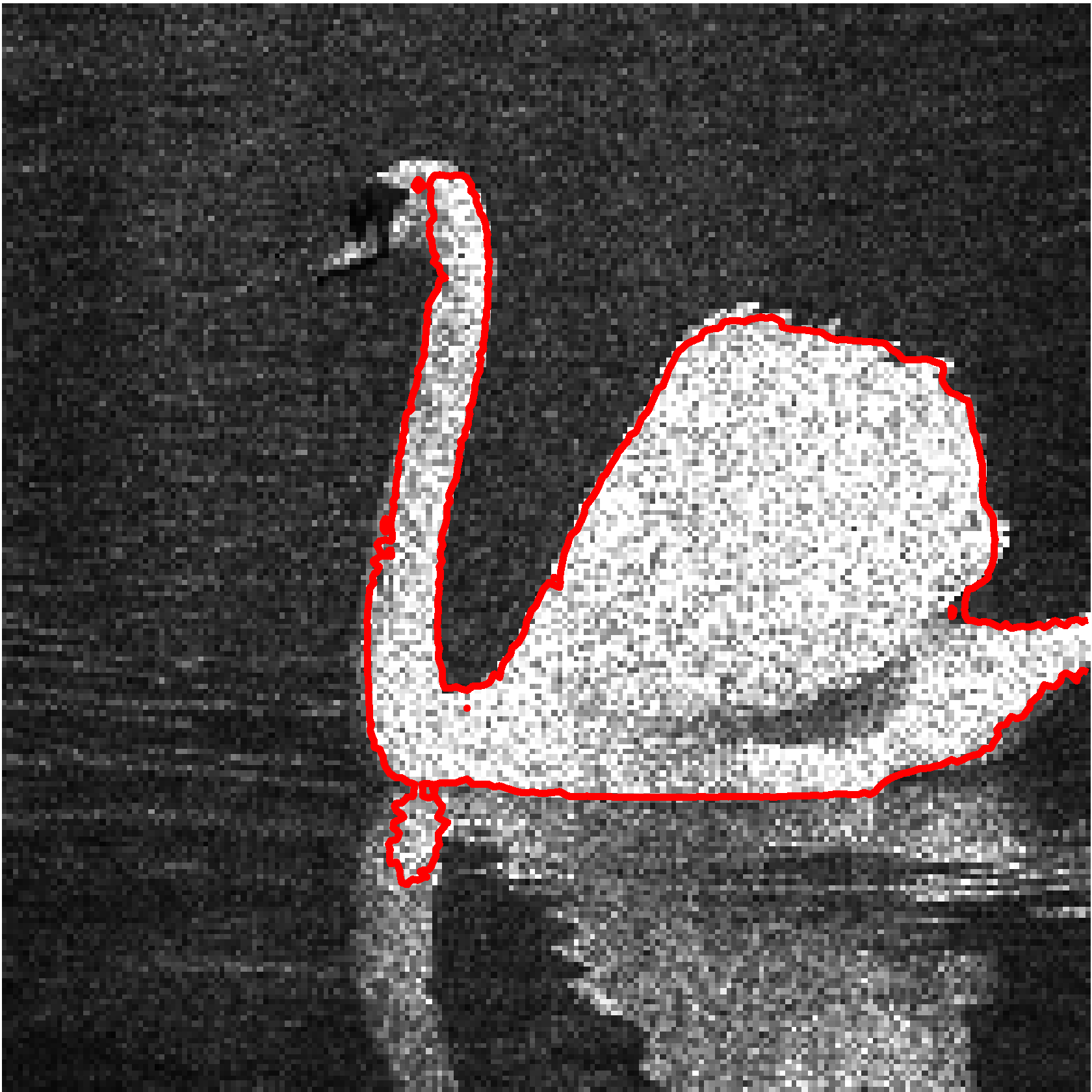}
  \end{subfigure}
  \hfill
  \begin{subfigure}[b]{0.16\linewidth}
      \centering
      \includegraphics[width=\textwidth]{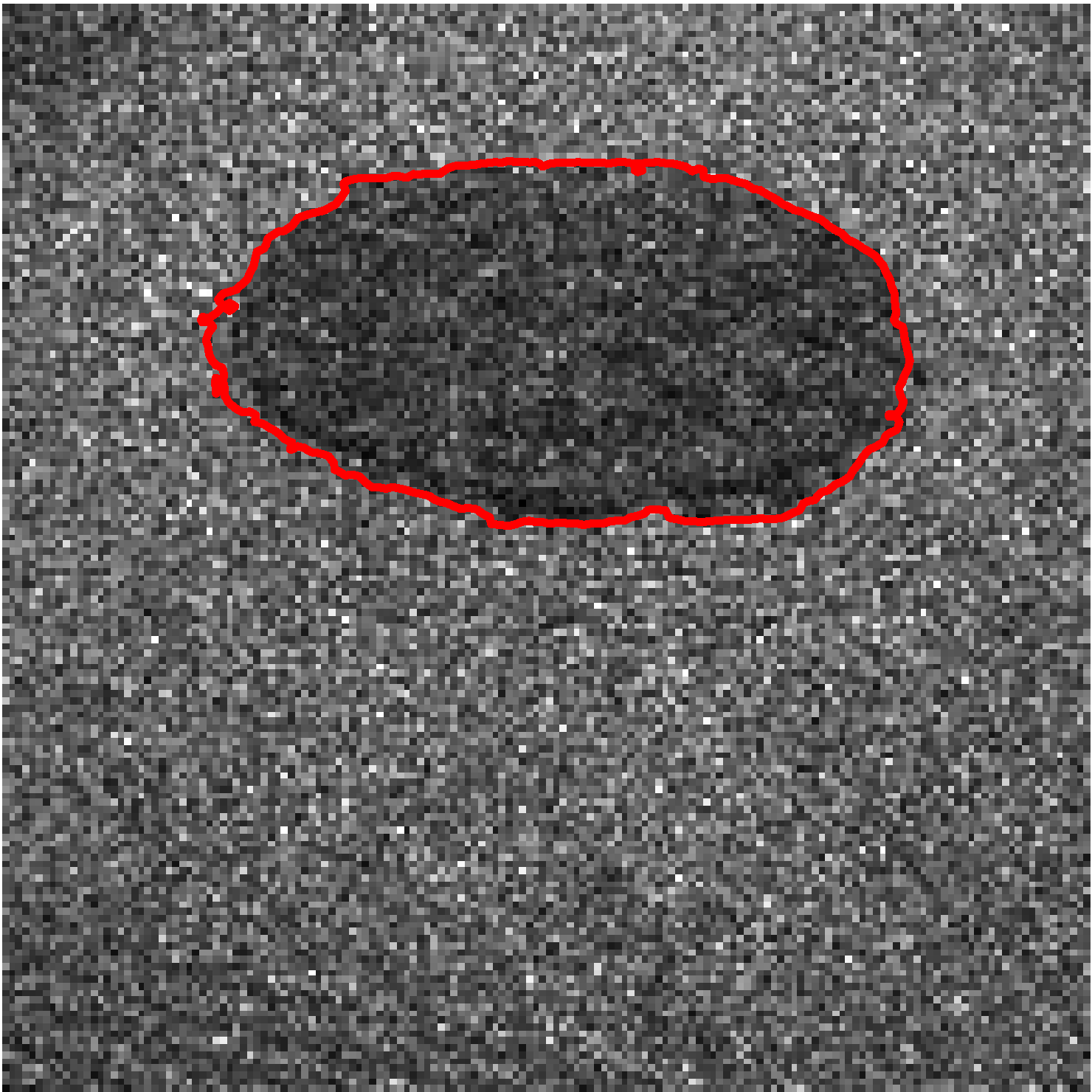}
  \end{subfigure}
  \hfill
  \begin{subfigure}[b]{0.16\linewidth}
      \centering
      \includegraphics[width=\textwidth]{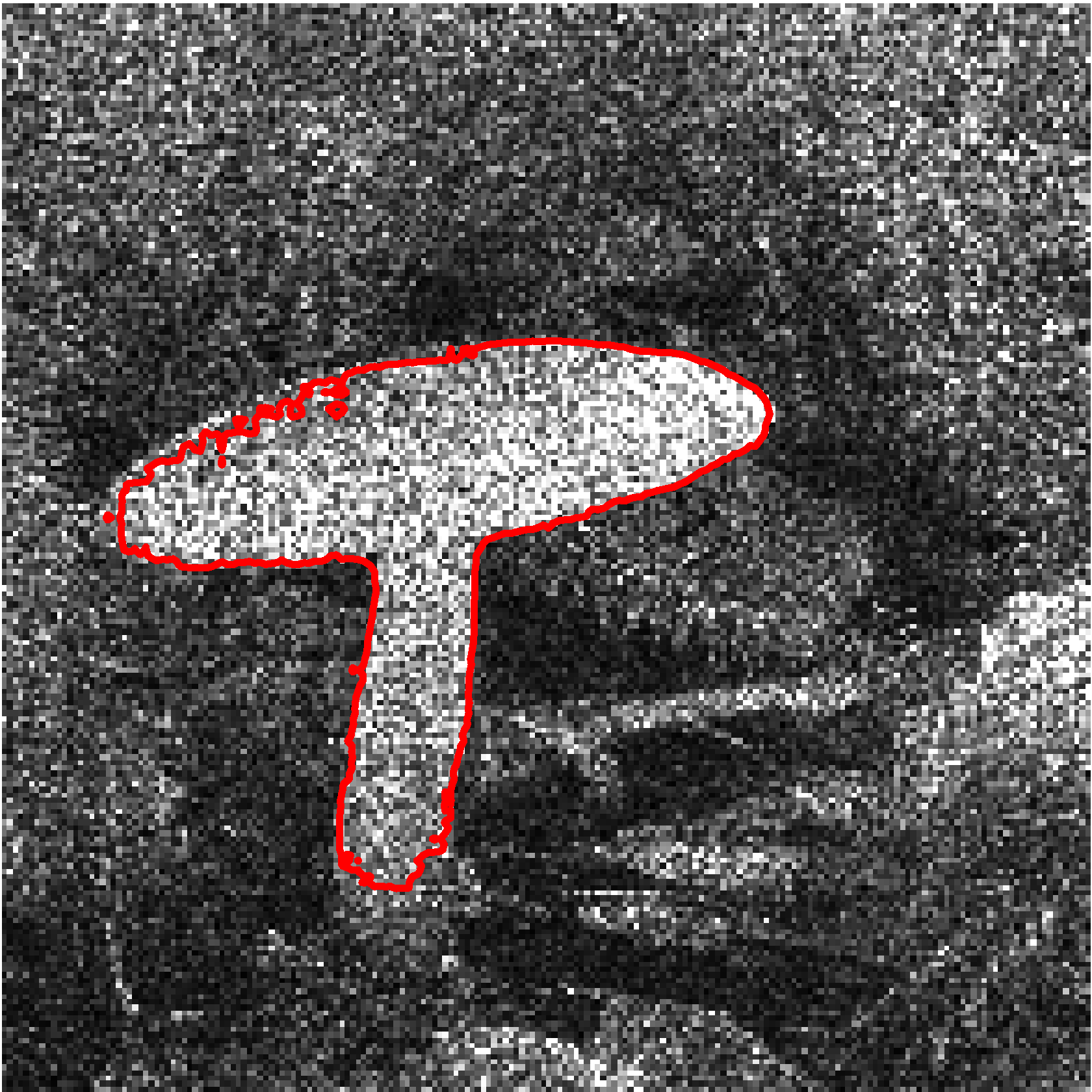}
  \end{subfigure}
  \hfill
  \begin{subfigure}[b]{0.16\linewidth}
      \centering
      \includegraphics[width=\textwidth]{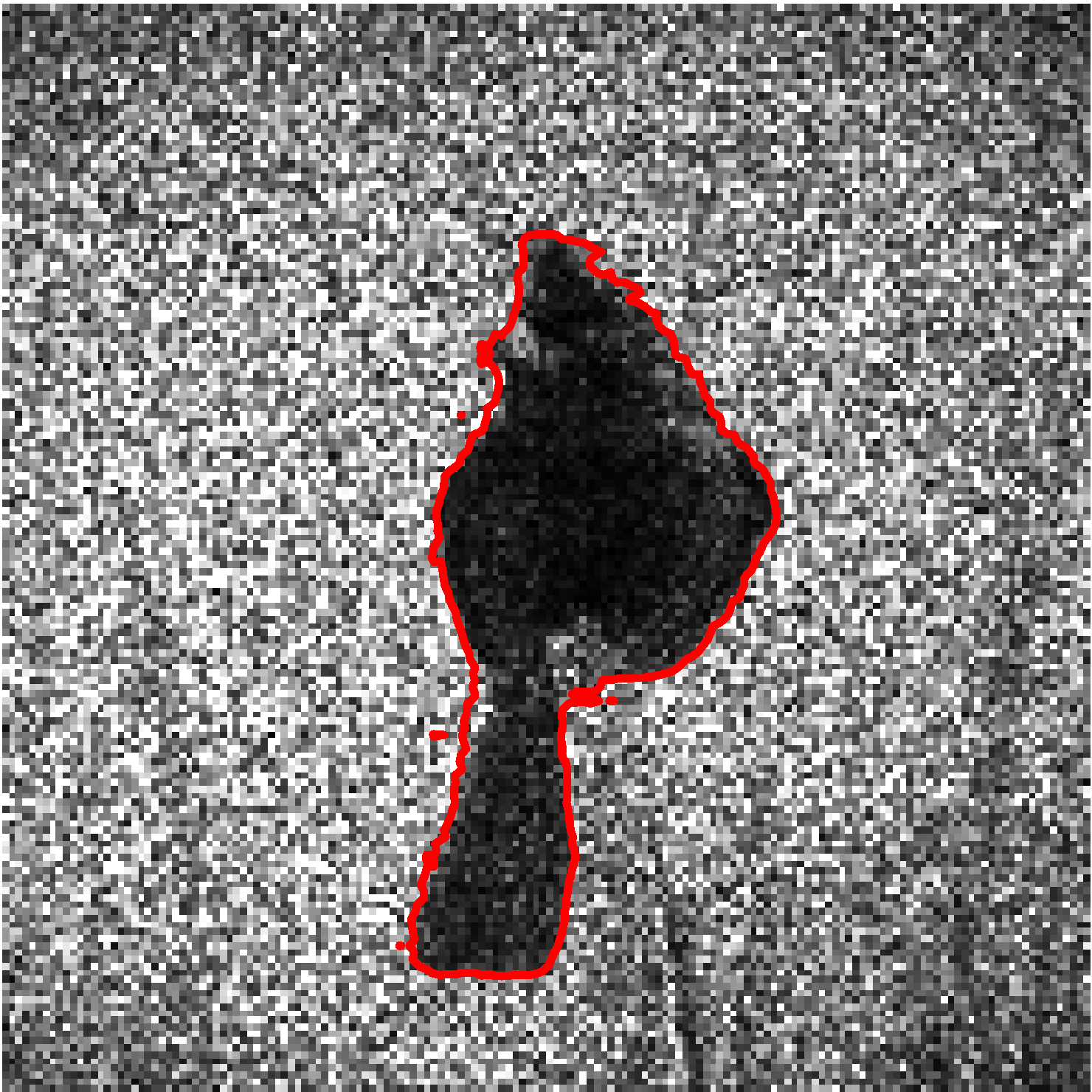}
  \end{subfigure}

   \begin{subfigure}[b]{0.16\linewidth}
      \centering
      \includegraphics[width=\textwidth]{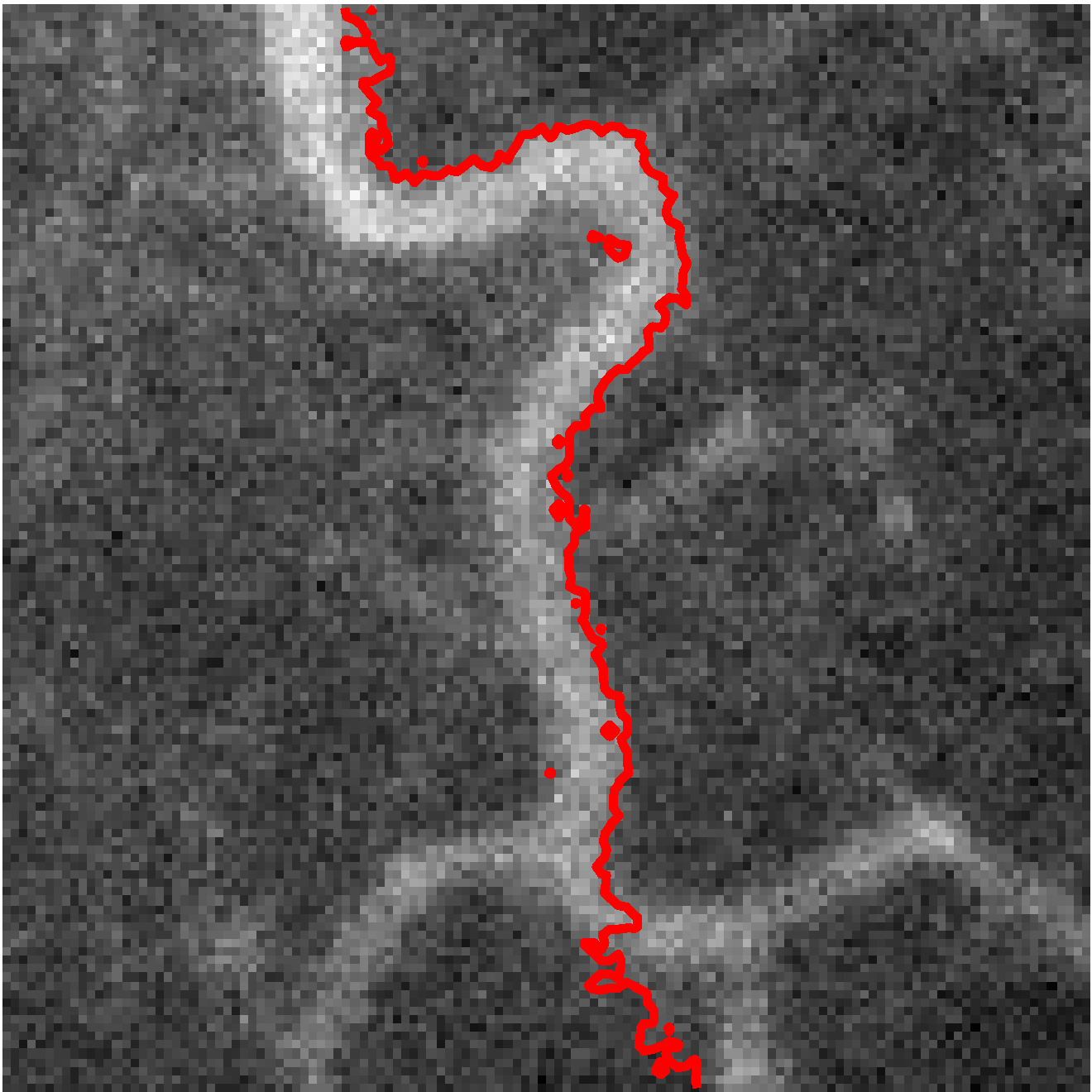}
  \end{subfigure}
  \hfill
   \begin{subfigure}[b]{0.16\linewidth}
      \centering
      \includegraphics[width=\textwidth]{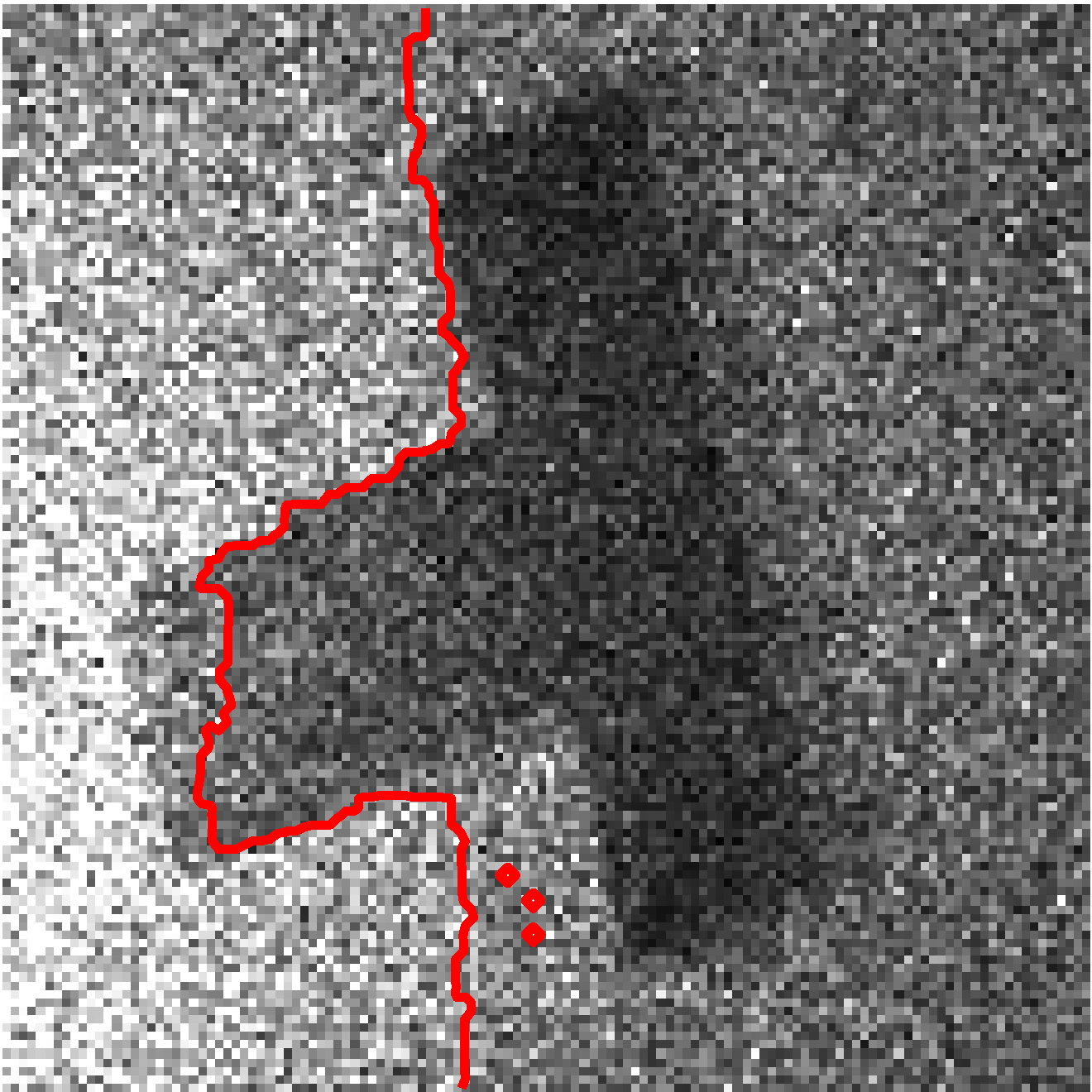}
  \end{subfigure}
  \hfill
  \begin{subfigure}[b]{0.16\linewidth}
      \centering
      \includegraphics[width=\textwidth]{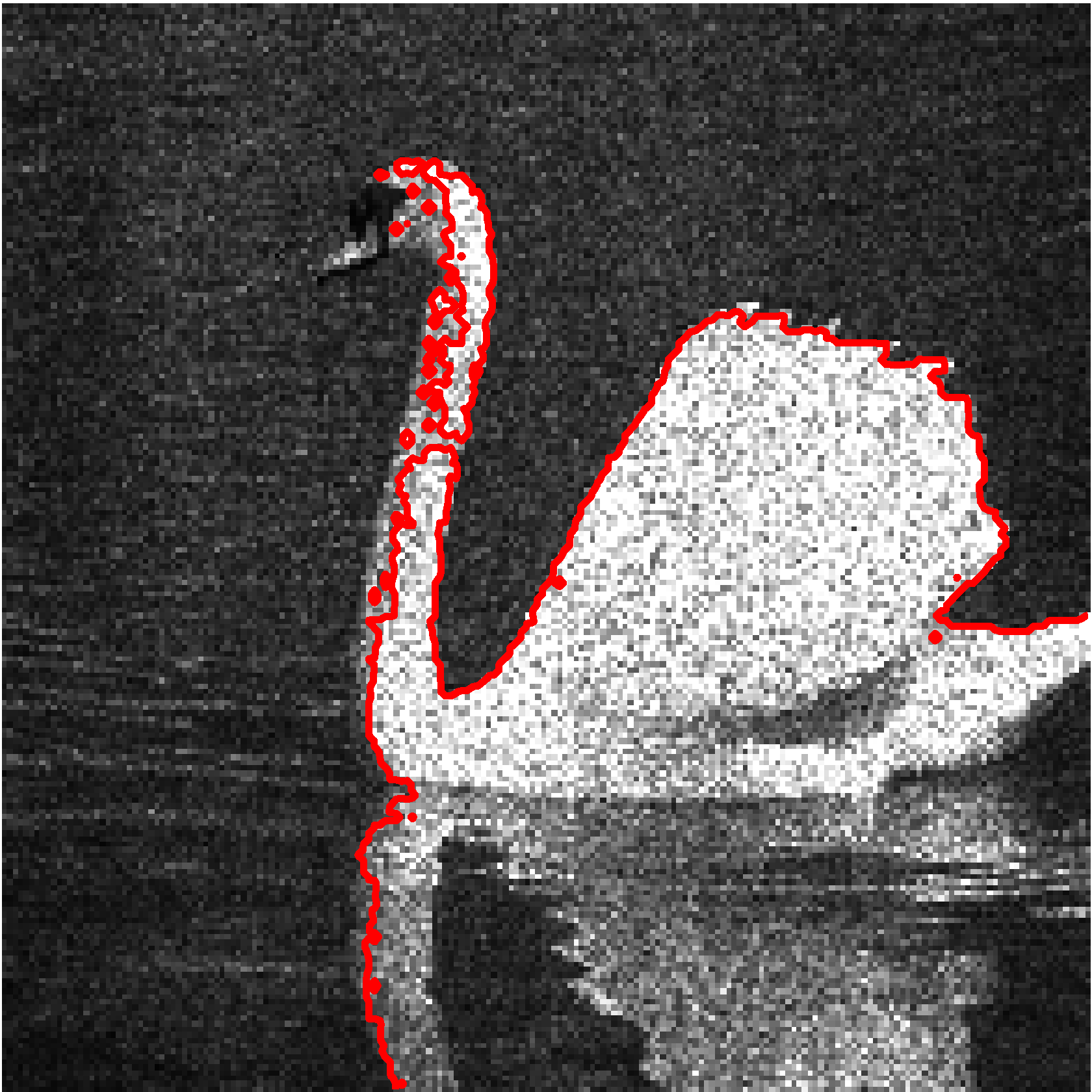}
  \end{subfigure}
  \hfill
  \begin{subfigure}[b]{0.16\linewidth}
      \centering
      \includegraphics[width=\textwidth]{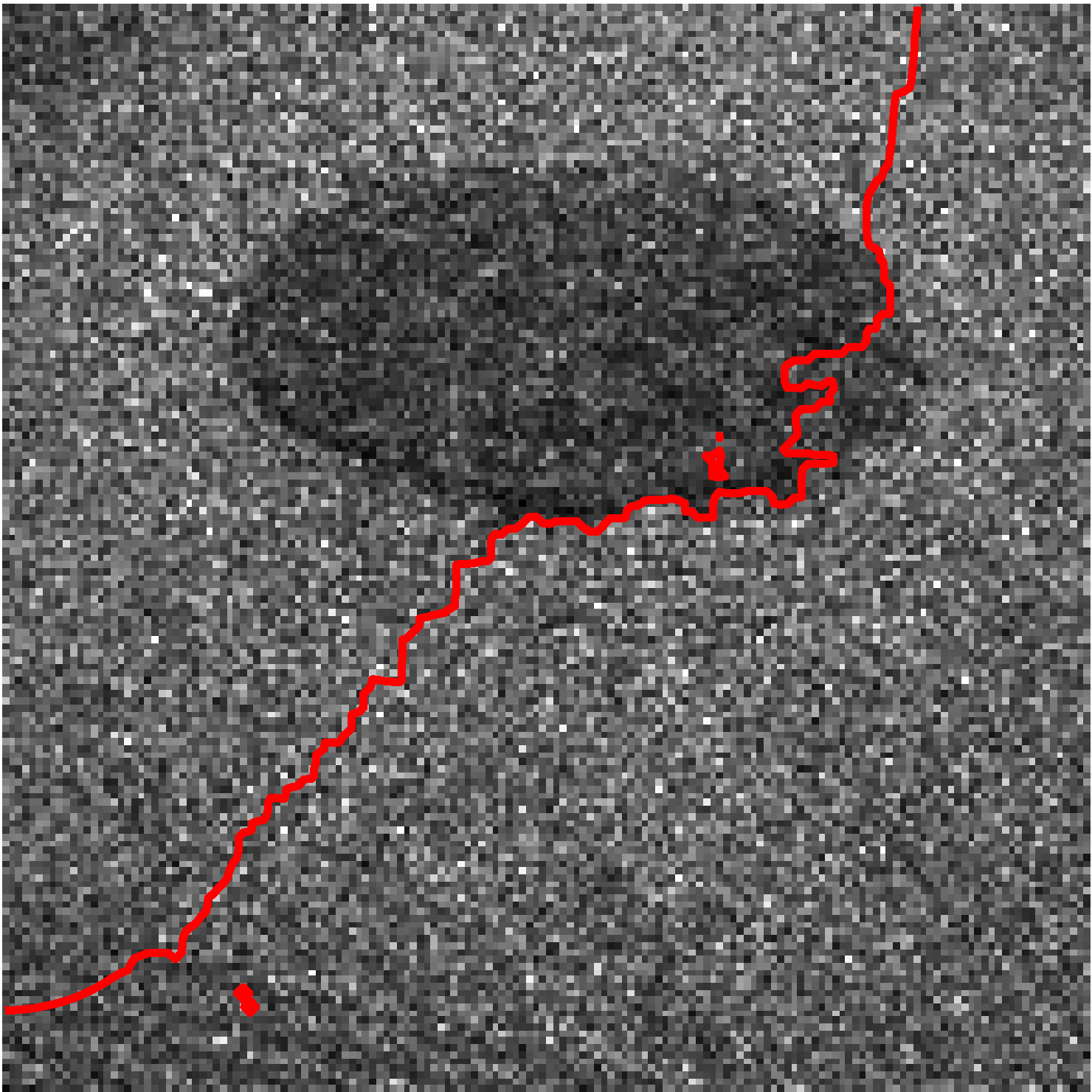}
  \end{subfigure}
  \hfill
   \begin{subfigure}[b]{0.16\linewidth}
      \centering
      \includegraphics[width=\textwidth]{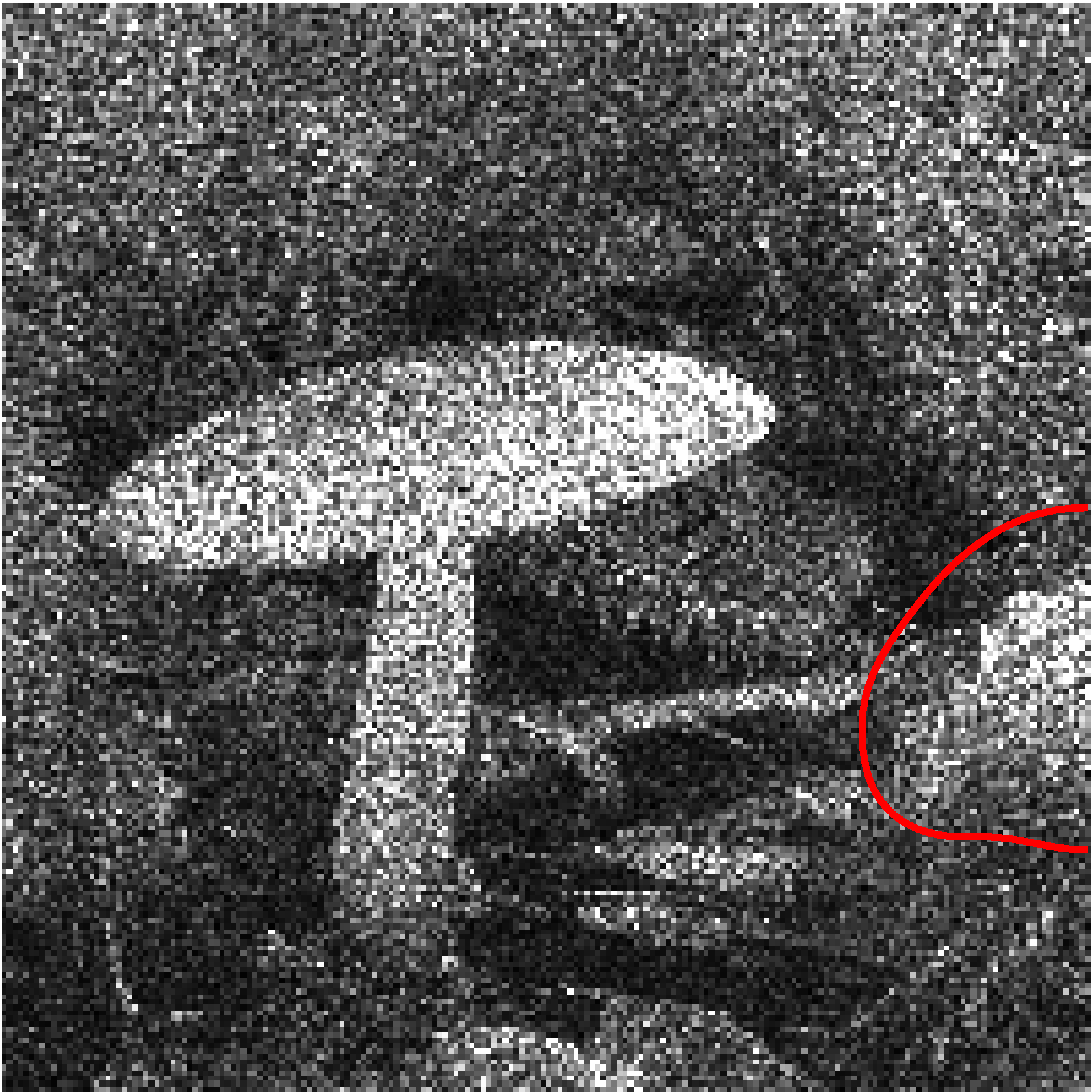}
  \end{subfigure}
  \hfill
  \begin{subfigure}[b]{0.16\linewidth}
      \centering
      \includegraphics[width=\textwidth]{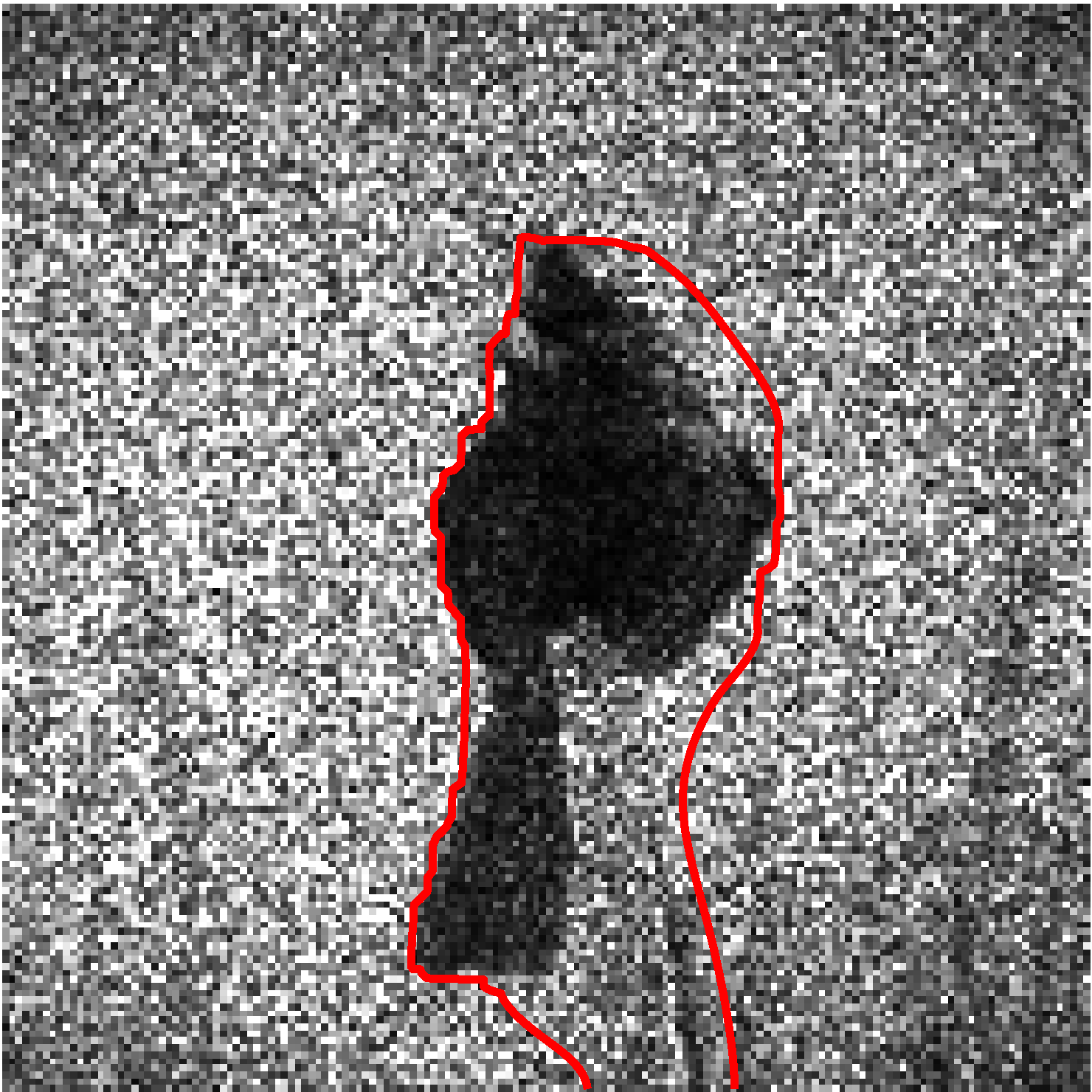}
  \end{subfigure}

  \begin{subfigure}[b]{0.16\linewidth}
      \centering
      \includegraphics[width=\textwidth]{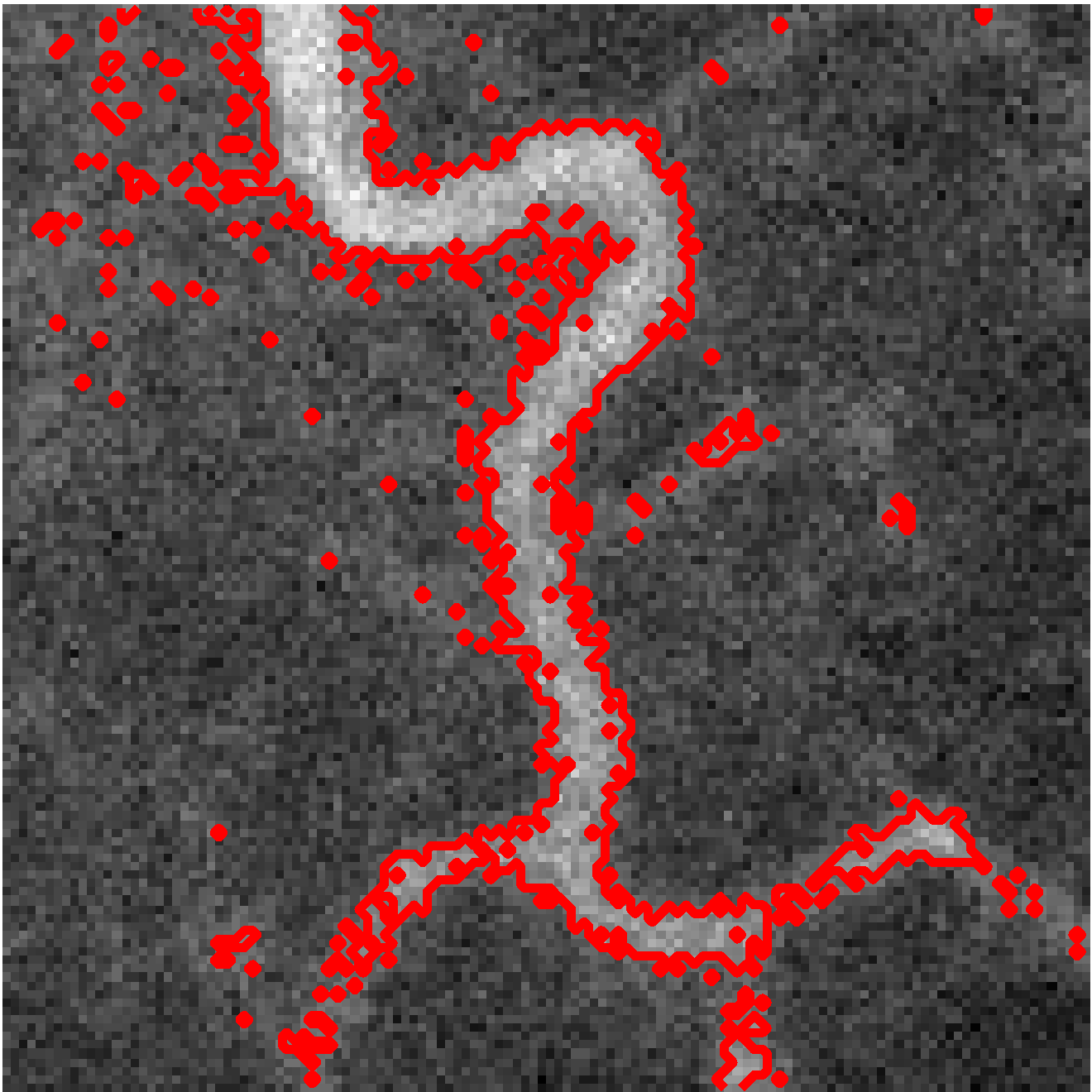}
  \end{subfigure}
  \hfill
  \begin{subfigure}[b]{0.16\linewidth}
      \centering
      \includegraphics[width=\textwidth]{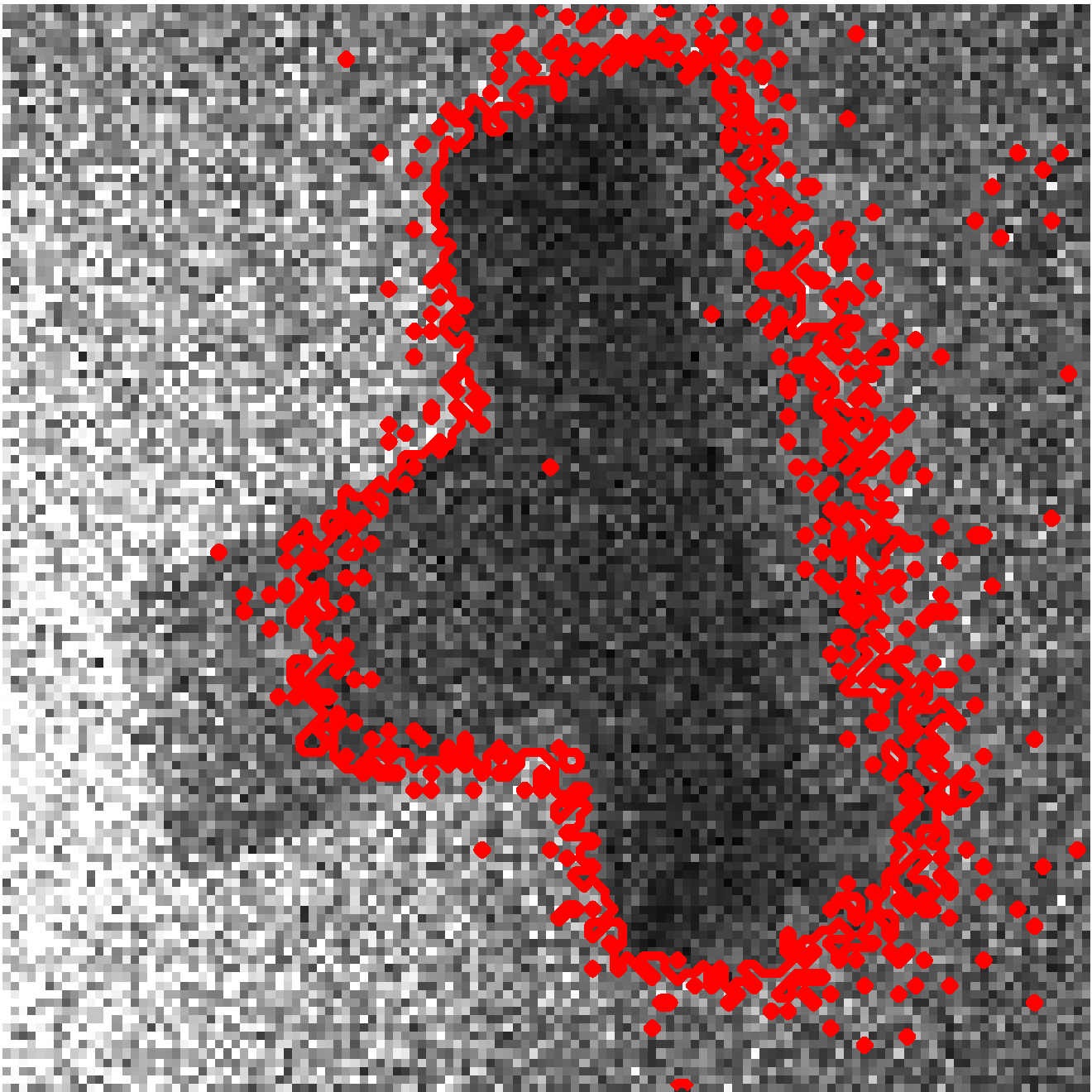}
  \end{subfigure}
  \hfill
  \begin{subfigure}[b]{0.16\linewidth}
      \centering
      \includegraphics[width=\textwidth]{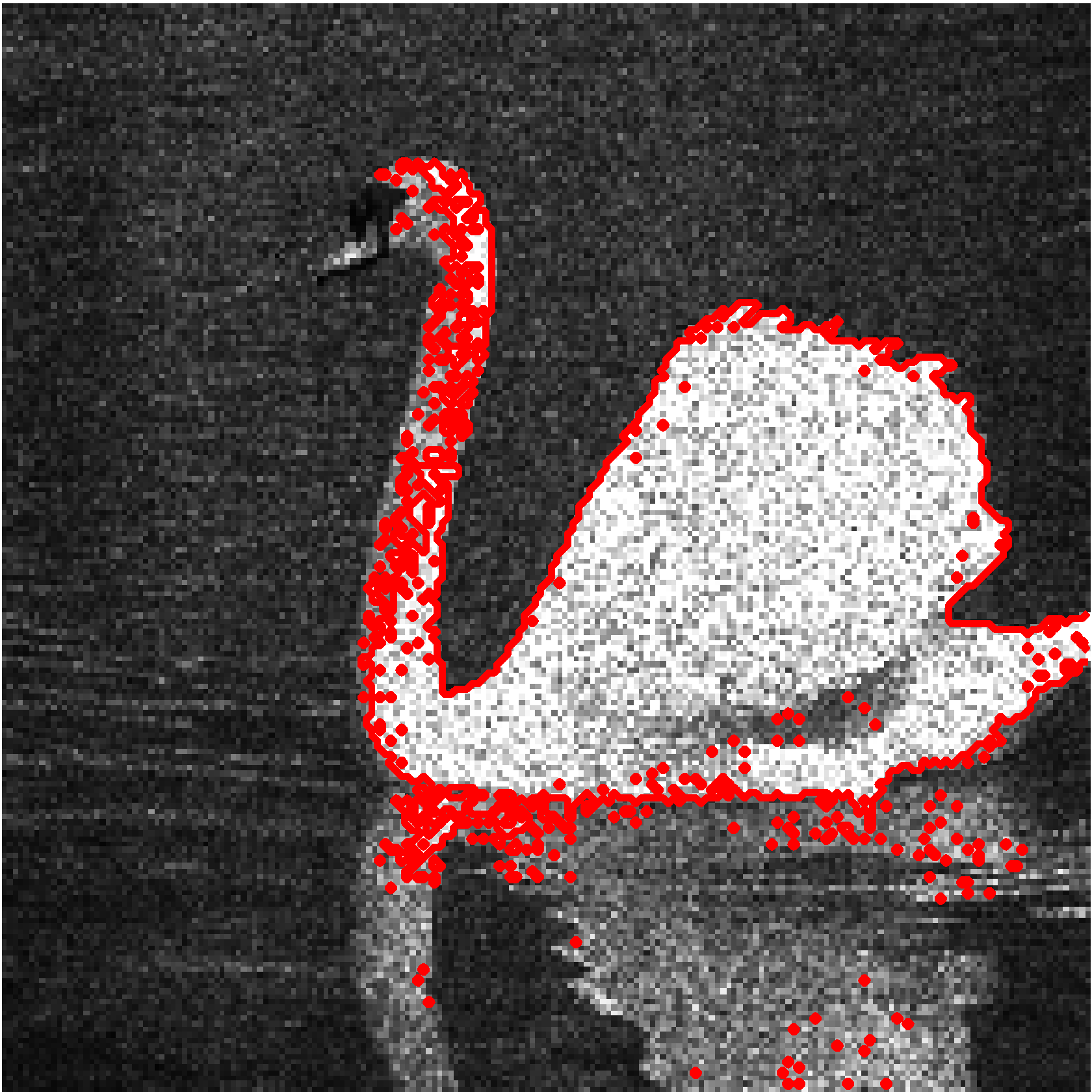}
  \end{subfigure}
  \hfill
  \begin{subfigure}[b]{0.16\linewidth}
      \centering
      \includegraphics[width=\textwidth]{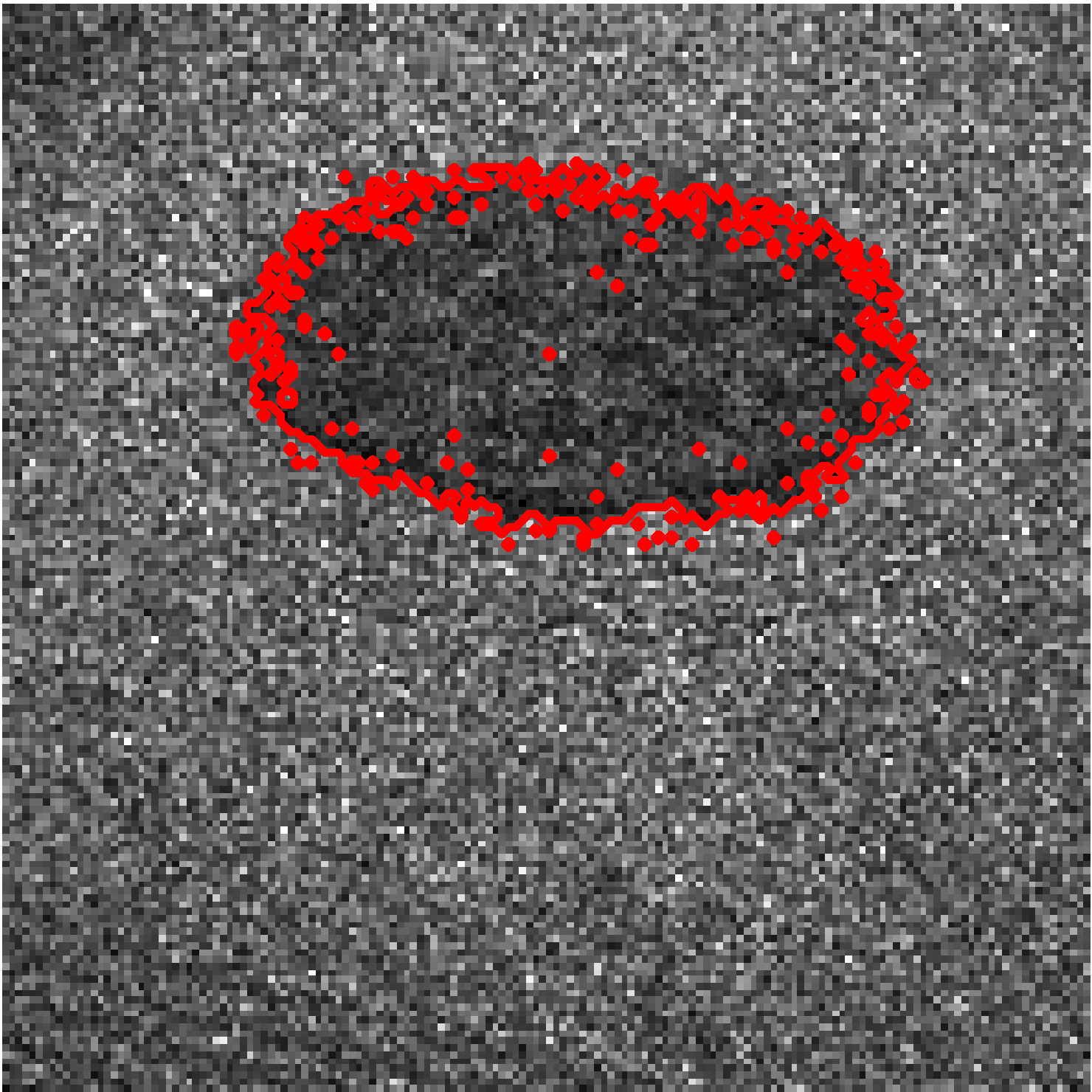}
  \end{subfigure}
  \hfill
  \begin{subfigure}[b]{0.16\linewidth}
      \centering
      \includegraphics[width=\textwidth]{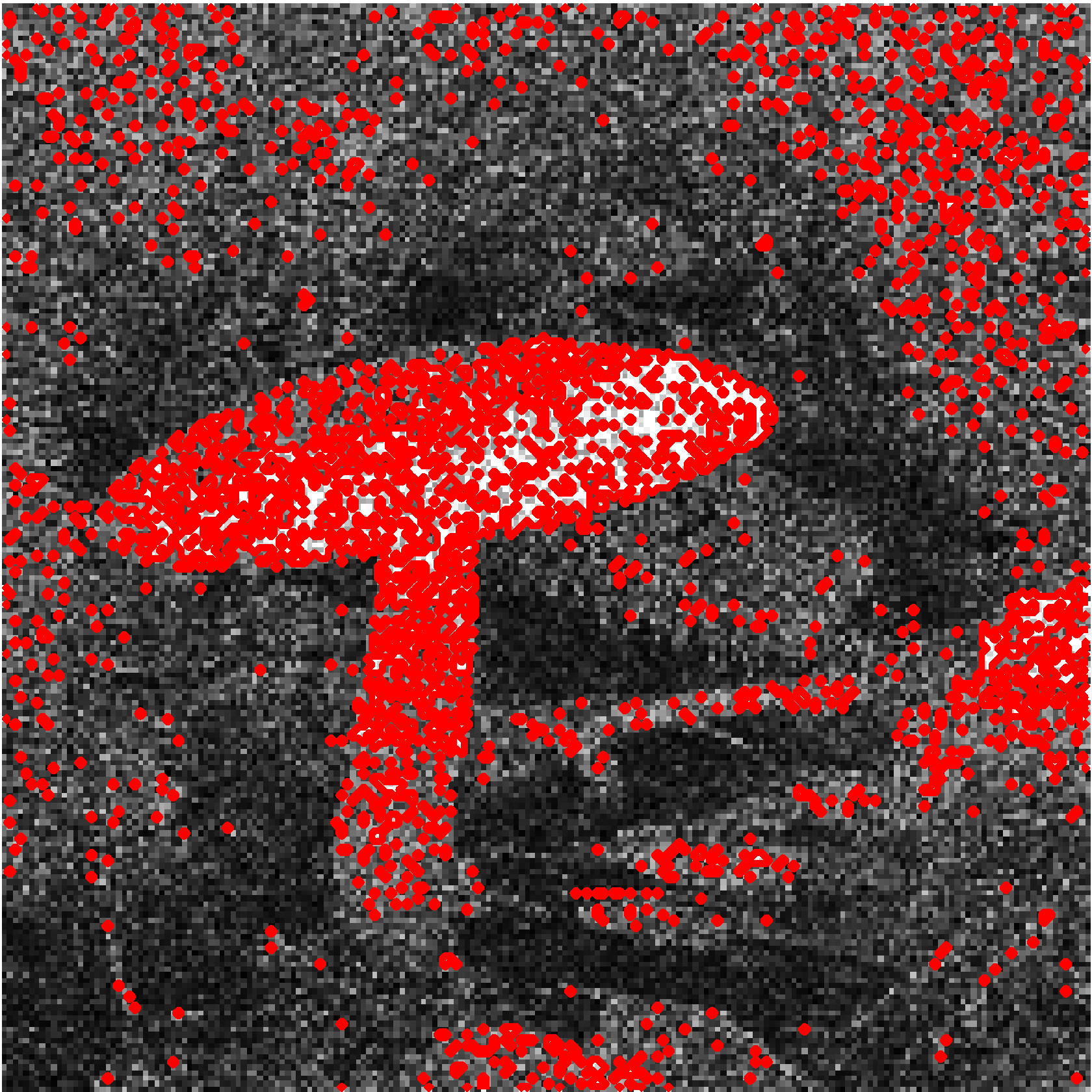}
  \end{subfigure}
  \hfill
  \begin{subfigure}[b]{0.16\linewidth}
      \centering
      \includegraphics[width=\textwidth]{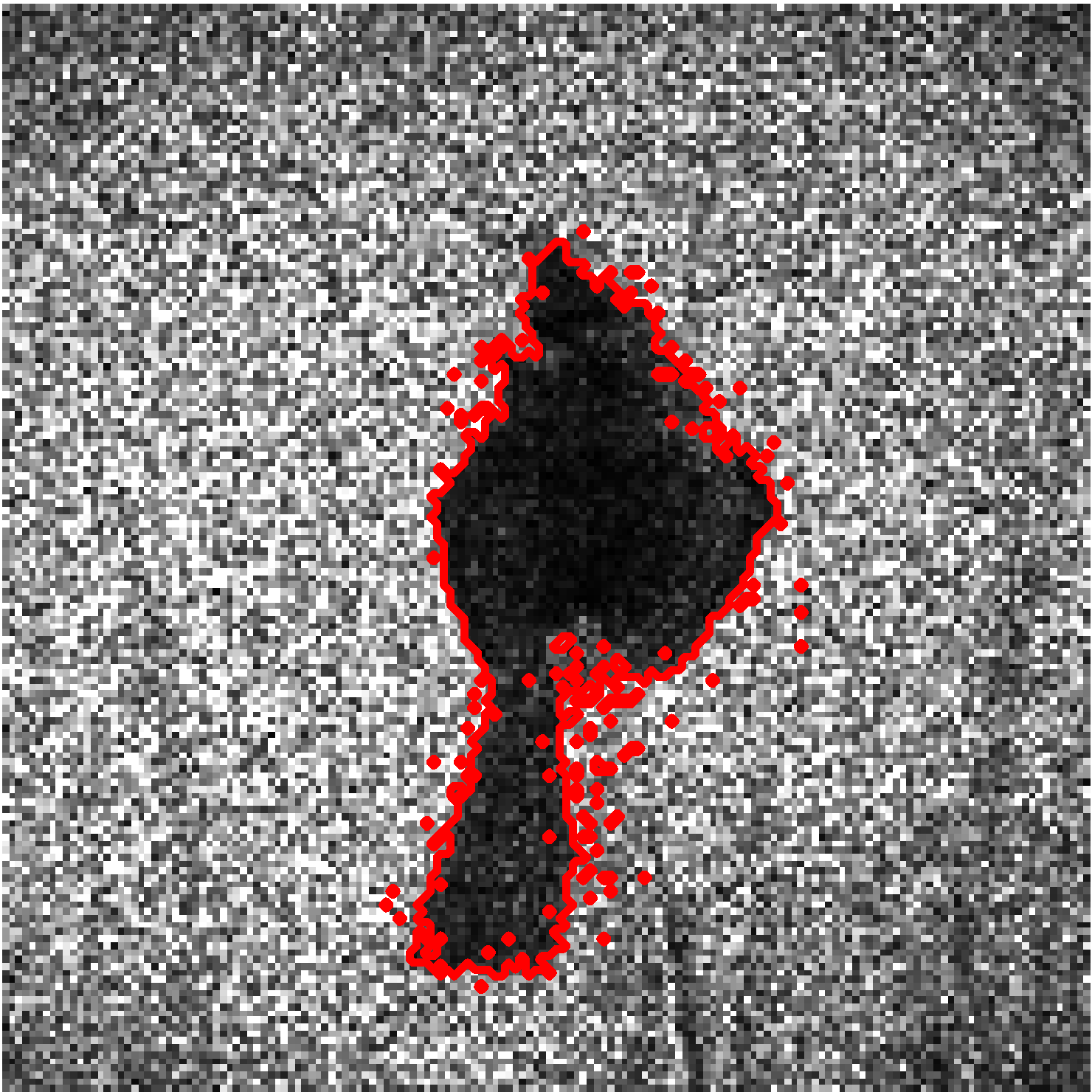}
  \end{subfigure}

  \begin{subfigure}[b]{0.16\linewidth}
      \centering
      \includegraphics[width=\textwidth]{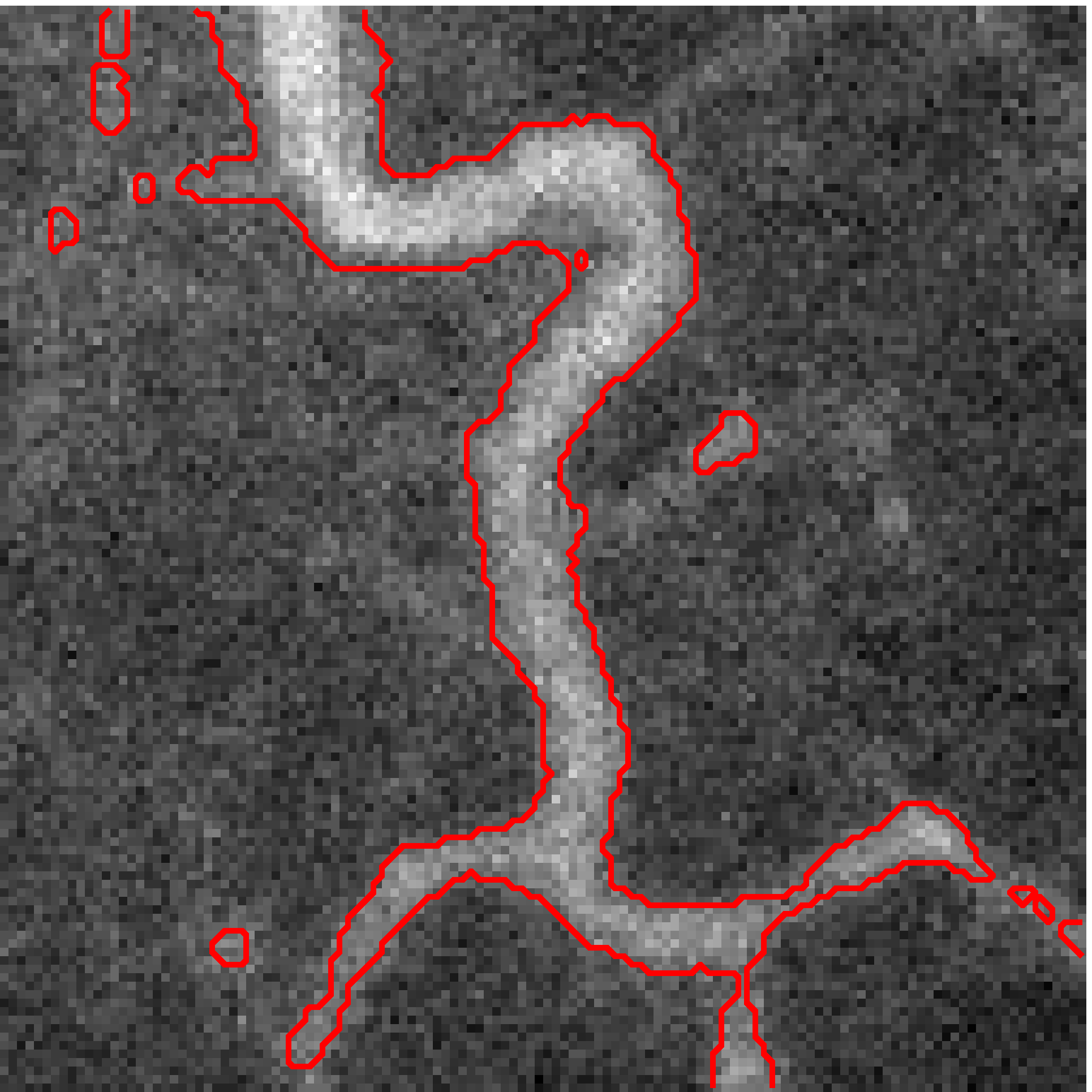}
  \end{subfigure}
  \hfill
  \begin{subfigure}[b]{0.16\linewidth}
      \centering
      \includegraphics[width=\textwidth]{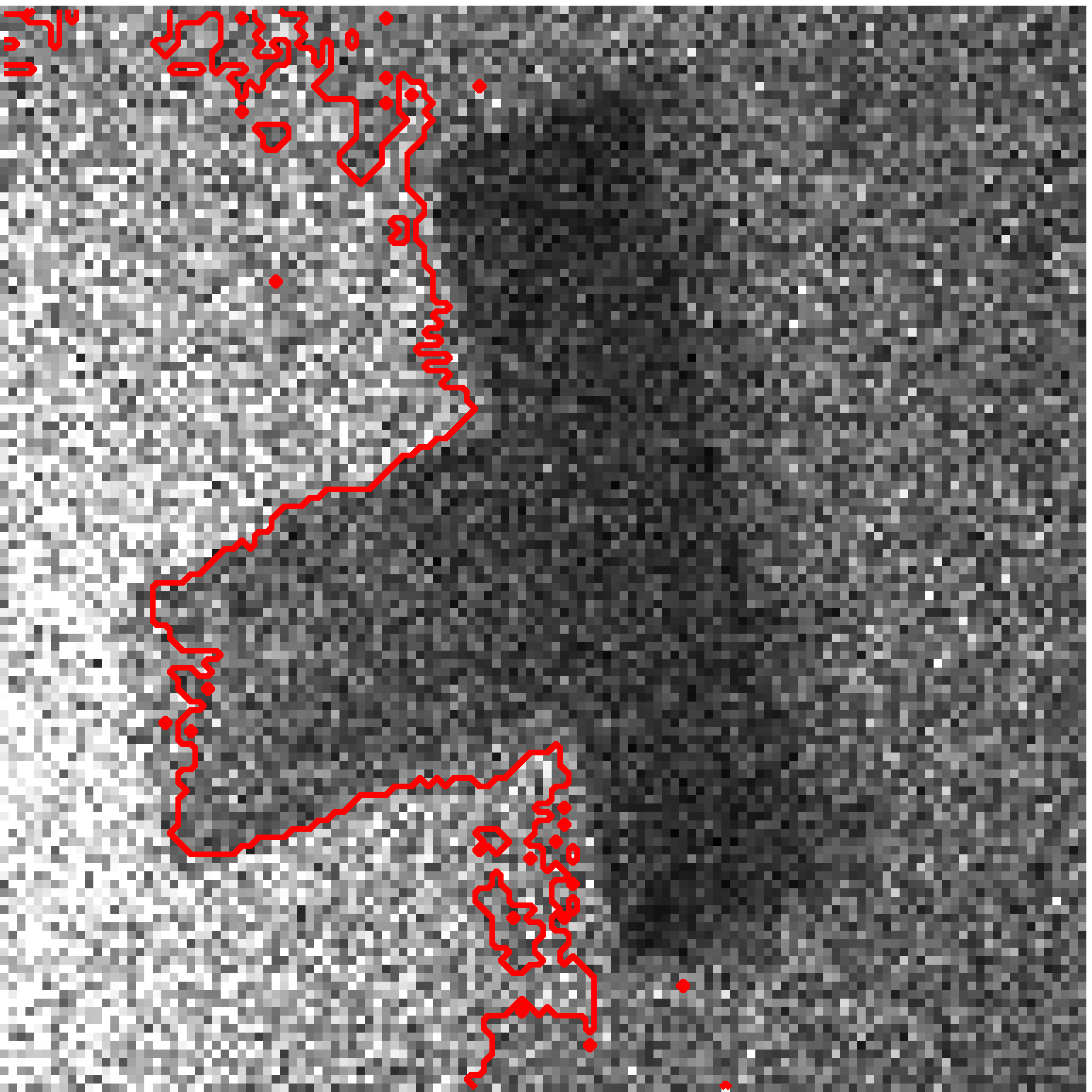}
  \end{subfigure}
  \hfill
  \begin{subfigure}[b]{0.16\linewidth}
      \centering
      \includegraphics[width=\textwidth]{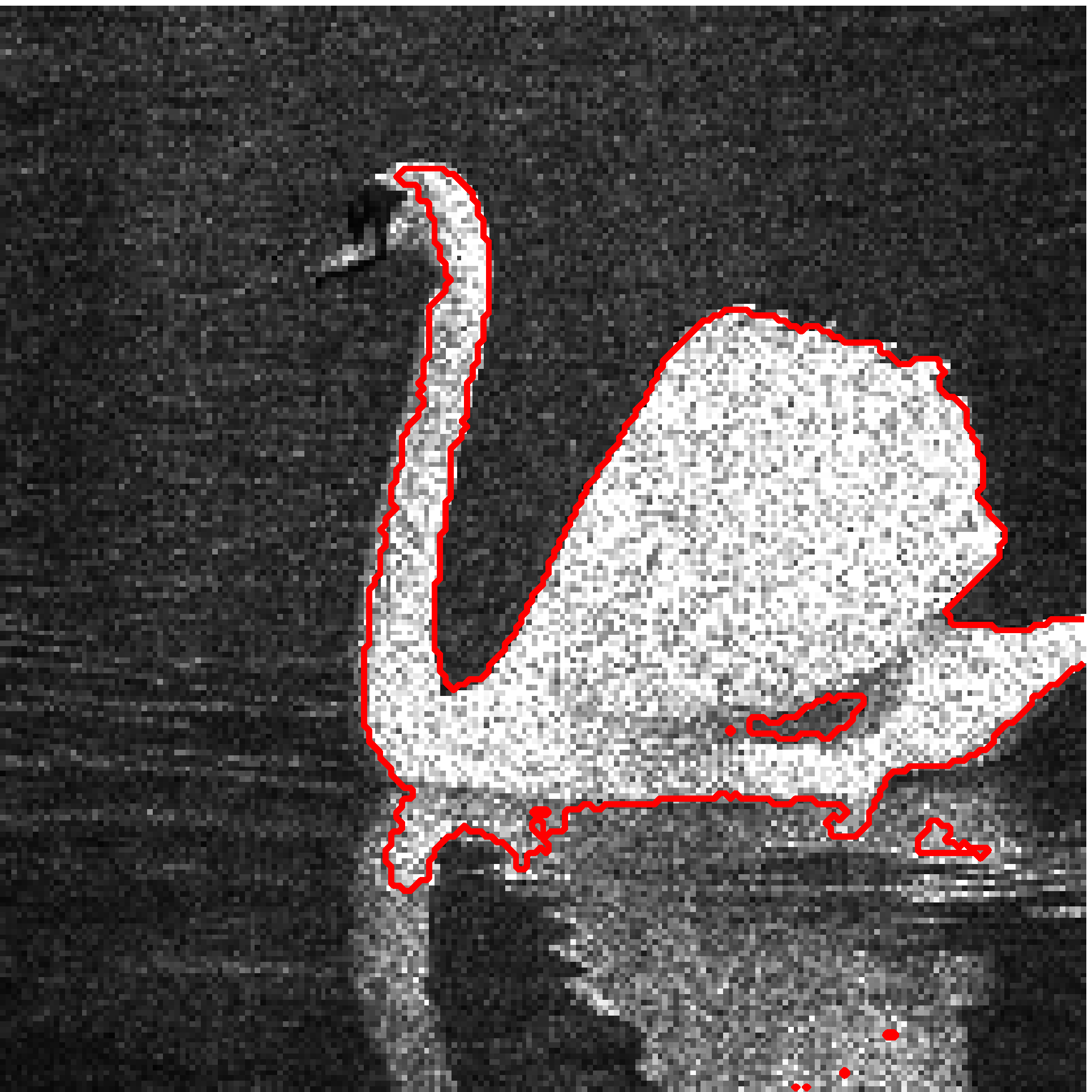}
  \end{subfigure}
  \hfill
  \begin{subfigure}[b]{0.16\linewidth}
      \centering
      \includegraphics[width=\textwidth]{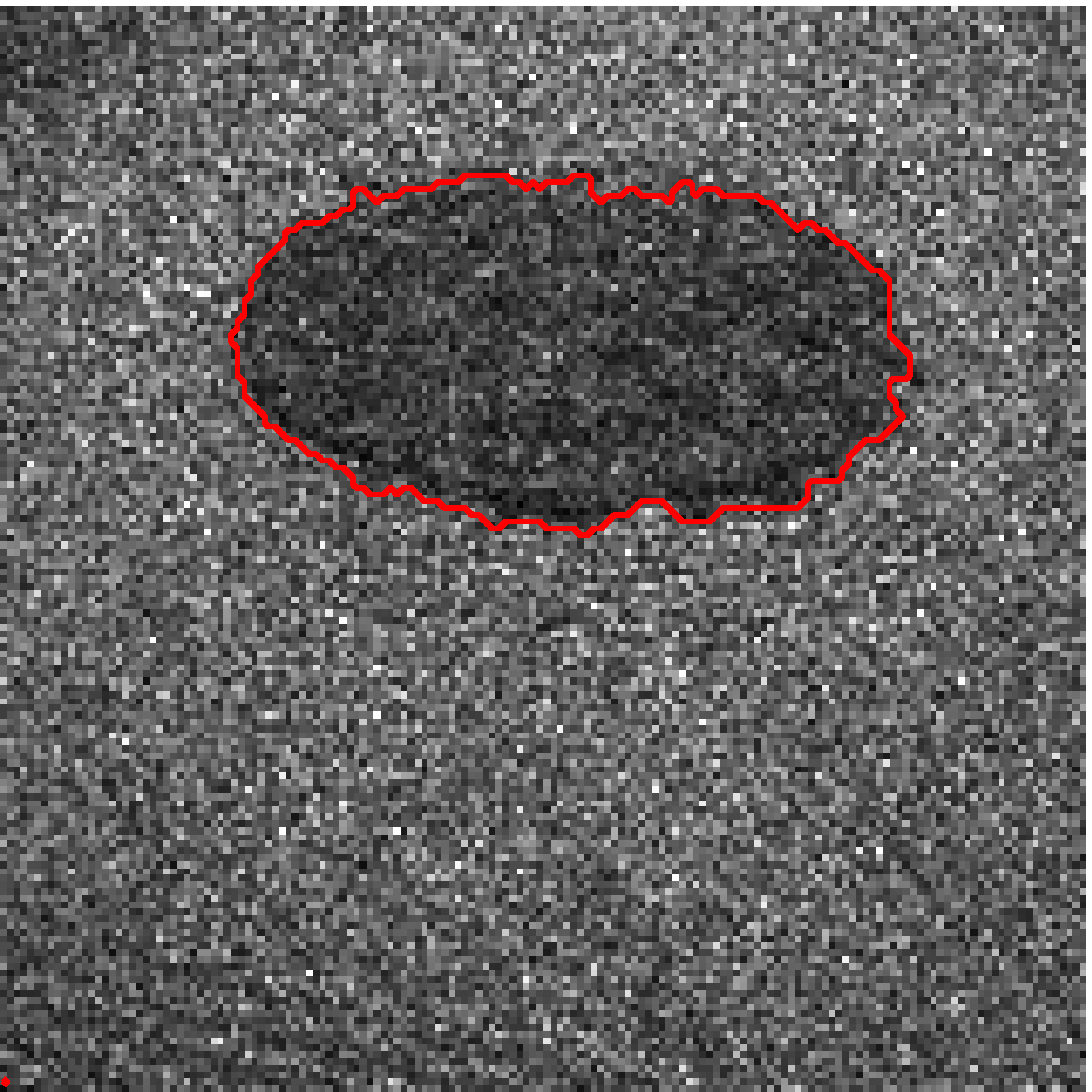}
  \end{subfigure}
  \hfill
  \begin{subfigure}[b]{0.16\linewidth}
      \centering
      \includegraphics[width=\textwidth]{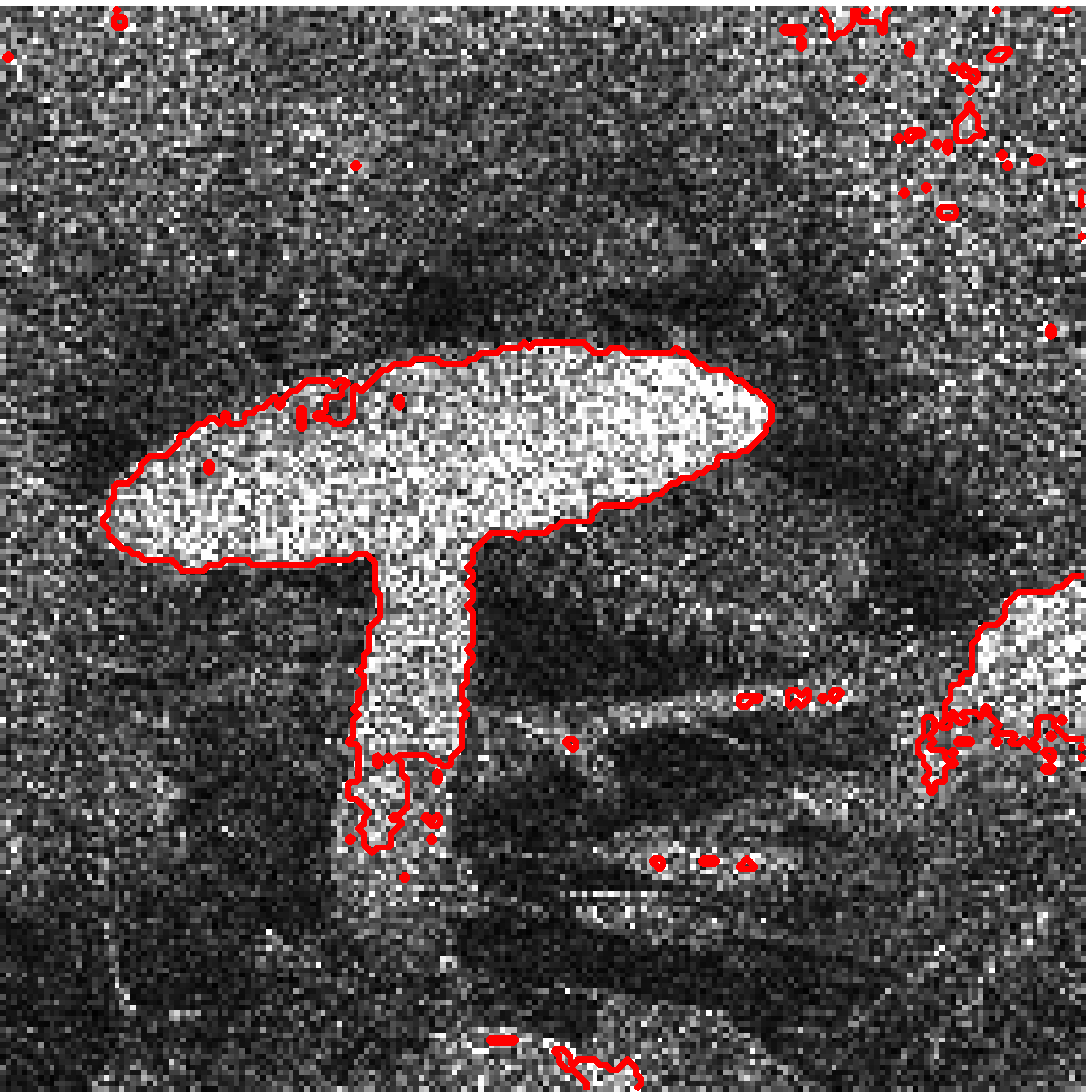}
  \end{subfigure}
  \hfill
  \begin{subfigure}[b]{0.16\linewidth}
      \centering
      \includegraphics[width=\textwidth]{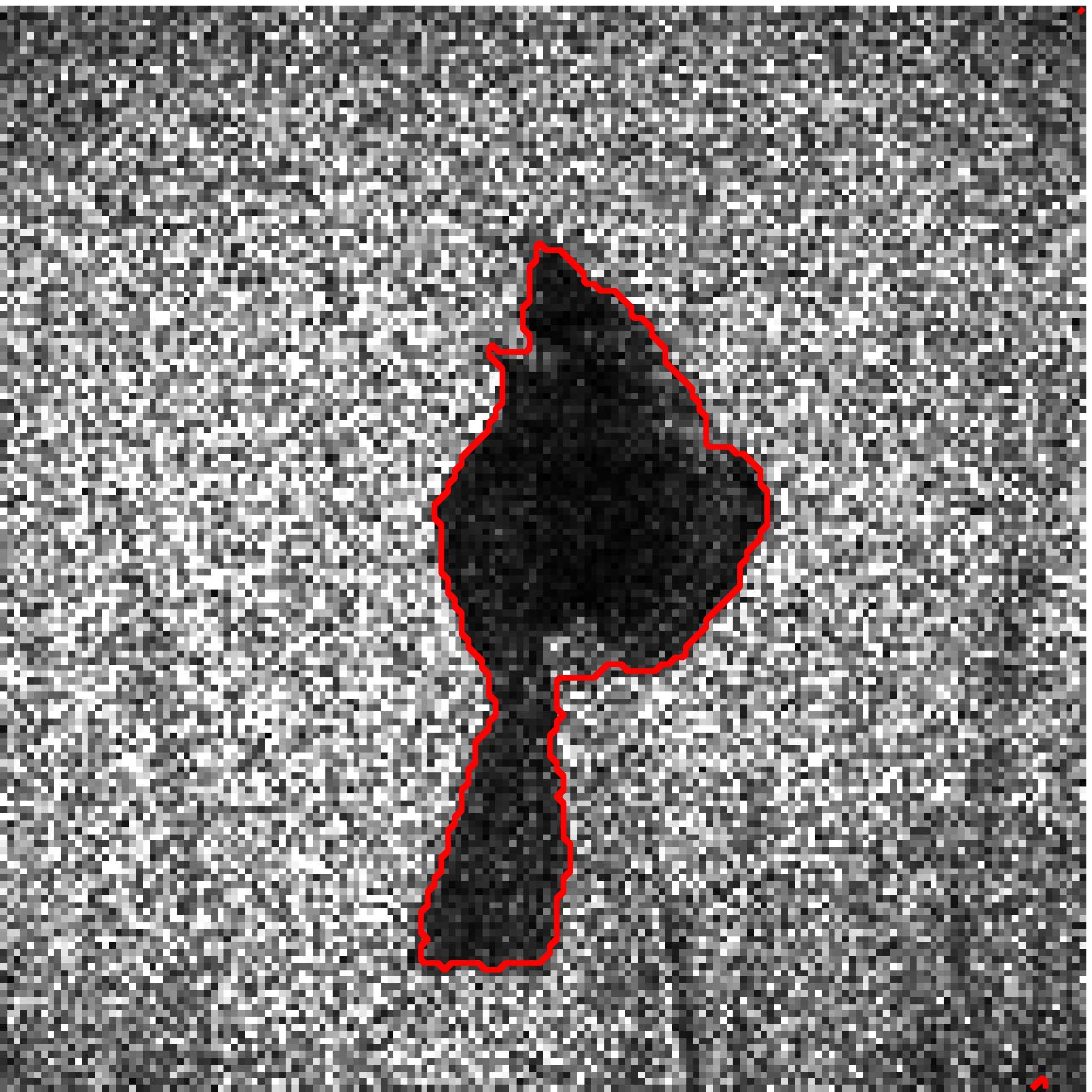}
  \end{subfigure}

  \begin{subfigure}[b]{0.16\linewidth}
      \centering
      \includegraphics[width=\textwidth]{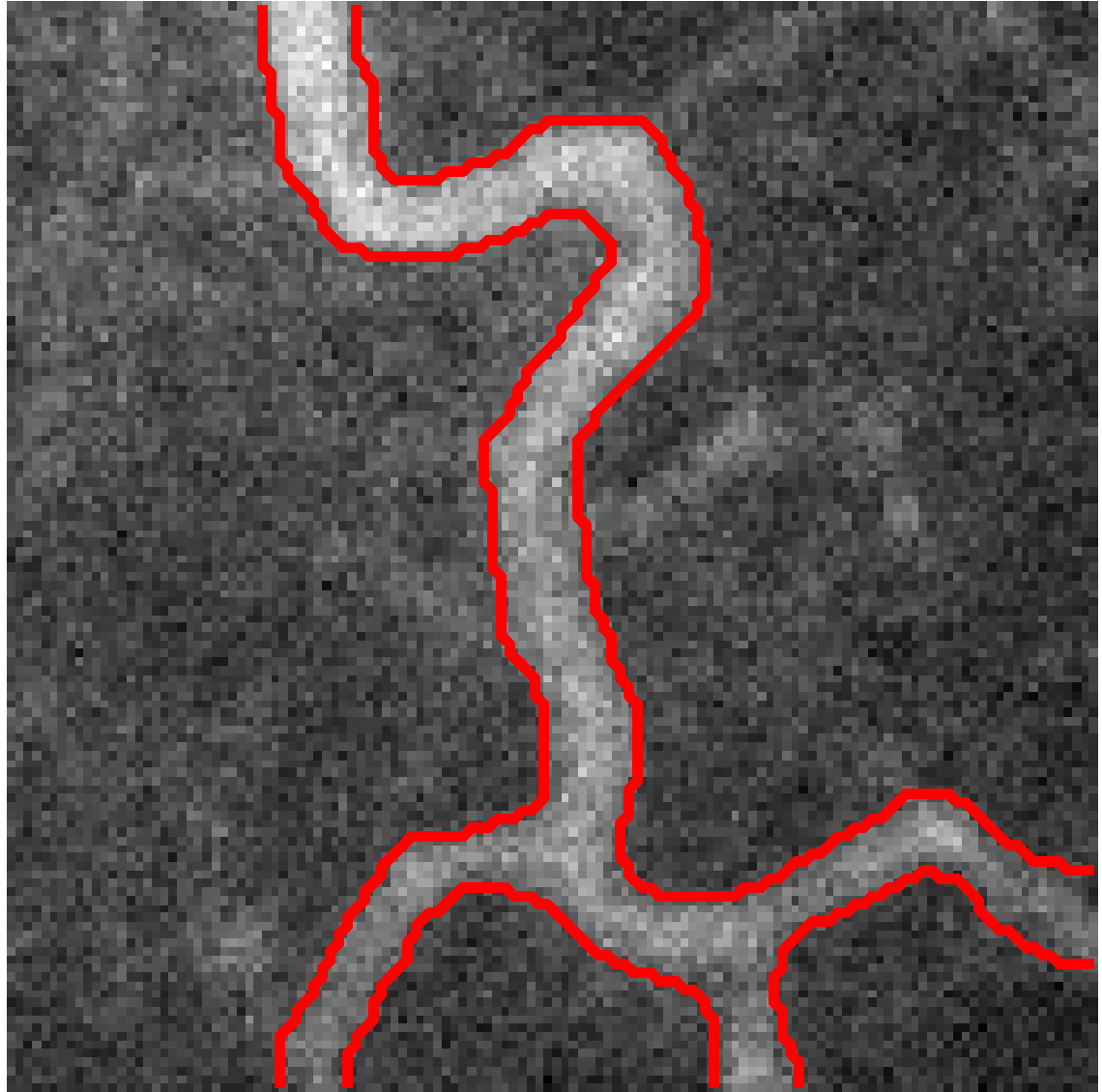}
      \caption{Poisson}
  \end{subfigure}
  \hfill
  \begin{subfigure}[b]{0.16\linewidth}
      \centering
      \includegraphics[width=\textwidth]{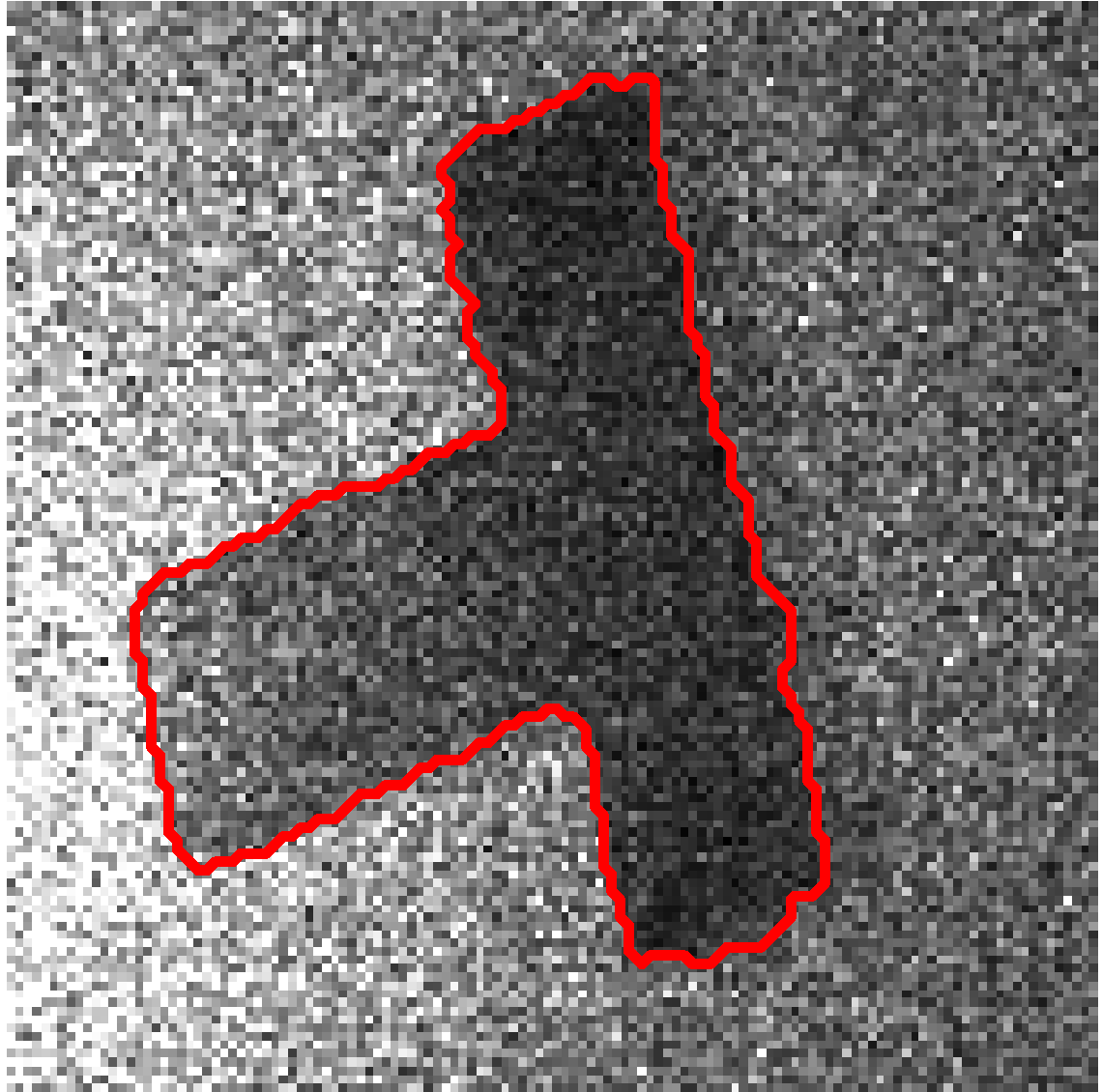}
      \caption{$L=10$}
  \end{subfigure}
  \hfill
  \begin{subfigure}[b]{0.16\linewidth}
      \centering
      \includegraphics[width=\textwidth]{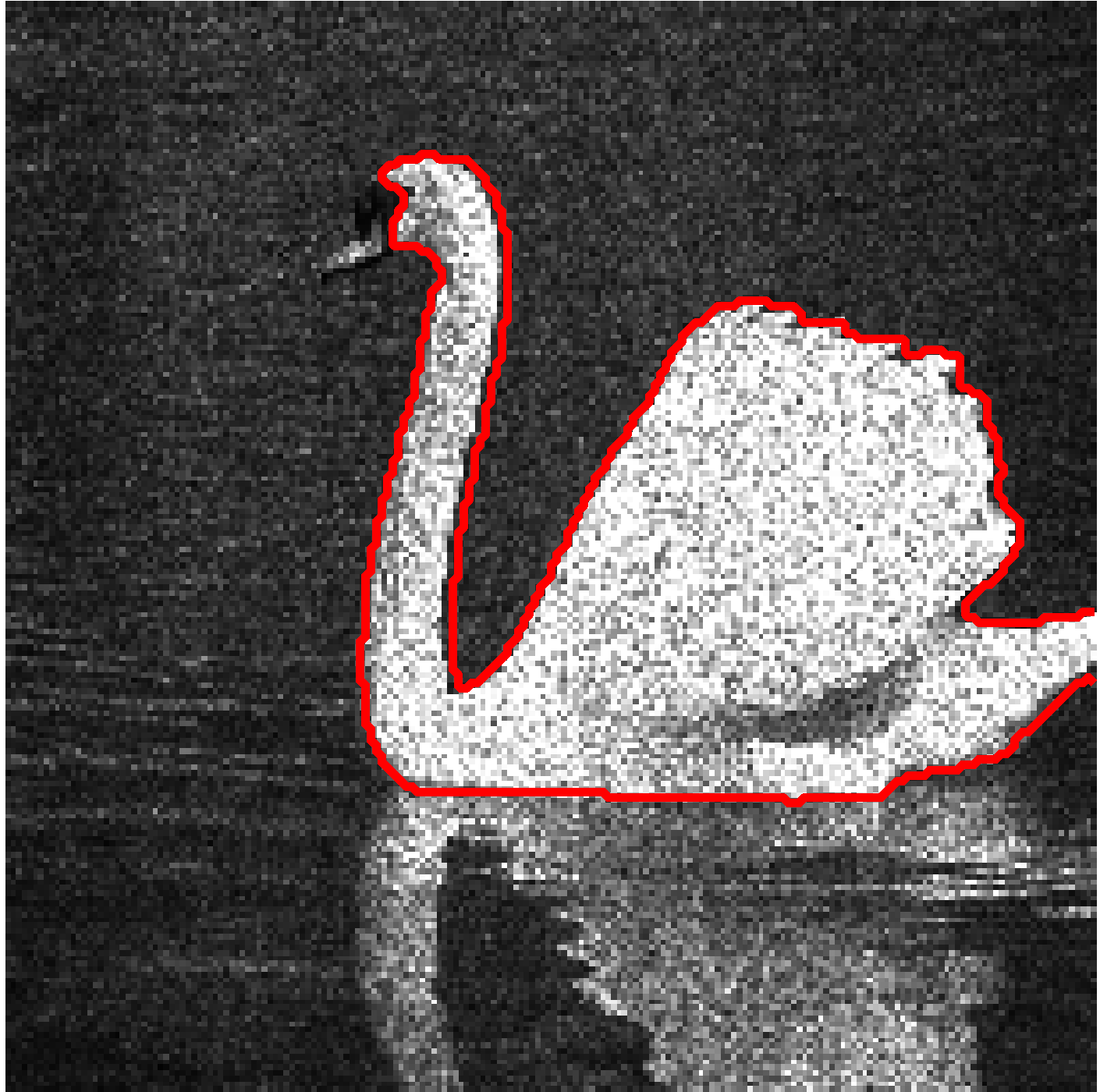}
      \caption{$L=10$}
      \label{fig:16-L10}
  \end{subfigure}
  \hfill
  \begin{subfigure}[b]{0.16\linewidth}
      \centering
      \includegraphics[width=\textwidth]{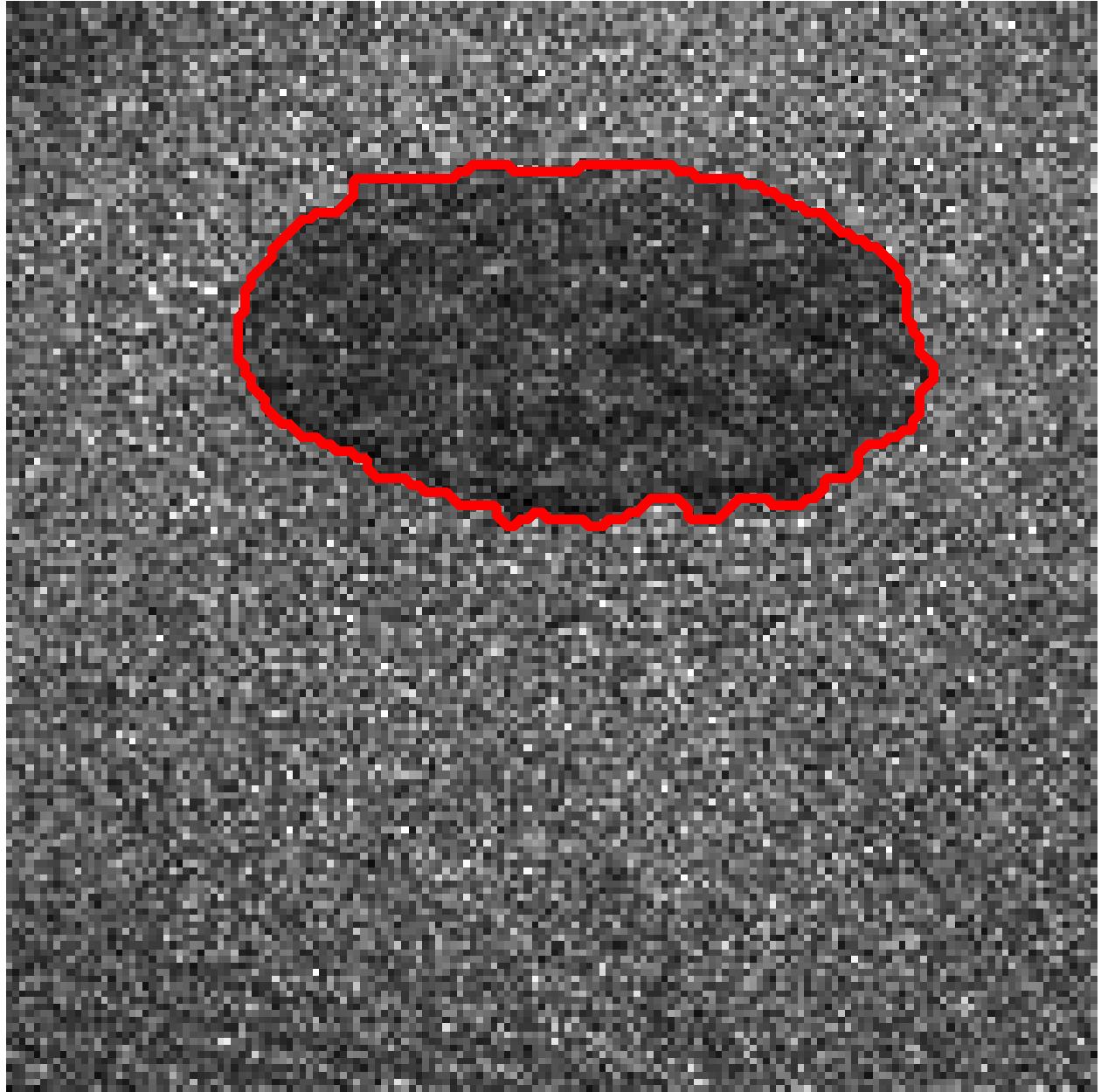}
      \caption{$L=10$}
      \label{fig:9-L10}
  \end{subfigure}
  \hfill
  \begin{subfigure}[b]{0.16\linewidth}
      \centering
      \includegraphics[width=\textwidth]{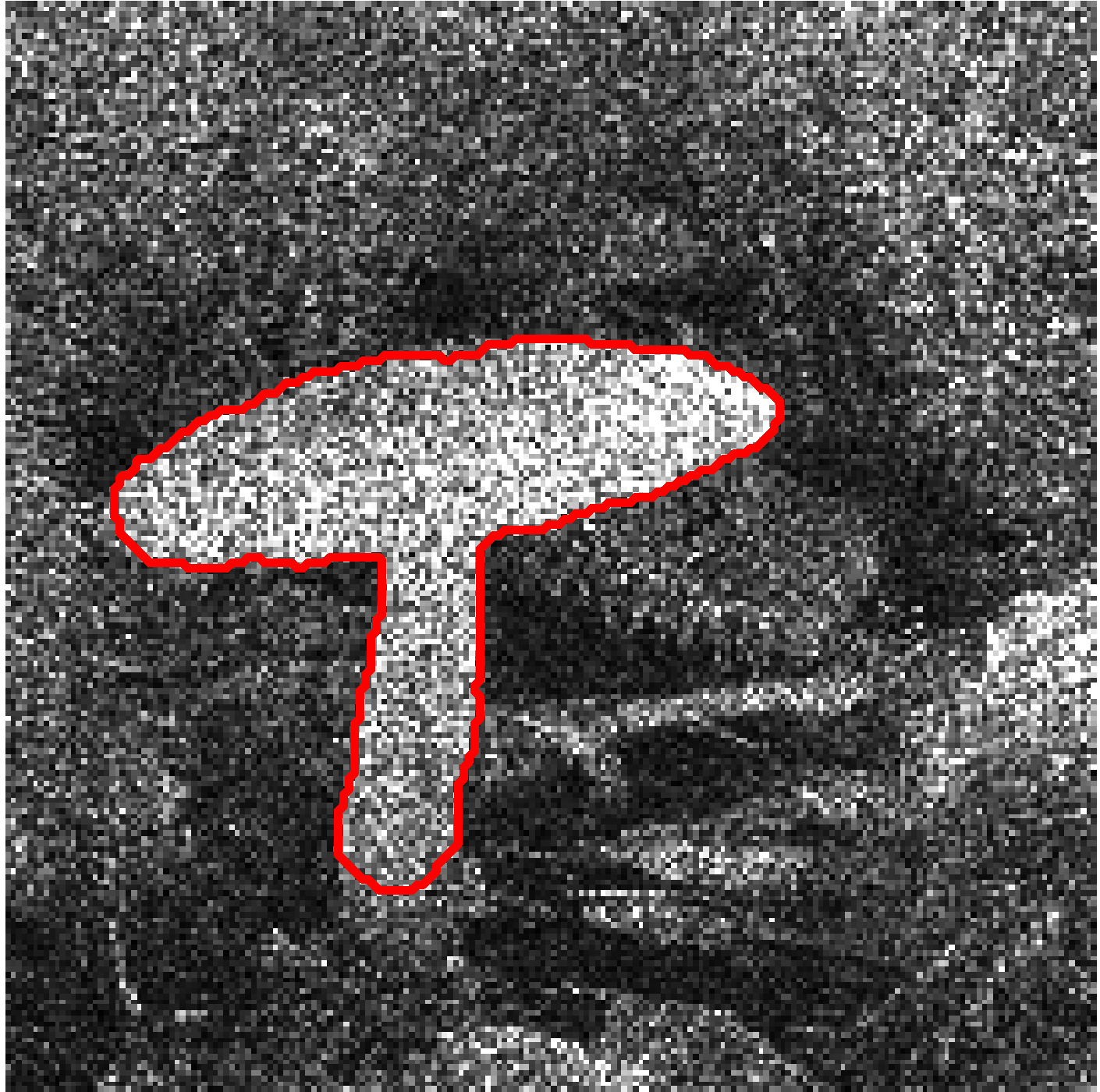}
      \caption{$L=4$}
      \label{fig:12-L4}
  \end{subfigure}
  \hfill
  \begin{subfigure}[b]{0.16\linewidth}
      \centering
      \includegraphics[width=\textwidth]{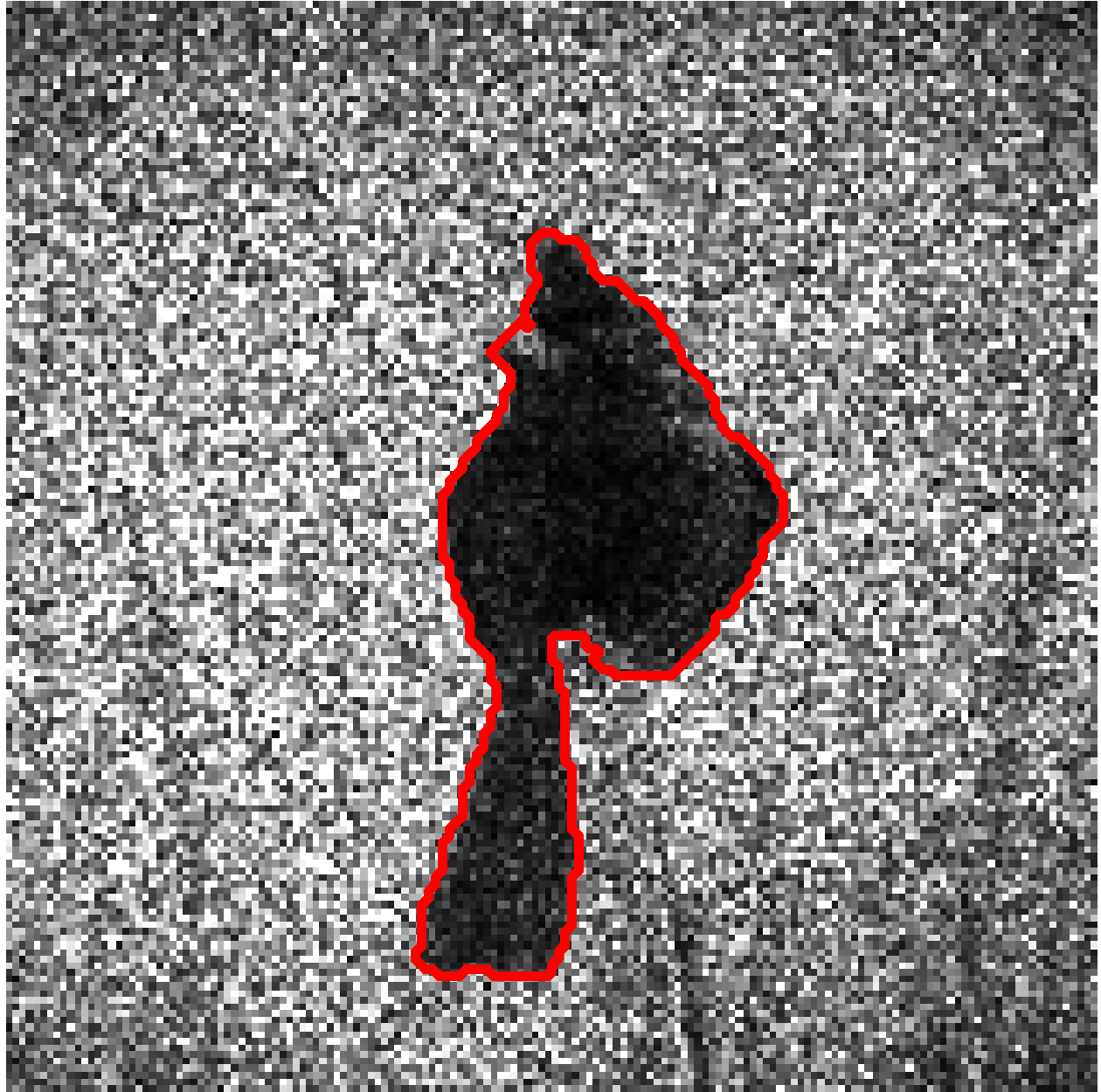}
      \caption{$L=4$}
      \label{fig:20-L4}
  \end{subfigure}

  \caption{Comparison with state-of-the-art models. Row1: input images with initial contours; Rows 2–8: results of LIC, TSS, VLSGIS, NCASTV, ICTM-CV, ICTM-LVF-CV, and our model, respectively}
  \label{fig:comparison}
\end{figure}
\begin{table}
\centering
\caption{DSC, IoU, Accuracy and $\kappa$ values of different models for the segmentation results }
\label{tab:seg_results}
\begin{tabularx}{\textwidth}{XXllll}
  \toprule
 Image & Model & DSC & IoU & Accuracy & $\kappa$\\
  \midrule
  \multirow[t]{7}{*}{Fig.\ref{fig:9-L10}}
   & LIC & 0.3150 & 0.1869 & 0.4497 & 0.0782 \\
   & TSS & 0.8674 & 0.7659 & 0.9548 & 0.8403 \\
   & VLSGIS & 0.9533 & 0.9108 & 0.9845 & 0.9441 \\
   & NCASTV & 0.4511 & 0.2913 & 0.6378  & 0.2750 \\
   & ICTM-CV & 0.9219 & 0.8552 & 0.9759 & 0.9077 \\
   & ICTM-LVF-CV & 0.8639 & 0.7605 & 0.9504 & 0.8342 \\
   & Our & \textbf{0.9677} & \textbf{0.9374} & \textbf{0.9895} & \textbf{0.9614} \\
   \multirow[t]{7}{*}{Fig.~\ref{fig:20-L4}}
   & LIC & 0.9389 & 0.8849 & 0.9862 & 0.9311 \\
   & TSS & 0.8993 & 0.8170 & 0.9781 & 0.8871 \\
   & VLSGIS & 0.8673 & 0.7656 & 0.9652 & 0.8475 \\
   & NCASTV & 0.7447 & 0.5933 & 0.9236 & 0.7027 \\
   & ICTM-CV & 0.9135 & 0.8407 & 0.9805 & 0.9025 \\
   & ICTM-LVF-CV & 0.8175 & 0.6913 & 0.9547 & 0.7925\\
   & Our & \textbf{0.9544} & \textbf{0.9127} & \textbf{0.9898} & \textbf{0.9486} \\
   \multirow[t]{6}{*}{Fig.~\ref{fig:u-8}}
   & LIC & 0.9526 & 0.9095 & 0.9990 & 0.9521 \\
   & TSS & 0.0352 & 0.0179 & 0.5522& 0.0138 \\
   & VLSGIS & 0.9447 & 0.8951 & 0.9988 & 0.9440 \\
   & ICTM-CV & 0.0011 & 0.0005 & 0.9654 & -0.0142 \\
   & ICTM-LVF-CV & 0.0360 & 0.0184 & 0.6076 & 0.0147 \\
   & Our & \textbf{0.9702} & \textbf{0.9421} & \textbf{0.9993} & \textbf{0.9698} \\
   \multirow[t]{6}{*}{Fig.~\ref{fig:u-72}}
   & LIC & 0.9498 & 0.9043 & 0.9917 & 0.9452 \\
   & TSS & 0.2575 & 0.1478 & 0.5890 & 0.1399 \\
   & VLSGIS & 0.9489 & 0.9027 & 0.9917 & 0.9444 \\
   & ICTM-CV & 0.9036 & 0.8241 & 0.9856 & 0.8958\\
   & ICTM-LVF-CV & 0.2355 & 0.1335 & 0.5221 & 0.1118 \\
   & Our & \textbf{0.9562} & \textbf{0.9161} & \textbf{0.9930} & \textbf{0.9524} \\
   \multirow[t]{6}{*}{Fig.~\ref{fig:u-87}}
   & LIC & 0.1267 & 0.0676 & 0.7205 & 0.0922 \\
   & TSS & 0.3041 & 0.1793 & 0.9217 & 0.2802 \\
   & VLSGIS & 0.9328 & 0.8749 & 0.9971 & 0.9313 \\
   & ICTM-CV & 0.9582 & 0.9197 & 0.9984 & 0.9573\\
   & ICTM-LVF-CV & 0.0654 & 0.0338 & 0.4212 & 0.0273 \\
   & Our & \textbf{0.9632} & \textbf{0.9290} & \textbf{0.9985} & \textbf{0.9624} \\
   \multirow[t]{6}{*}{Fig.~\ref{fig:u-107}}
   & LIC & 0.1868 & 0.1030 & 0.8266 & 0.1548 \\
   & TSS & 0.0572 & 0.0294 & 0.5278 & 0.0176 \\
   & VLSGIS & 0.9162 & 0.8454 & 0.9965 & 0.9144 \\
   & ICTM-CV & 0.0633 & 0.0327 & 0.4581 & 0.0238 \\
   & ICTM-LVF-CV & 0.0648 & 0.0335 & 0.4670 & 0.0254 \\
   & Our & \textbf{0.9338} & \textbf{0.8758} & \textbf{0.9972} & \textbf{0.9324} \\
   \multirow[t]{6}{*}{Fig.~\ref{fig:u-186}}
   & LIC & 0.5716 & 0.4002 & 0.8955 & 0.5206 \\
   & TSS & 0.6612 & 0.4939 & 0.9357 & 0.6268 \\
   & VLSGIS & 0.9494 & 0.9038 & 0.9924 & 0.9454 \\
   & ICTM-CV & 0.9264 & 0.8630 & 0.9893 & 0.9207 \\
   & ICTM-LVF-CV & 0.1120 & 0.0593 & 0.6666 & 0.0753 \\
   & Our & \textbf{0.9646} & \textbf{0.9315} & \textbf{0.9946} & \textbf{0.9616} \\
  \bottomrule
\end{tabularx}
\end{table}
Comparative experiments are conducted with several state-of-the-art methods, and the corresponding results are presented in Fig.~\ref{fig:comparison}. Real images exhibiting intensity inhomogeneity, selected from the BSD500 dataset \cite{arbelaez2010contour}, are used in these experiments, as illustrated in Figs.~\ref{fig:16-L10}--\ref{fig:20-L4}. To evaluate the robustness of the models against noise, Poisson noise and multiplicative Gamma noise are added to the original images. In Table~\ref{tab:seg_results}, we present quantitative metrics, including the DSC, IoU, Accuracy and $\kappa$ values of the segmentation results. Both the LIC and ICTM-CV models contain a length term but lack an additional denoising term. Although the length term enhances the robustness to noise, the segmentation contours produced by LIC and ICTM-CV exhibit irregular boundaries and isolated points, especially for multiplicative Gamma noise. In contrast, models that incorporate explicit denoising terms show a significant reduction in such irregular boundaries. This also demonstrates the importance of incorporating denoising terms into our model.
The NCASTV model is based on normalized cuts and uses the TV regularizer to enhance the robustness to noise. As illustrated in Fig.~\ref{fig:comparison}, the segmentation results of the NCASTV model fail to capture fine details. Moreover, this model is ineffective for segmenting images with severe intensity inhomogeneity.
The TSS model, a two-stage method, yields inaccurate results when segmenting images with intensity inhomogeneity, as the target and background regions often have similar intensity values, making the thresholding method unreliable.
The VLSGIS, ICTM-LVF-CV, and the proposed model all incorporate denoising terms into the segmentation framework. Consequently, the segmentation results of these three models are less sensitive to noise. However, the ICTM-LVF-CV model does not perform well on images with severe intensity inhomogeneity.
The VLSGIS model produces accurate segmentation results for both noisy and intensity-inhomogeneous images, but tends to miss fine object details. Compared with state-of-the-art models, our model better preserves object details. Our model also demonstrates greater robustness to intensity inhomogeneity owing to the incorporated bias correction. As shown in Table~\ref{tab:seg_results}, our model also achieves accurate segmentation results under high noise levels.

\subsection{Application on Medical Image Segmentation}
In this section, we conduct segmentation experiments on medical images, which are often affected by intensity inhomogeneity and high noise. We select ultrasound images from the Breast Ultrasound Images Dataset (BUSI)\footnote{\url{https://www.kaggle.com/datasets/sabahesaraki/breast-ultrasound-images-dataset}}, which exhibit severe intensity inhomogeneity in the background and are contaminated by complex noise, making accurate segmentation challenging. Additionally, we choose brain MR images from the Simulated Brain Database (SBD)\footnote{\url{http://www.bic.mni.mcgill.ca/brainweb/}} for multi-phase segmentation experiments. The dataset consists of synthetic brain MR images corrupted by $9\%$ noise and $40\%$ intensity inhomogeneity, where the gray matter (GM) and white matter (WM) have similar intensity values.
\subsubsection{Two-Phase Segmentation for Ultrasound Images}
\begin{figure}
  \centering
  \begin{subfigure}[b]{0.19\linewidth}
      \centering
      \includegraphics[width=\textwidth]{u-8-original.png}
  \end{subfigure}
  \hfill
  \begin{subfigure}[b]{0.19\linewidth}
      \centering
      \includegraphics[width=\textwidth]{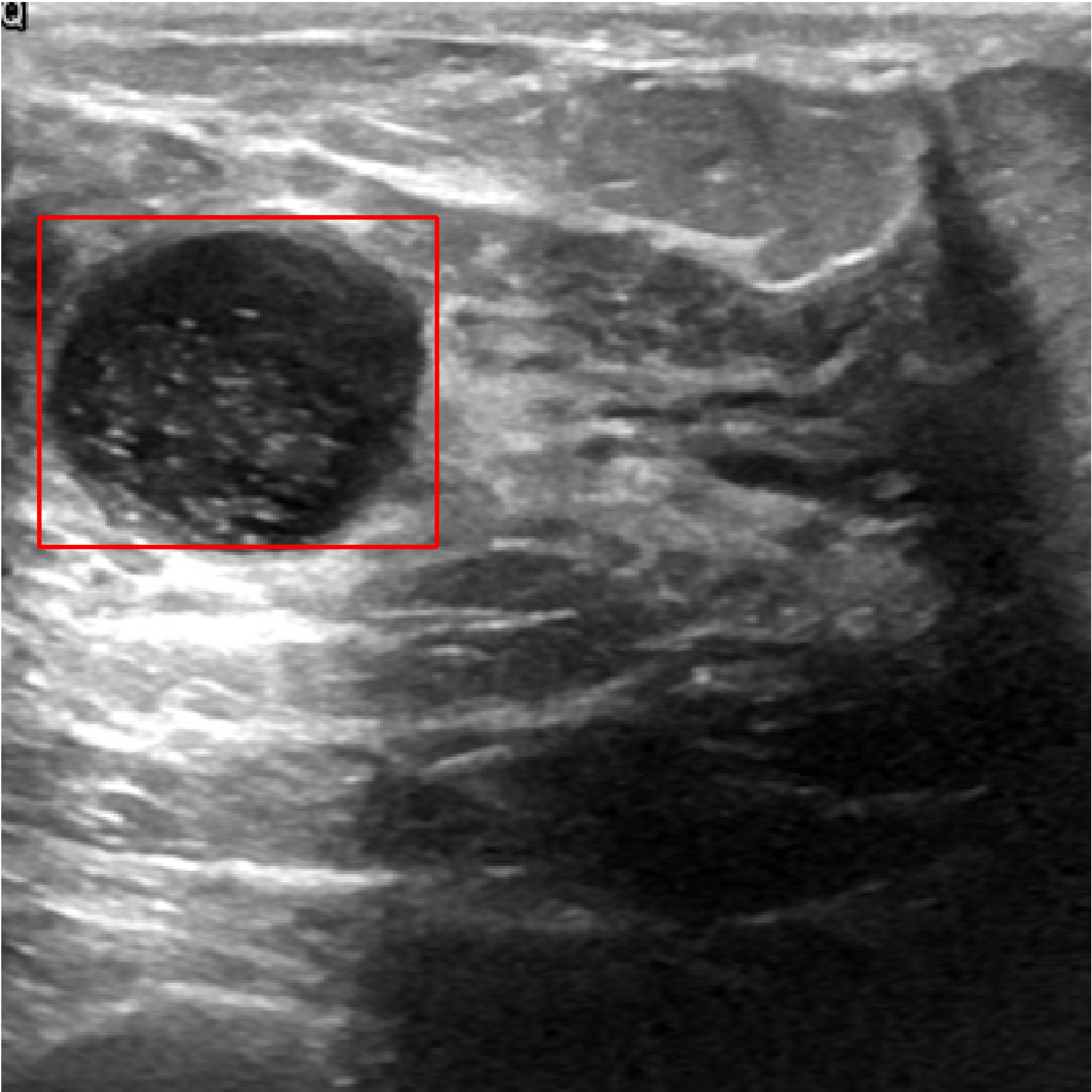}
  \end{subfigure}
  \hfill
  \begin{subfigure}[b]{0.19\linewidth}
      \centering
      \includegraphics[width=\textwidth]{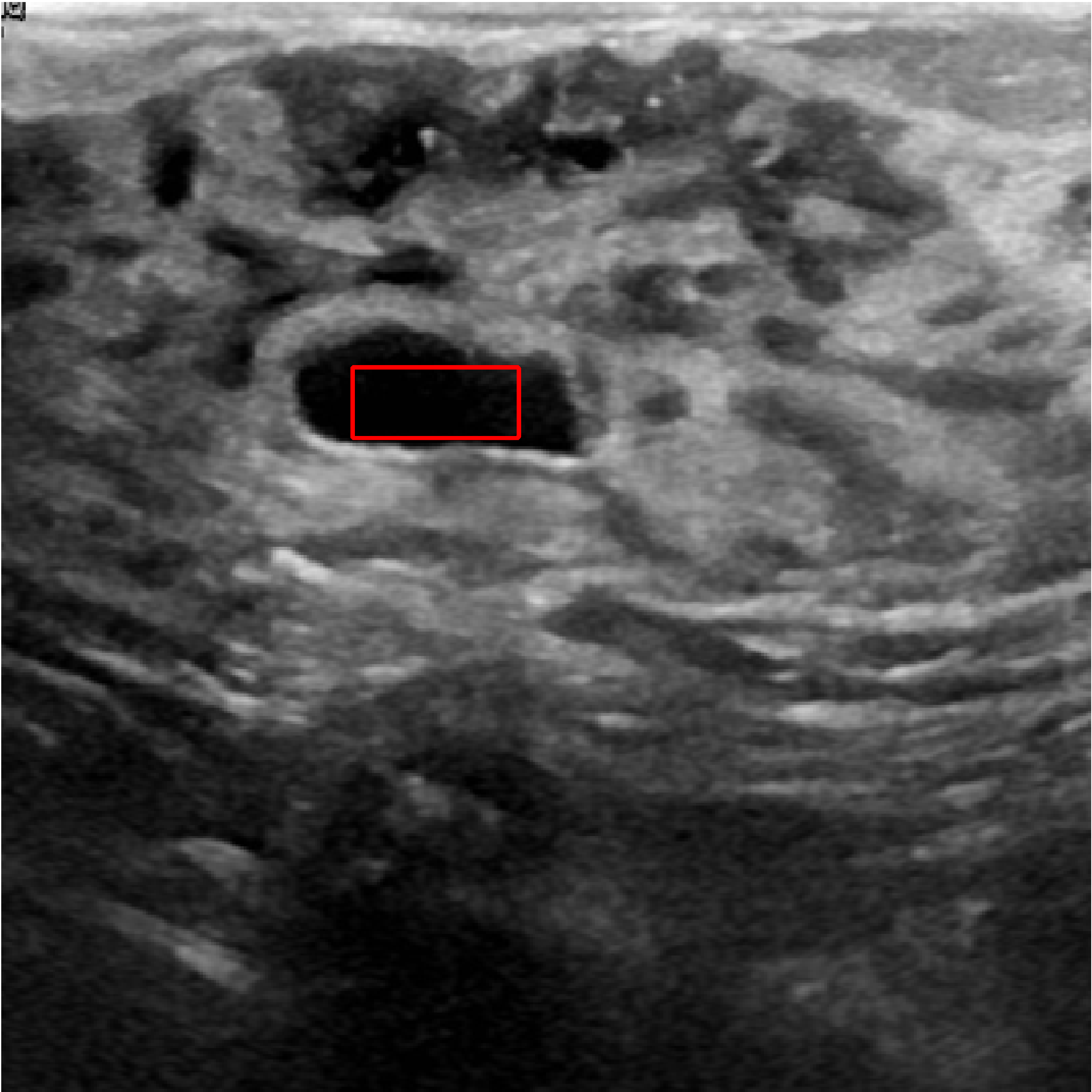}
  \end{subfigure}
  \hfill
  \begin{subfigure}[b]{0.19\linewidth}
      \centering
      \includegraphics[width=\textwidth]{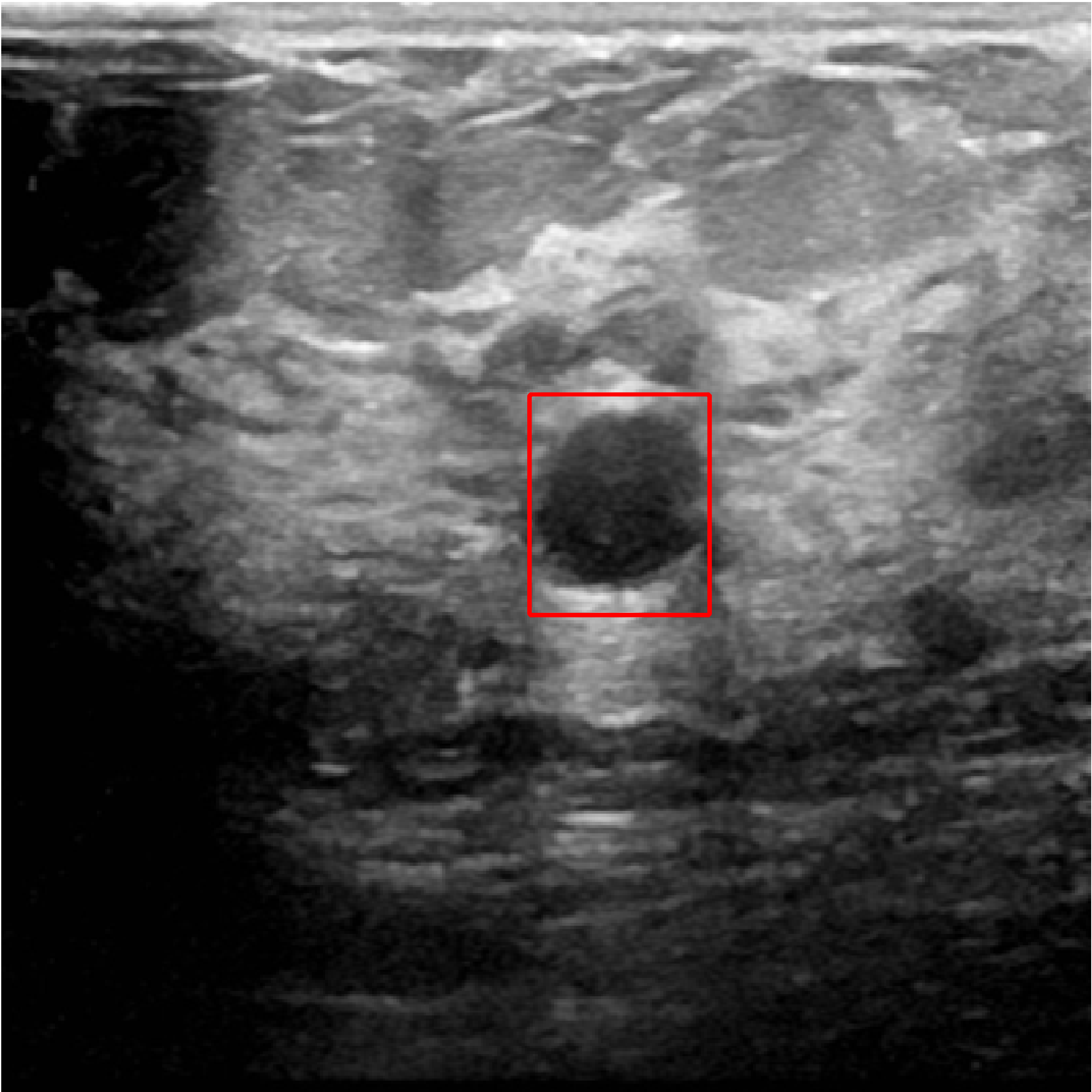}
  \end{subfigure}
  \hfill
  \begin{subfigure}[b]{0.19\linewidth}
      \centering
      \includegraphics[width=\textwidth]{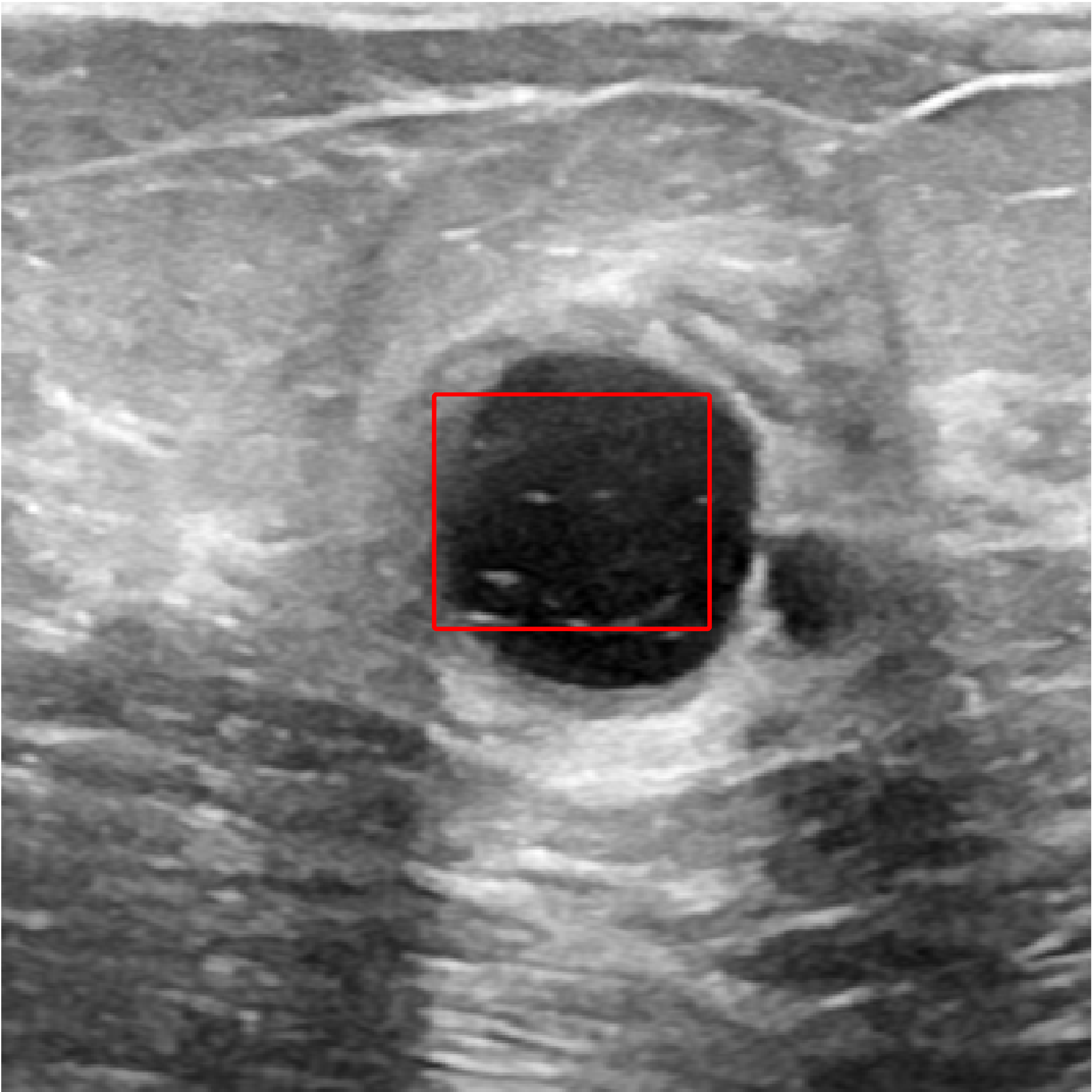}
  \end{subfigure}

  \begin{subfigure}[b]{0.19\linewidth}
      \centering
      \includegraphics[width=\textwidth]{u-8-LIC-seg-contour.png}
  \end{subfigure}
  \hfill
  \begin{subfigure}[b]{0.19\linewidth}
      \centering
      \includegraphics[width=\textwidth]{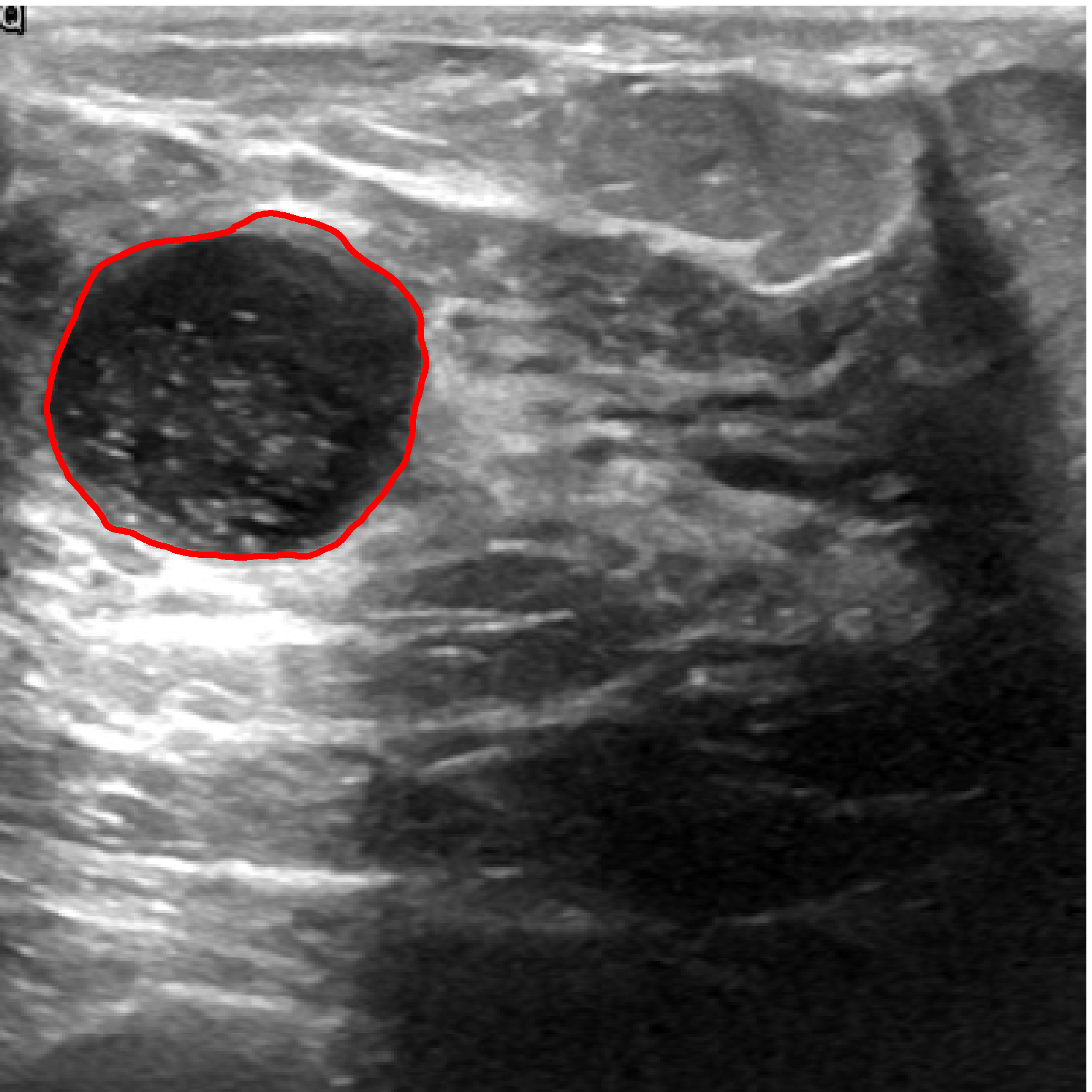}
  \end{subfigure}
  \hfill
  \begin{subfigure}[b]{0.19\linewidth}
      \centering
      \includegraphics[width=\textwidth]{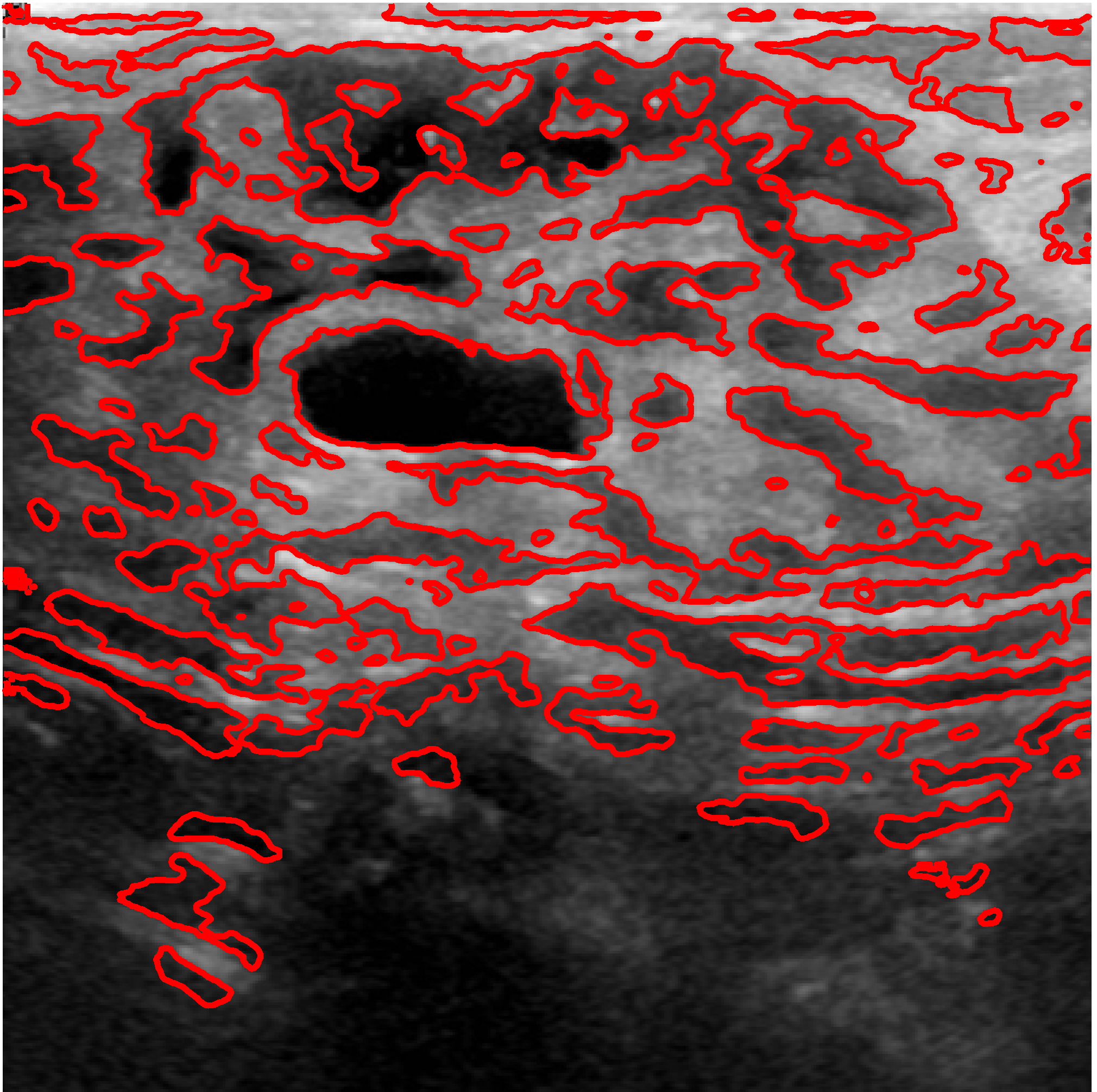}
  \end{subfigure}
  \hfill
  \begin{subfigure}[b]{0.19\linewidth}
      \centering
      \includegraphics[width=\textwidth]{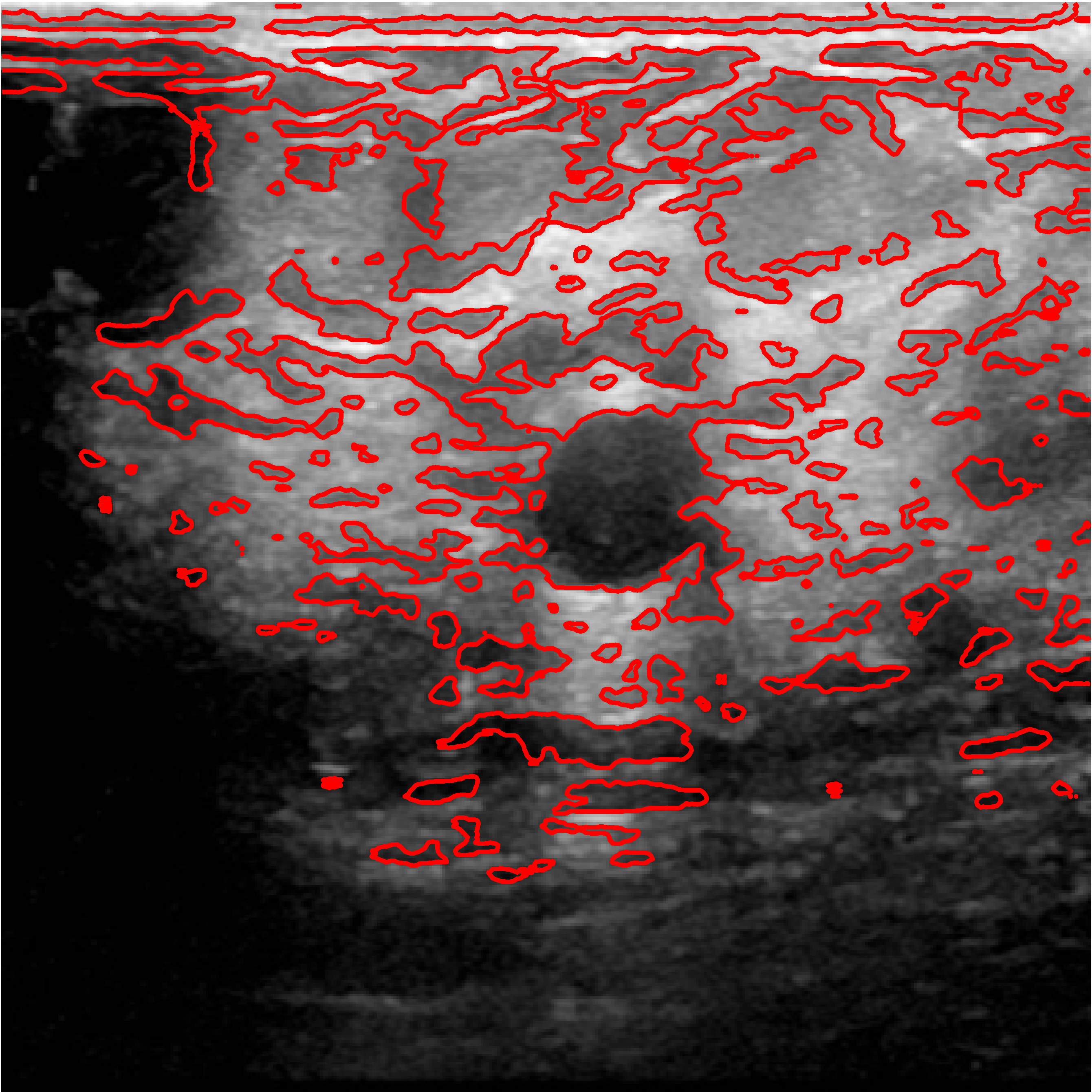}
  \end{subfigure}
  \hfill
  \begin{subfigure}[b]{0.19\linewidth}
      \centering
      \includegraphics[width=\textwidth]{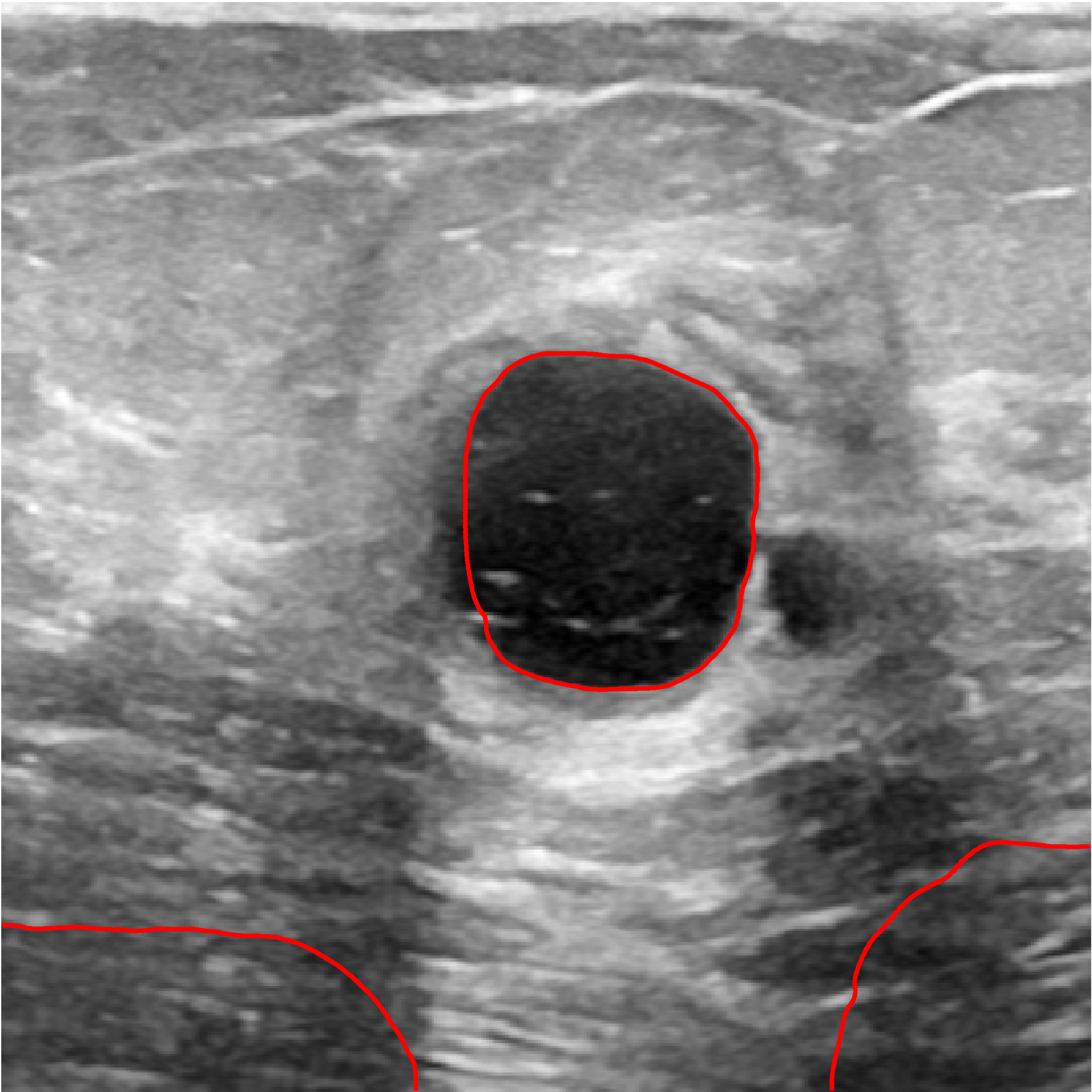}
  \end{subfigure}

  \begin{subfigure}[b]{0.19\linewidth}
      \centering
      \includegraphics[width=\textwidth]{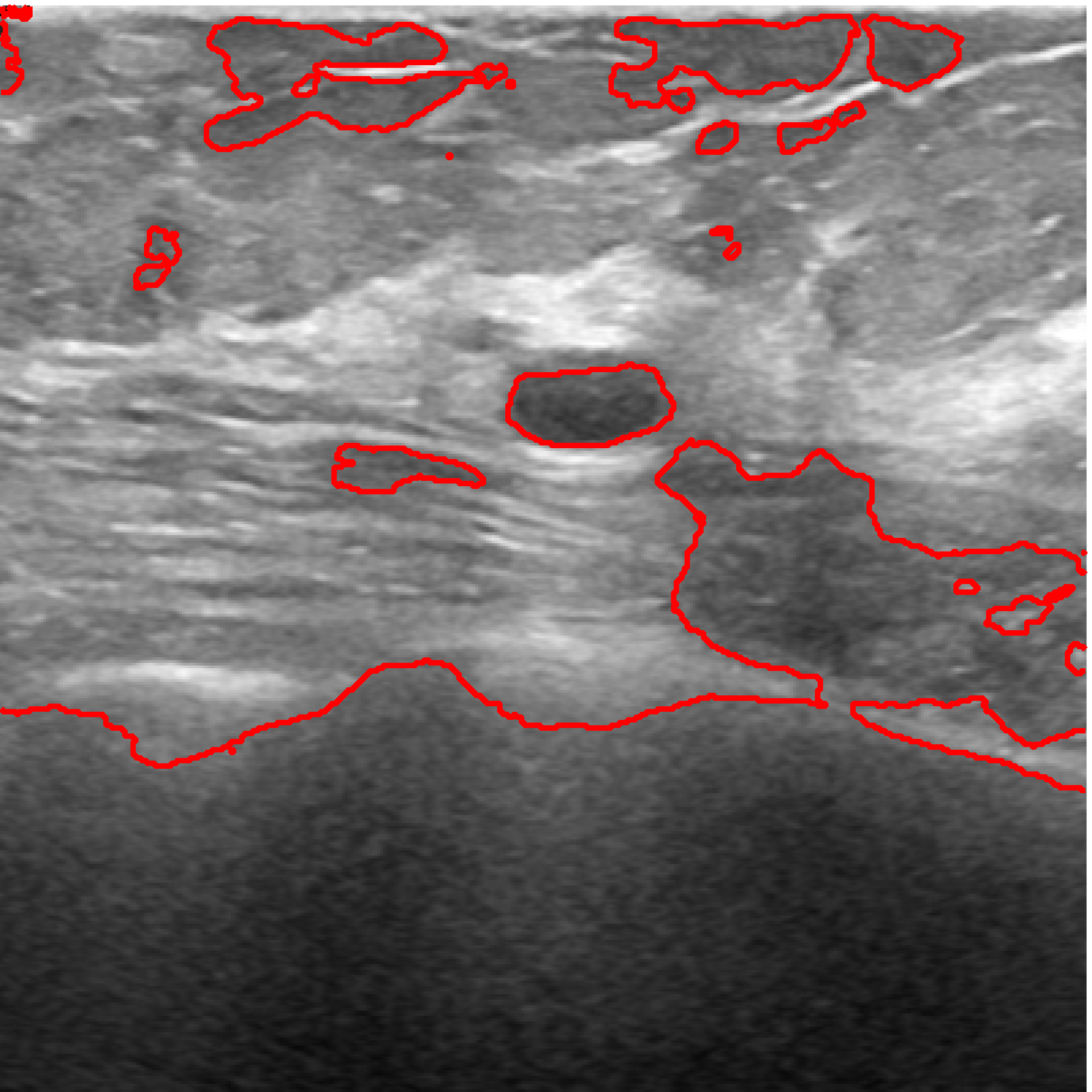}
  \end{subfigure}
  \hfill
  \begin{subfigure}[b]{0.19\linewidth}
      \centering
      \includegraphics[width=\textwidth]{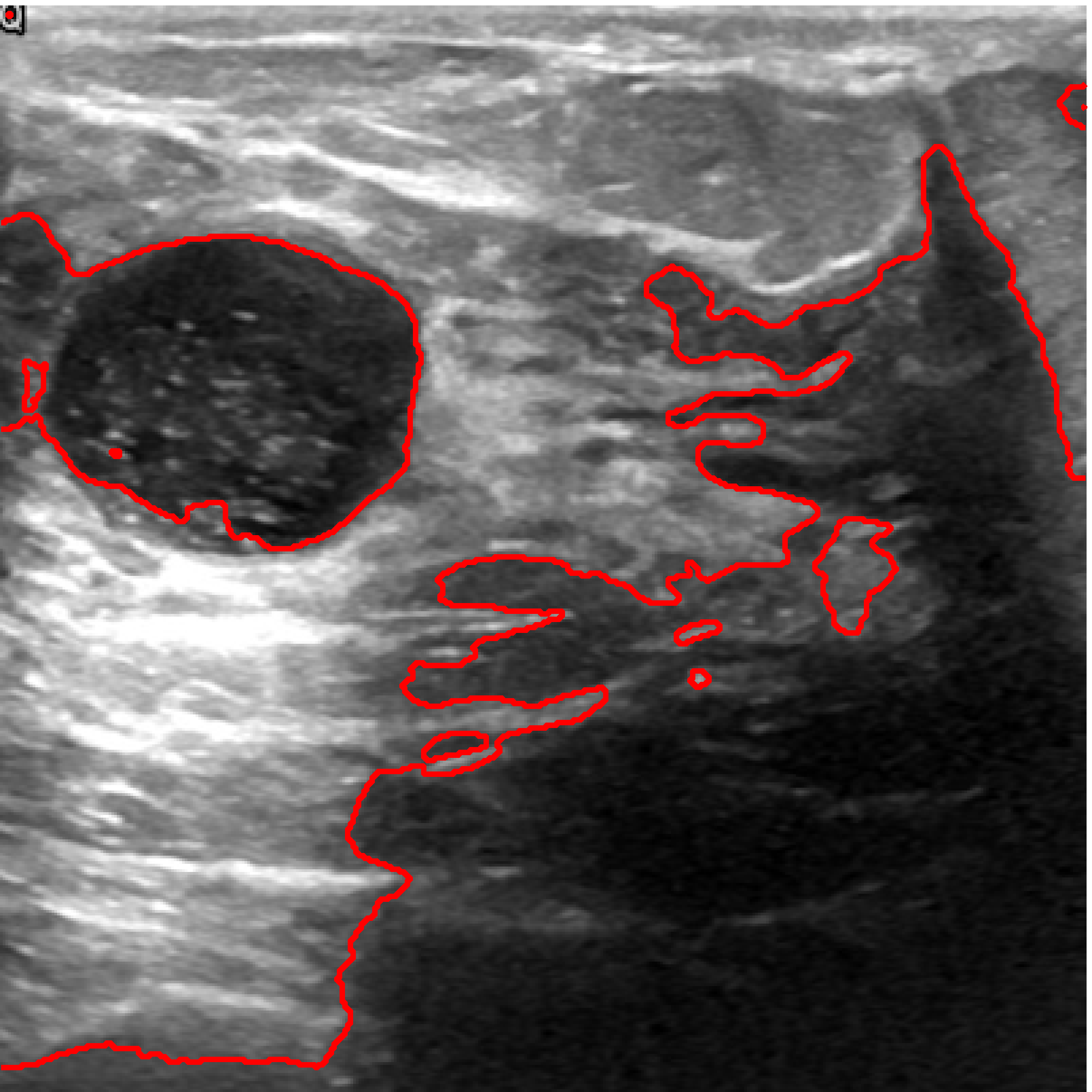}
  \end{subfigure}
  \hfill
  \begin{subfigure}[b]{0.19\linewidth}
      \centering
      \includegraphics[width=\textwidth]{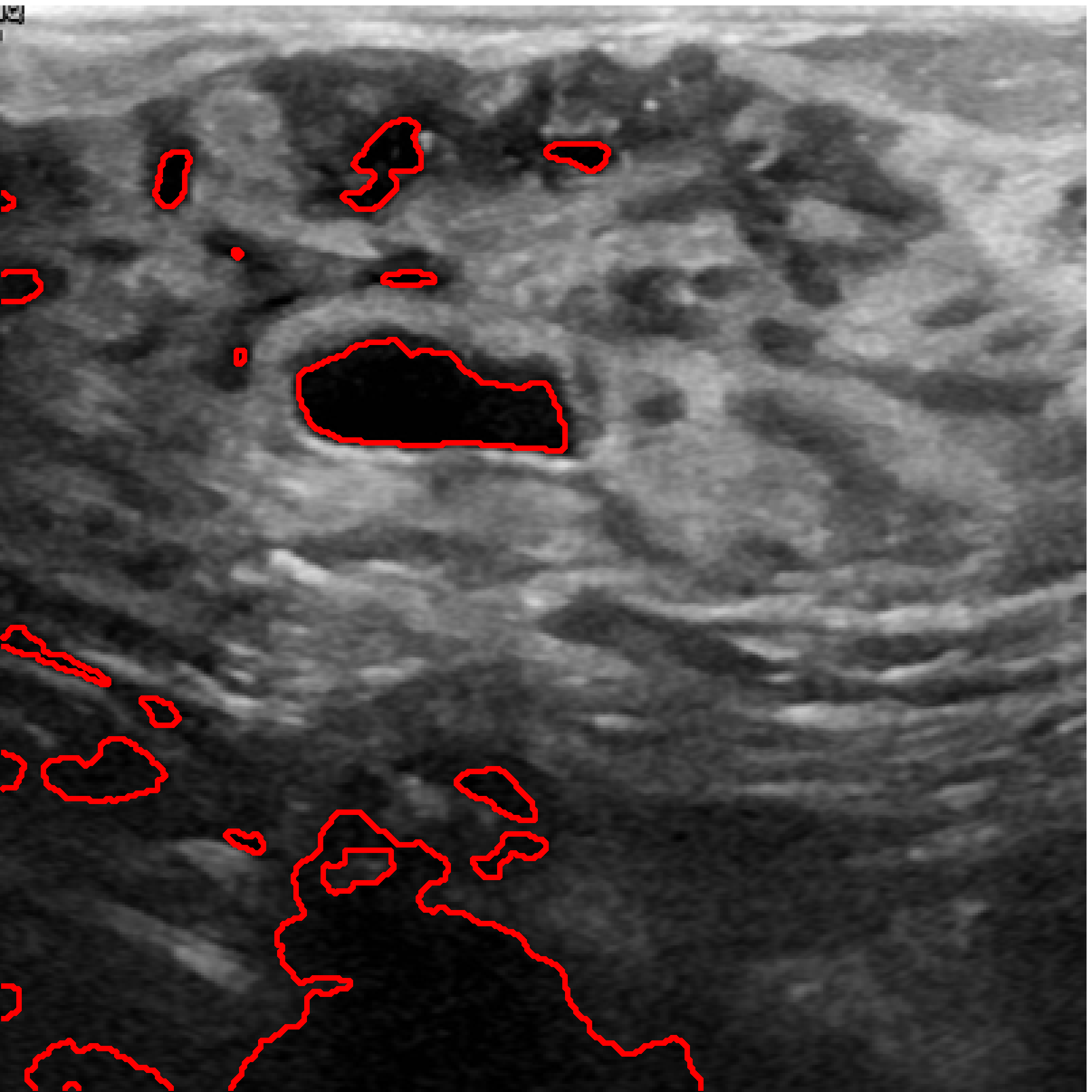}
  \end{subfigure}
  \hfill
  \begin{subfigure}[b]{0.19\linewidth}
      \centering
      \includegraphics[width=\textwidth]{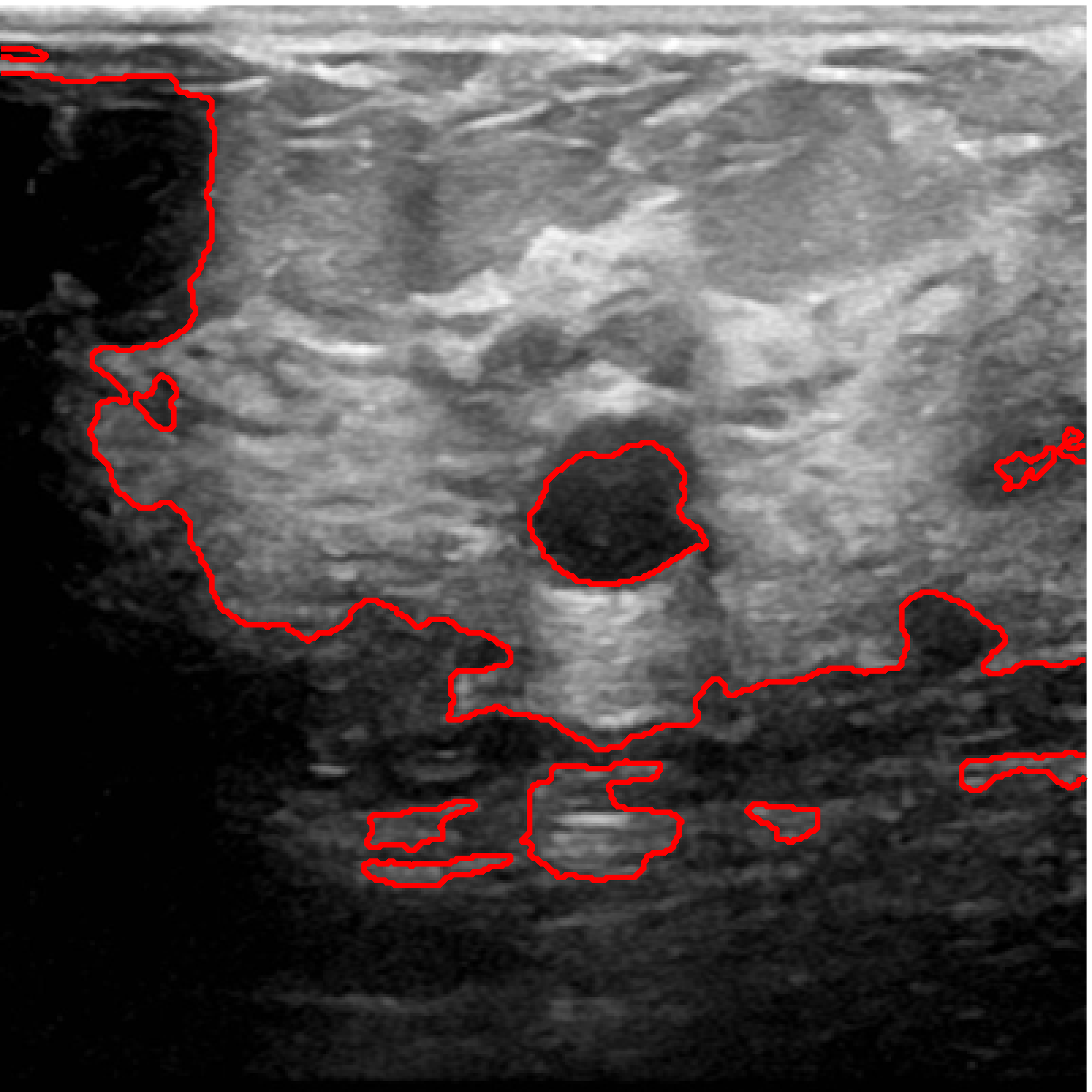}
  \end{subfigure}
  \hfill
  \begin{subfigure}[b]{0.19\linewidth}
      \centering
      \includegraphics[width=\textwidth]{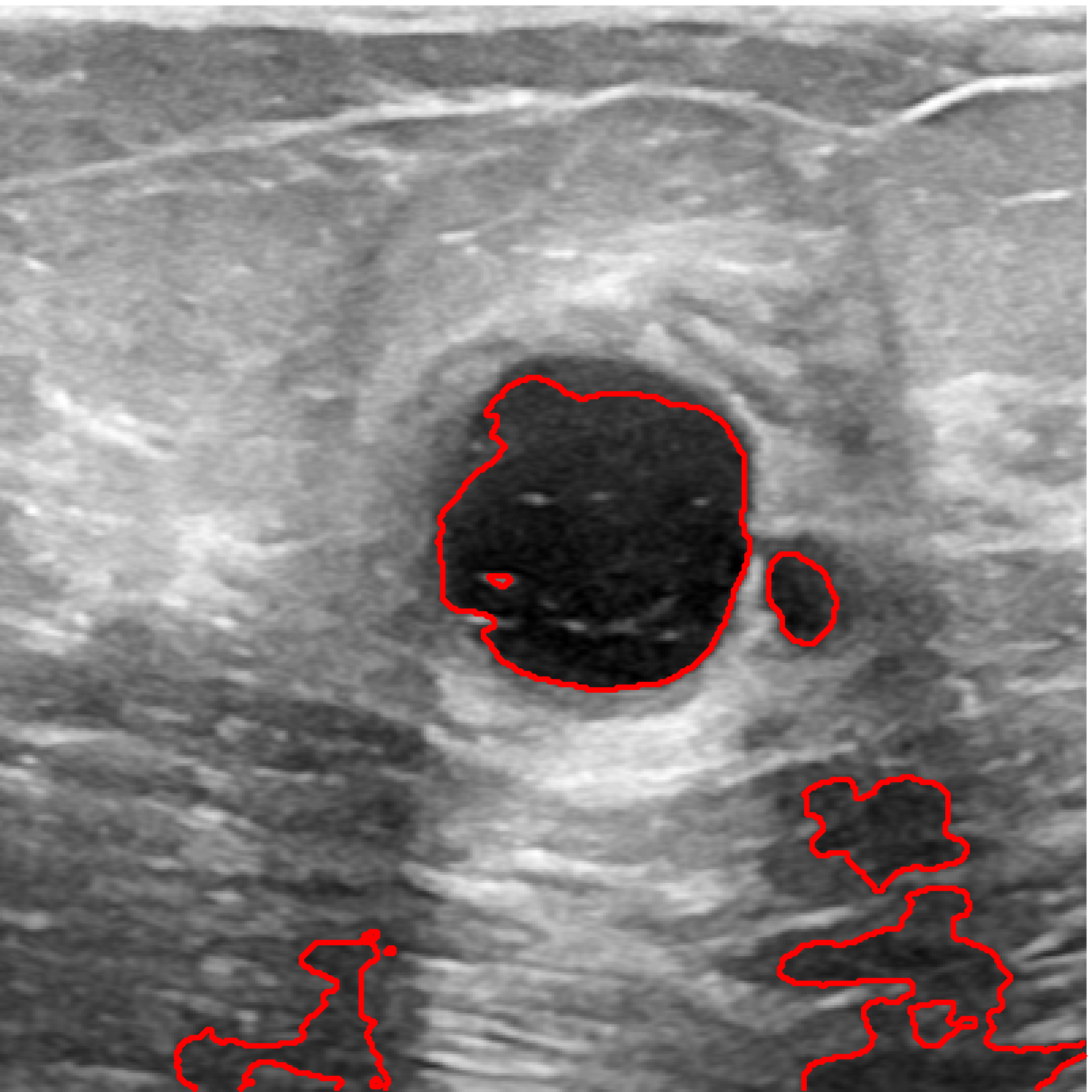}
  \end{subfigure}

  \begin{subfigure}[b]{0.19\linewidth}
      \centering
      \includegraphics[width=\textwidth]{u-8-VLSGIS-seg-contour.png}
  \end{subfigure}
  \hfill
  \begin{subfigure}[b]{0.19\linewidth}
      \centering
      \includegraphics[width=\textwidth]{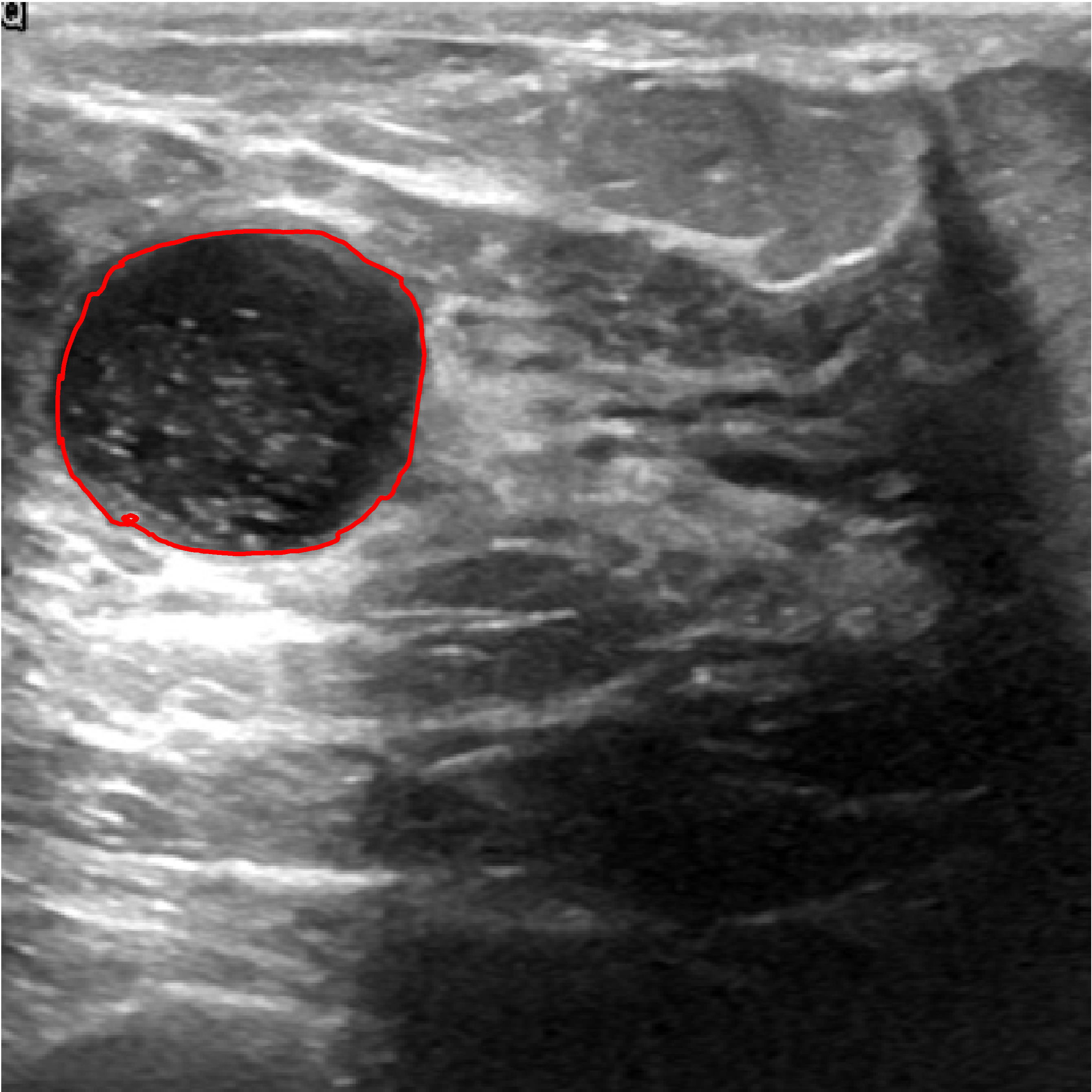}
  \end{subfigure}
  \hfill
  \begin{subfigure}[b]{0.19\linewidth}
      \centering
      \includegraphics[width=\textwidth]{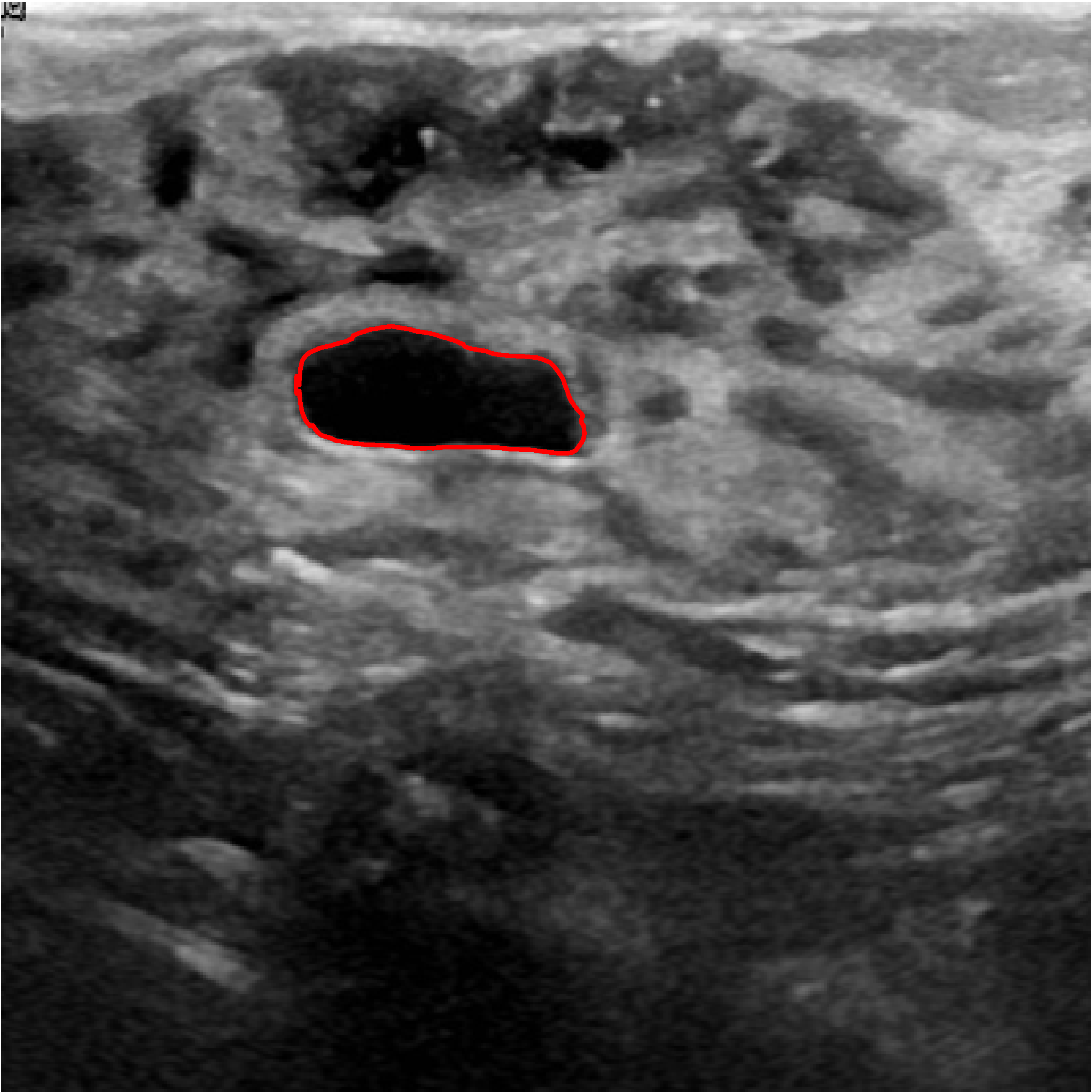}
  \end{subfigure}
  \hfill
  \begin{subfigure}[b]{0.19\linewidth}
      \centering
      \includegraphics[width=\textwidth]{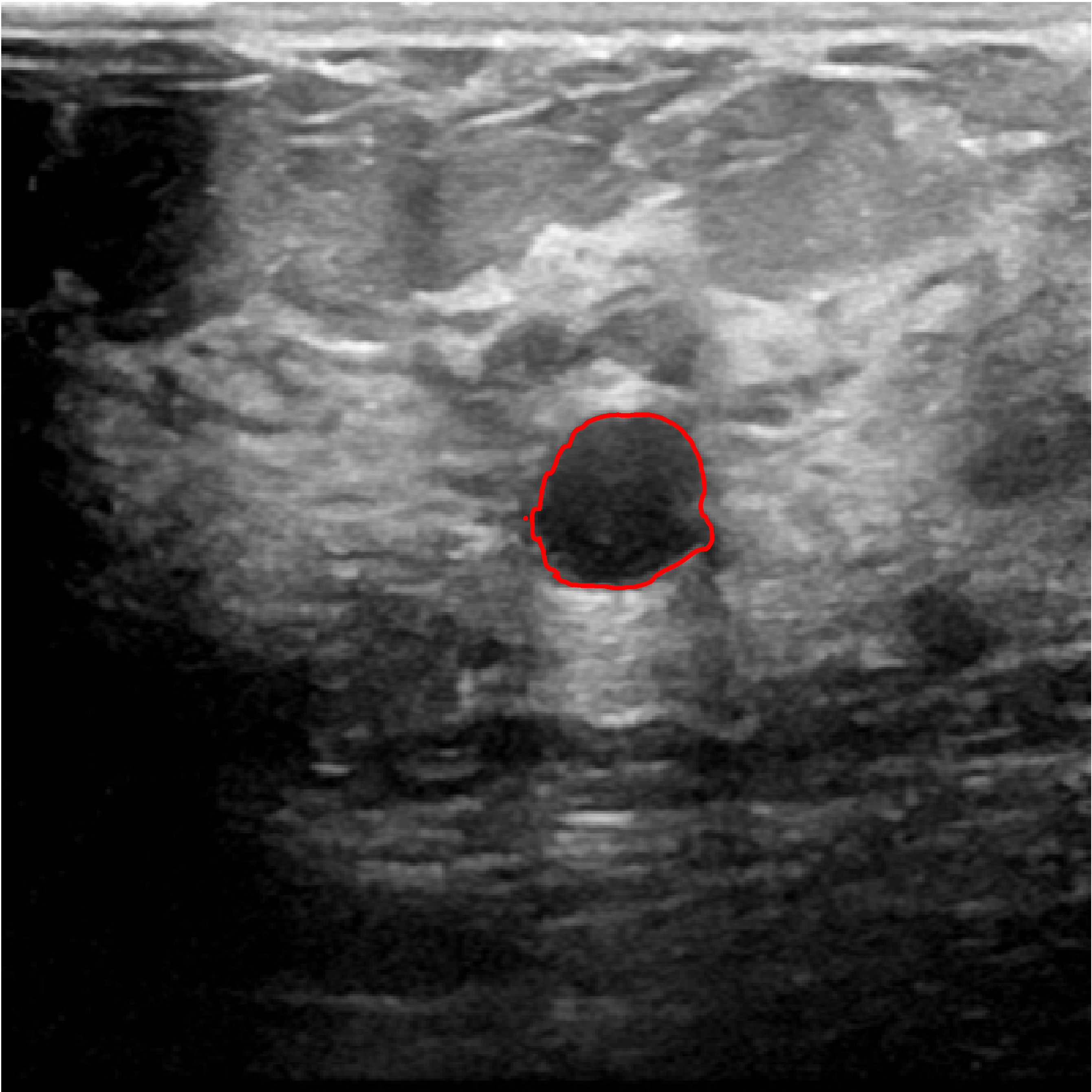}
  \end{subfigure}
  \hfill
  \begin{subfigure}[b]{0.19\linewidth}
      \centering
      \includegraphics[width=\textwidth]{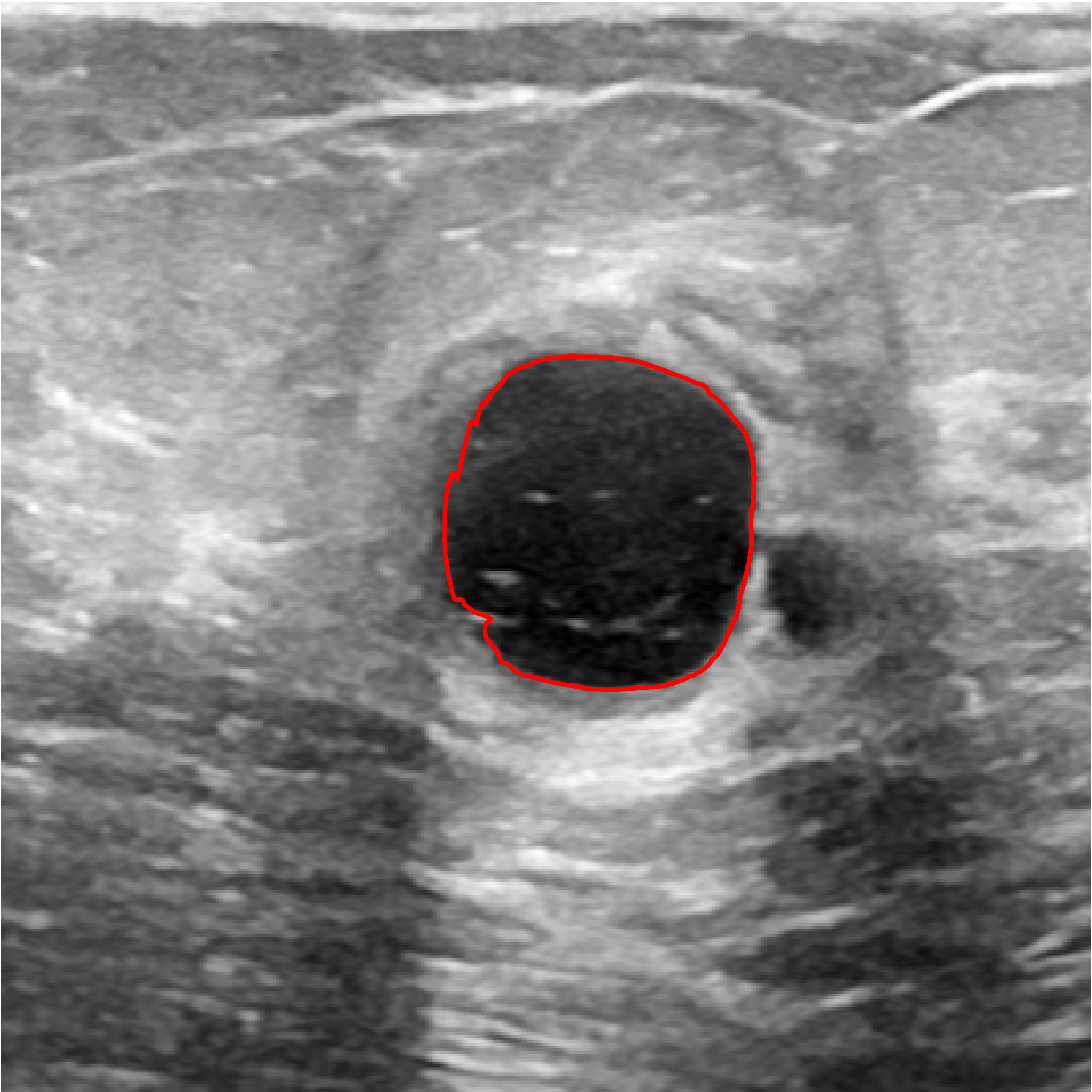}
  \end{subfigure}

  \begin{subfigure}[b]{0.19\linewidth}
      \centering
      \includegraphics[width=\textwidth]{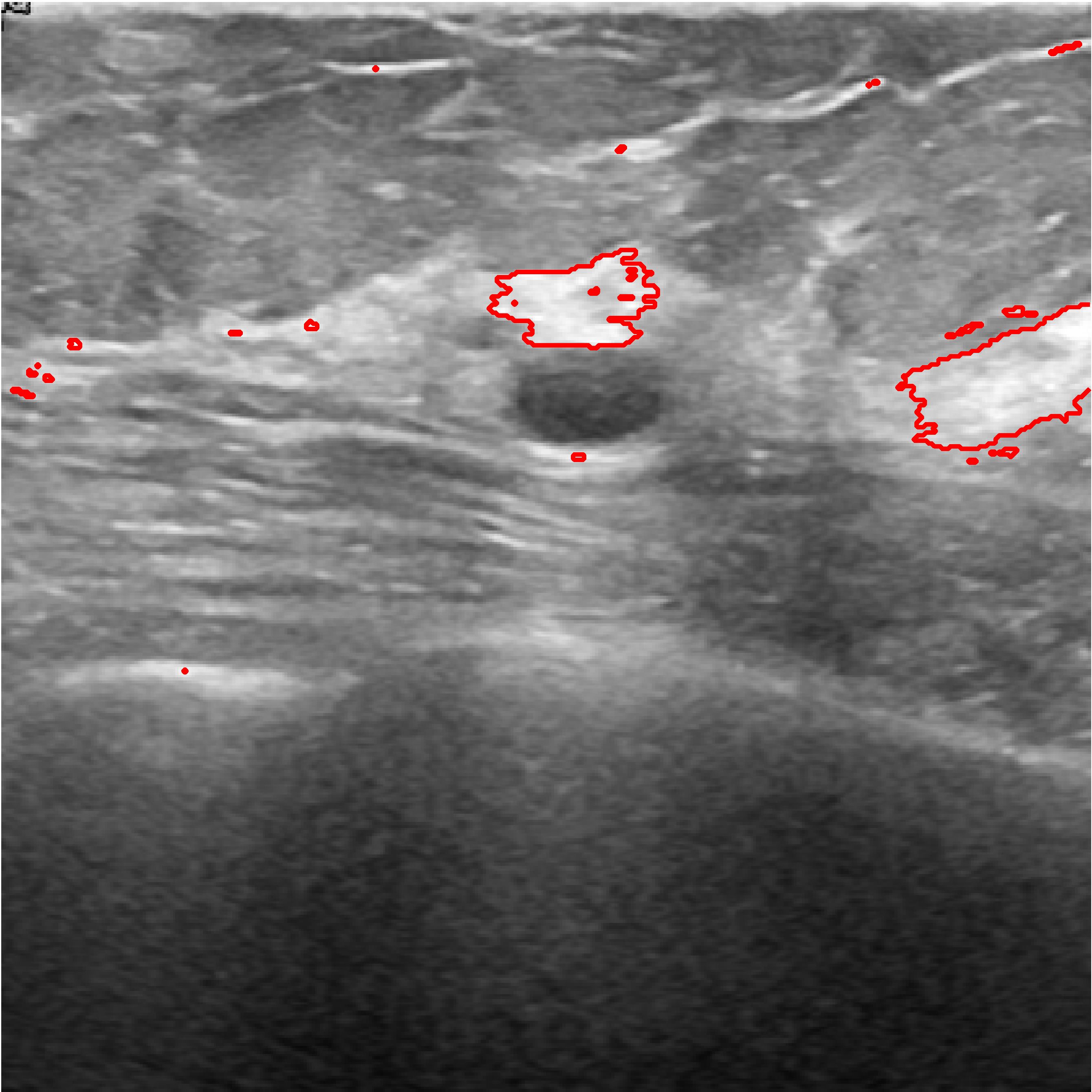}
  \end{subfigure}
  \hfill
  \begin{subfigure}[b]{0.19\linewidth}
      \centering
      \includegraphics[width=\textwidth]{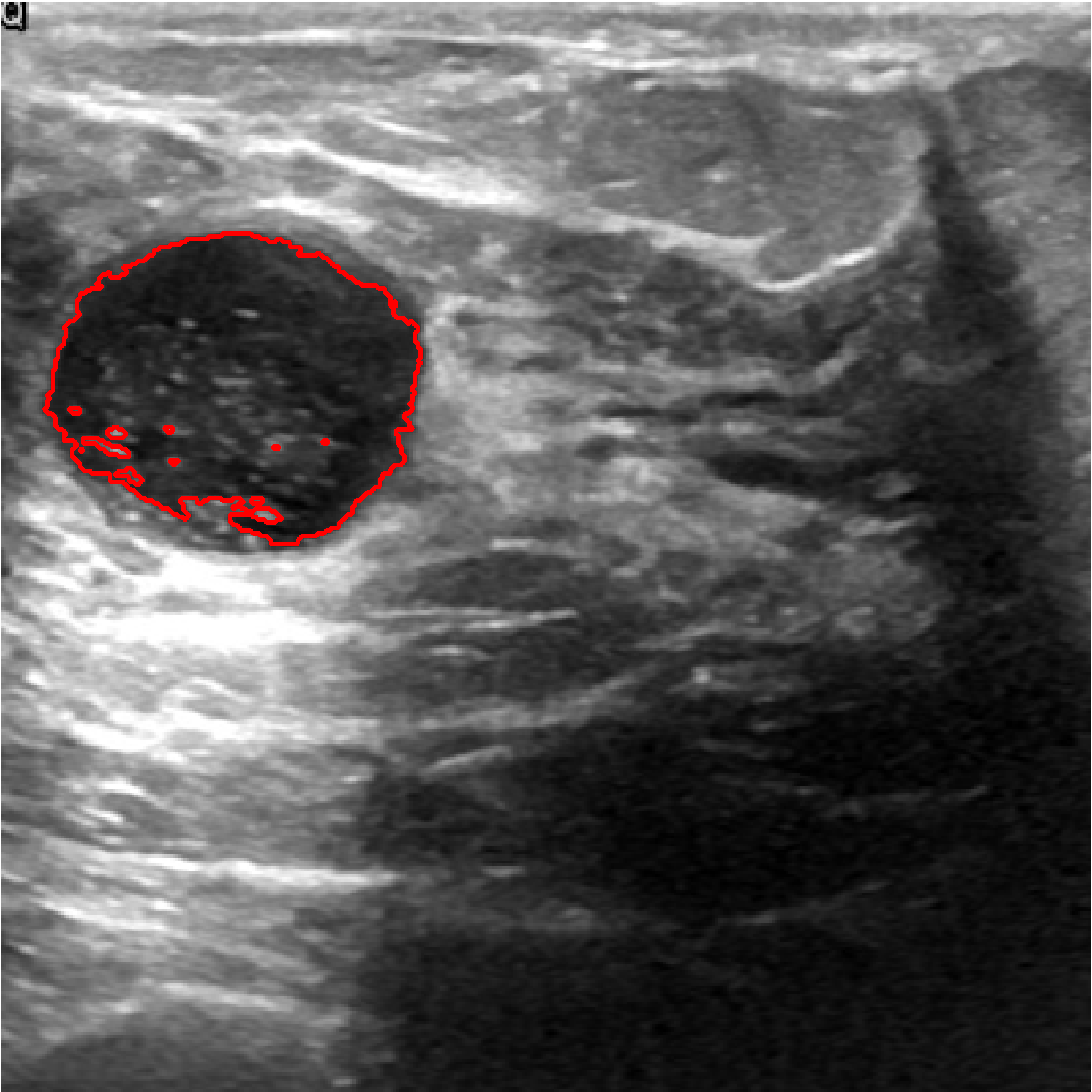}
  \end{subfigure}
  \hfill
  \begin{subfigure}[b]{0.19\linewidth}
      \centering
      \includegraphics[width=\textwidth]{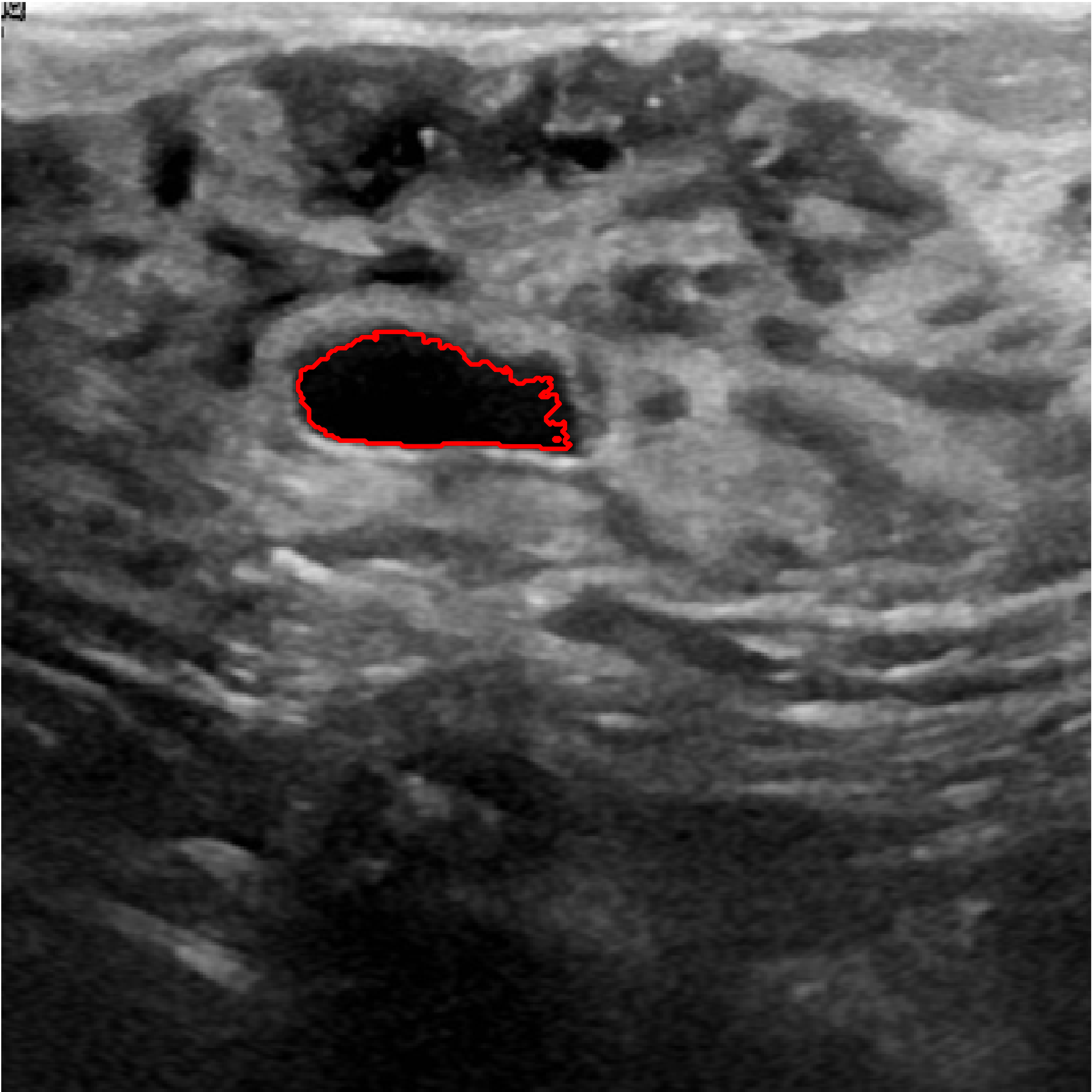}
  \end{subfigure}
  \hfill
  \begin{subfigure}[b]{0.19\linewidth}
      \centering
      \includegraphics[width=\textwidth]{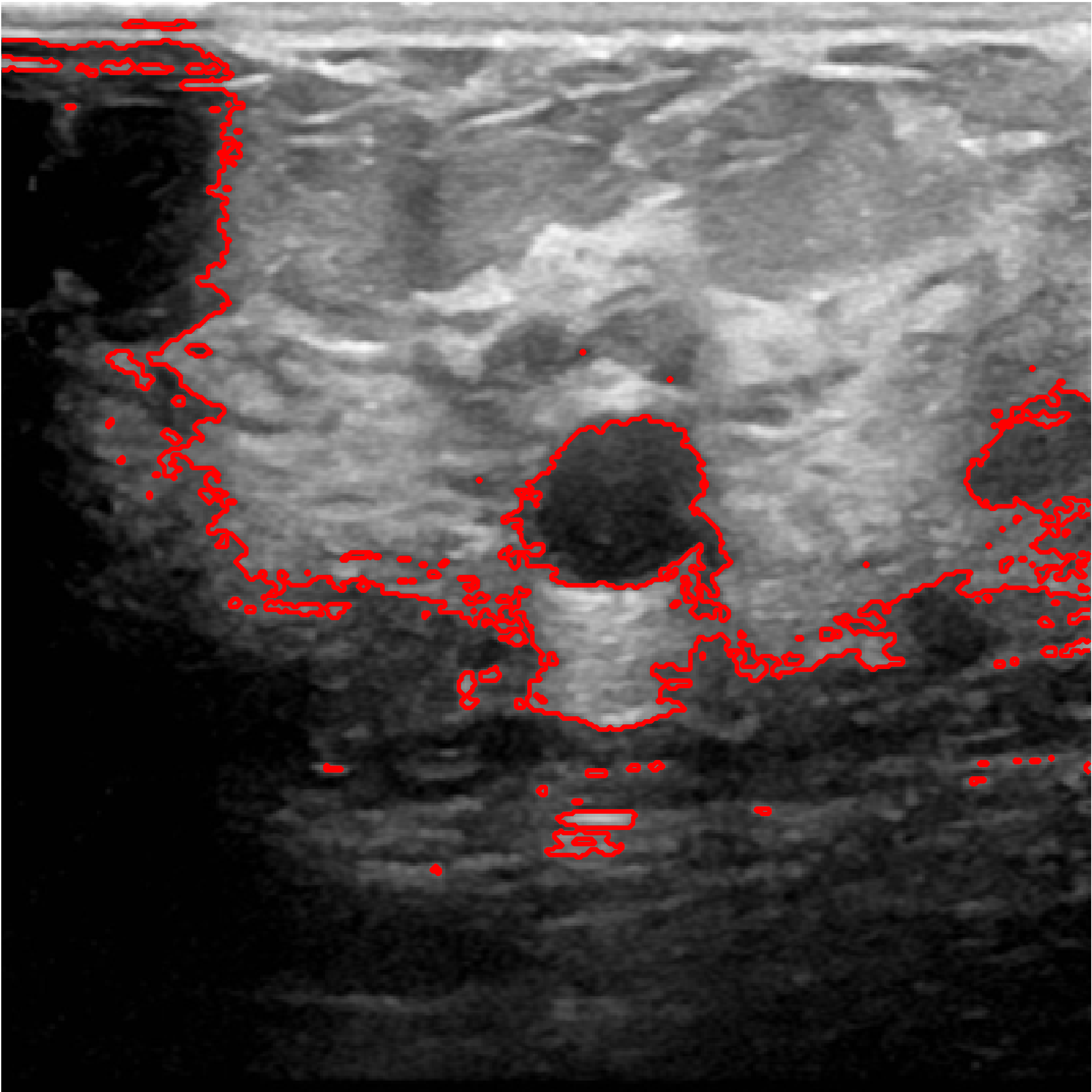}
  \end{subfigure}
  \hfill
  \begin{subfigure}[b]{0.19\linewidth}
      \centering
      \includegraphics[width=\textwidth]{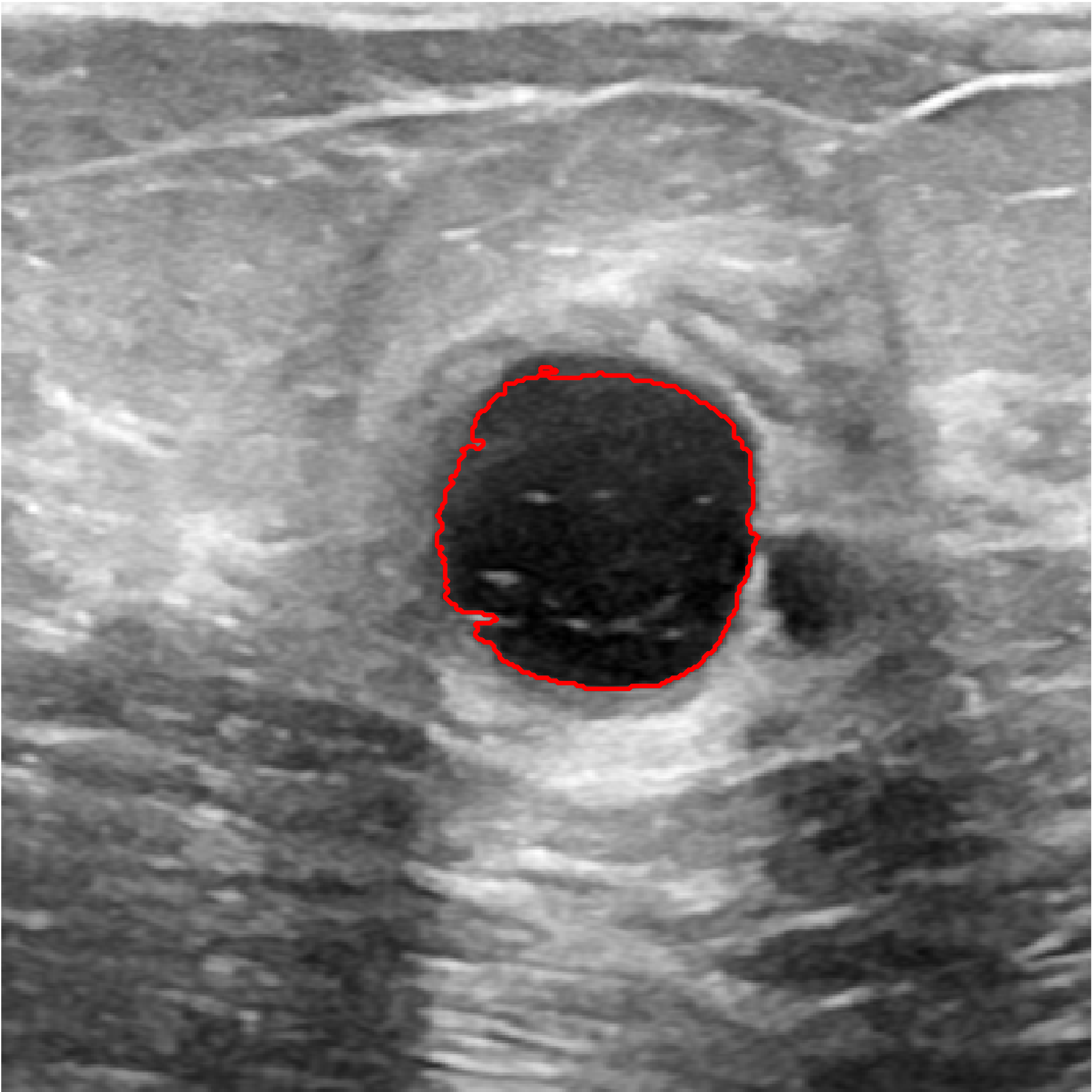}
  \end{subfigure}

  \begin{subfigure}[b]{0.19\linewidth}
      \centering
      \includegraphics[width=\textwidth]{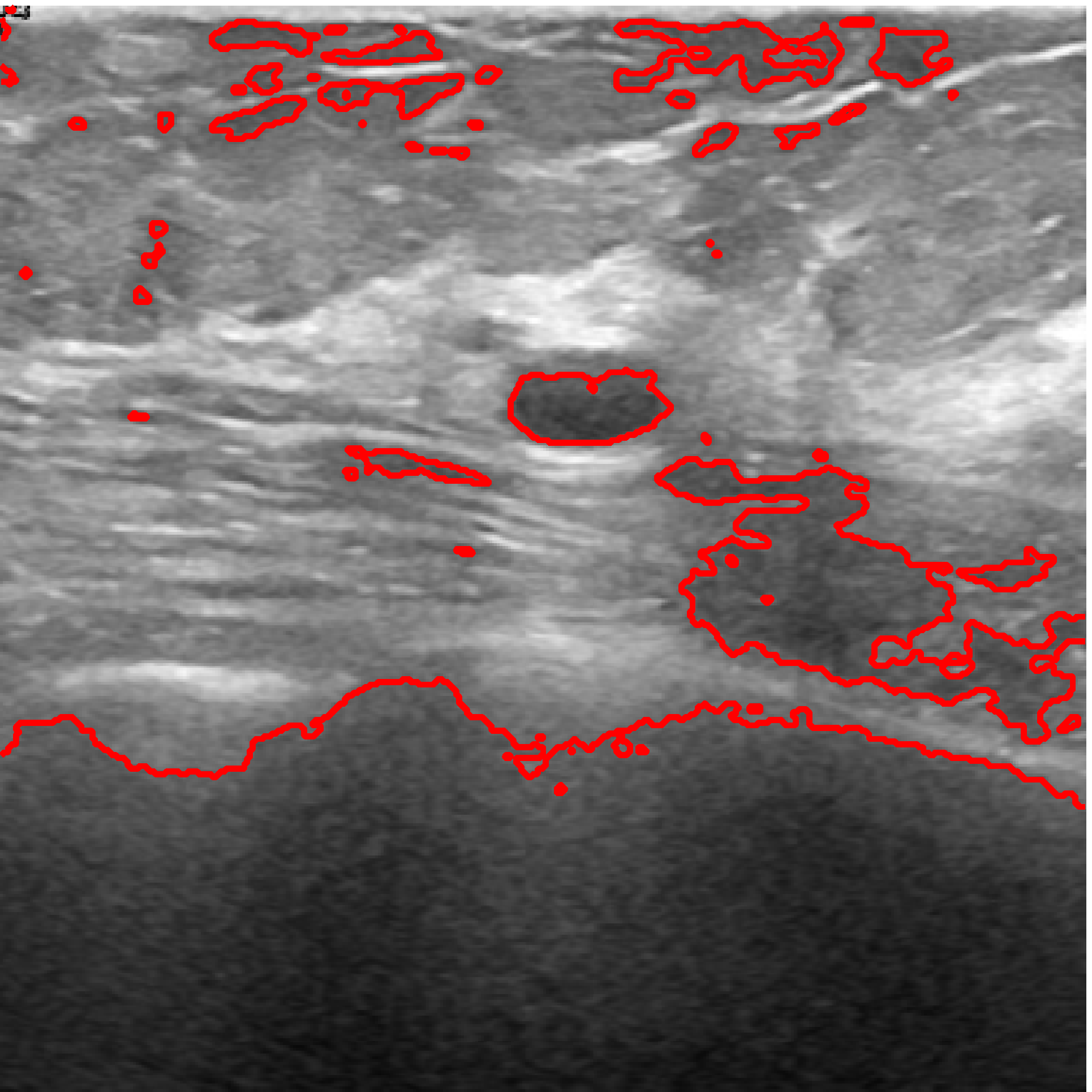}
  \end{subfigure}
  \hfill
  \begin{subfigure}[b]{0.19\linewidth}
      \centering
      \includegraphics[width=\textwidth]{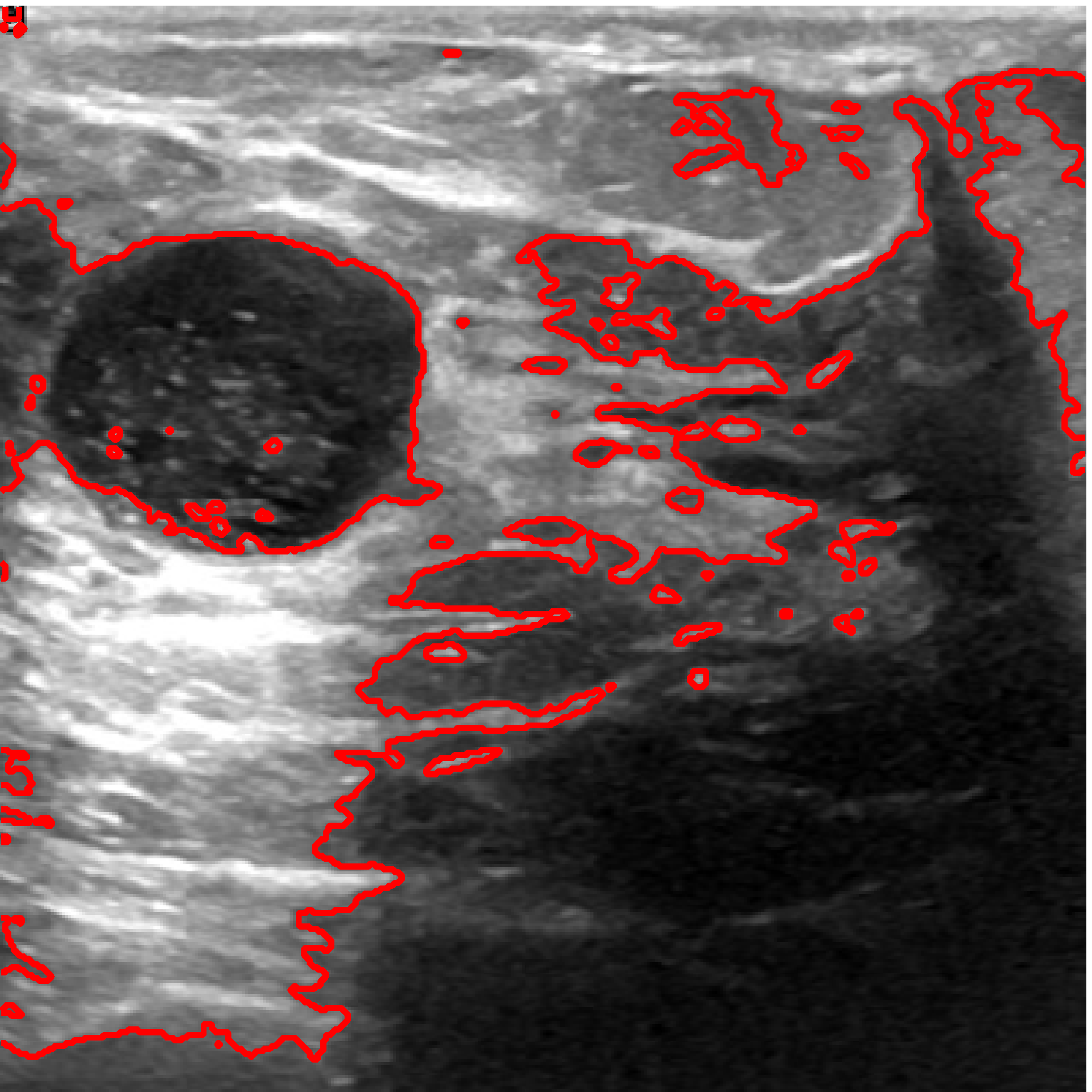}
  \end{subfigure}
  \hfill
  \begin{subfigure}[b]{0.19\linewidth}
      \centering
      \includegraphics[width=\textwidth]{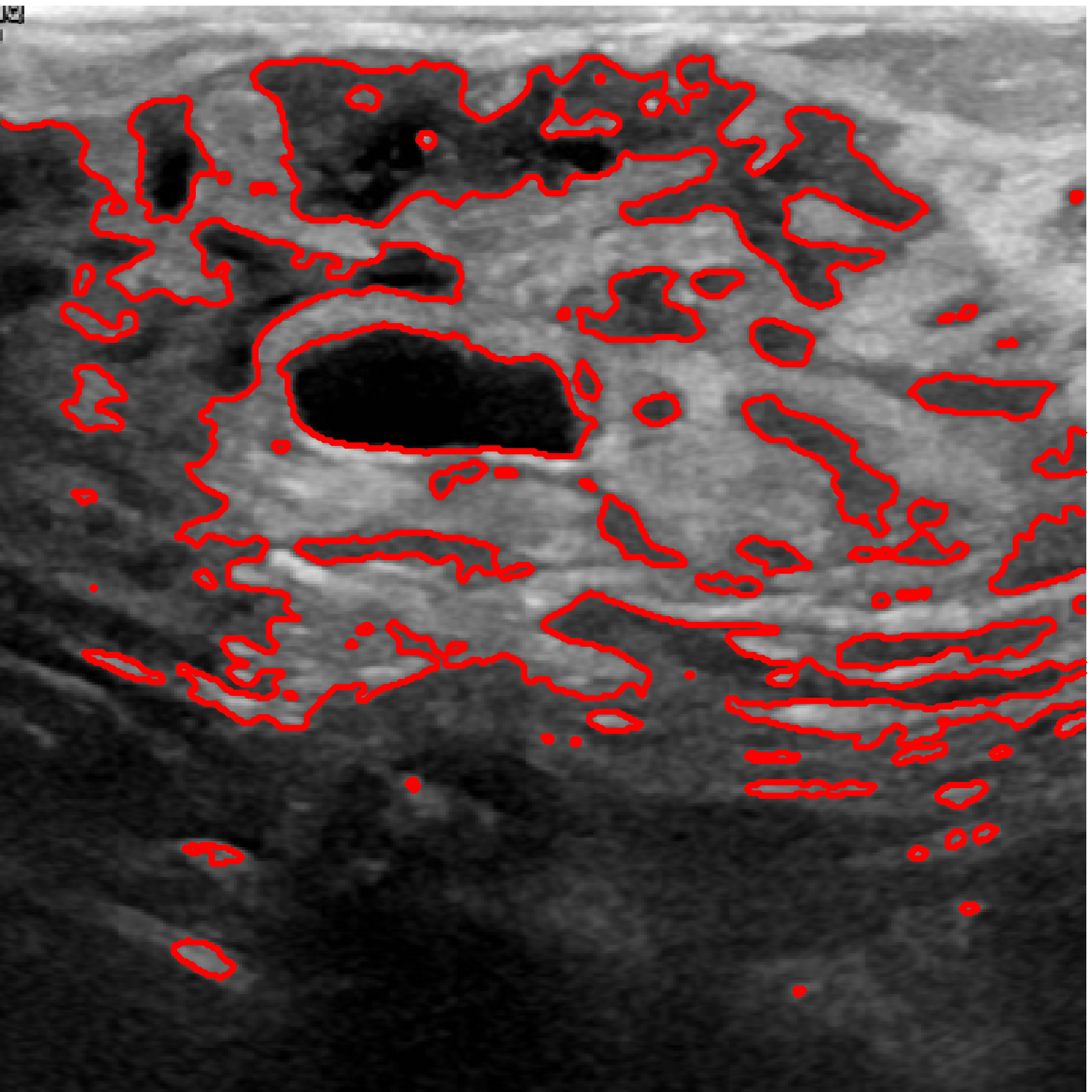}
  \end{subfigure}
  \hfill
  \begin{subfigure}[b]{0.19\linewidth}
      \centering
      \includegraphics[width=\textwidth]{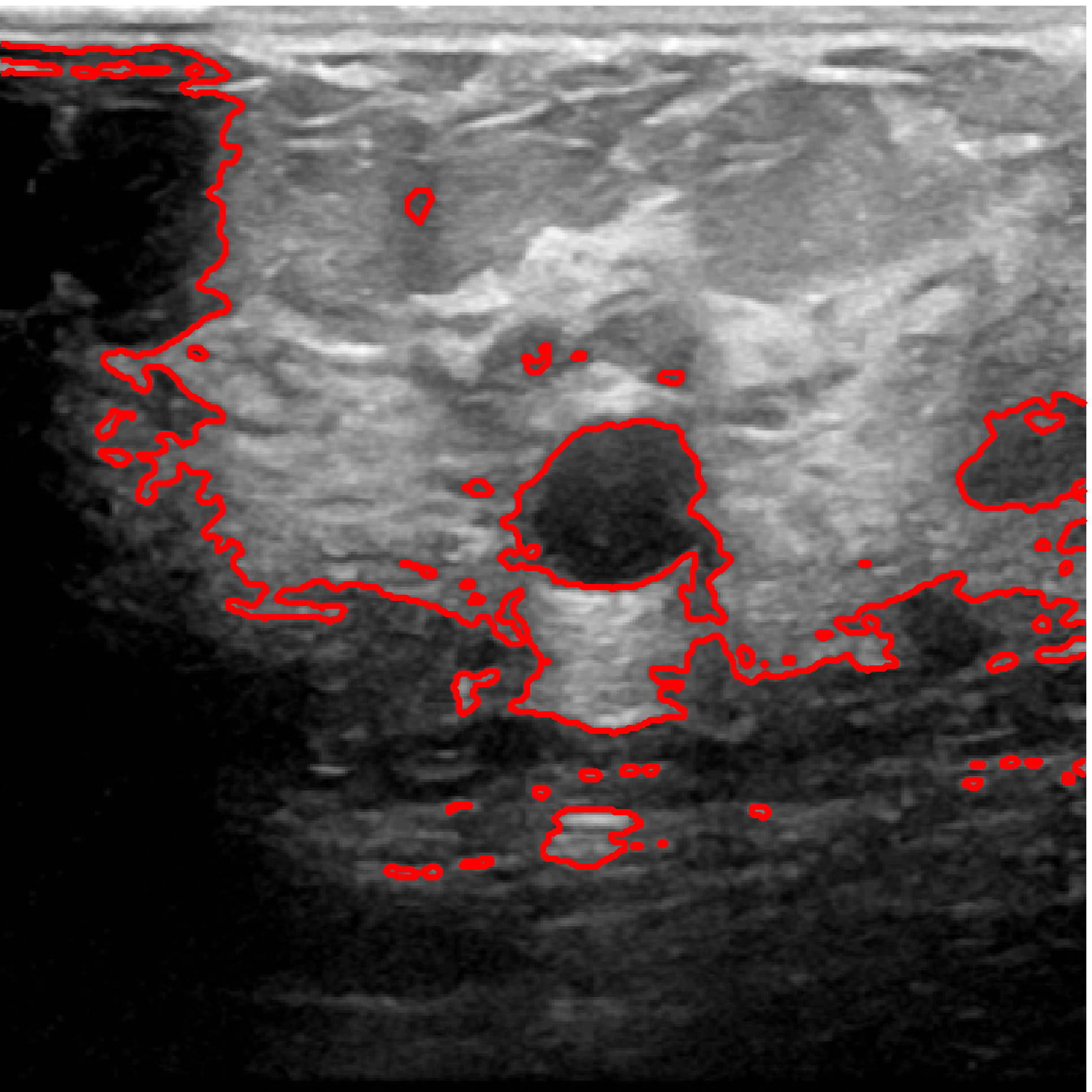}
  \end{subfigure}
  \hfill
  \begin{subfigure}[b]{0.19\linewidth}
      \centering
      \includegraphics[width=\textwidth]{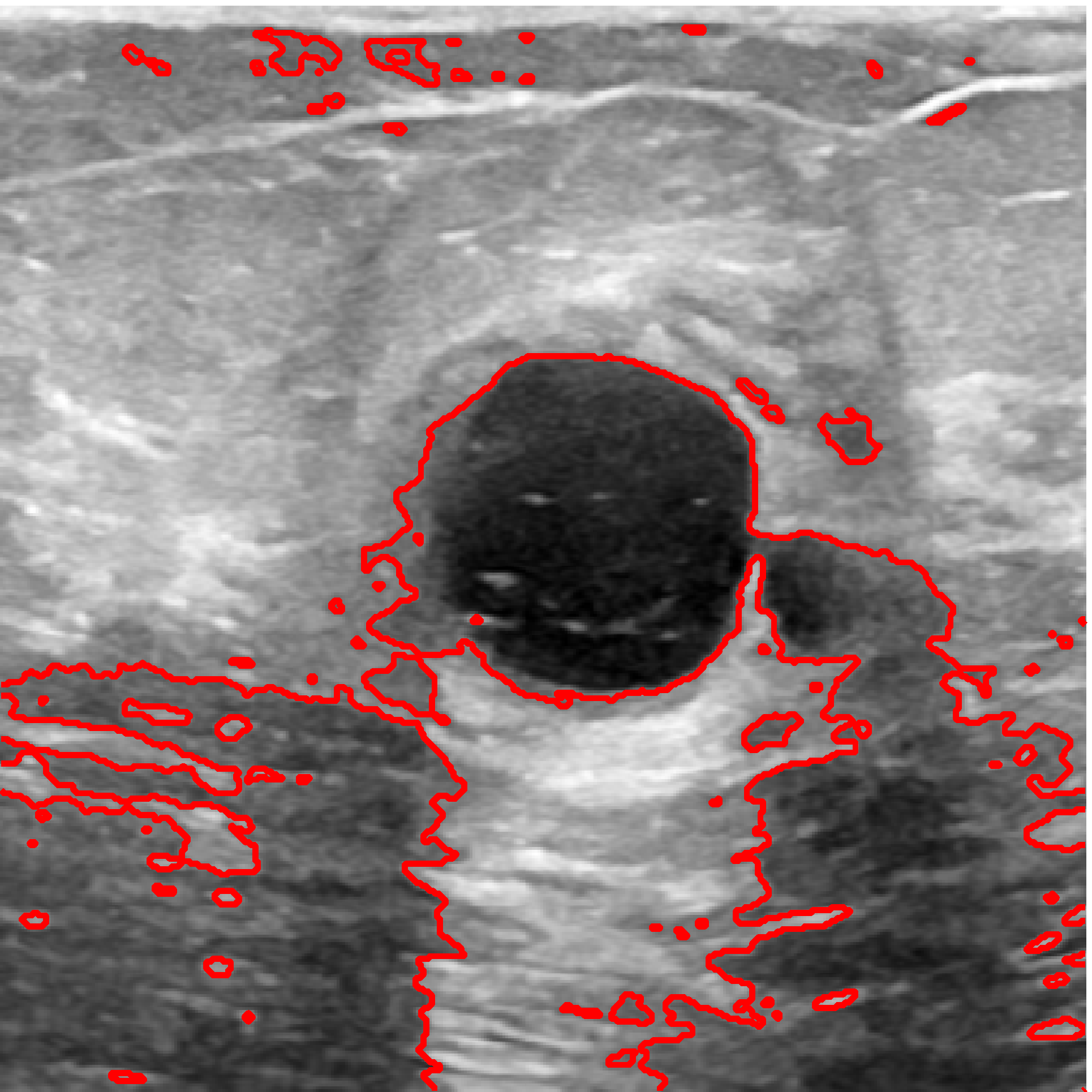}
  \end{subfigure}

  \begin{subfigure}[b]{0.19\linewidth}
      \centering
      \includegraphics[width=\textwidth]{u-8-our-seg-contour.png}
  \end{subfigure}
  \hfill
  \begin{subfigure}[b]{0.19\linewidth}
      \centering
      \includegraphics[width=\textwidth]{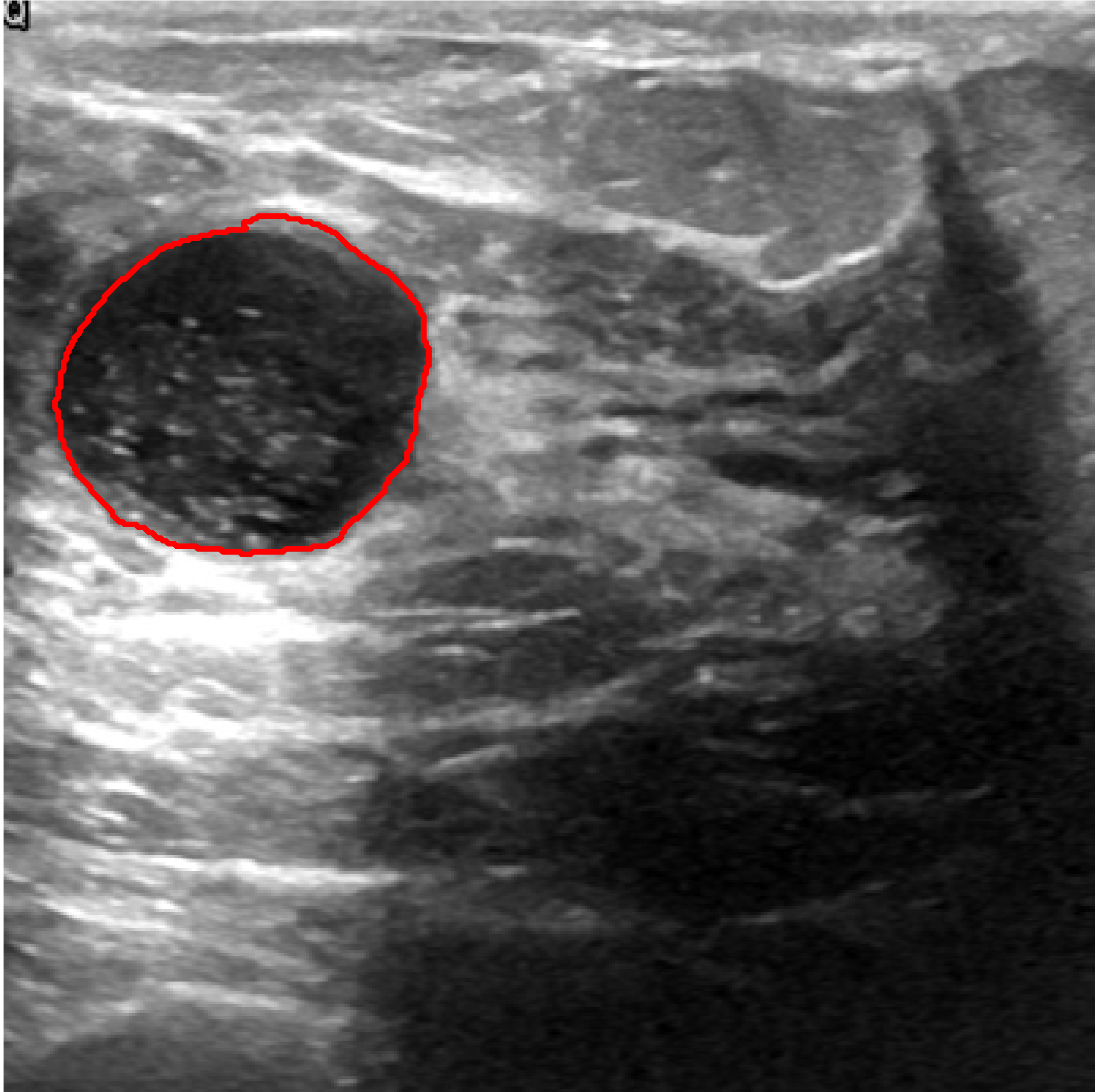}
  \end{subfigure}
  \hfill
  \begin{subfigure}[b]{0.19\linewidth}
      \centering
      \includegraphics[width=\textwidth]{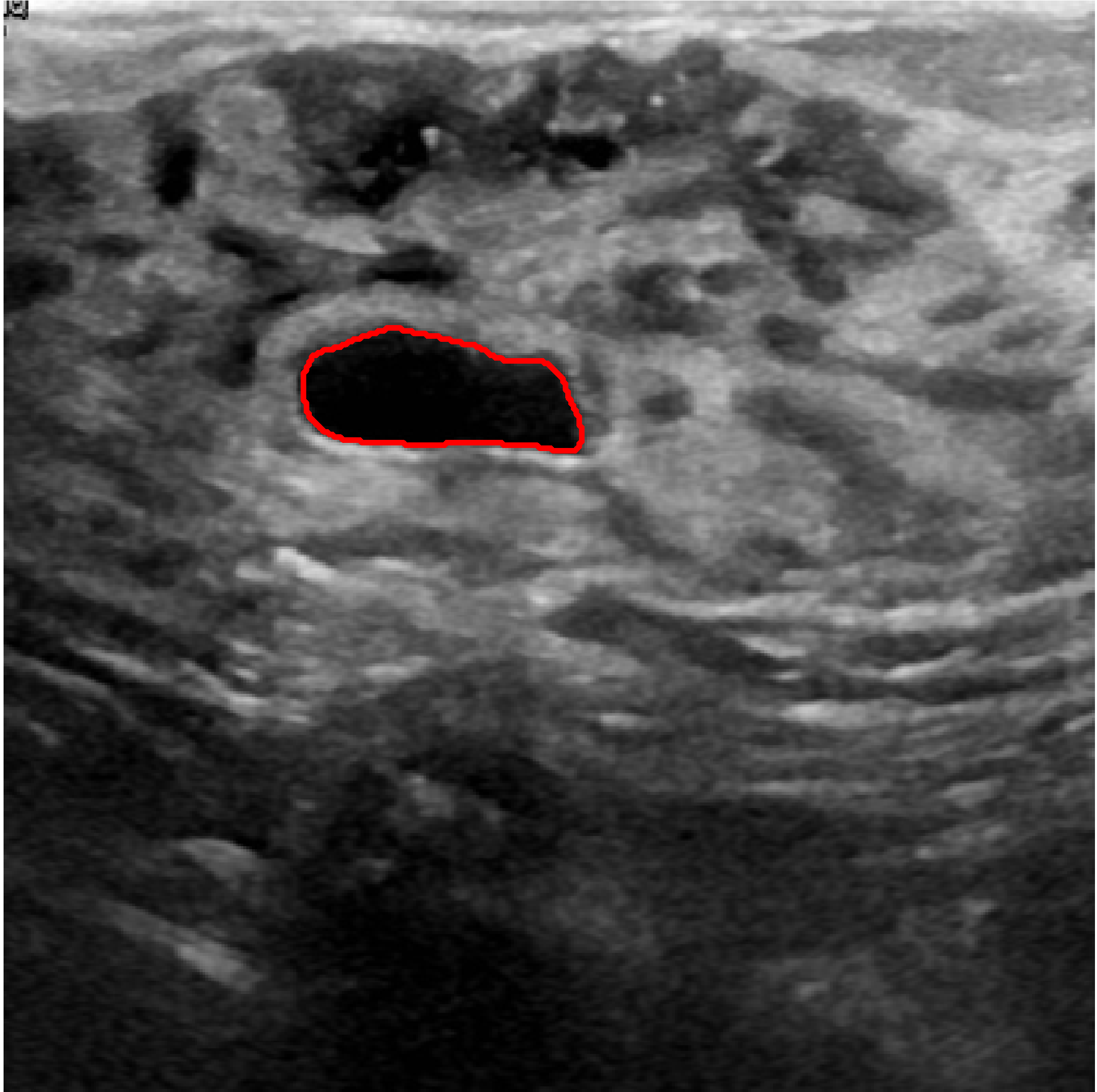}
  \end{subfigure}
  \hfill
  \begin{subfigure}[b]{0.19\linewidth}
      \centering
      \includegraphics[width=\textwidth]{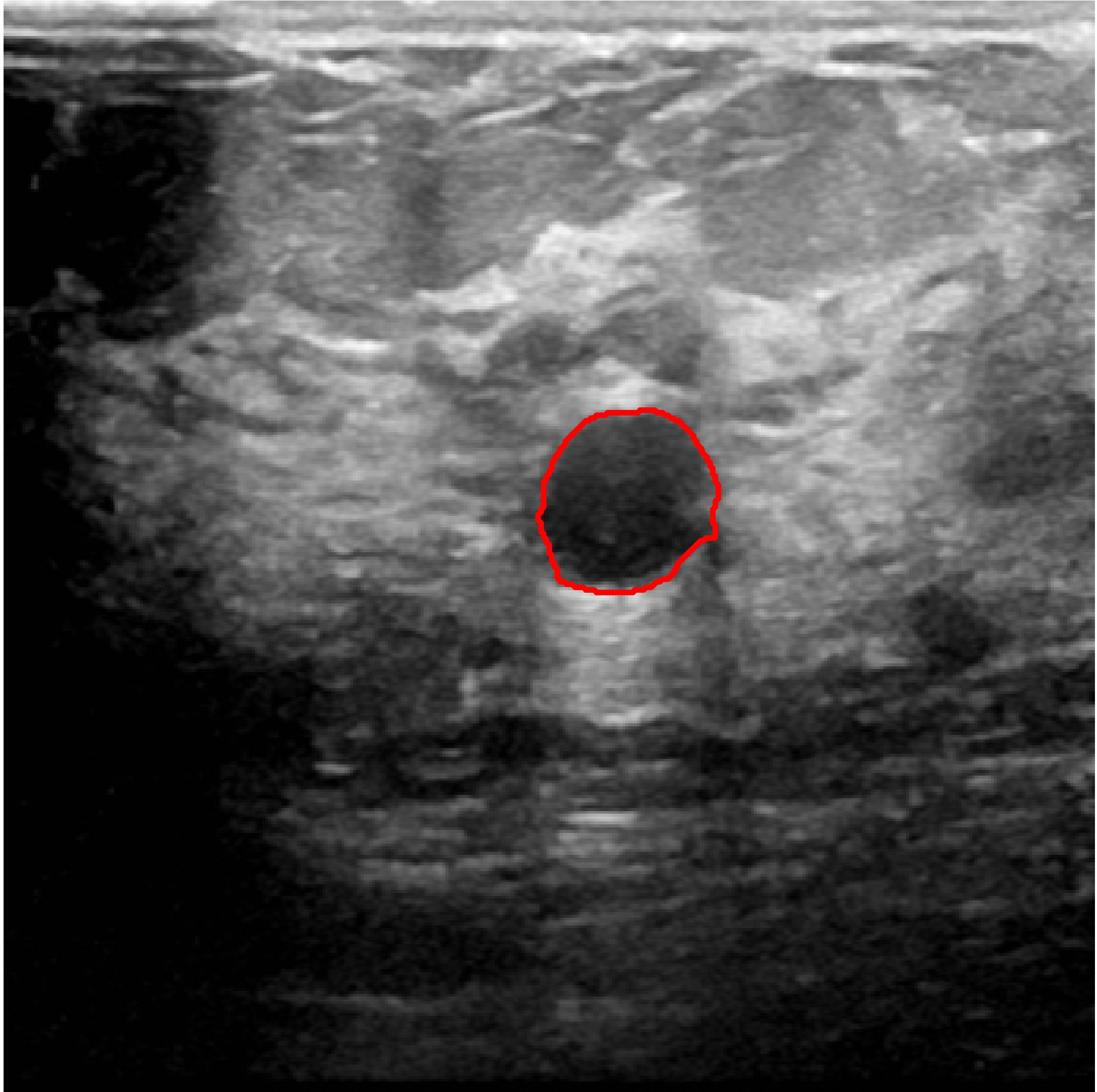}
  \end{subfigure}
  \hfill
  \begin{subfigure}[b]{0.19\linewidth}
      \centering
      \includegraphics[width=\textwidth]{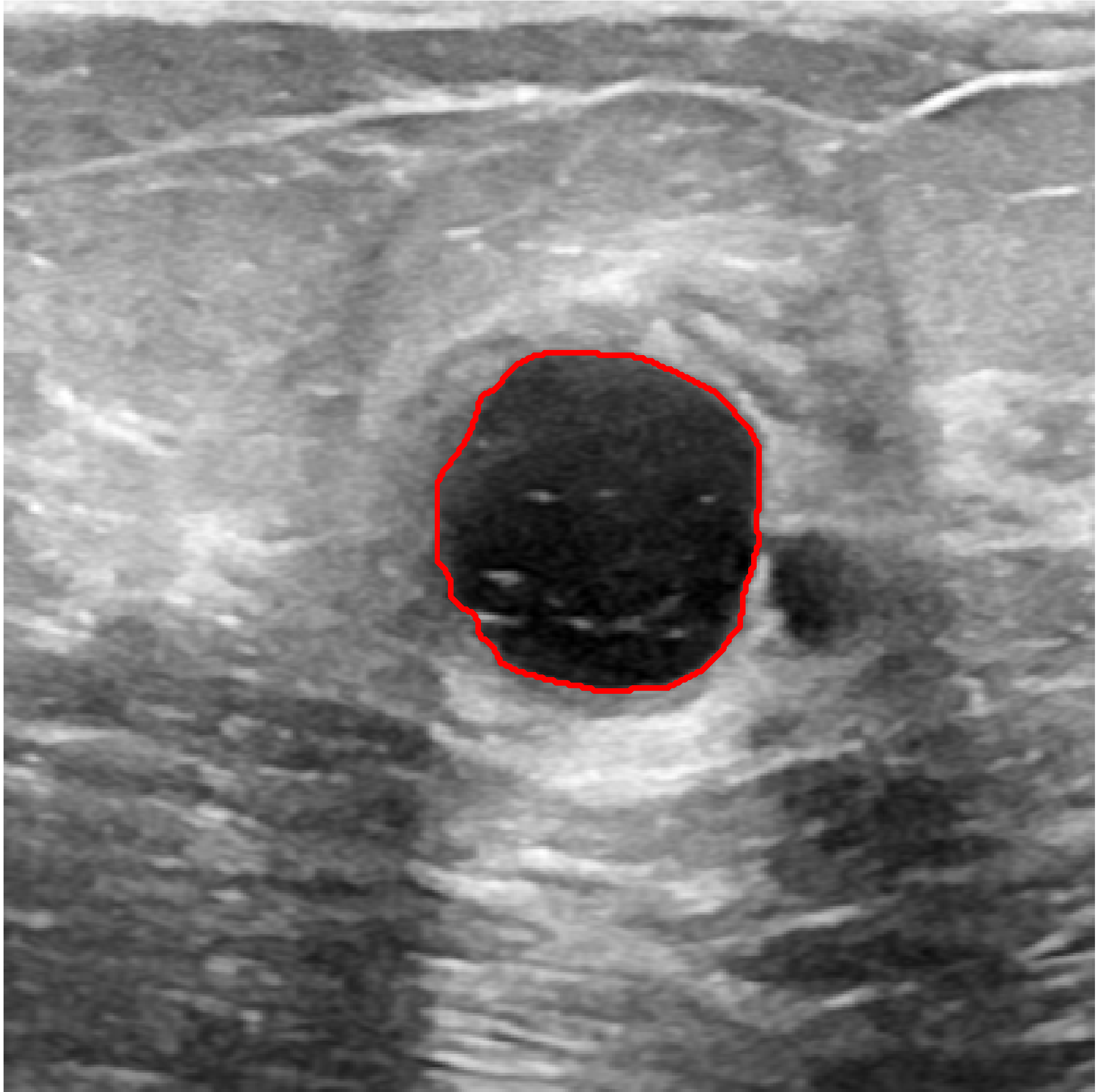}
  \end{subfigure}

  \begin{subfigure}[b]{0.19\linewidth}
      \centering
      \includegraphics[width=\textwidth]{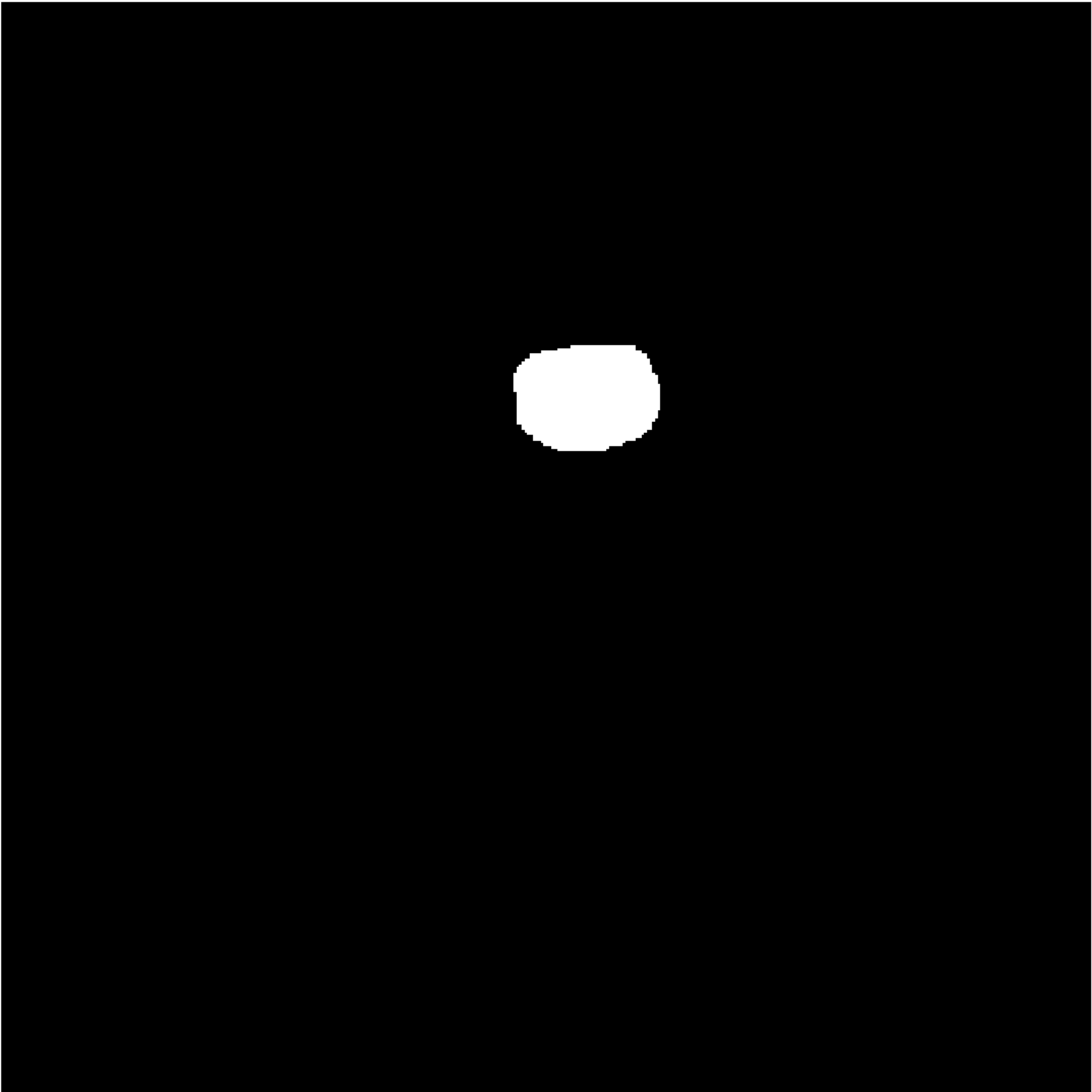}
      \caption{}
      \label{fig:u-8}
  \end{subfigure}
  \hfill
  \begin{subfigure}[b]{0.19\linewidth}
      \centering
      \includegraphics[width=\textwidth]{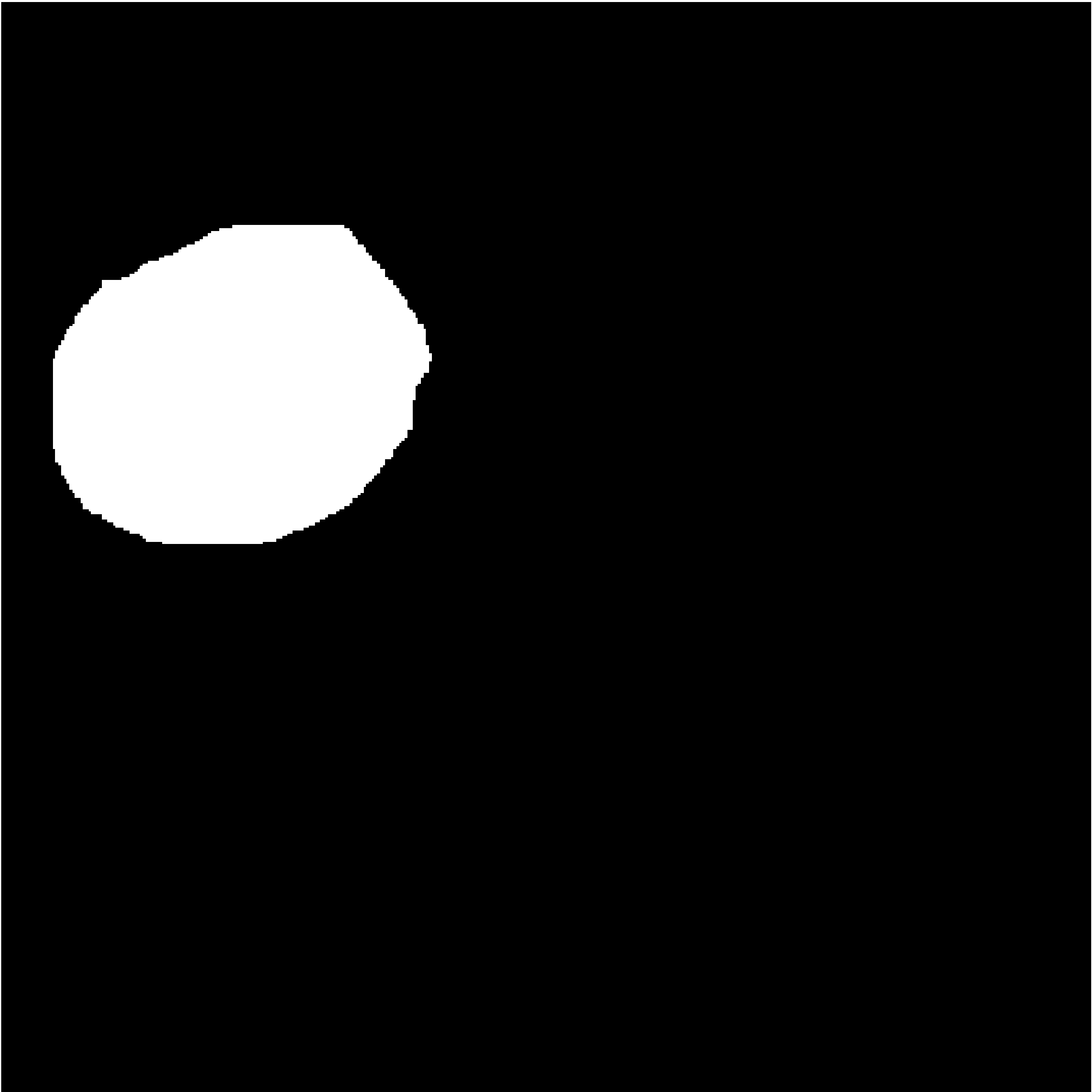}
      \caption{}
      \label{fig:u-72}
  \end{subfigure}
  \hfill
  \begin{subfigure}[b]{0.19\linewidth}
      \centering
      \includegraphics[width=\textwidth]{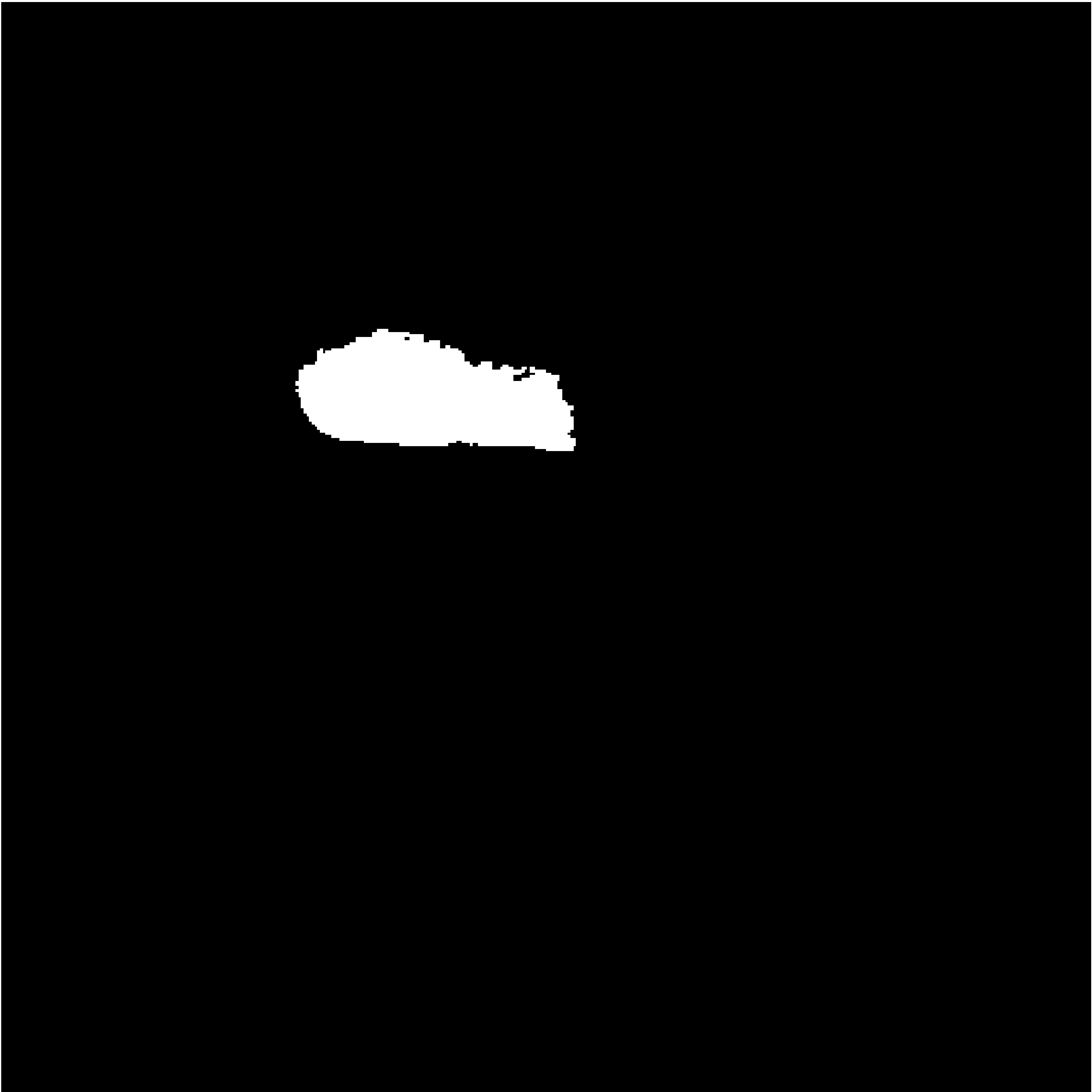}
      \caption{}
      \label{fig:u-87}
  \end{subfigure}
  \hfill
  \begin{subfigure}[b]{0.19\linewidth}
      \centering
      \includegraphics[width=\textwidth]{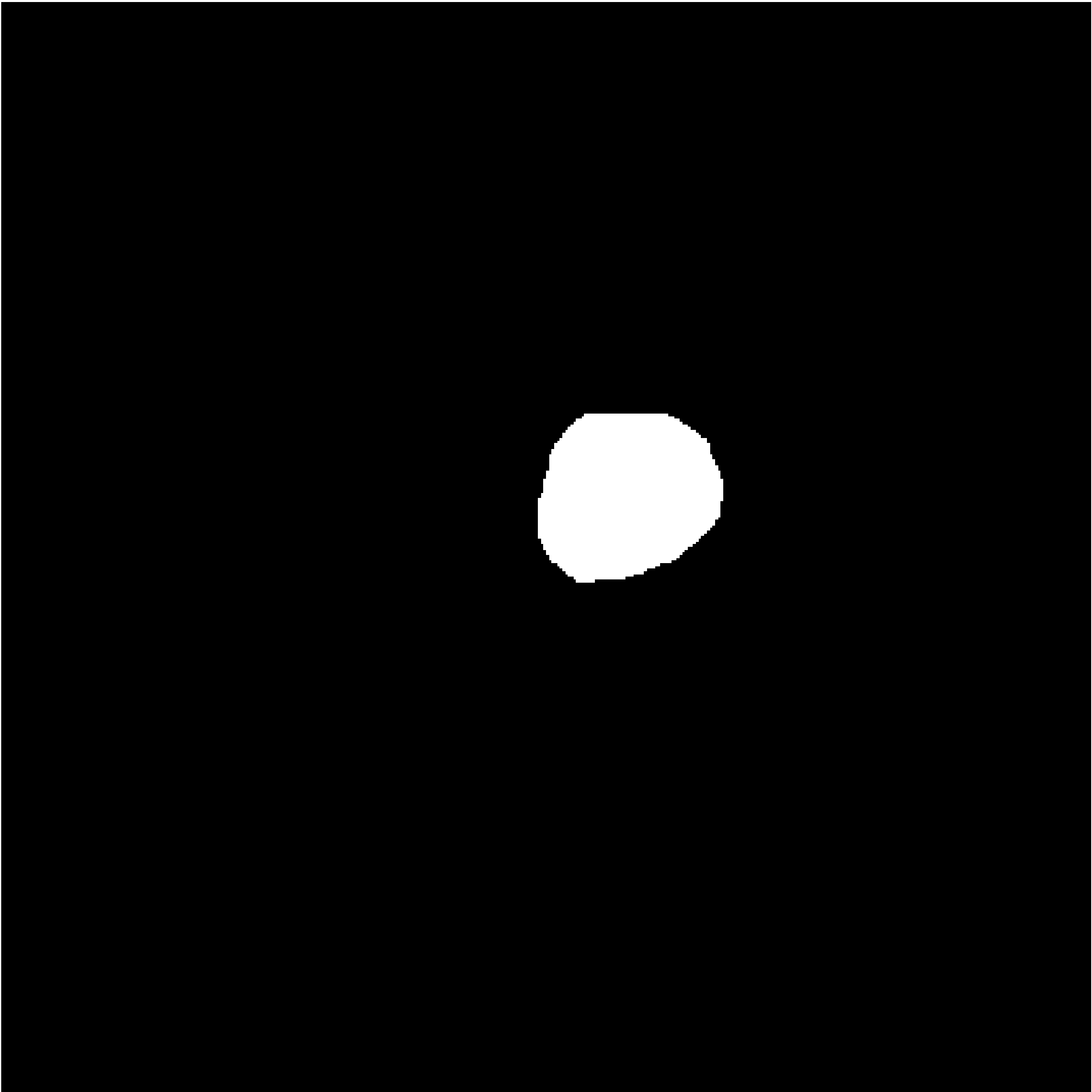}
      \caption{}
      \label{fig:u-107}
  \end{subfigure}
  \hfill
  \begin{subfigure}[b]{0.19\linewidth}
      \centering
      \includegraphics[width=\textwidth]{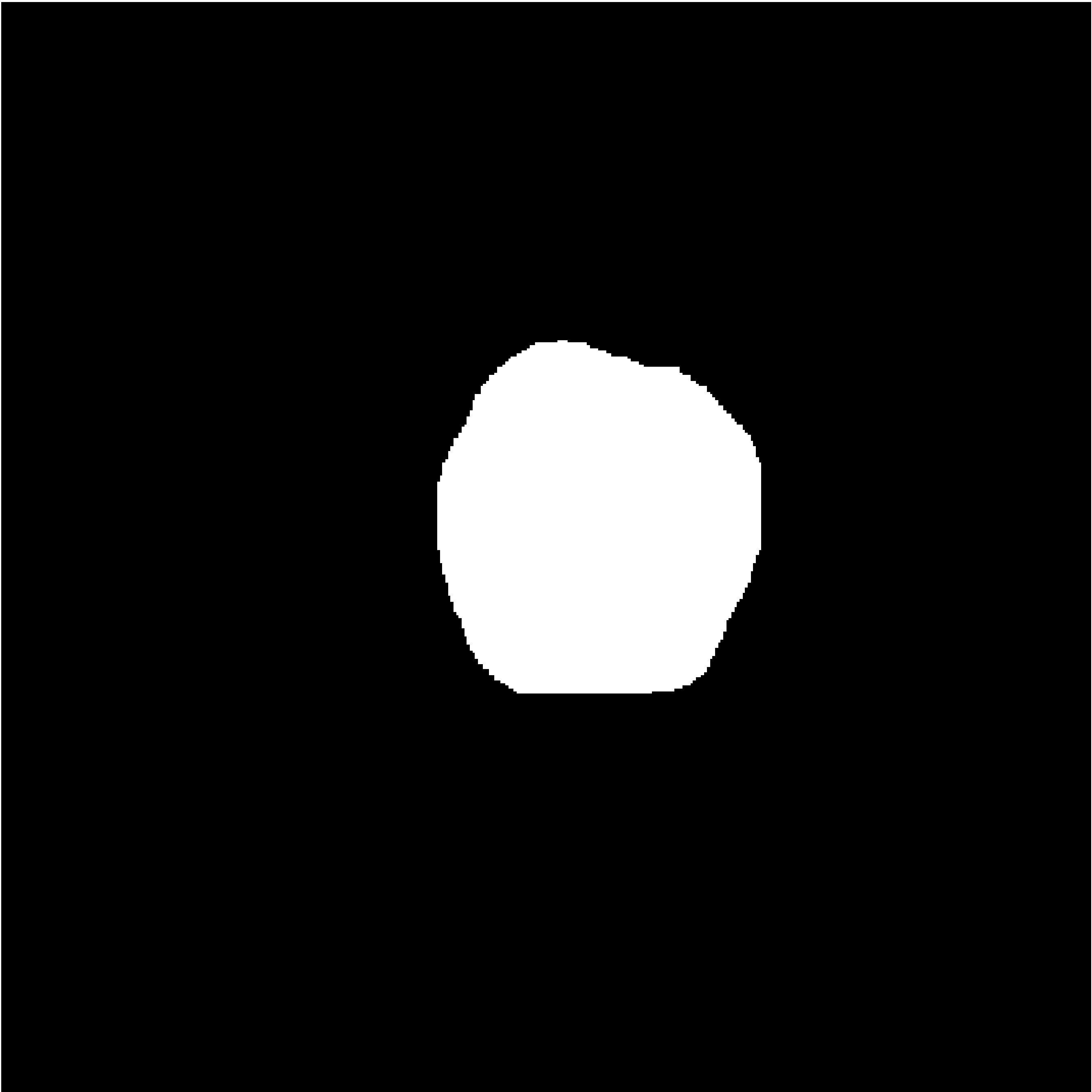}
      \caption{}
      \label{fig:u-186}
  \end{subfigure}

  \caption{Comparison with state-of-the-art models on breast ultrasound images. Row 1: input images with initial contours; Rows 2--7: segmentation results obtained by LIC, TSS, VLSGIS, ICTM-CV, ICTM-LVF-CV, and our model, respectively; The last row presents the ground truth }
  \label{fig:ultra}
\end{figure}

 To assess the performance of our model, we conducted comparative experiments with state-of-the-art methods on the BUSI dataset. In these experiments, all images and the corresponding ground truth were resized to $400\times400$. The segmentation results are presented in Fig.~\ref{fig:ultra} and the corresponding quantitative metrics are summarized in Table~\ref{tab:seg_results}. Due to the severe intensity inhomogeneity in the images, many models often classify the background regions incorrectly as foreground regions. However, both our model and the VLSGIS model successfully segment the lesion regions. Furthermore, the quantitative metrics reported in Table~\ref{tab:seg_results} show that our model achieves more reliable results on images with severe intensity inhomogeneity and noise.

\subsubsection{Three-Phases Segmentation for MR Brain Image}
\begin{figure}
\centering
  \begin{subfigure}[b]{0.16\linewidth}
      \centering
      \includegraphics[width=\textwidth]{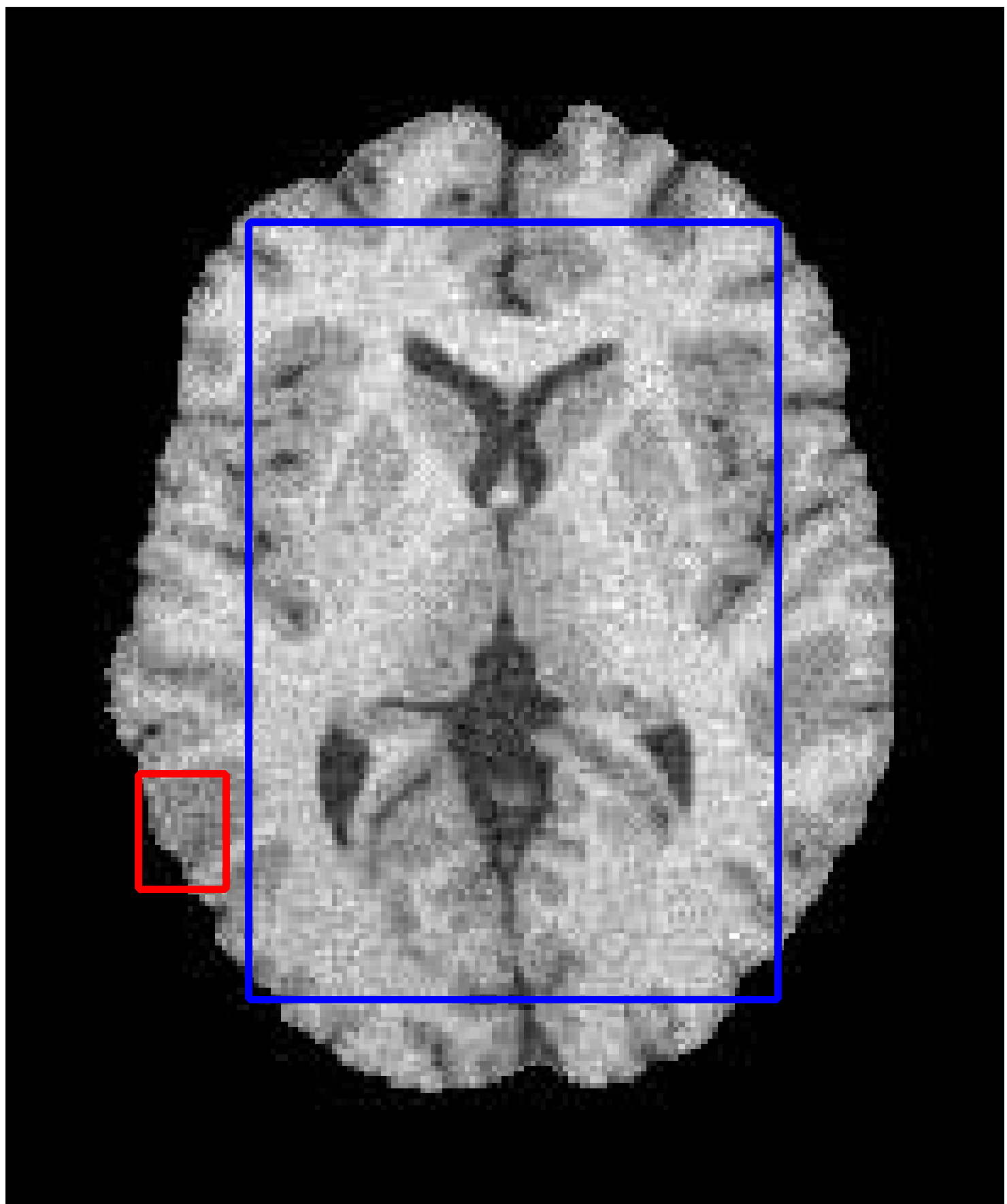}
  \end{subfigure}
  \hfill
  \begin{subfigure}[b]{0.16\linewidth}
      \centering
      \includegraphics[width=\textwidth]{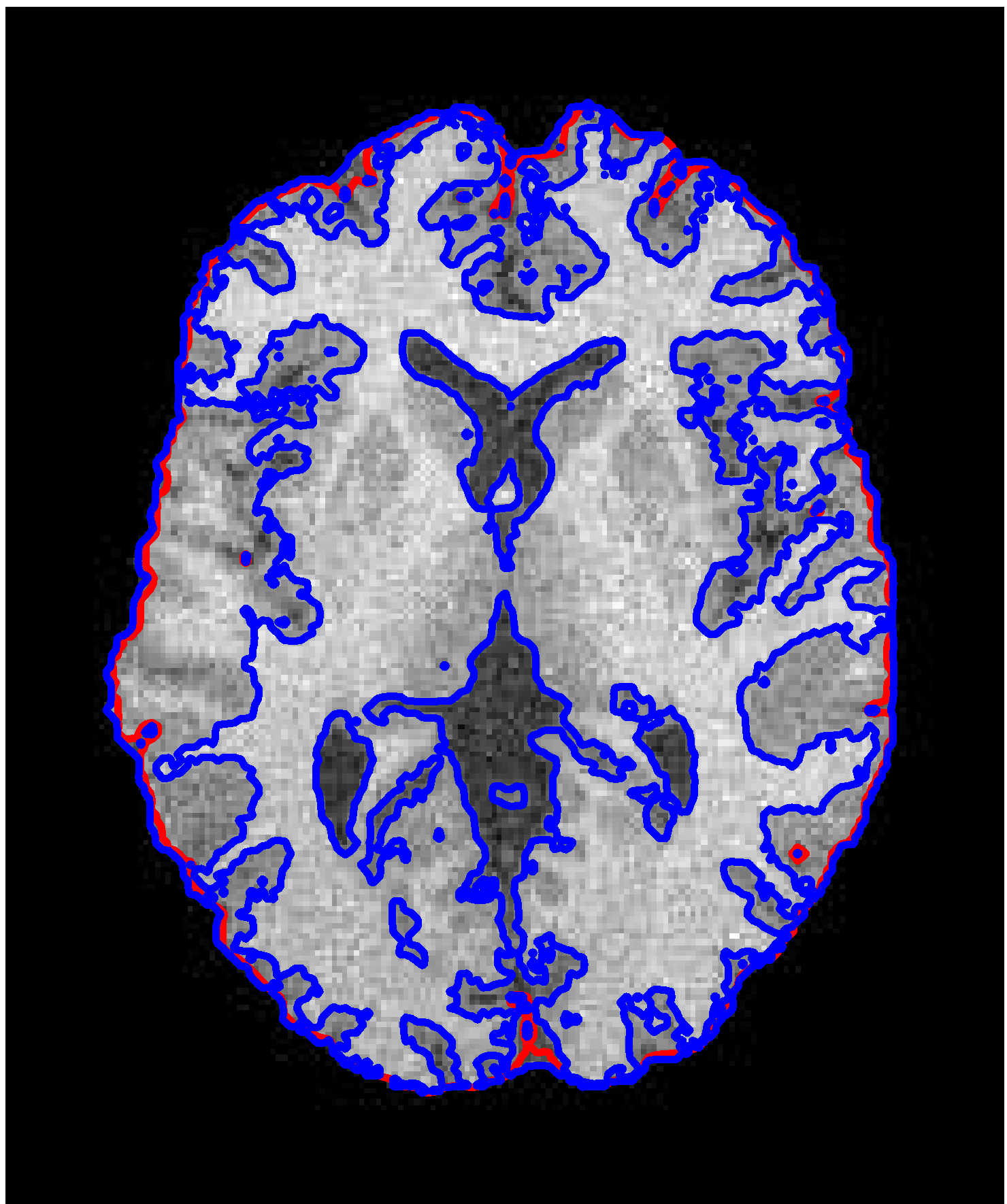}
  \end{subfigure}
  \hfill
  \begin{subfigure}[b]{0.16\linewidth}
      \centering
      \includegraphics[width=\textwidth]{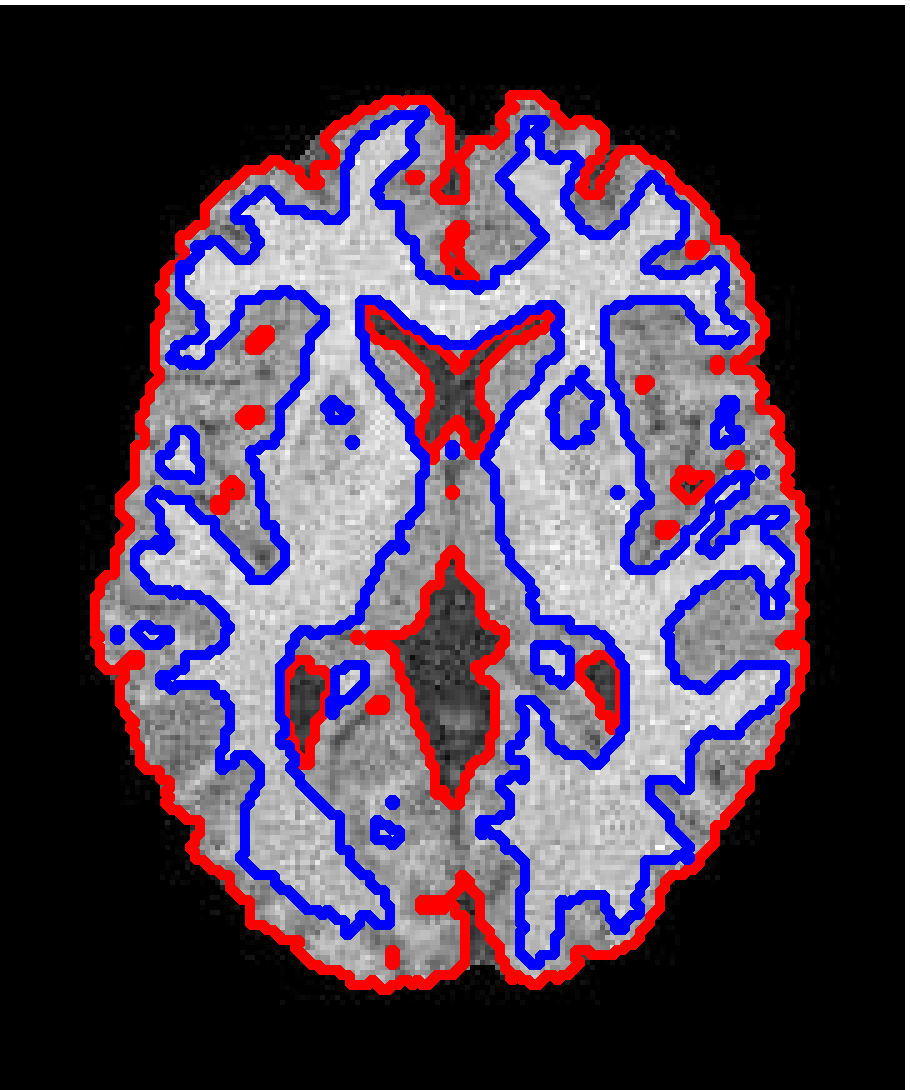}
  \end{subfigure}
 \hfill
  \begin{subfigure}[b]{0.16\linewidth}
      \centering
      \includegraphics[width=\textwidth]{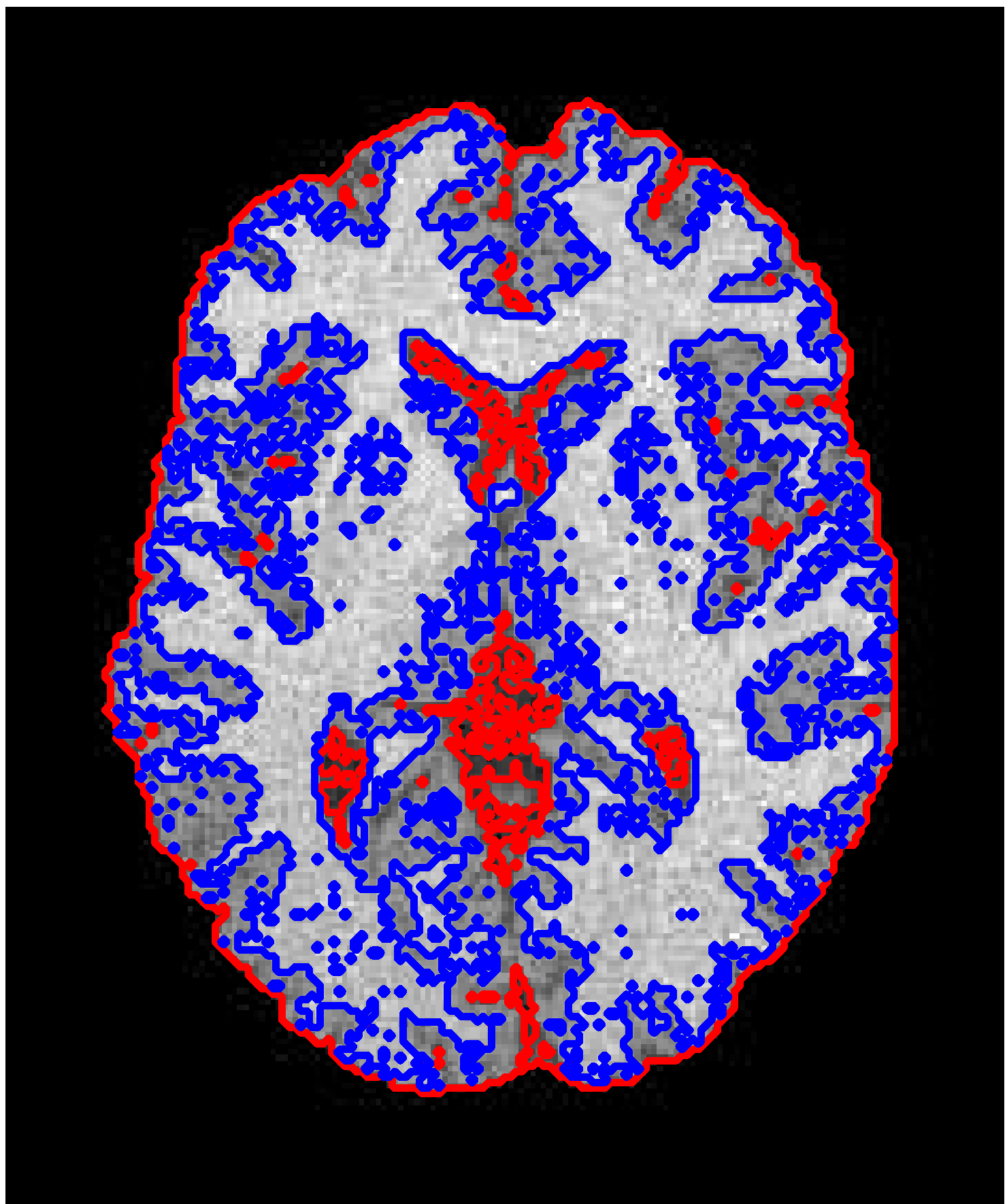}
  \end{subfigure}
   \hfill
  \begin{subfigure}[b]{0.16\linewidth}
      \centering
      \includegraphics[width=\textwidth]{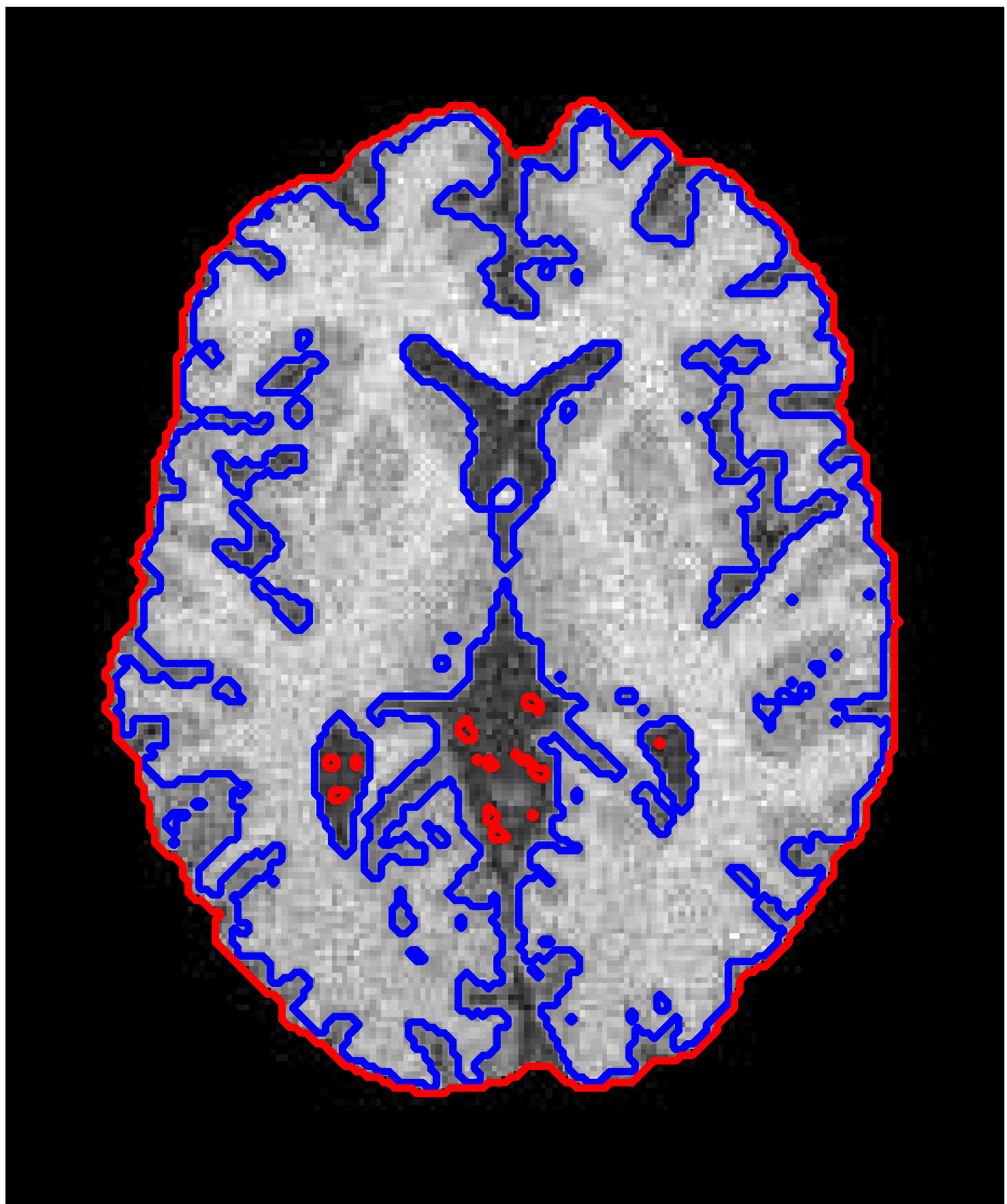}
  \end{subfigure}
 \hfill
  \begin{subfigure}[b]{0.16\linewidth}
      \centering
      \includegraphics[width=\textwidth]{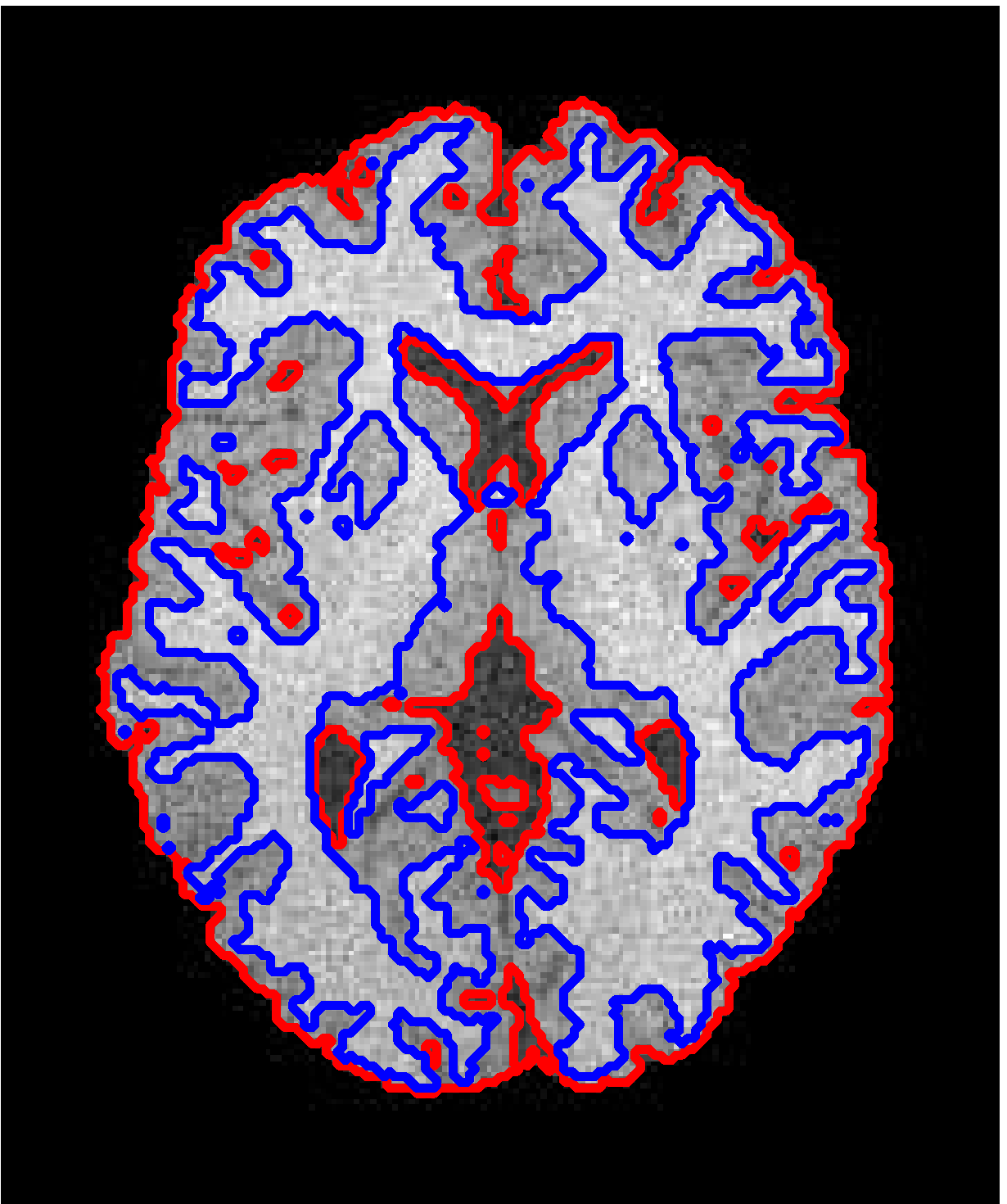}
  \end{subfigure}

 \begin{subfigure}[b]{0.16\linewidth}
      \centering
      \includegraphics[width=\textwidth]{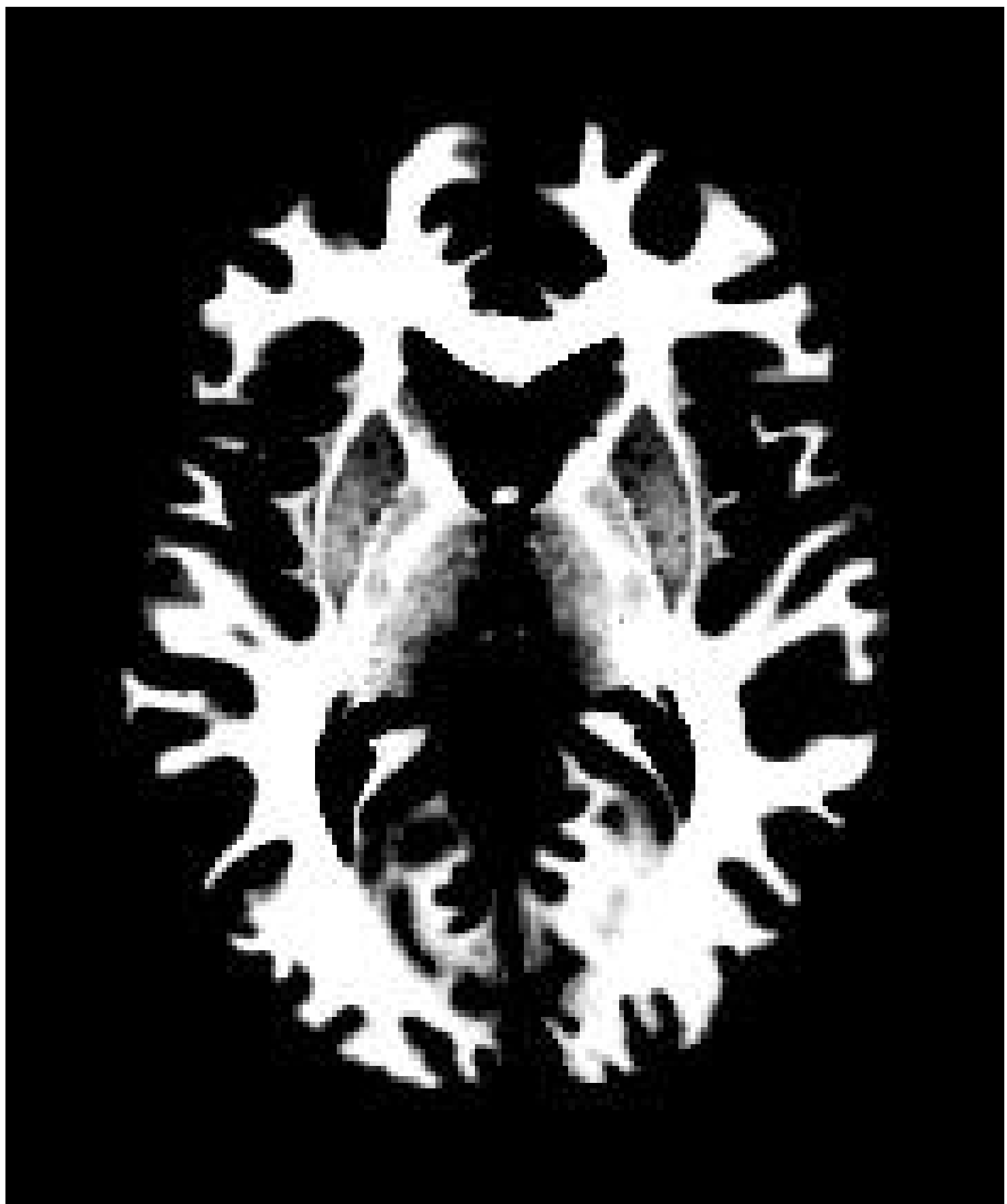}
  \end{subfigure}
   \hfill
  \begin{subfigure}[b]{0.16\linewidth}
      \centering
      \includegraphics[width=\textwidth]{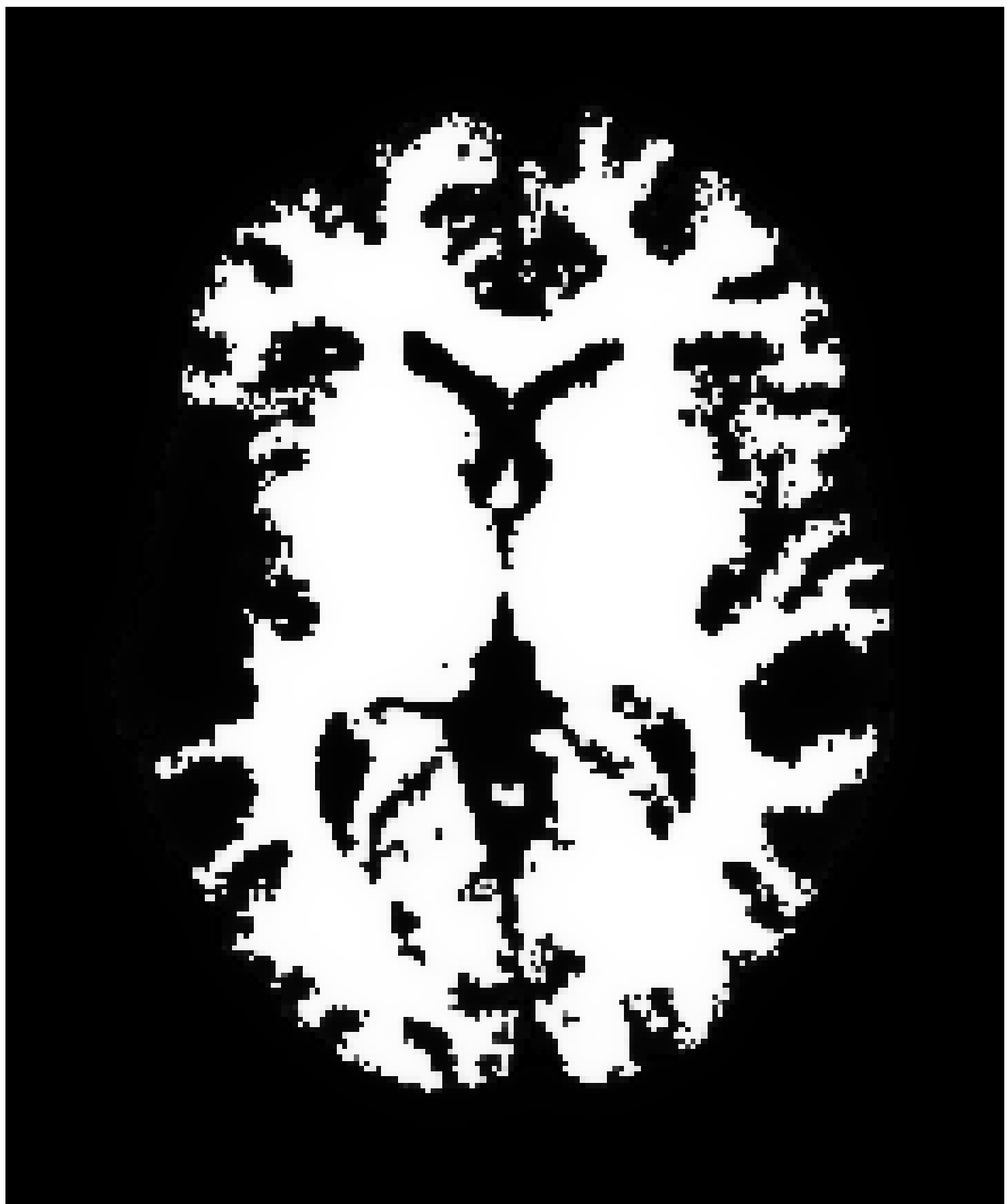}
  \end{subfigure}
  \hfill
  \begin{subfigure}[b]{0.16\linewidth}
      \centering
      \includegraphics[width=\textwidth]{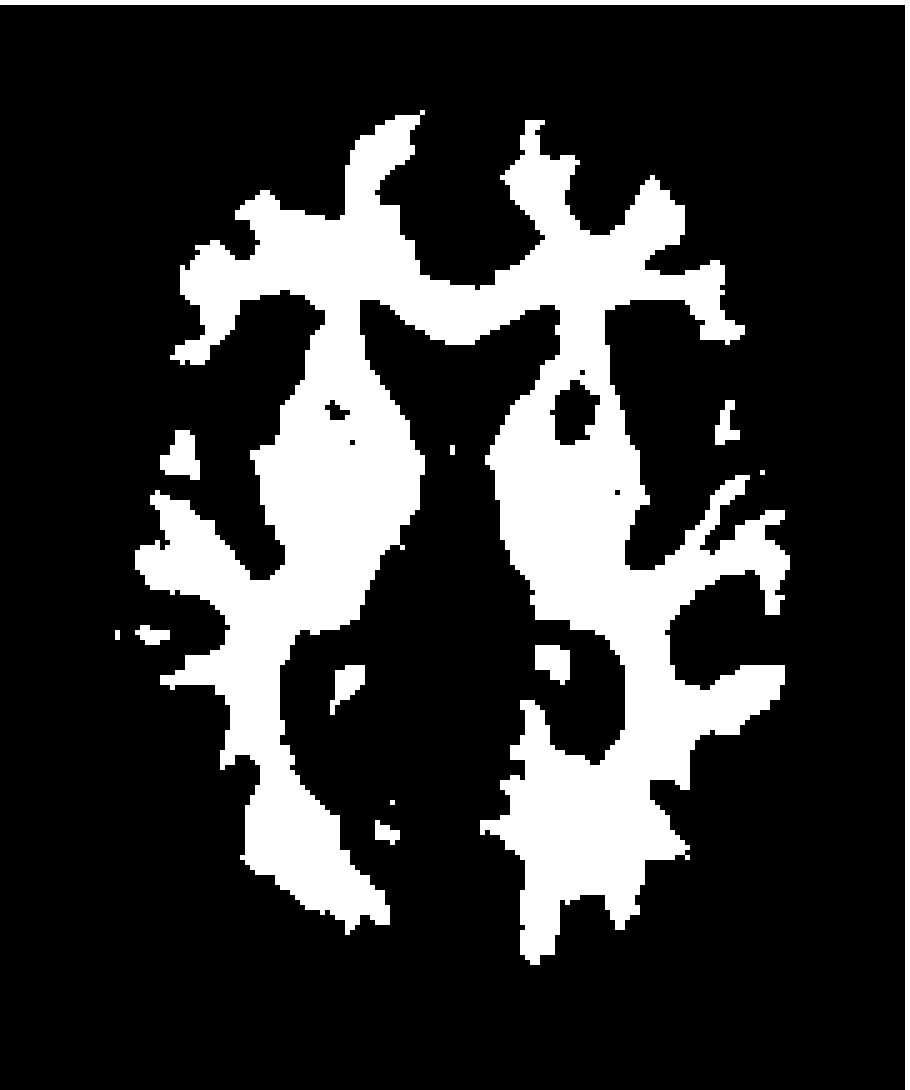}
  \end{subfigure}
   \hfill
  \begin{subfigure}[b]{0.16\linewidth}
      \centering
      \includegraphics[width=\textwidth]{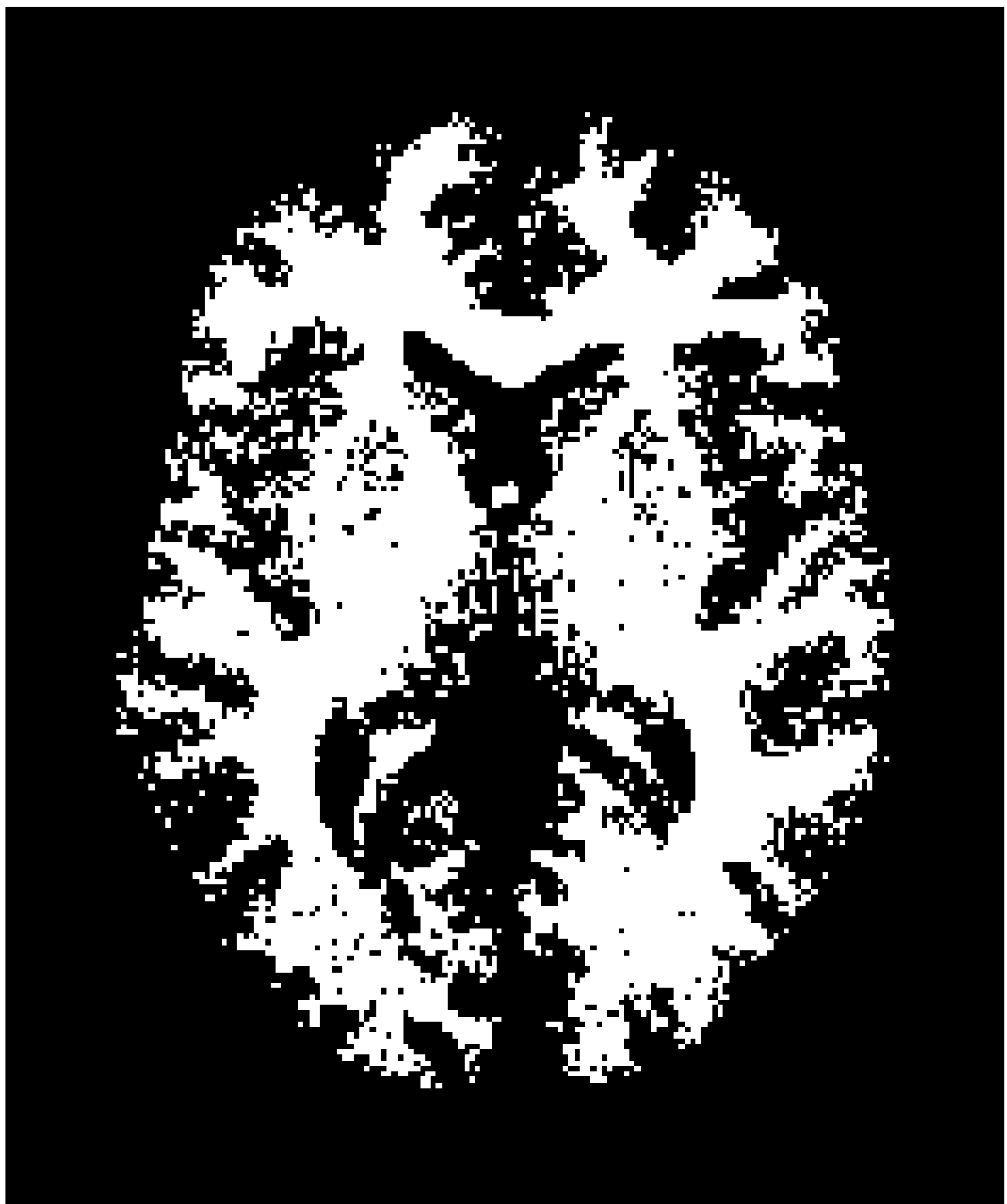}
  \end{subfigure}
  \hfill
  \begin{subfigure}[b]{0.16\linewidth}
      \centering
      \includegraphics[width=\textwidth]{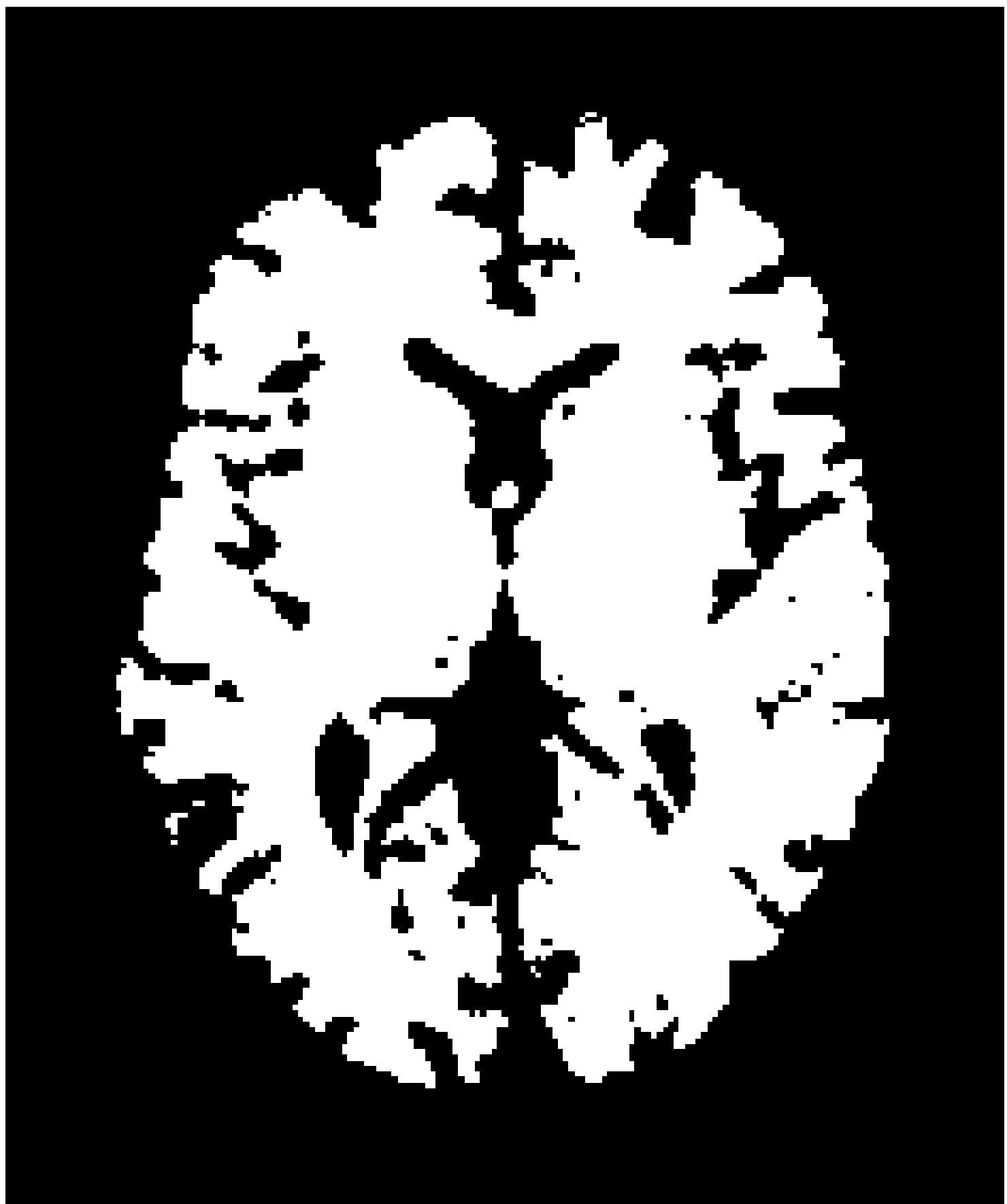}
  \end{subfigure}
   \hfill
  \begin{subfigure}[b]{0.16\linewidth}
      \centering
      \includegraphics[width=\textwidth]{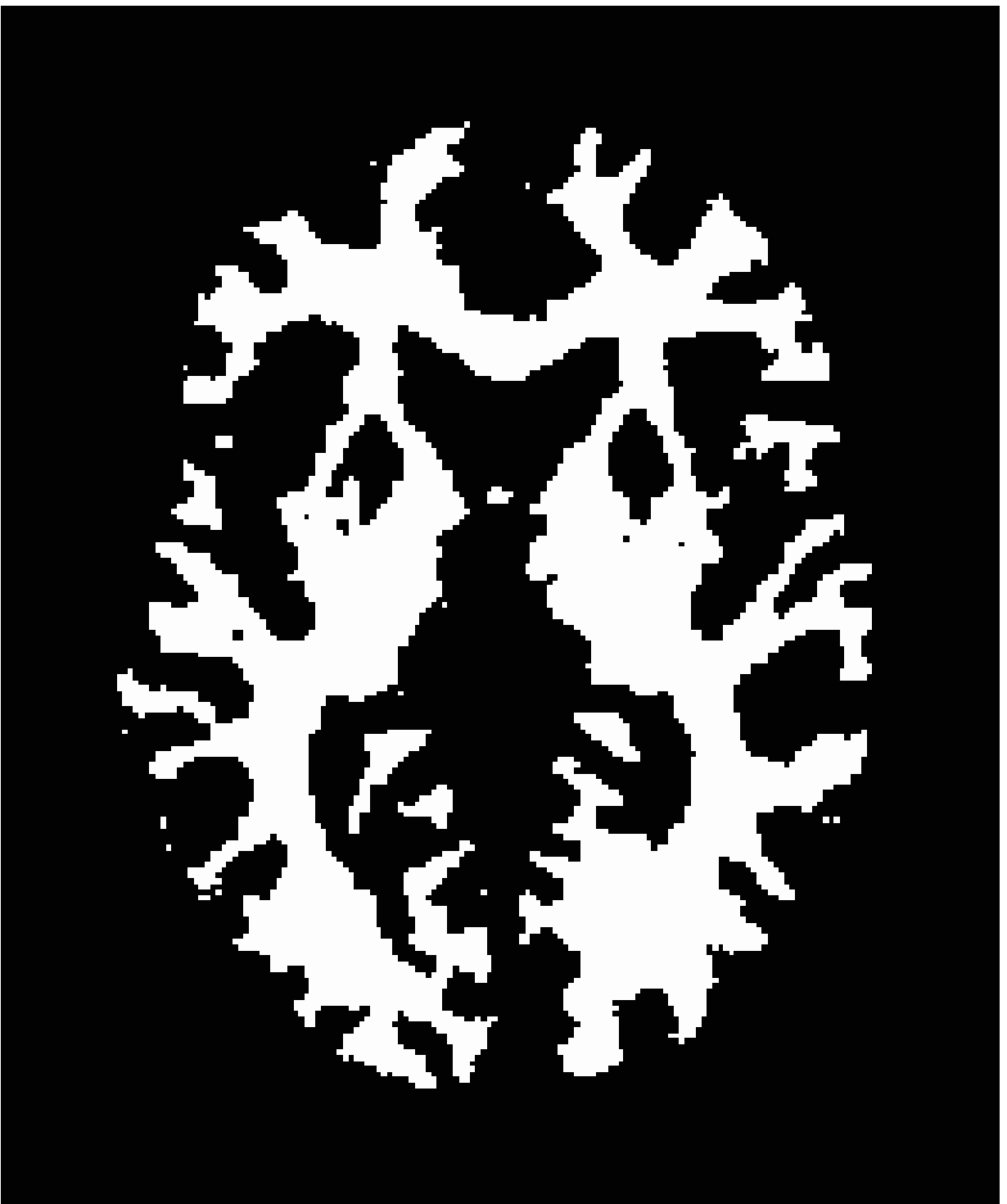}
  \end{subfigure}

 \begin{subfigure}[b]{0.16\linewidth}
      \centering
      \includegraphics[width=\textwidth]{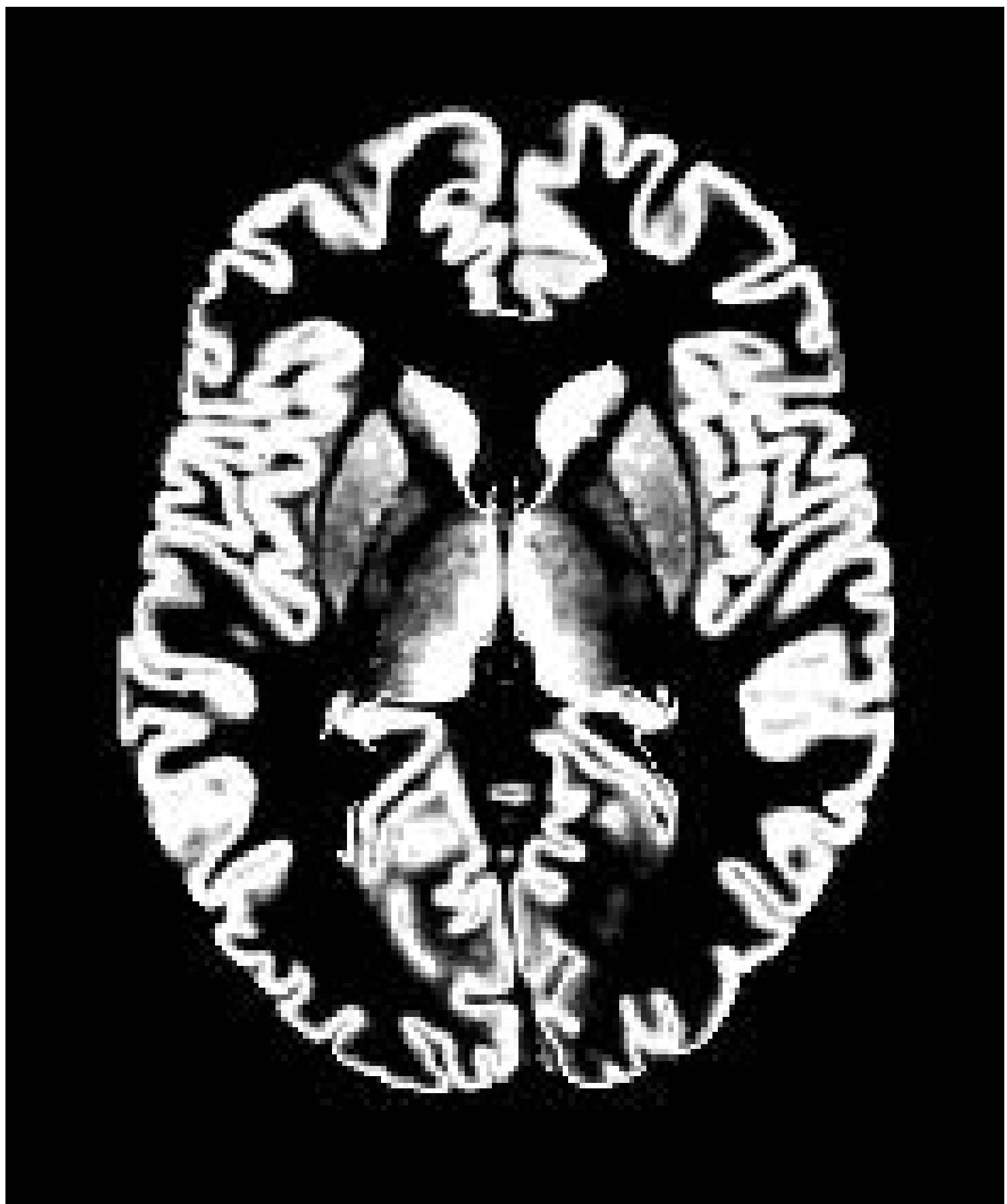}
      \caption{}
      \label{fig:76-initial}
  \end{subfigure}
   \hfill
  \begin{subfigure}[b]{0.16\linewidth}
      \centering
      \includegraphics[width=\textwidth]{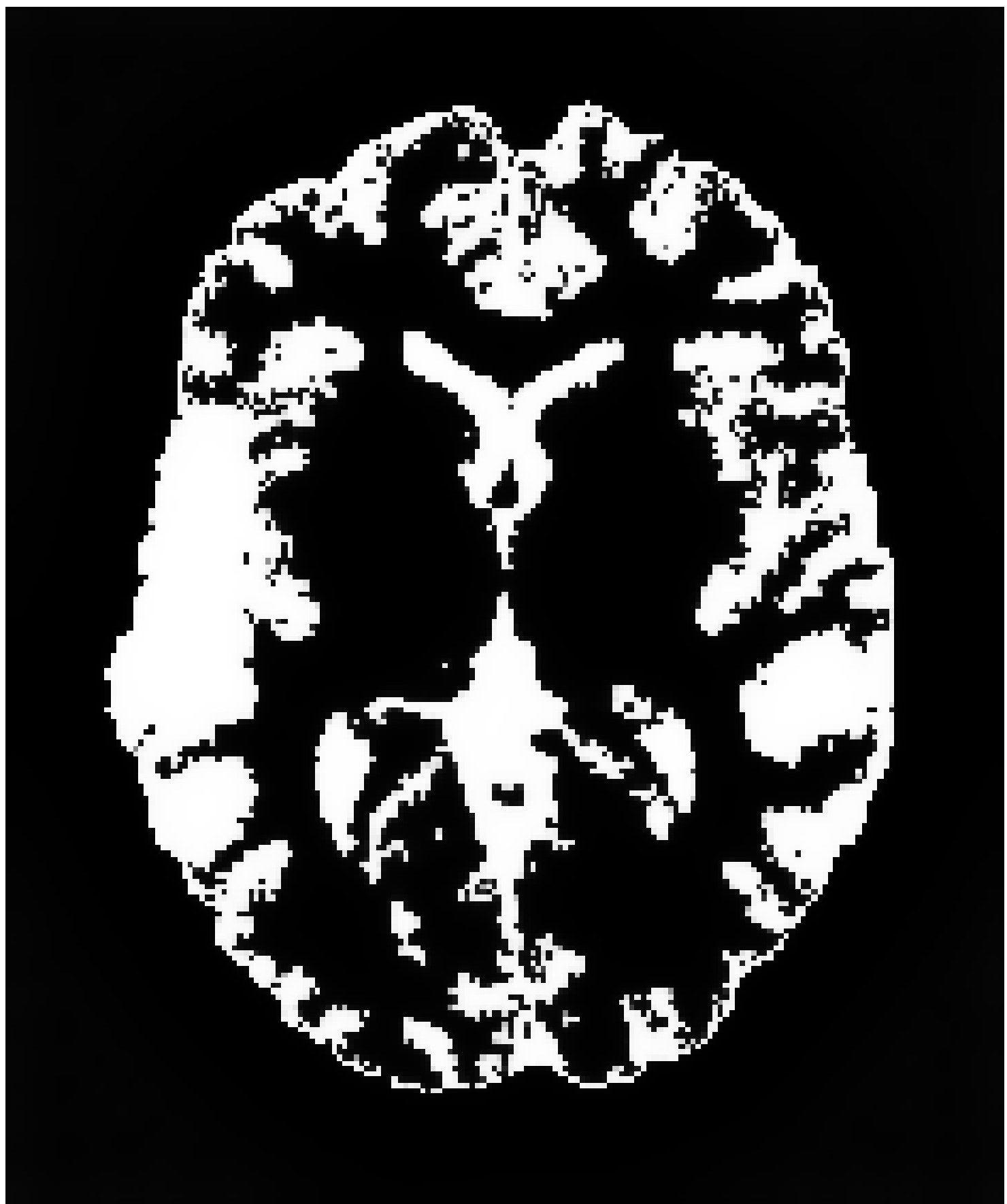}
       \caption{}
       \label{fig:76-LIC-seg}
  \end{subfigure}
  \hfill
  \begin{subfigure}[b]{0.16\linewidth}
      \centering
      \includegraphics[width=\textwidth]{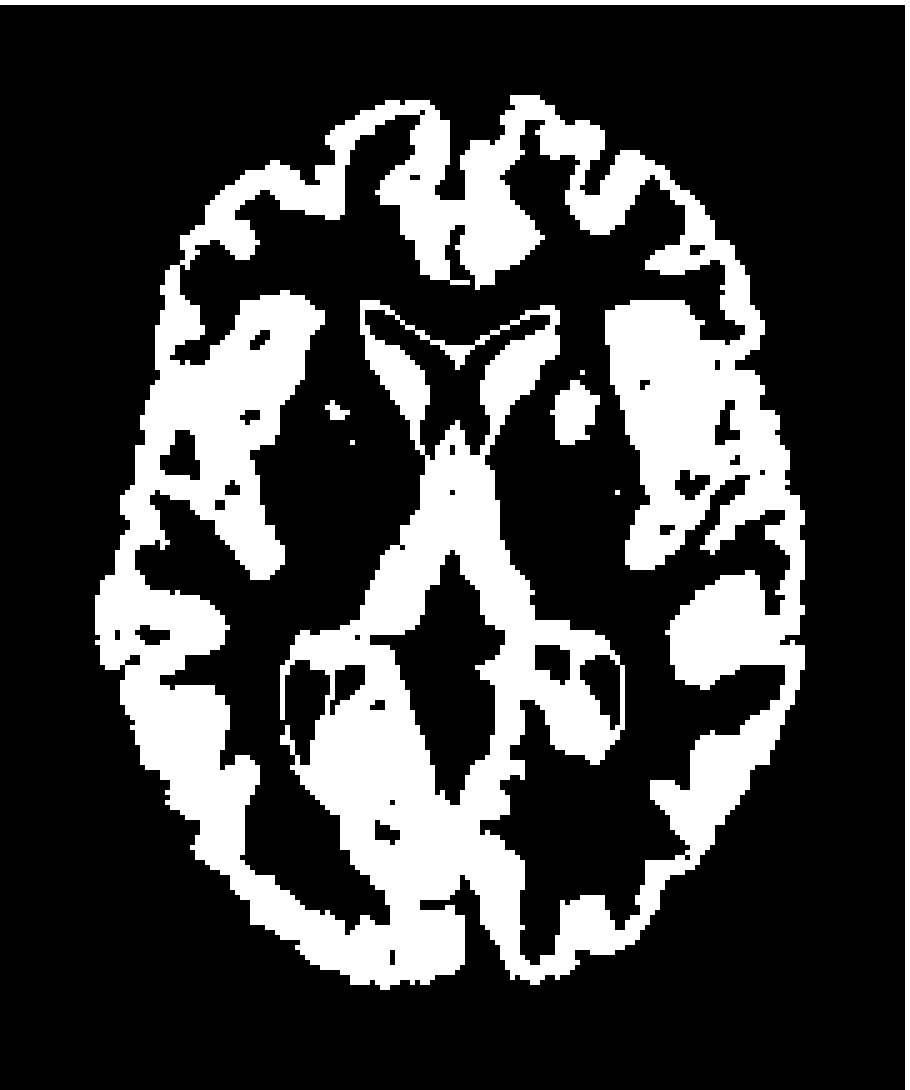}
       \caption{}
       \label{fig:76-TSS-seg}
  \end{subfigure}
   \hfill
  \begin{subfigure}[b]{0.16\linewidth}
      \centering
      \includegraphics[width=\textwidth]{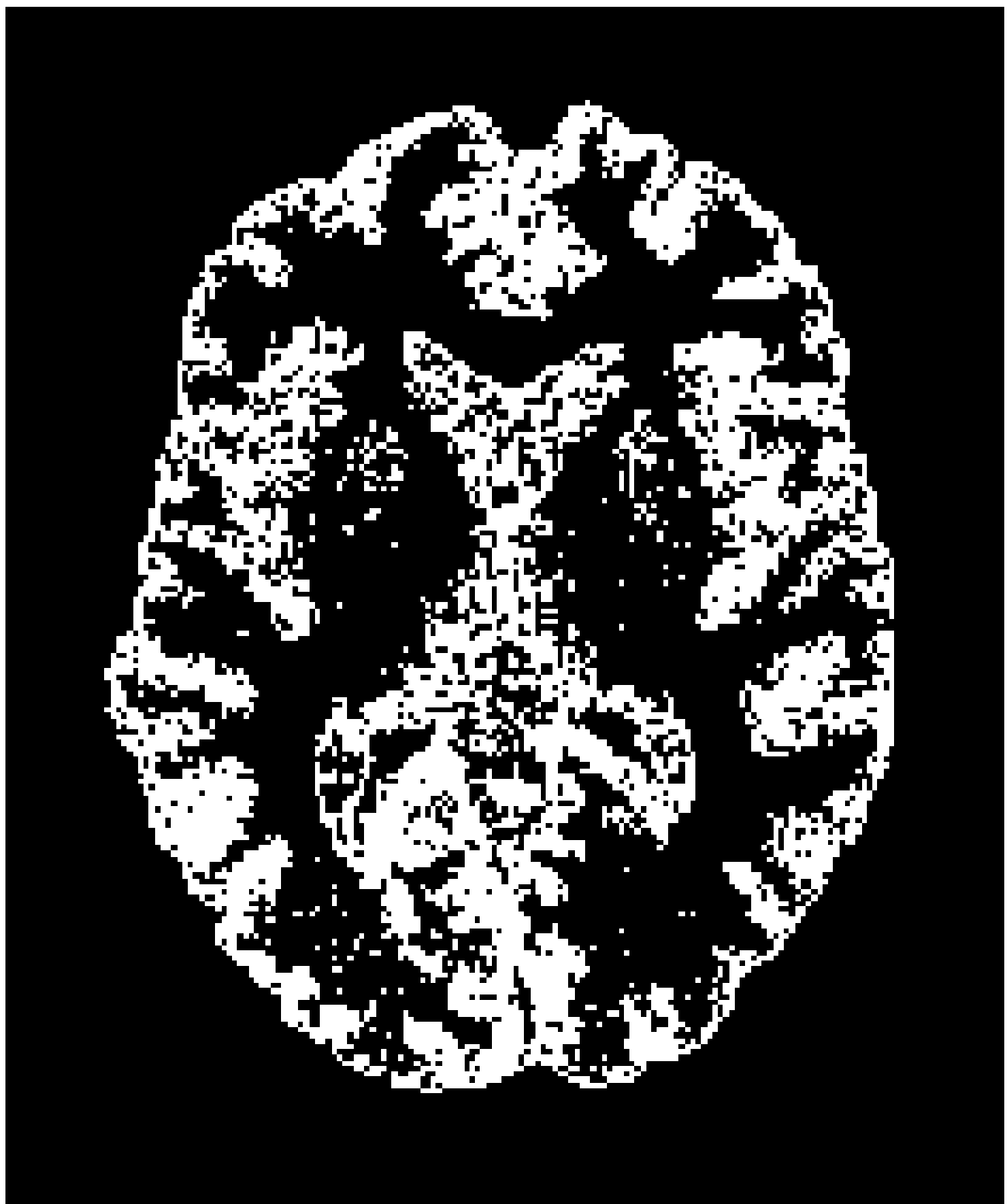}
       \caption{}
       \label{fig:76-ICTMCV-seg}
  \end{subfigure}
   \hfill
   \begin{subfigure}[b]{0.16\linewidth}
      \centering
      \includegraphics[width=\textwidth]{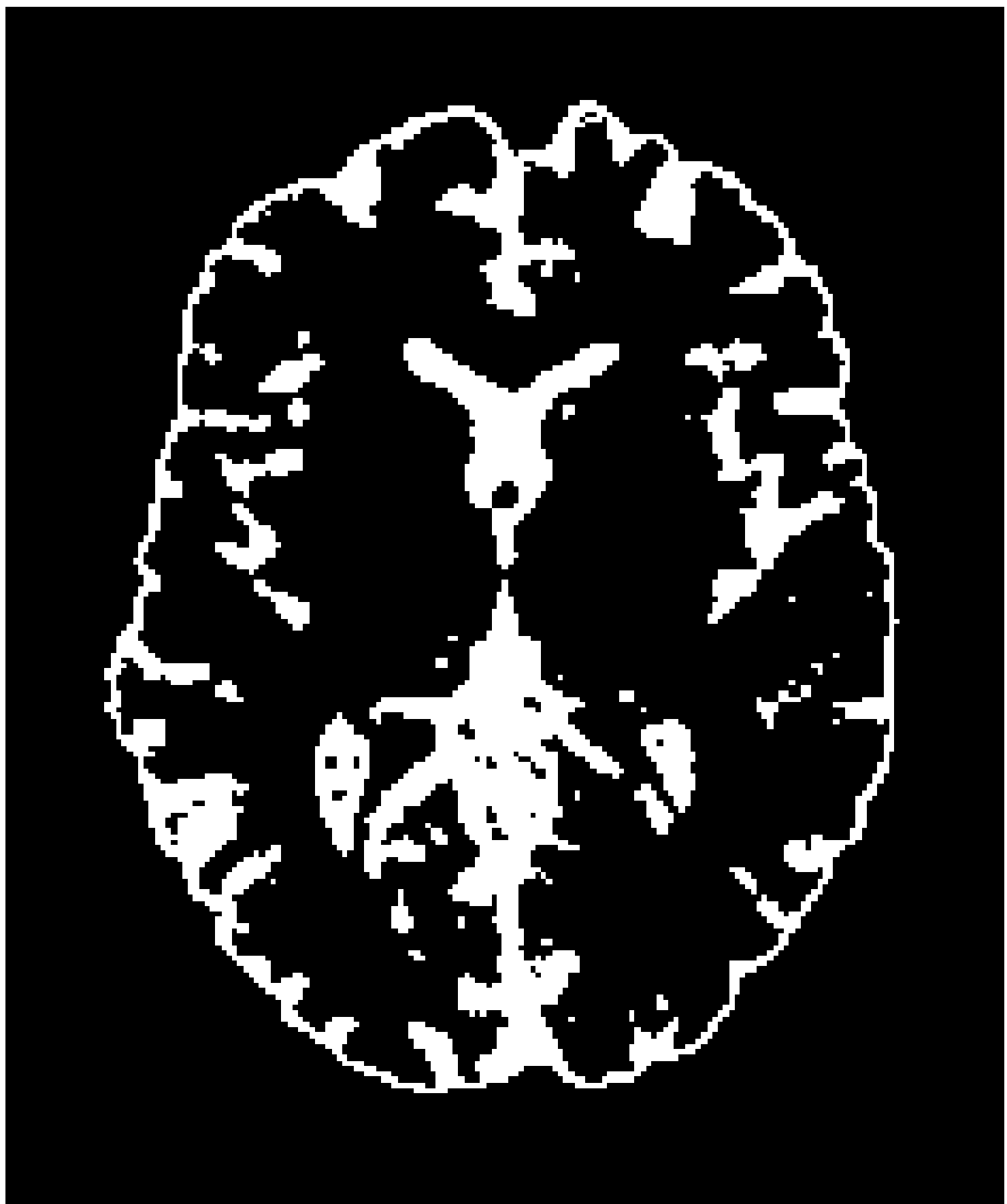}
       \caption{}
       \label{fig:76-LVFCV-seg}
  \end{subfigure}
   \hfill
  \begin{subfigure}[b]{0.16\linewidth}
      \centering
      \includegraphics[width=\textwidth]{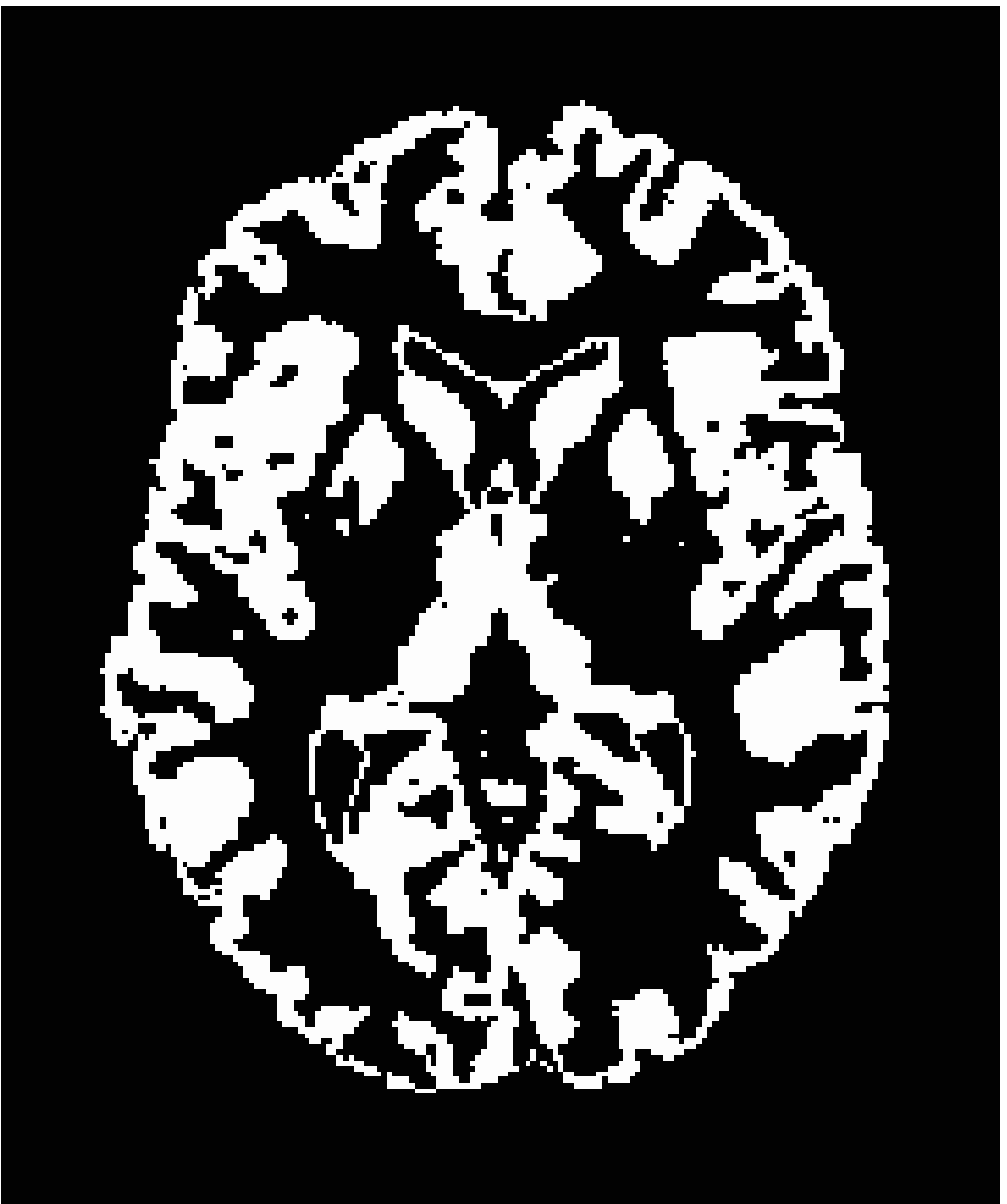}
       \caption{}
       \label{fig:76-our-seg}
  \end{subfigure}

  \begin{subfigure}[b]{0.16\linewidth}
      \centering
      \includegraphics[width=\textwidth]{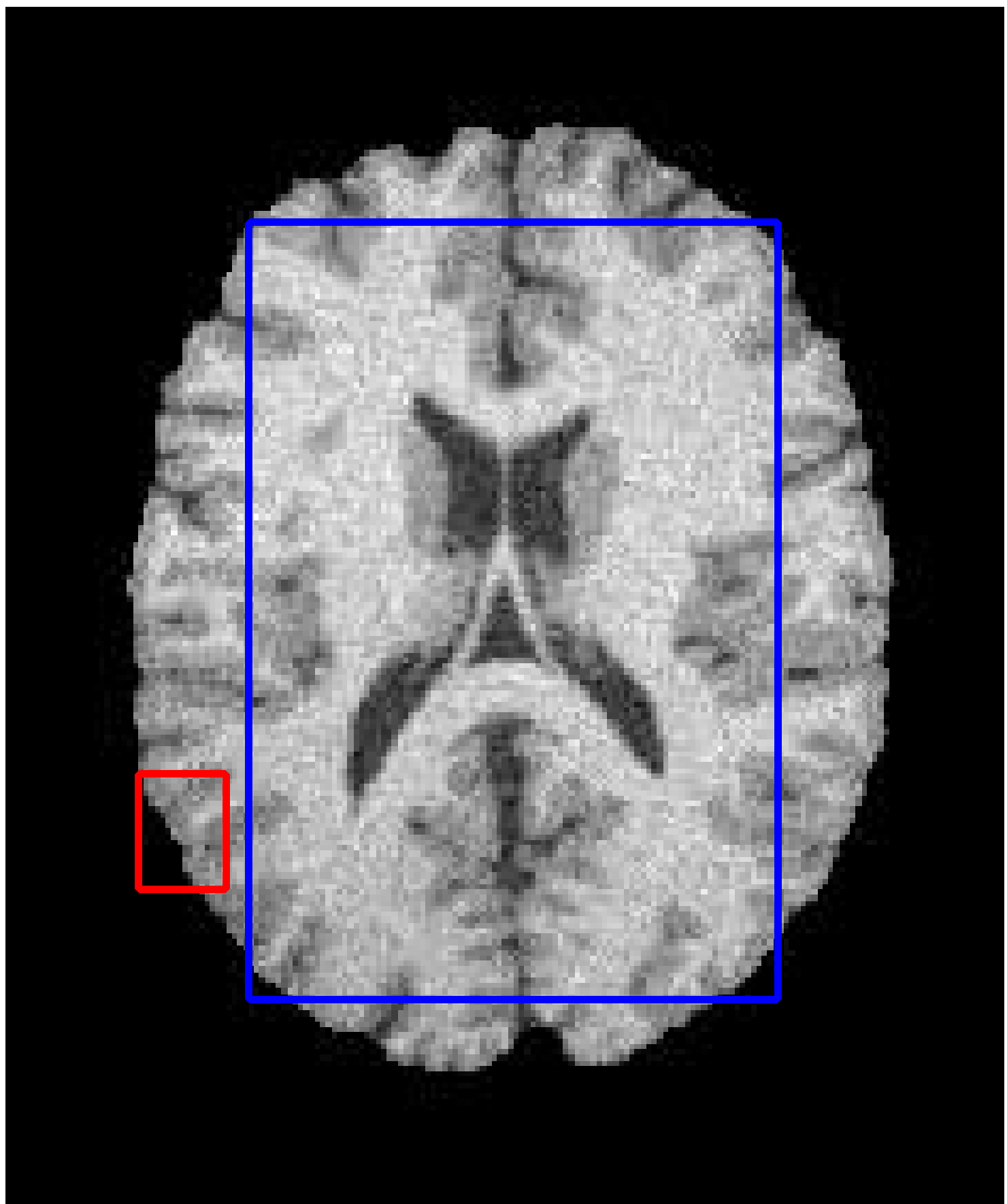}
  \end{subfigure}
  \hfill
  \begin{subfigure}[b]{0.16\linewidth}
      \centering
      \includegraphics[width=\textwidth]{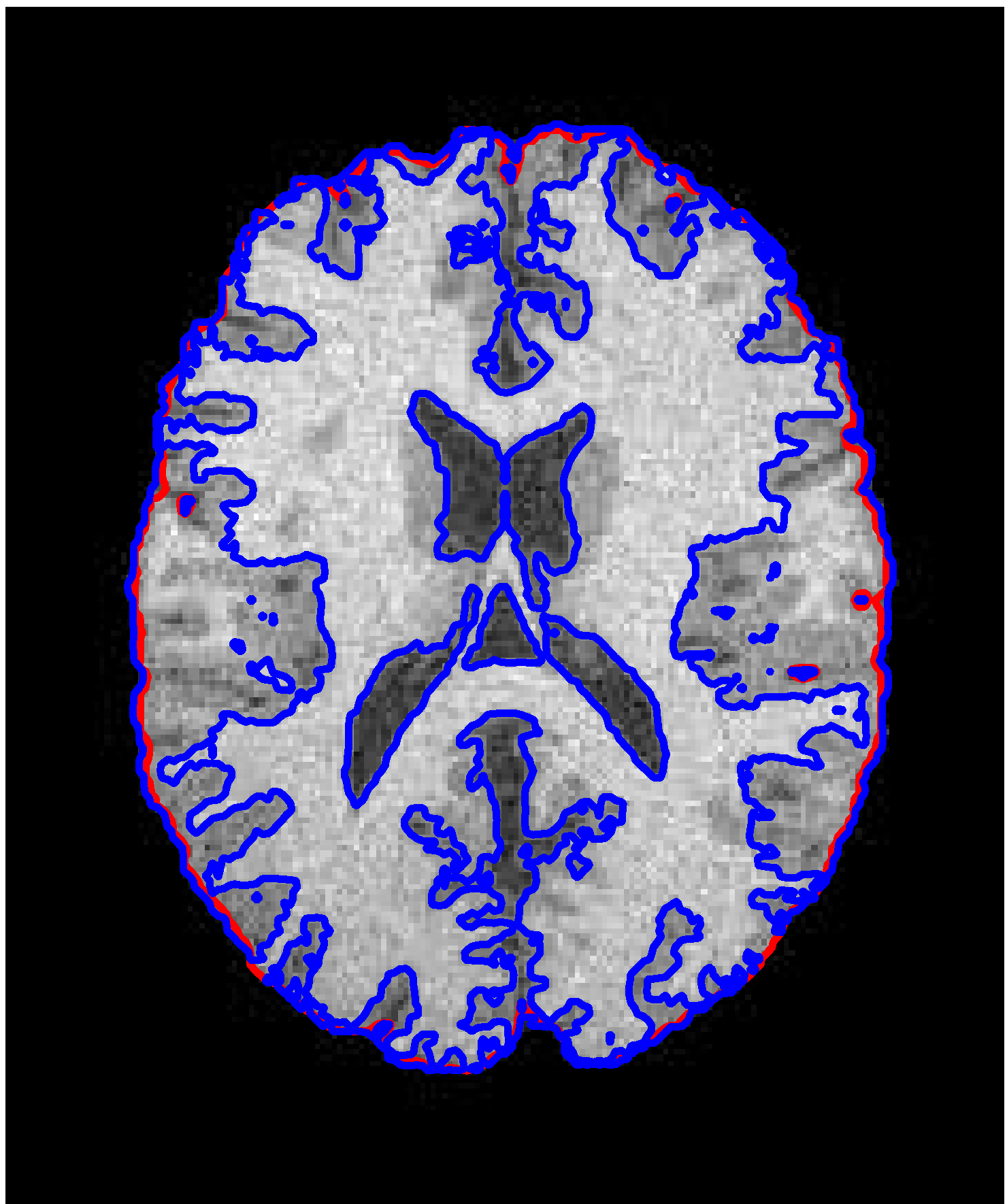}
  \end{subfigure}
  \hfill
  \begin{subfigure}[b]{0.16\linewidth}
      \centering
      \includegraphics[width=\textwidth]{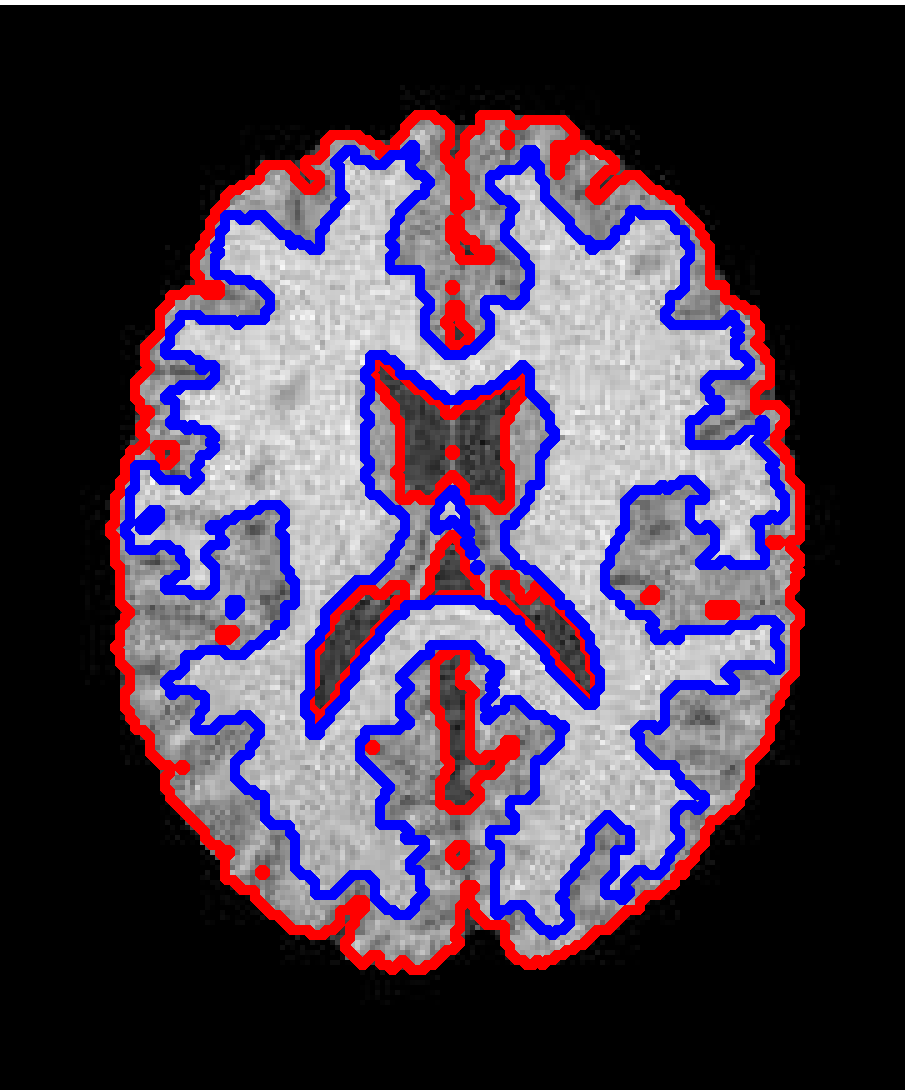}
  \end{subfigure}
 \hfill
  \begin{subfigure}[b]{0.16\linewidth}
      \centering
      \includegraphics[width=\textwidth]{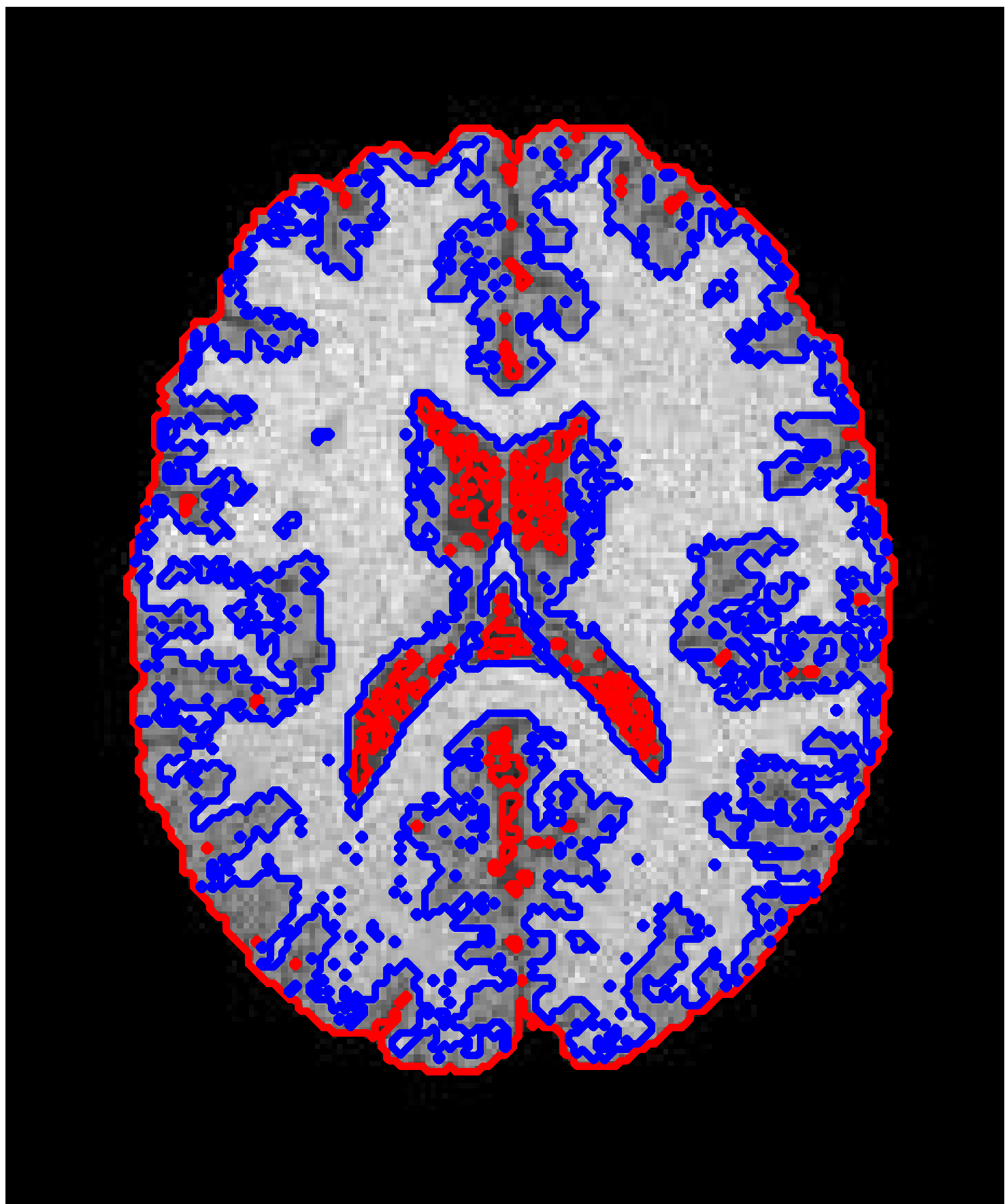}
  \end{subfigure}
  \hfill
  \begin{subfigure}[b]{0.16\linewidth}
      \centering
      \includegraphics[width=\textwidth]{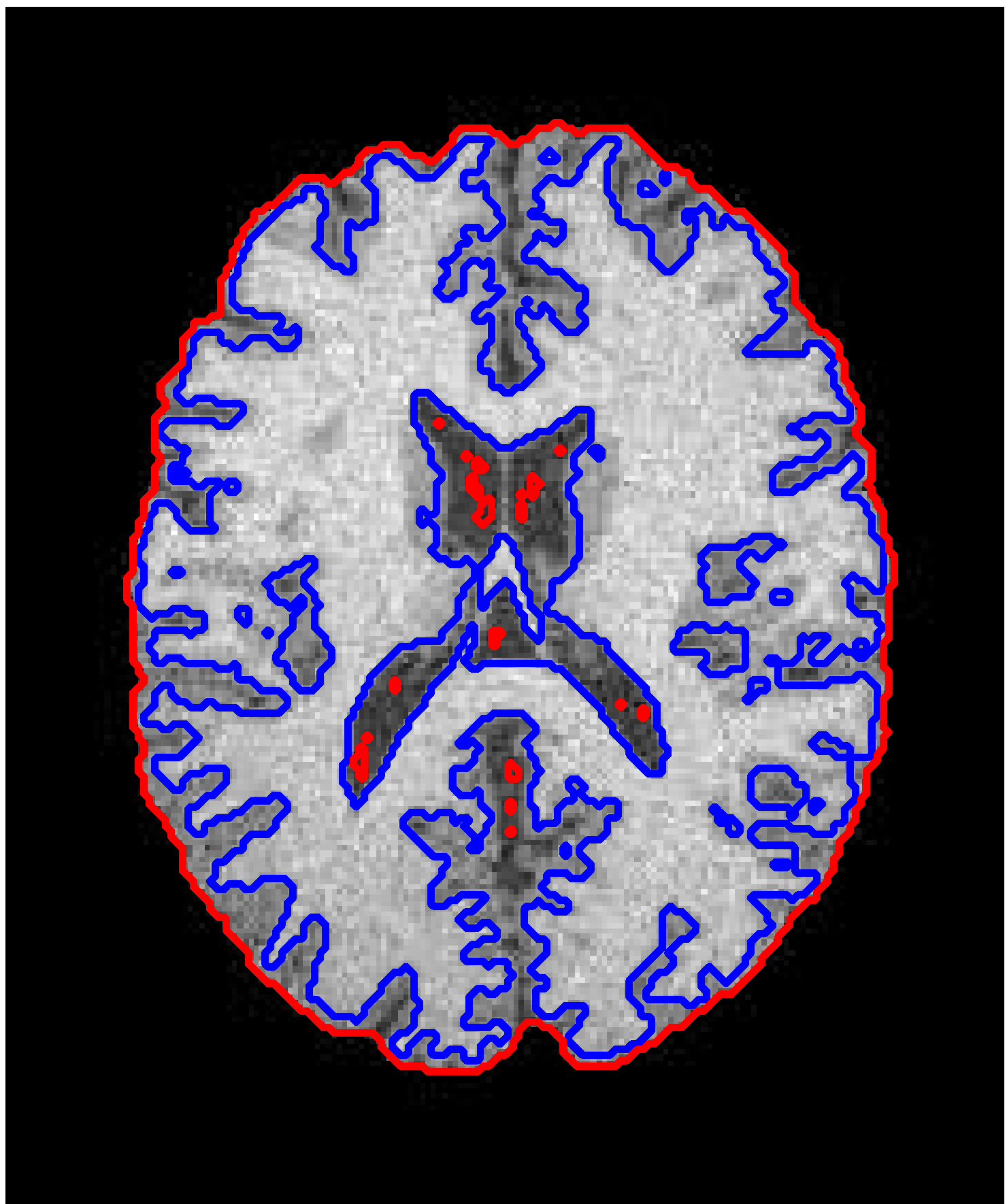}
  \end{subfigure}
 \hfill
  \begin{subfigure}[b]{0.16\linewidth}
      \centering
      \includegraphics[width=\textwidth]{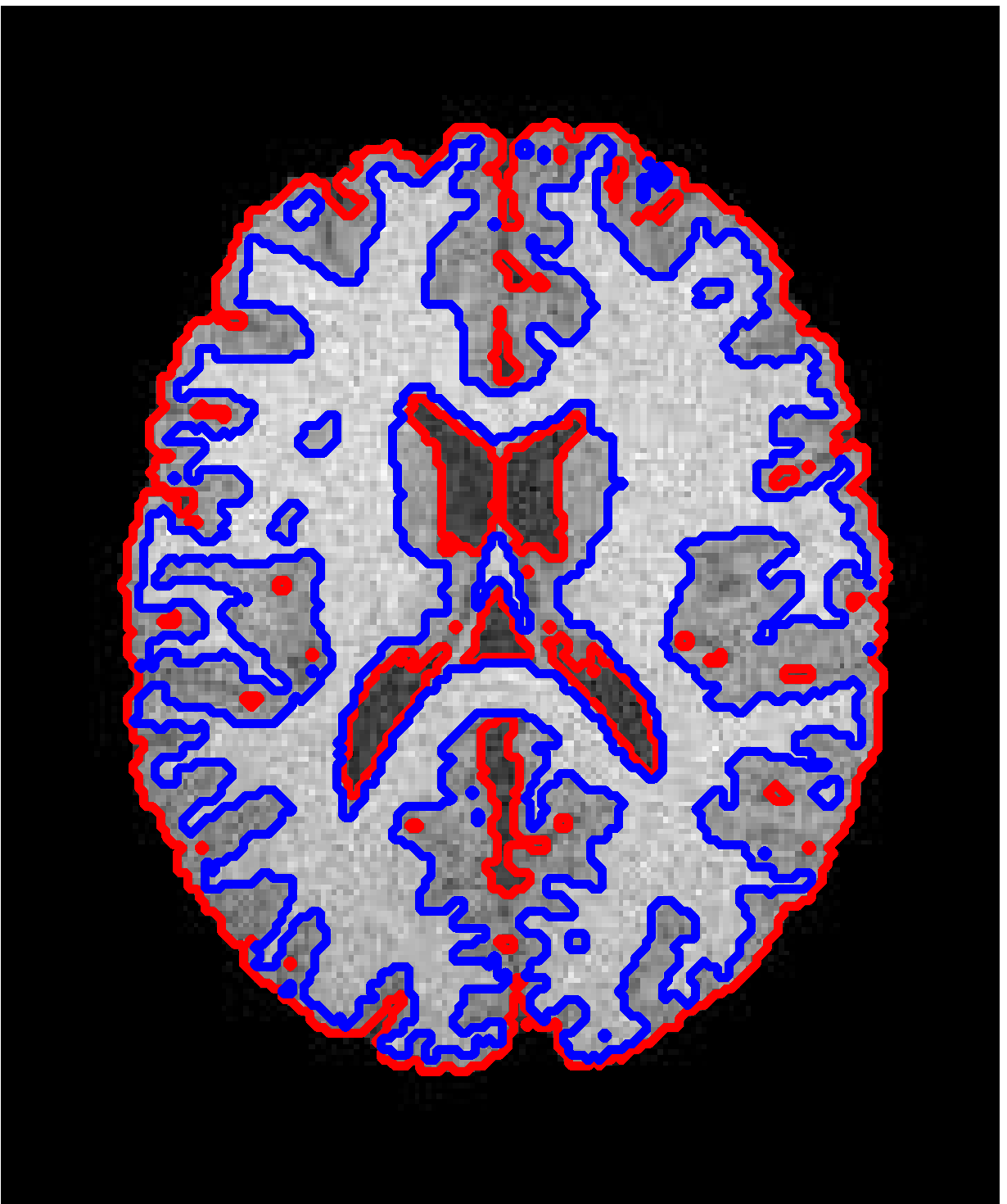}
  \end{subfigure}

 \begin{subfigure}[b]{0.16\linewidth}
      \centering
      \includegraphics[width=\textwidth]{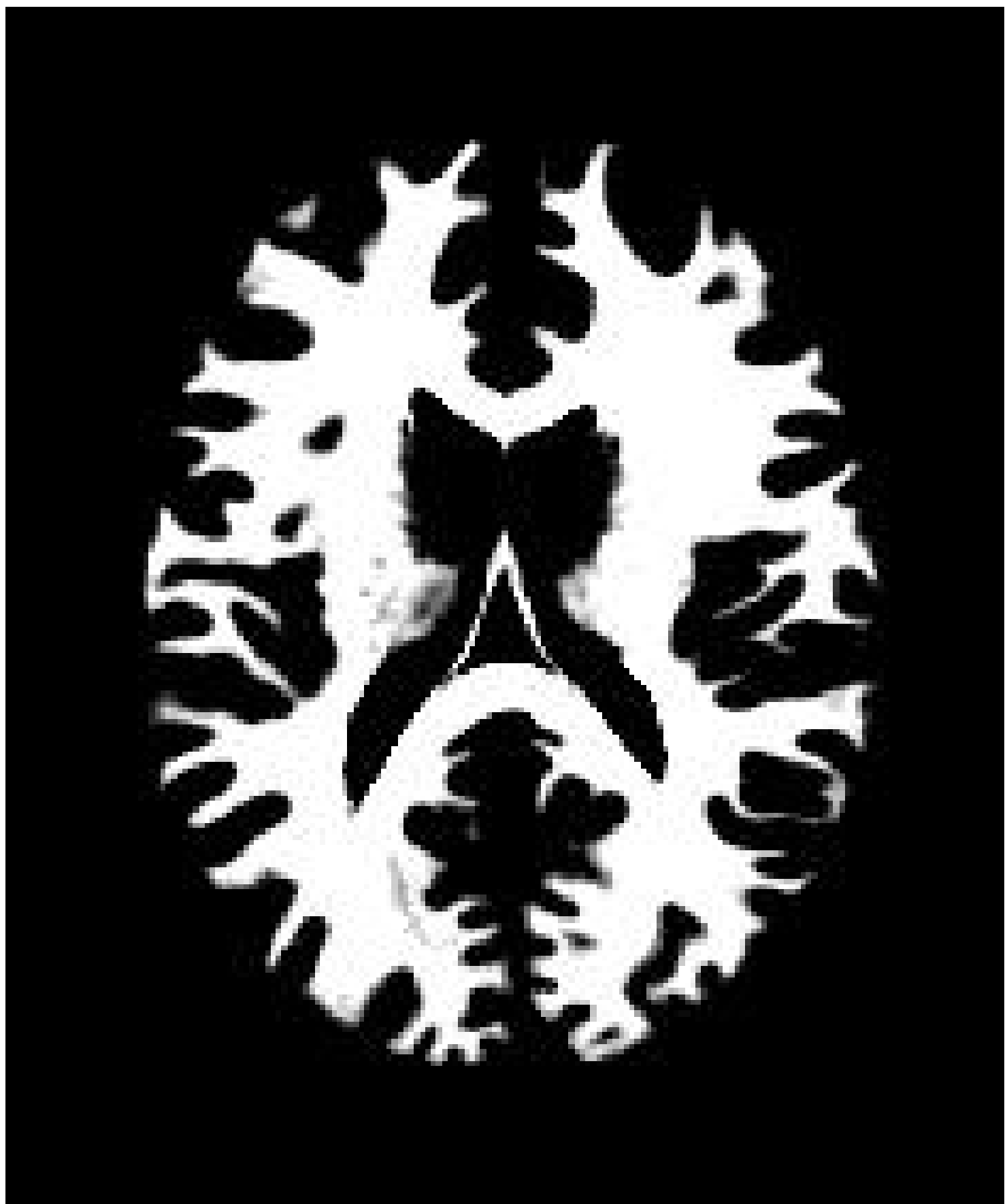}
  \end{subfigure}
   \hfill
  \begin{subfigure}[b]{0.16\linewidth}
      \centering
      \includegraphics[width=\textwidth]{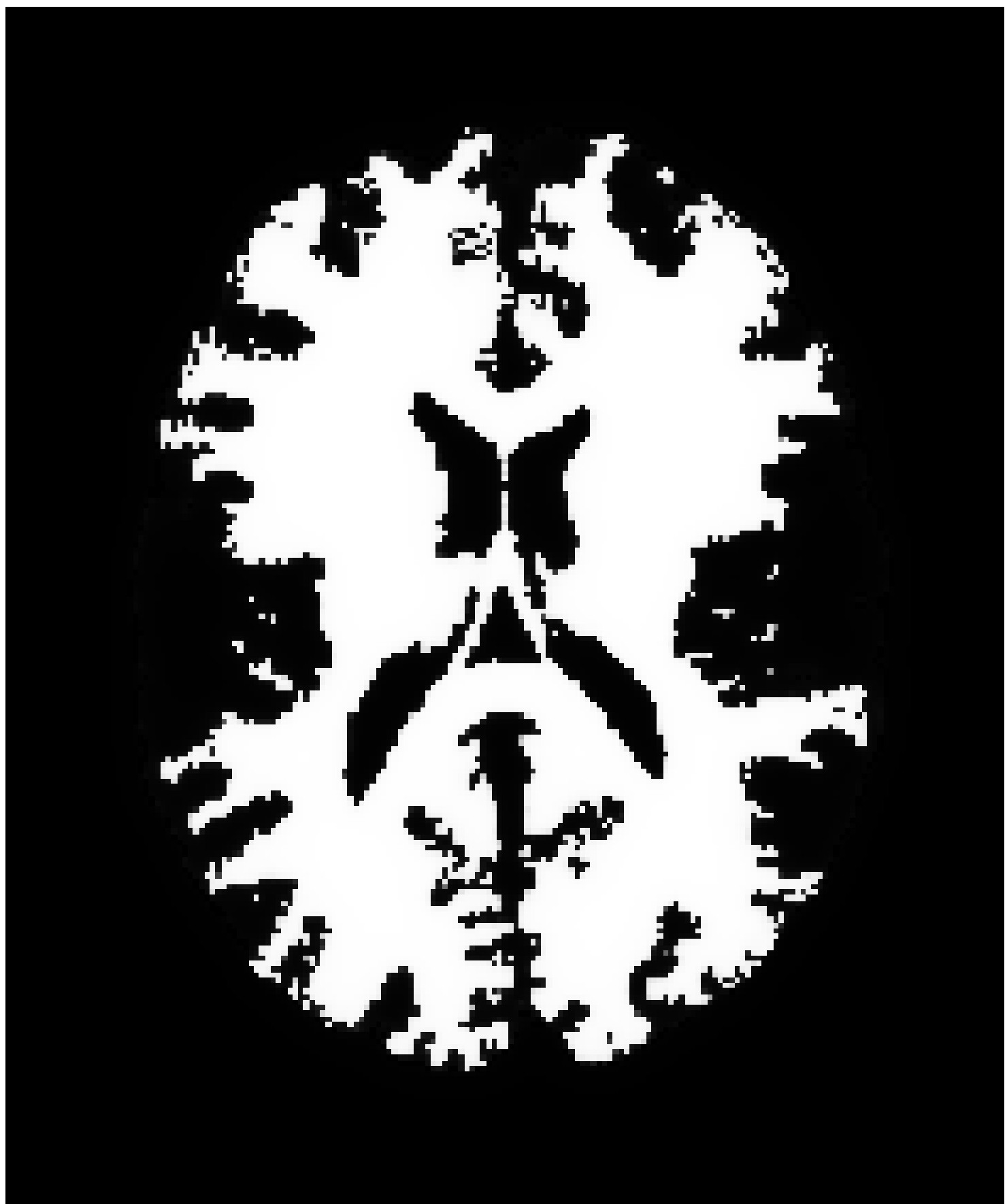}
  \end{subfigure}
  \hfill
  \begin{subfigure}[b]{0.16\linewidth}
      \centering
      \includegraphics[width=\textwidth]{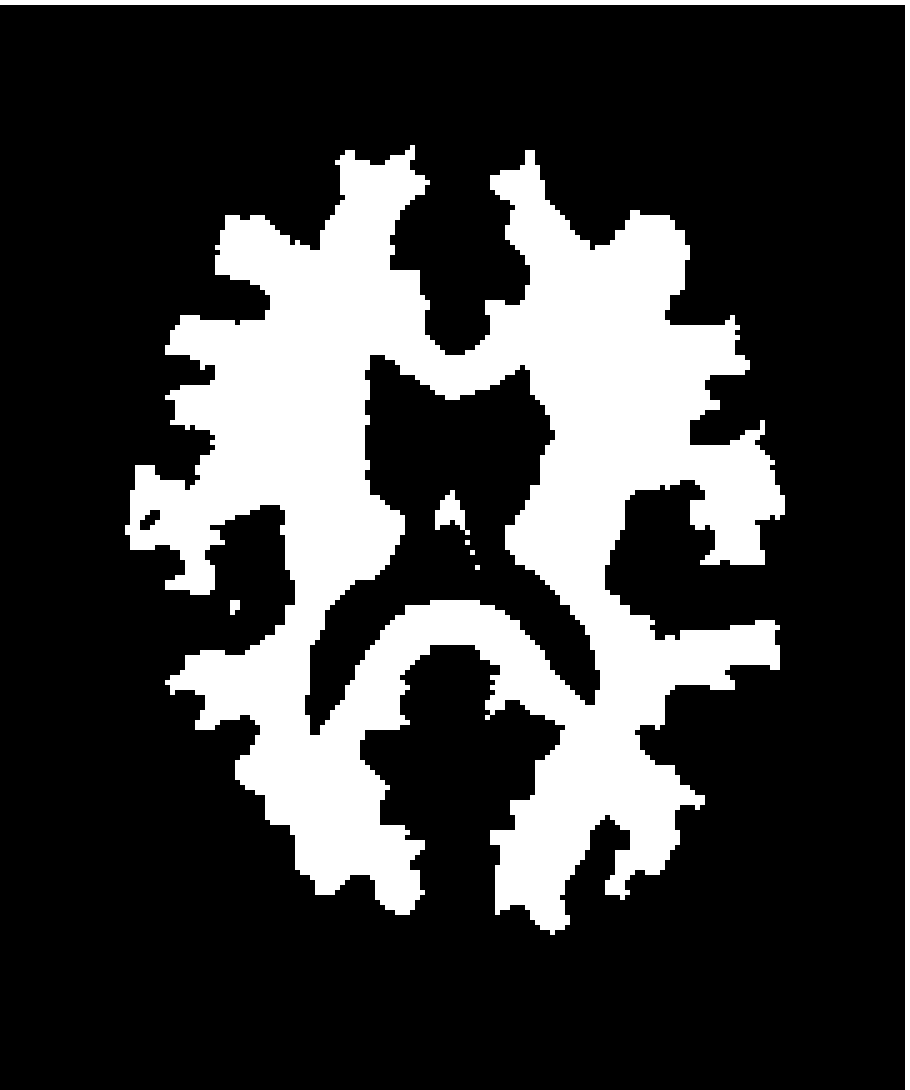}
  \end{subfigure}
   \hfill
  \begin{subfigure}[b]{0.16\linewidth}
      \centering
      \includegraphics[width=\textwidth]{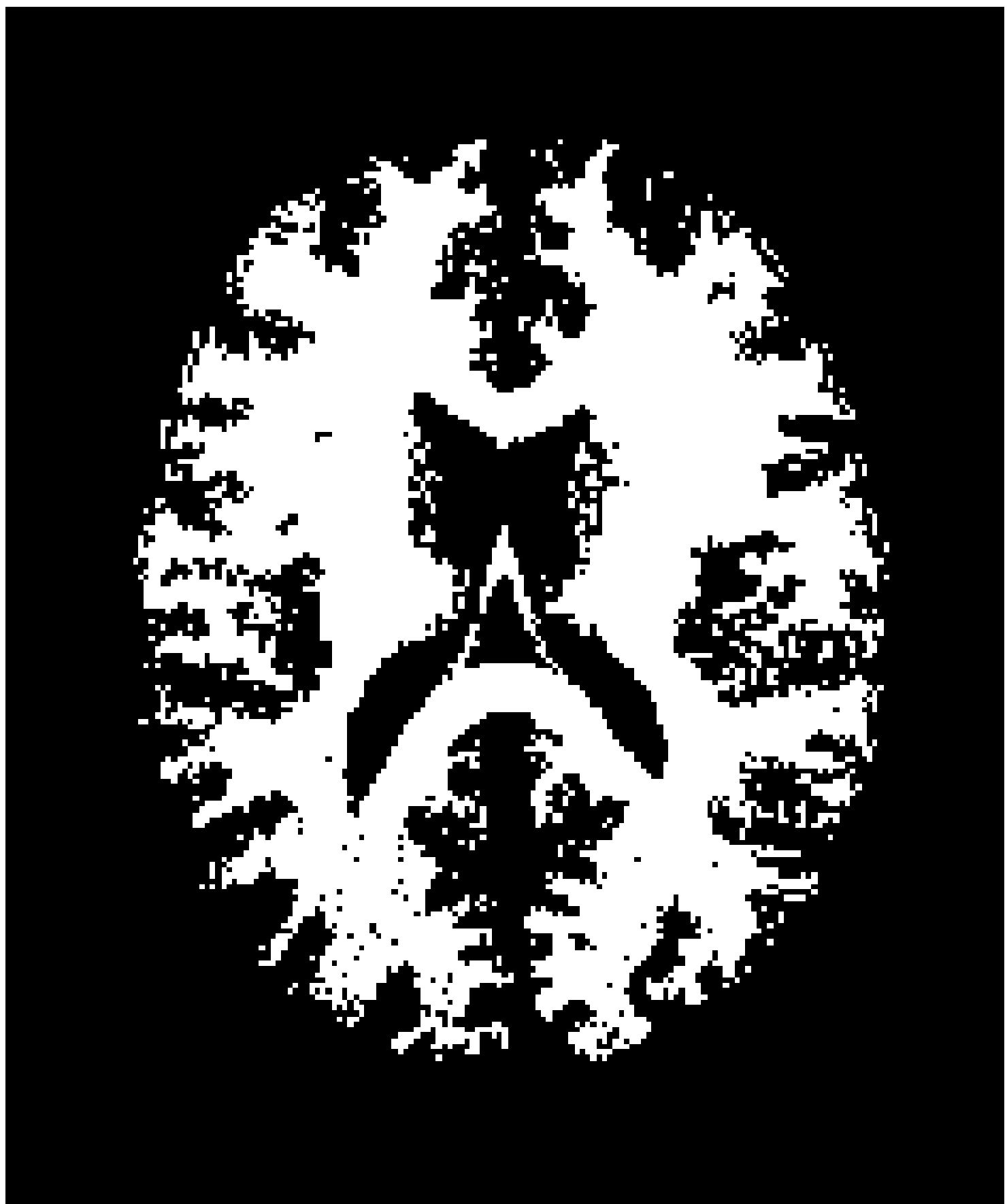}
  \end{subfigure}
  \hfill
  \begin{subfigure}[b]{0.16\linewidth}
      \centering
      \includegraphics[width=\textwidth]{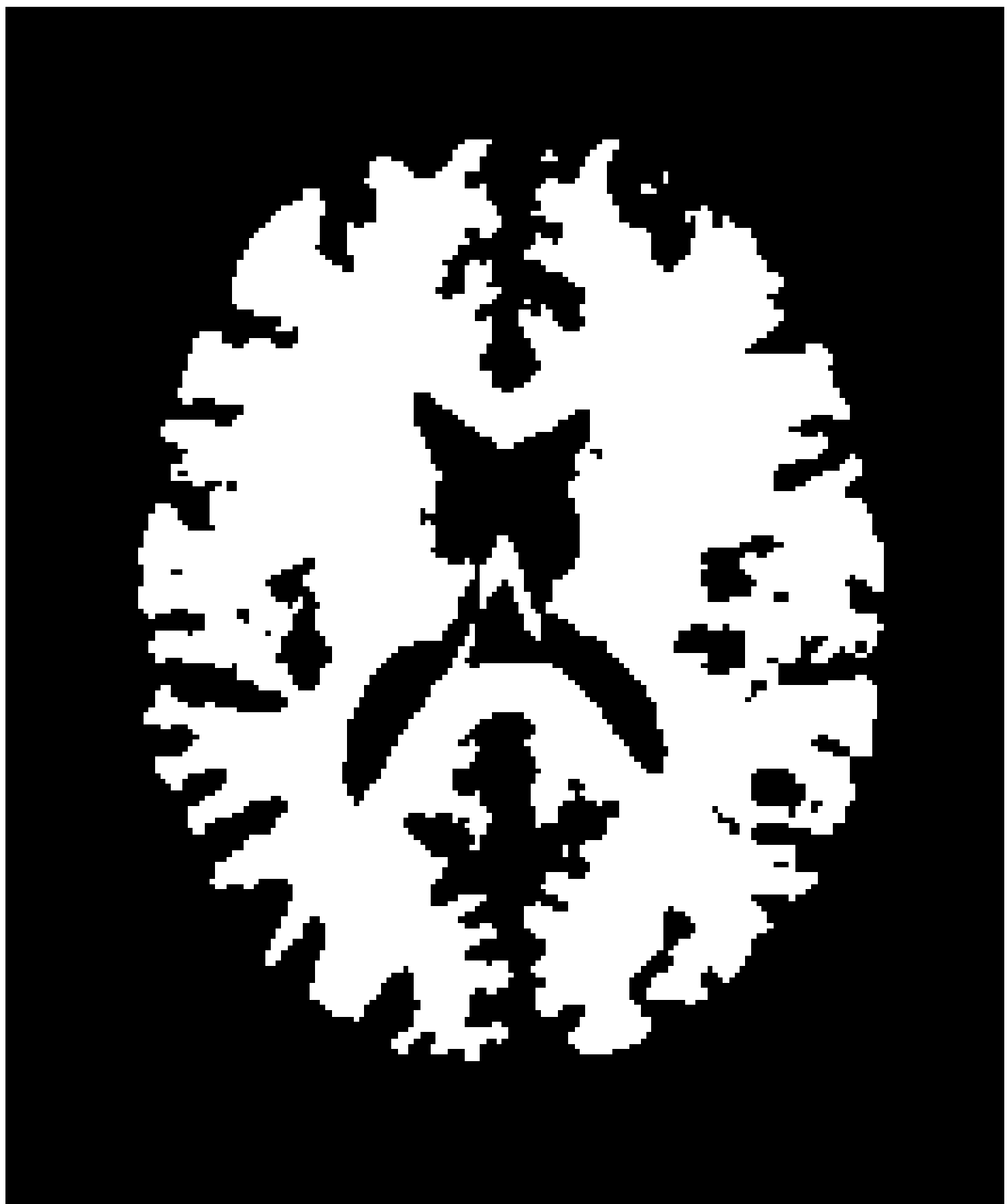}
  \end{subfigure}
   \hfill
  \begin{subfigure}[b]{0.16\linewidth}
      \centering
      \includegraphics[width=\textwidth]{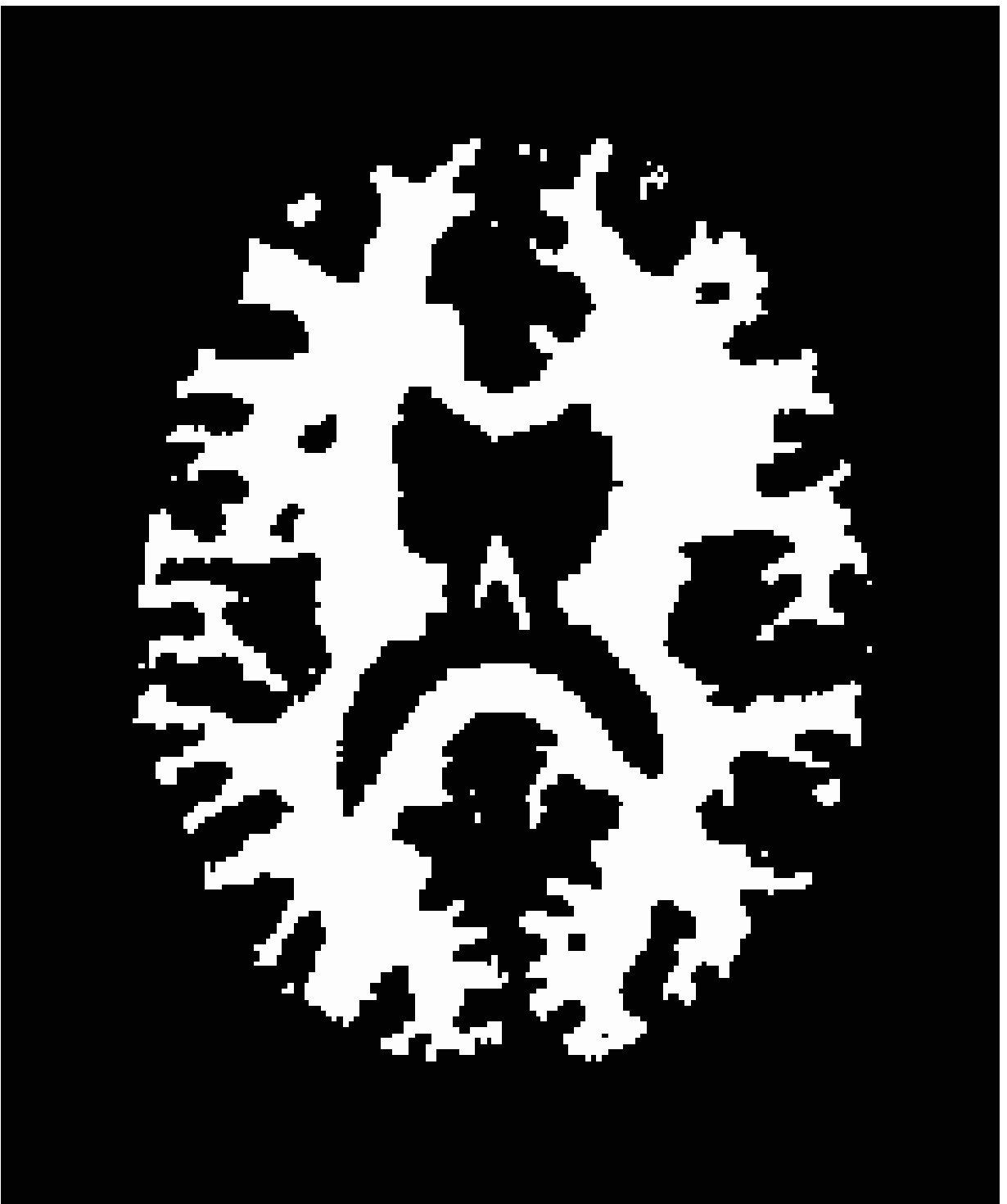}
  \end{subfigure}

 \begin{subfigure}[b]{0.16\linewidth}
      \centering
      \includegraphics[width=\textwidth]{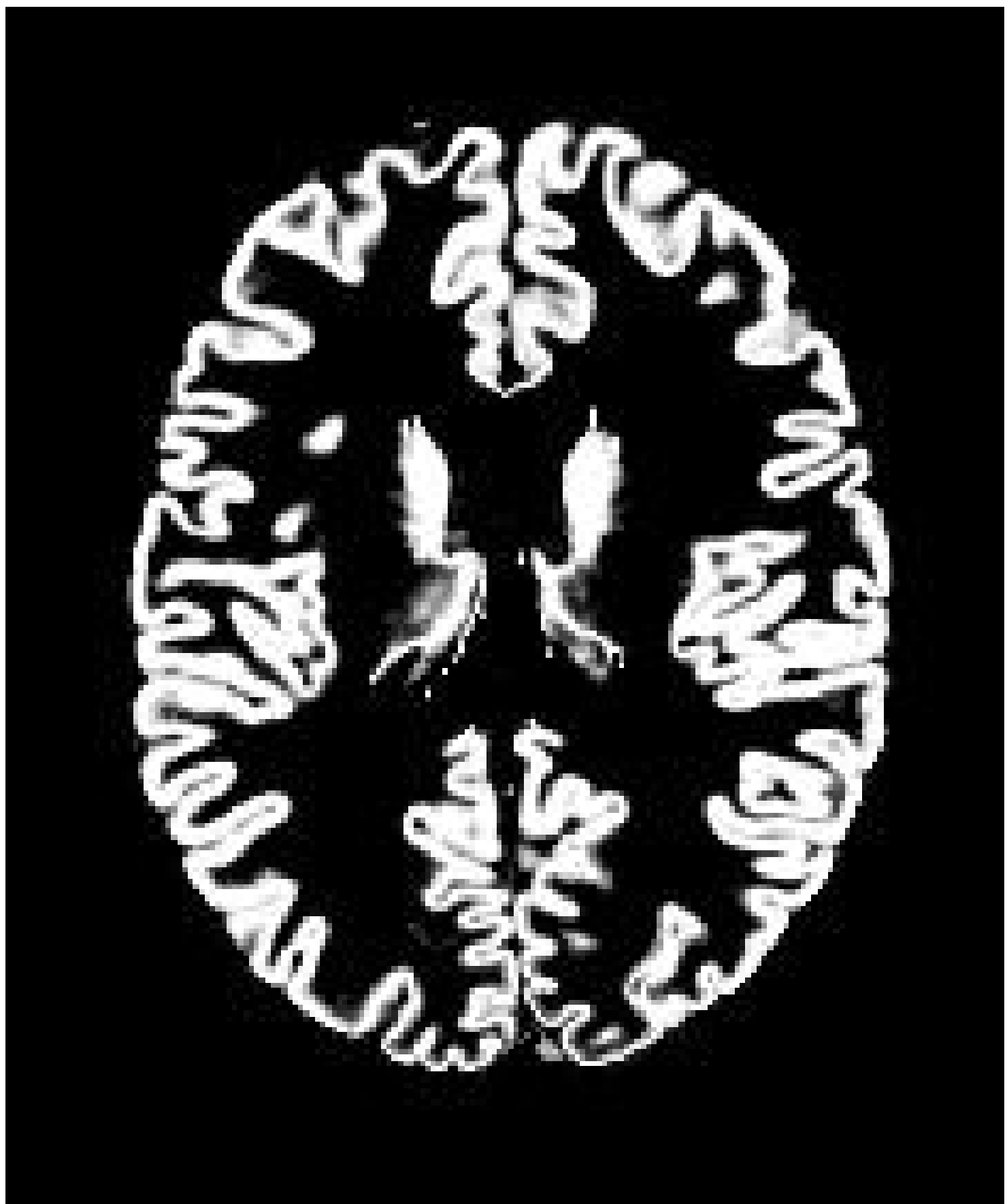}
       \caption{}
       \label{fig:90-initial}
  \end{subfigure}
   \hfill
  \begin{subfigure}[b]{0.16\linewidth}
      \centering
      \includegraphics[width=\textwidth]{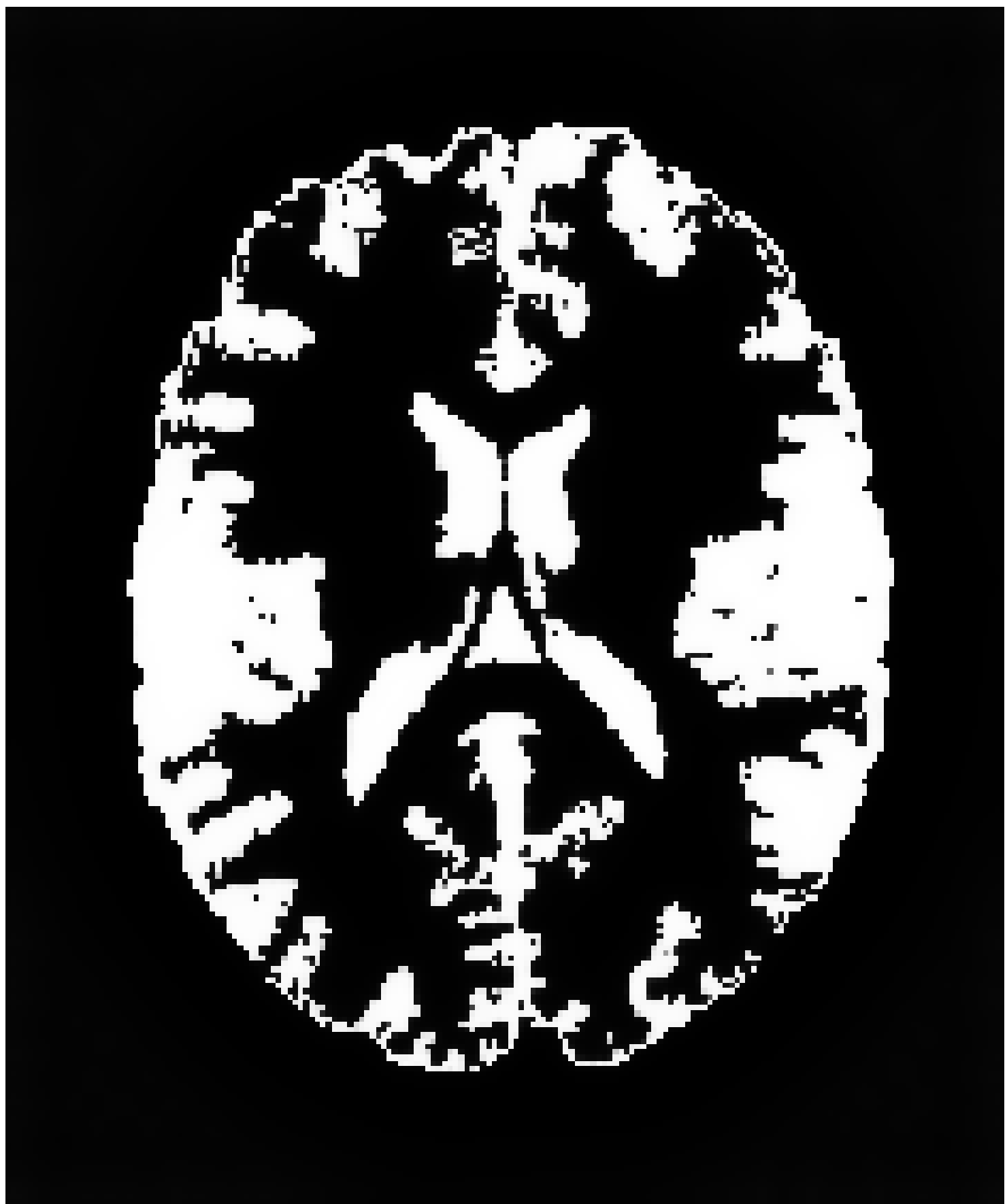}
       \caption{}
       \label{fig:90-LIC-seg}
  \end{subfigure}
  \hfill
  \begin{subfigure}[b]{0.16\linewidth}
      \centering
      \includegraphics[width=\textwidth]{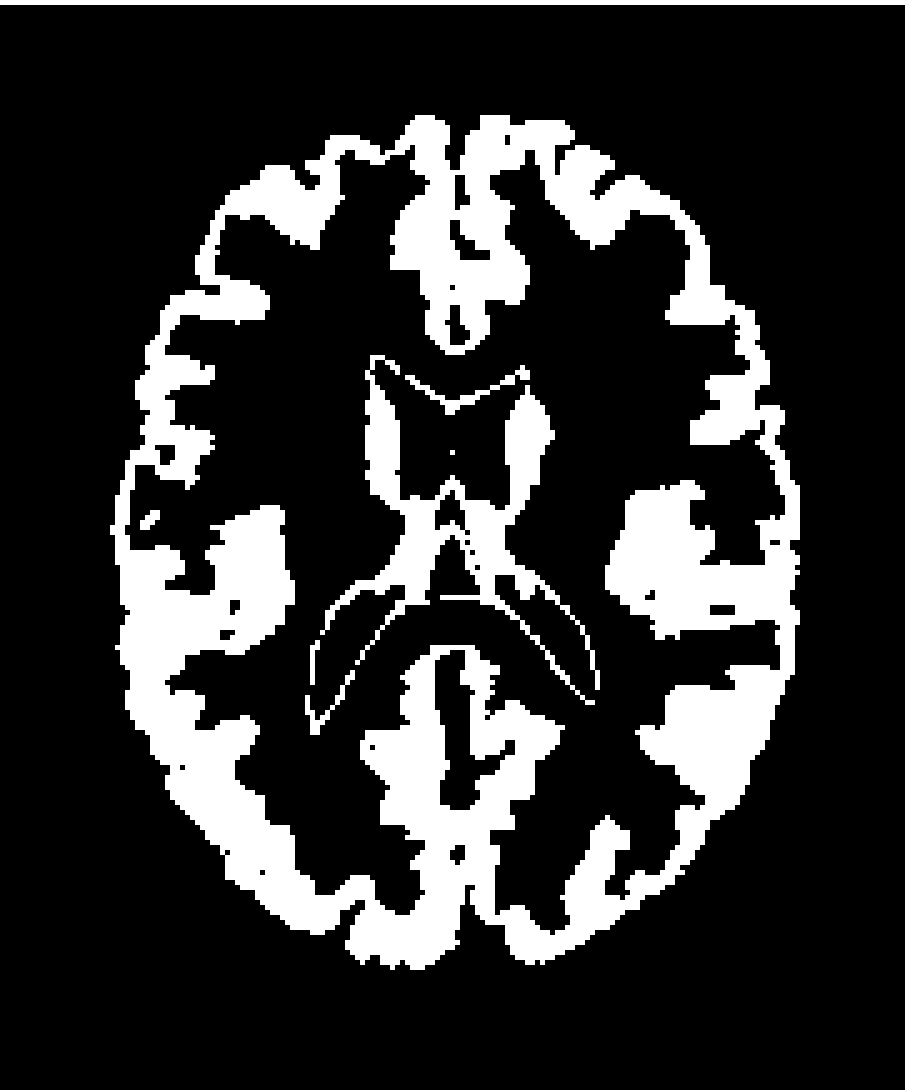}
       \caption{}
       \label{fig:90-TSS-seg}
  \end{subfigure}
   \hfill
  \begin{subfigure}[b]{0.16\linewidth}
      \centering
      \includegraphics[width=\textwidth]{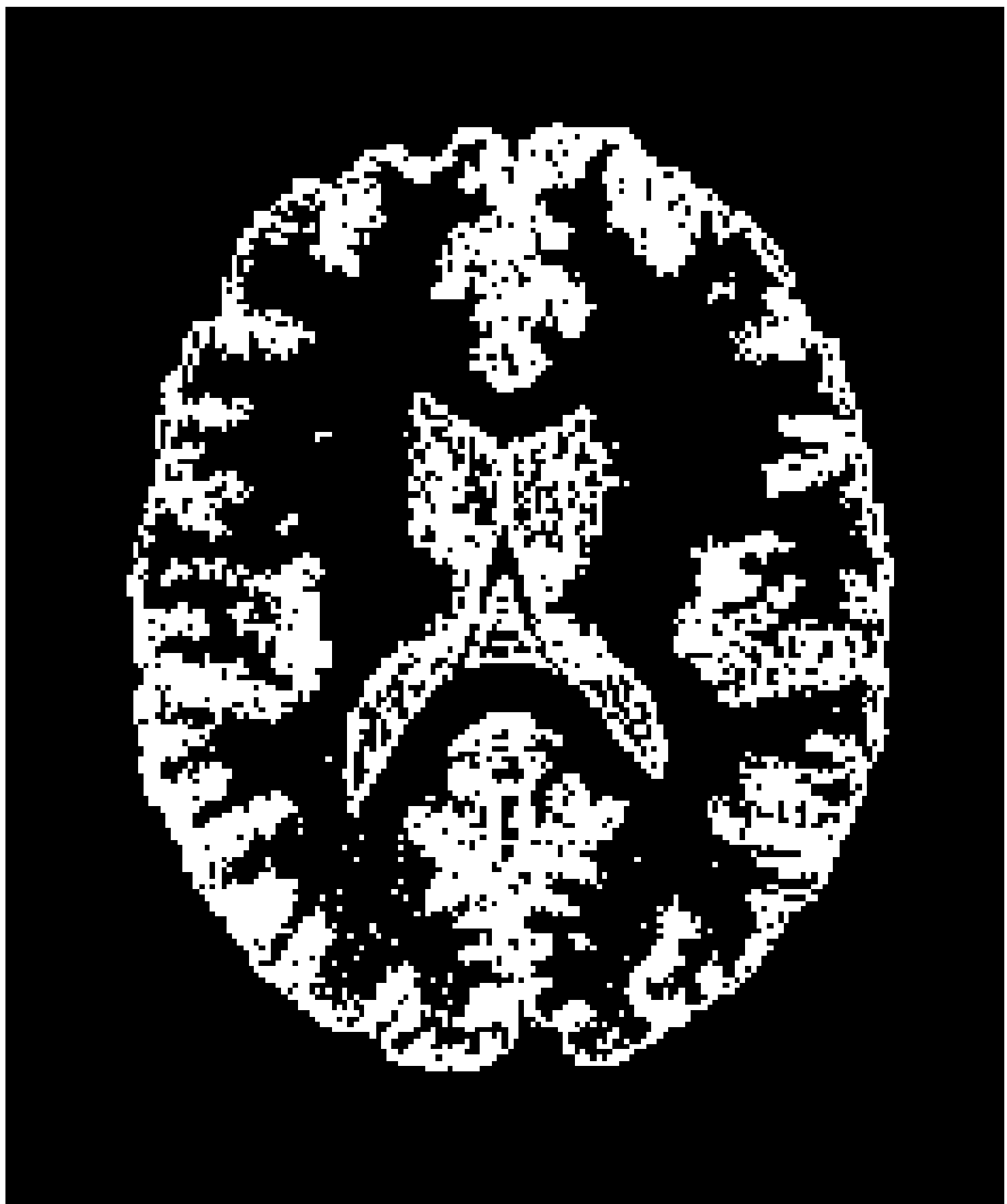}
       \caption{}
       \label{fig:90-ICTMCV-seg}
  \end{subfigure}
  \hfill
  \begin{subfigure}[b]{0.16\linewidth}
      \centering
      \includegraphics[width=\textwidth]{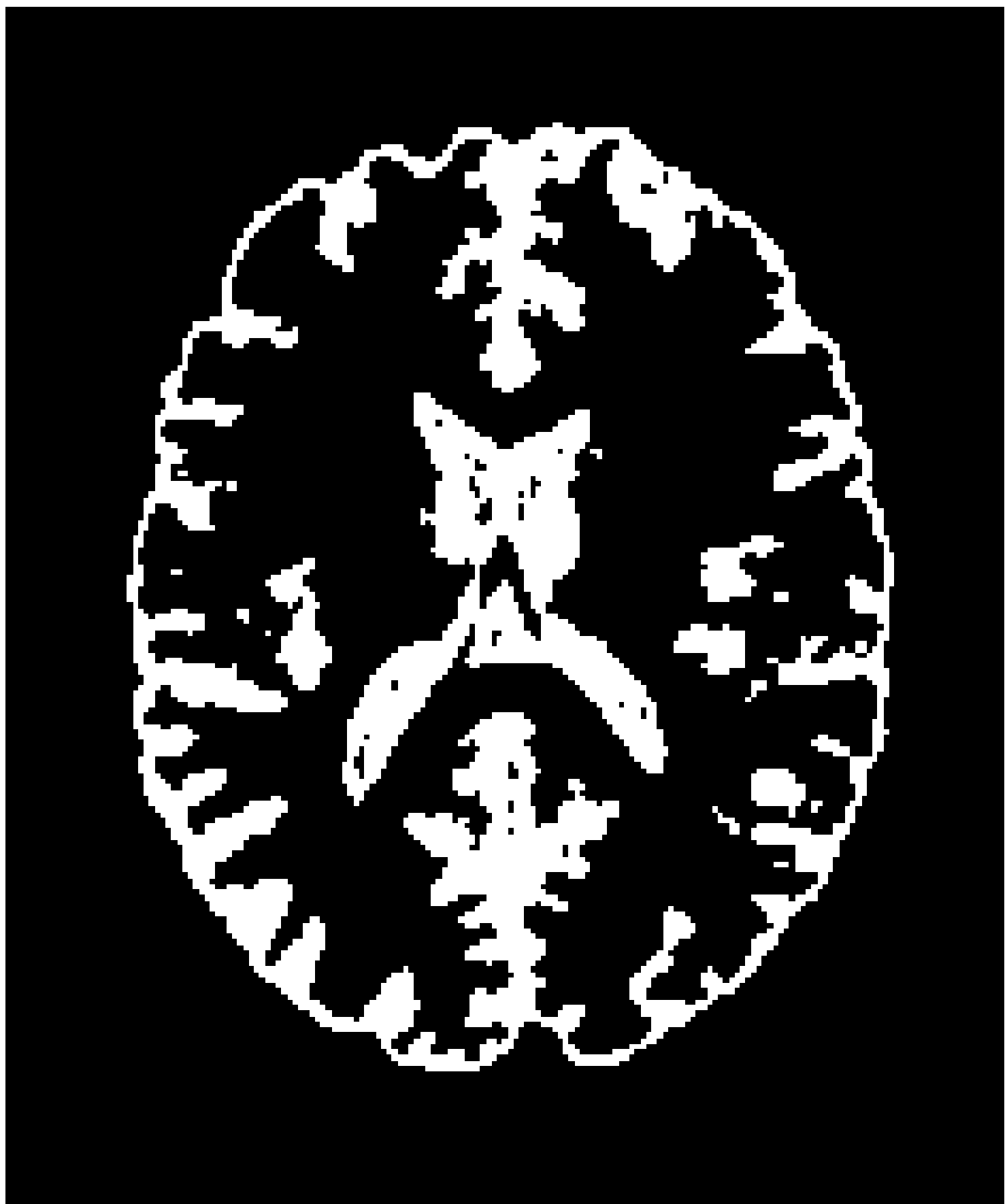}
       \caption{}
       \label{fig:90-LVFCV-seg}
  \end{subfigure}
   \hfill
  \begin{subfigure}[b]{0.16\linewidth}
      \centering
      \includegraphics[width=\textwidth]{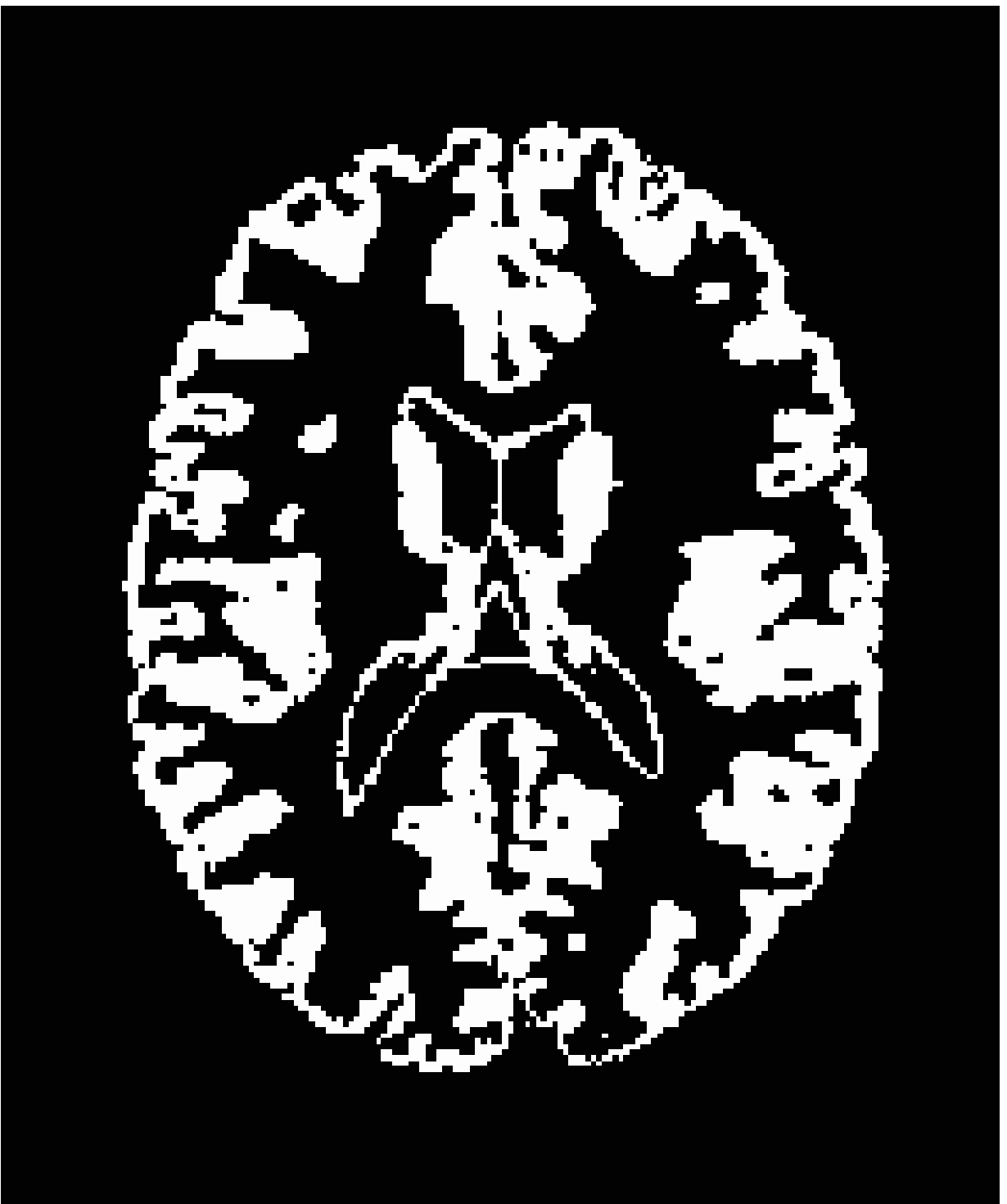}
       \caption{}
       \label{fig:90-our-gseg}
  \end{subfigure}

     \caption{Comparison with LIC, TSS, ICTM-CV, ICTM-LVF-CV models on brain MR image. Column 1: the input images with initial contours and the ground truth; Column 2--6: the segmentation results of the LIC, TSS, ICTM-CV, ICTM-LVF-CV model and our model}
  \label{fig:90}
\end{figure}
\begin{figure}
\centering
  \begin{subfigure}[b]{0.16\linewidth}
      \centering
      \includegraphics[width=\textwidth]{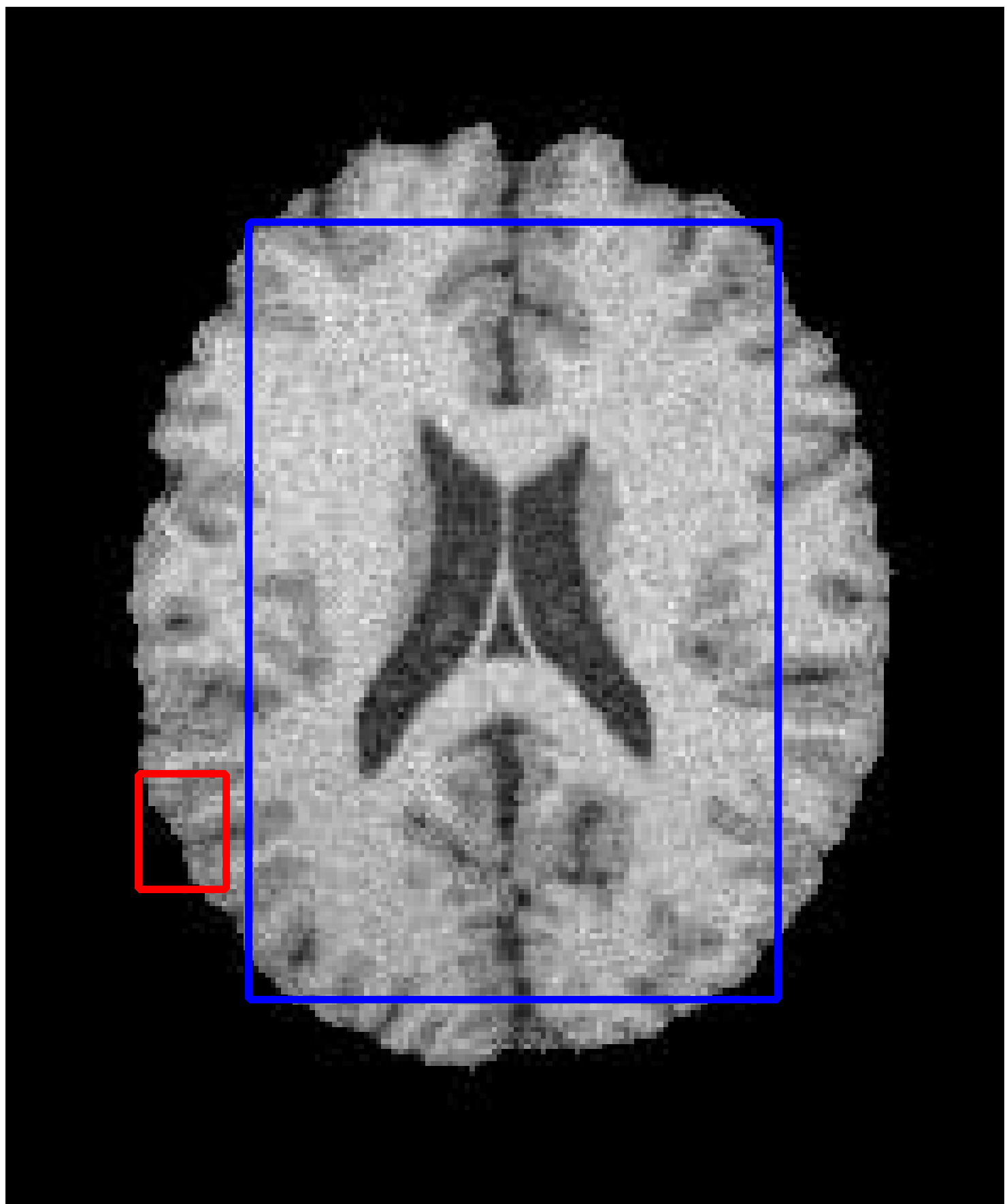}
  \end{subfigure}
  \hfill
  \begin{subfigure}[b]{0.16\linewidth}
      \centering
      \includegraphics[width=\textwidth]{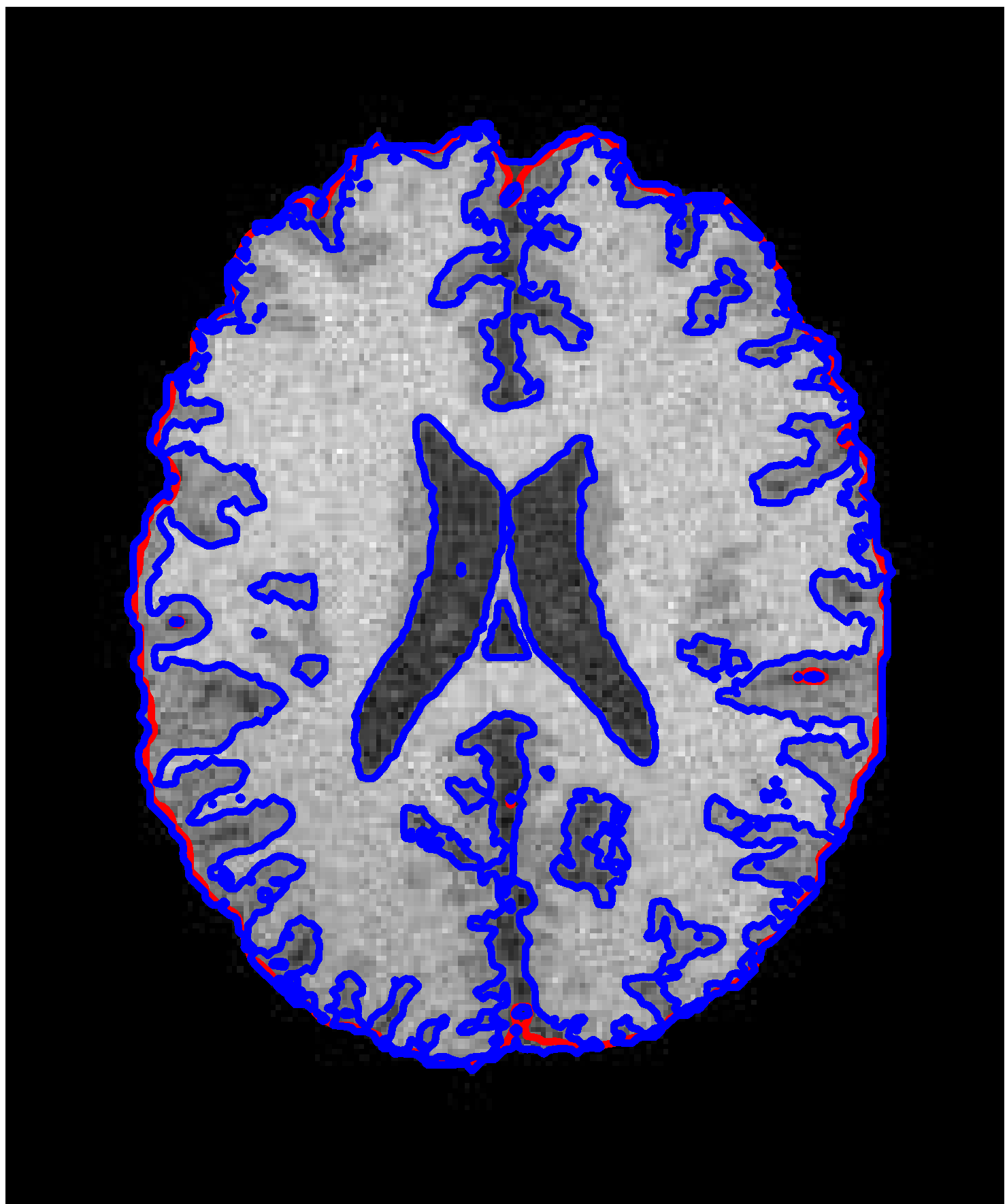}
  \end{subfigure}
  \hfill
  \begin{subfigure}[b]{0.16\linewidth}
      \centering
      \includegraphics[width=\textwidth]{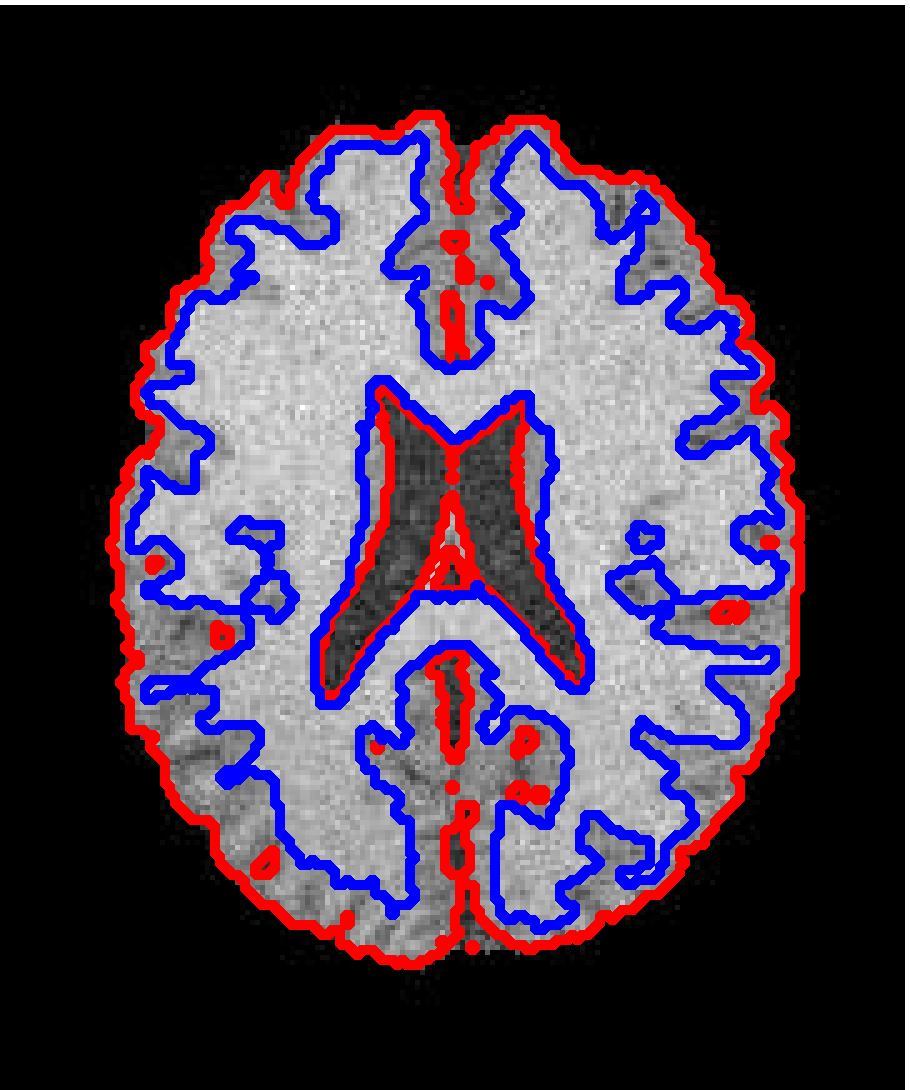}
  \end{subfigure}
 \hfill
  \begin{subfigure}[b]{0.16\linewidth}
      \centering
      \includegraphics[width=\textwidth]{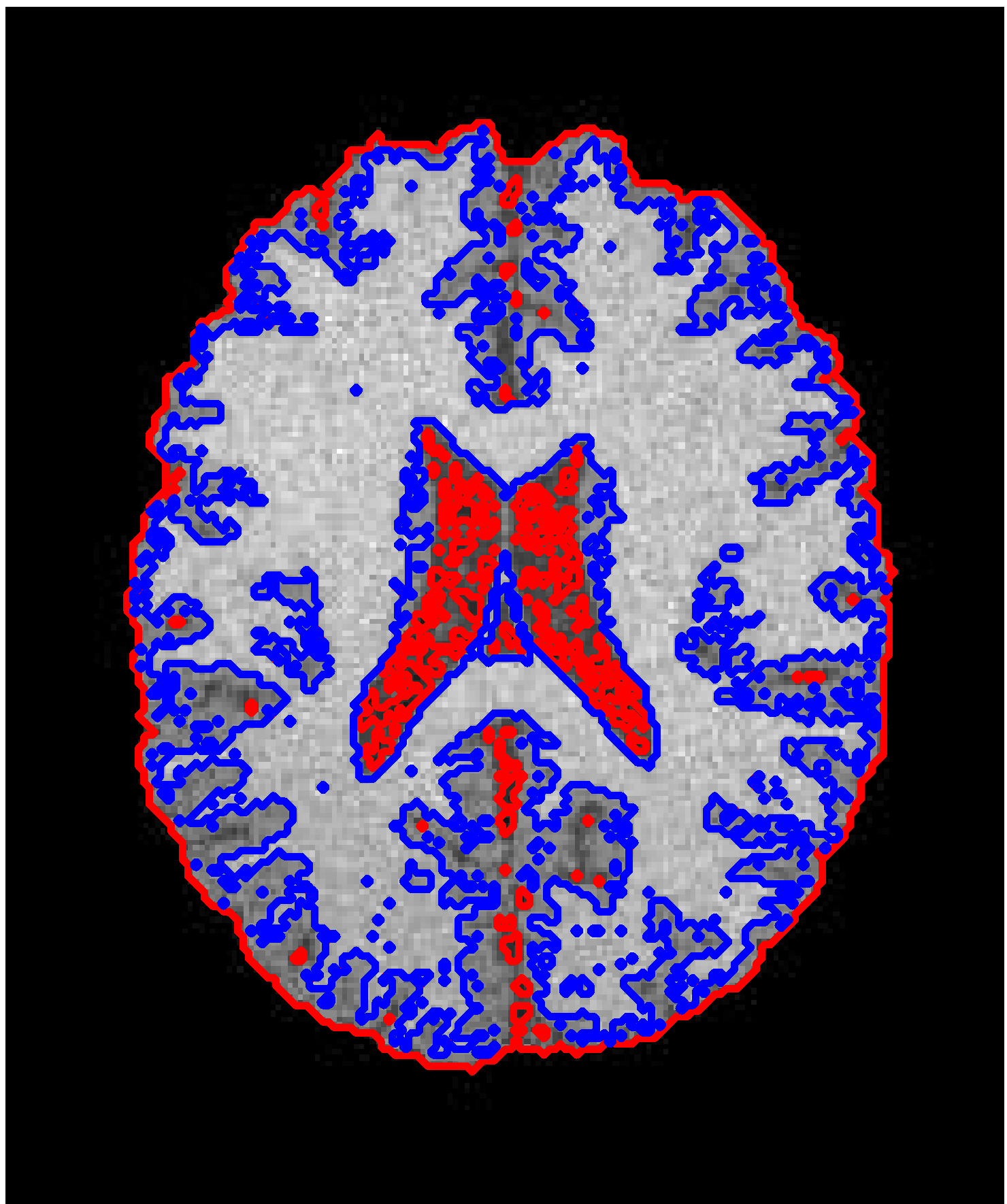}
  \end{subfigure}
  \hfill
  \begin{subfigure}[b]{0.16\linewidth}
      \centering
      \includegraphics[width=\textwidth]{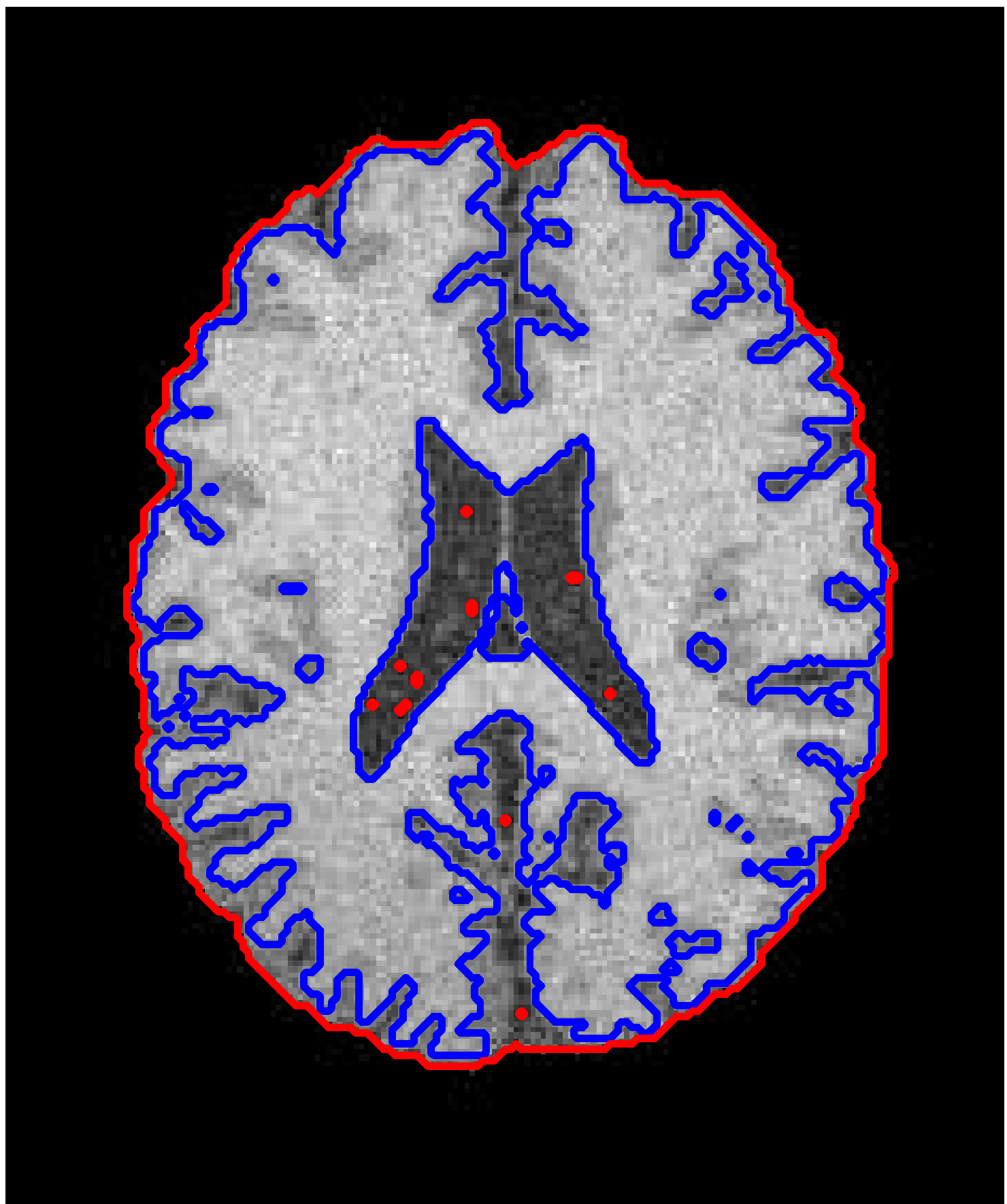}
  \end{subfigure}
 \hfill
  \begin{subfigure}[b]{0.16\linewidth}
      \centering
      \includegraphics[width=\textwidth]{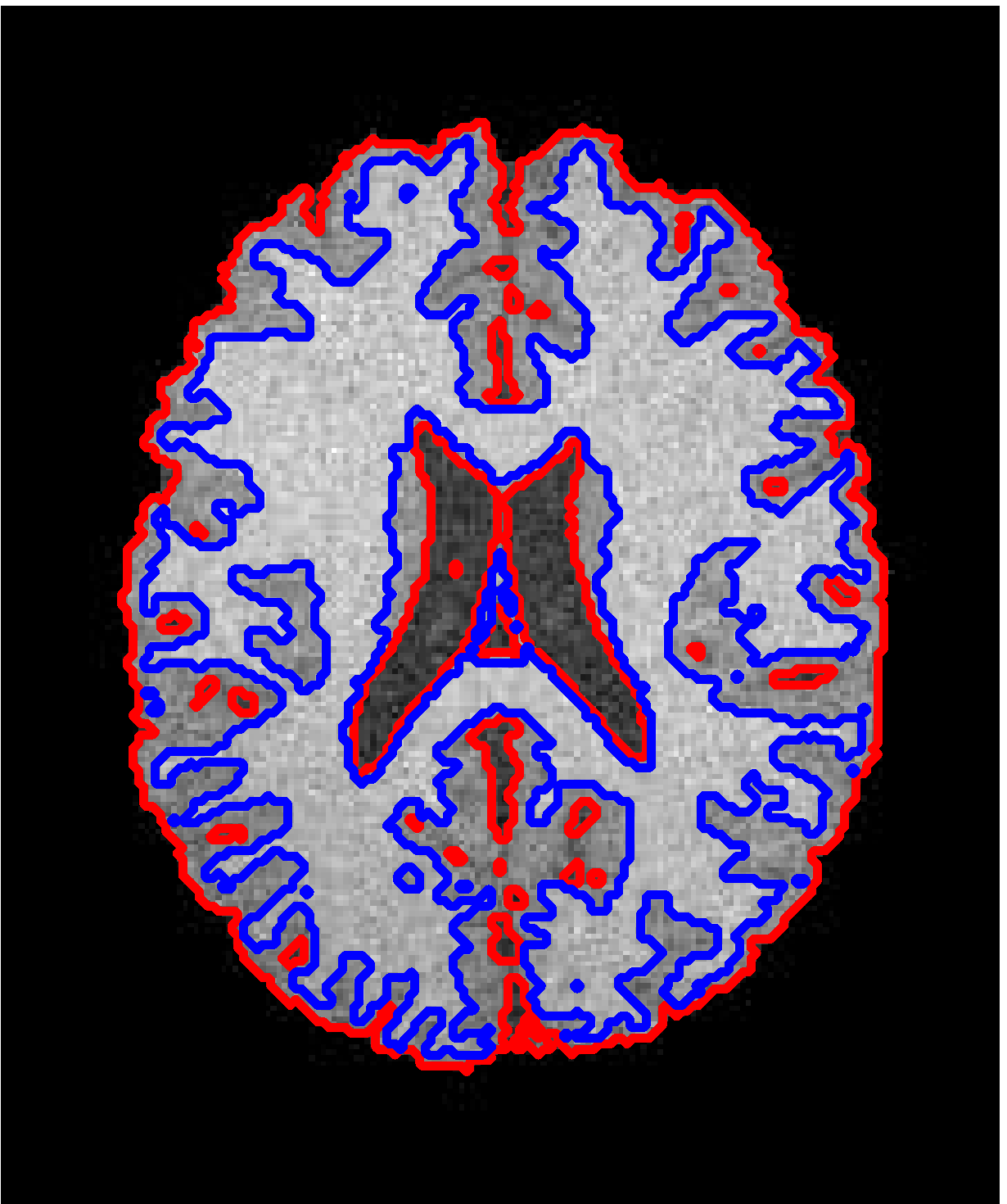}
  \end{subfigure}

 \begin{subfigure}[b]{0.16\linewidth}
      \centering
      \includegraphics[width=\textwidth]{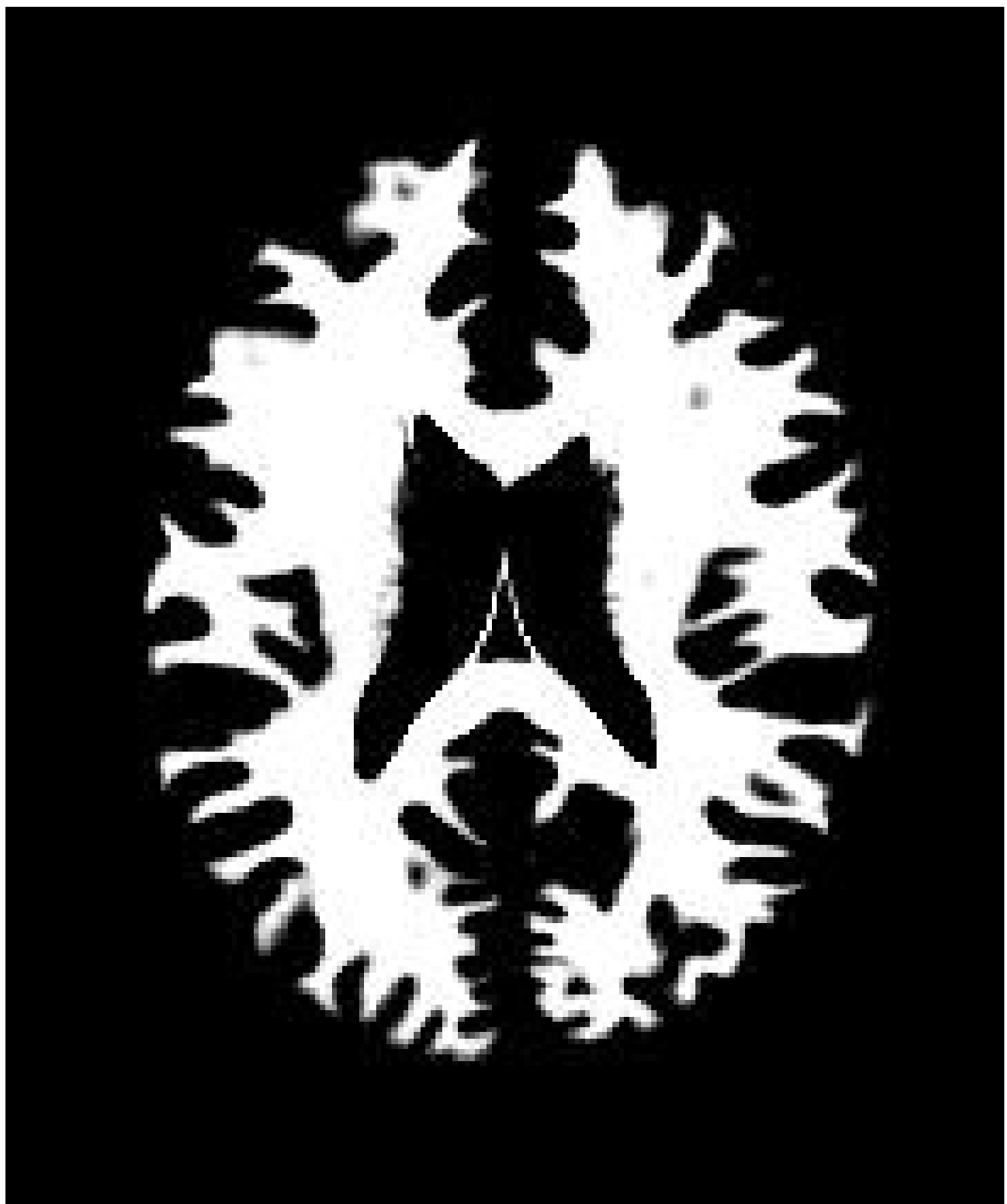}
  \end{subfigure}
   \hfill
  \begin{subfigure}[b]{0.16\linewidth}
      \centering
      \includegraphics[width=\textwidth]{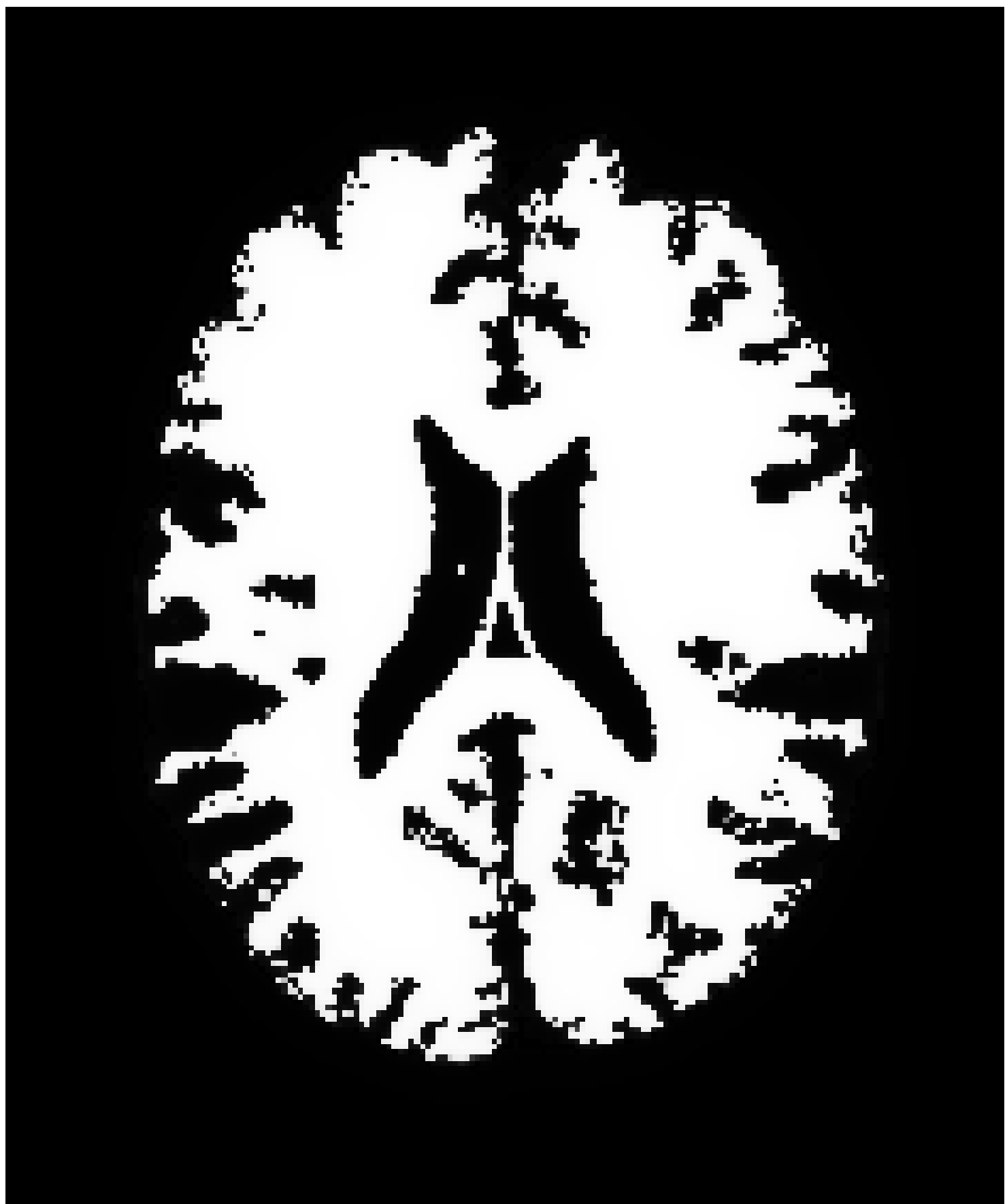}
  \end{subfigure}
  \hfill
  \begin{subfigure}[b]{0.16\linewidth}
      \centering
      \includegraphics[width=\textwidth]{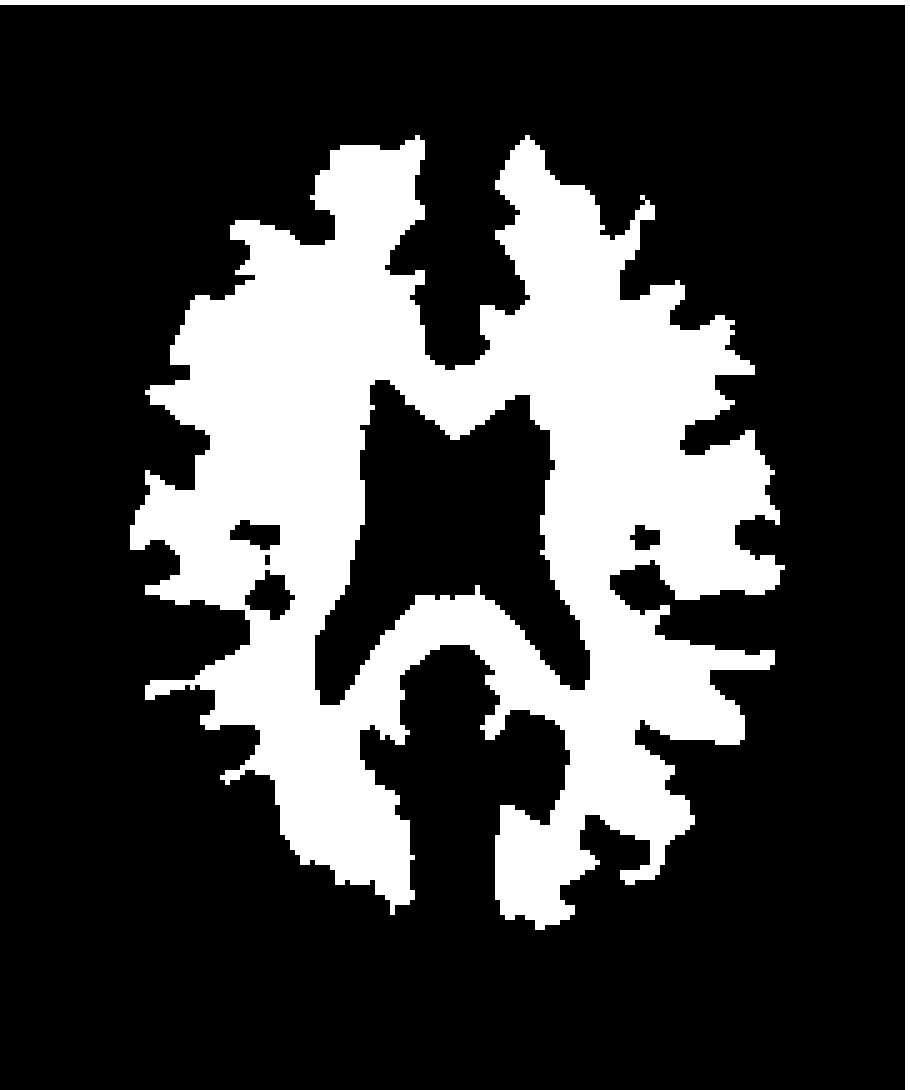}
  \end{subfigure}
   \hfill
  \begin{subfigure}[b]{0.16\linewidth}
      \centering
      \includegraphics[width=\textwidth]{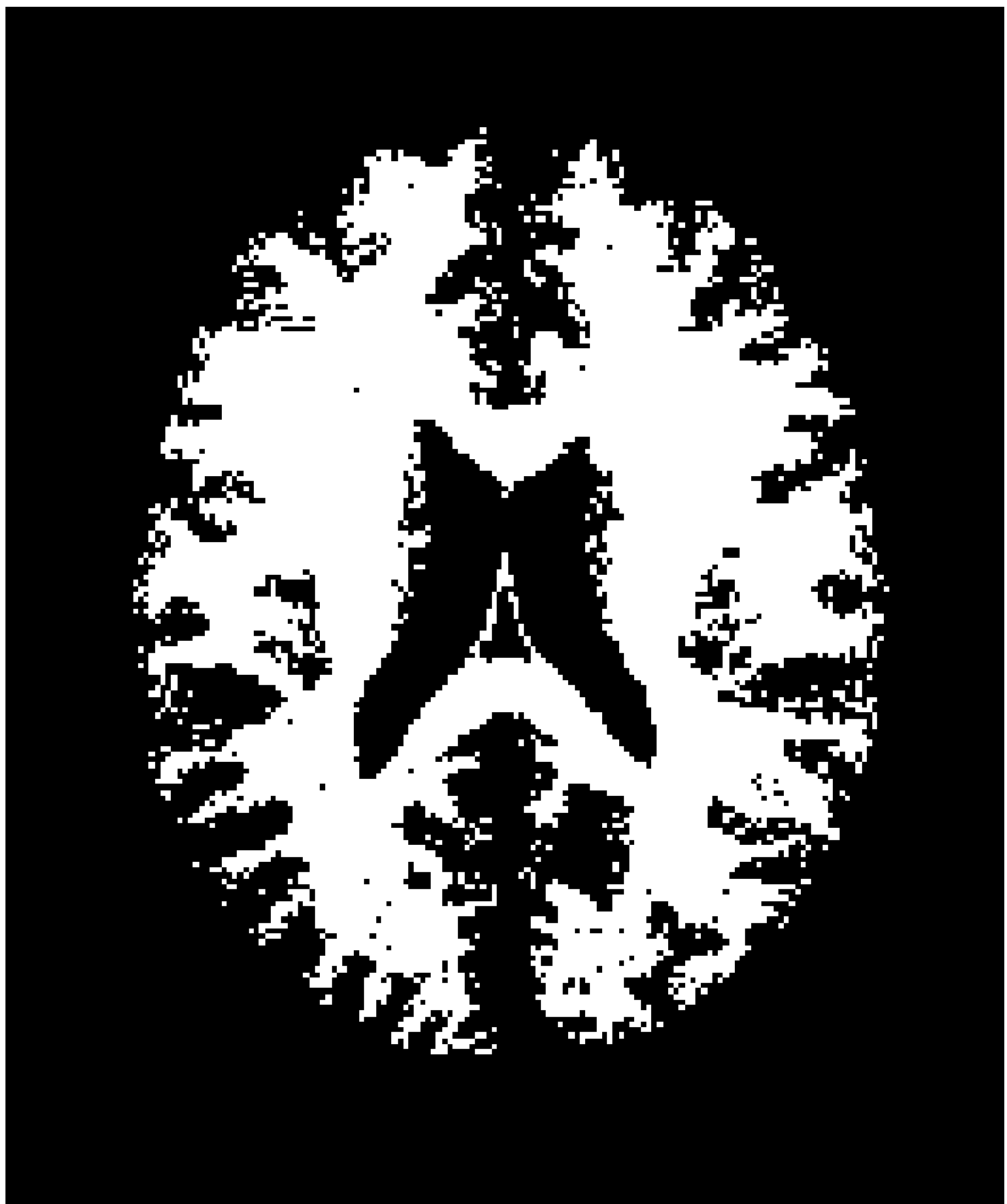}
  \end{subfigure}
  \hfill
  \begin{subfigure}[b]{0.16\linewidth}
      \centering
      \includegraphics[width=\textwidth]{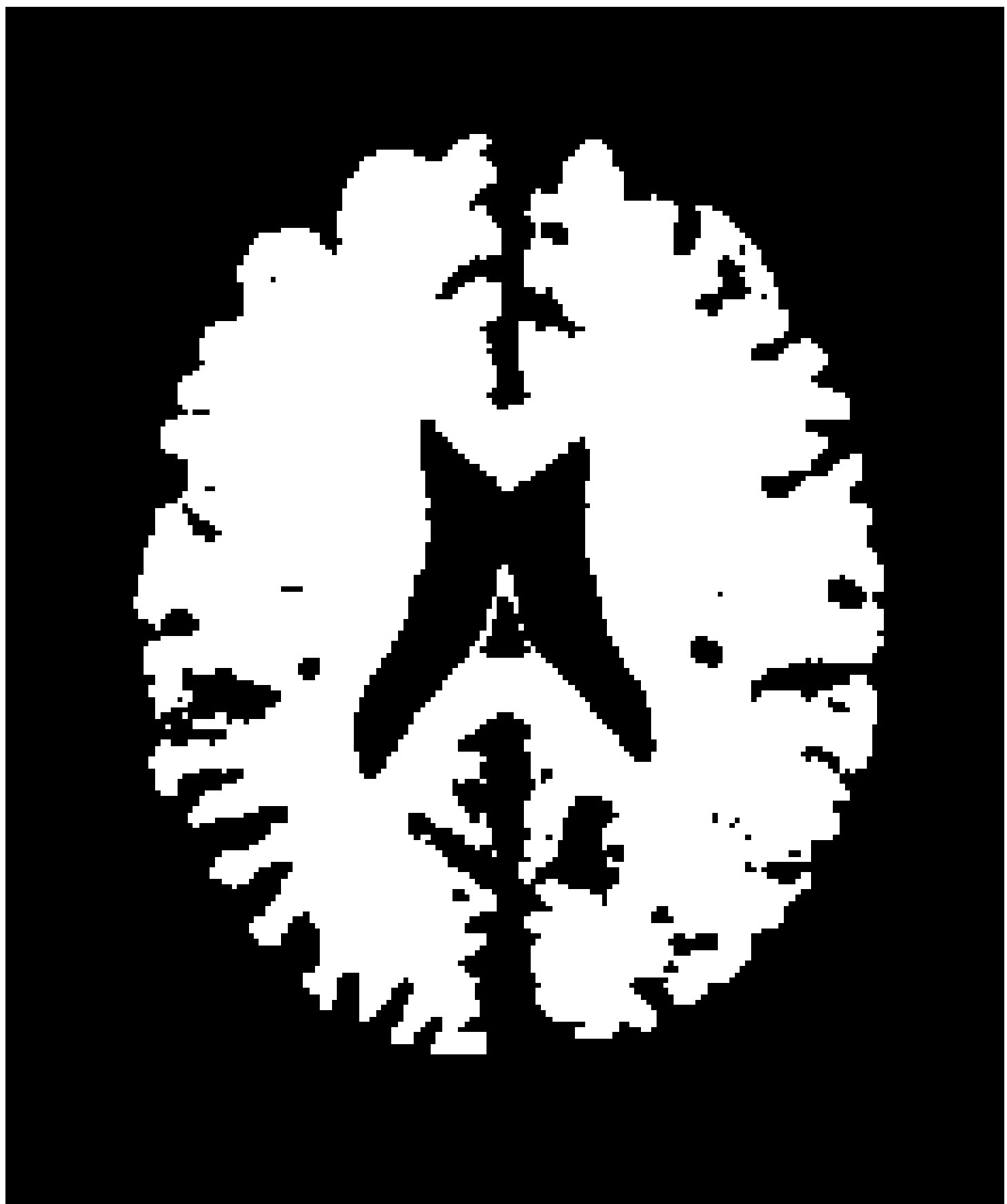}
  \end{subfigure}
   \hfill
  \begin{subfigure}[b]{0.16\linewidth}
      \centering
      \includegraphics[width=\textwidth]{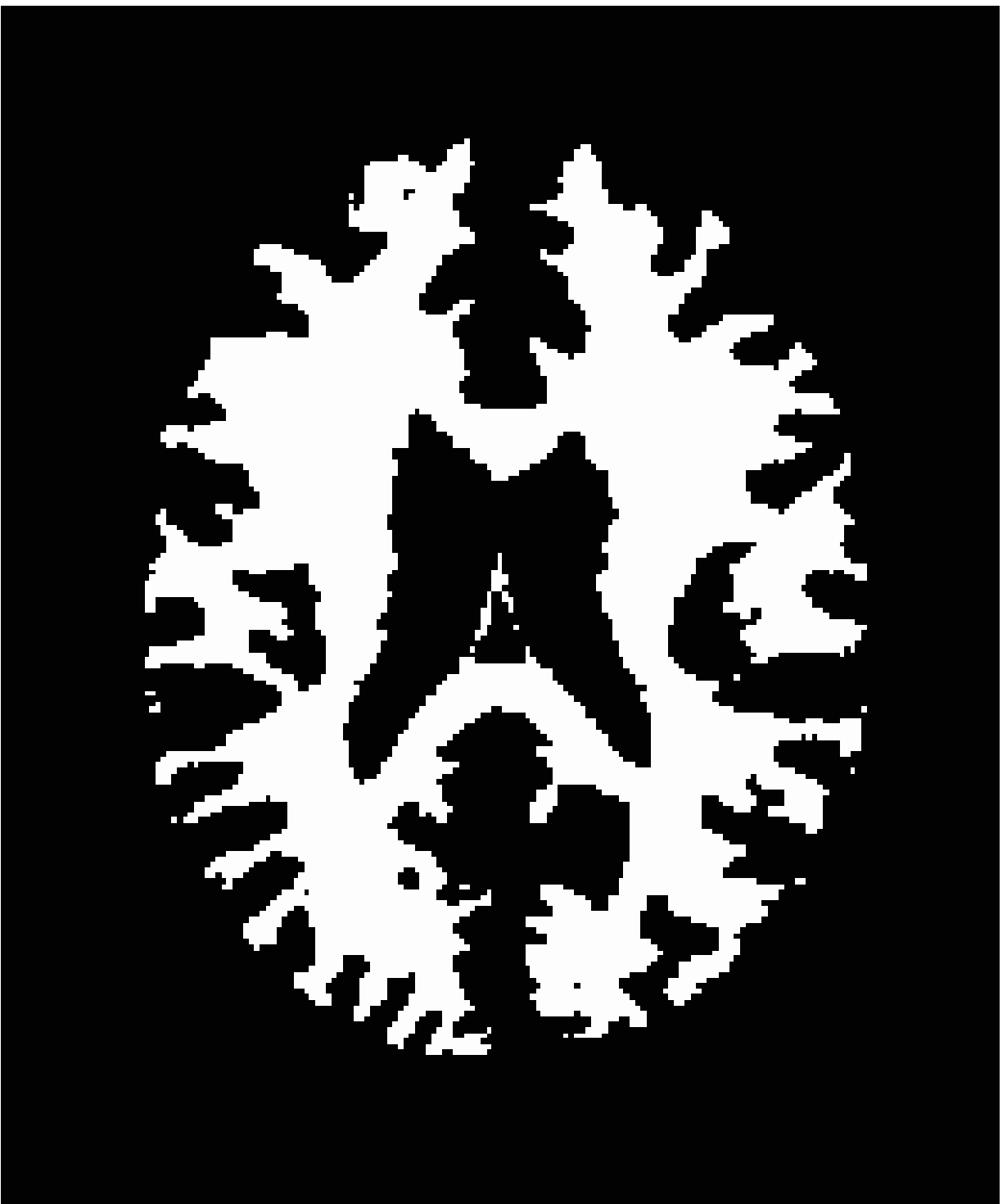}
  \end{subfigure}

 \begin{subfigure}[b]{0.16\linewidth}
      \centering
      \includegraphics[width=\textwidth]{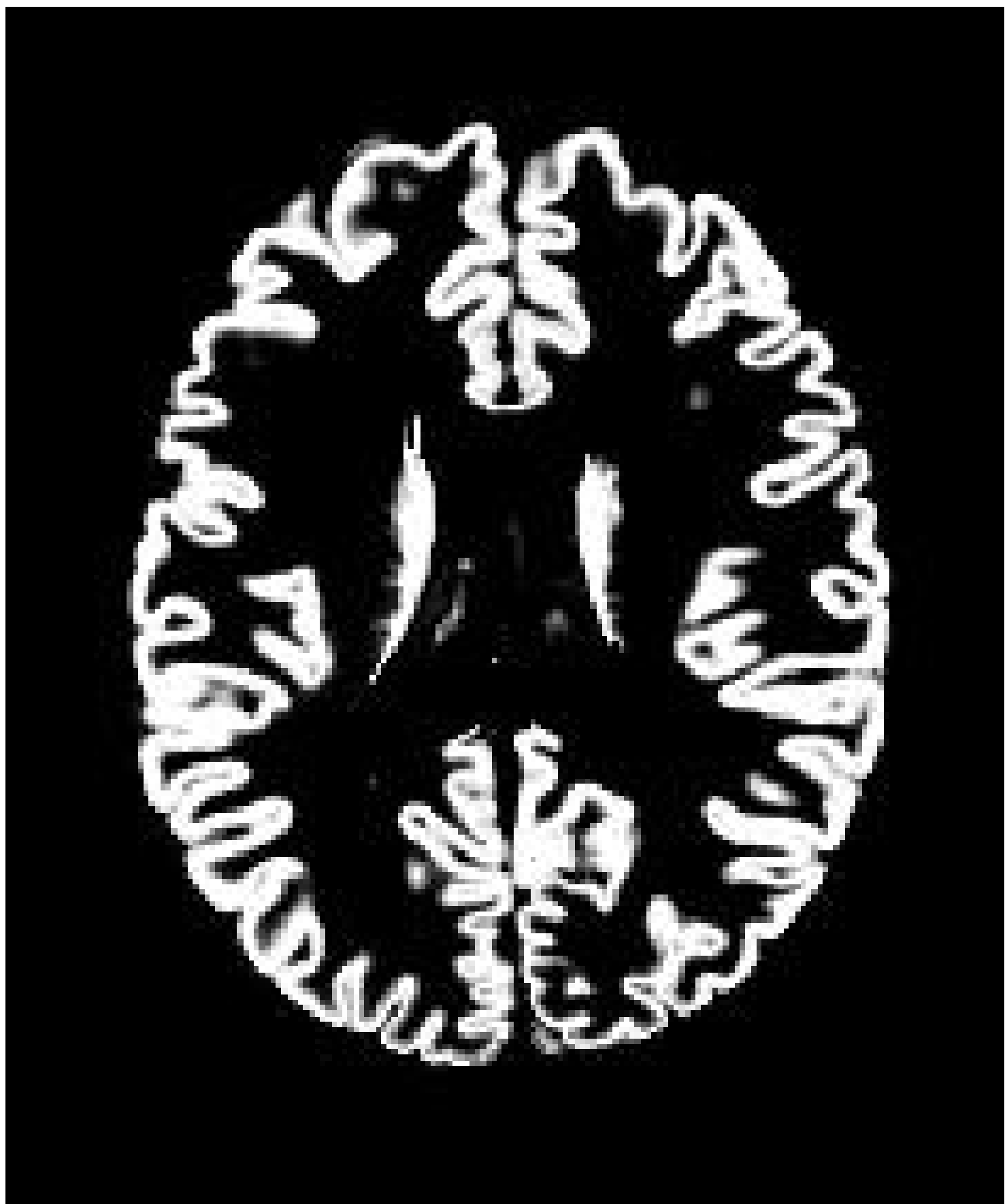}
       \caption{}
       \label{fig:93-initial}
  \end{subfigure}
   \hfill
  \begin{subfigure}[b]{0.16\linewidth}
      \centering
      \includegraphics[width=\textwidth]{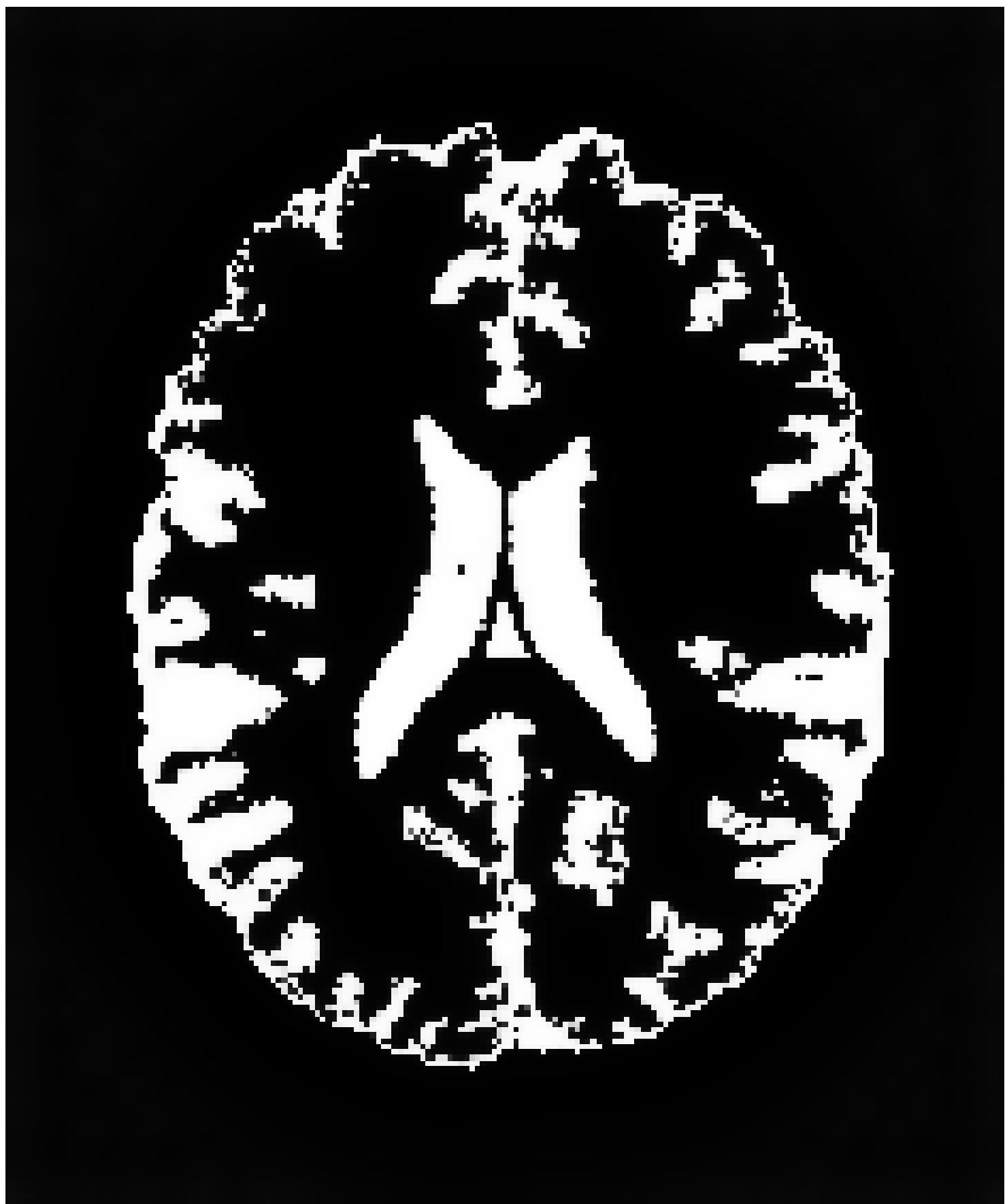}
       \caption{}
  \end{subfigure}
  \hfill
  \begin{subfigure}[b]{0.16\linewidth}
      \centering
      \includegraphics[width=\textwidth]{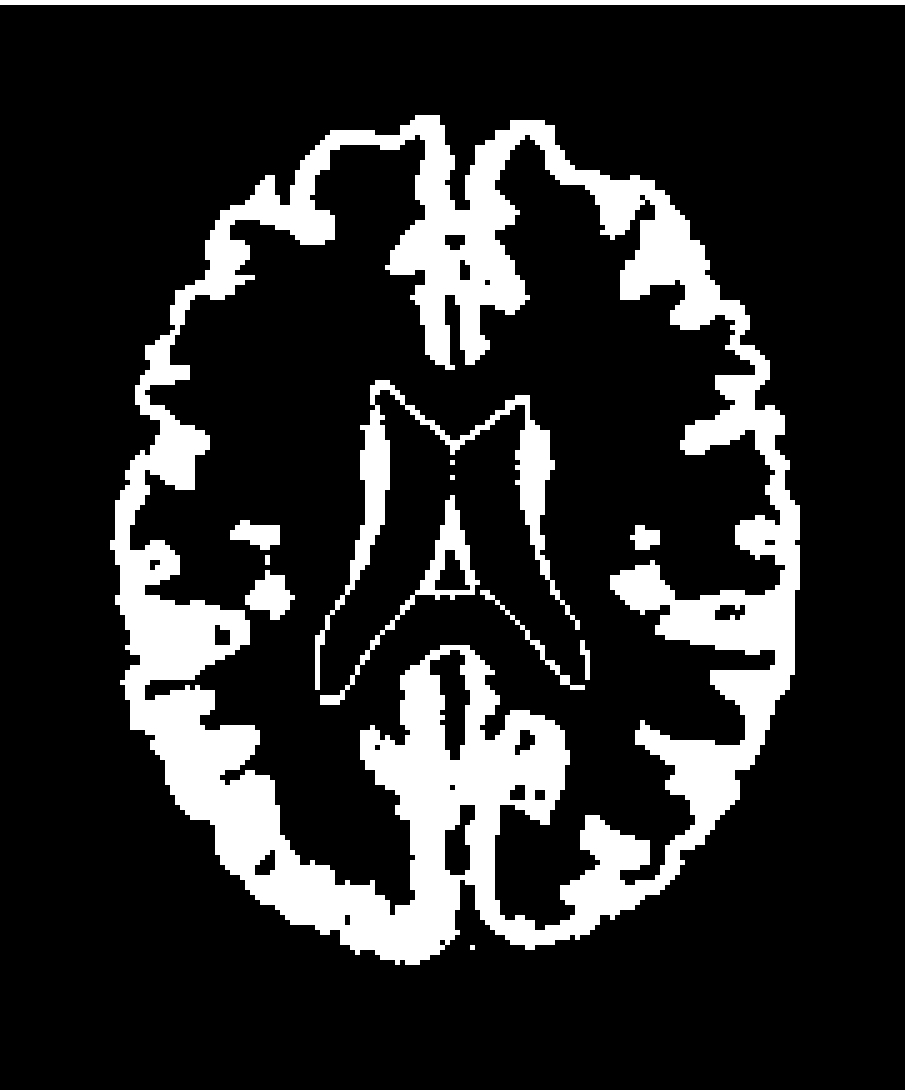}
       \caption{}
  \end{subfigure}
   \hfill
  \begin{subfigure}[b]{0.16\linewidth}
      \centering
      \includegraphics[width=\textwidth]{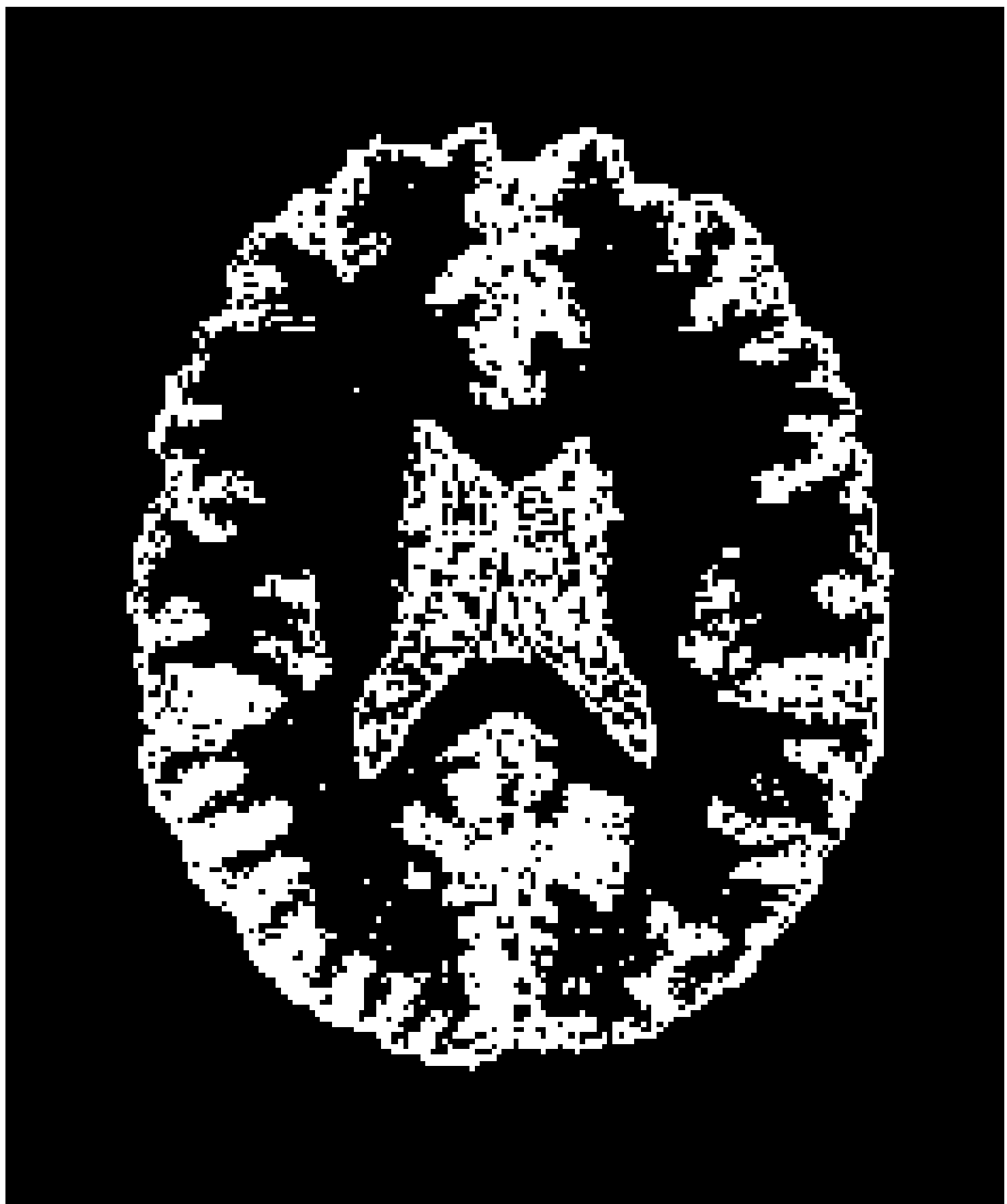}
       \caption{}
  \end{subfigure}
  \hfill
  \begin{subfigure}[b]{0.16\linewidth}
      \centering
      \includegraphics[width=\textwidth]{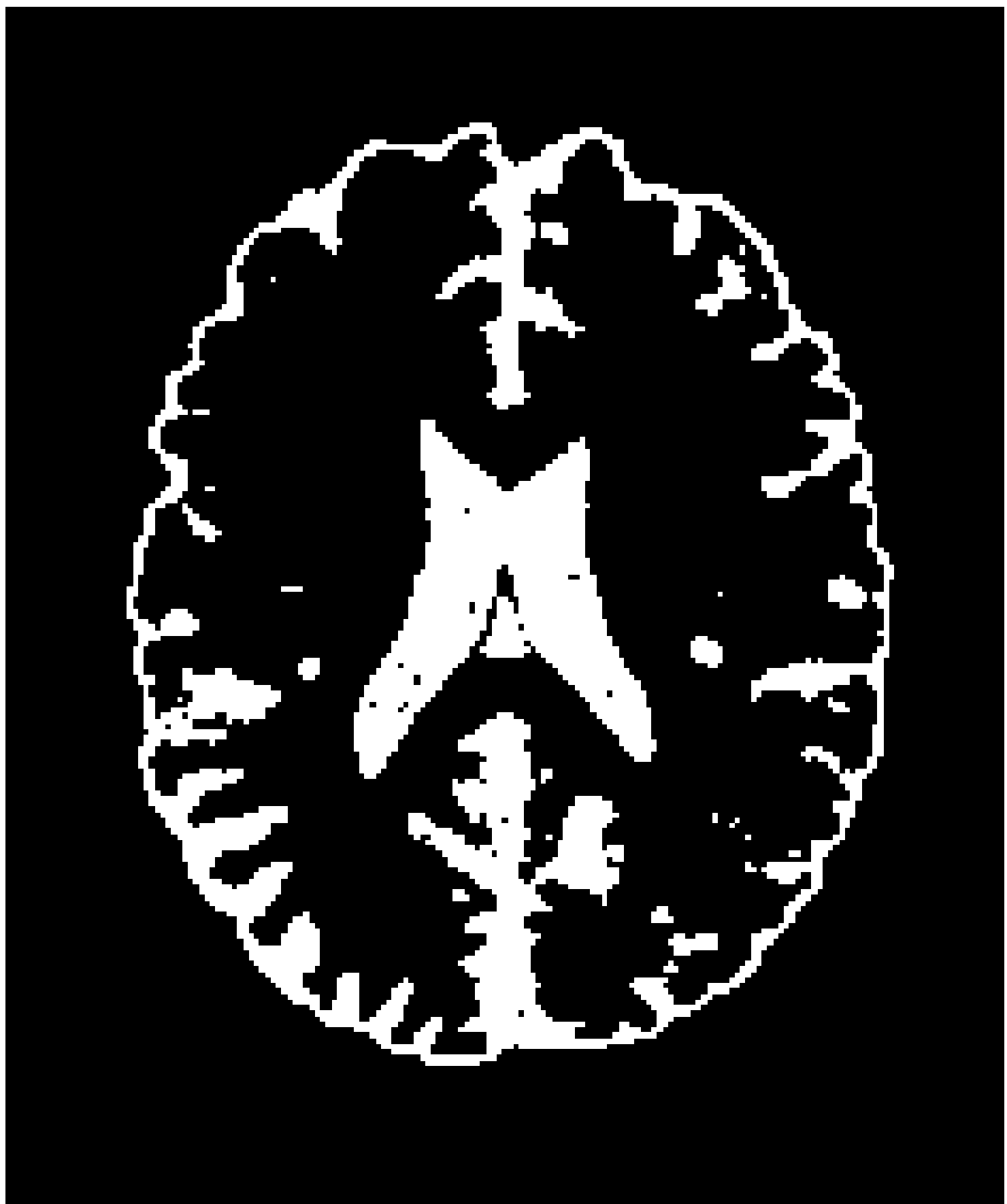}
       \caption{}
  \end{subfigure}
   \hfill
  \begin{subfigure}[b]{0.16\linewidth}
      \centering
      \includegraphics[width=\textwidth]{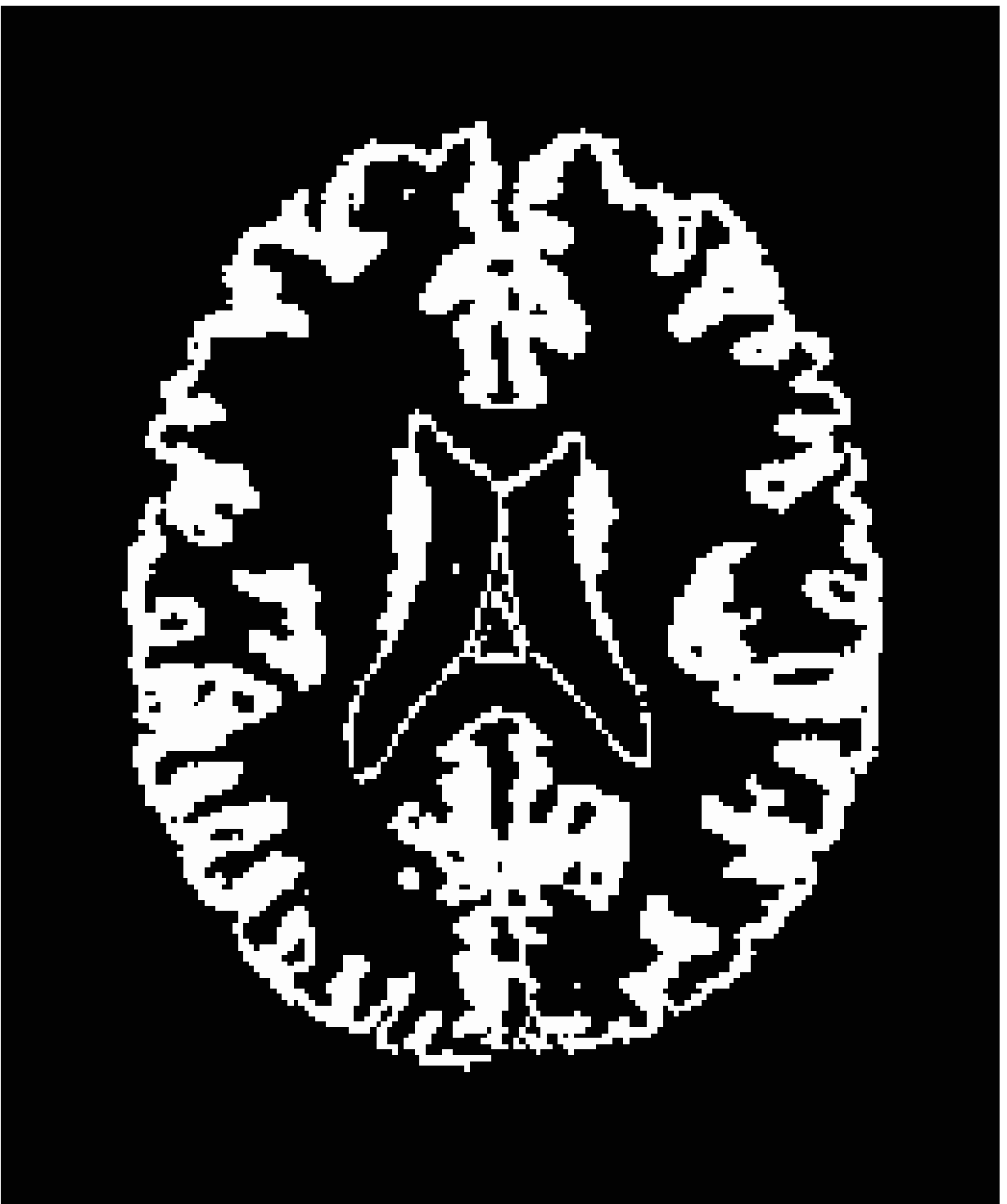}
       \caption{}
  \end{subfigure}

  \begin{subfigure}[b]{0.16\linewidth}
      \centering
      \includegraphics[width=\textwidth]{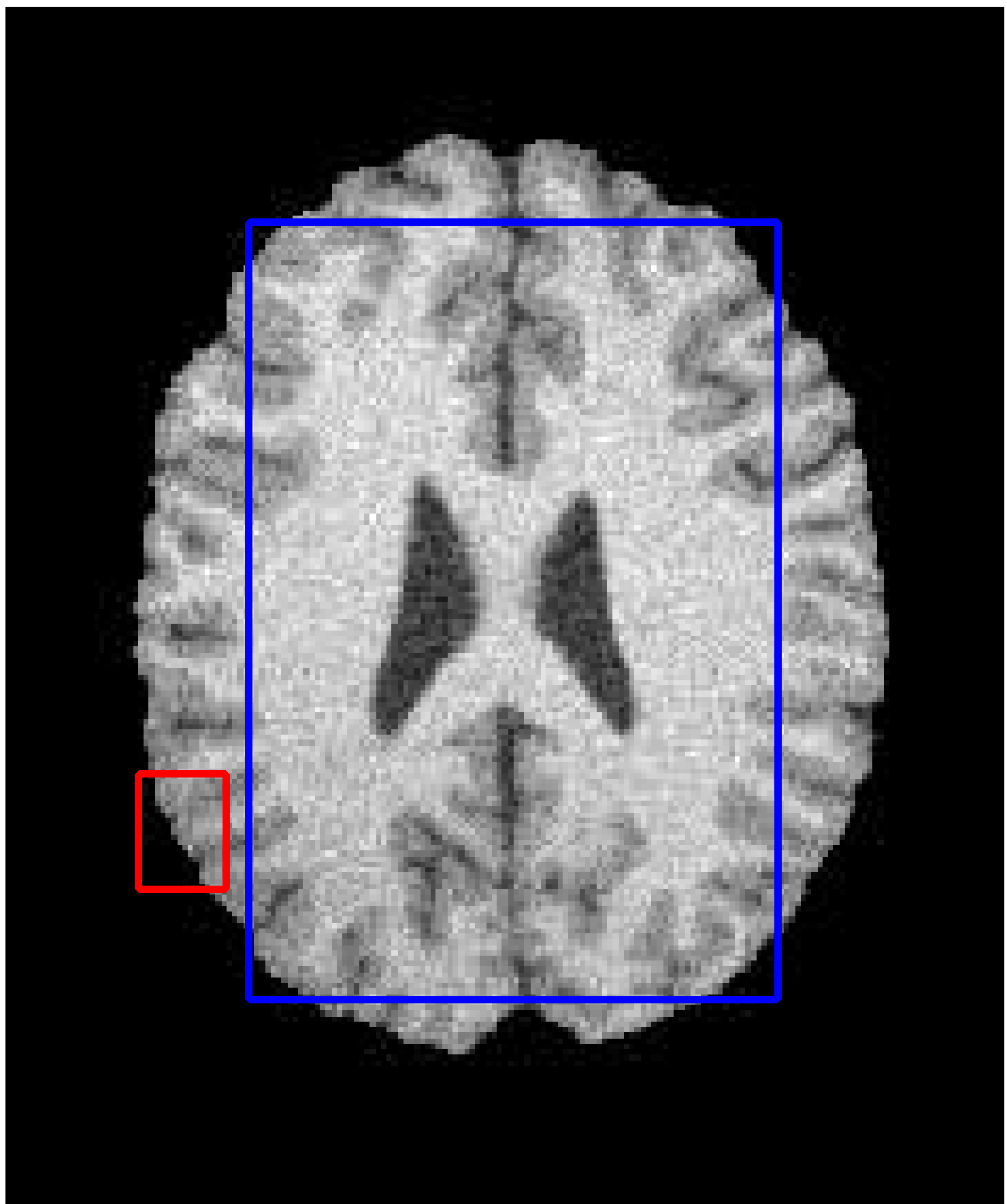}
  \end{subfigure}
  \hfill
  \begin{subfigure}[b]{0.16\linewidth}
      \centering
      \includegraphics[width=\textwidth]{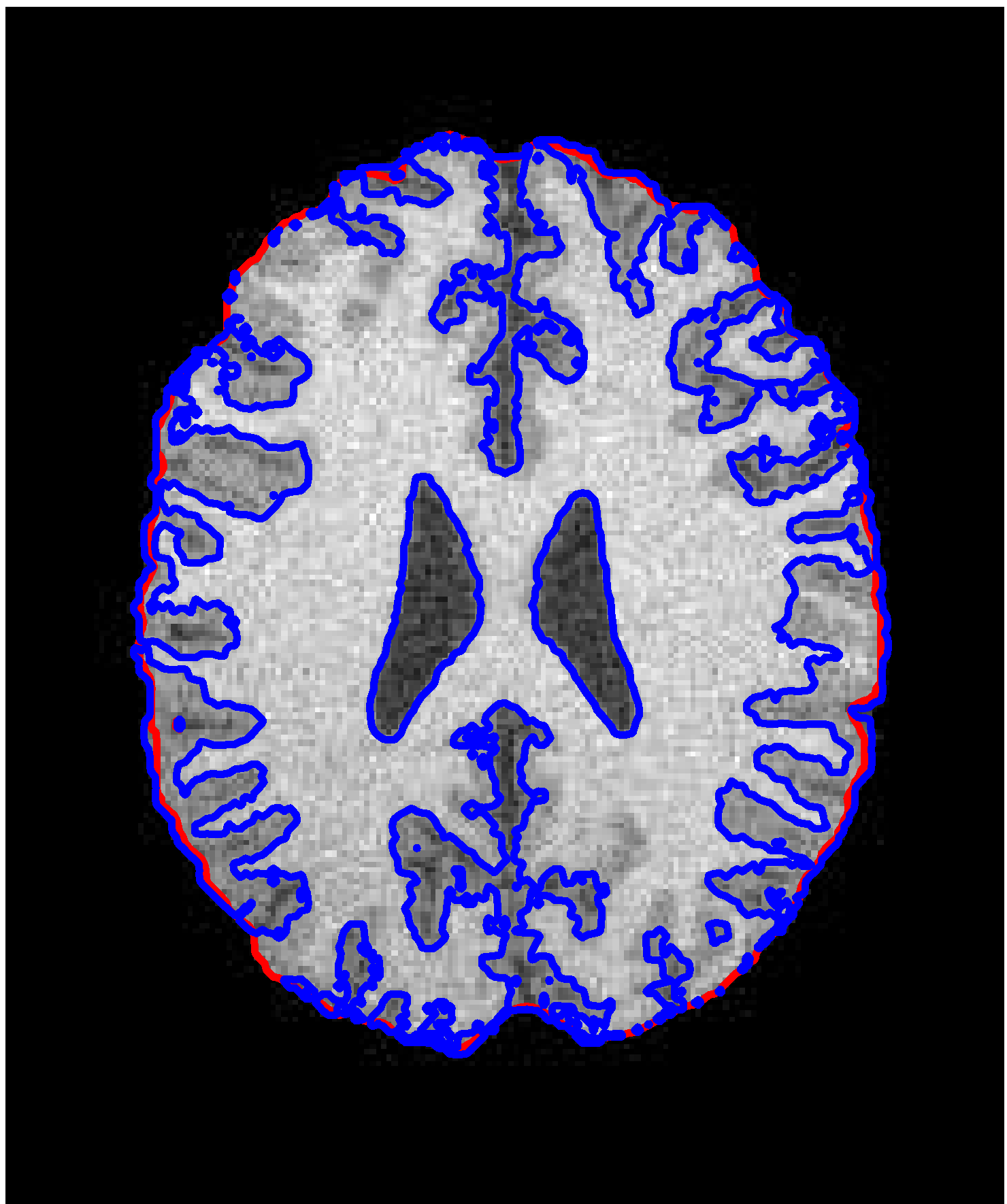}
  \end{subfigure}
  \hfill
  \begin{subfigure}[b]{0.16\linewidth}
      \centering
      \includegraphics[width=\textwidth]{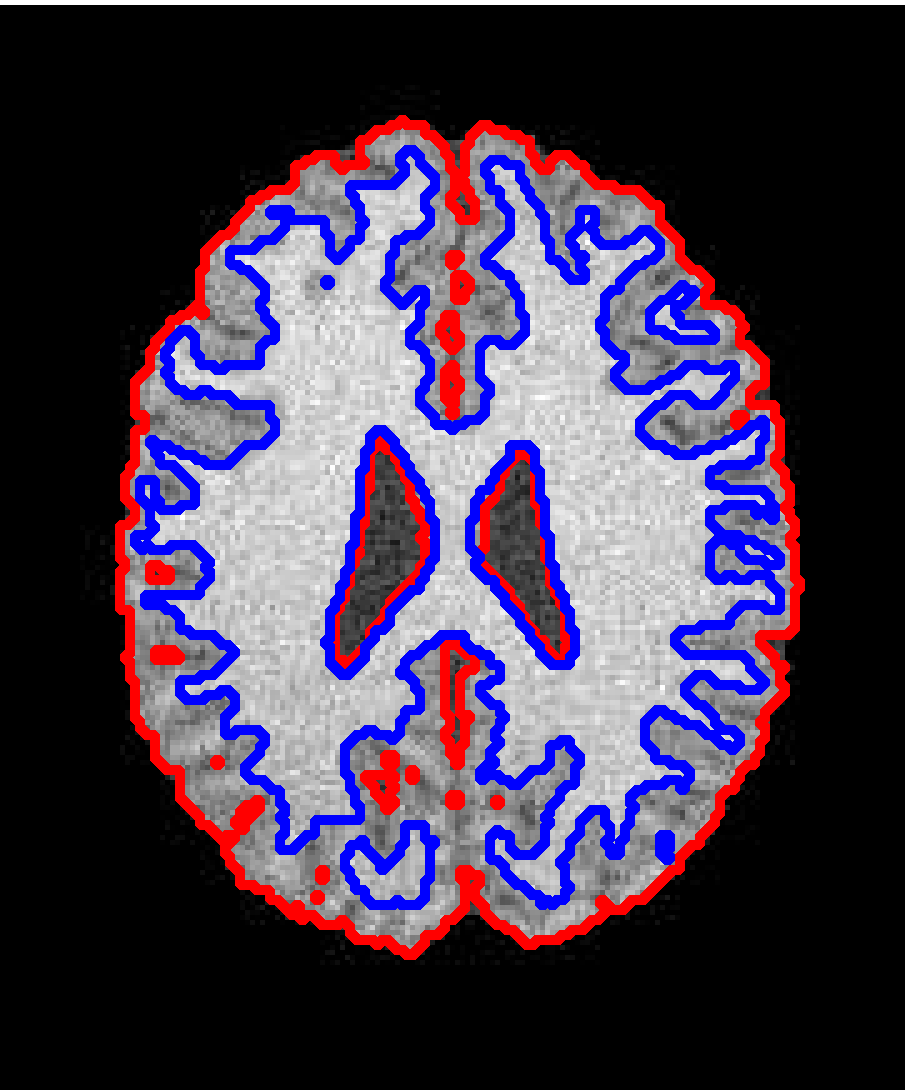}
  \end{subfigure}
 \hfill
  \begin{subfigure}[b]{0.16\linewidth}
      \centering
      \includegraphics[width=\textwidth]{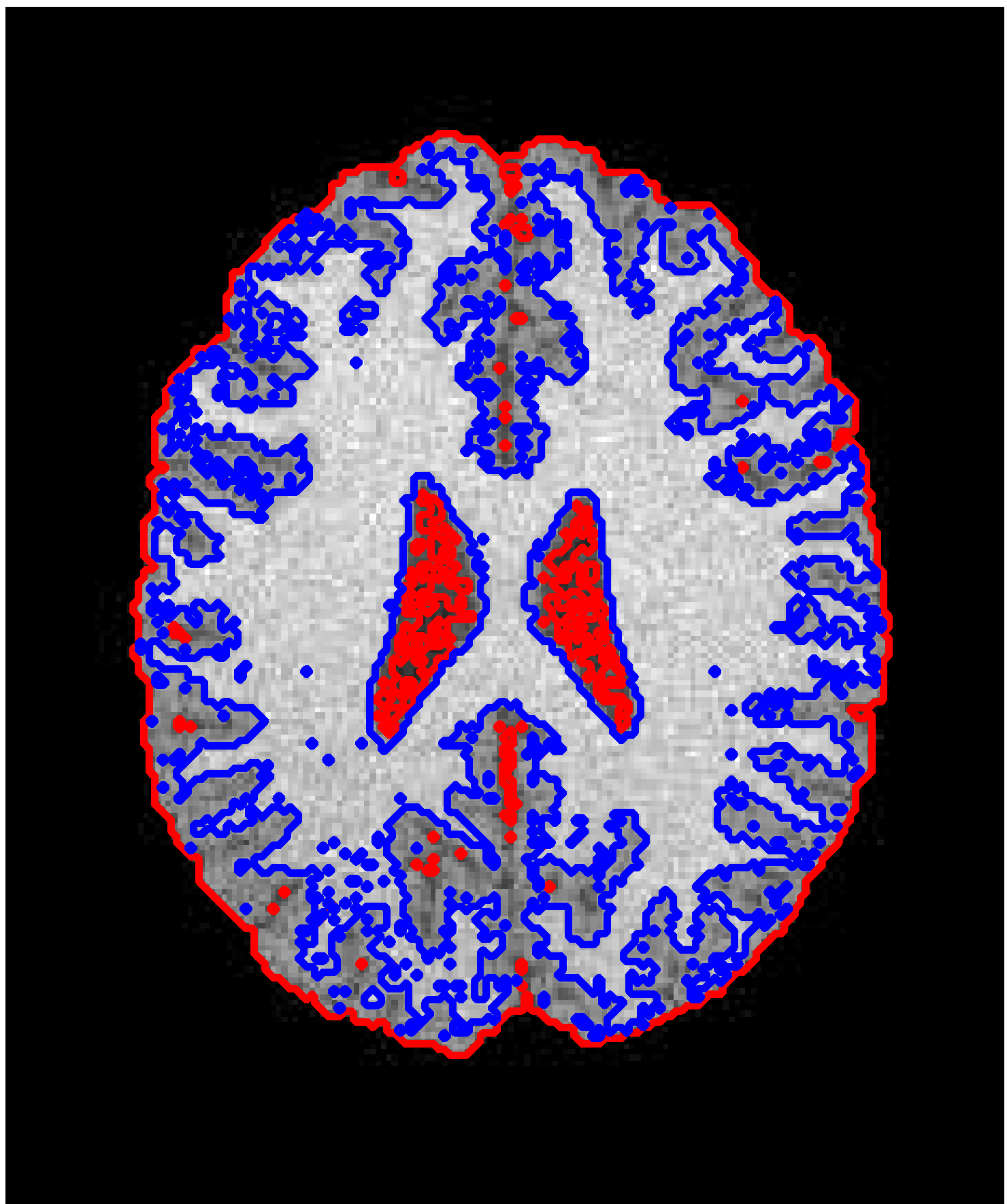}
  \end{subfigure}
  \hfill
  \begin{subfigure}[b]{0.16\linewidth}
      \centering
      \includegraphics[width=\textwidth]{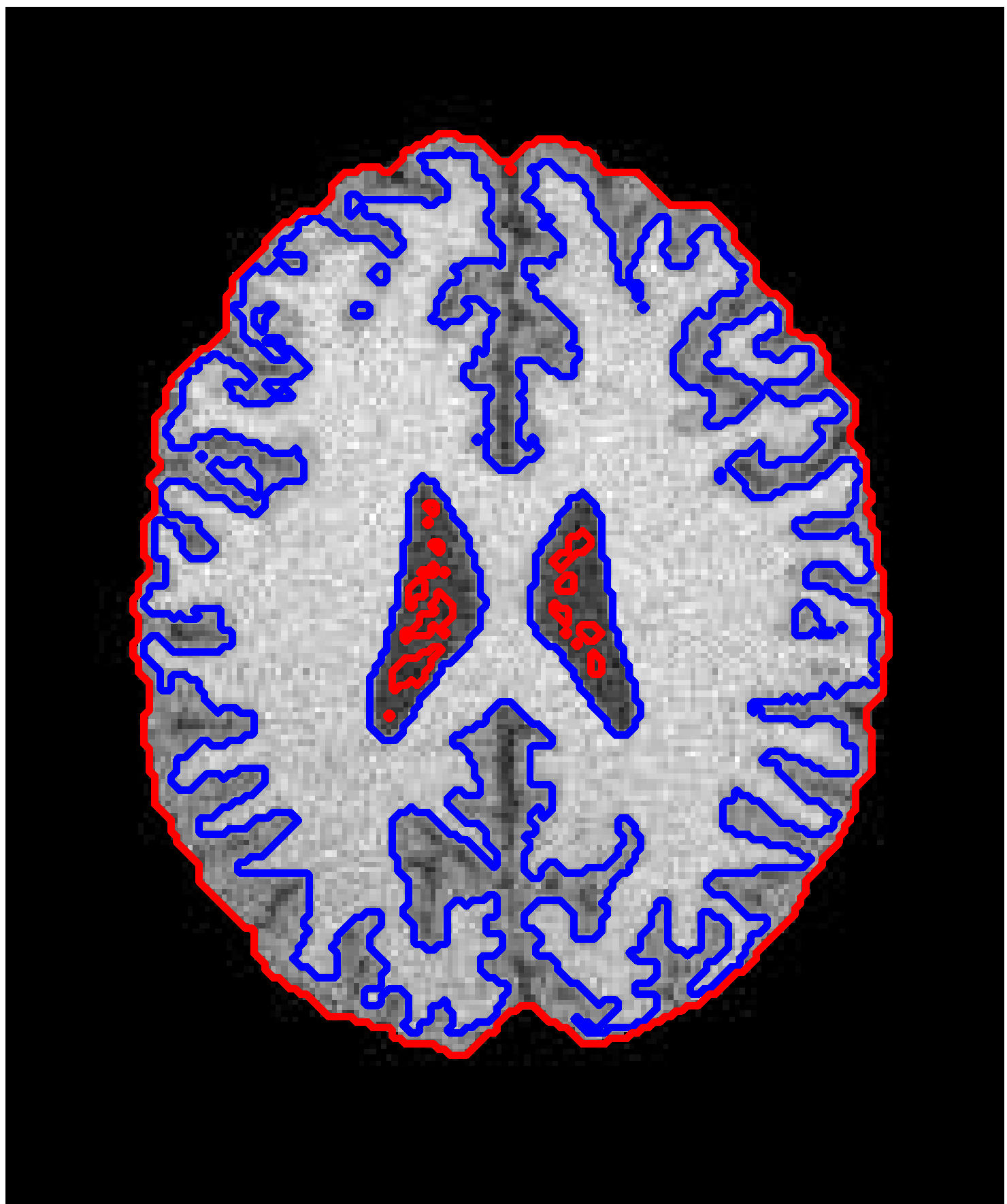}
  \end{subfigure}
 \hfill
  \begin{subfigure}[b]{0.16\linewidth}
      \centering
      \includegraphics[width=\textwidth]{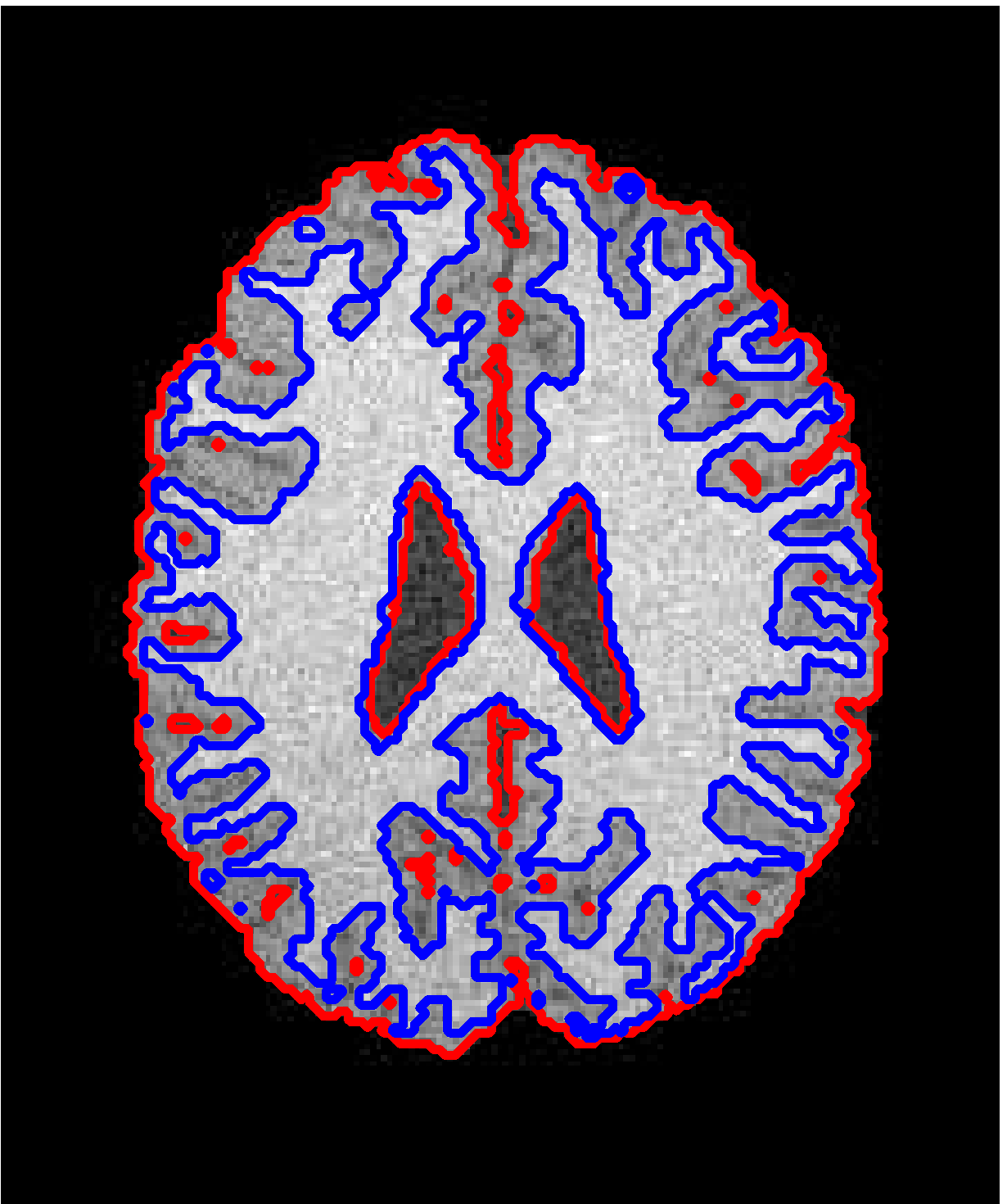}
  \end{subfigure}

 \begin{subfigure}[b]{0.16\linewidth}
      \centering
      \includegraphics[width=\textwidth]{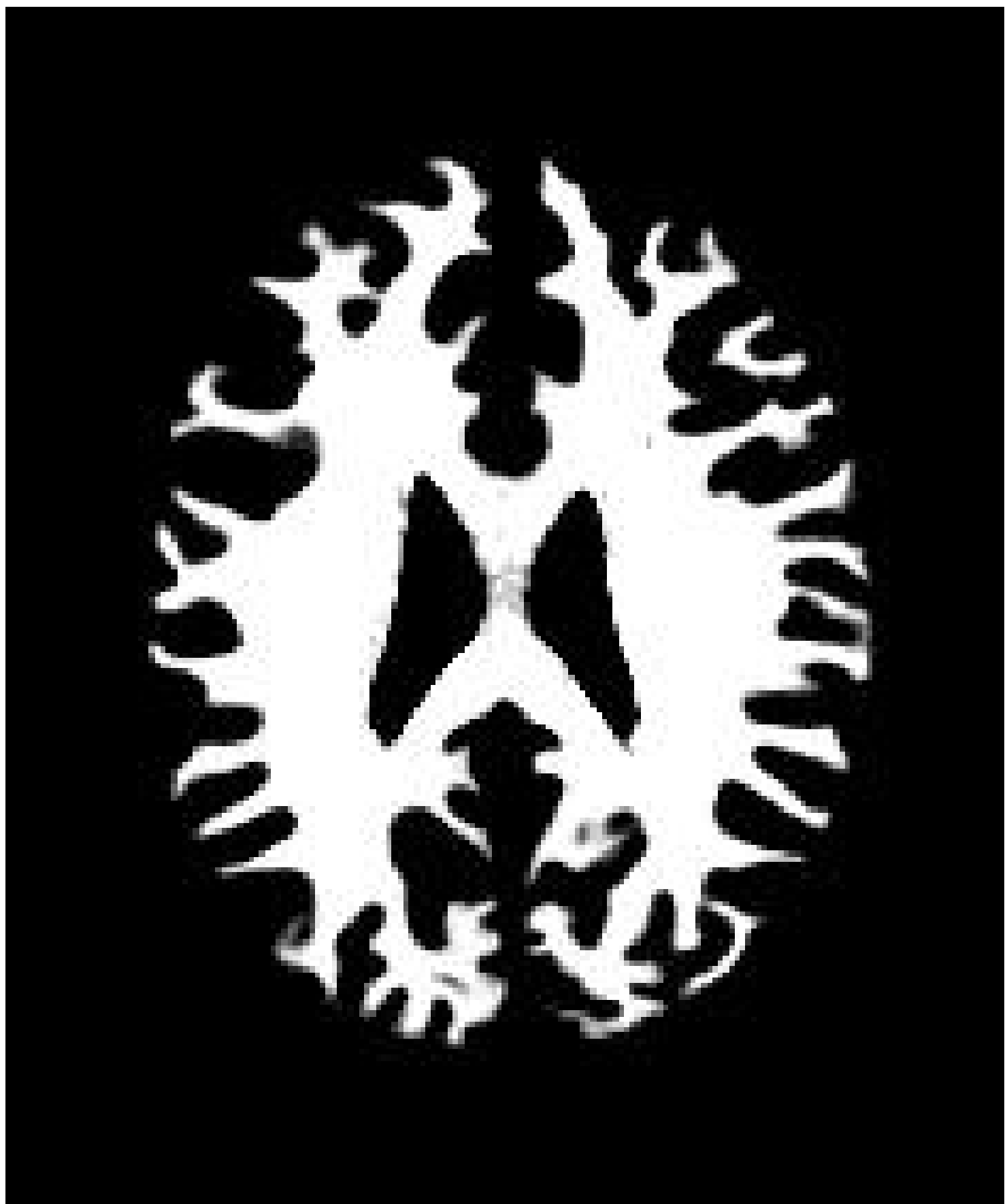}
  \end{subfigure}
   \hfill
  \begin{subfigure}[b]{0.16\linewidth}
      \centering
      \includegraphics[width=\textwidth]{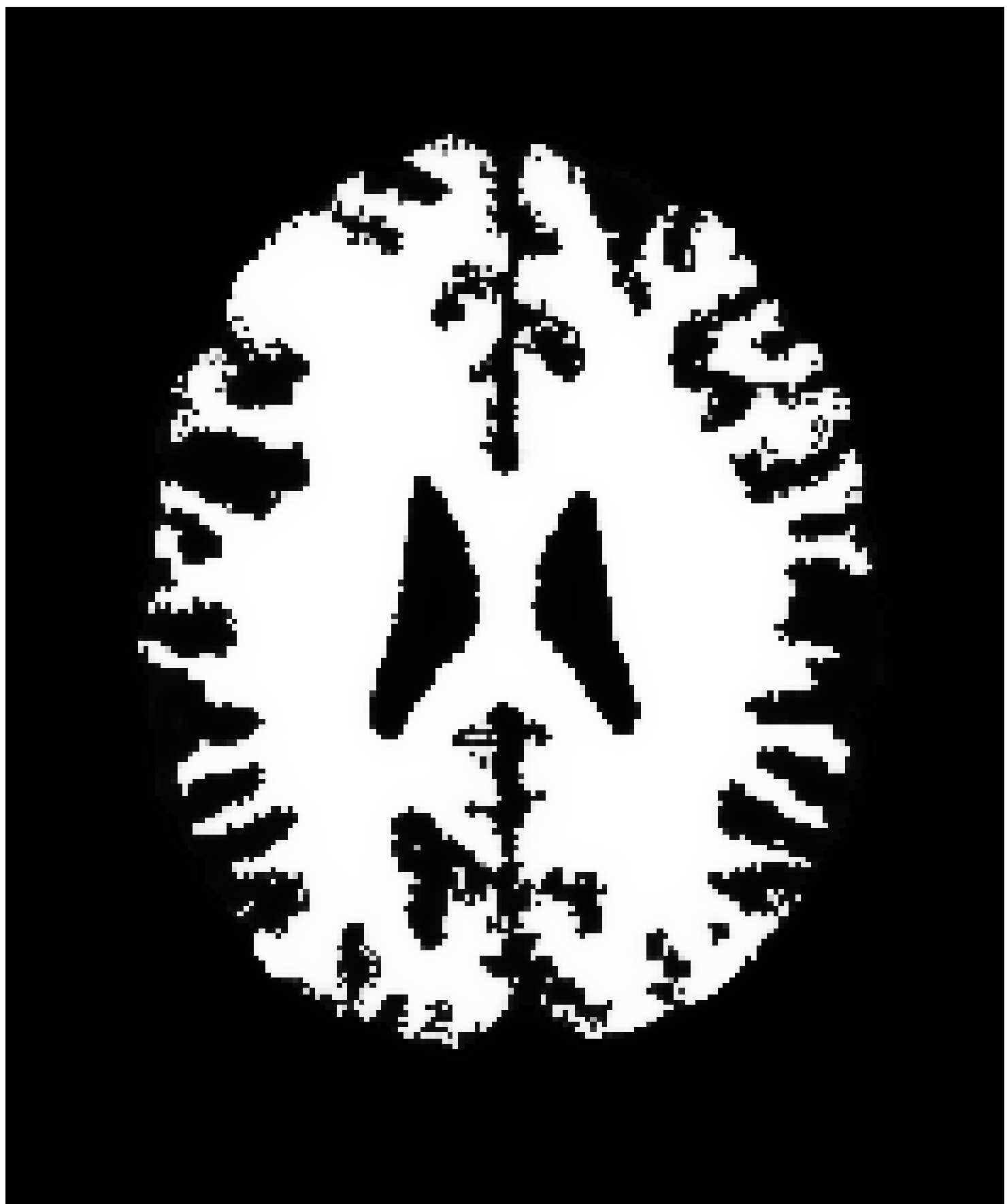}
  \end{subfigure}
  \hfill
  \begin{subfigure}[b]{0.16\linewidth}
      \centering
      \includegraphics[width=\textwidth]{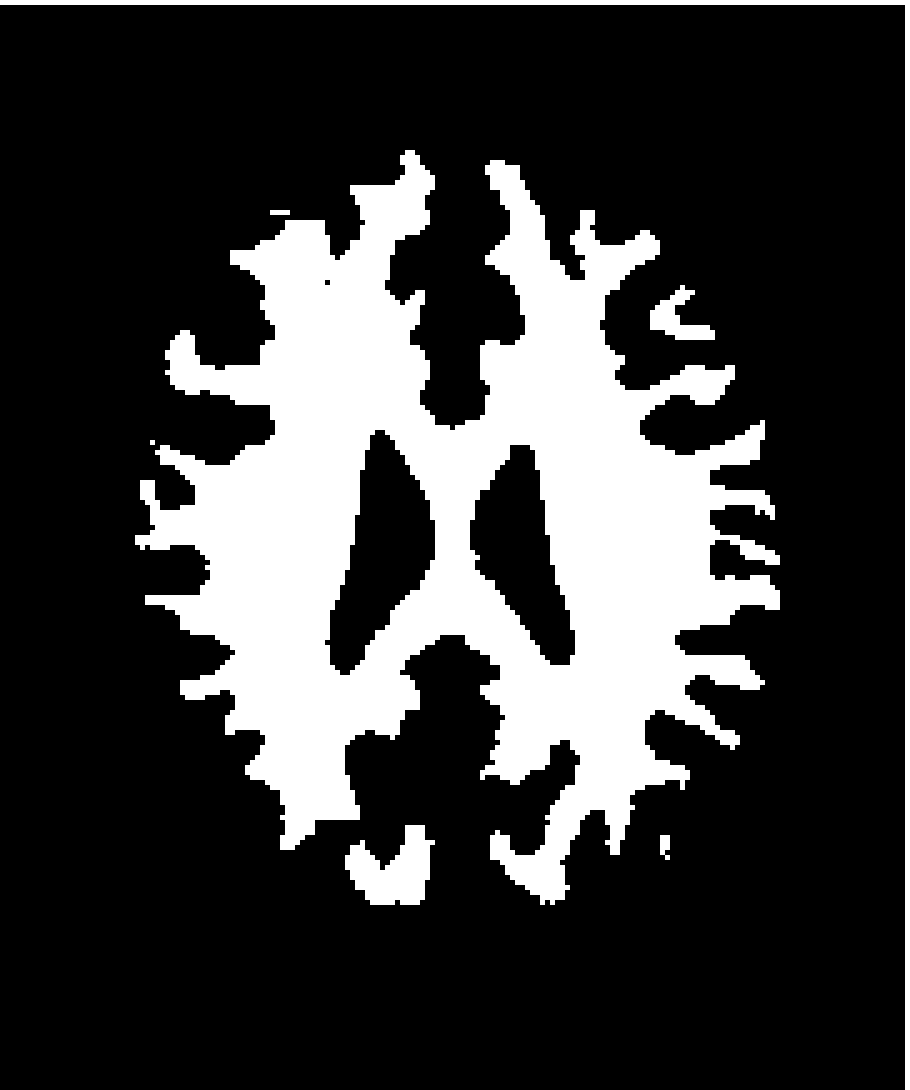}
  \end{subfigure}
   \hfill
  \begin{subfigure}[b]{0.16\linewidth}
      \centering
      \includegraphics[width=\textwidth]{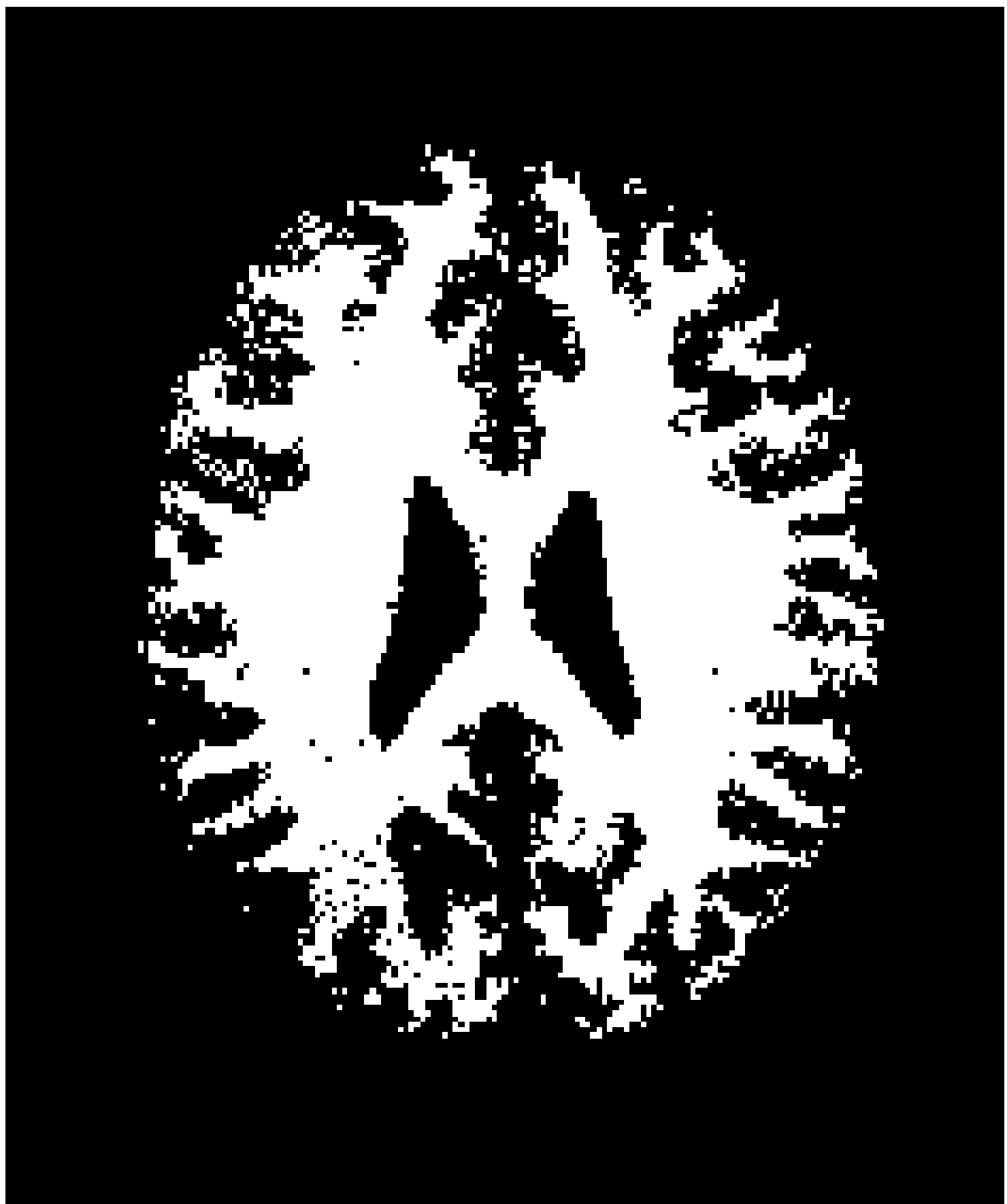}
  \end{subfigure}
  \hfill
  \begin{subfigure}[b]{0.16\linewidth}
      \centering
      \includegraphics[width=\textwidth]{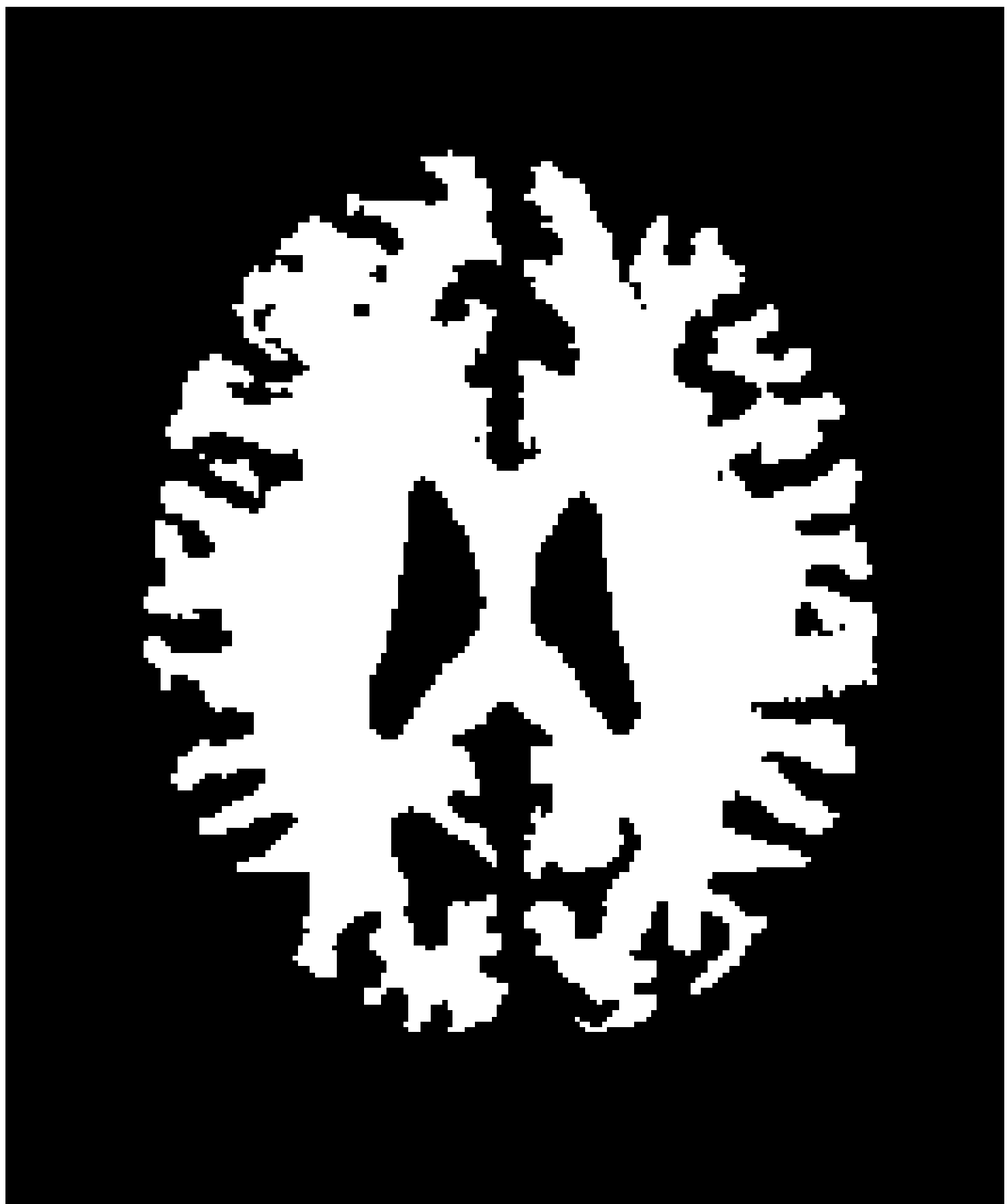}
  \end{subfigure}
   \hfill
  \begin{subfigure}[b]{0.16\linewidth}
      \centering
      \includegraphics[width=\textwidth]{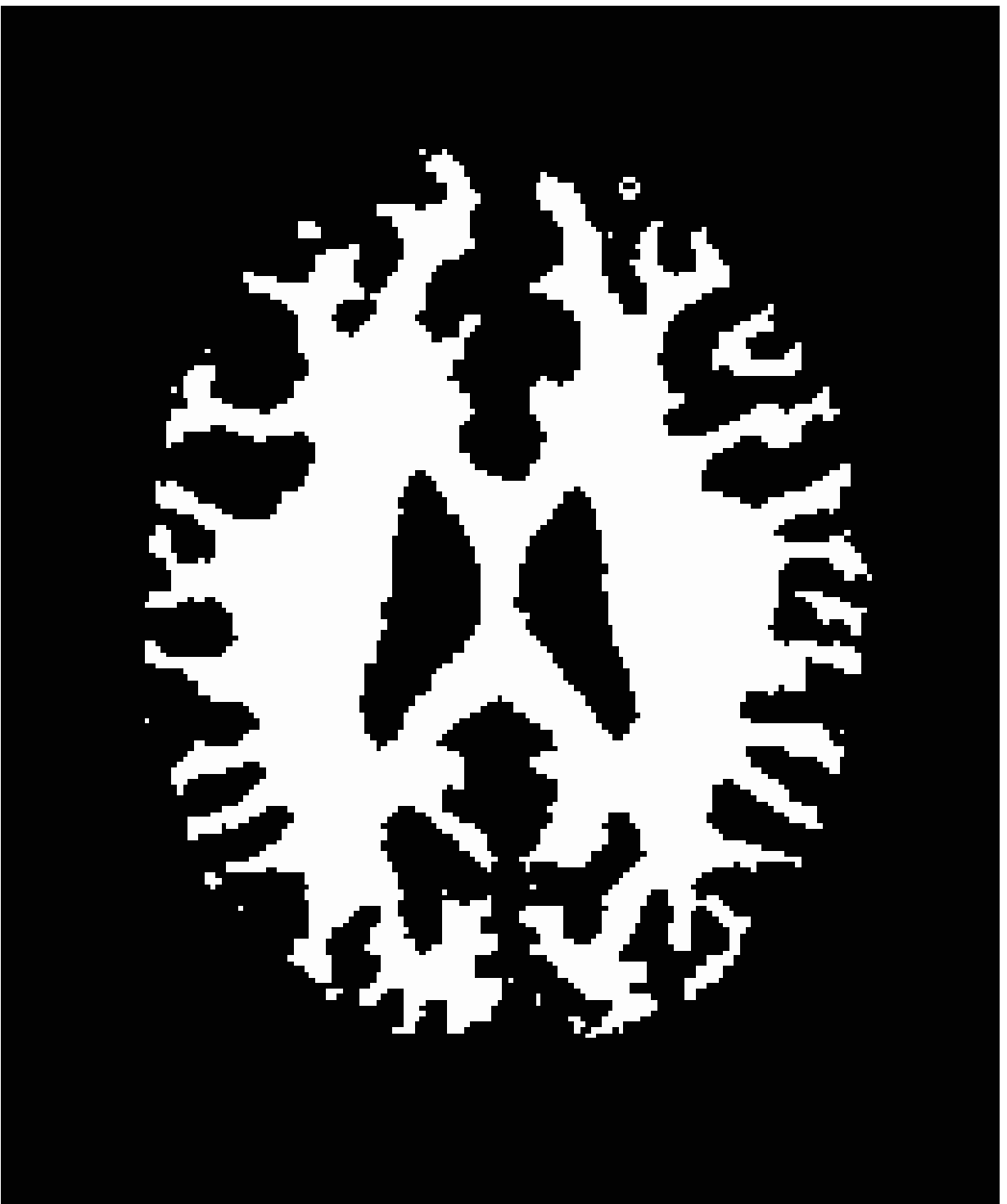}
  \end{subfigure}

 \begin{subfigure}[b]{0.16\linewidth}
      \centering
      \includegraphics[width=\textwidth]{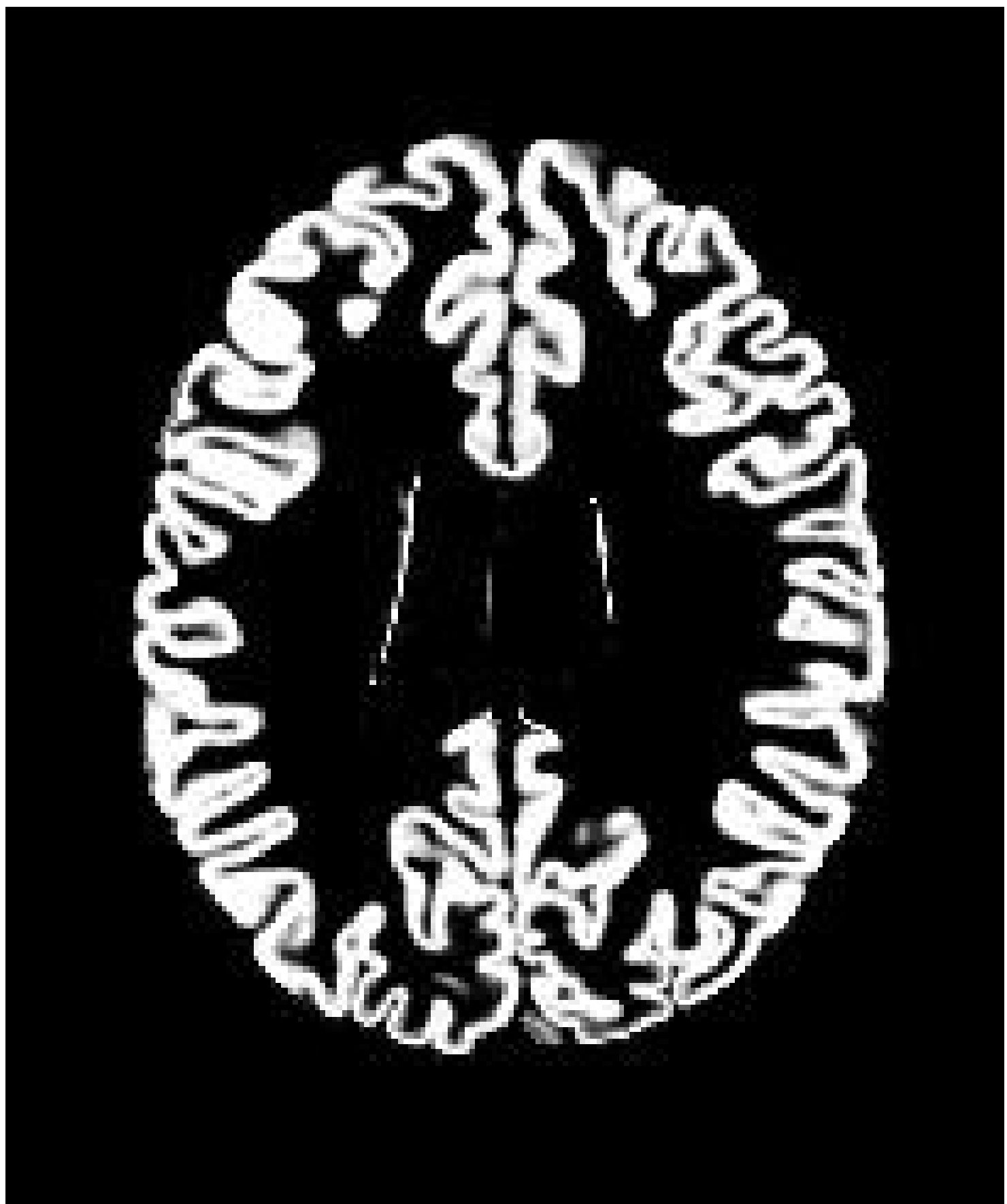}
       \caption{}
       \label{fig:98-initial}
  \end{subfigure}
   \hfill
  \begin{subfigure}[b]{0.16\linewidth}
      \centering
      \includegraphics[width=\textwidth]{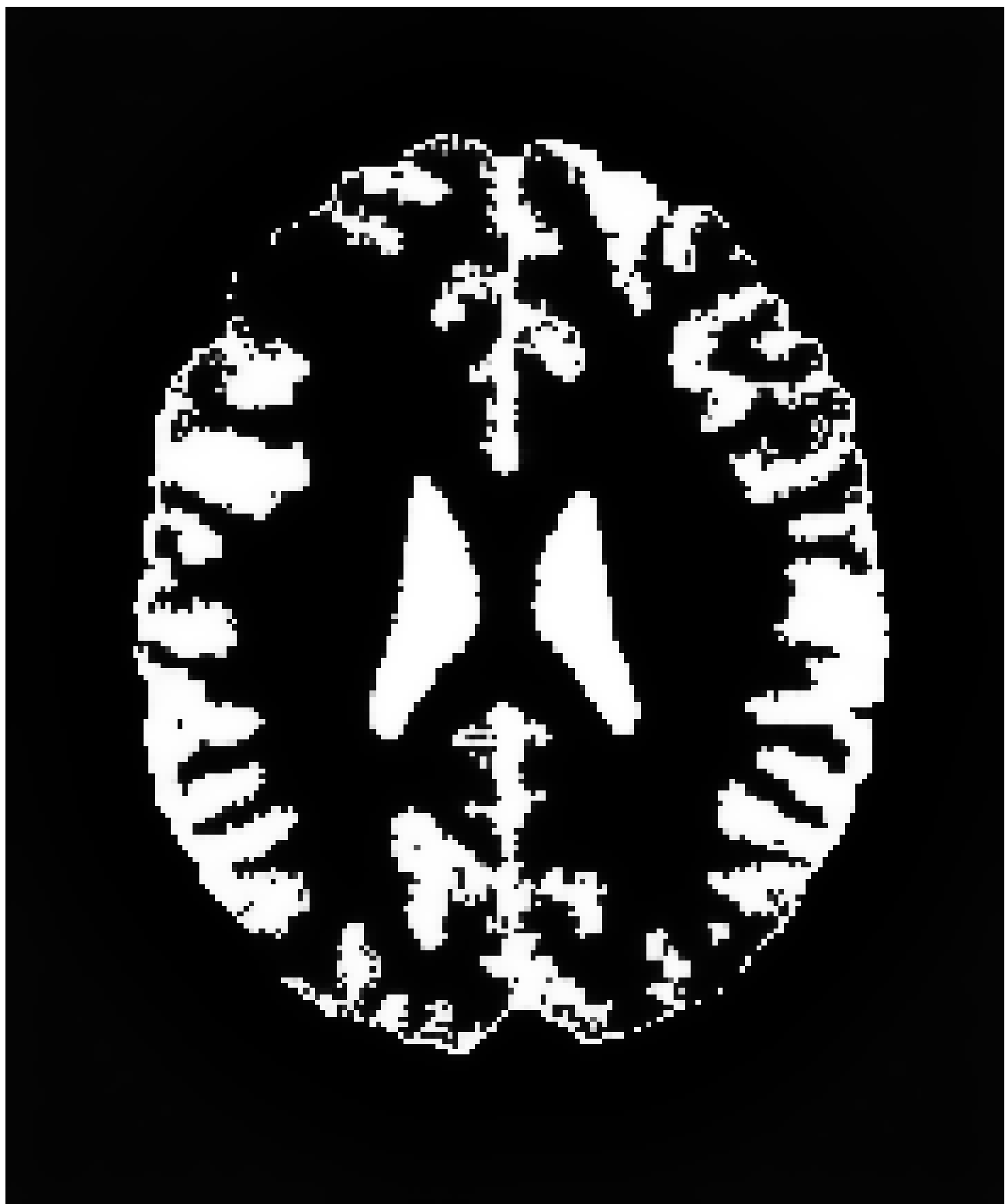}
       \caption{}
  \end{subfigure}
  \hfill
  \begin{subfigure}[b]{0.16\linewidth}
      \centering
      \includegraphics[width=\textwidth]{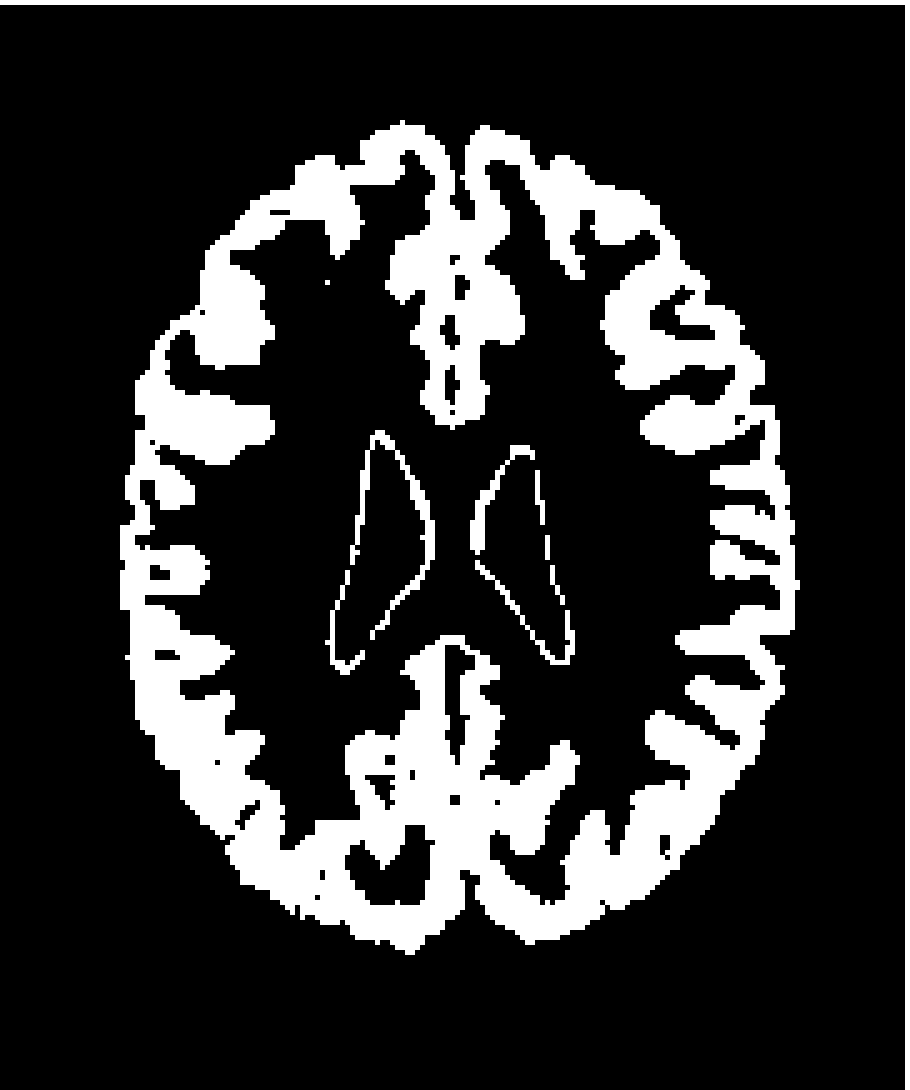}
       \caption{}
  \end{subfigure}
   \hfill
  \begin{subfigure}[b]{0.16\linewidth}
      \centering
      \includegraphics[width=\textwidth]{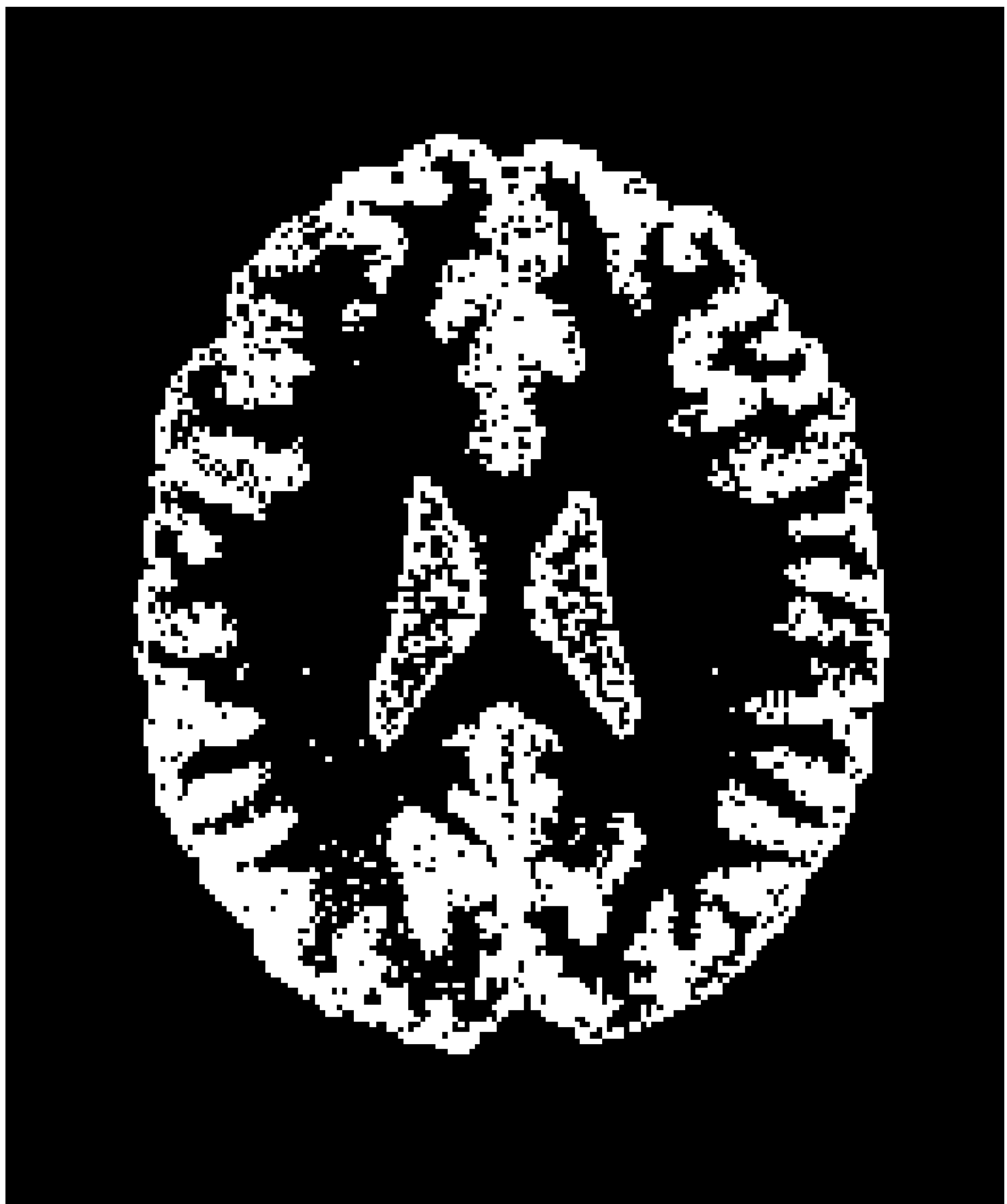}
       \caption{}
  \end{subfigure}
  \hfill
  \begin{subfigure}[b]{0.16\linewidth}
      \centering
      \includegraphics[width=\textwidth]{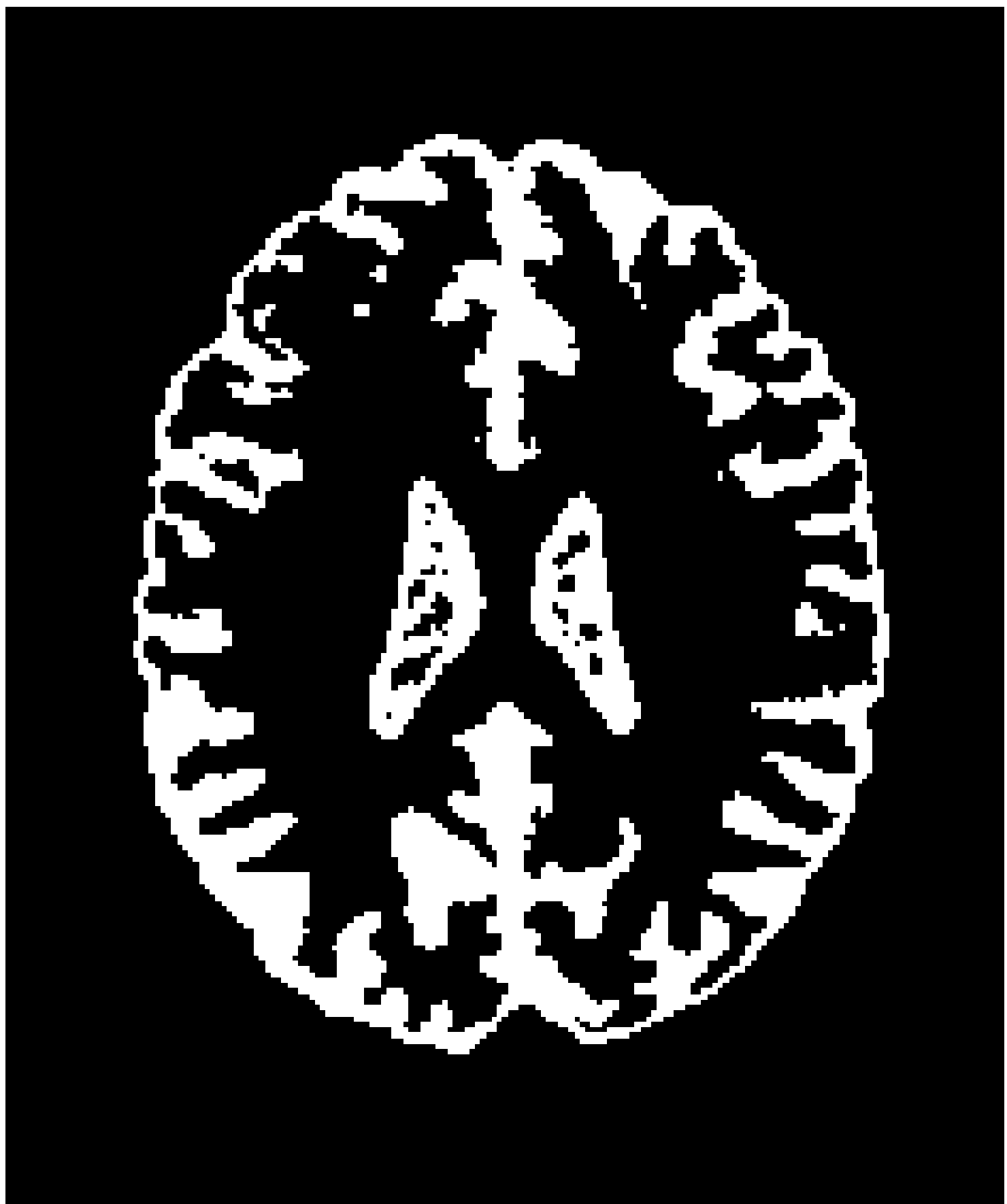}
       \caption{}
  \end{subfigure}
   \hfill
  \begin{subfigure}[b]{0.16\linewidth}
      \centering
      \includegraphics[width=\textwidth]{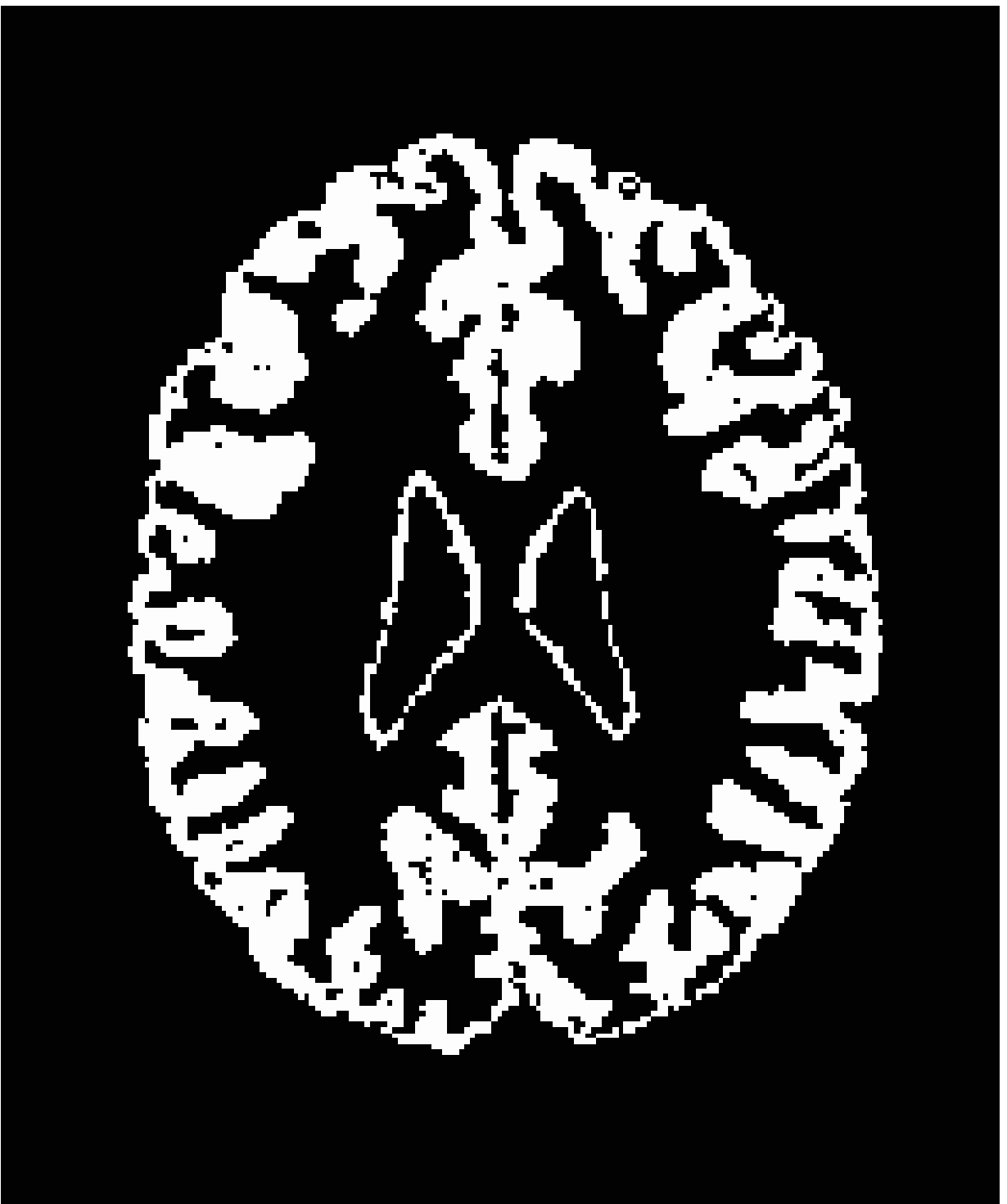}
       \caption{}
  \end{subfigure}

  \caption{Comparison with LIC, TSS, ICTM-CV, ICTM-LVF-CV models on brain MR image. Column 1: the input images with initial contours and the ground truth; Column 2--6: the segmentation results of the LIC, TSS, ICTM-CV, ICTM-LVF-CV model and our model}
  \label{fig:93}
\end{figure}
\begin{table}
    \caption{DSC, IoU and Accuracy values of different segmentation models for brain MR images}
    \label{tab:brain-result}
    \centering
    \begin{tabularx}{\textwidth}{llXXXXXX}
    \toprule
       \multirow{2}{*}{Image} & \multirow{2}{*}{Model} & \multicolumn{2}{l}{DSC} & \multicolumn{2}{l}{IoU} & \multicolumn{2}{l}{Accuracy} \\
       \cmidrule(lr){3-4} \cmidrule(lr){5-6}  \cmidrule(lr){7-8}
        & & WM & GM & WM & GM & WM & GM \\
       \midrule
       \multirow[t]{4}{*}{Fig.~\ref{fig:76-initial}}
       & LIC & 0.7860 & 0.6145 & 0.6475 & 0.4435 & 0.8882 & 0.8302 \\
       & TSS & 0.8807 & 0.8536 & 0.7868 & 0.7447 & 0.9491 & 0.9281 \\
       & ICTM-CV & 0.8428 & 0.7343 & 0.7283 & 0.5801 & 0.9219 & 0.8803 \\
        & ICTM-LVF-CV & 0.8445 & 0.3999 & 0.7308 & 0.2499 & 0.8835 & 0.6650 \\
       & Our & \textbf{0.9308} & \textbf{0.8948} & \textbf{0.8706} & \textbf{0.8096} & \textbf{0.9693} & \textbf{0.9495} \\
        \multirow[t]{4}{*}{Fig.~\ref{fig:90-initial} }
       & LIC & 0.8455 & 0.6617 & 0.7323 & 0.4944 & 0.9190 & 0.8703 \\
       & TSS & 0.9074 & 0.8352 & 0.8305 & 0.7171 & 0.9562 & 0.9357 \\
       & ICTM-CV & 0.8930 & 0.7408 & 0.8067 & 0.5883 & 0.9447 & 0.9018 \\
       & ICTM-LVF-CV & 0.8599 & 0.5110 & 0.7542 & 0.3432 & 0.9041 & 0.7479 \\
       & Our & \textbf{0.9498} & \textbf{0.8861} & \textbf{0.9044} & \textbf{0.7955} & \textbf{0.9760} & \textbf{0.9551} \\
       \multirow[t]{4}{*}{Fig.~\ref{fig:93-initial}}
       & LIC & 0.8600 & 0.5744 & 0.7544 & 0.4030 & 0.9210 & 0.8644 \\
       & TSS & 0.9169 & 0.8250 & 0.8466 & 0.7021 & 0.9589 & 0.9409 \\
       & ICTM-CV & 0.9036 & 0.7185 & 0.8242 & 0.5607 & 0.9492 & 0.9037 \\
       & ICTM-LVF-CV & 0.8682 & 0.4410 & 0.7671 & 0.2829 & 0.9075 & 0.7408 \\
       & Our & \textbf{0.9556} & \textbf{0.8927} & \textbf{0.9149} & \textbf{0.8062} & \textbf{0.9784} & \textbf{0.9626} \\
       \multirow[t]{4}{*}{Fig.~\ref{fig:98-initial}}
       & LIC & 0.8572 & 0.6318 & 0.7501 & 0.4618 & 0.9240 & 0.8789 \\
       & TSS & 0.9221 & 0.8492 & 0.8554 & 0.7379 & 0.9643 & 0.9444 \\
       & ICTM-CV & 0.9074 & 0.7738 & 0.8304 & 0.6311 & 0.9545 & 0.9189 \\
       & ICTM-LVF-CV & 0.8512 & 0.5780 & 0.7410 & 0.4064 & 0.9051 & 0.7972 \\
       & Our & \textbf{0.9530} & \textbf{0.8868} & \textbf{0.9102} & \textbf{0.7966} & \textbf{0.9781} & \textbf{0.9588} \\
    \bottomrule
    \end{tabularx}
\end{table}

To verify the effectiveness of our model on multi-phase segmentation, we perform three-phase segmentation on MR images, dividing the images into WM, GM and background. We compare its performance with the LIC, TSS, ICTM-CV and ICTM-LVF-CV models. The experimental results are shown in Fig.~\ref{fig:90}--\ref{fig:93}. Table~\ref{tab:brain-result} reports the DSC, IoU and Accuracy values of the segmentation results for the compared models. Our model and the ICTM-CV model better preserve the structural details of GM and WM in the segmentation results. However, the segmentation boundaries produced by the LIC and ICTM-CV models are severely affected by noise. Additionally, due to the intensity inhomogeneity of MR images, where GM and WM have similar intensity values, the TSS model tends to produce inaccurate segmentation results. In contrast, our model demonstrates greater robustness to both noise and intensity inhomogeneity, leading to better segmentation results. Table~\ref{tab:brain-result} demonstrates that our model achieves better quantitative performance.

\section{Conclusions}
\label{sec:conclusions}
In this paper, we proposed a robust variational segmentation model for images with high noise and intensity inhomogeneity. To reduce the impact of Poisson noise and multiplicative Gamma noise, we integrated the I-divergence term into the segmentation energy. The proposed model performs three tasks simultaneously: denoising, bias correction, and segmentation. To efficiently solve this model, the subproblems were solved using the ICTM and RMSAV algorithms. Moreover,  we relaxed the model and theoretically proved that the relaxed model is equivalent to the original model in terms of solutions, to solve for the characteristic function $u_i$ representing the segmented regions. Segmentation experiments on both synthetic and real images demonstrate the effectiveness and robustness to noise of our model.

\begin{acknowledgements}
The authors would like to thank Dong Wang for providing the reference code of the ICTM algorithm and Chunming Li for providing the reference code of the LIC model. The authors also wish to express their sincere gratitude to the anonymous reviewers for their insightful comments and constructive suggestions, which helped improve the quality of this paper.
\end{acknowledgements}

%
\section*{Declarations}
\subsection*{Funding}
\sloppy
This work is partially supported by the National Natural Science Foundation of China (12301536,
 U21B2075, 12171123, 12371419, 12271130), the Natural Science Foundation of Heilongjiang Province
 (ZD2022A001), the Fundamental Research Funds for the Central Universities (HIT.NSRIF202302,
2022FRFK060020, 2022FRFK060031, 2022FRFK060014).

\subsection*{Data availability} The data of all images and codes involved in this paper are available from the corresponding author upon reasonable request.

\subsection*{Conflict of interest} The authors declare that they have no conflict of interest.

\bibliographystyle{spmpsci}      
\bibliography{ref}   
%
%
%

\end{document}